\documentclass{cspmA}
\usepackage{float}
\usepackage[backref=page]{hyperref} 
\usepackage{times}
\usepackage{epsfig}
\usepackage{graphicx}
\usepackage{amsmath, amsthm, amssymb}
\usepackage{mathrsfs}
\usepackage[chapter]{algorithm}
\usepackage{algorithmic}
\usepackage{natbib}
\usepackage{nicefrac}
\usepackage{enumerate}
\usepackage{thmtools}
\usepackage{fancyvrb}
\usepackage[nolist]{acronym}
\usepackage{caption}
\usepackage{makeidx}
\usepackage{comment}

\newcommand{\commentAlt}[1]{\ignorespaces}
\newcommand{\commentLongAlt}[1]{\ignorespaces}

\renewcommand{\Pr}{\field{P}}
\DeclareMathOperator{\Tr}{Tr}

\newcommand{\bb}{\boldsymbol{b}}
\newcommand{\bc}{\boldsymbol{c}}
\newcommand{\bd}{\boldsymbol{d}}
\newcommand{\be}{\boldsymbol{e}}
\newcommand{\bg}{\boldsymbol{g}}
\newcommand{\bh}{\boldsymbol{h}}
\newcommand{\bj}{\boldsymbol{j}}

\newcommand{\bm}{\boldsymbol{m}}
\newcommand{\bn}{\boldsymbol{n}}
\newcommand{\bp}{\boldsymbol{p}}
\newcommand{\bq}{\boldsymbol{q}}
\newcommand{\bs}{\boldsymbol{s}}
\newcommand{\bx}{\boldsymbol{x}}
\newcommand{\bu}{\boldsymbol{u}}
\newcommand{\by}{\boldsymbol{y}}

\newcommand{\bA}{\boldsymbol{A}}
\newcommand{\bB}{\boldsymbol{B}}
\newcommand{\bC}{\boldsymbol{C}}

\newcommand{\bG}{\boldsymbol{G}}

\newcommand{\bI}{\boldsymbol{I}}

\newcommand{\bM}{\boldsymbol{M}}

\newcommand{\bQ}{\boldsymbol{Q}}

\newcommand{\bS}{\boldsymbol{S}}

\newcommand{\bX}{\boldsymbol{X}}

\newcommand{\bxi}{\boldsymbol{\xi}}

\newcommand{\bz}{\boldsymbol{z}}
\newcommand{\bZ}{\boldsymbol{Z}}

\newcommand{\bw}{\boldsymbol{w}}

\newcommand{\bU}{\boldsymbol{U}}
\newcommand{\bv}{\boldsymbol{v}}

\newcommand{\bbeta}{\boldsymbol{\beta}}

\newcommand{\blambda}{\boldsymbol{\lambda}}

\newcommand{\bnu}{\boldsymbol{\nu}}
\newcommand{\bpi}{\boldsymbol{\pi}}
\newcommand{\bSigma}{\boldsymbol{\Sigma}}
\newcommand{\btheta}{\boldsymbol{\theta}}
\newcommand{\bTheta}{\boldsymbol{\Theta}}

\newcommand{\argmin}{\mathop{\mathrm{argmin}}}
\newcommand{\argmax}{\mathop{\mathrm{argmax}}}
\newcommand{\conv}{\mathop{\mathrm{conv}}}
\newcommand{\interior}{\mathop{\mathrm{int}}}
\newcommand{\dom}{\mathop{\mathrm{dom}}}
\newcommand{\bdry}{\mathop{\mathrm{bdry}}}

\newcommand{\field}[1]{\mathbb{#1}}

\newcommand{\R}{\field{R}}
\newcommand{\Nat}{\field{N}}
\newcommand{\E}{\field{E}}
\newcommand{\Var}{\mathrm{Var}}

\newcommand{\diag}[1]{\mbox{\rm diag}\!\left\{{#1}\right\}}

\newcommand{\wh}{\widehat}

\newcommand{\sign}{{\rm sign}}

\DeclareMathOperator{\SARegret}{SA-Regret}

\DeclareMathOperator{\DRegret}{D-Regret}
\DeclareMathOperator{\TRegret}{T-Regret}
\DeclareMathOperator{\PRegret}{P-Regret}
\DeclareMathOperator{\Regret}{Regret}

\DeclareMathOperator{\Wealth}{Wealth}

\DeclareMathOperator{\Risk}{Risk}
\DeclareMathOperator{\Prox}{Prox}
\DeclareMathOperator{\Type}{\mathcal{T}}
\DeclareMathOperator{\KL}{KL}
\DeclareMathOperator{\entropy}{H}
\DeclareMathOperator{\KLBern}{KL_\text{Bern}}
\DeclareMathOperator{\dist}{dist}

\DeclareMathOperator{\Gen}{Gen}
\DeclareMathOperator{\barGen}{\overline{Gen}}

\newcommand{\indicator}{\iota}
\newcommand{\indevent}{\mathbf{1}}
\newcommand{\ones}{\mathbf{1}}

\usepackage{pgfplots}
\pgfplotsset{compat=1.18}
\usepgfplotslibrary{groupplots}
\pgfplotsset{compat=newest}
\usepgfplotslibrary{fillbetween}
\usetikzlibrary{arrows.meta,patterns, patterns.meta}
\usetikzlibrary{calc}

\newtheorem{theorem}{Theorem}[chapter]
\newtheorem{lemma}[theorem]{Lemma}
\newtheorem{corollary}[theorem]{Corollary}
\newtheorem{remark}[theorem]{Remark}
\newtheorem{proposition}[theorem]{Proposition}
\newtheorem{example}[theorem]{Example}
\newtheorem{definition}[theorem]{Definition}
\newtheorem{exer}{Problem}[chapter]

\title{Online Learning: A Modern Introduction Using Convex Optimization}
\author{Francesco Orabona\\
KAUST\\
{\tt\small francesco@orabona.com}
}

\makeindex

\index{AA|see{Aggregating Algorithm}}
\index{WAA|see{Weighted Average algorithm}}
\index{KT|see{Krichevsky--Trofimov}}
\index{KL|see{Kullback--Leibler divergence}}
\index{LEA|see{learning with expert advice}}
\index{ONS|see{Online Newton Step algorithm}}
\index{OCO|see{online convex optimization}}
\index{OLO|see{online linear optimization}}
\index{FTRL|see{Follow-the-Regularized-Leader algorithm}}
\index{OMD|see{Online Mirror Descent algorithm}}
\index{OGD|see{Online Gradient Descent algorithm}}
\index{OSD|see{Online Subgradient Descent algorithm}}
\index{learning rate|seealso{stepsize}}
\index{stepsize|seealso{learning rate}}
\index{comparator|seealso{competitor}}
\index{competitor|seealso{comparator}}
\index{valid set|seealso{feasible set}}
\index{feasible set|seealso{valid set}}
\index{duality map|seealso{mirror map}}
\index{mirror map|seealso{duality map}}

\begin{document}
\maketitle

\begin{acronym}[FTRL] 
\acro{FTRL}[FTRL]{Follow-the-Regularized-Leader}
\acro{FTL}[FTL]{Follow-the-Leader}
\acro{OGD}[OGD]{Online Gradient Descent}
\acro{OSD}[OSD]{Online Subgradient Descent}
\acro{OMD}[OMD]{Online Mirror Descent}
\acro{AA}[AA]{Aggregating Algorithm}
\acro{WAA}[WAA]{Weighted Average Algorithm}
\acro{ONS}[ONS]{Online Newton Step}
\acro{OCO}[OCO]{Online Convex Optimization}
\acro{OLO}[OLO]{Online Linear Optimization}
\acro{LEA}[LEA]{Learning with Expert Advice}
\acro{EG}[EG]{Exponentiated Gradient}
\acro{KT}[KT]{Krichevsky--Trofimov}
\acro{Exp3}[Exp3]{Exponential-weight algorithm for Exploration and Exploitation}
\acro{ETC}[ETC]{Explore-then-Commit}
\acro{ADER}[ADER]{Adaptive learning for Dynamic EnviRonment}
\acro{UCB}[UCB]{Upper Confidence Bound}
\acro{RM}[RM]{Regret Matching}
\acro{RM+}[RM+]{Regret Matching+}
\acro{SVM}[SVM]{Support Vector Machine}
\acro{KL}[KL]{Kullback--Leibler}
\acro{ERM}[ERM]{Empirical Risk Minimization}
\acro{KKT}[KKT]{Karush--Kuhn--Tucker}
\acro{iid}[i.i.d.]{independent and identically distributed}
\acro{INF}[INF]{Implicitly Normalized Forecaster}
\acro{VC}[VC]{Vapnik--Chervonenkis}
\end{acronym}

\frontmatter

\thispagestyle{empty}

\begin{center}
\itshape
To Irene and Dante
\end{center}

\vspace*{\fill}

\tableofcontents
\setcounter{tocdepth}{2}


\chapter*{Foreword}
This book is about online learning, a mathematical framework that guides the design of algorithms that learn by interacting with their environment. This approach to machine learning differs from the more traditional statistical view, which distinguishes between a training phase, in which a learning algorithm builds a prediction model from a training set, and a testing phase, in which the trained model is evaluated on a separate test set. In online learning, these two phases are interleaved: data are received sequentially, and the algorithm must first make a prediction on each incoming data point and only then update its model using the resulting feedback. The two theories are not in competition: they look at the same problem from different angles, and researchers choose the theory that offers the tools best suited to solving their problem. Indeed, both theories tackle the same fundamental machine learning question: what is the most efficient algorithmic way of using past observations to make future predictions? However, they differ with respect to the assumptions they make about the source of the learning data. In statistical learning, data are obtained through independent draws from a fixed and unknown distribution. In online learning, data are generated sequentially by a source that need not be stochastic and may even be adversarial. Still, one can obtain standard generalization guarantees by running an online learning algorithm with suitable regret guarantees on a data sequence obtained through independent draws (a so-called online-to-batch conversion), a technique that connects the online world to the statistical one. More generally, we can run online algorithms to solve stochastic convex optimization problems efficiently, which is just one of the many powerful applications of the online framework that the reader will find in the book.

The connection to convex optimization, one of the book's main themes, operates at a fundamental level. Indeed, online learning is often described in terms of online convex optimization, a setting in which the goal is to keep the algorithm's cumulative loss close to that of a suitable comparator, while facing an arbitrary sequence of convex loss functions. Building on this connection, one can transfer many powerful ideas and techniques from convex optimization—such as gradient descent and its generalizations—to the online world. Convex analysis thus becomes a unifying language for deriving the theory behind online algorithms, and the main technical tool used throughout the book. This brings us to one of the most important notions of online learning: regret. Regret addresses the following fundamental issue: what is the right way to measure the performance of an online learning algorithm when the data come from an arbitrary source? If we assume that data are the result of independent draws from a fixed distribution, the risk of a predictor is traditionally measured in terms of its expected loss. We then choose a reference class of predictors (e.g., linear models) and seek to bound the estimation error, expressed as the difference between the predictor's risk and the risk of any comparator in the reference class. In the online world, however, there are two important differences: first, we cannot define statistical risk, because there is no statistical source generating the data; second, there is no single predictor, but rather an online algorithm outputting a sequence of predictions. To measure performance, we can then average all the losses accumulated by the algorithm when making predictions at each data point in the sequence (in contrast to statistical risk, I often call this quantity ``sequential risk''). This empirical average reflects how well the algorithm adapts to any regularity present in the data sequence. Note that by doing so, we are not measuring the performance of a static prediction model generated after training a learning algorithm on a dataset, but rather the performance of a learning algorithm run on an individual sequence of data. Using this sequential risk, regret can be viewed as the online counterpart of estimation error. In normalized form, it is the difference between the algorithm's average loss and the average loss of a comparator. Multiplying by the number of rounds gives the more standard cumulative regret. Unlike in statistical learning, where risk is a property of predictors, online learning gives us more freedom in choosing the comparator: as the book explains, we can compare the algorithm not only with a fixed predictor, but also with dynamic comparator sequences that may use different models at different times. This makes online learning a natural framework for designing algorithms that remain robust under highly nonstationary data, especially when a fixed distributional model is inadequate.

In the online convex optimization setting, regret analysis copes with arbitrary loss sequences by requiring each loss function to be convex. This allows us to upper bound regret by its linearized version, thus reducing online convex optimization (OCO) to online linear optimization (OLO). One of the key algorithmic principles for controlling linearized regret is stability through regularization. After making a prediction on a new data point, the algorithm typically updates its model by balancing two goals: reducing the current loss and staying close to the previous model. The regularizer therefore shapes the trajectory followed by the algorithm as longer and longer prefixes of the data sequence are observed. Unfortunately, standard online algorithms face a difficult dilemma: as we explained, regret is the difference between the sequential risk of the algorithm and the sequential risk of any comparator in the reference class. Ideally, we would like regret bounds that are optimal simultaneously for all comparators. However, to achieve optimal regret against any single model, the algorithm would have to choose its regularizer (or the learning rate associated with it) based on information about the comparator (e.g., the norm of the vector of coefficients for a linear model). This seems to prevent simultaneous optimality, as the algorithm must choose a specific learning rate, which might be suboptimal for some of the comparators. One of the most striking contributions of the book is a theory showing that parameter-free online algorithms, achieving essentially optimal bounds simultaneously over all comparators, can be derived from algorithms designed to play a simple coin-betting game. This theory, developed by Francesco Orabona and his co-authors over the last decade, uses a connection between prediction and betting to turn optimal gambling algorithms into optimal online learners, so that a lower bound on the wealth of the gambler becomes an upper bound on the regret of the learner. Parameter-free algorithms have now become one of the most interesting and active areas of online learning research.

Gambling is just one of the many connections explored in the book. One chapter relates online learning to portfolio management and sequential investment; another explores its interface with saddle point optimization and game theory. These interconnections with other fields reflect the way the ideas and techniques behind online learning developed over several decades in different disciplines, such as information theory and statistics, game theory, and computer science. Orabona does an excellent job of tracing the different sources, indicating where a definition first appeared and which authors should be credited for a given result or technique. The rigor and accuracy of the historical notes in the book are exemplary and should be taken as a model for surveys and monographs.

As noted at the beginning, online learning is an abstraction for an agent that learns by interacting with its environment. Regret makes this process quantitative: vanishing regret means that the algorithm's sequential risk approaches that of the comparator. A crucial aspect governing this process is the interface through which the algorithm obtains feedback from the environment. Consider, for example, a simple sequential classification problem in which, upon receiving a data point, the learner must predict a label from a fixed set of, say, ten labels. After each prediction, the environment could reveal the correct label (full feedback), reveal only whether the prediction was correct (bandit feedback), or the feedback itself may be selective, as in a spam filter that receives user feedback only on emails that reach the inbox. This adds a new dimension to the design of online learning algorithms; namely, identifying the best possible regret rate for a given feedback structure. Although, among partial-information models, the book focuses on the multi-armed bandit setting, the mathematical foundations of online learning are necessary background for studying online learning with partial feedback.

Speaking of mathematics, readers of this book will not be disappointed: every mathematical detail is laid out precisely, including all the preliminary notions needed to prove each result. Orabona strives to find the simplest possible proof, but he does not make it any simpler than that. The book is a treasure trove of inequalities, technical devices, and intermediate results that form the arsenal of any researcher interested in these topics.

I wrote a book on these topics two decades ago and have been delighted by the surprising conceptual and technical advances made since then. This book, which provides an accurate, rigorous, in-depth account of these advances, is the ideal starting point for all those who want to explore this garden of delights known as online learning.
\bigskip

\noindent Nicol\`{o} Cesa-Bianchi\\
\noindent Forte dei Marmi\\
\noindent June 1, 2026

\chapter*{Preface}

\begin{chapterquote}
\emph{Everything should be as simple as it can be, but not simpler\\
(attributed to Albert Einstein)}
\end{chapterquote}


In this book, I introduce the concepts of online learning through a modern view based on convex optimization.
Here, online learning refers to the framework of regret minimization under worst-case assumptions. I attempted to unify all the literature as instantiations of Online Mirror Descent and Follow-the-Regularized-Leader (and their variants). I paid particular attention to the issue of tuning the parameters of the algorithms, through \emph{adaptive} and \emph{parameter-free} online learning algorithms. The bandit setting is also briefly discussed, touching on the problem of adversarial and stochastic multi-armed bandits.
Building on fundamental algorithms and concepts, I also cover advanced topics, including black-box reductions, saddle-point optimization, sequential investment, and non-stationary forms of regret analysis. Finally, I conclude with a selection of applications of online learning to domains far from it, such as generalization theory and concentration inequalities.

I attempted to maintain an informal, yet mathematically rigorous, tone throughout the book. Moreover, all the included proofs have been carefully chosen to be as simple and as short as possible. This also means that sometimes I have added one or two additional assumptions, just to simplify the proofs.

\noindent\textbf{Why do we need another book on online learning?} My belief is that offline and online learning are two sides of the same coin. So, I fully used the formalism of convex analysis, typical of the non-smooth offline optimization literature, to show the similarities and build bridges. Moreover, I tried to cover very recent and very advanced results, which are often not part of online learning books. For example, parameter-free methods are (by definition!) the optimal ones for online learning, yet they are not widely understood, maybe due to the lack of a reference book on the topic. I also spent a considerable amount of time to gather all the history of these methods in my ``History Bits'' sections. I hope the serious students will take advantage of these sections to further deepen their study of this topic. A note on citations: It is customary in the computer science literature to only cite the journal version of a result that first appeared in a conference. The rationale is that the conference version is only a preliminary version, while the journal one is often more complete and sometimes more correct. In this book, I will not use this custom. Instead, in the presence of the conference and journal versions of the same paper, I will cite both, to clearly delineate the history of the ideas, their first inventors, and the unavoidable rediscoveries.

\noindent\textbf{Structure of the book.}
No prior knowledge of convex analysis is required. So, given the amount of new tools and definitions I had to cover, I introduced the necessary math a bit at a time. Hence, the book is meant to be read in sequential order. In particular, the first 7 chapters cover the basics of a course on online learning, with a possible exception of Chapter 5 on lower bounds that can be omitted for an easier set of lectures. Chapters 8 (Online Linear Classification) and 9 (Multi-Armed Bandit) can be included or not, depending on which topics one wants to cover.
Chapters 10 to 14 cover more advanced topics, such as universal portfolio, black-box reductions, and parameter-free algorithms, appropriate for a graduate-level class. Chapter 5 acts as a bridge between the foundational methods and the parameter-free ones, because it shows the \emph{necessity} for better algorithms.
Finally, Chapters 15 and 16 cover applications of online learning to other areas, such as saddle-point optimization, boosting, non-convex non-smooth optimization, and generalization theory, that might offer leads to new research directions.
While all the chapters have a minimal amount of interconnection, the major dependencies are in the figure.

\begin{figure}[h]
\centering
\resizebox{\textwidth}{!}{%
\begin{tikzpicture}[
  base/.style={circle, draw, minimum size=5mm, inner sep=0pt, font=\small\bfseries},
  normal/.style={base, fill=white},
  applications/.style={
    base,
    pattern={Dots[distance=2pt, radius=0.35pt]}
  },
  advanced/.style={base, fill=gray!25},
  halfadvanced/.style={
    base,
    path picture={
      \fill[white]
        (path picture bounding box.south west)
        rectangle
        (path picture bounding box.north);
      \fill[gray!25]
        (path picture bounding box.south)
        rectangle
        (path picture bounding box.north east);
    }
  },
  edge/.style={-{Latex[length=1.2mm, width=1.0mm]}, thick},
  chapter/.style={font=\scriptsize, anchor=west}
]

\node[chapter] at (-4.8,  2.2) {1. What is Online Learning?};
\node[chapter] at (-4.8,  1.85) {2. Online Subgradient Descent};
\node[chapter] at (-4.8,  1.50) {3. Online-to-Batch Conversions};
\node[chapter] at (-4.8,  1.15) {4. Beyond $\sqrt{T}$ Regret};
\node[chapter] at (-4.8, 0.8){5. Lower Bounds for Online Linear Optimization};
\node[chapter] at (-4.8, 0.45){6. Online Mirror Descent};
\node[chapter] at (-4.8,  0.10) {7. Follow-the-Regularized-Leader};
\node[chapter] at (-4.8, -0.25) {8. Online Linear Classification};
\node[chapter] at (-4.8, -0.6) {9. Multi-Armed Bandit};
\node[chapter] at (-4.8, -0.95) {10. Universal Portfolio Algorithms};
\node[chapter] at (-4.8, -1.3) {11. Weighted Average Algorithm and Aggregating Algorithm};
\node[chapter] at (-4.8, -1.65){12. Black-Box Reductions};
\node[chapter] at (-4.8, -2){13. Parameter-free Online Linear Optimization};
\node[chapter] at (-4.8, -2.35) {14. Dynamic, Strongly Adaptive, and Tracking Regret};
\node[chapter] at (-4.8, -2.7) {15. Saddle-Point Optimization and Online Algorithms};
\node[chapter] at (-4.8, -3.05) {16. From Online Learning to X};

\node[normal] (1) at (2.4,  2.2) {1};
\node[normal] (2) at (2.4,  1.2) {2};
\node[normal] (3) at (2.4,  0.2) {3};
\node[normal] (4) at (2.4, -0.8) {4};
\node[normal] (6) at (2.4, -1.8) {6};
\node[normal] (7) at (2.4, -2.8) {7};

\node[normal] (8) at (4, 0.2) {8};
\node[normal] (9) at (4, -0.8) {9};

\node[halfadvanced] (5)  at (4.8,  2.2) {5};
\node[advanced]     (11) at (5.6, -0.8) {11};
\node[advanced]     (10) at (7.2, -1.8) {10};
\node[advanced]     (12) at (5.6, 0.2) {12};
\node[advanced]     (13) at (7.2,  0.2) {13};
\node[advanced]     (14) at (7.2, -0.8) {14};

\node[applications] (15) at (8.8,  -2.8) {15};
\node[applications] (16) at (8.8,  0.2) {16};

\draw[edge] (1) -- (2);
\draw[edge] (2) -- (3);
\draw[edge] (3) -- (4);

\draw[edge] (4) -- (5);
\draw[edge] (5) -- (13);

\draw[edge] (4) -- (6);
\draw[edge] (6) -- (7);

\draw[edge] (7) -- (8);
\draw[edge] (7) -- (9);
\draw[edge] (7) -- (11);
\draw[edge] (7) -- (10);
\draw[edge] (7) -- (12);
\draw[edge] (7) -- (15);
\draw[edge] (10) -- (16);
\draw[edge] (12) -- (13);

\draw[edge] (13) -- (14);
\draw[edge] (13) -- (16);

\draw[dashed] (4.8,-3.3) -- (4.8,2.7);
\draw[dashed] (8.0,-3.3) -- (8.0,2.7);

\node[font=\small] at (3.2,-3.75) {Fundamentals};
\node[font=\small] at (6.4,-3.75) {Advanced theory};
\node[font=\small] at (9.1,-3.75) {Applications};

\end{tikzpicture}
}
\captionsetup{labelformat=empty}
\caption{The dependence structure of the chapters.}
\commentAlt{The graph structure of the chapters: chapters 1-9 are fundamentals, chapters 10-14 are advanced theory, chapter 15-16 are applications. Chapters 1-7 should be read sequentially. Chapters 10-12 and 15 depends on Chapter 7. Chapter 13 depends on 12. Chapter 14 depends on 13. Chapter 16 depends on 13.}
\end{figure}

\noindent\textbf{Acknowledgments.}
I thank all the people who checked the proofs and reasoning in these notes. In particular, the students in my first class who mercilessly pointed out my mistakes, Nicol\`{o} Campolongo, who found all the typos in my formulas, and Jake Abernethy, for the brainstorming on presentation strategies. Other people that helped me with comments, feedback, references, and/or hunting typos (in alphabetical order): Zeyad Aljaali, Andreas Argyriou, Param Kishor Budhraja, Nicol\`{o} Cesa-Bianchi, Sahil Chaudhary, Keyi Chen, Mingyu Chen, Peiqing Chen, Ashok Cutkosky, Ryan D'Orazio, Gerardo Dur\'{a}n-Mart\'{i}n, Alon Gonen, Peijia Guo, Dirk van der Hoeven, Daniel Hsu, Gergely Imreh, Andrew Jacobsen, Emmeran Johnson, Kwang-Sung Jun, Micha\l{} Kempka, Ji-Ha Kim, Andrew Christian Kroer, Joon Kwon, Pierre Laforgue, Wei-Cheng Lee, Chuang-Chieh Lin, Haipeng Luo, Shashank Manjunath, David Mart\'{i}nez-Rubio, Valentina Masarotto, Aryan Mokhtari, Antoine Moulin, Gergely Neu, Ankit Pensia, Viacheslav D. Potapov, Yousef Radwan, Abed Razawy, Daniel Roy, Ludovic Schwartz, Alex Shtoff, Antonio Silveti-Falls, Yanze Song, Luca Viano, Guanghui Wang, Yulian Wu, Jiujia Zhang, Peng Zhao, and Xingyu Zhou.




\mainmatter

\chapter{What is Online Learning?}
\label{ch:first}

Consider the following repeated guessing game\index{guessing game|(textbf}:

In each round $t=1,\dots,T$
\begin{itemize}
\item An adversary chooses a real number $y_t \in [0,1]$ and keeps it secret;
\item You try to guess the real number, choosing $x_t \in [0,1]$;
\item The adversary's number is revealed, and you pay the squared difference $(x_t-y_t)^2$.
\end{itemize}
\index{guessing game|)textbf}
Basically, we want to guess a sequence of numbers as precisely as possible.
To make it a game, we must now define a ``winning condition''. Let's see what makes sense to consider as a winning condition.

Let's start by making the game easier for the player. Let's assume that the adversary is drawing \ac{iid} numbers from some fixed distribution over $[0,1]$. However, he is still free to decide which distribution to use at the beginning of the game. If we knew the distribution, we could just predict the mean of the distribution at each round, and in expectation we would pay $\sigma^2 T$, where $\sigma^2$ is the variance of the distribution. We cannot do better than that! However, given that we do not know the distribution, it is natural to benchmark our strategy with respect to the optimal one. That is, it is natural to measure the quantity
\begin{equation}
\label{eq:stoch_regret}
\E_{Y_1, \dots, Y_T}\left[\sum_{t=1}^T (x_t - Y_t)^2\right] - \sigma^2 T,
\end{equation}
or, equivalently, considering the average
\begin{equation}
\label{eq:av_stoch_regret}
\frac{1}{T}\E_{Y_1, \dots Y_T}\left[\sum_{t=1}^T (x_t - Y_t)^2\right] - \sigma^2~.
\end{equation}
Clearly, these quantities are nonnegative, and they seem to be a good measure, because they are somehow normalized with respect to the ``difficulty'' of the numbers generated by the adversary, through the variance of the distribution. This is not the only possible measure of our ``success'', but it is certainly a reasonable one. It would make sense to consider a strategy ``successful'' if the difference in \eqref{eq:stoch_regret} grows sublinearly over time and, equivalently, if the difference in \eqref{eq:av_stoch_regret} goes to zero as the number of rounds $T$ goes to infinity. That is, on average over the rounds, we would like our algorithm to be able to approach the optimal performance.

\noindent\textbf{Minimizing Regret.}
Having arrived at what seems to be a good measure of success of the algorithm, let's now rewrite \eqref{eq:stoch_regret} in an equivalent way:
\[
\E\left[\sum_{t=1}^T (x_t - Y_t)^2\right] - \min_{ x \in [0,1]} \ \E\left[\sum_{t=1}^T (x-Y_t)^2\right]~.
\]
Now, the last step: let's remove the assumption on how the data is generated, consider any arbitrary sequence of $y_t$, and let's keep using the same measure of success. If the algorithm is deterministic, we can remove the expectation because there is no stochasticity anymore. So, we get that we will win the game if
\[
\Regret_T:=\sum_{t=1}^T (x_t - y_t)^2 - \min_{x \in [0,1]} \ \sum_{t=1}^T (x - y_t)^2
\]
grows sublinearly with $T$. The quantity above is called the \textbf{regret}\index{regret}, because it measures how much the algorithm regrets not having played on every round the best single choice in hindsight. We will denote it by $\Regret_T$.

Our reasoning should provide sufficient justification for this metric; however, throughout this book, we will see that it also makes sense from both a convex optimization and a machine learning perspective.

Note that in the stochastic case, the optimal strategy is given by a single best prediction, so it was natural to compare against it. Instead, with arbitrary sequences, it is not clear anymore that this is a good competitor. For example, we might consider a sequence of competitors instead of a single one. Indeed, it can be done, but the single competitor is still interesting in a variety of settings and simpler to explain. So, for most of this book, we will use a single competitor, while we will consider a sequence of competitors in Chapter~\ref{ch:dynamic}.

Let's now generalize the online guessing game, considering that the algorithm outputs a vector $\bx_t$ in the \textbf{valid set}\footnote{In some cases, we can make the game easier for the algorithm by letting it choose the prediction from a set $\mathcal{W}\supset \mathcal{V}$.}\index{valid set|textbf} $\mathcal{V} \subseteq \R^d$ (also called \textbf{feasible set}\index{feasible set|textbf}), and it pays a \textbf{loss} $\ell_t: \mathcal{V} \to \R$ that measures how good the prediction of the algorithm was in each round. Also, let's consider a \textbf{comparator}\index{comparator|textbf} (also called \textbf{competitor}\index{competitor|textbf}), that is, an arbitrary predictor $\bu$ in $\mathcal{V} \subseteq \R^d$ and let's parameterize the regret with respect to it: $\Regret_T(\bu)$.
We can now define online learning: designing and analyzing algorithms to minimize the regret over a sequence of loss functions with respect to an arbitrary competitor $\bu \in \mathcal{V} \subseteq \R^d$:
\[
\Regret_T(\bu):=\sum_{t=1}^T \ell_t(\bx_t) - \sum_{t=1}^T \ell_t(\bu)~.
\]
It is worth stressing that an online algorithm does not know $\bu$ or the value of the corresponding regret in order to guarantee an upper bound on the regret.
We will say that the algorithm is \textbf{no-regret}\index{no-regret algorithm|textbf} when, for every fixed $\bu\in\mathcal{V}$, its regret is at most sublinear, that is, $\lim_{T\to \infty}\ \Regret_T(\bu)/T \leq 0$.

\begin{remark}
Strictly speaking, the regret is also a function of the losses $\ell_1, \dots, \ell_T$. However, we will suppress this dependence for simplicity of notation.
\end{remark}

This framework is quite powerful, and it allows us to reformulate a bunch of different problems in machine learning and optimization as similar games. More generally, with the regret framework, we can analyze situations in which the data are not independent draws from a fixed distribution, yet we would like to guarantee that the algorithm is ``learning'' something. For example, online learning can be used to analyze
\begin{itemize}
\item Predicting clicks on banners on web pages;
\item Routing on a network;
\item Convergence to equilibrium of repeated games.
\end{itemize}
As we will see, it can \emph{also} be used to analyze stochastic optimization algorithms, e.g., Stochastic Gradient Descent\index{Stochastic Gradient Descent algorithm}.

\index{guessing game|(}
Let's now go back to our number-guessing game, and let's try a strategy to win it. Of course, this is one of the simplest examples of online learning, without a real application. Yet, going through it, we will uncover most of the key ingredients in online learning algorithms and their analysis.

\noindent\textbf{A Winning Strategy.}
Can we win the number guessing game? We can, not only in the form we described it, but also assuming that the adversary picks his number \emph{after} observing our choice. Moreover, we do not assume anything about how the adversary is deciding the numbers. In fact, the numbers can be chosen \emph{adversarially}, that is, explicitly trying to make us lose the game. This is why we will call the mechanism generating the numbers the \textbf{adversary}\index{adversary}.

Now, the fact that the numbers are adversarially chosen means that we can immediately rule out any strategy based on any statistical modeling of the data. In fact, it cannot work because the moment we estimate something and act on our estimate, the adversary can immediately change the way the data is generated, ruining us. So, we have to think about something else. Yet, surprisingly enough, many times online learning algorithms will look like classic ones from statistical estimation, even if they work for different reasons.

Now, let's try to design a strategy to make the regret provably sublinear in time, \emph{regardless of how the adversary chooses the numbers}.
The first thing we do is to take a look at the best strategy in hindsight, that is, the argmin of the second term of the regret. It should be immediate to see that
\[
x^\star_T
:= \argmin_{x \in [0,1]} \ \sum_{t=1}^T (x - y_t)^2
= \frac{1}{T} \sum_{t=1}^T y_t~.
\]
Now, given that we do not know the future, for sure, we cannot use $x^\star_T$ as our guess in each round. However, we do know the past, so a reasonable strategy in each round could be to output the best number over the past. Why would such a strategy work? For sure, the reason why it could work is not that we expect the future to be like the past, because it is not true! Instead, we want to leverage the fact that the optimal guess over time cannot change too much between rounds, so we can try to ``track'' it over time.

Hence, on each round $t\geq 2$, our strategy is to guess $x_t = x_{t-1}^\star=\frac{1}{t-1} \sum_{i=1}^{t-1} y_i$, and in the first round we guess any number between 0 and 1. Such a strategy is usually called \textbf{\ac{FTL}}\index{Follow-the-Leader algorithm|textbf}, because you are following what would have been the optimal thing to do on the past rounds (i.e., the leader).

Let's now try to show that this strategy will allow us to win the game. Given that this is a simple example, we will prove its regret guarantee using first principles, while in the next chapters we will introduce and use very general proof methods. First, we will need a small lemma.
\begin{lemma}[Be-the-Leader Lemma]
\label{lemma:be_leader}
\index{Be-the-Leader!lemma|textbf}
Let $\mathcal{V} \subseteq \R^d$ and $\ell_t :\mathcal{V} \to \R$ be an arbitrary sequence of loss functions.
Assume the existence of $\bx^\star_t \in \argmin_{\bx \in \mathcal{V}} \ \sum_{i=1}^t \ell_i(\bx)$, a minimizer in $\mathcal{V}$ of the cumulative loss over the first $t$ rounds. Then, we have
\[
\sum_{t=1}^T \ell_t(\bx^\star_{t})
\leq \sum_{t=1}^T \ell_t(\bx^\star_{T})~.
\]
\end{lemma}
\begin{proof}
We prove it by induction on $T$. The base case is
\[
\ell_1(\bx^\star_1)
\leq \ell_1(\bx^\star_{1}),
\]
which is trivially true.
Now, for $T\geq2$, we assume that $\sum_{t=1}^{T-1} \ell_t(\bx^\star_{t}) \leq \sum_{t=1}^{T-1} \ell_t(\bx^\star_{T-1})$ is true, and we must prove the stated inequality, that is,
\[
\sum_{t=1}^T \ell_t(\bx^\star_{t})
\leq \sum_{t=1}^T \ell_t(\bx^\star_{T})~.
\]
This inequality is equivalent to
\begin{equation}
\label{eq:lemma1}
\sum_{t=1}^{T-1} \ell_t(\bx^\star_{t})
\leq \sum_{t=1}^{T-1} \ell_t(\bx^\star_{T}),
\end{equation}
where we removed the last element of the sums because they are the same.
Now observe that
\[
\sum_{t=1}^{T-1} \ell_t(\bx^\star_{t})
\leq \sum_{t=1}^{T-1} \ell_t(\bx^\star_{T-1}),
\]
by induction hypothesis, and
\[
\sum_{t=1}^{T-1} \ell_t(\bx^\star_{T-1})
\leq \sum_{t=1}^{T-1} \ell_t(\bx^\star_{T})
\]
because $\bx^\star_{T-1}$ is a minimizer of the l.h.s. in $\mathcal{V}$ and $\bx^\star_{T} \in \mathcal{V}$.
Chaining these two inequalities, we have that \eqref{eq:lemma1} is true, and so the lemma is proven.
\end{proof}
Basically, the above lemma quantifies the idea that knowing the future and being adaptive to it is typically better than not being adaptive to it.

With this lemma, we can now prove that the regret will grow sublinearly; in particular, it will be at most \emph{logarithmic} in time. Note that we will not prove that our strategy is minimax optimal, even if it is possible to show that the logarithmic dependence on time is unavoidable for this problem.
\begin{theorem}
Let $y_1, \dots, y_T \in [0,1]$ be an arbitrary sequence of real numbers. Let the algorithm's output be $x_t=x_{t-1}^\star:=\frac{1}{t-1}\sum_{i=1}^{t-1} y_i$ for $t\geq 2$ and $x_1=x_0^\star:=0.5$. Then, we have
\[
\Regret_T
= \sum_{t=1}^T (x_t - y_t)^2 - \min_{x \in [0,1]} \ \sum_{t=1}^T (x - y_t)^2
\leq 4 + 4\ln T~.
\]
\end{theorem}
\begin{proof}
We use Lemma~\ref{lemma:be_leader}\index{Be-the-Leader!lemma} to upper bound the regret:
\begin{align*}
\sum_{t=1}^T (x_t - y_t)^2 - \min_{x \in [0,1]} \ \sum_{t=1}^T (x - y_t)^2
&= \sum_{t=1}^T (x^\star_{t-1} - y_t)^2 - \sum_{t=1}^T (x^\star_T - y_t)^2 \\
&\leq \sum_{t=1}^T (x^\star_{t-1} - y_t)^2 - \sum_{t=1}^T (x^\star_t - y_t)^2~.
\end{align*}
Now, let's take a look at each difference in the sum in the last equation.
For $t=1$, we have
\[
(x^\star_{0} - y_1)^2 - (x^\star_1 - y_1)^2
\leq 0.5^2~.
\]
For $t\geq 2$, we have that
\begin{align*}
(x^\star_{t-1} - y_t)^2 - (x^\star_t - y_t)^2
&= (x^\star_{t-1})^2 - 2 y_t x^\star_{t-1} - (x^\star_{t})^2 + 2 y_t x^\star_{t} \\
&= (x^\star_{t-1}+x^\star_{t} - 2y_t)(x^\star_{t-1}-x^\star_{t}) \\
&\leq |x^\star_{t-1}+x^\star_{t} - 2y_t|\,|x^\star_{t-1}-x^\star_{t}|
\leq 2 |x^\star_{t-1}-x^\star_{t}| \\
&=2\left|\frac{1}{t-1} \sum_{i=1}^{t-1} y_i -\frac{1}{t} \sum_{i=1}^{t} y_i\right| \\
&=2\left|\left(\frac{1}{t-1}-\frac{1}{t}\right) \sum_{i=1}^{t-1} y_i - \frac{y_t}{t}\right| \\
&\leq 2\left|\frac{1}{t(t-1)}\sum_{i=1}^{t-1} y_i\right| + \frac{2|y_t|}{t}
\leq \frac{2}{t} + \frac{2|y_t|}{t}
\leq \frac{4}{t}~.
\end{align*}
Hence, overall, we have
\[
\sum_{t=1}^T (x_t - y_t)^2 - \min_{x \in [0,1]} \ \sum_{t=1}^T (x - y_t)^2
\leq 0.25+4\sum_{t=2}^T\frac{1}{t}
\leq 4 \sum_{t=1}^T \frac{1}{t}~.
\]


\begin{figure}[t]
\centering
\begin{tikzpicture}
\begin{axis}[
    width=7cm,
    xmin=0, xmax=12,
    ymin=0, ymax=1.1,
    xlabel={$x$},
    ytick={0.1,0.2,0.3,0.4,0.5,0.6,0.7,0.8,0.9,1},
    xtick={0,2,4,6,8,10},
    axis lines=left,
    legend style={
        at={(0.95,0.95)},
        anchor=north east,
        draw=black,
        fill=white,
        font=\small,
        cells={anchor=west}
    },
]

\draw[thick,fill=lightgray] (1,0) rectangle (2,1);
\draw[thick,fill=lightgray] (2,0) rectangle (3,1/2);
\draw[thick,fill=lightgray] (3,0) rectangle (4,1/3);
\draw[thick,fill=lightgray] (4,0) rectangle (5,1/4);
\draw[thick,fill=lightgray] (5,0) rectangle (6,1/5);
\draw[thick,fill=lightgray] (6,0) rectangle (7,1/6);
\draw[thick,fill=lightgray] (7,0) rectangle (8,1/7);
\draw[thick,fill=lightgray] (8,0) rectangle (9,1/8);
\draw[thick,fill=lightgray] (9,0) rectangle (10,1/9);
\draw[thick,fill=lightgray] (10,0) rectangle (11,1/10);

\addplot[thick, black, domain=2:11, samples=400]   {1/x};
\addlegendentry{$1/x$}
\addplot[thick, dotted, black, domain=2:11, samples=400]  {1/(x-1)};
\addlegendentry{$1/(x-1)$}
\end{axis}
\end{tikzpicture}
\caption{Upper bounding the sum with an integral.}
\label{fig:sum_log}
\commentAlt{Figure~\ref{fig:sum_log}. Bar plot showing rectangles of heights 1, 1/2, ..., 1/10 compared with the curves 1/x and 1/(x-1), illustrating an integral upper bound on a harmonic sum.}
\end{figure}

To upper bound the last sum, observe that we are trying to find an upper bound to the gray area in Figure~\ref{fig:sum_log}. As you can see from the picture, it can be upper bounded by 1 plus the integral of $\frac{1}{t-1}$ from $2$ to $T+1$. So, we have
\[
\sum_{t=1}^T\frac{1}{t}
\leq 1+\int_{2}^{T+1} \! \frac{1}{t-1} \, \mathrm{d}t
= 1+ \ln T~.
\]
Putting everything together gives the stated bound.
\end{proof}
\index{guessing game|)}
Let's write in words the steps of the proof: Lemma~\ref{lemma:be_leader}\index{Be-the-Leader!lemma} allows us to upper bound the regret against the single best guess with the regret against the competitor sequence $x^\star_1, \dots, x^\star_T$. In turn, given that we produce in each round the prediction $x^\star_{t-1}$ and $|x^\star_t-x^\star_{t-1}|$ goes to zero like $1/t$, the total regret is sublinear in time.

There are a few things to stress about this strategy. The strategy does not have parameters to tune (e.g., learning rates, regularizers). Note that tuning parameters on validation data, as we typically do in machine learning, does not make sense in online learning: we have only one stream of data, and we cannot run our algorithm over it multiple times to select the best parameter! Also, this strategy does not need to maintain a complete record of the past, but only a ``summary'' of it, through the running average. This gives a computationally efficient algorithm. When we design online learning algorithms, we will strive to achieve all these characteristics.
The final point to stress is that the algorithm does not use gradients: gradients are useful, and we will use them a lot, but they do not constitute the entire world of online learning.

Before going on, I want to remind the reader that, as seen above, this is different from the classic setting in statistical machine learning. So, for example, ``overfitting'' has no meaning here. The same holds for ``generalization gap'' and similar ideas linked to a training/testing scenario.

In the next chapters, we will introduce several algorithms for online learning, and one of them will be a strict generalization of the strategy we used in the example above.

\section{History Bits}
\index{regret|(}
The concept of ``regret'' seems to have been proposed in \citet{Savage51}, an exposition and review of the book by \citet{Wald50} on a foundation of statistical decision problems based on zero-sum two-person games. \citet{Savage51} introduces the idea of considering the difference between the utility of the best action in a given state and the utility obtained by any action under the same state. The proposed optimal strategy was the one minimizing such regret over the worst possible state.
\citet{Savage51} called this concept ``loss'' and did not like the word ``regret'' because ``that term seems to me charged with emotion and liable to lead to such misinterpretation as that the loss necessarily becomes known''~\citep[page 163]{Savage54}. The name ``regret'' instead seems to have been suggested in \citet{Milnor51}.

However, Savage's definition is a modification of the one proposed by \citet{Wald50}, who instead proposed to maximize the utility directly, under the assumption that the utility of the best action for any state is 0.
While minimizing the regret or minimizing the negative utility under the assumption of \citet{Wald50} are mathematically equivalent, \citet[pages 169--170]{Savage54} explains that Wald considered the regret formulation different from what he proposed, while Savage attributed to his idea ``little or no originality''.
\index{regret|)}

Extending the definition of \citet{Savage51,Savage54} to a sequence of games, \citet{Hannan57} designed a randomized algorithm for zero-sum repeated games with a fixed loss matrix with a vanishing expected average regret.
Hence, the concept of regret seems to originate from game theory, but, strangely enough, it passed through the work of two mathematical statisticians.

The Be-the-Leader lemma (Lemma~\ref{lemma:be_leader})\index{Be-the-Leader!lemma} is due to \citet{Hannan57}.

\section*{Exercises}

\begin{exer}
Extend the previous algorithm and analysis to the case when the adversary selects a vector $\by_t \in \R^d$ such that $\|\by_t\|_2\leq1$, the algorithm guesses a vector $\bx_t \in\R^d$, and the loss function is $\|\bx_t-\by_t\|^2_2$. Show an upper bound on the regret logarithmic in $T$ and that does not depend on $d$. Among other things, you will probably need the Cauchy--Schwarz inequality\index{inequality!Cauchy--Schwarz|textbf}: $|\langle \bx,\by\rangle| \leq \|\bx\|_2 \|\by\|_2$.
\end{exer}

\begin{exer}
Consider the number-guessing game with absolute loss instead of squared loss: $\ell_t(x)=|x-y_t|$, where $x,y_t\in[0,1]$.
Show that the best fixed prediction in hindsight is any median of the sequence
$y_1,\dots,y_T$.
\end{exer}

\acresetall

\chapter{Online Subgradient Descent}
\label{ch:osd}

In this chapter, we will introduce the \ac{OSD} algorithm: a generic online algorithm to solve online problems with convex losses. First, we will introduce \ac{OGD} for convex differentiable functions, then we will extend it to non-differentiable functions.

\acresetall

\section{Online Learning with Convex Differentiable Losses}

To summarize what we said in the first chapter, let's define online learning as the following general game:
\begin{itemize}
\item For $t=1,\dots,T$
\begin{itemize}
\item Output $\bx_t \in \mathcal{V}\subseteq \R^d$
\item Pay the loss $\ell_t(\bx_t)$, where $\ell_t:\mathcal{V} \to \R$
\item Receive some feedback on $\ell_t$
\end{itemize}
\item End for
\end{itemize}
The aim of this game is to minimize the regret\index{regret|textbf} with respect to any competitor $\bu \in \mathcal{V}$:
\[
\Regret_T(\bu):=\sum_{t=1}^T \ell_t(\bx_t) - \sum_{t=1}^T \ell_t(\bu)~.
\]
We also said that the way the losses $\ell_t$ are decided is adversarial.
Now, without making any additional assumptions, we cannot hope to solve this problem. Hence, we must understand what reasonable assumptions we can make. Typically, we will try to restrict the choice of the loss functions in some way. This is considered reasonable because most of the time, we have some say in deciding the set from which the loss functions are picked. So, for example, we will consider only \emph{convex} loss functions. However, convexity might not be enough, so we might restrict the class a bit more to, for example,  convex functions with bounded gradients. On the other hand, assuming knowledge of something about the future is not considered a reasonable assumption, because we very rarely have any control over the future.
In general, the stronger the assumptions, the better the guarantee on the regret will be. The best algorithms we will see will guarantee a sublinear regret against the weakest assumption we can make, guaranteeing \emph{at the same time} a smaller regret for \emph{easy} adversaries.

It is also important to remember why minimizing regret is a good objective: given that we do not assume anything about how the adversary generates the loss functions, minimizing regret is a good metric that takes into account the difficulty of the problem. If an online learning algorithm is able to guarantee a sublinear regret, it means that its performance on average will approach the performance of any fixed strategy. As said, we will see that in many situations, if the adversary is ``weak'', for example, it is a fixed stochastic distribution over the loss functions, being prepared for the worst-case scenario will not preclude us from getting the best guarantee anyway.

For a while, we will focus on the case that $\ell_t$ are convex, and this problem will be called \textbf{\ac{OCO}}\index{online convex optimization}. Later, we will see how to \emph{convexify} some specific non-convex online problems.

\begin{remark}
I will now introduce some math concepts. If you have a background in Convex Analysis, this will be easy stuff for you. On the other hand, if you have never seen these things before, they might look a bit scary. Let me tell you the right way to look at them: \emph{these are tools that will make our job easier}. Without these tools, it would be basically impossible to design any online learning algorithm. And, no, it is not enough to test the algorithms on some machine learning dataset, because fixed datasets are not adversarial. Without a correct proof, you might not realize that your online algorithm fails on particular sequences of losses, as it happened with Adam~\citep{ReddiKK18}.
I promise you that once you understand the key mathematical concepts, online learning is actually easy.
\end{remark}

\subsection{Convex Analysis Bits: Convexity}

\begin{figure}[h]
\centering
\begin{tikzpicture}
\draw[thick,fill=lightgray] (0,0) -- (30:1) arc (30:330:1) -- cycle;
\draw[thick,fill=lightgray] (2,-1) -- (4,-1) -- (3,1) -- (2,1) -- (3,0) -- cycle;
\end{tikzpicture}
\hspace{3cm}
\begin{tikzpicture}
\draw[thick,fill=lightgray] (0,0) circle (1);
\draw[thick,fill=lightgray] (2,-1) -- (4,-1) -- (3,1) -- cycle;
\end{tikzpicture}
\caption{Non-convex (left) and convex (right) sets.}
\label{fig:convex_set}
\commentAlt{Figure~\ref{fig:convex_set}. Four shaded planar sets. The two sets on the left are non-convex, including a crescent-like set and an indented polygon; the two sets on the right are convex, a disk and a triangle.}
\end{figure}

\begin{definition}
$\mathcal{V} \subseteq \R^d$ is \textbf{convex}\index{convex set|textbf} if for any $\bx,\by\in \mathcal{V}$ and any $\lambda \in (0,1)$, we have $\lambda \bx+ (1-\lambda) \by \in \mathcal{V}$.
\end{definition}
In words, this means that the set $\mathcal{V}$ has no `holes' or inward dents, see Figure~\ref{fig:convex_set}.

We will make use of \textbf{extended-real-valued functions}\index{function!extended-real-valued|textbf}, that is, functions that take values in $\R\cup\{-\infty,+\infty\}$. For $f$ an extended-real-valued function on $\R^d$, its \textbf{domain} is the set $\dom f = \{\bx \in \R^d : f(\bx) < +\infty\}$\index{domain|textbf}.

Extended-real-valued functions\index{function!extended-real-valued} allow us to easily consider constrained sets and are a standard notation in Convex Optimization~\citep[see, e.g.,][]{BoydV04}. For example, if I want the predictions of the algorithm $\bx_t$ and the competitor $\bu$ to be in a set $\mathcal{V} \subset \R^d$, I can just add $\indicator_{\mathcal{V}}(\bx)$ to all the losses, where $\indicator_{\mathcal{V}}:\R^d\to (-\infty, +\infty]$ is the \textbf{indicator function of the set $\mathcal{V}$}\index{indicator function|textbf} defined as
\[
\indicator_{\mathcal{V}}(\bx) = \begin{cases} 0, & \bx \in \mathcal{V},\\ +\infty, & \text{otherwise.} \end{cases}
\]
In this way, the only way for the algorithm and for the competitor to suffer finite loss is to predict inside the set $\mathcal{V}$.
Also, extended-real-valued functions\index{function!extended-real-valued} will make the use of \emph{Fenchel conjugates} more direct, see Section~\ref{sec:fenchel}.

\begin{figure}[t]
\centering
\begin{tikzpicture}
\begin{axis}[xmax=6,ymax=2, xmin=-0.5, ymin=-0.4,
          axis lines=center,
          width=7cm,
          ticks=none,
          xlabel = $x$,
          ylabel = $f(x)$]
\addplot[name path=A, thick, black, domain=1:5,samples=1000] {1/sqrt(5.1-x)*1/max(x-0.7,0)};
\addplot[draw=none,name path=B] {4};
\addplot [lightgray] fill between [of = A and B, soft clip={domain=1:5}];
\draw[dashed] (axis cs:1,0) -- (axis cs:1,4);
\draw[dashed] (axis cs:5,0) -- (axis cs:5,4);
\draw[thick] (axis cs:1,1.6462) -- (axis cs:1,4);
\draw[thick] (axis cs:5,0.7354) -- (axis cs:5,4);
\draw[<->, thick] (axis cs:1,0) -- (axis cs:5,0);
\node at (axis cs:3,-0.3) {$\dom f$};
\node at (axis cs:3,1.4) {Epigraph};
\end{axis}
\end{tikzpicture}
\hspace{1cm}
\begin{tikzpicture}
\begin{axis}[xmax=3,ymax=1, xmin=-0.25, ymin=-0.2,
          axis lines=center,
          ticks=none,
          width=7cm,
          xlabel = $x$,
          ylabel = $f(x)$]
\addplot[name path=A, thick,black,domain=.5:1.1,samples=1000] {.5+sin(deg(4*pi*(x-0.5)-2*pi))*(x-1)};
\addplot[name path=B, thick,black,domain=1.5:2.5,samples=1000] {(x-2)^2+0.25};
\addplot[draw=none,name path=C] {1};
\addplot [lightgray] fill between [of = A and C, soft clip={domain=.5:1.1}];
\addplot [lightgray] fill between [of = B and C, soft clip={domain=1.5:2.5}];
\draw[dashed] (axis cs:.5,0) -- (axis cs:.5,1);
\draw[dashed] (axis cs:1.1,0) -- (axis cs:1.1,1);
\draw[dashed] (axis cs:1.5,0) -- (axis cs:1.5,1);
\draw[dashed] (axis cs:2.5,0) -- (axis cs:2.5,1);
\draw[thick] (axis cs:.5,.5) -- (axis cs:.5,1);
\draw[thick] (axis cs:1.1,0.5951) -- (axis cs:1.1,1);
\draw[thick] (axis cs:1.5,0.5^2+0.25) -- (axis cs:1.5,1);
\draw[thick] (axis cs:2.5,0.5^2+0.25) -- (axis cs:2.5,1);
\draw[<->, thick] (axis cs:.5,0) -- (axis cs:1.1,0);
\draw[<->, thick] (axis cs:1.5,0) -- (axis cs:2.5,0);
\node at (axis cs:1.25,-0.15) {$\dom f$};
\draw[->] (axis cs:1.25,-.1) -- (axis cs:0.8,-.05);
\draw[->] (axis cs:1.25,-.1) -- (axis cs:2,-.05);
\node at (axis cs:1.3,0.2) {Epigraph};
\draw[->] (axis cs:1.3,.25) -- (axis cs:0.6,.5);
\draw[->] (axis cs:1.3,.25) -- (axis cs:2,.5);
\end{axis}
\end{tikzpicture}
\caption{Convex (left) and nonconvex (right) functions.}
\label{fig:convex_nonconvex_epi}
\commentAlt{Figure~\ref{fig:convex_nonconvex_epi}. Two epigraph diagrams. The left plot shows a convex function on a connected domain with a convex shaded epigraph; the right plot shows a function on two separated intervals with a disconnected, nonconvex epigraph.}
\end{figure}

\emph{Convex functions} will be an essential ingredient in online learning.
\begin{definition}
Let $f:\R^d \to [-\infty, +\infty]$. $f$ is \textbf{convex}\index{function!convex|textbf} if the epigraph\index{epigraph|textbf} of the function, $\{(\bx, y) \in \R^{d+1} : y\geq f(\bx)\}$, is convex.
\end{definition}
We can see a visualization of this definition in Figure~\ref{fig:convex_nonconvex_epi}. Note that the definition implies that the domain of a convex function is convex.
The $\indicator_{\mathcal{V}}(\bx)$ is convex iff $\mathcal{V}$ is convex, so each convex set is associated with a convex function.
Also, observe that if $f:\R^d \to (-\infty, +\infty]$ is convex, $f + \indicator_{\mathcal{V}}: \R^d \to (-\infty, +\infty]$, where $\mathcal{V}$ is convex, is also convex.

The definition above gives rise to the following characterization for convex functions that do not assume the value $-\infty$.
\begin{theorem}[{\citealp[Theorem 4.1]{Rockafellar70}}]
Let $f:\R^d\to (-\infty, +\infty]$ be such that $\dom f$ is a convex set. Then $f$ is convex iff, for any $0<\lambda <1$, we have
\[
f(\lambda \bx + (1-\lambda) \by) \leq \lambda f(\bx) + (1-\lambda) f(\by), \quad \forall \bx, \by \in \dom f~.
\]
\end{theorem}

\begin{example}
\emph{Affine functions}\index{function!affine}, $f(\bx)=\langle \bz, \bx\rangle + b$, are convex.
\end{example}

\begin{example}
Norms are always convex; the proof is left as an exercise.
\end{example}

How to recognize a convex function? In the most general case, you have to rely on its definition. However, most of the time, we will recognize them as being obtained by operations that preserve convexity. For example:
\begin{itemize}
\item If $f$ and $g$ are convex, then their linear combination with non-negative weights is also convex.
\item The composition with an affine transformation\index{function!affine} preserves the convexity.
\item If $f:\R^d \to \R$ and $g:\R \to \R$ are convex functions and $g$ is non-decreasing, then $h(\bx) = g ( f ( \bx ) )$ is convex.
\item The pointwise supremum of convex functions is convex.
\end{itemize}
The proofs are left as exercises.

A very important property of \emph{differentiable} convex functions is that we can construct a linear lower bound to the function.
\begin{theorem}[{\citealp[Theorem 25.1 and Corollary 25.1.1]{Rockafellar70}}]
\label{thm:grad_ineq}
Suppose $f:\R^d\to (-\infty, +\infty]$ is a convex function and let $\bx \in \interior \dom f$.
If $f$ is differentiable at $\bx$ then
\[
f(\by) \geq f(\bx) + \langle \nabla f(\bx), \by - \bx\rangle, \quad  \forall \by \in \R^d~.
\]
\end{theorem}

We will also use the first-order optimality condition for differentiable convex functions:
\begin{theorem}
\label{thm:constr_opt_condition}
Let $\mathcal{V}$ be a non-empty convex set, $\bx^\star \in \mathcal{V}$, and $f$ a function convex and differentiable over an open set that contains $\mathcal{V}$.
Then, $\bx^\star \in \argmin_{\bx \in \mathcal{V}} \ f(\bx)$ iff $\langle \nabla f(\bx^\star), \by -\bx^\star \rangle \geq 0, \ \forall \by \in \mathcal{V}$.
\end{theorem}
\begin{proof}
Let us first assume that $\bx^\star$ satisfies $\langle \nabla f(\bx^\star), \by -\bx^\star \rangle \geq 0, \ \forall \by \in \mathcal{V}$. Then, by Theorem~\ref{thm:grad_ineq}, for any $\by \in \mathcal{V}$, we have that
\[
f(\by)
\geq f(\bx^\star) + \langle \nabla f(\bx^\star), \by - \bx^\star\rangle
\geq f(\bx^\star),
\]
that is, $\bx^\star$ is the minimizer of $f$ over $\mathcal{V}$.

Now, assume that $\bx^\star$ is the minimizer of $f$ over $\mathcal{V}$ and assume that there exists $\by \in \mathcal{V}$ such that $\langle \nabla f(\bx^\star), \by -\bx^\star \rangle < 0$. Consider $\bz(\alpha)=\alpha \by+(1-\alpha)\bx^\star$ where $\alpha \in [0,1]$. Note that $\bz(\alpha) \in \mathcal{V}$ and denote by $h(\alpha)=f(\bz(\alpha))$. We have that $h'(0)=\langle \nabla f(\bx^\star), \by-\bx^\star\rangle <0$. Hence, by the definition of derivative, for $\alpha^\star>0$ sufficiently small we have $h(\alpha^\star)<h(0)$, that is, $f(\bz(\alpha^\star))<f(\bx^\star)$, which contradicts that $\bx^\star$ is the minimizer over $\mathcal{V}$.
\end{proof}
In words, at the constrained minimum, the gradient makes an angle of $90^\circ$ or less with all the feasible variations $\by - \bx^\star$, hence we cannot decrease the function any further by moving inside $\mathcal{V}$. Moreover, if $\bx^\star \in \interior \mathcal{V}$, by choosing $\epsilon$ small enough such that $\by=\bx^\star-\epsilon \nabla f(\bx^\star) \in \mathcal{V}$, we obtain that $\bx^\star$ is a minimum iff $\nabla f(\bx^\star)=\boldsymbol{0}$.

Another critical property of convex functions is Jensen's inequality\index{inequality!Jensen's|textbf}.
\begin{theorem}[Jensen's inequality]
\label{thm:jensen}
Let $f : \R^d \to (-\infty, +\infty]$ be a measurable convex function and $\bx$ be an $\R^d$-valued random element on some probability space such that $\E[\bx]$ exists and $\bx \in \dom f$ with probability 1. Then, we have
\[
\E[f(\bx)] \geq f(\E[\bx])~.
\]
\end{theorem}

We can now see our first \ac{OCO} algorithm in the case that the functions are convex and differentiable.

\subsection{Online Gradient Descent}

\index{Online Gradient Descent algorithm|(textbf}
In the first chapter, we saw a simple strategy to obtain a logarithmic regret in the guessing game. The strategy was to use the best prediction over the past, that is, the \textbf{\ac{FTL}}\index{Follow-the-Leader algorithm} strategy. In formulas,
\[
\bx_t=\argmin_{\bx \in \mathcal{V}} \  \sum_{i=1}^{t-1} \ell_i(\bx),
\]
and, in the first round, we can play any admissible point.
One might wonder if this strategy always works, but the answer is negative!

\begin{example}[Failure of \ac{FTL}]
\label{ex:failure_ftl}
Let $\mathcal{V} = [-1,1]$ and consider the sequence of losses $\ell_t(x) = z_t x$, where
\begin{align*}
z_1 &= -0.5,\\
z_{t} &= 1, \ t=2, 4, \dots\\
z_{t} &= -1, \ t=3, 5, \dots
\end{align*}
Then, apart from the first round where the prediction of \ac{FTL} is arbitrary in $[-1,1]$, the predictions of \ac{FTL} will be $x_t = 1$ for $t$ even and $x_t = -1$ for $t$ odd. The cumulative loss of the \ac{FTL} algorithm after $T$ rounds will therefore be $T-1-\frac{x_1}{2}$ while the cumulative loss of the fixed solution $u = 0$ is 0. Thus, the regret of \ac{FTL} with respect to $u=0$ is $T-1-\frac{x_1}{2}\geq T-\frac{3}{2}$.
\end{example}

Hence, we will see an alternative strategy that guarantees sublinear regret for convex functions with bounded gradients.
The strategy is called \textbf{Projected Online Gradient Descent}, or just \textbf{\ac{OGD}}, see Algorithm~\ref{alg:pogd}. In each step, we update the prediction of the algorithm by adding the negative gradient of the loss multiplied by a scalar $\eta_t>0$ called the \textbf{learning rate}\index{learning rate} or \textbf{stepsize}\index{stepsize}. Then, we project the result back onto the feasible set $\mathcal{V}$. Some might notice that this algorithm is similar to Stochastic Gradient Descent\index{Stochastic Gradient Descent algorithm}, but it is not the same thing: here, the loss functions are different at each step, and they are not drawn from a fixed distribution but adversarially chosen. We will later see that \ac{OGD} can \emph{also} be used as Stochastic Gradient Descent.

\begin{algorithm}[t]
\caption{Projected Online Gradient Descent (OGD)}
\label{alg:pogd}
\begin{algorithmic}[1]
{
    \REQUIRE{Non-empty closed convex set $\mathcal{V} \subseteq \R^d$, $\bx_1 \in \mathcal{V}$, $\eta_1,\dots,\eta_T>0$}
    \FOR{$t=1$ {\bfseries to} $T$}
    \STATE{Output $\bx_t \in \mathcal{V}$}
    \STATE{Pay the loss $\ell_t(\bx_t)$, for a convex $\ell_t$, differentiable on an open set containing $\mathcal{V}$}
    \STATE Set $\bg_t= \nabla \ell_t(\bx_t)$
    \STATE{$\bx_{t+1} = \Pi_{\mathcal{V}}(\bx_t - \eta_t \bg_t) = \argmin_{\by \in \mathcal{V}} \ \|\bx_t - \eta_t \bg_t -\by\|_2$}
    \ENDFOR
}
\end{algorithmic}
\end{algorithm}

First, we show the following two auxiliary results.
The first proposition proves that Euclidean projections\index{projection!Euclidean|textbf}, the unique point $\Pi_\mathcal{V}(\bx):=\argmin_{\by \in \mathcal{V}} \ \|\bx-\by\|_2$, always decrease the distance to points inside the set.
\begin{proposition}
\label{prop:Euclidean_proj}
Let $\bx \in \R^d$ and $\by \in \mathcal{V}$, where $\mathcal{V} \subseteq \R^d$ is a non-empty closed convex set and define $\Pi_{\mathcal{V}}(\bx):=\argmin_{\by \in \mathcal{V}} \ \|\bx-\by\|_2$. Then, $\|\Pi_{\mathcal{V}}(\bx)-\by\|_2 \leq \|\bx-\by\|_2$.
\end{proposition}
\begin{proof}
First of all, observe that $\argmin_{\by \in \mathcal{V}} \ \|\bx-\by\|_2 = \argmin_{\by \in \mathcal{V}} \ \frac12 \|\bx-\by\|^2_2$.
So, from the optimality condition of Theorem~\ref{thm:constr_opt_condition} on the function $f(\by)=\frac12 \|\bx-\by\|^2_2$, we obtain
\[
\langle \Pi_{\mathcal{V}}(\bx)-\bx, \by-\Pi_{\mathcal{V}}(\bx)\rangle \geq 0~.
\]
Therefore,
\begin{align*}
\|\by - \bx\|_2^2
&= \|\by-\Pi_{\mathcal{V}}(\bx)+\Pi_{\mathcal{V}}(\bx)-\bx\|_2^2 \\
&= \|\by-\Pi_{\mathcal{V}}(\bx)\|^2_2 +2 \langle \by-\Pi_{\mathcal{V}}(\bx),\Pi_{\mathcal{V}}(\bx)-\bx\rangle + \|\Pi_{\mathcal{V}}(\bx)-\bx\|_2^2 \\
&\geq \|\by-\Pi_{\mathcal{V}}(\bx)\|_2^2~. \qedhere
\end{align*}
\end{proof}

The next lemma upper bounds the regret in a single iteration of Algorithm~\ref{alg:pogd}.
\begin{lemma}
\label{lemma:gd_one_step}
Let $\mathcal{V} \subseteq \R^d$ be a non-empty closed convex set, $\ell_t$ be convex and differentiable on an open set containing $\mathcal{V}$, and $\eta_t>0$. Set $\bg_t = \nabla \ell_t(\bx_t)$ and $\bx_{t+1} = \Pi_{\mathcal{V}}(\bx_t - \eta_t \bg_t)$. Then, $\forall \bu \in \mathcal{V}$, the following inequality holds
\[
\eta_t (\ell_t(\bx_t)- \ell_t(\bu) )
\leq \eta_t \langle \bg_t, \bx_t -\bu\rangle
\leq \frac{1}{2}\|\bx_t-\bu\|^2_2 - \frac{1}{2}\|\bx_{t+1}-\bu\|^2_2 + \frac{\eta_t^2}{2}\|\bg_t\|_2^2~.
\]
\end{lemma}
\begin{proof}
From Proposition~\ref{prop:Euclidean_proj} and Theorem~\ref{thm:grad_ineq}, we have that
\begin{align*}
\|\bx_{t+1}-\bu\|^2_2 - \|\bx_{t}-\bu\|^2_2
&\leq \|\bx_t-\eta_t \bg_t -\bu\|^2_2 - \|\bx_{t}-\bu\|^2_2 \\
&= -2 \eta_t \langle \bg_t, \bx_t - \bu\rangle + \eta_t^2 \|\bg_t\|^2_2 \\
&\leq -2 \eta_t (\ell_t(\bx_t) - \ell_t(\bu)) + \eta_t^2 \|\bg_t\|^2_2~.
\end{align*}
Reordering, we have the stated bound.
\end{proof}

We can prove the following regret guarantee for Algorithm~\ref{alg:pogd}.
\begin{theorem}
\label{thm:pogd}
Let $\mathcal{V} \subseteq \R^d$ be a non-empty closed convex set with diameter $D$, that is, $D:=\sup_{\bx,\by\in \mathcal{V}} \|\bx-\by\|_2$. Let $\ell_1, \dots, \ell_T$ be an arbitrary sequence of convex and differentiable functions on an open set containing $\mathcal{V}$. Pick any $\bx_1 \in \mathcal{V}$ and assume $0<\eta_{t+1}\leq \eta_{t}, \ t=1, \dots, T-1$. Then, $\forall \bu \in \mathcal{V}$, Algorithm~\ref{alg:pogd} satisfies the following regret bound
\[
\sum_{t=1}^T (\ell_t(\bx_t) - \ell_t(\bu))
\leq \frac{D^2}{2\eta_{T}} + \sum_{t=1}^T \frac{\eta_t}{2} \|\bg_t\|^2_2 - \frac{1}{2\eta_T} \|\bx_{T+1}-\bu\|^2_2~.
\]
Moreover, if $\eta_t$ is constant, i.e., $\eta_t=\eta \ \forall t=1,\dots,T$, for all $\bu \in \mathcal{V}$, we have
\[
\sum_{t=1}^T (\ell_t(\bx_t) - \ell_t(\bu))
\leq \frac{\|\bu-\bx_1\|_2^2}{2\eta} + \frac{\eta}{2} \sum_{t=1}^T \|\bg_t\|^2_2 - \frac{1}{2\eta} \|\bx_{T+1}-\bu\|^2_2~.
\]
\end{theorem}
\begin{proof}
Dividing the inequality in Lemma~\ref{lemma:gd_one_step} by $\eta_t$ and summing over $t=1,\dots,T$, we have
\begin{align*}
\sum_{t=1}^T &(\ell_t(\bx_t) - \ell_t(\bu))
\leq \sum_{t=1}^T \left(\frac{\|\bx_{t}-\bu\|^2_2}{2\eta_t} - \frac{\|\bx_{t+1}-\bu\|^2_2}{2\eta_t}\right) + \sum_{t=1}^T \frac{\eta_t \|\bg_t\|^2_2}{2}  \\
&= \frac{\|\bx_{1}-\bu\|^2_2}{2\eta_1} - \frac{\|\bx_{T+1}-\bu\|^2_2}{2\eta_T}  + \sum_{t=1}^{T-1} \left(\frac{1}{\eta_{t+1}}-\frac{1}{\eta_t}\right)\frac{\|\bx_{t+1}-\bu\|^2_2}{2}  + \sum_{t=1}^T \frac{\eta_t \|\bg_t\|^2_2}{2} \\
&\leq \frac{D^2}{2\eta_1}  + \frac{D^2}{2} \sum_{t=1}^{T-1} \left(\frac{1}{\eta_{t+1}}-\frac{1}{\eta_{t}}\right) + \sum_{t=1}^T \frac{\eta_t \|\bg_t\|^2_2}{2}  - \frac{\|\bx_{T+1}-\bu\|^2_2}{2\eta_T} \\
&= \frac{D^2}{2\eta_1}  + D^2 \left(\frac{1}{2\eta_{T}}-\frac{1}{2\eta_1}\right) + \sum_{t=1}^T \frac{\eta_t \|\bg_t\|^2_2}{2}  - \frac{\|\bx_{T+1}-\bu\|^2_2}{2\eta_T} \\
&= \frac{D^2}{2\eta_{T}} + \sum_{t=1}^T \frac{\eta_t}{2} \|\bg_t\|^2_2- \frac{1}{2\eta_T} \|\bx_{T+1}-\bu\|^2_2~.
\end{align*}

In the same way, when $\eta_t$ is constant, we have
\begin{align*}
\sum_{t=1}^T &(\ell_t(\bx_t) - \ell_t(\bu))
\leq \sum_{t=1}^T \left(\frac{1}{2\eta}\|\bx_{t}-\bu\|^2_2 - \frac{1}{2\eta}\|\bx_{t+1}-\bu\|^2_2\right) + \frac{\eta}{2}\sum_{t=1}^T \|\bg_t\|^2_2 \\
&= \frac{1}{2\eta}\|\bx_{1}-\bu\|^2_2 - \frac{1}{2\eta}\|\bx_{T+1}-\bu\|^2_2 + \frac{\eta}{2}\sum_{t=1}^T \|\bg_t\|^2_2~. \qedhere
\end{align*}
\end{proof}

We can immediately observe a few things.
\begin{itemize}
\item The terms $\|\bg_t\|^2_2$ are due to the fact that we output $\bx_t$ before knowing the loss $\ell_t$. Indeed, these terms become non-positive if we were allowed to observe $\ell_t$ before producing $\bx_t$, see Section~\ref{sec:prescient_omd}.
\item In the case of constant $\eta$, the bound holds even when $D=\infty$. Moreover, the first term in the regret depends on our \emph{initial condition}, that is, how far $\bx_1$ is from the competitor $\bu$.
\item The last term in both bounds is non-positive, so it is usually discarded. Yet, some advanced proofs require it, so it is good to be aware of it.
\item If we want to use time-varying learning rates, we need to control the terms $\|\bx_{t}-\bu\|^2_2$. To do it, we assume a bounded domain $\mathcal{V}$. However, this assumption is false in most of the machine learning applications. We will also see in Section~\ref{sec:unbounded_ogd_fail} that this is not an artifact of the proof, that is, \ac{OGD} can fail in unbounded domains with time-varying learning rates. That said, in the stochastic setting, you can still use a time-varying learning rate in SGD with an unbounded domain, as we will see in Chapter~\ref{ch:o2b}.
\item Another important observation is that the regret bound helps us choose the learning rates $\eta_t$. Indeed, it is the only guideline we have. Any other choice that is not justified by a regret analysis is not justified at all.
\end{itemize}

As we said, the presence of untuned parameters like the learning rates makes no sense in online learning. So, we have to decide on a strategy to set them.
A simple choice is to find the constant learning rate that minimizes the bounds for a fixed number of iterations.
We have to consider the expression
\[
\frac{\|\bu-\bx_1\|_2^2}{2\eta} + \frac{\eta}{2} \sum_{t=1}^T \|\bg_t\|_2^2
\]
and minimize with respect to $\eta$. It is easy to see that the optimal $\eta$ is $\frac{\|\bu-\bx_1\|_2}{\sqrt{\sum_{t=1}^T \|\bg_t\|_2^2}}$, that would give the regret bound
\[
\|\bu-\bx_1\|_2\sqrt{\sum_{t=1}^T \|\bg_t\|_2^2}~.
\]
However, we have a problem: in order to use this stepsize, we should know all the future gradients (that depend on $\eta$ itself!), and the distance between the optimal solution and the initial point. This is clearly impossible: remember that the adversary can choose the sequence of functions. Hence, it can observe your choice of learning rates and decide the sequence so that your learning rate is no longer the correct one!

Indeed, it turns out that this kind of rate is completely impossible because it is ruled out by a lower bound we will show in Chapter~\ref{ch:lower}.
Yet, we will see that it is indeed possible to achieve very similar rates using \emph{adaptive} (Section~\ref{sec:lstar}) and \emph{parameter-free} algorithms (Chapter~\ref{ch:parameterfree}). For the moment, we can observe that we might be happy to minimize a loose upper bound. In particular, assume that
the norm of the gradients is bounded by $L$, that is $\|\bg_t\|_2 \leq L$.
Also, assuming a bounded diameter, we can upper bound $\|\bu-\bx_1\|_2$ by $D$. Hence, we have
\[
\eta^\star = \argmin_\eta \ \frac{D^2}{2\eta} + \frac{\eta L^2 T}{2} = \frac{D}{L \sqrt{T}},
\]
that gives a regret bound of
\begin{equation}
\label{eq:pogd_tuned_regret}
D L \sqrt{T}~.
\end{equation}
So, indeed, the regret is sublinear in time.
In Chapter~\ref{ch:lower}, we will also prove that the this regret upper bound is optimal up to constant multiplicative factors.


\begin{example}
\label{example:guessing_ogd}
\index{guessing game|(}
Consider the guessing game of the first chapter; we can easily solve it with \ac{OGD}.
Indeed, we just need to calculate the gradients, prove that they are bounded, and find a way to calculate the projection of a real number in $[0,1]$.
So, $\ell'_t(x)=2(x-y_t)$, that is bounded for $x,y_t \in [0,1]$.
The projection on $[0,1]$ is just $\Pi_{[0,1]}(x)=\min(\max(x,0),1)$.
With the optimal learning rate, the resulting regret would be $\mathcal{O}(\sqrt{T})$, which is worse than the one we found in the first chapter. Hence, \ac{OGD} with this choice of the learning rate is suboptimal for this problem. We will see how to make it optimal in Chapter~\ref{ch:beyond}.
\index{guessing game|)}
\end{example}

\begin{example}
Let's consider an example of the \ac{OCO} setting. Consider the problem of predicting at day $t$ the opening price of a stock based on a linear combination of the opening prices of the past $d$ days, represented by a vector $\bz_t \in \R^d$. So, our prediction at time $t$ will be $\langle \bz_t, \bx_t\rangle$ and the feasible set is $\mathcal{V}=\R^d$. Once we make our prediction, we receive the true opening price $y_t$, and we pay the \textbf{Huber loss}\index{Huber loss}, a convex loss function robust to outliers, on the difference between our prediction and the opening price:
\[
\ell_t(\bx) =
\begin{cases}
\frac12 (\langle \bz_t, \bx \rangle-y_t)^2, & \text{ for } |\langle \bz_t, \bx\rangle-y_t|\leq \delta,\\
\delta (|\langle \bz_t, \bx\rangle-y_t|-\frac12 \delta), & \text{otherwise}
\end{cases}
\]
where $\delta$ is the threshold that decides what counts as an outlier, and it should be penalized with a linear error rather than a quadratic one.
For this problem, we can use \ac{OGD} with constant learning rate $\eta \propto 1/\sqrt{T}$. So, we just need to calculate the gradients:
\[
\nabla \ell_t(\bx) =
\begin{cases}
(\langle \bz_t, \bx\rangle-y_t) \bz_t, & \text{ for } |\langle \bz_t, \bx\rangle-y_t|\leq \delta,\\
\delta \sign(\langle \bz_t, \bx\rangle-y_t) \bz_t , & \text{otherwise.}
\end{cases}
\]
Assuming that all $\|\bz_t\|_2$ are bounded, we also have that the gradients are bounded, satisfying the assumptions of Theorem~\ref{thm:pogd}. Hence, running \ac{OGD}, on average, the performance of the algorithm will approach the performance of any fixed linear predictor.
\end{example}
\index{Online Gradient Descent algorithm|)textbf}


\section{Online Subgradient Descent}
\label{sec:osd}

In the previous section, we introduced Projected Online Gradient Descent. However, the differentiability assumption for the $\ell_t$ is quite strong. What happens when the losses are convex but not differentiable? For example, $\ell_t(x)=|x-10|$. Note that this situation is more common than one would think. For example, the hinge loss\index{hinge loss|textbf}, $\ell_t(\bx)=\max(1-y_t\langle \bz_t,\bx\rangle,0)$,
is not differentiable.
It turns out that we can just use \ac{OGD}, substituting \emph{subgradients} for the gradients.
For this, we need some more convex analysis.

\subsection{Convex Analysis Bits: Subgradients}

\index{subgradient|(textbf}

First, we need a couple of technical definitions.
\begin{definition}
A function $f:\R^d \to [-\infty, +\infty]$ is \textbf{closed}\index{function!closed|textbf} iff $\{\bx: f(\bx) \leq \alpha\}$ is closed for every $\alpha \in \R$.
\end{definition}
Note that in any Euclidean space (and more generally in any Hausdorff space\index{Hausdorff space}) a function is closed iff it is lower semicontinuous\index{function!lower semicontinuous}~\citep[Lemma 1.24]{BauschkeC17}.

\begin{example}
The indicator function\index{indicator function} of a set $\mathcal{V} \subset \R^d$ is closed iff $\mathcal{V}$ is closed.
\end{example}

\begin{definition}
If a function\index{function!proper|textbf} $f$ is nowhere $-\infty$ and finite somewhere, then $f$ is called \textbf{proper}.
\end{definition}
In this book, we are mainly interested in convex proper functions, which better conform to our intuition of what a convex function looks like.

\begin{example}
The indicator function\index{indicator function} of a set $\mathcal{V} \subset \R^d$ is proper\index{function!proper} iff $\mathcal{V}$ is non-empty.
\end{example}

Let's now define formally what a subgradient is.
\begin{definition}
For a proper\index{function!proper} function $f:\R^d \to (-\infty, +\infty]$, we define a \textbf{subgradient} of $f$ in $\bx \in \R^d$ as a vector $\bg \in \R^d$ that satisfies
\[
f(\by)\geq f(\bx) + \langle \bg, \by-\bx\rangle, \quad \forall \by \in \R^d~.
\]
\end{definition}
Basically, a subgradient of $f$ in $\bx$ is any vector $\bg$ that allows us to construct a linear lower bound to $f$.
Note that the subgradient might not be unique, so we denote the \emph{set} of subgradients of $f$ in $\bx$ by $\partial f(\bx)$, called the \textbf{subdifferential of $f$ at $\bx$}.\index{subdifferential set|textbf}
A function $f$ is \textbf{subdifferentiable} at $\bx$\index{function!subdifferentiable|textbf} if $\partial f(\bx)\neq \emptyset$.

%
%

Observe that if $f$ is proper\index{function!proper} and convex, then $\partial f(\bx)$ is empty for $\bx \notin \dom f$, because the inequality cannot be satisfied when $f(\bx)=+\infty$. Also, the domain of $\partial f$, denoted by $\dom \partial f$, is the set of all $\bx \in \R^d$ such that $\partial f(\bx)$ is non-empty; it is a subset of $\dom f$.
A proper\index{function!proper} convex function $f$ is always\footnote{The stronger version of this theorem uses the \emph{relative interior} of the domain of $f$, however, we never need to use it, so we will keep it simple.} subdifferentiable\index{function!subdifferentiable} on $\interior \dom f$ \citep[Theorem 23.4]{Rockafellar70}.

Note that we did not assume the function $f$ to be convex in the definition of a subgradient. However, the following theorem tells us that if we have subgradients everywhere, then the function must be convex.
\begin{theorem}
\label{thm:subgradients_everywhere_implies_convexity}
Let $f:\R^d \to \R$ and $\mathcal{V}\subseteq \R^d$ be a convex set. If $\partial f(\bx)\neq \emptyset$ for all $\bx \in \mathcal{V}$, then the restriction of $f$ to $\mathcal{V}$ is convex\index{function!convex}.
\end{theorem}
\begin{proof}
For any $\bx_1,\bx_2 \in \mathcal{V}$ and for any $\lambda \in [0,1]$, consider $\by=\lambda \bx_1+(1-\lambda) \bx_2 \in \mathcal{V}$ and $\bg \in \partial f(\by)$. Then, we have
\begin{align*}
f(\bx_1) &\geq f(\by) +\langle \bg, \bx_1-\by\rangle = f(\by) + (1-\lambda)\langle \bg, \bx_1 - \bx_2\rangle,\\
f(\bx_2) &\geq f(\by) +\langle \bg, \bx_2-\by\rangle = f(\by) + \lambda\langle \bg, \bx_2 - \bx_1\rangle~.
\end{align*}
Multiplying the first inequality by $\lambda$ and the second one by $1-\lambda$ and adding them together, we have
\[
\lambda f(\bx_1) + (1-\lambda) f(\bx_2)
\geq f(\by) = f(\lambda \bx_1+(1-\lambda) \bx_2),
\]
that implies the convexity of $f$ in $\mathcal{V}$.
\end{proof}

The unique subgradient of a differentiable function is just the gradient, as quantified in the next theorem.
\begin{theorem}[{\citealp[Theorem 25.1]{Rockafellar70}}]
\label{thm:diff_subdiff}
If the function $f:\R^d \to [-\infty, +\infty]$ is convex and finite in $\bx$, it is differentiable in $\bx$ iff the subdifferential consists of a single element, that turns out to be $\nabla f(\bx)$.
\end{theorem}

\index{subgradient!sum of functions|(}
We can also calculate subgradients of the sum of functions.
\begin{theorem}
\label{thm:sum_subgradients}
Let $f_1, \dots ,f_m$ be proper functions on $\R^d$, and $F=f_1+\dots+f_m$. Then, $\partial F(\bx) \supseteq \partial f_1(\bx) + \dots + \partial f_m(\bx), \forall \bx$. Note that if one of the sets $\partial f_i(\bx)$ is empty, the right-hand side is empty and the inclusion is vacuous.
Moreover, if $f_1, \dots ,f_m$ are also convex, closed, and $\dom f_m \cap \bigcap_{i=1}^{m-1} \interior \dom f_i \neq \emptyset$, then actually $\partial F(\bx) = \partial f_1(\bx) + \dots + \partial f_m(\bx), \forall \bx$.
\end{theorem}
\begin{proof}
For any $\bz$, when the right-hand side is non-empty, choose $\bg_i \in \partial f_i(\bz)$ for $i=1, \dots,m$.
From the definition of subgradient, we have
\[
F(\bx)
= \sum_{i=1}^m f_i(\bx)
\geq \sum_{i=1}^m (f_i(\bz) + \langle \bg_i, \bx - \bz\rangle)
= F(\bz) + \left\langle \sum_{i=1}^m \bg_i, \bx-\bz\right\rangle~.
\]
Hence, $\sum_{i=1}^m \bg_i \in \partial F(\bz)$.
Since this holds for every choice of $\bg_i \in \partial f_i(\bz)$, we have $\partial F(\bz) \supseteq \partial f_1(\bz)+\dots+\partial f_m(\bz)$.

For the second statement, see \citet[Corollary 16.50]{BauschkeC17}.
\end{proof}
\index{subgradient!sum of functions|)}

\begin{example}
Let $f(x)=|x|$. Then, the subdifferential set $\partial f(x)$ is
\[
\partial f(x) = \begin{cases} \{1\}, & x>0,\\ [-1,1], & x=0,\\ \{-1\}, & x<0~. \end{cases}
\]
\end{example}

\begin{example}
\label{example:normal_cone}
Let's calculate the subgradient of the indicator function\index{indicator function!subgradients of} for a non-empty convex set $\mathcal{V} \subset \R^d$.
By definition, $\bg \in \partial \indicator_{\mathcal{V}}(\bx)$ if
\[
\indicator_{\mathcal{V}}(\by) \geq \indicator_{\mathcal{V}}(\bx) + \langle \bg, \by-\bx\rangle, \quad \forall \by \in \R^d~.
\]
This condition implies that $\bx \in \mathcal{V}$ and $0 \geq \langle \bg, \by-\bx\rangle$ for all $\by \in \mathcal{V}$ (because for $\by \notin \mathcal{V}$ the inequality is always verified). The set of all $\bg$ that satisfy the above inequality is called the \textbf{normal cone to $\mathcal{V}$ at $\bx$}\index{normal cone|textbf} and it is denoted by $\mathcal{N}_{\mathcal{V}}(\bx)$. Note that the normal cone for any $\bx\in \interior \mathcal{V}$ is $\{\boldsymbol{0}\}$ (Hint: take $\by=\bx+\epsilon\bg$). For example, for $\mathcal{V}=\{\bx \in \R^d : \|\bx\|_2\leq 1\}$, $\mathcal{N}_{\mathcal{V}}(\bx) = \{\alpha \bx:\alpha \geq0\}$ for all $\bx : \|\bx\|_2=1$.
\end{example}

\index{subgradient!max of functions|(}
Another useful theorem is to calculate the subdifferential of the pointwise maximum of convex functions.
\begin{theorem}[{\citealp[Theorem 18.5]{BauschkeC17}}]
\label{thm:subdiff_max}
Let $(f_i)_{i\in I}$ be a finite set of proper convex functions from $\R^d$ to $(-\infty, +\infty]$. Suppose $\bx \in \bigcap_{i \in I} \dom f_i$ and each $f_i$ is continuous at $\bx$. Set $F = \max_{i \in I} f_i$ and let $A(\bx) = \{i \in I | f_i(\bx)=F(\bx)\}$ be the set of the active functions. Then, denoting by $\conv$ the \textbf{convex hull}\index{convex hull|textbf}, we have
\[
\partial F(\bx) = \conv \bigcup_{i \in A(\bx)} \partial f_i(\bx)~.
\]
\end{theorem}

\begin{example}[Subgradients of the Hinge loss]
Consider the hinge loss $\ell(\bx)=\max(1-\langle \bz,\bx\rangle,0)$\index{hinge loss!subgradients of} for $\bz \in \R^d$. The subdifferential set is
\[
\partial \ell(\bx)
= \begin{cases}
\{\boldsymbol{0}\}, & 1-\langle \bz,\bx\rangle <0\\
\{-\alpha \bz : \alpha \in [0,1]\}, & 1-\langle \bz,\bx\rangle =0\\
\{-\bz\}, & \text{otherwise}
\end{cases}~.
\]
\end{example}
\index{subgradient!max of functions|)}

\index{subgradient!affine transformations|(}
Finally, we can show a result on the subgradient of affine transformations.
\begin{theorem}
\label{thm:affine_subgrad}
\index{function!affine}
Let $f:\R^m \to (-\infty, +\infty]$ be proper. Define $h(\bx)=f(\bA \bx+\bb)$, where $\bA\in \R^{m \times d}$ and $\bb \in \R^m$. Then, we have $\bA^\top \partial f(\bA \bx+\bb) \subseteq \partial h(\bx)$.
\end{theorem}
\begin{proof}
For any $\bg \in \partial f(\bA \bx +\bb)$, we want to show that $\bA^\top \bg$ is a subgradient of $h$ in $\bx$. From the definition of subgradient, for all $\by \in \R^d$ we have
\begin{align*}
h(\by)
&=f(\bA\by+\bb)
\geq f(\bA \bx+\bb) + \langle \bg, \bA\by+\bb-(\bA\bx+\bb)\rangle\\
&= h(\bx) + \langle \bA^\top \bg, \by-\bx\rangle,
\end{align*}
that implies our stated result.
\end{proof}
\index{subgradient!affine transformations|)}

\index{subgradient!Lipschitz convex function|(}
We also have a handy result that upper bounds the norm of subgradients of convex \emph{Lipschitz} functions.

\begin{definition}
Let $f:\R^d \to (-\infty, +\infty]$. We say that $f$ is \textbf{$L$-Lipschitz}\index{function!Lipschitz|textbf} over a set $\mathcal{V}\subseteq \dom f$ with respect to a norm $\|\cdot\|$ if $|f(\bx)-f(\by)|\leq L \|\bx-\by\|, \ \forall \bx,\by \in \mathcal{V}$.
\end{definition}

\begin{theorem}
\label{thm:subgradient_lipschitz}
Let $f:\R^d \to (-\infty, +\infty]$ be convex and proper. Then, $f$ is $L$-Lipschitz\index{function!Lipschitz} in $\interior \dom f$ with respect to the L$_2$ norm iff for all $\bx \in \interior \dom f$ and $\bg \in \partial f(\bx)$ we have $\|\bg\|_2\leq L$.
\end{theorem}
\begin{proof}
First, we assume $f$ to be $L$-Lipschitz, then $|f(\bx)-f(\by)|\leq L \|\bx-\by\|_2, \ \forall \bx, \by \in \interior \dom f$. Let $\bg \in \partial f(\bx)$.
If $\|\bg\|_2 =0$, then the theorem is true. So, we can safely assume it to be different from 0.
For $\epsilon>0$ small enough $\by=\bx+\epsilon \frac{\bg}{\|\bg\|_2} \in \interior \dom f$, then
\begin{align*}
L \epsilon
= L \|\bx-\by\|_2
\geq |f(\by)-f(\bx)|
\geq f(\by) - f(\bx)
\geq \langle \bg, \by -\bx\rangle
= \epsilon \|\bg\|_2,
\end{align*}
that implies that $\|\bg\|_2\leq L$.

For the other implication, the subgradient definition and Cauchy--Schwarz inequality\index{inequality!Cauchy--Schwarz} give
\[
f(\bx)-f(\by)
\leq \|\bg\|_2 \|\bx-\by\|_2
\leq L \|\bx-\by\|_2,
\]
for any $\bx,\by \in \interior \dom f$. Taking $\bg \in \partial f(\by)$, we also get
\[
f(\by)-f(\bx) \leq L \|\bx-\by\|_2~. \qedhere
\]
\end{proof}
\index{subgradient!Lipschitz convex function|)}

Finally, let's dispel the common misconception that a convex function must be differentiable on its domain except possibly at countably many points. For example, consider the convex function $f:\R^2\to\R$ defined as $f(\bx)=|x_1|$ that is not differentiable on the line segment between $\bv=(0,0)$ and $\bw=(0,1)$.
\index{subgradient|)textbf}

\subsection{Analysis with Subgradients}

\index{Online Subgradient Descent algorithm|(textbf}

As I promised you, with the proper mathematical tools, analyzing online algorithms becomes easy. Indeed, switching from gradient to subgradient comes for free! In fact, our analysis of \ac{OGD} with differentiable losses holds as is when using subgradients instead of gradients.
The reason is that the only property of the gradients that we used in the proof of Theorem~\ref{thm:pogd} was that
\[
\ell_t(\bx_t) - \ell_t(\bu) \leq \langle \bg_t, \bx_t - \bu\rangle,
\]
where $\bg_t = \nabla \ell_t(\bx_t)$. However, the exact same property holds when $\bg_t \in \partial \ell_t(\bx_t)$.
Hence, the regret bounds we proved hold as well, just replacing differentiability with subdifferentiability and gradients with subgradients.
In particular, we have the following lemma.
\begin{lemma}
\label{lemma:sd_one_step}
Let $\mathcal{V} \subseteq \R^d$ be a non-empty closed convex set, $\ell_t:\R^d \to (-\infty, +\infty]$ a function subdifferentiable on $\mathcal{V}$, and $\eta_t>0$. Set $\bg_t \in \partial \ell_t(\bx_t)$ and $\bx_{t+1} = \Pi_{\mathcal{V}}(\bx_t - \eta_t \bg_t)$. Then, $\forall \bu \in \mathcal{V}$, we have
\[
\eta_t (\ell_t(\bx_t)- \ell_t(\bu) )
\leq \eta_t \langle \bg_t, \bx_t -\bu\rangle
\leq \frac{1}{2}\|\bx_t-\bu\|^2_2 - \frac{1}{2}\|\bx_{t+1}-\bu\|^2_2 + \frac{\eta_t^2}{2}\|\bg_t\|_2^2~.
\]
\end{lemma}

So, we can state the \textbf{\ac{OSD}} algorithm in Algorithm~\ref{alg:posd}, where the only difference from Algorithm~\ref{alg:pogd} is line \ref{code:update_in_ogd}.

\begin{algorithm}[t]
\caption{Projected Online Subgradient Descent (OSD)}
\label{alg:posd}
\begin{algorithmic}[1]
{
    \REQUIRE{Non-empty closed convex set $\mathcal{V} \subseteq \R^d$, $\bx_1 \in \mathcal{V}$, $\eta_1,\dots,\eta_T>0$}
    \FOR{$t=1$ {\bfseries to} $T$}
    \STATE{Output $\bx_t \in \mathcal{V}$}
    \STATE{Pay the loss $\ell_t(\bx_t)$, where $\ell_t$ is subdifferentiable on $\mathcal{V}$}
    \STATE{Set $\bg_t \in \partial \ell_t(\bx_t)$} \label{code:update_in_ogd}
    \STATE{$\bx_{t+1} = \Pi_{\mathcal{V}}(\bx_t - \eta_t \bg_t) = \argmin_{\by \in \mathcal{V}} \ \|\bx_t - \eta_t \bg_t -\by\|_2$}
    \ENDFOR
}
\end{algorithmic}
\end{algorithm}

\begin{example}
\index{guessing game|(}
Consider the guessing game of the first chapter again, but now change the loss function to the absolute loss of the difference: $\ell_t(x)=|x-y_t|$.
Now we will need to use \ac{OSD}, because the functions are non-differentiable.
We can easily see that
\[
\partial \ell_t(x)
=\begin{cases}
\{1\}, & x>y_t\\
[-1,1], & x=y_t\\
\{-1\}, & x<y_t~.
\end{cases}
\]
Again, running \ac{OSD} with the optimal learning rate on this problem will give us immediately a regret of $\mathcal{O}(\sqrt{T})$ as $T \to \infty$, without having to design a specific strategy for it.
\index{guessing game|)}
\end{example}

\index{Online Subgradient Descent algorithm|)textbf}

\subsection{Unit Analysis Bits}

\index{unit analysis|(textbf}

One might wonder why I did not ``simplify the math'' by removing some of the constants in the bounds, for example, by setting the Lipschitz constant $L$ to $1$. While this is very common in machine learning papers, this is a bad idea because it makes it (i) difficult to check the correctness of an equation, and (ii) difficult to correctly implement the algorithm. To explain why this is the case, we will have to consider the concepts of ``units''.

When we look at a mathematical formula coming from a physical model, we are used to the idea that each quantity has a ``unit'' of measurement, e.g., meters, seconds, Joules. A fundamental principle is that any physically meaningful equation must be dimensionally consistent. That is, the units on the left-hand side of an equation must be the same as the units on the right-hand side. For example, you cannot sum a quantity measured in meters with one measured in seconds.

It turns out that this simple idea is not limited to physical models and it can be used as a sanity check for \emph{any} mathematical formula, including the ones we see in online learning. Let's assign a symbolic unit to each quantity we are dealing with. For example, let's denote the units of our predictions $\bx$ by $[\bx]$ and the units of the loss $\ell$ by $[\ell]$.
We can now check each formula by following a few simple rules. First, we have to make sure that each term has the same units.
Another rule to keep in mind is that transcendental functions are only defined on unitless quantities. So, for example, we can take logarithms and exponentials of ratios of quantities with the same units, to obtain a unitless quantity, but not of quantities with units.
Finally, while probabilities are unitless, the probability density function has the units of inverse of the random variable it represents, because it represents the probability per unit of the variable it is defined over.

The importance of this check cannot be overstated. When we see an equation where the units do not match, either it is plainly wrong, or there is a constant with units hidden somewhere. Moreover, the dimensional analysis will immediately tell us the dependence of a quantity on the others. Let's do a practical example considering \ac{OSD}.

Consider the \ac{OSD} update in Algorithm~\ref{alg:posd}, where $\mathcal{V}=\R^d$: $\bx_{t+1} = \bx_t - \eta \bg_t$.
For this equation to be coherent, the term $\eta \bg_t$ must have the same units as $\bx_t$. Let's see if this gives us some constraints on the units of the learning rate $\eta$.
The units of $\bg_t$, from the definition of subgradient, are the units of the loss divided by the units of the variable we are subdifferentiating with respect to. So, we have
\[
[\bg_t] = \frac{[\ell]}{[\bx]}~.
\]
Now, from the consistency of the update rule, we must have $[\bx_t] = [\eta \bg_t] = [\eta] [\bg_t]$. Hence, we can derive the units of the learning rate:
\[
[\eta] = \frac{[\bx]}{[\bg_t]} = \frac{[\bx]}{[\ell]/[\bx]} = \frac{[\bx]^2}{[\ell]}~.
\]
This shows that \emph{the learning rate is not a unitless quantity}. So, even without deriving the regret guarantee, from the unit analysis, we immediately know that the learning rate must depend on the problem characteristics in a certain way.

Finally, let's perform a sanity check on the regret bound for \ac{OSD} with a constant learning rate (that is, the one in Theorem~\ref{thm:pogd}):
\[
\sum_{t=1}^T (\ell_t(\bx_t) - \ell_t(\bu))
\leq \frac{\|\bu-\bx_1\|_2^2}{2\eta} + \frac{\eta}{2} \sum_{t=1}^T \|\bg_t\|^2_2~.
\]
The left-hand side is a sum of losses, so its unit is $[\ell]$. Let's check the two terms on the right-hand side, to verify that they also have the unit of $[\ell]$.
For the first term, we have
\[
\left[\frac{\|\bu-\bx_1\|_2^2}{\eta}\right] = \frac{[\bx]^2}{[\eta]} = \frac{[\bx]^2}{[\bx]^2/[\ell]} = [\ell]~.
\]
For the second term, we have
\[
\left[\eta \sum_{t=1}^T \|\bg_t\|^2_2\right] = [\eta] [\bg_t]^2 = \frac{[\bx]^2}{[\ell]} \frac{[\ell]^2}{[\bx]^2} = [\ell]~.
\]

Let me stress that having a formula with wrong units is not only a theoretical issue, but a very practical one. Indeed, constants with hidden units in any algorithm can be problematic. Let's show it with an example.
For example, suppose our variable $\bx_t$ represents a location in meters, and we use a learning rate schedule such as $\eta = 1/\sqrt{T}$. Now, suppose we change our unit of measurement for $\bx_t$ from meters to kilometers. Our new variable is $\bx'_t = \bx_t / 1000$. Since this is merely a change of units, the underlying learning problem is identical, and we should expect the algorithm's behavior to be invariant. However, the gradient with respect to the new variable is $\bg'_t = \nabla_{\bx'} \ell_t(1000 \bx'_t) = 1000 \nabla_{\bx} \ell_t(\bx_t) = 1000 \bg_t$. If we still use the dimensionless learning rate $\eta = 1/\sqrt{T}$, the update in the new coordinate system becomes $\bx'_{t+1} = \bx'_t - \eta \bg'_t = \bx_t/1000 - (1/\sqrt{T}) (1000 \bg_t)$. Converting back to meters by multiplying by 1000, we get $\bx_{t+1} = \bx_t - (1000^2/\sqrt{T}) \bg_t$. The effective learning rate has been scaled by a factor of one million, which will drastically change the algorithm's path, all because of a simple change in units.

This failure occurs because of the constant with hidden units in the learning rate, which hides the fact that we should change the learning rate when we change the units of $\bx_t$. In fact, the correct learning rate $\eta$ has units of $[\bx]^2/[\ell]$ and should have been scaled by a factor of $(1/1000)^2 = 1/1000000$ when we switched the units of $\bx_t$ from meters to kilometers.


\index{unit analysis|)textbf}

\section{Linear Regret: From Convex Losses to Linear Losses}
\label{sec:oco_to_olo}

\index{online linear optimization|(}

Let's take a deeper look at this step:
\[
\ell_t(\bx_t) - \ell_t(\bu)
\leq \langle \bg_t, \bx_t-\bu\rangle, \quad \forall \bu \in \mathcal{V}~.
\]
Summing over time, we have
\[
\sum_{t=1}^T (\ell_t(\bx_t) - \ell_t(\bu))
\leq \sum_{t=1}^T \langle \bg_t, \bx_t-\bu\rangle, \quad \forall \bu \in \mathcal{V}~.
\]
Now, define the linear (and convex) losses $\tilde{\ell}_t(\bx):=\langle \bg_t, \bx\rangle$, so we have
\[
\sum_{t=1}^T (\ell_t(\bx_t) - \ell_t(\bu))
\leq \sum_{t=1}^T (\tilde{\ell}_t(\bx_t) - \tilde{\ell}_t(\bu)), \quad \forall \bu \in \mathcal{V}~.
\]
This is more powerful than it seems: we upper-bounded the regret with respect to the convex losses $\ell_t$ with a regret with respect to another sequence of linear losses. This is important because it implies that we can build online algorithms that deal only with linear losses, and through the reduction above, they can be seamlessly used as \ac{OCO} algorithms! Note that this does not imply that this reduction is always optimal, as we saw in Example~\ref{example:guessing_ogd}. But, it allows us to easily construct optimal \ac{OCO} algorithms in many interesting cases.

So, we will often consider just the problem of minimizing the \textbf{linear regret}\index{regret!linear|textbf}
\[
\sum_{t=1}^T \langle \bg_t, \bx_t\rangle - \sum_{t=1}^T \langle \bg_t, \bu\rangle, \quad \forall \bu \in \mathcal{V},
\]
for an arbitrary sequence of vectors $\bg_1, \dots, \bg_T \in \R^d$.
This problem is called \textbf{\ac{OLO}}\index{online linear optimization|)}.

\section{History Bits}
\citet{Shor64} first introduced the subgradient descent method using a related concept of generalized gradients for non-differentiable functions.
The concept of subgradients and the possibility of developing a calculus for them appeared for the first time in \citet{Rockafellar63}.
Projected \ac{OGD} with time-varying learning rate and the name ``Online Convex Optimization''\index{online convex optimization} were introduced by \citet{Zinkevich03}, but the \ac{OCO} framework was introduced earlier by \citet{Gordon99, Gordon99b}, which also introduced the idea of reducing \ac{OCO} to \ac{OLO}.

Before the \ac{OCO} framework, the online learning community focused on specific losses and mostly linear predictors, see \citet{Cesa-BianchiL06}. Moreover, the concept of subgradients is also a recent addition to the online learning literature, and the earliest papers using subgradients of loss functions seem to be \citet{Zhang04b,ShalevS06b}.
However, subgradients were implicitly used in previous analyses too, see for example \citet[Theorem 4]{Cesa-Bianchi99}.

It is impossible not to notice that the literature on \ac{OCO} is heavily biased toward the use of bounded domains.
Yet, there are very few practical applications of \ac{OCO} where the domain is bounded. Moreover, as shown in Theorem~\ref{thm:pogd}, it is possible to have a sublinear regret even in unbounded domains.
I speculate that the insistence on using bounded domains is due to two main reasons: (i) the less general definition of regret (derived from the setting of learning with experts (Section~\ref{sec:lea})) as
\[
\sum_{t=1}^T \ell_t(\bx_t) - \min_{\bu \in \mathcal{V}}\sum_{t=1}^T \ell_t(\bu)
\]
that requires the existence of a minimizer of the sum of the losses, and (ii) the fact that \ac{OSD} with time-varying learning rates has a vacuous regret upper bound on unbounded domains. However, we now know that both issues can easily be fixed: (i) by defining the regret with respect to arbitrary competitors, and (ii) by using different algorithms, for example \ac{FTRL} which we will see in Chapter~\ref{ch:ftrl}. Indeed, even in offline optimization, the existence of a minimizer is not required to guarantee that the suboptimality gap goes to zero, as we will see in Example~\ref{ex:svm_with_oco}. So, it is probably time for the community to stop using the crutch of bounded domains just to avoid technical difficulties.

\section{Exercises}

\begin{exer}
Prove that $\sum_{t=1}^T \frac{1}{\sqrt{t}} \leq 2\sqrt{T}-1$.
\end{exer}

\begin{exer}
Using the inequality in the previous exercise, prove that the learning rate $\eta_t= \frac{D}{L\sqrt{t}}$ gives rise to a regret only a constant multiplicative factor worse than the one in \eqref{eq:pogd_tuned_regret}.
\end{exer}

\begin{exer}
Calculate the subdifferential set of the $\epsilon$-insensitive loss\index{epsilon-insensitive loss@$\epsilon$-insensitive loss}: $f(x)=\max(|x-y|-\epsilon,0)$. It is a loss used in regression problems where we do not want to penalize predictions $x$ within $\pm \epsilon$ of the correct value $y$.
\end{exer}

\begin{exer}
Using the definition of subgradient, find the subdifferential set of $f(\bx)=\|\bx\|_2$, $\bx \in \R^d$.
\end{exer}

\begin{exer}
Consider Projected Online Subgradient Descent for the Example~\ref{ex:failure_ftl} on the failure of Follow-the-Leader: can we use it on that problem? Would it guarantee sublinear regret? How would the behaviour of the algorithm differ from \ac{FTL}?
\end{exer}

\acresetall

\chapter{Online-to-Batch Conversions}
\label{ch:o2b}

In this chapter, we take a break from online learning theory and see some applications of online learning to other domains.
For example, we may wonder what the connection is between online learning and stochastic optimization. Given that projected \ac{OSD} looks basically the same as projected stochastic (sub)gradient descent, they must have something in common.
Indeed, for example, we can reduce stochastic optimization of convex functions to \ac{OCO}. Let's see how.

\acresetall

\section{From Online Learning to Stochastic Optimization}

\index{online-to-batch conversion|(textbf}

\begin{theorem}
\label{thm:o2b}
Let $\mathcal{V}$ be a non-empty closed convex set of $\R^d$, $F(\bx)=\E[f(\bx,\bxi)]$ where the expectation is with respect to $\bxi$ drawn from a distribution $\rho$ over some vector space $\mathcal{D}$, and $f:\R^d \times \mathcal{D}\to (-\infty,+\infty]$ is convex and subdifferentiable in the first argument on $\mathcal{V}$. Draw $T$ samples $\bxi_1,\dots,\bxi_T$ i.i.d. from $\rho$ and construct the sequence of losses $\ell_t(\bx) = \alpha_t f(\bx,\bxi_t)$, where $\alpha_t>0$ are deterministic. Run any \ac{OCO} algorithm over the losses $\ell_t$ on $\mathcal{V}$, to construct the sequence of predictions $\bx_1,\dots,\bx_{T}$. The algorithm may be randomized, but its internal randomness is independent of $\bxi_1,\dots,\bxi_T$.
Then, we have
\[
\E\left[F\left(\frac{1}{\sum_{t=1}^T \alpha_t}\sum_{t=1}^T \alpha_t \bx_t\right)\right]
\leq F(\bu) + \frac{\E[\Regret_T(\bu)]}{\sum_{t=1}^T \alpha_t}, \quad \forall \bu \in \mathcal{V},
\]
where the expectation is with respect to $\bxi_1,\dots, \bxi_T$ and the internal randomness of the algorithm.
\end{theorem}
\begin{proof}
We first show that
\begin{equation}
\label{eq:o2b_1}
\E\left[\sum_{t=1}^T \alpha_t F(\bx_t)\right]
= \E\left[ \sum_{t=1}^T \ell_t(\bx_t)\right]~.
\end{equation}
In fact, from the linearity of the expectation we have
\[
\E\left[ \sum_{t=1}^T \ell_t(\bx_t)\right]
= \sum_{t=1}^T \E\left[\ell_t(\bx_t)\right]~.
\]
Then, from the law of total expectation, letting $\mathscr{F}_{t-1}$ be the sigma-algebra generated by $\bxi_1,\dots,\bxi_{t-1}$ and the internal random bits used by the algorithm to compute $\bx_t$, we have
\begin{align*}
\E\left[\ell_t(\bx_t)\right]
= \E\left[\E\left[\ell_t(\bx_t)\mid \mathscr{F}_{t-1}\right]\right]
= \E\left[\E\left[\alpha_t f(\bx_t,\bxi_t)\mid \mathscr{F}_{t-1}\right]\right]
= \E\left[\alpha_t F(\bx_t)\right],
\end{align*}
where we used the fact that $\bx_t$ is $\mathscr{F}_{t-1}$-measurable, and that $\bxi_t$ is independent of $\mathscr{F}_{t-1}$.
Hence, \eqref{eq:o2b_1} is proved.

It remains only to use Jensen's inequality\index{inequality!Jensen's} (Theorem~\ref{thm:jensen}), using the fact that $F$ is convex, to have
\[
F\left(\frac{1}{\sum_{t=1}^T \alpha_t} \sum_{t=1}^T \alpha_t \bx_t\right)
\leq  \frac{1}{\sum_{t=1}^T \alpha_t} \sum_{t=1}^T \alpha_t F(\bx_t)~.
\]
Dividing the regret by $\sum_{t=1}^T \alpha_t$ and using the above inequalities gives the stated theorem.
\end{proof}

Let's now see some applications of this result. Let's use the above theorem to transform \ac{OSD} into Stochastic Subgradient Descent to minimize the training loss of a classifier.
\begin{example}
\label{ex:svm_with_oco}
Consider a problem of binary classification, with inputs $\bz_i \in \R^d$ and outputs $y_i \in \{-1,1\}$. The loss function is the hinge loss\index{hinge loss}: $f(\bx, (\bz,y)) = \max(1-y\langle \bz, \bx\rangle,0)$.
Suppose that you want to minimize the training loss over a training set of $N$ samples, $\{(\bz_i,y_i)\}_{i=1}^N$. Also, assume the maximum L$_2$ norm of the samples is $R$. That is, we want to minimize
\[
\min_{\bx} \ \left(F(\bx):= \frac{1}{N}\sum_{i=1}^N \max(1-y_i\langle \bz_i, \bx\rangle,0)\right)~.
\]
Run the reduction described in Theorem~\ref{thm:o2b} for $T$ iterations using \ac{OSD}. In each iteration, construct $\ell_t(\bx)=\max(1-y_t\langle \bz_t, \bx\rangle,0)$ by sampling a training point uniformly at random from $\{1,\dots, N\}$. Set $\bx_1=\boldsymbol{0}$ and $\eta=\frac{1}{R\sqrt{T}}$.
Assuming $\argmin_{\bx} F(\bx)\neq \emptyset$, we have that
\[
\E\left[F\left(\frac{1}{T}\sum_{t=1}^T \bx_t\right)\right] - F(\bx^\star)
\leq R\frac{\|\bx^\star\|^2_2 + 1}{2\sqrt{T}},
\]
for all $\bx^\star \in \argmin_{\bx} F(\bx)$.
In words, we used an \ac{OCO} algorithm to stochastically optimize a function, transforming the regret guarantee into a convergence rate guarantee.
\end{example}

In this last example, we have to use a constant learning rate to minimize the training loss over the entire space $\R^d$. In the next one, we will see a different approach that allows us to \emph{implicitly} use a varying learning rate without the need of a bounded feasible set.
\begin{example}
\label{ex:non_uniform_o2b}
Consider the same setting of the previous example, and let's change the way in which we construct the online losses. Now, use $\ell_t(\bx)=\frac{1}{R\sqrt{t}} \max(1-y_t\langle \bz_t, \bx\rangle,0)$ and step size $\eta=1$. Hence, we have
\begin{align*}
\E\left[F\left(\frac{1}{\sum_{t=1}^T \frac{1}{\sqrt{t}} }\sum_{t=1}^T \frac{1}{\sqrt{t}} \bx_t\right)\right] - F(\bx^\star)
&\leq \frac{\|\bx^\star\|_2^2}{2\sum_{t=1}^T \frac{1}{R\sqrt{t}}} +\frac{1}{2 \sum_{t=1}^T \frac{1}{R\sqrt{t}}} \sum_{t=1}^T \frac{1}{t}\\
&\leq R\frac{\|\bx^\star\|_2^2+1+\ln T}{4 \sqrt{T+1}-4},
\end{align*}
where we used $\sum_{t=1}^T \frac{1}{\sqrt{t}} \geq 2\sqrt{T+1}-2$.
\end{example}

\begin{remark}
Using the online-to-batch conversion with online subgradient descent to minimize the expectation of convex Lipschitz functions we can obtain a convergence rate of $\mathcal{O}(\frac{1}{\sqrt{T}})$, that is optimal for this class of problems. This should dispel the common misconception that online algorithms are suboptimal in the stochastic setting because they are designed to work in the adversarial case. Indeed, the opposite is true: virtually any optimal guarantee for offline optimization can be recovered using online learning algorithms.
\end{remark}

I stressed the fact that the only meaningful way to define a regret is with respect to an arbitrary point in the feasible set. This is obvious in the case where we consider unconstrained \ac{OLO}, because the optimal competitor is unbounded. But, it is also true in unconstrained \ac{OCO}. Let's see an example of this.
\begin{example}
\label{example:logistic_regression}
Consider a problem of binary classification, with inputs $\bz_i \in \R^d$ and outputs $y_i \in \{-1,1\}$. The loss function is the logistic loss\index{logistic loss|textbf}: $f(\bx, (\bz,y)) = \ln(1+\exp(-y\langle \bz, \bx\rangle))$.
Suppose that you want to minimize the training loss over a training set of $N$ samples, $\{(\bz_i,y_i)\}_{i=1}^N$. Also, assume the maximum L$_2$ norm of the samples is $R$. That is, we want to minimize
\[
\min_{\bx} \ \left(F(\bx):= \frac{1}{N}\sum_{i=1}^N \ln(1+\exp(-y_i\langle \bz_i, \bx\rangle))\right)~.
\]
So, run the reduction described in Theorem~\ref{thm:o2b} for $T$ iterations using \ac{OSD}. In each iteration, construct $\ell_t(\bx)=\ln(1+\exp(-y_t\langle \bz_t, \bx\rangle))$ by sampling a training point uniformly at random from $\{1,\dots,N\}$. Set $\bx_1=\boldsymbol{0}$ and $\eta=\frac{1}{R\sqrt{T}}$. We have that
\[
\E\left[F\left(\frac{1}{T}\sum_{t=1}^T \bx_t\right)\right]
\leq \frac{R}{2\sqrt{T}} + \min_{\bu \in \R^d} \left(F(\bu) + R\frac{\|\bu\|^2_2}{2\sqrt{T}}\right)~.
\]
In words, we will be $\frac{R}{2\sqrt{T}}$ away from the optimal value of the function $F$ plus a \emph{regularizer} $\|\bu\|^2_2$, where the weight of the regularization is $\frac{R}{2\sqrt{T}}$.
Now, let's consider the case that the training set is linearly separable, which means that the infimum of $F$ is 0 and the optimal solution does not exist, i.e., informally speaking, it has norm equal to infinity. So, any convergence guarantee that depends on $\bx^\star$ would be vacuous. On the other hand, our guarantee above still makes perfect sense.
\end{example}

Note that the above examples only deal with training loss. However, in the next sections we show a more interesting application of the online-to-batch conversion, that is to directly minimize the generalization error. Moreover, we will see guarantees in high probability, rather than just in expectation.

\subsection{Bits on Concentration Inequalities}

We will use a concentration inequality to prove the high probability guarantee, but we will need to go beyond the sum of i.i.d. random variables. In particular, we will use the concept of \emph{martingales}.

Let $(\Omega, (\mathscr{F}_t)_{t=0}^T, \Pr)$ be a filtered probability space, where $T \in \Nat \cup \{\infty\}$, and let $Z_0, Z_1, \dots$ be random variables on $\Omega$. Recall that the sequence $(Z_t)_{t=0}^T$ is adapted if $Z_t$ is $\mathscr{F}_t$-measurable for all $0 \le t \le T$.

\begin{definition}
A sequence of random variables $Z_0, Z_1, \dots$ is a \textbf{martingale}\index{martingale|textbf} with respect to the filtration $(\mathscr{F}_t)_{t=0}^\infty$ if it is adapted to $(\mathscr{F}_t)_{t=0}^\infty$ and, for all $t\geq0$, it satisfies:
\begin{align*}
&\E[|Z_t|]< \infty,
&\E[Z_{t+1}|\mathscr{F}_{t}] = Z_t, \text{ almost surely.}
\end{align*}
\end{definition}

\begin{definition}
A sequence of random variables $Z_0, Z_1, \dots$ is a \textbf{supermartingale}\index{supermartingale|textbf} with respect to the filtration $(\mathscr{F}_t)_{t=0}^\infty$ if it is adapted to $(\mathscr{F}_t)_{t=0}^\infty$ and, for all $t\geq0$, it satisfies:
\begin{align*}
&\E[|Z_t|]< \infty,
&\E[Z_{t+1}|\mathscr{F}_{t}] \leq Z_t, \text{ almost surely.}
\end{align*}
\end{definition}

\begin{example}
\label{example:wealth_martingale}
Consider a fair coin where each $c_t\in\{-1,1\}$ is independent of the past, and an online betting algorithm that bets $|x_t|$ money on each round on the side of the coin equal to $\sign(x_t)$. We win or lose money at even odds, so the net gain up to round $t\geq 1$ is $Z_t=\sum_{i=1}^t c_i x_i$ and $Z_0=0$.
Define $\mathscr{F}_t=\sigma(c_1,\dots,c_t)$.
Since we use a deterministic online betting algorithm, $x_t$ depends only on the past, and therefore $x_t$ is $\mathscr{F}_{t-1}$-measurable.
So, we have
\[
\E[Z_t|\mathscr{F}_{t-1}]
=\E[Z_{t-1} + x_t c_t|\mathscr{F}_{t-1}]
=Z_{t-1} + \E[x_t c_t|\mathscr{F}_{t-1}]
=Z_{t-1}~.
\]
Hence, $(Z_t)_t$ is a martingale\index{martingale}.
If we throw away part of the wealth in each round, we obtain a supermartingale\index{supermartingale}.
\end{example}

For martingales with bounded differences we can prove high probability guarantees as for bounded i.i.d. random variables. The following theorem will be the key result we will need.
\begin{theorem}[Hoeffding--Azuma inequality]
\label{thm:azuma}
\index{inequality!Hoeffding--Azuma|textbf}
Let $Z_0, \dots, Z_{T}$ be a martingale\index{martingale}, where $|Z_t - Z_{t-1}| \leq B, t=1, \dots, T$ almost surely. Then, we have
\[
\Pr\{Z_T - Z_0 \geq \epsilon\}
\leq \exp\left(-\frac{\epsilon^2}{2 B^2 T}\right)~.
\]
Also, the same upper bound holds on $\Pr\{Z_0 - Z_T \geq \epsilon\}$.
\end{theorem}

\subsection{High-Probability Guarantees for Online-to-Batch Conversion}

We now show how the online-to-batch conversion we have introduced can be strengthened to produce guarantees in high probability.
\begin{theorem}
\label{thm:o2b_high_prob}
Let $\mathcal{V} \subset \R^d$, $F(\bx)=\E[f(\bx,\bxi)]$, where the expectation is with respect to $\bxi$ drawn from $\rho$ with support over some set $\mathcal{D}$, $f:\mathcal{V} \times \mathcal{D}\to [0,1]$, and $f(\bx, \cdot)$ measurable for every fixed $\bx$. Draw $T$ samples $\bxi_1,\dots,\bxi_T$ i.i.d. from $\rho$ and construct the sequence of losses $\ell_t(\bx) = f(\bx,\bxi_t)$. Let $\mathscr{A}$ be any deterministic online learning algorithm over the losses $\ell_t$ that outputs the sequence of predictions $\bx_1,\dots,\bx_{T+1}$ and guarantees $\Regret_T(\bu) \leq R(\bu,T)$ for all $\bu \in \mathcal{V}$, for a function $R: \mathcal{V} \times \Nat \to \R$.
Let $\delta \in (0,1)$. Assume that $F(\bu) + \frac{R(\bu,T)}{T}$ has a minimizer in $\mathcal{V}$. Then, with probability at least $1-\delta$, it holds that
\[
\frac{1}{T}\sum_{t=1}^T F(\bx_t)
\leq \min_{\bu \in \mathcal{V}} \ F(\bu) + \frac{R(\bu,T)}{T} + 2\sqrt{\frac{2\ln\frac{2}{\delta}}{T}}~.
\]
\end{theorem}
\begin{proof}
Let $\mathscr{F}_{t-1}$ be the sigma-algebra generated by $\bxi_1,\dots,\bxi_{t-1}$.
Since $\mathscr{A}$ is an online algorithm, $\bx_t$ depends only on the past samples $\bxi_1,\dots,\bxi_{t-1}$, and therefore $\bx_t$ is $\mathscr{F}_{t-1}$-measurable.

Define
\[
Z_t:=\sum_{i=1}^t \bigl(F(\bx_i)-\ell_i(\bx_i)\bigr),
\quad t=1,\dots,T,
\]
and $Z_0:=0$.
We claim that $(Z_t)_{t=0}^T$ is a martingale\index{martingale} with respect to the filtration $(\mathscr{F}_t)_{t=0}^T$.

First, $Z_t$ is $\mathscr{F}_t$-measurable for every $t$, since for each $i\leq t$, the random variable $\bx_i$ is $\mathscr{F}_{i-1}$-measurable and $\ell_i(\bx_i)=f(\bx_i,\bxi_i)$ is $\mathscr{F}_i$-measurable.
Moreover, $|F(\bx_i)-\ell_i(\bx_i)|\leq 1$, because both $F(\bx_i)$ and $\ell_i(\bx_i)$ belong to $[0,1]$, so $Z_t$ is integrable.

Now observe that
\[
\E[\ell_t(\bx_t)\mid \mathscr{F}_{t-1}]
= \E[f(\bx_t,\bxi_t)\mid \mathscr{F}_{t-1}]
= F(\bx_t),
\]
where we used that $\bx_t$ is $\mathscr{F}_{t-1}$-measurable and that $\bxi_t$ is independent of $\mathscr{F}_{t-1}$.
Hence,
\begin{align*}
\E[Z_t\mid \mathscr{F}_{t-1}]
&=\E\left[Z_{t-1}+F(\bx_t)-\ell_t(\bx_t) \middle| \mathscr{F}_{t-1}\right]\\
&=Z_{t-1}+F(\bx_t)-\E[\ell_t(\bx_t)\mid \mathscr{F}_{t-1}]
=Z_{t-1}~.
\end{align*}
So $(Z_t)_{t=0}^T$ is indeed a martingale.

Also, for every $t$, $|Z_t-Z_{t-1}|=|F(\bx_t)-\ell_t(\bx_t)|\leq 1$.
Therefore, by Theorem~\ref{thm:azuma}, we have
\[
\Pr\left\{\sum_{t=1}^T \bigl(F(\bx_t)-\ell_t(\bx_t)\bigr)\geq \epsilon\right\}
=\Pr\{Z_T-Z_0\geq \epsilon\}
\leq \exp\left(-\frac{\epsilon^2}{2T}\right)~.
\]
Thus, with probability at least $1-\delta/2$,
\[
\frac{1}{T}\sum_{t=1}^T F(\bx_t)
\leq \frac{1}{T}\sum_{t=1}^T \ell_t(\bx_t)+\sqrt{\frac{2\ln\frac{2}{\delta}}{T}}~.
\]

Now fix any $\bu\in\mathcal{V}$. By the regret guarantee,
\[
\frac{1}{T}\sum_{t=1}^T \ell_t(\bx_t)
= \frac{\Regret_T(\bu)}{T} +\frac{1}{T}\sum_{t=1}^T \ell_t(\bu)
\leq \frac{R(\bu,T)}{T} +\frac{1}{T}\sum_{t=1}^T \ell_t(\bu)~.
\]

It remains to upper bound $\frac{1}{T}\sum_{t=1}^T \ell_t(\bu)$ by $F(\bu)$ with high probability.
Choose now $\bu$ to be any fixed minimizer of $F(\bx)+\frac{R(\bx,T)}{T}$ over $\mathcal{V}$.
Since this $\bu$ is deterministic, the random variables $\ell_t(\bu)=f(\bu,\bxi_t)$ are i.i.d.
Define $Y_t := \sum_{i=1}^t \bigl(F(\bu)-\ell_i(\bu)\bigr)$ for $t=1,\dots,T$, and $Y_0:=0$.
Then, $(Y_t)_{t=0}^T$ is a martingale\index{martingale} with respect to the same filtration $(\mathscr{F}_t)_{t=0}^T$, because
\[
\E[\ell_t(\bu)\mid\mathscr{F}_{t-1}]
= \E[\ell_t(\bu)]
= F(\bu)~.
\]
Moreover, $|Y_t-Y_{t-1}|\leq 1$, so Azuma's inequality gives that, with probability at least $1-\delta/2$,
\[
\frac{1}{T}\sum_{t=1}^T \ell_t(\bu)
\leq F(\bu)+\sqrt{\frac{2\ln\frac{2}{\delta}}{T}}~.
\]

Putting everything together and using the \emph{union bound} (for events $A_1,\dots,A_n$, we have $\Pr(\cup_{i=1}^n A_i)\leq \sum_{i=1}^n \Pr(A_i)$)\index{union bound|textbf}, we have the stated bound.
\end{proof}

The theorem above upper bounds the average value of the $T$ different solutions, while we are interested in producing a single solution.
If $F$ is a convex function and $\mathcal{V}$ is convex, then we can lower bound the l.h.s. of the inequalities in the theorem with the function evaluated on the average of the $\bx_t$. That is
\[
F\left(\frac{1}{T}\sum_{t=1}^T \bx_t\right)
\leq \frac{1}{T} \sum_{t=1}^T F(\bx_t)~.
\]

\begin{remark}
Note that using the online-to-batch conversion to optimize a population risk, i.e., $F(\bx)=\E_{\bxi \sim \rho}[f(\bx,\bxi)]$, is not the same as using the online-to-batch conversion to minimize the empirical risk, $\hat{F}(\bx)=\frac{1}{T}\sum_{t=1}^T f(\bx,\bxi_t)$. In both cases, we first sample $T$ random vectors $\mathcal{S}=\{\bxi_1, \dots, \bxi_T\}$ from $\rho$. Then, to minimize the empirical risk we sample \emph{with replacement} from $\mathcal{S}$. Instead, to minimize the population risk we sample uniformly \emph{without replacement} from $\mathcal{S}$, which is equivalent to using $T$ i.i.d. samples from $\rho$. This is not a minor difference because one can show that there are cases where sampling without replacement will not minimize the empirical risk~\citep{Vansover-HagerKL25}.
\end{remark}

\index{online-to-batch conversion|)textbf}

\section{Application: Agnostic PAC Learning}
\label{sec:agnostic_pac}

In this section, we show another application of online-to-batch methods to obtain statistical learning guarantees.
Here, we assume that we have a prediction strategy $\phi_{\bx}$ parametrized by a vector $\bx$ and we want to learn the relationship between an input $\bz$ and its associated label $y$. Moreover, we will assume that $(\bz,y)$ is drawn from a joint probability distribution $\rho$. Also, we are equipped with a loss function, $\ell(\hat{y},y)$, that measures how good our prediction $\hat{y}=\phi_{\bx}(\bz)$ is, compared to the true label $y$.
So, learning the relationship can be cast as minimizing the expected loss of our predictor
\[
\min_{\bx \in \mathcal{V}} \ \E_{(\bz,y)\sim\rho}[\ell(\phi_{\bx}(\bz), y)]~.
\]
In machine learning terms, the object above is nothing else than the \emph{expected test loss} of our predictor.

Note that the above setting assumes labeled samples, but we can generalize it further by considering \emph{Vapnik's general setting of learning}\index{Vapnik's general setting of learning}, where we collapse the prediction function and the loss into a single function. This allows, for example, to treat supervised and unsupervised learning in the same unified way.
So, we want to minimize the \textbf{risk}\index{risk|textbf}
\[
\min_{\bx \in \mathcal{V}} \ \left(\Risk(\bx):=\E_{\bxi\sim\rho} [f(\bx,\bxi)]\right),
\]
where $\rho$ is an unknown distribution over $\mathcal{D}$ and $f:\mathcal{V} \times \mathcal{D} \to \R$ is such that for every fixed $\bx$, $f(\bx,\cdot)$ is measurable. Also, the set $\mathcal{H}$ of the composition of a loss with all predictors that can be expressed by vectors $\bx$ in $\mathcal{V}$ is called the \textbf{hypothesis class}\index{hypothesis class|textbf}.

\begin{example}
In a linear regression task where the loss is the square loss, we have $\bxi=(\bz, y) \in \R^d \times \R$ and $\phi_{\bx}(\bz)=\langle \bz,\bx\rangle$. Hence, $f(\bx,\bxi)=(\langle \bz,\bx\rangle-y)^2$.
\end{example}

\begin{example}
In linear binary classification where the loss is the hinge loss\index{hinge loss}, we have $\bxi=(\bz, y) \in \R^d \times \{-1, 1\}$ and $\phi_{\bx}(\bz)=\langle \bz, \bx\rangle$. Hence, $f(\bx,\bxi)=\max(1-y\langle \bz,\bx\rangle,0)$.
\end{example}

\begin{example}
In binary classification with a neural network with the logistic loss\index{logistic loss}, we have $\bxi=(\bz,y) \in \R^d \times \{-1,1\}$ and $\phi_{\bx}$ is the network corresponding to the weights $\bx$.
Hence, $f(\bx,\bxi)=\ln(1+ \exp(-y \phi_{\bx}(\bz)))$.
\end{example}

The key difficulty of the above problem is that we do not know the distribution $\rho$.
Hence, there is no hope to exactly solve this problem. Instead, we are interested in understanding \emph{what is the best we can do if we have access to $T$ samples drawn i.i.d. from $\rho$}. More precisely, we want to upper bound the \textbf{excess risk}\index{excess risk|textbf}
\[
\Risk(\bx_T) - \inf_{\bx \in \mathcal{V}} \ \Risk(\bx),
\]
where $\bx_T$ is a predictor that was \emph{learned} using $T$ samples.

It should be clear that this is just an optimization problem and the one above is just the suboptimality gap. In this view, the objective of machine learning can be considered as a particular optimization problem.

\begin{remark}
Note that this is not the only way to approach the problem of learning. Indeed, the regret minimization model is an alternative model to learning. Moreover, another approach would be to try to estimate the distribution $\rho$ and then solve the risk minimization problem. No approach is superior to the other and each of them has its pros and cons.
\end{remark}

Given that we have access to the distribution $\rho$ through samples drawn from it, any procedure we might think to use to minimize the risk will be stochastic in nature. This means that we cannot provide a deterministic guarantee. Instead, \emph{we can try to prove that with high probability our minimization procedure will return a solution that is close to the minimizer of the risk}. It is also intuitive that the precision and probability we can guarantee must depend on how many samples we draw from $\rho$.

Quantifying the dependence of precision and probability of failure on the number of samples used is the objective of the \emph{agnostic Probably Approximately Correct} (PAC) framework, where the keyword ``agnostic'' refers to the fact that we do not assume anything about the best possible predictor. More precisely, given a precision parameter $\epsilon$ and a probability of failure $\delta$, we are interested in characterizing the \emph{sample complexity of the hypothesis class $\mathcal{H}$}\index{hypothesis class!sample complexity of}. A sample complexity of $T$ implies the existence of a learning algorithm using the hypothesis class $\mathcal{H}$ and outputting a solution $\bx_T$ using $T$ samples that has an excess risk upper bounded by $\epsilon$ with probability at least $1-\delta$  for every distribution $\rho$ over $\mathcal{D}$.
Note that the sample complexity does not depend on $\rho$, so it is a worst-case measure with respect to all the possible distributions.
This makes sense if you think that we know nothing about the distribution $\rho$, so if your guarantee holds for the worst distribution it will also hold for any other distribution.
Mathematically, we will say that the hypothesis class is agnostic PAC-learnable if such a sample-complexity function exists.
\begin{definition}
We will say that a hypothesis class $\mathcal{H}=\{ f(\bx,\cdot): \bx\in \mathcal{V}\}$ is \textbf{Agnostic-PAC-learnable}\index{PAC learning!agnostic} if there exists an algorithm $\mathscr{A}$ and a function $T(\epsilon, \delta):\R_{>0} \times (0,1) \to \Nat$ such that, for every distribution $\rho$ over $\mathcal{D}$, when $\mathscr{A}$ is used with $T\geq T(\epsilon, \delta)$ samples drawn i.i.d. from $\rho$, with probability at least $1-\delta$ the solution returned by the algorithm has excess risk at most $\epsilon$.
\end{definition}

Note that the Agnostic PAC learning setting does not say what procedure we should follow to find such sample complexity.
The approach most commonly used in machine learning to solve the learning problem is the so-called \textbf{\ac{ERM}} procedure\index{empirical risk!minimization|textbf}. It consists of drawing $T$ samples i.i.d. from $\rho$ and minimizing the \textbf{empirical risk}\index{empirical risk|textbf} defined as
\[
\widehat{\Risk}(\bx)
:= \frac{1}{T} \sum_{t=1}^T f(\bx,\bxi_t)~.
\]
The minimizer $\hat{\bx}_T = \argmin_{\bx \in \mathcal{V}} \ \widehat{\Risk}(\bx)$ is called the \emph{empirical risk minimizer}\index{empirical risk!minimizer of}.
In words, \ac{ERM} is nothing else than the minimization of some loss function on a training set.
However, in many interesting cases $\argmin_{\bx \in \mathcal{V}} \ \frac{1}{T} \sum_{t=1}^T f(\bx,\bxi_t)$ can be very far from the true optimum $\argmin_{\bx \in \mathcal{V}} \ \E[f(\bx,\bxi)]$. So, we need to modify the \ac{ERM} formulation in some way, e.g., using a \emph{regularization} term, a Bayesian prior of $\bx$, or more generally find conditions under which \ac{ERM} works.


The \ac{ERM} approach is so widespread that machine learning itself is often wrongly identified with some kind of minimization of the training loss. We now show that \ac{ERM} is not the entire world of machine learning, showing that \emph{the existence of a no-regret algorithm, that is an online learning algorithm with sublinear regret, guarantees Agnostic PAC learnability}. More precisely, we will see that an online algorithm with sublinear regret can be used to solve machine learning problems. This is not just a curiosity, for example this gives rise to computationally efficient algorithms, that can be achieved through \ac{ERM} only by running a two-step procedure, i.e., running \ac{ERM} with different parameters and selecting the best solution among them.


We can use Theorem~\ref{thm:o2b_high_prob} to produce a solution with small risk. In particular, if the risk is convex, we can output the average of the $\bx_t$, using Jensen's inequality\index{inequality!Jensen's}.

If the risk is not a convex function, we need a different way.
An alternative solution is to construct a \emph{stochastic predictor} that samples one of the $\bx_t$ uniformly at random and predicts with it. So, defining $I\sim \operatorname{Unif}(\{1,\dots,T\})$, we immediately have
\[
\E_I[\Risk(\bx_I)]
= \frac{1}{T} \sum_{t=1}^T \Risk(\bx_t)~.
\]

Yet another way is to try to find among the $T$ predictors the one with the smallest risk. This works because the average is lower bounded by the minimum. We can use $T/2$ samples for the online learning procedure and $T/2$ samples to generate a validation set\index{validation set} to evaluate the solution and pick the best one. The following theorem shows that selecting the predictor with the smallest empirical risk on a validation set will give us a predictor close to the best one with high probability.
\begin{theorem}
\label{thm:agnostic_pac_finite}
Let $\mathcal{V} \subset \R^d$, $\Risk(\bx)=\E[f(\bx,\bxi)]$, where the expectation is with respect to $\bxi$ drawn from $\rho$ with support over some set $\mathcal{D}$, and $f:\mathcal{V} \times \mathcal{D}\to [0,1]$. Let $\delta \in (0,1)$.
We have a finite set of vectors $\mathcal{S}=\{\bx_1, \dots, \bx_{|\mathcal{S}|}\}$ fixed before seeing $T$ random vectors $\bxi_1, \dots, \bxi_T$ drawn i.i.d. from $\rho$.
Denote by $\hat{\bx}=\argmin_{\bx \in \mathcal{S}} \ \widehat{\Risk}(\bx)$, where $\widehat{\Risk}(\bx)=\frac{1}{T} \sum_{t=1}^T f(\bx,\bxi_t)$. Then, with probability at least $1-\delta$, we have
\[
\Risk(\hat{\bx})
\leq \min_{\bx \in \mathcal{S}} \ \Risk(\bx) + 2\sqrt{\frac{2\ln(2|\mathcal{S}|/\delta)}{T}}~.
\]
\end{theorem}
\begin{proof}
We want to calculate the probability that the hypothesis that minimizes the validation error is far from the best hypothesis in the set. We cannot do it directly because we do not have the required independence to use a concentration inequality. Instead, \emph{we will upper bound the probability that there exists at least one function whose empirical risk is far from the risk.}
So, using the union bound\index{union bound}, we have
\begin{align*}
\Pr\left\{\exists \bx \in \mathcal{S}: |\Risk(\bx)-\widehat{\Risk}(\bx)| > \frac{\epsilon}{2}\right\}
&\leq \sum_{i=1}^{|\mathcal{S}|} \Pr\left\{|\Risk(\bx_i)-\widehat{\Risk}(\bx_i)| > \frac{\epsilon}{2}\right\}\\
&\leq 2 |\mathcal{S}| \exp\left(-\frac{\epsilon^2T}{8}\right)~.
\end{align*}
Hence, with probability at least $1-\delta$, we have that
\[
|\Risk(\bx)-\widehat{\Risk}(\bx)| \leq \frac{\epsilon}{2}, \quad \forall \bx \in \mathcal{S},
\]
where $\epsilon=2\sqrt{\frac{2\ln(2|\mathcal{S}|/\delta)}{T}}$.

We are now able to upper bound the risk of $\hat{\bx}$, just using the fact that the above applies to $\hat{\bx}$ too.
Defining $\bx^\star \in \argmin_{\bx \in \mathcal{S}} \ \Risk(\bx)$, we have
\[
\Risk(\hat{\bx})
\leq \widehat{\Risk}(\hat{\bx}) + \epsilon/2
\leq \widehat{\Risk}(\bx^\star) + \epsilon/2
\leq \Risk(\bx^\star) +\epsilon,
\]
where in the second inequality we used the fact that $\hat{\bx}$ minimizes the empirical risk.
\end{proof}

Using this theorem, we can use $T/2$ samples for the training and $T/2$ samples for the validation, where $T\geq2$ and even. Denoting by $\hat{\bx}_T$ the predictor with the best empirical risk on the validation set\index{validation set} among the $T/2$ generated during the online procedure, we have with probability at least $1-2 \delta$ that
\[
\Risk(\hat{\bx}_T)
\leq \min_{\bu \in \mathcal{V}} \ \Risk(\bu) + \frac{2 R(\bu,T/2)}{T} + 8\sqrt{\frac{\ln(T/\delta)}{T}}~.
\]

It is important to note that with any of the above three methods to produce one predictor from the $T$ generated ones by the online learning procedure, the sample complexity guarantee we get matches the one we would have obtained by ERM, up to polylogarithmic factors. In other words, there is nothing special about \ac{ERM} compared to the online learning approach to statistical learning. Moreover, \ac{ERM} implies the existence of a hypothetical procedure that perfectly minimizes the training loss. In reality, we should take into account the optimization error in the analysis of ERM. On the other hand, in the online learning approach we have a guarantee directly for the computed solution.

Another important point is that the above guarantee does not imply the existence of online learning algorithms with sublinear regret for any learning problem. It just says that, if it exists, it can be used in the statistical setting too.

\section{History Bits}
\index{online-to-batch conversion|(}
The specific shape of Theorem~\ref{thm:o2b} is new, but I would not be surprised if it appeared somewhere in the literature.
In particular, the uniform averaging is from \citet{Cesa-BianchiCG04}, but was proposed for the absolute loss in \citet{BlumKL99}. The non-uniform averaging of Example~\ref{ex:non_uniform_o2b} is from \citet{Zhang04b}, even though not proposed explicitly as an online-to-batch conversion.

A more recent method for online-to-batch conversion was introduced in \citet{Cutkosky19}, which independently rediscovered and generalized the averaging method in \citet{Nesterov15}. This new method allows us to prove the convergence of the last iterate rather than that of the weighted average, with a small change in any online learning algorithm.

Theorem~\ref{thm:o2b_high_prob} is from \citet{Cesa-BianchiCG04}, but here I used a second concentration to state it in terms of the competitor's true risk rather than its empirical risk. Theorem~\ref{thm:agnostic_pac_finite} is just the Agnostic PAC learning guarantee for \ac{ERM} for hypothesis classes with finite cardinality. \citet{Cesa-BianchiCG04} also gives an alternative procedure to select a single hypothesis among the $T$ generated during the online procedure that does not require splitting the data in training and validation. However, the obtained guarantee matches the one we have proved.
\index{online-to-batch conversion|)}


\section{Exercises}


\begin{exer}
Derive an explicit rate of convergence for SGD in Example~\ref{example:logistic_regression} by upper bounding the value of the minimum on the r.h.s. of the convergence rate guarantee. Hint: see \citet{JiT19}.
\end{exer}

\begin{exer}
Let $F(\bx)=\E[f(\bx,\bxi)]$, where $f(\cdot,\bxi)$ is convex and $L$-Lipschitz with respect to $\|\cdot\|_2$ for every $\bxi$. Assume that $\mathcal{V}$ is convex and has diameter at most $D$. For $a\in\R$, run \ac{OSD} on the losses $\ell_t(\bx)=t^a f(\bx,\bxi_t)$, where $\bxi_1,\dots,\bxi_T$ are i.i.d., and output the weighted average
\[
\bar{\bx}_T^{(a)}
=\frac{\sum_{t=1}^T t^a \bx_t}{\sum_{t=1}^T t^a}~.
\]
Choose the best constant learning rate in hindsight and prove an upper bound on
\[
\E\left[F(\bar{\bx}_T^{(a)})\right]-\min_{\bu\in\mathcal{V}} \ F(\bu)
\]
as a function of $a$ and $T$. In particular, show that for $a>-\frac12$ the rate is of order $\frac{DL}{\sqrt{T}}\frac{a+1}{\sqrt{2a+1}}$, up to lower order terms. What happens when $a=-\frac12$ and when $a<-\frac12$?
\end{exer}

\acresetall

\chapter{Beyond $\sqrt{T}$ Regret}
\label{ch:beyond}

The guessing game in the first chapter showed us that it is possible to get logarithmic regret in time. However, in Chapter~\ref{ch:osd}, we saw that we get only $\sqrt{T}$-regret with \ac{OSD} on the same game. What is the reason?
In this chapter, we will see that some sequences of losses can possess properties that make it easier for \ac{OSD} to learn, but we will have to change the learning rate a bit.

\acresetall

\section{Strong Convexity and Online Subgradient Descent}

The losses in the guessing game in Chapter~\ref{ch:first}, $\ell_t(x)=(x-y_t)^2$ on $[0,1]$, are not just Lipschitz. They possess some \emph{curvature} that can be exploited to achieve a better asymptotic regret. In a moment, we will see that the only change we will need to \ac{OSD} is a different learning rate, dictated as usual by the regret analysis.

The key concept we will need is \emph{strong convexity}.

\subsection{Convex Analysis Bits: Strong Convexity}

\index{function!strongly convex|(textbf}
Here, we introduce a stronger concept of convexity that allows us to build a better lower bound to a function. Instead of the linear lower bound achievable using subgradients, we will make use of a \emph{quadratic} lower bound.

\begin{definition}
\label{def:strong_convexity}
Let $\lambda\geq0$. A proper function $f : \R^d \to (-\infty, +\infty]$ is \textbf{$\lambda$-strongly convex} with respect to $\|\cdot\|$ over a convex set $\mathcal{V} \subseteq \dom f$ if
\[
f(\alpha \bx+ (1-\alpha)\by)
\leq \alpha f(\bx) + (1-\alpha) f(\by) -\frac{1}{2}\lambda \alpha (1-\alpha) \|\bx-\by\|^2,
\]
for all $\bx, \by \in \mathcal{V}$ and all $\alpha \in (0, 1)$.

We will also say that $f$ is \textbf{strongly convex in $\mathcal{V}$}, if there exists $\lambda>0$ and a norm such that the above holds.
\end{definition}
From the definition, it is clear that if a function is $\lambda$-strongly convex, it is also $\lambda'$-strongly convex for any $0\leq\lambda'<\lambda$. Moreover, convex functions are $0$-strongly convex.

We can also obtain an equivalent characterization in terms of subgradients.
\begin{lemma}
\label{lemma:strong_convexity}
Let $\lambda\geq0$. We have that $f : \R^d \to (-\infty, +\infty]$ is $\lambda$-strongly convex over a convex set $\mathcal{V} \subseteq \dom \partial f$ with respect to $\|\cdot\|$ iff
\[
\forall \bx, \by \in \mathcal{V}, \forall \bg \in \partial f(\by), \quad
f(\bx) \geq f(\by) + \langle \bg , \bx - \by \rangle + \frac{\lambda}{2} \| \bx - \by \|^2~.
\]
\end{lemma}
\begin{proof}
Let's first assume that $f$ is $\lambda$-strongly convex over $\mathcal{V}$ with respect to $\|\cdot\|$. Then, for any $\alpha \in (0,1)$ and any $\bg \in \partial f(\by)$, we have
\[
\langle \bg, \bx-\by\rangle
\leq \frac{f(\alpha \bx+ (1-\alpha)\by) - f(\by)}{\alpha}
\leq f(\bx) - f(\by) -\frac{1}{2}\lambda (1-\alpha) \|\bx-\by\|^2~.
\]
Taking the limit as $\alpha\to 0$, we obtain the statement.

Let's now assume that the inequality in the lemma holds, and let's prove that $f$ is $\lambda$-strongly convex.
Setting $\bv=\alpha \bx + (1-\alpha) \by$, for any $\bg \in \partial f(\bv)$, we have
\begin{align*}
\langle \bg, \bx-\bv\rangle &\leq f(\bx)-f(\bv)-\frac{\lambda}{2}\|\bx-\bv\|^2,\\
\langle \bg, \by-\bv\rangle &\leq f(\by)-f(\bv)-\frac{\lambda}{2}\|\by-\bv\|^2~.
\end{align*}
Summing these two inequalities with coefficients $\alpha$ and $1-\alpha$ and using the definition of $\bv$, we have
\begin{align*}
0
&=\langle \bg, \alpha \bx -\alpha \bv+ (1-\alpha)\by - (1-\alpha) \bv\rangle\\
&=\alpha f(\bx) - f(\alpha \bx +(1-\alpha)\by) +(1-\alpha) f(\by) -\alpha(1-\alpha)\frac{\lambda}{2}\|\bx-\by\|^2~. \qedhere
\end{align*}
\end{proof}



In words, the lemma above tells us that a strongly convex function can be lower bounded by a quadratic, where the linear term is the usual one constructed through the subgradient, and the quadratic term depends on the strong convexity. Hence, we have a tighter lower bound to the function than the one obtained by using convexity alone.
This is what we would expect using a Taylor expansion on a twice-differentiable convex function and lower-bounding the smallest eigenvalue of the Hessian. Indeed, we have the following theorem.
\begin{theorem}
\label{thm:hessian_strong_conv}
Let $f:\R^d\to (-\infty,+\infty]$ be proper and convex.
\begin{itemize}
\item $f$ is $\lambda$-strongly convex with respect to $\|\cdot\|$ in $ \dom \partial f$ iff
\begin{equation}
\label{eq:hessian_strong_conv1}
\langle \bg_{\bx} - \bg_{\by}, \bx - \by\rangle
\geq \lambda \|\bx - \by\|^2, \quad \forall \bx,\by \in \dom \partial f, \bg_{\bx} \in \partial f(\bx), \bg_{\by} \in \partial f(\by).
\end{equation}
\item Let $\bx,\by \in \dom f$. If $f$ is continuously twice differentiable on an open set containing the segment between $\by$ and $\bx$, and it holds that
\begin{equation}
\label{eq:hessian_strong_conv2}
\langle \nabla^2 f(\by+\alpha(\bx-\by))(\bx-\by), \bx-\by\rangle
\geq \lambda \|\bx-\by\|^2, \quad \forall \alpha \in [0,1],
\end{equation}
then we have
\[
f(\bx)
\geq f(\by) + \langle \nabla f(\by), \bx-\by \rangle + \frac{\lambda}{2} \|\bx-\by\|^2~.
\]
If the assumption holds for any $\bx,\by \in \dom f$, then $f$ is $\lambda$-strongly convex in $\dom f$ with respect to $\|\cdot\|$.
\end{itemize}
\end{theorem}
\begin{proof}
For the first statement, first assume that $f$ is $\lambda$-strongly convex with respect to $\|\cdot\|$. Then, from Lemma~\ref{lemma:strong_convexity}, we have
\begin{align*}
\langle \bg_{\bx}, \by - \bx\rangle &\leq f(\bx) - f(\by) - \frac{\lambda}{2} \|\bx-\by\|^2, \quad \forall \bg_{\bx} \in \partial f(\bx),\\
\langle \bg_{\by}, \bx - \by\rangle &\leq f(\by) - f(\bx) - \frac{\lambda}{2} \|\bx-\by\|^2, \quad \forall \bg_{\by} \in \partial f(\by)~.
\end{align*}
Summing the two inequalities, we have the stated bound.

Now, assume that the \eqref{eq:hessian_strong_conv1} holds.
Define $h(\alpha)=f(\by + \alpha (\bx - \by))$, $\bw_\alpha = \by + \alpha (\bx - \by)$, and $\bg_{\bw_\alpha} \in \partial f(\bw_\alpha)$. From Theorem~\ref{thm:affine_subgrad}, we have $\langle\bg_{\bw_\alpha},\bx - \by\rangle \in \partial h(\alpha)$. Moreover,
\[
\langle \bg_{\bw_\alpha}, \bx - \by\rangle - \langle \bg_{\by}, \bx - \by\rangle
= \frac{1}{\alpha} \langle \bg_{\bw_\alpha} - \bg_{\by}, \bw_\alpha - \by\rangle
\geq \frac{\lambda}{\alpha} \|\bw_\alpha - \by\|^2
= \lambda \alpha \|\bx - \by\|^2,
\]
where we used \eqref{eq:hessian_strong_conv1} in the inequality.
Using the fundamental theorem of calculus for extended-real-valued convex functions \index{fundamental theorem of calculus for extended-real-valued convex functions} (Theorem~\ref{thm:ftc}) and this last inequality, we have
\begin{align*}
f(\bx) - f(\by) - \langle \bg_{\by}, \bx - \by\rangle
&= h(1) - h(0) - \langle \bg_{\by}, \bx - \by\rangle\\
&= \int_0^1 \! (\langle\bg_{\bw_\alpha}, \bx - \by\rangle - \langle\bg_{\by},\bx - \by\rangle) \, \mathrm{d}\alpha
\geq \frac{\lambda}{2} \|\bx - \by\|^2~.
\end{align*}
Hence, by Lemma~\ref{lemma:strong_convexity}, $f$ is $\lambda$-strongly convex with respect to $\|\cdot\|$.

For the second statement, assume that $f$ is twice differentiable and \eqref{eq:hessian_strong_conv2} holds. Then,
\[
h''(\alpha)
= \langle \nabla^2 f(\by+\alpha(\bx-\by))(\bx-\by), \bx-\by\rangle
\geq \lambda \|\bx-\by\|^2~.
\]
Moreover, from Taylor's remainder theorem, we have that $h(1)=h(0)+h'(0)+h''(\beta)/2$, where $\beta \in [0,1]$. So, we have
\[
f(\bx)
= h(1)
= h(0)+h'(0)+h''(\beta)/2
\geq f(\by) + \langle \nabla f(\by), \bx-\by \rangle + \frac{\lambda}{2} \|\bx-\by\|^2~.
\]
If it holds for all $\bx,\by$, then, by Lemma~\ref{lemma:strong_convexity}, $f$ is $\lambda$-strongly convex with respect to $\|\cdot\|$.
\end{proof}

\begin{example}
Let $f(\bx)=\frac{1}{2}\|\bx\|_2^2$. Using Theorem~\ref{thm:hessian_strong_conv}, we have that $f$ is 1-strongly convex with respect to $\|\cdot\|_2$ in $\R^d$.
\end{example}

%

\begin{figure}[t]
\centering
\begin{tikzpicture}
\begin{axis}[
          width=7cm,
          xmin=-2,xmax=2,
          axis lines=middle,
          xtick={-2,-1,0,1,2},
          xlabel={$x$},
          legend pos=north west,
          legend style={font=\tiny},
          legend image post style={
              scale=0.5,
          },
          legend cell align={left},
          ]
\addplot[thick,black,samples=200] {abs(x)+x^2};
\addlegendentry{$f(x)=|x|+x^2$}
\addplot[thick,dotted,samples=200]  {x^2};
\addlegendentry{Lower bound 1: $x^2$}
\addplot[thick,dashdotted,samples=200]  {x^2+0.5*x};
\addlegendentry{Lower bound 2: $x^2+x/2$}
\end{axis}
\end{tikzpicture}
\caption{Possible lower bounds to the strongly convex non-differentiable function $f(x)=|x|+x^2$.}
\label{fig:strongly_convex}
\commentAlt{Figure~\ref{fig:strongly_convex}. Plot of the nondifferentiable strongly convex function f(x)=|x|+x^2 together with two quadratic lower bounds, x^2 and x^2+x/2.}
\end{figure}

However, a strongly convex function does not need to be twice differentiable. Indeed, we do not even need plain differentiability. Hence, the use of the subgradient implies that the quadratic lower bound does not have to be uniquely determined, as in the next example.
\begin{example}
Consider the strongly convex function $f(x)=|x|+x^2$. In Figure~\ref{fig:strongly_convex}, we show two possible quadratic lower bounds to the function at $x=0$.
\end{example}

We also have the following useful property of sums of strongly convex functions.
\begin{theorem}
Let $f:\R^d\to (-\infty, +\infty]$ be $\mu_1$-strongly convex and $g:\R^d \to (-\infty, +\infty]$ be $\mu_2$-strongly convex in a non-empty convex set $\mathcal{V} \subseteq \dom f\cap \dom g$ with respect to $\|\cdot\|$. Then, $f+g$ is $\mu_1+\mu_2$-strongly convex in $\mathcal{V}$ with respect to $\|\cdot\|$.
\end{theorem}
\begin{proof}
It is enough to sum the inequality in the definition of strong convexity (Definition~\ref{def:strong_convexity}) for both functions.
\end{proof}

\index{function!strongly convex|)textbf}

%
%

\subsection{Online Subgradient Descent for Strongly Convex Losses}
\index{function!strongly convex|(}

\begin{theorem}
\label{thm:log_regret}
Let $\mathcal{V}$ be a non-empty closed convex set in $\R^d$.
Assume that the functions $\ell_t: \R^d \to (-\infty, +\infty]$ are $\mu_t$-strongly convex with respect to $\|\cdot\|_2$ and subdifferentiable on $\mathcal{V}$, where $\mu_t>0$. Use \ac{OSD} in Algorithm~\ref{alg:posd} with stepsizes equal to $\eta_t=\frac{1}{\sum_{i=1}^t \mu_i}$. Then, for any $\bu \in \mathcal{V}$, we have the following regret guarantee
\[
\sum_{t=1}^T (\ell_t(\bx_t) - \ell_t(\bu))
\leq \frac{1}{2} \sum_{t=1}^T \frac{ \|\bg_t\|^2_2}{\sum_{i=1}^t \mu_i}~.
\]
\end{theorem}
\begin{proof}
From the assumption of $\mu_t$-strong convexity of the functions $\ell_t$, we have that
\[
\ell_t(\bx_t)-\ell_t(\bu)
\leq \langle \bg_t, \bx_t - \bu \rangle - \frac{\mu_t}{2} \|\bx_t-\bu\|^2_2~.
\]
From the fact that $\eta_t=\frac{1}{\sum_{i=1}^t \mu_i}$, we have
\begin{align*}
&\frac{1}{2\eta_1}-\frac{\mu_1}{2}=0,\\
&\frac{1}{2\eta_t}-\frac{\mu_t}{2}=\frac{1}{2\eta_{t-1}}, \ t=2,\dots,T~.
\end{align*}
Hence, use Lemma~\ref{lemma:sd_one_step} and sum from $t=1,\dots,T$, to obtain
\begin{align*}
\sum_{t=1}^T &\left(\ell_t(\bx_t)- \ell_t(\bu)\right)\\
&\leq \sum_{t=1}^T \left(\frac{\|\bx_t-\bu\|_2^2}{2\eta_t} - \frac{\|\bx_{t+1}-\bu\|_2^2}{2\eta_t} - \frac{\mu_t \|\bx_t-\bu\|^2_2}{2}  + \frac{\eta_t \|\bg_t\|_2^2}{2}\right) \\
&= - \frac{\|\bx_{2}-\bu\|_2^2}{2\eta_1} + \sum_{t=2}^T \left(\frac{\|\bx_t-\bu\|_2^2}{2 \eta_{t-1}} - \frac{\|\bx_{t+1}-\bu\|_2^2}{2 \eta_t}\right) + \sum_{t=1}^T \frac{\eta_t \|\bg_t\|_2^2}{2}  ~.
\end{align*}
Observing that the first sum on the r.h.s. is a telescopic sum, we have the stated bound.
\end{proof}

\begin{corollary}
\label{cor:log_regret}
Under the assumptions of Theorem~\ref{thm:log_regret}, if in addition we have $\mu_t = \mu >0$ and $\ell_t$ is $L$-Lipschitz with respect to $\|\cdot\|_2$ on an open set containing $\mathcal{V}$, for $t=1, \dots, T$, then we have
\[
\sum_{t=1}^T \ell_t(\bx_t) - \sum_{t=1}^T \ell_t(\bu)
\leq \frac{L^2}{2 \mu} \left(1+\ln T\right)~.
\]
\end{corollary}

\begin{remark}
Notice that this corollary implicitly requires a bounded domain, otherwise the loss functions will not be Lipschitz, given that they are also strongly convex.
\end{remark}

\begin{remark}
Corollary~\ref{cor:log_regret} \emph{does not} imply that for any $T$ the regret will be smaller than using learning rates $\propto \frac{1}{L\sqrt{t}}$. Indeed, depending on the other quantities in the upper bounds, the upper bound in Corollary~\ref{cor:log_regret} is better than that of \ac{OSD} with Lipschitz losses only for $T$ large enough. Moreover, a better upper bound does not necessarily imply a better regret.
\end{remark}

\begin{example}
\label{ex:guessing_strongly_convex}
\index{guessing game|(}
Consider once again the guessing game in the first chapter: $\ell_t(x)=(x-y_t)^2$. Note that the loss functions are $2$-strongly convex with respect to $|\cdot|$.
Hence, setting $\eta_t=\frac{1}{2t}$ and $\ell_t'(x)=2(x-y_t)$ gives a regret of $\ln(T)+1$.
\index{guessing game|)}
\end{example}

\index{online-to-batch conversion|(}
We can also use the online-to-batch conversion on strongly convex stochastic problems.
\begin{example}
\label{ex:non_uniform_o2b_strongly}
As done before, we can use the online-to-batch conversion to use Corollary~\ref{cor:log_regret} to obtain stochastic subgradient descent algorithms for strongly convex stochastic functions.
For example, consider the classic Support Vector Machine objective,
\[
\min_{\bx} \ F(\bx):=\frac{\lambda}{2}\|\bx\|_2^2 + \frac{1}{N} \sum_{i=1}^N \max(1-y_i \langle \bz_i, \bx\rangle,0),
\]
or any other regularized formulation like regularized logistic regression,
\[
\min_{\bx} \ F(\bx):=\frac{\lambda}{2}\|\bx\|_2^2 + \frac{1}{N} \sum_{i=1}^N \ln(1+\exp(-y_i \langle \bz_i, \bx\rangle)),
\]
where $\bz_i \in \R^d$, $\|\bz_i\|_2\leq R$, and $y_i \in\{-1,1\}$.
First, notice that the minimizer of both expressions has to be in the L$_2$ ball of radius proportional to $\sqrt{\frac{1}{\lambda}}$ (proof left as an exercise). Hence, we can set $\mathcal{V}$ equal to this set.
Then, setting $\ell_t(\bx)=\frac{\lambda}{2}\|\bx\|_2^2+\max(1-y_t \langle \bz_t, \bx\rangle,0)$ or $\ell_t(\bx)=\frac{\lambda}{2}\|\bx\|_2^2+\ln(1+\exp(-y_t \langle \bz_t, \bx\rangle))$ results in $\lambda$-strongly convex loss functions. Using Corollary~\ref{cor:log_regret} and Theorem~\ref{thm:o2b} gives immediately
\[
\E\left[F\left(\frac{1}{T}\sum_{t=1}^T \bx_t\right)\right] - \min_{\bx} \ F(\bx)
= \mathcal{O}\left(\frac{\ln T}{ \lambda T}\right)~.
\]
However, we can do better! We can use non-uniform weights in Theorem~\ref{thm:o2b} to remove the log term and obtain the optimal convergence rate for the stochastic optimization of strongly convex functions.
Observe that $\ell_t(\bx)=\frac{\lambda t}{2}\|\bx\|_2^2+t \max(1-y_t \langle \bz_t, \bx\rangle,0)$ or $\ell_t(\bx)=\frac{\lambda t}{2}\|\bx\|_2^2+t \ln(1+\exp(-y_t \langle \bz_t, \bx\rangle))$ are $\lambda t$-strongly convex loss functions.
So, using Theorem~\ref{thm:log_regret}, we have that $\eta_t=\frac{2}{\lambda t(t+1)}$ and Theorem~\ref{thm:o2b} gives immediately
\[
\E\left[F\left(\frac{2}{T(T+1)}\sum_{t=1}^T t\, \bx_t\right)\right] - \min_{\bx} \ F(\bx)
= \mathcal{O}\left(\frac{1}{\lambda T}\right)~.
\]
\end{example}

\index{online-to-batch conversion|)}
\index{function!strongly convex|)}

\section{Adaptive Algorithms: $\mathscr{L}^\star$ bounds and AdaGrad}
\label{sec:lstar}

In this section, we will explore different conditions where we can get better regret upper bounds than $\mathcal{O}(D L \sqrt{T})$ as $T\to \infty$. Also, we will obtain these improved guarantees in an \emph{automatic} way. That is, the algorithm will be \emph{adaptive} to the characteristics of the sequence of loss functions, without having to rely on information about the future.

\index{learning rate!adaptive|(}
\subsection{Adaptive Learning Rates for Online Subgradient Descent}

Consider the minimization of the regret with linear losses:
\[
\Regret_T(\bu)=\sum_{t=1}^T \langle \bg_t, \bx_t\rangle - \sum_{t=1}^T \langle \bg_t, \bu\rangle~.
\]
Using \ac{OSD}, in Chapter~\ref{ch:osd}, we said that the regret for bounded domains can be upper bounded by
\[
\sum_{t=1}^T \langle \bg_t, \bx_t\rangle - \sum_{t=1}^T \langle \bg_t, \bu\rangle
\leq \frac{D^2}{2 \eta_T} + \frac{1}{2}\sum_{t=1}^T \eta_t \|\bg_t\|_2^2~.
\]
With a fixed learning rate $\eta_t=\eta$, the learning rate that minimizes this upper bound on the regret is
\[
\eta^\star = \frac{D}{\sqrt{\sum_{t=1}^T \|\bg_t\|_2^2}}~.
\]
Unfortunately, as we said, this learning rate cannot be used because it assumes knowledge of the future rounds.
However, we might be lucky, and we might try to just approximate it in each round using the knowledge up to time $t$.
That is, we might try to use
\begin{equation}
\label{eq:ada_eta}
\eta_t = \frac{D}{\sqrt{\sum_{i=1}^t \|\bg_i\|_2^2}},
\end{equation}
and just skip the rounds in which $\bg_t=\boldsymbol{0}$ to avoid possible divisions by 0.
Observe that $\eta_T=\eta^\star$, so the first term of the regret would be exactly what we need!
For the other term, the optimal learning rate would give us
\[
\frac{1}{2}\sum_{t=1}^T \eta_T^\star \|\bg_t\|_2^2
= \frac{1}{2} D \sqrt{\sum_{t=1}^T \|\bg_t\|_2^2}~.
\]
Now, let's see what we obtain with our approximation in the other term of the regret:
\[
\frac{1}{2}\sum_{t=1}^T \eta_t \|\bg_t\|^2_2
= \frac{1}{2} D \sum_{t=1}^T \frac{\|\bg_t\|_2^2}{\sqrt{\sum_{i=1}^t \|\bg_i\|_2^2}}~.
\]
We need a way to upper-bound that sum. The way to treat these sums, as we did in Chapter~\ref{ch:first}, is to try to approximate them with integrals.
So, we can use the following very handy lemma that generalizes a lot of similar specific ones.
\begin{lemma}
\label{lemma:sum_integral_bounds}
Let $a_0, \dots, a_T\geq 0$ and $f:[0,+\infty)\to [0, +\infty)$ a non-increasing continuous function.
Then
\begin{align*}
\sum_{t=1}^T a_t f\left(a_0+\sum_{i=1}^{t} a_i\right)
&\leq \int_{a_0}^{\sum_{t=0}^T a_t} \! f(x) \, \mathrm{d}x~.
\end{align*}
\end{lemma}
\begin{proof}
Denote by $s_t=\sum_{i=0}^{t} a_i$.
\begin{align*}
a_t f\left(a_0+ \sum_{i=1}^{t} a_i\right)
= a_t f(s_t)
=  \int_{s_{t-1}}^{s_t} f(s_t) d x
\leq \int_{s_{t-1}}^{s_t} f(x) d x~.
\end{align*}
Summing over $t=1, \dots, T$, we have the stated bound.
\end{proof}

Using this lemma with $f(x)=\frac{1}{\sqrt{x}}$ and $a_0\to 0$, we have that
\[
\frac{1}{2} D \sum_{t=1}^T \frac{\|\bg_t\|_2^2}{\sqrt{\sum_{i=1}^t \|\bg_i\|_2^2}}
\leq D \sqrt{\sum_{t=1}^T \|\bg_t\|_2^2}~.
\]
Surprisingly, this term is only a factor of 2 worse than what we would have got from the optimal choice of $\eta^\star$. However, this learning rate can be computed without knowledge of the future, and it can actually be used!
Overall, with this choice, we get
\begin{equation}
\label{eq:ada_grad_norm}
\Regret_T(\bu)
=\sum_{t=1}^T \ell_t(\bx_t) - \sum_{t=1}^T \ell_t(\bu)
\leq \frac32 D \sqrt{\sum_{t=1}^T \|\bg_{t}\|_2^2}~.
\end{equation}
Note that it is possible to improve the constant in front of the bound to $\sqrt{2}$ by multiplying the learning rates by $\frac{\sqrt{2}}{2}$.
So, putting everything together, we have the following theorem.
\begin{theorem}
\label{thm:ada_grad_norm}
Let $\mathcal{V} \subset \R^d$ be a closed non-empty convex set with diameter $D$, that is, $\max_{\bx,\by\in \mathcal{V}} \|\bx-\by\|_2 \leq D$. Let $\ell_1, \dots, \ell_T$ be an arbitrary sequence of convex functions $\ell_t:\R^d \to (-\infty, +\infty]$ subdifferentiable on $\mathcal{V}$ for $t=1, \dots,T$. Pick any $\bx_1 \in \mathcal{V}$, and run \ac{OSD} with $\eta_{t}=\frac{\sqrt{2}D}{2\sqrt{\sum_{i=1}^t \|\bg_i\|_2^2}}, \ t=1, \dots, T$, and do not update on rounds when $\bg_t=\boldsymbol{0}$. Then, $\forall \bu \in \mathcal{V}$, the following regret bound holds
\[
\sum_{t=1}^T (\ell_t(\bx_t) - \ell_t(\bu))
\leq D \sqrt{2\sum_{t=1}^T \|\bg_{t}\|_2^2}
= \sqrt{2} \min_{\eta>0} \left(\frac{D^2}{2\eta} + \frac{\eta}{2}\sum_{t=1}^T \|\bg_t\|_2^2\right).
\]
\end{theorem}
The second equality in the theorem clearly shows the advantage of this learning rate: we obtain (almost) the same guarantee we would have obtained knowing the future gradients!

This is an interesting result on its own: it gives a principled way to set the learning rates with an almost optimal guarantee. Observe that this approach avoids knowing the Lipschitz constants of the functions, so the algorithm adapts to them. The downside is that it works only in the bounded case.
Another important observation is that the sum of the squared gradients acts as an \emph{intrinsic notion of time}, better suited than $T$ to capture the dependence on time.
\index{learning rate!adaptive|)}

There are also other consequences of this simple regret upper bound as we will now see by specializing this result to the case that the losses are \emph{smooth} or \emph{self-bounded}.

\subsection{Convex Analysis Bits: Dual Norms, Smooth, and Self-Bounded Functions}

We now consider a family of loss functions that have the characteristic of being lower bounded by the squared norm of their subgradients.
We will also introduce the concept of \emph{dual norms}. While dual norms are not strictly needed for this topic, they give more generality, and at the same time, they allow me to slowly introduce some of the concepts that will be needed for the chapter on Online Mirror Descent.

\begin{definition}
The \textbf{dual norm}\index{norm!dual|textbf} $\|\cdot\|_\star$ of a norm $\|\cdot\|$ is defined as $\|\btheta\|_\star=\max_{\bx: \|\bx\|\leq 1} \ \langle \btheta, \bx\rangle$.
\end{definition}

\begin{remark}
The definition of dual norm immediately implies $\langle \btheta, \bx\rangle \leq \|\btheta\|_\star \|\bx\|$.
\end{remark}

\index{norm!dual|(}
\begin{example}
\label{example:dual_l2}
The dual norm of the L$_2$ norm is the L$_2$ norm. We can easily prove it. First of all, if $\btheta=\boldsymbol{0}$, then the dual norm is 0 too. Hence, let's assume that $\btheta\neq\boldsymbol{0}$. We have $\|\btheta\|_\star=\max_{\bx: \|\bx\|_2\leq 1} \ \langle \btheta, \bx\rangle \leq \|\btheta\|_2$ by Cauchy--Schwarz inequality\index{inequality!Cauchy--Schwarz}. Also, set $\bv=\frac{\btheta}{\|\btheta\|_2}$, so $\max_{\bx: \|\bx\|_2\leq 1} \ \langle \btheta, \bx\rangle \geq \langle \btheta, \bv\rangle = \|\btheta\|_2$.
\end{example}

\begin{example}
\label{example:dual_lp}
Let $p\geq 1$. The L$_p$ norm of a vector $\bx \in \R^d$ is defined as $\|\bx\|_p=(\sum_{i=1}^d |x_i|^p)^{1/p}$\index{norm!$p$-}. The dual norm is the $q$-norm where $\frac{1}{p}+\frac{1}{q}=1$. Note that the dual of the L$_1$ norm is the L$_\infty$ norm, defined as $\|\bx\|_\infty = \max_{i=1, \dots, d} |x_i|$. The proof is left as an exercise to the reader.
\end{example}

\begin{example}
\label{example:dual_norm_a}
Let $\bA$ be a positive definite matrix, then it is possible to show that $\|\bx\|_{\bA}:=\sqrt{\bx^\top \bA \bx}$ is a norm. The dual norm is $\|\bx\|_{\bA^{-1}} = \sqrt{\bx^\top \bA^{-1}\bx}$. In fact, we have that the dual norm of $\|\cdot\|_{\bA}$ is defined as
\begin{align*}
\|\btheta\|_\star
&= \max_{\bx: \|\bx\|_{\bA}\leq 1} \langle \btheta, \bx\rangle
= \max_{\bx: \bx^\top \bA \bx \leq 1} \btheta^\top \bx
= \max_{\by: \by^\top \by \leq 1} \btheta^\top \bA^{-1/2} \by
= \max_{\by: \|\by\|_2 \leq 1} (\bA^{-1/2} \btheta)^\top \by\\
&= \|\bA^{-1/2} \btheta\|_2
= \sqrt{\btheta^\top \bA^{-1} \btheta},
\end{align*}
where we have used the change of variable $\by=\bA^{1/2} \bx$ in the third equality and the dual norm of the L$_2$ norm from Example~\ref{example:dual_l2} in the second to last equality.
\end{example}
\index{norm!dual|)}

If you do not know the concept of \emph{operator norms}, the concept of dual norm can be a bit weird at first. One way to understand it is that it is a way to measure how ``big'' linear functionals are.
For example, consider the linear function $f(\bx)=\langle \bz,\bx\rangle$, we want to try to understand how big it is. So, we can measure $\max_{\bx\neq 0}\frac{\langle \bz,\bx\rangle}{\|\bx\|}$, that is, we measure how big the output of the linear functional is compared to its input $\bx$, where $\bx$ is measured with some norm.
Now, you can show that the above is equivalent to the dual norm of $\bz$.

We have the following lemma that will be useful in the later chapters.
\begin{lemma}
\label{lemma:dual_norm_attained}
Let $\|\cdot\|$ be a norm on $\R^d$, and let $\|\cdot\|_\star$ be its dual norm.
Then, for every $\btheta \in \R^d$ and every $\alpha \in \R_{\geq 0}$, there exists $\bx \in \R^d$ such that $\|\bx\|=\alpha$ and $\langle \btheta,\bx\rangle = \|\btheta\|_\star \|\bx\|$.
\end{lemma}
\begin{proof}
If $\alpha=0$, choose $\bx=\boldsymbol{0}$.
Hence, assume $\alpha>0$.

If $\btheta=\boldsymbol{0}$, choose any $\bx$ such that $\|\bx\|=\alpha$.
Otherwise, by the definition of dual norm, there exists $\bv$ with $\|\bv\|\leq 1$ such that
\[
\langle \btheta,\bv\rangle
= \max_{\bu:\|\bu\|\leq 1} \ \langle \btheta,\bu\rangle
= \|\btheta\|_\star~.
\]
Since $\btheta\neq \boldsymbol{0}$, the function $\bv \mapsto \langle \btheta,\bv\rangle$ is nonconstant and linear, and so it attains its maximum over $\{\bv:\|\bv\|\leq 1\}$ on the boundary.
Thus, we can choose $\bv$ such that $\|\bv\|=1$. So, $\bx=\alpha \bv$ satisfies $\|\bx\|=\alpha$ and $\langle \btheta,\bx\rangle = \|\btheta\|_\star \|\bx\|$.
\end{proof}

\index{subgradient!Lipschitz convex function|(}
We can also extend Theorem~\ref{thm:subgradient_lipschitz} to generic dual norms.
\begin{theorem}
\label{thm:subgradient_lipschitz_dual}
Let $f:\R^d \to (-\infty,+\infty]$ be proper and convex. Then, the following are equivalent:
\begin{enumerate}
    \item $f$ is $L$-Lipschitz on $\interior \dom f$ with respect to $\|\cdot\|$, i.e., $|f(\bx)-f(\by)| \leq L\|\bx-\by\| \quad \forall \bx,\by \in \interior \dom f$.
    \item For every $\bx \in \interior \dom f$ and every $\bg \in \partial f(\bx)$, $\|\bg\|_\star \leq L$.
\end{enumerate}
\end{theorem}
\begin{proof}
First assume that $f$ is $L$-Lipschitz on $\interior \dom f$.

Fix $\bx \in \interior \dom f$ and $\bg \in \partial f(\bx)$. By the definition of dual norm,
\[
\|\bg\|_\star = \sup_{\|\bu\|\leq 1} \ \langle \bg,\bu\rangle~.
\]
So, it suffices to show that $\langle \bg,\bu\rangle \leq L$ for every $\bu$ with $\|\bu\|\leq 1$.

Let $\bu \in \R^d$ satisfy $\|\bu\|\leq 1$. Since $\bx \in \interior \dom f$, for $\epsilon>0$ small enough we have
\[
\by := \bx+\epsilon \bu \in \interior \dom f~.
\]
Using $L$-Lipschitz of $f$ and the subgradient inequality, we obtain
\begin{align*}
L\epsilon \|\bu\|
= L\|\by-\bx\|
\geq |f(\by)-f(\bx)|
\geq f(\by)-f(\bx)
\geq \langle \bg,\by-\bx\rangle
= \epsilon \langle \bg,\bu\rangle~.
\end{align*}
Hence, $\langle \bg,\bu\rangle \leq L\|\bu\| \leq L$.
Taking the supremum over all $\bu$ with $\|\bu\|\leq 1$ gives $\|\bg\|_\star \leq L$.

Now assume conversely that for every $\bz \in \interior \dom f$ and every $\bh \in \partial f(\bz)$, $\|\bh\|_\star \leq L$.
Take any $\bx,\by \in \interior \dom f$. Since $f$ is convex and $\bx \in \interior \dom f$, the subdifferential $\partial f(\bx)$ is nonempty, so choose $\bg \in \partial f(\bx)$. Then, by the subgradient inequality, we have
\[
f(\bx)-f(\by)
\leq \langle \bg,\bx-\by\rangle
\leq \|\bg\|_\star\,\|\bx-\by\|
\leq L\|\bx-\by\|~.
\]
Exchanging the roles of $\bx$ and $\by$ and combining the two inequalities, we obtain
\[
|f(\bx)-f(\by)| \leq L\|\bx-\by\|~. \qedhere
\]
\end{proof}
\index{subgradient!Lipschitz convex function|)}

%
%
%

Now, we can introduce smooth functions, using the dual norms defined above.
\begin{definition}
\index{function!smooth|(textbf}
Let $f$ be differentiable on an open set containing $\mathcal{V}$. We say that $f$ is \textbf{$s$-smooth} with respect to $\|\cdot\|$ if $\|\nabla f(\bx) -\nabla f(\by)\|_\star \leq s \|\bx-\by\|$ for all $\bx, \by \in \mathcal{V}$.
\end{definition}
Keeping in mind the intuition above on dual norms, taking the dual norm of a gradient makes sense if you associate each gradient with the linear functional $\langle \nabla f(\by), \bx\rangle$, that is, the one needed to create a linear approximation of $f$.

\begin{remark}
Note that smoothness does not imply convexity.
\end{remark}

Smooth functions have many properties, for example, a smooth function can be upper and lower bounded by a quadratic.
\begin{lemma}
\label{lemma:smooth_quadratic_upper_lower_bound}
Let $f:\mathcal{V} \to \R$ be $s$-smooth with respect to $\|\cdot\|$. Then, for any $\bx, \by \in \mathcal{V}$ such that the line segment between $\bx$ and $\by$ is in $\mathcal{V}$, we have
\[
|f(\by)-f(\bx) -\langle \nabla f(\bx), \by -\bx\rangle| \leq \frac{s}{2}\|\by - \bx\|^2~.
\]
\end{lemma}
\begin{proof}
First, notice that by the definition of smoothness, $\nabla f:\mathcal{V}\to \R^d$ is Lipschitz and so continuous. Hence, by the fundamental theorem of calculus, we have
\begin{align*}
f(\by)
&= f(\bx) + \int_0^1 \! \langle \nabla f(\bx+\tau (\by-\bx)), \by -\bx\rangle \, \mathrm{d} \tau \\
&= f(\bx) + \langle \nabla f(\bx), \by-\bx\rangle + \int_0^1 \! \langle \nabla f(\bx+\tau (\by-\bx))- \nabla f(\bx), \by -\bx\rangle \, \mathrm{d} \tau~.
\end{align*}
Therefore,
\begin{align*}
|f(\by)-f(\bx) -\langle \nabla f(\bx), \by -\bx\rangle|
&= \left|\int_0^1 \! \langle \nabla f(\bx+\tau (\by-\bx))- \nabla f(\bx), \by -\bx\rangle  \, \mathrm{d} \tau\right| \\
&\leq \int_0^1 \! |\langle \nabla f(\bx+\tau (\by-\bx))- \nabla f(\bx), \by -\bx\rangle|  \, \mathrm{d} \tau \\
&\leq \int_0^1 \! \|\nabla f(\bx+\tau (\by-\bx))- \nabla f(\bx)\|_\star \|\by -\bx\|  \, \mathrm{d} \tau \\
&\leq \int_0^1 \! \tau s \|\by -\bx\|^2  \, \mathrm{d} \tau
= \frac{s}{2} \|\by -\bx\|^2~. \qedhere
\end{align*}
\end{proof}

In the following, we will also need the following property.
\begin{theorem}
\label{thm:smooth}
Let $f:\R^d \to \R$ be $s$-smooth with respect to $\|\cdot\|$, and bounded from below. Then, for all $\bx \in \R^d$, we have
\[
\|\nabla f(\bx)\|_\star^2
\leq 2 s (f(\bx) - \inf_{\by \in \R^d} f(\by))~.
\]
\end{theorem}
\begin{proof}
From Lemma~\ref{lemma:smooth_quadratic_upper_lower_bound}, for any $\bx,\bv \in\R^d$, we have
\[
\langle -\nabla f(\bx), \bv\rangle - \frac{s}{2}\|\bv\|^2
\leq f(\bx)- f(\bx+\bv)
\leq f(\bx)-\inf_{\by \in \R^d} f(\by)~.
\]
Given that this holds for any $\bv$, we can take the supremum of the l.h.s. with respect to $\bv$. Using Example~\ref{example:conj_squared_norm}, we have
\[
\frac{1}{2 s}\|\nabla f(\bx) \|_\star^2
= \sup_{\bv} \ \langle -\nabla f(\bx), \bv\rangle - \frac{s}{2}\|\bv\|^2
\leq f(\bx)-\inf_{\by \in \R^d} f(\by)~. \qedhere
\]
\end{proof}
\index{function!smooth|)textbf}

Sometimes we do not need smoothness or differentiability, but only the property of the above theorem. We call convex functions that satisfy such an inequality \emph{self-bounded}.
\begin{definition}
\label{def:self-bounded}
Let $f:\R^d \to (-\infty, +\infty]$ be bounded from below, and subdifferentiable on a set $\mathcal{V}$. We say that $f$ is $s$\textbf{-self-bounded}\index{function!self-bounded|textbf} in $\mathcal{V}$ with respect to $\|\cdot\|$ if
\[
\|\bg\|^2_\star \leq 2s (f(\bx)-\inf_{\by \in \R^d} \ f(\by)), \quad \forall \bx \in \mathcal{V}, \forall \bg \in \partial f(\bx)~.
\]
\end{definition}

\begin{remark}
If $\mathcal{V}$ is convex, then self-bounded\index{function!self-bounded} functions are also convex, because we are assuming that they are subdifferentiable on $\mathcal{V}$ (Theorem~\ref{thm:subgradients_everywhere_implies_convexity}).
\end{remark}

Clearly, a convex $s$-smooth function\index{function!smooth} is also $s$-self-bounded\index{function!self-bounded}, but the converse is not true, as shown in the next example.
\begin{example}
Let $f:\R\to \R$ be defined as $f(x)=\frac12 x^2 + |x-2|$. The function $f$ is not differentiable at $2$, hence it is not smooth. However, it is easy to verify that it is $9$-self-bounded\index{function!self-bounded}.
\end{example}

\subsection{$\mathscr{L}^\star$ bounds}

\index{L* bound@$\mathscr{L}^\star$ bound|(}
We now introduce the $\mathscr{L}^\star$ bounds, which depend on the cumulative competitor loss that is usually denoted by $\mathscr{L}^\star$.

Assume that the loss functions $\ell_1, \dots, \ell_T$ are $s$-self-bounded\index{function!self-bounded} on $\mathcal{V}$ with respect to $\|\cdot\|_2$. From this assumption, without loss of generality, we can assume that each of them is bounded from below by 0.
Under these assumptions, we can obtain bounds that depend on the cumulative loss of the competitor rather than time.

From the regret of \ac{OSD} (that is the same as in Theorem~\ref{thm:pogd}) and Definition~\ref{def:self-bounded}, for a constant learning rate $\eta$ we obtain
\[
\sum_{t=1}^T (\ell_t(\bx_t) - \ell_t(\bu))
\leq \frac{\|\bu-\bx_1\|_2^2}{2\eta} + \eta \sum_{t=1}^T s \ell_t(\bx_t), \quad \forall \bu \in \mathcal{V}~.
\]
Reordering, it implies
\[
\sum_{t=1}^T (\ell_t(\bx_t) - \ell_t(\bu))
\leq \frac{\eta s}{1-\eta s} \sum_{t=1}^T \ell_t(\bu) + \frac{\|\bu-\bx_1\|_2^2}{2\eta (1-\eta s)}, \quad \forall \bu \in \mathcal{V}~.
\]
Assuming $\eta \leq \frac{1}{2s}$, we simplify this regret upper bound in
\[
\sum_{t=1}^T (\ell_t(\bx_t) - \ell_t(\bu))
\leq 2\eta s \sum_{t=1}^T \ell_t(\bu) + \frac{\|\bu-\bx_1\|_2^2}{\eta}, \quad \forall\bu \in \mathcal{V}~.
\]
This is already an interesting result because it guarantees that a fixed learning rate that depends only on $s$ can achieve a vanishing average regret if there exists a competitor $\bu \in \mathcal{V}$ whose cumulative loss grows sublinearly. However, we could do better.
In fact, for a fixed $\bu \in \mathcal{V}$, setting $\eta = \min(\frac{\alpha}{\sqrt{2 s \sum_{t=1}^T \ell_t(\bu)}},\frac{1}{2 s})$ for $\alpha>0$, we obtain
\[
\sum_{t=1}^T (\ell_t(\bx_t) - \ell_t(\bu))
\leq  \max\!\left[\left(\frac{\|\bu-\bx_1\|_2^2}{\alpha}+\alpha\right)\!\sqrt{2 s \sum_{t=1}^T \ell_t(\bu)}, 2s (\|\bu-\bx_1\|_2^2+\alpha^2)\right]\!.
\]
Comparing this bound to the one of \ac{OGD} with Lipschitz losses, we see that here the dependence is on $\sqrt{\sum_{t=1}^T \ell_t(\bu)}$ instead of $\sqrt{T}$. The cumulative loss of the competitor can be much smaller than $T$ and, in particular, can be even 0. In this case, the regret is upper-bounded by a constant. Moreover, in this latter case, we can afford to use a learning rate $\eta$ that depends only on the self-boundedness\index{function!self-bounded} constant. However, this result is not fully satisfactory because it requires the knowledge of the future through the cumulative loss of the competitor. In the following, we show how to easily get rid of this limitation with a different choice of the learning rate.

In fact, under the same assumptions on the losses, from the regret in Theorem~\ref{thm:ada_grad_norm} and Definition~\ref{def:self-bounded}, we immediately obtain
\[
\Regret_T(\bu)
=\sum_{t=1}^T \ell_t(\bx_t) - \sum_{t=1}^T \ell_t(\bu)
\leq 2 D \sqrt{s \sum_{t=1}^T \ell_t(\bx_t)}, \quad \forall \bu \in \mathcal{V},
\]
where $D$ is the diameter of $\mathcal{V}$, assumed to be bounded.
This is an implicit bound, in the sense that $\sum_{t=1}^T \ell_t(\bx_t)$ appears on both sides of the inequality. To make it explicit, we will use the following simple lemma (proof left as an exercise).
\begin{lemma}
\label{lemma:sqrt}
Let $a,b,c\geq0$, and $x\geq0$ such that $x - \sqrt{a x+b} \leq c$. Then $x \leq \frac{a}{2} + c + \sqrt{\frac{a^2}{4}+b+ac}\leq a + c +\sqrt{b+ac}$.
\end{lemma}
So, we have the following theorem.
\begin{theorem}
\label{thm:l_star}
Let $\mathcal{V} \subset \R^d$ be a closed non-empty convex set with diameter $D$, that is, $D:=\max_{\bx,\by\in \mathcal{V}} \|\bx-\by\|_2$. Let $\ell_1, \dots, \ell_T$ be an arbitrary sequence of non-negative convex functions $\ell_t:\R^d \to (-\infty, +\infty]$ $s$-self-bounded\index{function!self-bounded} in $\mathcal{V}$ with respect to $\|\cdot\|_2$. Pick any $\bx_1 \in \mathcal{V}$, run projected \ac{OSD} with $\eta_{t}=\frac{\sqrt{2}D}{2\sqrt{\sum_{i=1}^t \|\bg_i\|_2^2}}, \ t=1, \dots, T$, and do not update on rounds in which $\bg_t=\boldsymbol{0}$. Then, $\forall \bu \in \mathcal{V}$, the following regret bound holds
\[
\Regret_T(\bu)
=\sum_{t=1}^T \ell_t(\bx_t) - \sum_{t=1}^T \ell_t(\bu)
\leq 4s D^2 + 2 D \sqrt{s \sum_{t=1}^T \ell_t(\bu)}~.
\]
\end{theorem}

This regret guarantee is very interesting because in the worst case it is $\mathcal{O}(\sqrt{T})$, but in the best case scenario it becomes a constant! In fact, if there exists a $\bu \in \mathcal{V}$ such that $\sum_{t=1}^T \ell_t(\bu)=0$ we get a constant regret. Basically, if the losses are ``easy'', the algorithm \emph{adapts} to this situation and gives us a better regret.

\begin{example}
Consider an online linear classification problem with the squared hinge loss\index{hinge loss!squared}. So, each loss is defined as $\ell_t(\bx)=\max(1-y_t \langle \bz_t, \bx\rangle,0)^2$ for labels $y_t \in \{-1, 1\}$ and features $\bz_t \in \R^d$. If we assume that $\|\bz_t\|_2$ is bounded for all $t$, then the losses are self-bounded\index{function!self-bounded} (and even smooth\index{function!smooth}) (proof left as an exercise). Under this assumption, if the problem is linearly separable in $\mathcal{V}$, i.e., there exists $\bu \in \mathcal{V}$ such that $\ell_t(\bu)=0$ for all $t$, Theorem~\ref{thm:l_star} will guarantee a constant regret.
\end{example}
\index{L* bound@$\mathscr{L}^\star$ bound|)}

\subsection{AdaGrad}
\label{sec:adagrad}

\index{AdaGrad algorithm|(textbf}
We now present another application of the regret bound in \eqref{eq:ada_grad_norm}. \emph{AdaGrad}, that stands for Adaptive Gradient, is an \ac{OCO} algorithm that aims at being adaptive to the sequence of (sub)gradients.

We will present a proof that only covers hyperrectangles as feasible sets $\mathcal{V}$. On the other hand, the restriction makes the proof almost trivial.
Let's see how it works.

AdaGrad has two key ingredients:
\begin{itemize}
\item A coordinate-wise learning process;
\item The adaptive learning rates in \eqref{eq:ada_eta}\index{learning rate!adaptive}.
\end{itemize}
For the first ingredient, as we said in Section~\ref{sec:oco_to_olo}, the regret of any \ac{OCO} problem can be upper bounded by the regret of the Online Linear Optimization (OLO) problem. That is,
\[
\sum_{t=1}^T \ell_t(\bx_t) - \sum_{t=1}^T \ell_t(\bu)
\leq \sum_{t=1}^T \langle \bg_t, \bx_t\rangle - \sum_{t=1}^T \langle \bg_t, \bu\rangle~.
\]
Now, the essential observation is to explicitly write the inner product as a sum of products over the individual coordinates:
\begin{align*}
\sum_{t=1}^T \langle \bg_t, \bx_t\rangle - \sum_{t=1}^T \langle \bg_t, \bu\rangle
&= \sum_{t=1}^T \sum_{i=1}^d g_{t,i} x_{t,i} - \sum_{t=1}^T \sum_{i=1}^d g_{t,i} u_i\\
&= \sum_{i=1}^d \left(\sum_{t=1}^T g_{t,i} x_{t,i} - \sum_{t=1}^T g_{t,i} u_i\right)
= \sum_{i=1}^d \Regret_{T,i}(u_i),
\end{align*}
where we denoted by $\Regret_{T,i}(u_i)$ the regret of the one-dimensional \ac{OLO} problem over coordinate $i$, that is $\sum_{t=1}^T g_{t,i} x_{t,i} - \sum_{t=1}^T g_{t,i} u_i$.
In words, \emph{we can decompose the original linear regret as the sum of $d$ \ac{OLO} regret minimization problems, and we can try to focus on each one of them separately}.

A good candidate for the one-dimensional problems is \ac{OSD} with the learning rates in \eqref{eq:ada_eta}.
We can specialize the regret in \eqref{eq:ada_grad_norm} to the one-dimensional case for linear losses, so we get for each coordinate $i$
\[
\sum_{t=1}^T g_{t,i} x_{t,i} - \sum_{t=1}^T g_{t,i} u_i
\leq \sqrt{2} D_i \sqrt{\sum_{t=1}^T g_{t,i}^2}~.
\]
This choice gives us the \textbf{AdaGrad} algorithm in Algorithm~\ref{alg:adagrad}.

\begin{algorithm}[t]
\caption{AdaGrad for Hyperrectangles}
\label{alg:adagrad}
\begin{algorithmic}[1]
{
    \REQUIRE{$\mathcal{V} = \{\bx \in \R^d: a_i \leq x_i\leq b_i\}$, $\bx_1 \in \mathcal{V}$}
    \FOR{$t=1$ {\bfseries to} $T$}
    \STATE{Output $\bx_t$}
    \STATE{Pay the loss $\ell_t(\bx_t)$, where $\ell_t$ is subdifferentiable on $\mathcal{V}$}
    \STATE{Set $\bg_t \in \partial \ell_t(\bx_t)$}
    \FOR{$i=1$ {\bfseries to} $d$}
    \IF{$g_{t,i}\neq 0$}
    \STATE{$x_{t+1,i} = \max(\min(x_{t,i} - \eta_{t,i} g_{t,i},b_i),a_i)$ where $\eta_{t,i}=\frac{\sqrt{2} D_i}{2\sqrt{\sum_{j=1}^t g_{j,i}^2}}$}
    \ELSE
    \STATE{$x_{t+1,i} = x_{t,i}$}
    \ENDIF
    \ENDFOR
    \ENDFOR
}
\end{algorithmic}
\end{algorithm}

Putting everything together, we have the following regret guarantee immediately.
\begin{theorem}
\label{thm:adagrad}
Let $\mathcal{V} = \{\bx \in \R^d: a_i \leq x_i\leq b_i\}$ with diameters along each coordinate equal to $D_i=b_i-a_i<\infty$. Let $\ell_1, \dots, \ell_T$ be an arbitrary sequence of convex functions $\ell_t:\R^d \to (-\infty, +\infty]$ subdifferentiable on $\mathcal{V}$ for $t=1, \dots,T$. Pick any $\bx_1 \in \mathcal{V}$. Then, $\forall \bu \in \mathcal{V}$, Algorithm~\ref{alg:adagrad} guarantees
\[
\Regret_T(\bu)
=\sum_{t=1}^T \ell_t(\bx_t) - \sum_{t=1}^T \ell_t(\bu)
\leq \sqrt{2} \sum_{i=1}^d D_i \sqrt{\sum_{t=1}^T g_{t,i}^2}~.
\]
\end{theorem}

Is this a better regret bound compared to the one in Theorem~\ref{thm:ada_grad_norm}? It depends!
To compare the two, let's first consider the case that $\mathcal{V}$ is a hyperrectangle.
Then, we have to compare
\begin{align*}
D\sqrt{\sum_{t=1}^T \|\bg_{t}\|_2^2} && \text{versus} && \sum_{i=1}^d D_i \sqrt{\sum_{t=1}^T g_{t,i}^2}~.
\end{align*}
From Cauchy--Schwarz\index{inequality!Cauchy--Schwarz}, we have that $\sum_{i=1}^d D_i \sqrt{\sum_{t=1}^T g_{t,i}^2}\leq D\sqrt{\sum_{t=1}^T \|\bg_{t}\|_2^2}$. So, \emph{assuming the same sequence of subgradients}, AdaGrad never has worse regret on hyperrectangles.
For a more precise quantification of the gain of AdaGrad, let's now assume that $\mathcal{V}$ is a hypersquare.
Also, note that
\begin{equation}
\label{eq:adagrad_comparison}
\sqrt{\sum_{t=1}^T \|\bg_{t}\|_2^2} \leq \sum_{i=1}^d \sqrt{\sum_{t=1}^T g_{t,i}^2} \leq \sqrt{d} \sqrt{\sum_{t=1}^T \|\bg_{t}\|_2^2},
\end{equation}
where the lower bound is by the fact that the L$_1$ norm is bigger than the L$_2$ norm, and the upper bound is given by Cauchy--Schwarz inequality\index{inequality!Cauchy--Schwarz}. So, in the case that $\mathcal{V}$ is a hypercube we have $D_i=D_\infty = \max_{\bx, \by} \|\bx-\by\|_\infty$ and $D=\sqrt{d} D_\infty$, the bound of AdaGrad is between $1/\sqrt{d}$ and $1$ times the bound of Theorem~\ref{thm:ada_grad_norm}. In other words, if we are lucky with the subgradients, we might save a factor of $\sqrt{d}$ in the guarantee.

However, what happens if the domain is an L$_2$ ball? First of all, it is possible to generalize AdaGrad to work on L$_2$ balls, and the guarantee remains the same. Hence, in this case, we have $D_i=D_\infty=D$, so from \eqref{eq:adagrad_comparison}, we have
\[
\frac{1}{\sqrt{d}}\sum_{i=1}^d D_i \sqrt{\sum_{t=1}^T g_{t,i}^2}
\leq D \sqrt{\sum_{t=1}^T \|\bg_{t}\|_2^2}
\leq \sum_{i=1}^d D_i \sqrt{\sum_{t=1}^T g_{t,i}^2} ~.
\]
Hence, for an L$_2$ ball, the opposite happens: the bound in Theorem~\ref{thm:ada_grad_norm} is never worse than the one of AdaGrad and, depending on the subgradients, we can gain a $\sqrt{d}$ factor.
Overall, the \emph{shape of the domain} determines the potential gain of one approach over the other, and the specific sequence of subgradients determines the actual gain.
It is possible to show that hyperrectangles are indeed the best domains for AdaGrad.
We will explore this issue of choosing the online algorithm based on the shape of the feasible set $\mathcal{V}$ when we introduce Online Mirror Descent in Chapter~\ref{ch:omd}.

Another big advantage of AdaGrad is the property of being \emph{coordinate-wise scale-free}\index{algorithm!coordinate-wise scale-free|textbf}\index{algorithm!scale-free|textbf}, according to the following definition.
\begin{definition}
\label{def:scale-free}
We say that an online learning algorithm is \textbf{coordinate-wise scale-free} when the iterates of the algorithm do not change if, assuming the same sequence of subgradients, each coordinate of the subgradients is multiplied by a different positive constant. We will say that the algorithm is simply \textbf{scale-free} if the same property holds when each coordinate is multiplied by the same positive constant.
\end{definition}

Another way to say it is that the update of AdaGrad is invariant to the units of each coordinate of the subgradients. This fact is not immediately apparent from the regret because by scaling the coordinate of the subgradients, the optimal solution $\bu$ would also scale accordingly, but the fixed diameters of the feasible set hide it. This might be useful in the case that the ranges of coordinates of the gradients are vastly different from one another. Indeed, this does happen in many machine learning problems, for example, in the stochastic optimization of deep neural networks, where the first layers have gradients of different magnitudes compared to the last layers.
\index{AdaGrad algorithm|)textbf}

\begin{remark}
AdaGrad is usually known as a stochastic optimization algorithm, but in reality, it was proposed for the online setting. To use it as a stochastic algorithm, you should use an online-to-batch conversion, otherwise you do not have any guarantee of convergence.
\end{remark}

\section{History Bits}
The concept of strong convexity\index{function!strongly convex} was first defined in \citet{Polyak66}.

The logarithmic regret in Corollary~\ref{cor:log_regret} was shown for the first time in the seminal paper~\citet{HazanKKA06,HazanAK07}. The general statement in Theorem~\ref{thm:log_regret} was proven by \citet{HazanRB08}.

\index{online-to-batch conversion|(}
The non-uniform averaging for the online-to-batch conversion of Example~\ref{ex:non_uniform_o2b_strongly} is from \citet{LacosteSB12}, but there it is not proposed as an online-to-batch conversion.
The basic idea of solving the \ac{SVM} problem with \ac{OSD} and online-to-batch conversion of Example~\ref{ex:non_uniform_o2b_strongly} was the Pegasos algorithm~\citep{Shalev-ShwartzSS07},\index{Pegasos algorithm} for many years the most used optimizer for \ac{SVM}s.
\index{online-to-batch conversion|)}

\index{learning rate!adaptive|(}
The adaptive learning rate in \eqref{eq:ada_eta} first appeared in \citet{StreeterM10}. However, similar methods were used a long time before. Indeed, the key observation to approximate oracle quantities with estimates up to time $t$ was first proposed in the self-confident algorithms~\citep{AuerCG02}\index{algorithm!self-confident}, where the learning rate is inversely proportional to the square root of the cumulative loss of the algorithm, and for self-bounded\index{function!self-bounded} losses it implies the $\mathscr{L}^\star$ bounds similar to the one in Theorem~\ref{thm:l_star}.\index{learning rate!adaptive|)} \index{L* bound@$\mathscr{L}^\star$ bound|(}The $\mathscr{L}^\star$ bound for the square loss and linear predictors was introduced by \citet{Cesa-BianchiLW96}. In the past, several works focused on obtaining regret upper bounds depending on constant times $\mathscr{L}^\star$ \citep[see, e.g.,][]{KivinenW97}, however, these guarantees are meaningful only if $\mathscr{L}^\star$ is sublinear in $T$. \citet{Zhang04b} explored the use of $\mathscr{L}^\star$ bounds in stochastic optimization.\index{L* bound@$\mathscr{L}^\star$ bound|)} The observation that the cumulative sum of the squared gradients acts as an intrinsic notion of time comes from the statistics literature, see the discussion in \citet{BlackwellF73}.

AdaGrad\index{AdaGrad algorithm|(} was proposed in basically identical form independently by two groups at the same conference: \citet{McMahanS10} and \citet{DuchiHS10}.
The analysis presented here is the one in \citet{StreeterM10} that does not handle generic feasible sets and does not support the ``full-matrices'' proposed in \citet{DuchiHS10}, i.e., full-matrix learning rates instead of diagonal ones. However, in machine learning applications, AdaGrad is usually used without a projection step (even if doing so provably destroys the worst-case performance, see Theorem~\ref{thm:lower_bound_osd}). Also, in the adversarial setting, full matrices do not seem to offer advantages in terms of regret compared to diagonal ones, see the discussion in \citet[Section 5]{Cutkosky20b}.

The AdaGrad learning rate is usually written as
\[
\eta_{t,i}=\frac{D_i}{\epsilon+\sqrt{\sum_{j=1}^t g_{j,i}^2}},
\]
where $\epsilon>0$ is a small constant used to prevent division by zero. In reality, $\epsilon$ is not necessary: there should be no update when the coordinate of the gradient is 0~\citep{OrabonaP15,OrabonaP18,AgarwalAHKZ20}. Moreover, removing $\epsilon$ makes the updates coordinate-wise scale-free\index{algorithm!coordinate-wise scale-free}, as stressed in \citet{OrabonaP15,OrabonaP18}. Scale-freeness in online learning has been introduced in \citet{Cesa-BianchiMS05,Cesa-BianchiMS07} for the setting of \ac{LEA} and in \citet{OrabonaP15,OrabonaP18} for \ac{OCO}.

\index{learning rate!adaptive|(}
AdaGrad inspired an incredible number of clones, most of them with similar, worse, or no regret guarantees. The keyword ``adaptive'' itself has shifted its meaning over time. It used to denote the ability of the algorithm to obtain the same guarantee as if it knew in advance a particular property of the data (i.e., adaptive to the gradients/noise/scale = (almost) same performance as if it knew the gradients/noise/scale in advance). Indeed, in statistics, this keyword is used with the same meaning. Nowadays, instead ``adaptive learning rates'' seems to denote any kind of coordinate-wise learning rates that do not guarantee anything in particular.
\index{learning rate!adaptive|)}
\index{AdaGrad algorithm|)}

\section{Exercises}

\begin{exer}
Prove that \ac{OSD} in Example~\ref{ex:guessing_strongly_convex} with $x_1=0$ is exactly the Follow-the-Leader\index{Follow-the-Leader algorithm} strategy for that particular problem.
\end{exer}

\begin{exer}
Prove that $\ell_t(\bx)=\|\bx-\bz_t\|_2^2$ is $2$-strongly convex\index{function!strongly convex} with respect to $\|\cdot\|_2$, derive the \ac{OSD} update for it, and its regret guarantee.
\end{exer}


\begin{exer}
Show that online subgradient descent on a bounded domain $\mathcal{V}$ with learning rates $\eta_t \propto 1/t$ obtains a $\mathcal{O}(\ln(1+\mathscr{L}^\star))$ regret bound for Lipschitz, self-bounded\index{function!self-bounded}, and strongly convex\index{function!strongly convex} losses.
\end{exer}


\begin{exer}
Prove that the logistic loss\index{logistic loss} $\ell(\bx)=\ln(1+\exp(-y\langle\bz,\bx\rangle))$, where $\|\bz\|_2\leq 1$ and $y \in \{-1,1\}$ is $\frac{1}{4}$-smooth\index{function!smooth} with respect to $\|\cdot\|_2$.
\end{exer}

\begin{exer}
Prove Lemma~\ref{lemma:sqrt}.
\end{exer}

\acresetall

\chapter{Lower Bounds for Online Linear Optimization}
\label{ch:lower}

In this chapter, we will present some lower bounds for \ac{OLO}. Since linear losses are convex, this immediately gives us lower bounds for \ac{OCO}.
We will consider both the constrained and the unconstrained case. The lower bounds in this chapter are important because they inform us about which algorithms are optimal  and where the gaps in our knowledge are.

\acresetall

\index{lower bound!bounded online linear optimization|(textbf}
\section{Lower Bound for Bounded Online Linear Optimization}

We will first consider the constrained bounded case. Finding a lower bound amounts to finding a strategy for the adversary that forces a certain regret onto the algorithm, \emph{no matter what the algorithm does}.
We will use the probabilistic method\index{probabilistic method} to construct our lower bound.

The basic method relies on the fact that if for a given $K\in \R$ we can construct a sequence of random vectors $\tilde{\bg}_1, \dots, \tilde{\bg}_T$ such that
\[
\E\left[\sum_{t=1}^T \langle \tilde{\bg}_t, \bx_t\rangle\right] \geq K,
\]
this implies that there exists a sequence $\bg_1, \dots, \bg_T$ among all the possible random sequences such that
\[
\sum_{t=1}^T \langle \bg_t, \bx_t\rangle \geq K~.
\]
It is easy to see why this is true: if for \emph{all} sequences $\sum_{t=1}^T \langle \tilde{\bg}_t, \bx_t\rangle <K$, then the expectation would also be strictly less than $K$, contradicting our assumption.

For us, it means that we prove the existence of a ``difficult'' sequence of functions through a result on the expectation with respect to a distribution over stochastic functions. Why do we rely on expectations rather than actually constructing an adversarial sequence? Because the use of stochastic loss functions makes it very easy to deal with arbitrary algorithms. In particular, we will choose a distribution over stochastic loss functions that makes the expected loss of the algorithm equal to 0, independently of the strategy of the algorithm.

\begin{theorem}
\label{thm:lower_bound_constr}
Let $\mathcal{V} \subset \R^d$ be any non-empty bounded closed convex subset. Let $D = \max_{\bv,\bw \in \mathcal{V}} \|\bv - \bw\|_2>0$ be the diameter of $\mathcal{V}$. Let $\mathscr{A}$ be any deterministic algorithm for \ac{OLO} on $\mathcal{V}$. Let $T$ be any positive integer. Then, there exists a sequence of vectors $\bg_1, \dots, \bg_T$ with $\|\bg_t\|_2\leq L$ and $\bu \in \mathcal{V}$ such that the regret of algorithm $\mathscr{A}$ satisfies
\[
\Regret_T(\bu)
= \sum_{t=1}^T \langle \bg_t, \bx_t \rangle - \sum_{t=1}^T \langle \bg_t, \bu \rangle
\geq \frac{\sqrt{2}LD \sqrt{T}}{4}~.
\]
\end{theorem}
\begin{proof}
Let's denote by $\Regret_T := \max_{\bu \in \mathcal{V}} \Regret_T(\bu)$.
Let $\bv, \bw \in \mathcal{V}$ such that $\|\bv-\bw\|_2=D$. Let $\bz=\frac{\bv-\bw}{\|\bv-\bw\|_2}$, so that $\langle \bz, \bv-\bw\rangle = D$.
Let $\epsilon_1, \dots, \epsilon_T$ be i.i.d. Rademacher random variables\index{random variable!Rademacher}, that is $\Pr\{\epsilon_t=1\}=\Pr\{\epsilon_t=-1\}=1/2$, and set the vector of the stochastic linear losses $\tilde{\bg}_t=L \epsilon_t \bz$.
So, we have
\begin{align*}
&\E_{\tilde{\bg}_1,\dots,\tilde{\bg}_T} \left[ \sum_{t=1}^T \langle \tilde{\bg}_t, \bx_t\rangle - \min_{\bu \in \mathcal{V}} \sum_{t=1}^T \langle \tilde{\bg}_t, \bu\rangle\right]\\
&\quad= \E_{\epsilon_1,\dots,\epsilon_T}\left[ \sum_{t=1}^T L \epsilon_t \langle \bz, \bx_t\rangle - \min_{\bu \in \mathcal{V}} \sum_{t=1}^T L \epsilon_t \langle \bz, \bu\rangle\right]\\
&\quad= \E_{\epsilon_1,\dots,\epsilon_T}\left[- \min_{\bu \in \mathcal{V}} \sum_{t=1}^T  L \epsilon_t \langle \bz, \bu\rangle\right]
= \E_{\epsilon_1,\dots,\epsilon_T}\left[\max_{\bu \in \mathcal{V}} \sum_{t=1}^T  - L \epsilon_t \langle \bz, \bu\rangle\right]\\
&\quad= \E_{\epsilon_1,\dots,\epsilon_T}\left[\max_{\bu \in \mathcal{V}} \sum_{t=1}^T  L \epsilon_t \langle \bz, \bu\rangle\right]
\geq \E_{\epsilon_1,\dots,\epsilon_T}\left[\max_{\bu \in \{\bv, \bw\}} \sum_{t=1}^T L \epsilon_t \langle \bz, \bu\rangle\right] \\
&\quad= \E_{\epsilon_1,\dots,\epsilon_T}\left[\frac{1}{2}\sum_{t=1}^T  L \epsilon_t \langle \bz, \bv+\bw\rangle + \frac{1}{2}\left| \sum_{t=1}^T L \epsilon_t \langle \bz, \bv-\bw\rangle\right|\right] \\
&\quad= \frac{L}{2}\E_{\epsilon_1,\dots,\epsilon_T}\left[\left|\sum_{t=1}^T  \epsilon_t \langle \bz, \bv-\bw\rangle\right|\right]
= \frac{LD}{2}\E_{\epsilon_1,\dots,\epsilon_T}\left[\left|\sum_{t=1}^T  \epsilon_t \right|\right]
\geq \frac{\sqrt{2}LD \sqrt{T}}{4}~.
\end{align*}
where in the first equality we used $\E[\epsilon_t]=0$ and the independence of $\epsilon_t$ and $\bx_t$, the fact that $\epsilon_t$ and $-\epsilon_t$ follow the same distribution in the fourth equality, $\max(a,b)=\frac{a+b}{2}+\frac{|a-b|}{2}$ in the fifth equality, and Khintchine inequality\index{inequality!Khintchine's} (Theorem~\ref{thm:khintchine}) in the last inequality.

Now, given that the expectation is lower bounded by a positive constant, there exists a sequence of realizations of the random variables that gives the same lower bound.
\end{proof}

\begin{remark}
Unlike similar proofs, we do not assume $\mathcal{V}$ to be symmetric with respect to $\boldsymbol{0}$.
\end{remark}

\begin{remark}
If the algorithm is randomized, one can easily extend the proof for a lower bound on the expected regret of the algorithm with respect to its internal randomization.
\end{remark}

The lower bound is within a constant multiplicative factor from the upper bound we proved for \ac{OSD} with learning rates $\eta_t=\frac{D}{L\sqrt{t}}$ or $\eta=\frac{D}{L\sqrt{T}}$. This means that \ac{OSD} is asymptotically optimal with both settings of the learning rate.

At this point, there is an important consideration: how can this be the optimal regret when we managed to prove a better regret, for example, with adaptive learning rates in Section~\ref{sec:lstar}? The subtlety is that, constraining the adversary to play $L$-Lipschitz losses, the adversary could always force on the algorithm at least the regret in Theorem~\ref{thm:lower_bound_constr}. However, we can design algorithms that take advantage of \emph{suboptimal plays of the adversary}. Indeed, for example, if the adversary plays in a way that all the subgradients have the same norm equal to $L$, there is nothing to adapt to!
\index{lower bound!bounded online linear optimization|)textbf}

\index{lower bound!unconstrained online subgradient descent|(textbf}
\section{Lower Bound for Unconstrained Online Subgradient Descent}
\label{sec:unbounded_ogd_fail}

Here, we will focus on a specific algorithm, that is, \ac{OSD}. We want to show that the limitation we saw in Chapter~\ref{ch:osd} of \ac{OSD} with time-varying stepsizes to be used only on bounded domains is real. In fact, we can prove the following lower bound.
\begin{theorem}
\label{thm:lower_bound_osd}
Let $\alpha \in (0,1)$, $\phi:(0,1) \to (0,1-\ln 2)$ defined as $\phi(\alpha):=\frac{1}{2-\alpha} + \frac{(1/2)^{1-\alpha}-1}{1-\alpha}$, and $T\geq \frac{2}{(1-\alpha) \phi(\alpha)}$.
For unprojected \ac{OSD} with stepsizes $\eta_t=t^{-\alpha}$ and $\bx_1=\boldsymbol{0}$, there exists a sequence of $T$ convex and 1-Lipschitz losses such that
\[
\Regret_T(\boldsymbol{0})\geq \frac12 \phi(\alpha) T^{2-\alpha}~.
\]
Also, we have that $\lim_{\alpha\to1} \ \phi(\alpha)=1-\ln 2\geq 0.3$.
\end{theorem}
\begin{proof}
We assume $d=1$. For $d \ge 2$, we simply embed the one-dimensional loss vectors into the first coordinate of $\R^d$. Note that the condition on $T$ implies $T\geq 2$. Consider the sequence
\[
(\ell_1(x), \dots, \ell_T(x))
= ( \underbrace{-x, \dots, -x}_{\lceil T/2 \rceil}, \underbrace{x,  \dots, x}_{\lfloor T/2 \rfloor})~.
\]
That is, the first half consists of $-x$'s, the second of $+x$'s. For $t \le \lceil T/2 \rceil$, we have $x_{t+1} = x_t + t^{-\alpha}$.
Unrolling the recurrence and using $x_1 = 0$ we get
\[
x_t
= \sum_{i=1}^{t-1} i^{-\alpha}, \quad t \le \lceil T/2 \rceil + 1~.
\]
On the other hand, for $t \ge \lceil T/2 \rceil + 1$, we have $x_{t+1} = x_t - t^{-\alpha}$.
Unrolling the recurrence up to $x_{\lceil T/2 \rceil + 1}$ we get
\[
x_t
= x_{\lceil T/2 \rceil + 1} \ \ - \sum_{i=\lceil T/2 \rceil + 1}^{t-1} i^{-\alpha}
= \sum_{i=1}^{\lceil T/2 \rceil} i^{-\alpha} \ \ - \sum_{i=\lceil T/2 \rceil + 1}^{t-1} i^{-\alpha},
\quad  t \ge \lceil T/2 \rceil + 1~.
\]
We are ready to lower bound the regret:
\begin{align*}
\Regret_T(0)
& = - \sum_{t=1}^{\lceil T/2 \rceil} x_t + \sum_{t=\lceil T/2 \rceil + 1}^T x_t  \\
& = - \sum_{t=1}^{\lceil T/2 \rceil} \sum_{i=1}^{t-1} i^{-\alpha} + \sum_{t=\lceil T/2 \rceil + 1}^T \left( \sum_{i=1}^{\lceil T/2 \rceil} i^{-\alpha} - \sum_{i=\lceil T/2 \rceil + 1}^{t-1} i^{-\alpha} \right) \\
& = - \sum_{i=1}^{\lceil T/2 \rceil} \frac{\lceil T/2 \rceil - i}{i^{\alpha}} + \lfloor T/2 \rfloor \sum_{i=1}^{\lceil T/2 \rceil} i^{-\alpha} - \sum_{i=\lceil T/2 \rceil + 1}^T \frac{T - i}{i^{\alpha}} \\
& = - \sum_{i=1}^{\lceil T/2 \rceil} \frac{\lceil T/2 \rceil - \lfloor T/2 \rfloor}{i^{\alpha}} + \sum_{i=1}^{T} i^{1-\alpha}  - T \sum_{i=\lceil T/2 \rceil + 1}^T i^{-\alpha}~.
\end{align*}
Since $\lceil T/2 \rceil-\lfloor T/2 \rfloor\le 1$ and using the integral bounds, we have
\begin{align*}
\Regret_T(0)
& \ge - 1 - \int_{i=1}^{\lceil T/2 \rceil} \! x^{-\alpha} \, \mathrm{d}x + \int_{0}^T \! x^{1-\alpha} \, \mathrm{d}x - T \int_{\lceil T/2 \rceil}^T \! x^{-\alpha} \, \mathrm{d}x \\
& = - 1 - \frac{\lceil T/2 \rceil^{1-\alpha} - 1 }{1-\alpha} + \frac{T^{2-\alpha}}{2-\alpha} - T \frac{T^{1-\alpha} - \lceil T/2 \rceil^{1-\alpha}}{1-\alpha} \\
& \geq - 1 + \frac{1}{1-\alpha}- \frac{\lceil T/2 \rceil^{1-\alpha} }{1-\alpha} + \frac{T^{2-\alpha}}{2-\alpha} - \frac{T^{2-\alpha} - T ( T/2) ^{1-\alpha}}{1-\alpha} \\
& \ge - \frac{T^{1-\alpha}}{1-\alpha} + \left( \frac{1}{2-\alpha} + \frac{(1/2)^{1-\alpha}}{1-\alpha}- \frac{1}{1-\alpha}\right) T^{2-\alpha}\\
& = - \frac{T^{1-\alpha}}{1-\alpha} +\phi(\alpha) T^{2-\alpha}
\geq \frac{1}{2}\phi(\alpha) T^{2-\alpha}~. \qedhere
\end{align*}
\end{proof}

This lower bound tells us that \ac{OSD} can indeed fail in unbounded domains when used with a polynomially decreasing stepsize. However, it does not rule out the possibility of another algorithm working in the same setting. Indeed, we will see in Chapter~\ref{ch:ftrl} that Follow-the-Regularized-Leader achieves sublinear regret on unbounded domains with time-varying regularizers. Yet, its dependence on the other quantities will still be suboptimal, and we will obtain the optimal bound only with parameter-free algorithms in Chapter~\ref{ch:parameterfree}. Indeed, in the next section, we prove that unconstrained \ac{OLO} is actually more difficult than \ac{OLO} in bounded domains, \emph{for any algorithm}.
\index{lower bound!unconstrained online subgradient descent|)textbf}

\index{lower bound!unconstrained online linear optimization|(textbf}
\section{Lower Bound for Unconstrained Online Linear Optimization}
\label{sec:lower_unconstrained_olo}

In the unconstrained setting, we proved that \ac{OSD} with $\bx_1=\boldsymbol{0}$ and constant learning rate of $\eta=\frac{1}{L\sqrt{T}}$ gives a regret of $\frac12 L(\|\bu\|_2^2+1)\sqrt{T}$ for any $\bu \in \R^d$. Is this regret optimal? It is clear that the regret must be at least linear in $\|\bu\|_2$. In fact, we could select a specific $\bu$, pass the information of $\|\bu\|_2$ to the online learning algorithm and make the problem constrained in $\mathcal{V}=\{\bx:\|\bx\|_2\leq \|\bu\|_2\}$, so that the lower bound for bounded \ac{OLO} would hold. However, we now show that the correct dependence in $\|\bu\|_2$ is more than linear, so the unconstrained setting is strictly more difficult than the bounded one.

The approach I will follow is to \emph{reduce the \ac{OLO} game to an online game of betting on a coin}, where the lower bounds are easier to prove.
So, let's introduce the coin-betting online game:\index{coin-betting game|(}
\begin{itemize}
\item Start with an initial amount of money $\epsilon>0$.
\item In each round, the algorithm bets a fraction of its current wealth on the outcome of a coin.
\item The outcome of the coin is revealed, and the algorithm wins or loses its bet at even odds.
\end{itemize}
The aim of this online game is to win as much money as possible. Also, as in all the online games we consider, we do not assume anything about how the outcomes of the coin are decided.

We will denote by $c_t \in \{-1,1\}, t=1, \dots, T$ the outcomes of the coin. The absolute value of the signed betting fraction $\beta_t \in [-1,1]$\index{signed betting fraction} will denote the fraction of money to bet, and its sign will denote on which side we are betting. The net gain of the algorithm from the beginning of the game until the end of round $t$ will be denoted by $r_{t}$. Given that the money is won or lost at even odds, we have
\[
\overbrace{r_t +\epsilon}^{\text{Money at the end of round } t}
= \overbrace{r_{t-1} +\epsilon}^{\text{Money at the beginning of round } t} + \overbrace{c_t \beta_t (r_{t-1}+\epsilon)}^\text{Money won or lost}
= \epsilon \prod_{i=1}^t (1+\beta_i c_i),
\]
where we used the fact that $r_0=0$. We will also denote by $x_t=\beta_t (\epsilon + r_{t-1})$ the signed bet of the algorithm on round $t$.

If we got all the outcomes of the coin correct, we would double our money in each round, so that $\epsilon+r_T=\epsilon 2^T$. On the other hand, if the adversary can always select a coin outcome that is the opposite of our bet, then we would lose money in each round. Hence, we are interested in the best any algorithm can do when the adversary is constrained to give us a sequence of coins where one side appears more often than the other.

To facilitate the use of this theorem later on, we will also slightly generalize the problem assuming that the coins are in $\{-L, L\}$ instead of $\{-1, 1\}$, where $L>0$, consequently we impose $\beta_t \in [-\frac{1}{L}, \frac{1}{L}]$.

\begin{theorem}
\label{thm:lower_bound_coin}
Let $T\geq 1$ be even and $0< q\leq \frac{T}{2}$ integer. Then, for any online coin-betting algorithm that guarantees non-negative wealth on any sequence of $T$ coins in $\{-L,L\}$ and starts with initial wealth $\epsilon$, there exists a sequence of coins such that $|\sum_{t=1}^T c_t| \geq 2q L$ and the wealth of the algorithm is upper bounded by
\[
\frac{3 \epsilon}{2} \left(\frac{2q}{\sqrt{T}}+1\right) \exp\!\left(T \cdot \KLBern\!\left(\frac{1}{2}+\frac{q}{T} ; \frac{1}{2}\right)\right)
\leq \frac{3\epsilon}{2} \left(\frac{2q}{\sqrt{T}}+1\right) \exp\!\left( 2 \frac{q^2}{T}+ 3.1 \frac{q^4}{T^3}\right)\!,
\]
where $\KLBern(p;q)$ denotes the \ac{KL} divergence\index{Kullback--Leibler divergence!between Bernoulli distributions|textbf} between two Bernoulli distributions with parameters $p$ and $q$:
\[
\KLBern(p;q) := p \ln \left( \frac{p}{q} \right) + (1-p) \ln \left( \frac{1-p}{1-q} \right)~.
\]
\end{theorem}
\begin{proof}
Let $Y_1, \dots, Y_T$ be independent random variables that take the value of $1$ with probability 0.5 and $-1$ with probability 0.5.
Hence, we have that $\E[ \sum_{t=1}^T x_t L Y_t ]=0$, and also $\sum_{t=1}^T x_t L Y_t  \geq -\epsilon$ for the hypothesis on the betting algorithm.
For any $q\geq0$, it follows that
\begin{align*}
0&=\E\left[\sum_{t=1}^T x_t Y_t \right] \\
&= \E\left[\sum_{t=1}^T x_t Y_t \middle| \left|\sum_{t=1}^T Y_t\right|< 2 q\right] \Pr\left\{\left|\sum_{t=1}^T Y_t\right| < 2q \right\}\\
&\quad +\E\left[\sum_{t=1}^T x_t Y_t \middle| \left|\sum_{t=1}^T Y_t\right| \geq 2 q\right] \Pr\left\{\left|\sum_{t=1}^T Y_t\right|\geq 2 q\right\} \\
 &\geq -\frac{\epsilon}{L}+\left(\frac{\epsilon}{L}+\E\left[\sum_{t=1}^T x_t Y_t \middle| \left|\sum_{t=1}^T Y_t\right| \geq 2 q\right]\right) \Pr\left\{\left|\sum_{t=1}^T Y_t\right| \geq 2 q\right\},
\end{align*}
hence
\begin{align*}
\E\left[\sum_{t=1}^T x_t Y_t \middle| \left|\sum_{t=1}^T Y_t\right| \geq 2 q\right]
&\leq \frac{\epsilon}{L\,\Pr\left\{\left|\sum_{t=1}^T Y_t\right| \geq 2 q\right\}} - \frac{\epsilon}{L}\\
&= \frac{\epsilon}{2 L\,\Pr\left\{\sum_{t=1}^T Y_t \geq 2 q\right\}} - \frac{\epsilon}{L}~.
\end{align*}
Using the fact that $\Pr\left\{\sum_{t=1}^T Y_t \geq 2 q\right\} = \Pr\left\{\sum_{t=1}^T \frac{Y_t + 1}{2} \geq \frac{1}{2}T+ q \right\}$, where $\frac{Y_t + 1}{2}$ are Bernoulli random variables\index{random variable!Bernoulli}, we can apply Lemma~\ref{lemma:bin}, to lower bound the tail of the distribution:
\begin{align*}
\Pr\left\{\sum_{t=1}^T Y_t \geq 2 q\right\}
\geq \frac{1}{3} \frac{1}{\frac{2q}{\sqrt{T}}+1}\exp\left(-T \cdot \KLBern\left(\frac12+\frac{q}{T} ; \frac{1}{2}\right)\right)~.
\end{align*}
Hence, we have
\begin{align*}
\E\left[\sum_{t=1}^T  x_t L Y_t \middle| \left|\sum_{t=1}^T Y_t \right| \geq 2 q\right]
\leq \frac{3}{2} \epsilon \left(\frac{2q}{\sqrt{T}}+1\right) \exp\left(T \cdot \KLBern\left(\frac{1}{2}+\frac{q}{T} ; \frac{1}{2}\right)\right)- \epsilon~.
\end{align*}
Given that the minimum over a set is smaller than or equal to the expectation with respect to any distribution over the set, there exists a sequence of $c_1, \dots, c_T \in \{-L,L\}^T$ such that $\frac{1}{L}|\sum_{t=1}^T c_t| \geq 2q$ and the wealth of the algorithm is deterministically upper bounded by
\[
\frac{3}{2} \epsilon \left(\frac{2q}{\sqrt{T}}+1\right) \exp\left(T \cdot \KLBern\left(\frac{1}{2}+\frac{q}{T} ; \frac{1}{2}\right)\right)~.
\]

For the second upper bound, it is enough to use the elementary inequality
\[
\KLBern\left(\frac{1}{2}+x ; \frac{1}{2}\right)
\leq 2 x^2 + 3.1 x^4,  \quad |x|\leq \frac12~. \qedhere
\]
\end{proof}

\begin{remark}
It is also possible to upper bound the l.h.s. of the inequality in the above theorem by a quantity that depends on the wealth of the best constant betting fraction.

For the expression of the optimal wealth on the $c_1, \dots, c_T$, consider the wealth of a strategy that bets a constant signed fraction of money $\beta$. Starting with initial money $\epsilon$, after $T$ rounds the wealth is $\epsilon \prod_{t=1}^T (1+\beta c_t)$.
By taking the derivative of the logarithm of the wealth, it is immediate to verify that the $\beta^\star$ that maximizes the above quantity is $\frac{\sum_{t=1}^T c_t}{L^2 T}$.
Denote by $k=|\{c_t: c_t=L\}|$, hence we have $\beta^\star=\frac{2k-T}{L T}$. Hence, the optimal wealth is
\begin{align*}
\epsilon (1+\beta^\star L)^k (1-\beta^\star L)^{T-k}
&= \epsilon \exp\left( k \ln \frac{2k}{T} + (T-k) \ln \left(2-\frac{2k}{T}\right)\right)\\
&= \epsilon \exp\left( T \cdot \KLBern\left(\frac{k}{T} ;\frac{1}{2}\right)\right)~.
\end{align*}
Equivalently, given that $\frac{1}{L}\sum_{t=1}^T c_t=2k-T$, we also have that
\begin{align*}
\max_{-1/L\leq \beta\leq 1/L} \prod_{t=1}^T (1+\beta c_t)
&= \exp\left( T \cdot \KLBern\left(\frac{\sum_{t=1}^T c_t }{2 L T}+\frac12;\frac{1}{2}\right)\right)\\
&= \exp\left( T \cdot \KLBern\left(\frac{|\sum_{t=1}^T c_t |}{2 L T}+\frac12;\frac{1}{2}\right)\right),
\end{align*}
where in the second equality we used the fact that $\KLBern(x+\frac12;\frac12)=\KLBern(-x+\frac12;\frac12)$.
Hence, we have
\begin{align*}
&\frac{3}{2} \epsilon \left(\frac{2q}{\sqrt{T}}+1\right) \exp\left(T \cdot \KLBern\left(\frac{1}{2}+\frac{q}{T} ; \frac{1}{2}\right)\right)\\
&\quad\leq \frac{3}{2} \epsilon \left(\frac{|\sum_{t=1}^T c_t|}{L\sqrt{T}}+1\right) \exp\left(T \cdot \KLBern\left(\frac{1}{2}+\frac{|\sum_{t=1}^T c_t|}{2 L T} ; \frac{1}{2}\right)\right)\\
&\quad= \frac{3}{2} \left(\frac{|\sum_{t=1}^T c_t|}{L\sqrt{T}}+1\right) \max_{\beta} \epsilon \prod_{t=1}^T (1+c_t \beta)~.
\end{align*}
\end{remark}
\index{coin-betting game|)}

Now, let's connect the coin-betting game with \ac{OLO}, thanks to the next theorem.
\begin{theorem}
\label{thm:epsilon_regret_olo_is_betting}
Let $\epsilon_t$ be a non-negative non-decreasing sequence and $\mathscr{A}$ an \ac{OLO} algorithm that guarantees $\Regret_t(\boldsymbol{0})\leq \epsilon_t$ for any sequence of $\bg_1,\dots,\bg_t \in \R^d$ with $\|\bg_i\|_2\leq L$. Then, for any $T\geq 0$ there exists $\bbeta_t$ such that $\bx_t= \bbeta_t (\epsilon_T-\sum_{i=1}^{t-1} \langle \bg_i, \bx_i\rangle)$ and $\|\bbeta_t\|_2\leq\frac{1}{L}$ for $t=1, \dots, T$.
\end{theorem}
\begin{proof}
Define $r_t=-\sum_{i=1}^t \langle \bg_i, \bx_i\rangle$ as the ``reward'' of the algorithm.
So, we have
\[
\Regret_t(\bu)
=\sum_{i=1}^t \langle \bg_i, \bx_i\rangle - \sum_{i=1}^t \langle \bg_i, \bu\rangle
= -r_t + \left\langle \sum_{i=1}^t \bg_i, \bu\right\rangle~.
\]
Since we assumed that $\Regret_t(\boldsymbol{0})\leq \epsilon_t$, we always have $r_t \geq -\epsilon_t$.
Using this, we claim that $L \|\bx_t\|_2 \leq r_{t-1} + \epsilon_t$ for all $t=1,\dots,T$. To see this, assume that there is a sequence $\bg_1, \dots, \bg_{t-1}$ that gives $L \|\bx_t\|_2 > r_{t-1} + \epsilon_t$. If $\|\bx_t\|_2=0$ then the stated inequality holds. Instead, for $\|\bx_t\|_2\neq 0$, set $\bg_t=L\frac{\bx_t}{\|\bx_t\|_2}$. For this sequence, we would have $r_t = r_{t-1} - L\|\bx_t\|_2 < -\epsilon_t$, that contradicts the observation that $r_t \geq -\epsilon_t$.

So, from the fact that $L\|\bx_t\|_2 \leq r_{t-1} + \epsilon_t \leq r_{t-1} + \epsilon_T$ we have that there exists $\bbeta_t$ such that $\bx_t = \bbeta_t (\epsilon_T+r_{t-1})$ and $\|\bbeta_t\|_2 \leq \frac{1}{L}$.
\end{proof}

This theorem informs us of something important: \emph{any \ac{OLO} algorithm that suffers at most a non-decreasing regret in $t$ against the competitor $\bu=\boldsymbol{0}$ predicts in the form of a ``vectorial'' coin-betting algorithm}.

From this connection to online betting, we now derive a lower bound for \ac{OLO}.

\begin{theorem}
\label{thm:olo_coin_betting_lower_bound}
Let $T \in\Nat$ be even and let $\mathscr{A}$ be any \ac{OLO} algorithm that guarantees regret at most $\epsilon_T>0$ against the null competitor on any sequence of $T$ linear and $L$-Lipschitz losses $\ell_t:\R \to \R$ for $t=1, \dots, T$. Let $U>0$ be such that $1 \leq W(\frac{\sqrt{e T} U L}{8 \epsilon_T})\leq \frac{\sqrt{T}}{2}$. Then, for this algorithm, there exists a sequence of $g_t$ with $|g_t|\leq L$ and a competitor $u \in \R$ with $|u|=U$, such that
\[
\sum_{t=1}^T g_t (x_t-u)
\geq R_T(U) := U L \sqrt{T}\left(\sqrt{2  W\left(\frac{\sqrt{T} U L}{5 \epsilon_T} \right)}-1\right) - 2 U L + \epsilon_T,
\]
where $W:\R_{\geq0} \to \R_{\geq0}$ is the Lambert function\index{Lambert function} (see Appendix~\ref{sec:lambert}).

Moreover, assuming there exist $0<K_1\leq K_2<\infty $ such that $K_1 \leq \epsilon_T \leq K_2$ for all $T$, we have $\lim_{T\to +\infty} \ \frac{R_T(U)}{UL\sqrt{T \ln T}} = 1$.
\end{theorem}
\begin{proof}
Consider the linear losses $\ell_t(x)= g_t x$, where $|g_t|\leq L$.

Given that the algorithm guarantees a regret of at most $\epsilon_T$ against the null competitor, from Theorem~\ref{thm:epsilon_regret_olo_is_betting} we have that, if we feed the algorithm with linear losses, we can reinterpret the algorithm as a betting algorithm with initial wealth $\epsilon_T$ and that guarantees non-negative wealth, where the wealth at time $T$ is defined as $\epsilon_T - \sum_{t=1}^T g_t x_t$.

Set $A=W\left(\frac{\sqrt{e T} U L}{8 \epsilon_T} \right)$, $\tilde{q}=\frac{\sqrt{T}}{2}\sqrt{2A-1}$, and $q=\lfloor \tilde{q}\rfloor$, so that $\tilde{q} \geq q \geq \tilde{q}-1$.

%


Observe that the constraint on $U$ ensures that
\begin{equation}
\label{eq:proof_olo_coin_betting_lower_bound_eq1}
q
\leq \tilde{q}
\leq \frac{T^{3/4}}{2}
\leq \frac{T}{2},
\end{equation}
so we can safely use Theorem~\ref{thm:lower_bound_coin}, that tells us that there exists a sequence of $c_1, \dots, c_T \in \{-L,L\}$ such that $|\sum_{t=1}^T c_t|\geq 2qL$ and
\[
\sum_{t=1}^T c_t x_t
\leq \frac{3}{2} \left(\frac{2q}{\sqrt{T}}+1\right) \epsilon_T \exp\left(2 \frac{q^2}{T} \right) \exp\left(3.1 \frac{q^4}{T^3}\right) - \epsilon_T~.
\]
Moreover, observe that
\begin{align*}
\exp\left(2 \frac{q^2}{T} \right) \exp\left(3.1 \frac{q^4}{T^3}\right)
&\leq \exp\left(2 \frac{\tilde{q}^2}{T} \right) \exp\left(3.1 \frac{\tilde{q}^4}{T^3}\right)
\leq \exp\left(2 \frac{\tilde{q}^2}{T} \right) \exp\left(\frac{3.1}{16}\right)\\
&\leq \frac{4}{3} \exp\left(2 \frac{\tilde{q}^2}{T} \right),
\end{align*}
where in second inequality we used $\frac{\tilde{q}^4}{T^3}\leq \frac{1}{16}$ from \eqref{eq:proof_olo_coin_betting_lower_bound_eq1}.

Set $g_t=-c_t$ and choose $u=-U\sign(\sum_{t=1}^T g_t)$, so we have
\begin{align*}
\sum_{t=1}^T g_t(x_t-u)
&= U \left|\sum_{t=1}^T g_t\right| + \sum_{t=1}^T g_t x_t\\
&\geq U \cdot 2q L - 2 \left(\frac{2\tilde{q}}{\sqrt{T}}+1\right) \epsilon_T \exp\left(2\frac{\tilde{q}^2}{T}\right)+\epsilon_T\\
&\geq U L \cdot (2\tilde{q}-2) - 2 \left(\frac{2\tilde{q}}{\sqrt{T}}+1\right) \epsilon_T \exp\left(2\frac{\tilde{q}^2}{T}\right)+\epsilon_T~.
\end{align*}
Now, from the condition on $U$, we have $\frac{2\tilde{q}}{\sqrt{T}}\geq 1$, hence we have
\begin{align}
\sum_{t=1}^T g_t(x_t-u)
&= U \left|\sum_{t=1}^T g_t\right| + \sum_{t=1}^T g_t x_t \nonumber \\
&\geq 4 \epsilon_T \left[\frac{\sqrt{T} U L}{4\epsilon_T} \cdot \frac{2\tilde{q}}{\sqrt{T}} - \frac{2\tilde{q}}{\sqrt{T}}\exp\left(2\frac{\tilde{q}^2}{T}\right)\right] - 2 U L + \epsilon_T~. \label{eq:proof_olo_coin_betting_lower_bound_eq2}
\end{align}

Let $a\geq 1$, then
\[
x^\star :=
\argmax_{x \geq 0} \ a x - x \exp(x^2/2)
= \sqrt{2 W\left(\frac{a \sqrt{e}}{2}\right)-1}~.
\]
So, we have
\[
a x^\star - x^\star \exp((x^\star)^2/2)
= a\sqrt{2 W\left(\frac{a \sqrt{e}}{2}\right)-1}\left(1 - \frac{1}{2 W\left(\frac{a \sqrt{e}}{2}\right)}\right)~.
\]
Hence, our choice of $\tilde{q}$ maximizes the lower bound in \eqref{eq:proof_olo_coin_betting_lower_bound_eq2}. Hence, we obtain
\[
\sum_{t=1}^T g_t(x_t-u)
\geq U L \sqrt{T} \sqrt{2 W\left(\frac{\sqrt{e\,T} U L}{8\epsilon_T} \right)-1}\left(1- \frac{1}{2 W\left(\frac{\sqrt{e\,T} U L}{8\epsilon_T} \right)}\right) - 2 U L + \epsilon_T~.
\]
The condition on $U$ ensures that $W(\frac{\sqrt{e T} U L}{8 \epsilon_T})\geq 1$.
So, using the elementary inequality $(1-1/(2x))\sqrt{2x-1}\geq \sqrt{2x}-1$ for $x\geq1$ and simplifying the numerical constants, we get the stated bound.

The limit statement is obtained using Theorem~\ref{thm:limit_lambert}.
\end{proof}

\begin{remark}
The leading constant $\sqrt{2}$ is asymptotically optimal because there exist algorithms with a matching upper bound, see Section~\ref{sec:history_lower}.
\end{remark}

This theorem implies that \ac{OSD} with learning rate $\eta=\frac{\alpha}{L \sqrt{T}}$ does not have the optimal dependence on $\|\bu\|_2$ for any $\alpha>0$.
\index{lower bound!unconstrained online linear optimization|)textbf}

In Chapter~\ref{ch:parameterfree}, we will see that the connection between coin betting and \ac{OLO} can also be used to design an \ac{OLO} algorithm. This will give us \emph{optimal unconstrained \ac{OLO} algorithms with the surprising property of not requiring a learning rate at all}.

\section{History Bits}
\label{sec:history_lower}
\index{lower bound!bounded online linear optimization|(}
The lower bound for bounded \ac{OLO} is quite standard, and the proof presented is a simplified version of the one in \citet{OrabonaP15,OrabonaP18}.
One could also use the function in the lower bound for offline optimization in \citet[Section 3.2.1]{Nesterov04}, but it would require the additional assumption that $d>T$ and that $\bx_t$ lies in the span of the previous subgradients. This limitation on $d$ is unavoidable: without it, offline convex optimization becomes easier while \ac{OCO} is equally hard for any number of dimensions.
\index{lower bound!bounded online linear optimization|)}

\index{lower bound!unconstrained online subgradient descent|(}
The lower bound for \ac{OSD} in Theorem~\ref{thm:lower_bound_osd} is a minor generalization of the one in \citet{OrabonaP15, OrabonaP18}. They also provide a similar lower bound for the \ac{EG} algorithm.
\index{lower bound!unconstrained online subgradient descent|)}

\index{lower bound!unconstrained online linear optimization|(}
Strangely enough, both the online learning and the optimization literature have almost ignored the issue of lower bounds for the unconstrained case. The connection between coin betting and \ac{OLO} was first unveiled in \citet{OrabonaP16}. Theorem~\ref{thm:epsilon_regret_olo_is_betting} is a mild generalization of \citet[Theorem 9]{McMahanJ13}. A similar, but more general, result was also rediscovered in \citet{Cutkosky18}.
The first lower bound for unconstrained \ac{OLO} is from \citet{StreeterM12}, but their proof relied on using the value of $\inf_T \Pr\{ \sum_{i=1}^T X_i \geq \sqrt{T}\}$ where $X_i$ are Rademacher random variables\index{random variable!Rademacher}.
\citet{StreeterM12} claim that this value is $7/64$ that corresponds to $T=6$, but they do not provide a proof for it. The same value was also conjectured by \citet{HitczenkoK94} and proved formally only in 2023 by \citet{HollomP23}. A proof avoiding that step completely was given by \citet{Orabona13}. Theorem~\ref{thm:olo_coin_betting_lower_bound} is new: it is a more precise version of the lower bound in \citet{Orabona13} and it has the asymptotically optimal constant $\sqrt{2}$.
One way to achieve the optimal constant in lower bounds using the tail of Binomial distributions was shown in \citet{OrabonaP15b}. Independently, \citet{ZhangCP22} proved a lower bound for unconstrained \ac{OCO} with the optimal constant but only when $\epsilon_T=\mathcal{O}(\sqrt{T})$, with an algorithm with a matching upper bound. The specific method used here to obtain the optimal constant is new. \citet{McMahanO14} proposed an algorithm that matches the lower bound up to a multiplicative constant, while \citet{ZhangCP22} matched the multiplicative constant too.
More recently, \citet{CarmonH24} proved a lower bound of $\Omega(\|\bx^\star-\bx_1\|L\sqrt{\frac{1}{T}\ln \frac{L\|\bx^\star-\bx_1\|}{\epsilon}})$ in the stochastic setting on the expected suboptimality gap. Their lower bound implies the same lower bound we stated on the regret of \ac{OCO} through the online-to-batch conversion. However, their lower bound requires the use of convex non-linear Lipschitz functions, while the result here just requires linear losses. Their results also imply that the parameter-free algorithm in \citet{McMahanO14} coupled with the online-to-batch conversion is optimal for the setting of stochastic unconstrained optimization of Lipschitz convex functions.
\index{lower bound!unconstrained online linear optimization|)}


\section{Exercises}

\begin{exer}
Fix $U>0$ and $\mathcal{V}=\R^d$. Mimicking the proof of Theorem~\ref{thm:lower_bound_constr}, prove that for any \ac{OCO} algorithm there exists a $\bu^\star$ and a sequence of loss functions such that $\Regret_T(\bu^\star)\geq\frac{1}{2}\|\bu^\star\|_2 L \sqrt{T}$ where $\|\bu^\star\|_2=U$ and the loss functions are $L$-Lipschitz with respect to $\|\cdot\|_2$.
\end{exer}

\begin{exer}
Extend the proof of Theorem~\ref{thm:lower_bound_constr} to an arbitrary norm $\|\cdot\|$ to measure the diameter of $\mathcal{V}$ and with $\|\bg_t\|_\star \leq L$.
\end{exer}



\acresetall

\chapter{Online Mirror Descent}
\label{ch:omd}
\index{Online Mirror Descent algorithm|(textbf}

In this chapter, we will introduce the \ac{OMD} algorithm. To explain its genesis, I think it is essential to understand what subgradients do. In particular, the negative subgradients are not always pointing towards a direction that minimizes the function. Then, we will introduce the Bregman divergences to generalize the notion of distance implicitly used in online subgradient descent. Finally, we will see extensions and applications of \ac{OMD}.

\acresetall

\section{Subgradients are not Informative}

We have seen that in online learning, we receive a sequence of loss functions, and we have to output a vector before observing the loss function on which we will be evaluated. However, we can gain a lot of intuition if we consider the easy case that the sequence of loss functions is always a fixed function, i.e., $\ell_t(\bx)=\ell(\bx)$. If our hypothetical online algorithm does not work in this situation, then it will not work on the more general case.

Hence, considering the case of fixed loss functions, let's take a look at the key step in the proof of the upper bound to the regret for \ac{OSD} in Lemma~\ref{lemma:sd_one_step}. We used the following property of the subgradients:
\begin{equation}
\label{eq:subgradient}
\ell(\bx_t)-\ell(\bu)\leq \langle \bg_t, \bx_t -\bu\rangle, \quad \forall \bu \in \mathcal{V}~.
\end{equation}
In words, to minimize the l.h.s. of this equation, it is enough to minimize the r.h.s., that is nothing else than the instantaneous linear regret on the linear function $\langle \bg_t, \cdot\rangle$. This is the only reason why \ac{OSD} works! However, I am sure you have heard a million times the (wrong) intuition that gradient points towards the minimum, and you might be tempted to think that the same (even more wrong) intuition holds for subgradients. Indeed, I am sure that even if we proved the regret guarantee based on \eqref{eq:subgradient}, in the back of your mind, you keep thinking ``yeah, sure, it works because the subgradient tells me where to go to minimize the function''. Typically, this idea is so strong that I have to present explicit counterexamples to fully convince a person.

So, take a look at the following examples that illustrate the fact that a subgradient does not always point in a direction where the function decreases.

\begin{figure}
\centering
\begin{tikzpicture}
  \begin{groupplot}[
    group style={
      group size=2 by 1,
      horizontal sep=2cm,
    },
    width=7cm,
    height=6cm,
    domain=-2.5:2.5,
    domain y=-2.5:2.5,
    samples=61,
    samples y=61,
    xlabel={$x_1$},
    ylabel={$x_2$},
    colormap={gray}{
      rgb(0cm)=(0,0,0);
      rgb(1cm)=(0.85,0.85,0.85)
    },
    xmin=-2.5, xmax=2.5,
    ymin=-2.5, ymax=2.5,
    xtick={-2,-1,0,1,2},
    ytick={-2,-1,0,1,2},
    zmin=0, zmax=5,
    ztick={0,1,2,3,4,5},
    major grid style={gray!30},
    tick label style={font=\scriptsize}, 
  ]

    \nextgroupplot[
      view={-32.7}{53.2},
      zlabel={$f(x_1,x_2)$},
      xmajorgrids,
      ymajorgrids,
      zmajorgrids,
      tick style={draw=none},
    ]

    \addplot3[
      contour gnuplot={
        levels={
          0.2,0.4,0.6,0.8,1.0,
          1.2,1.4,1.6,1.8,2.0,
          2.2,2.4,2.6,2.8,3.0,
          3.2,3.4,3.6,3.8,4.0,
          4.2,4.4,4.6,4.8,5.0
        },
        labels=false,
      },
      very thin,
    ]
    {max(max(-x,x-y),x+y)};

    \nextgroupplot[
      view={0}{90},
      axis equal image,
      hide z axis,
    ]

    \addplot3[
      contour gnuplot={
        levels={
          0.2,0.4,0.6,0.8,1.0,
          1.2,1.4,1.6,1.8,2.0,
          2.2,2.4,2.6,2.8,3.0,
          3.2,3.4,3.6,3.8,4.0,
          4.2,4.4,4.6,4.8,5.0
        },
        labels=false,
      },
      very thin,
    ]
    {max(max(-x,x-y),x+y)};

    \draw[->, line width=1.5pt]
      (axis cs:1,0,0) -- (axis cs:0,-1,0);

  \end{groupplot}
\end{tikzpicture}
\caption{3D plot (left) and level sets (right) of $f(\bx)=\max[-x_1,x_1-x_2,x_1+x_2]$. A negative subgradient is indicated by the black arrow.}
\label{fig:subgradient_not_descend1}
\commentAlt{Figure~\ref{fig:subgradient_not_descend1}. Two plots of f(x)=max(-x_1, x_1-x_2, x_1+x_2): a three-dimensional contour plot and a top-down level-set plot. A black arrow on the level-set plot shows a negative subgradient direction that is not a descent direction.}
\end{figure}


\begin{figure}[h]
\centering
\begin{tikzpicture}
  \begin{groupplot}[
    group style={
      group size=2 by 1,
      horizontal sep=2cm,
    },
    width=7cm,
    height=6cm,
    domain=-2.5:2.5,
    domain y=-2.5:2.5,
    samples=61,
    samples y=61,
    xlabel={$x_1$},
    ylabel={$x_2$},
    colormap={gray}{
      rgb(0cm)=(0,0,0);
      rgb(1cm)=(0.85,0.85,0.85)
    },
    xmin=-2.5, xmax=2.5,
    ymin=-2.5, ymax=2.5,
    xtick={-2,-1,0,1,2},
    ytick={-2,-1,0,1,2},
    zmin=0, zmax=20,
    ztick={0,5,10,15,20},
    major grid style={gray!30},
    tick label style={font=\scriptsize},
  ]

    \nextgroupplot[
      view={-49.5}{61.2},
      zlabel={$f(x_1,x_2)$},
      xmajorgrids,
      ymajorgrids,
      zmajorgrids,
      tick style={draw=none},
    ]

    \addplot3[
      contour gnuplot={
        levels={1,2,3,4,5,6,7,8,9,10,11,12,13,14,15,16,17,18,19,20},
        labels=false,
      },
      very thin,
    ]
    {max(x^2 + (y-1)^2, x^2 + (y+1)^2)};

    \nextgroupplot[
      view={0}{90},
      axis equal image,
      hide z axis,
      zmin=0, zmax=1, 
    ]

    \addplot3[
      contour gnuplot={
        levels={1,2,3,4,5,6,7,8,9,10,11,12,13,14,15,16,17,18,19,20},
        labels=false,
      },
      very thin,
    ]
    {max(x^2 + (y-1)^2, x^2 + (y+1)^2)};

    \draw[->, line width=1.5pt]
      (axis cs:1,0,0) -- (axis cs:0.5,-0.5,0);

  \end{groupplot}
\end{tikzpicture}
\caption{3D plot (left) and level sets (right) of $f(\bx)=\max[x_1^2+(x_2+1)^2,x_1^2+(x_2-1)^2]$. A negative subgradient is indicated by the black arrow.}
\label{fig:subgradient_not_descend2}
\commentAlt{Figure~\ref{fig:subgradient_not_descend2}. Two plots of f(x)=max(x_1^2+(x_2+1)^2, x_1^2+(x_2-1)^2): a three-dimensional contour plot and a top-down level-set plot. A black arrow marks a negative subgradient direction that does not decrease the function.}
\end{figure}


\begin{example}
\label{ex:subgradient_not_descend}
Let $f(\bx)=\max[-x_1,x_1-x_2,x_1+x_2]$, see Figure~\ref{fig:subgradient_not_descend1}.
The vector $\bg=(1,1)$ is a subgradient in $\bx=(1,0)$ of $f(\bx)$. No matter how we choose the stepsize, moving in the negative direction of this subgradient will not decrease the objective function.
An even more extreme example is in Figure~\ref{fig:subgradient_not_descend2}, with the function $f(\bx)=\max[x_1^2+(x_2+1)^2,x_1^2+(x_2-1)^2]$.
Here, in the point $(1,0)$, any positive step in the direction of the displayed negative subgradient will \emph{increase} the objective function.
\end{example}

In both examples, one might think that it would be enough to add a little bit of noise to move away from the ``corners''. However, we have to remember that we are in the adversarial setting. So, assuming the adversary sees our randomization, it can always present us a function such that our prediction is exactly on a ``corner''. This means that our analysis of subgradient descent will not have to use the fact that the subgradients point towards descending directions, because it can be false in every single iteration.

\begin{remark}
Given the above, one might wonder how much ``information'' the subgradients carry. It turns out, quite a lot! In fact, for bounded domains in the offline case, i.e., when all the functions are the same, the cutting plane method\index{cutting plane method} can optimize the function exponentially fast just using subgradients~\citep[see, e.g., Chapter 3 in][]{Nemirovski95}. However, in the adversarial setting, subgradients become much weaker than in the offline setting, exactly because the adversary has the power to change the function in each round.
\end{remark}

\section{Reinterpreting the Online Subgradient Descent Algorithm}
How does \ac{OSD} work? It works exactly as I told you before: thanks to \eqref{eq:subgradient}. But what does that inequality really mean?

A way to understand how the \ac{OSD} algorithm works is to think that it minimizes a local approximation of the original objective function. This is not unusual for optimization algorithms. For example, the Newton algorithm constructs an approximation with a Taylor expansion truncated to the second term.
Thanks to the definition of subgradients, we can immediately build a linear lower bound to a function $f$ around $\bx_0$:
\[
f(\bx) \geq \tilde{f}(\bx) := f(\bx_0) + \langle \bg, \bx-\bx_0\rangle, \quad \forall \bx \in \mathcal{V}~.
\]
So, in our setting, this would mean that we update the online algorithm with the minimizer of a linear approximation of the loss function you received. Unfortunately, minimizing a linear function is unlikely to give us a good online algorithm. Indeed, over unbounded domains, the infimum of a linear function is $-\infty$.

So, let's introduce the other key concept: we constrain the minimization of this lower bound only in a neighborhood of $\bx_0$, where we have good reason to believe that the approximation is more precise. Moreover, in online learning, it makes sense not to go too far from the previous iteration because the losses are different in each step, and we do not want to give too much importance to the current loss. Coding the neighborhood constraint with an L$_2$ squared distance from $\bx_0$ less than some positive number $h$, we might think to use the following update
\begin{align*}
\bx_{t+1} = \argmin_{\bx \in \mathcal{V}} & \ f(\bx_t) + \langle \bg, \bx-\bx_t\rangle \\
\text{s.t.} & \ \|\bx_t-\bx\|^2 \leq h~.
\end{align*}
Equivalently, for some $\eta>0$, we can consider the unconstrained formulation
\begin{equation}
\label{eq:quad_approx}
\argmin_{\bx \in \mathcal{V}} \ \left(\hat{f}(\bx) := f(\bx_0) + \langle \bg, \bx-\bx_0\rangle + \frac{1}{2\eta}\|\bx_{0}-\bx\|_2^2\right)~.
\end{equation}
This is a well-defined update scheme that hopefully moves $\bx_t$ closer to the optimum of $f$.
See Figure~\ref{fig:local_approx} for a graphical representation in one dimension.

\begin{figure}[h]
\centering
\begin{tikzpicture}
\begin{axis}[
          xmax=14,ymax=15,
          axis lines=middle,
          restrict y to domain=0:20,
          ]
\addplot[thick,black,domain=1/2^8:14,samples=200] {1-ln(x)+x*x/10};
\addplot[thick,dotted, domain=1/2^6:14,samples=200]  {1-ln(8)+8*8/10+(x-8)*(-1/8+2*8/10)} node at (axis cs:13,11) {$\tilde{f}(\bx)$};
\addplot[dashdotted,domain=1/2^6:14,samples=200]  {1-ln(8)+8*8/10+(x-8)*(-1/8+2*8/10)+1/3*(x-8)^2} node at (axis cs:1.5,14) {$\hat{f}(\bx)$};
\node at (axis cs:8.9,4.9) {$f(\bx_0)$};
\node at (axis cs:8.3,0.5) {$\bx_0$};
\end{axis}
\end{tikzpicture}
\caption{Approximations of $f(\bx)$.}
\label{fig:local_approx}
\commentAlt{Figure~\ref{fig:local_approx}. Plot of a curved function f and two local approximations at x_0: a dotted linear approximation and a dash-dotted quadratic approximation.}
\end{figure}

And now the final element of our story: the argmin in \eqref{eq:quad_approx} is exactly the update we used in \ac{OSD}!
Indeed, solving the argmin and completing the square, we get
\begin{align}
&\argmin_{\bx \in \mathcal{V}} \ \langle \bg_t, \bx\rangle + \frac{1}{2\eta_t}\|\bx_{t}-\bx\|_2^2 \label{eq:osd_as_omd}\\
&\quad = \argmin_{\bx \in \mathcal{V}} \ \|\eta_t \bg_t\|^2 + 2\eta_t \langle \bg_t, \bx-\bx_t\rangle + \|\bx_{t}-\bx\|_2^2 \nonumber \\
&\quad = \argmin_{\bx \in \mathcal{V}} \ \|\bx - \bx_t +\eta_t \bg_t\|_2^2
= \Pi_{\mathcal{V}}(\bx_t - \eta_t \bg_t), \nonumber
\end{align}
where we used the fact that the argmin is independent of additive constants and positive rescalings, and $\Pi_{\mathcal{V}}$ is the Euclidean projection onto $\mathcal{V}$, i.e., $\Pi_{\mathcal{V}}(\bx)=\argmin_{\by \in \mathcal{V}} \ \|\bx-\by\|_2$.

The new way to write the update of \ac{OSD} in \eqref{eq:quad_approx} will be the core ingredient for designing \emph{\ac{OMD}}.
In fact, \ac{OMD} is a strict generalization of that update when we use a different way to measure the locality of $\bx$ from $\bx_t$.
That is, we measured the distance to the current point with the squared L$_2$ norm. What happens if we change the norm? Do we even have to use a norm?

To answer these questions, we have to introduce another useful mathematical object: the \emph{Bregman divergence}.

\section{Convex Analysis Bits: Bregman Divergence}

We first give a new definition, a slightly stronger notion of convexity.
\begin{definition}
Let $f:\mathcal{V} \subseteq \R^d \to \R$ and $\mathcal{V}$ be a convex set. $f$ is \textbf{strictly convex}\index{function!strictly convex|textbf} if
\[
f(\alpha \bx + (1-\alpha) \by) < \alpha f(\bx) + (1-\alpha) f(\by), \quad \forall \bx, \by \in \mathcal{V}, \bx\neq \by, 0<\alpha<1~.
\]
\end{definition}
From the definition, it is immediate to see that strong convexity with respect to any norm implies strict convexity.
Note that for a differentiable function, strict convexity also implies that $f(\by) > f(\bx) + \langle \nabla f(\bx), \by-\bx\rangle$ for $\bx\neq\by$~\citep[Proposition 17.10]{BauschkeC17}.

We now define our new notion of ``distance''.
\index{Bregman divergence|(textbf}
\begin{definition}
Let $\psi: \mathcal{X} \to \R$ be strictly convex\index{function!strictly convex} and differentiable on $\interior \mathcal{X} \neq\emptyset$, where $\mathcal{X}$ is the domain of $\psi$. The \textbf{Bregman divergence} with respect to $\psi$ is denoted by $B_\psi : \mathcal{X} \times \interior \mathcal{X} \to \R$ defined as
\[
B_\psi(\bx; \by)
:= \psi(\bx) - \psi(\by) - \langle \nabla \psi(\by), \bx - \by\rangle~.
\]
\end{definition}
We will call the function $\psi$ associated to $B_\psi$ the \textbf{distance generating function}\index{distance generating function|textbf}.

From the definition, we see that the Bregman divergence is always non-negative for $\bx,\by\in \interior \mathcal{X}$, from the convexity of $\psi$. However, something stronger holds.
By the strict convexity of $\psi$, for a fixed point $\by \in \interior \mathcal{X}$ we have that $\psi(\bx) \geq \psi(\by) +\langle \nabla \psi(\by), \bx-\by\rangle, \ \forall \bx \in \mathcal{X}$, with equality only for $\by=\bx$. Hence, the strict convexity allows us to use the Bregman divergence as a similarity measure between $\bx$ and $\by$. Moreover, this similarity measure \emph{changes} with the reference point $\by$. This also implies that, as you can see from the definition, the Bregman divergence is not symmetric.

Let me give you some more intuition on the concept of the Bregman divergence. Consider the case that $\psi$ is twice differentiable in an open ball $\mathcal{B}$ around $\by$ and $\bx \in \mathcal{B}$. So, by Taylor's theorem, there exists $0\leq\alpha\leq1$ such that
\[
B_\psi(\bx; \by) = \psi(\bx) - \psi(\by) - \nabla \psi(\by)^\top (\bx - \by) = \frac{1}{2} (\bx-\by)^\top \nabla^2 \psi(\bz) (\bx-\by),
\]
where $\bz=\alpha \bx+(1-\alpha)\by$. Hence, we are using a \emph{squared local norm}\index{norm!local} that depends on the Hessian of $\psi$. \emph{Different areas of the space will have a different value of the Hessian, and so the Bregman will behave differently}. We will use this exact idea in the local norm analyses of \ac{OMD} (Section~\ref{sec:omd_local_norms}) and \ac{FTRL} (Section~\ref{sec:ftrl_local_norms}).

We can also lower bound the Bregman divergence if the function $\psi$ is strongly convex.
In particular, if $\psi$ is $\lambda$-strongly convex\index{function!strongly convex} with respect to a norm $\|\cdot\|$ in $\interior \mathcal{X}$, then we have
\begin{equation}
\label{eq:bregman_strongly_convex}
B_\psi(\bx;\by)\geq \frac{\lambda}{2}\|\bx-\by\|^2, \quad \forall \bx, \by \in \interior \mathcal{X}~.
\end{equation}

\begin{example}
If $\psi(\bx)=\frac{1}{2}\|\bx\|^2_2$, then $B_\psi(\bx;\by)=\frac{1}{2}\|\bx\|^2_2-\frac{1}{2}\|\by\|^2_2-\langle \by,\bx-\by\rangle=\frac{1}{2}\|\bx-\by\|^2_2$.
\end{example}

\begin{example}
\label{example:kl}
Let $\mathcal{X} =\R^d_{\geq0}$ and $\psi(\bx)=\sum_{i=1}^d x_i \ln x_i$, the negative Shannon entropy\index{entropy!negative Shannon}. Then, for all $\bx \in \mathcal{X}$ and $\by \in \interior \mathcal{X}$ we have
\[
B_\psi(\bx;\by)
= \sum_{i=1}^d (x_i \ln x_i - y_i \ln y_i - (\ln(y_i)+1)(x_i-y_i))
= \sum_{i=1}^d \left(x_i \ln \frac{x_i}{y_i} - x_i+y_i\right),
\]
where we define $0 \ln 0:=0$.
This is called the \textbf{generalized Kullback--Leibler divergence}\index{Kullback--Leibler divergence!generalized|textbf}, where ``generalized'' is because $\bx$ and $\by$ do not have to be discrete probability distributions.
\end{example}

We also have the following immediate lemma that links the Bregman divergences between 3 points.
\begin{lemma}[{\citealp{ChenT93}}]
\label{lemma:bregman_3_points}
Let $B_\psi$ be the Bregman divergence with respect to $\psi: \mathcal{X} \to \R$. Then, for any three points $\bx,\by \in \interior \mathcal{X}$ and $\bz \in \mathcal{X}$, the following identity holds
\[
B_\psi(\bz; \bx) + B_\psi(\bx; \by) - B_\psi(\bz; \by) = \langle \nabla \psi(\by) - \nabla \psi(\bx), \bz - \bx\rangle~.
\]
\end{lemma}
\index{Bregman divergence|)textbf}


\section{Online Mirror Descent}

Based on what we said before, we can start from the equivalent formulation of the \ac{OSD} update,
\[
\bx_{t+1} = \argmin_{\bx \in \mathcal{V}} \ \langle \bg_t, \bx\rangle + \frac{1}{2\eta_t}\|\bx_{t}-\bx\|_2^2,
\]
and we can change the last term with another measure of distance. In particular, using the Bregman divergence\index{Bregman divergence} with respect to a function $\psi$, we have
\[
\bx_{t+1} = \argmin_{\bx \in \mathcal{V}} \ \langle \bg_t, \bx\rangle + \frac{1}{\eta_t}B_\psi(\bx;\bx_{t}),
\]
where we assumed that the argmin exists and it is unique.
These two updates are exactly the same when $\psi(\bx)=\frac{1}{2}\|\bx\|_2^2$.

So, we get \textbf{\acl{OMD}} in Algorithm~\ref{alg:omd}.

\begin{algorithm}[t]
\caption{Online Mirror Descent (OMD)}
\label{alg:omd}
\begin{algorithmic}[1]
{
    \REQUIRE{Non-empty closed convex $\mathcal{V} \subseteq \mathcal{X}\subseteq \R^d$, $\psi: \mathcal{X} \to \R$ strictly convex and differentiable on $\interior \mathcal{X}$, $\bx_1 \in \interior \mathcal{X} \cap \mathcal{V}$, $\eta_1,\dots,\eta_T>0$}
    \FOR{$t=1$ {\bfseries to} $T$}
    \STATE{Output $\bx_t \in \mathcal{V}$}
    \STATE{Pay the loss $\ell_t(\bx_t)$, where $\ell_t$ is subdifferentiable on $\mathcal{V}$}
    \STATE{Set $\bg_t \in \partial \ell_t(\bx_t)$}
    \STATE{Set $\bx_{t+1} \in \argmin_{\bx \in \mathcal{V}} \ \langle \bg_t, \bx\rangle + \frac{1}{\eta_t}B_\psi(\bx; \bx_{t})$}
    \ENDFOR
}
\end{algorithmic}
\end{algorithm}

However, without an additional assumption, this algorithm has a problem. Can you see it?
The problem is that $\bx_{t+1}$ might be on the boundary of $\mathcal{V}$, and in the next step, we would have to evaluate $B_\psi(\bx; \bx_{t+1})$ for a point on the boundary of $\mathcal{V}$. Given that $\mathcal{V} \subseteq \mathcal{X}$, we might end up on the boundary of $\mathcal{X}$ where the Bregman divergence\index{Bregman divergence} is not defined!

To fix this problem, different sufficient conditions can be used. In the following, we will use one of the following assumptions:
\begin{align}
&\lim_{\lambda \to 0} \ \langle \nabla \psi((1-\lambda)\bx+\lambda \by), \by-\bx\rangle = -\infty, \quad \forall \bx \in \bdry \mathcal{X}, \forall \by \in \interior \mathcal{X}  \label{eq:cond_omd1}\\
&\mathcal{V} \subseteq \interior \mathcal{X}~. \label{eq:cond_omd2}
\end{align}
If either of these conditions is true and the argmin exists on each round, then the algorithm is well-defined as proved in the following theorem.
\begin{theorem}
\label{thm:omd_well_defined}
Let $B_\psi$ be the Bregman divergence\index{Bregman divergence} with respect to $\psi: \mathcal{X} \to \R$. Let $\mathcal{V} \subseteq \mathcal{X}$ be a non-empty closed convex set. Assume \eqref{eq:cond_omd1} or \eqref{eq:cond_omd2} hold and, with the notation and initialization in Algorithm~\ref{alg:omd}, the argmin exists on all rounds. Then, $\bx_{t+1} \in \interior \mathcal{X}$.
\end{theorem}
\begin{proof}
In the case that \eqref{eq:cond_omd2} holds, we have that $\bx_{t+1} \in \mathcal{V}$ implies immediately that $\bx_{t+1} \in \interior \mathcal{X}$ .

Let's now assume that \eqref{eq:cond_omd1} holds and let's prove it by induction. The base case is true by the definition of $\bx_1$. Let's now assume that $\bx_t \in \interior \mathcal{X}$ and let's prove that $\bx_{t+1} \in \interior \mathcal{X}$. We will prove it by contradiction. So, assume that $\bx_{t+1} \in \bdry \mathcal{X}$.
Set $\bz \in \interior \mathcal{X} \cap \mathcal{V}$ and define $\phi(\lambda)=\langle \eta_t \bg_t, (1-\lambda)\bx_{t+1}+\lambda \bz\rangle + B_\psi((1-\lambda)\bx_{t+1}+\lambda \bz; \bx_{t})$ for $\lambda \in (0,1)$. From \eqref{eq:cond_omd1}, we have that
\begin{align*}
\lim_{\lambda \to 0} \phi'(\lambda)
&= \lim_{\lambda \to 0}  (\langle \eta_t \bg_t, \bz -\bx_{t+1}\rangle + \langle \nabla \psi(\bx_{t+1}+\lambda (\bz -\bx_{t+1}))-\nabla \psi(\bx_t),\bz-\bx_{t+1}\rangle)\\
&= -\infty.
\end{align*}
Hence, there exists $\epsilon>0$ such that
\[
\langle \eta_t \bg_t, \bx_\epsilon\rangle + B_\psi(\bx_\epsilon; \bx_{t})
= \phi(\epsilon)
< \phi(0)
= \langle \eta_t \bg_t, \bx_{t+1}\rangle + B_\psi(\bx_{t+1}; \bx_{t}),
\]
where $\bx_\epsilon:=(1-\epsilon)\bx_{t+1}+\epsilon \bz \in \interior \mathcal{X} \cap \mathcal{V}$. However, this contradicts the definition of $\bx_{t+1}$ as an argmin, proving that $\bx_{t+1}$ must be in $\interior \mathcal{X}$.
\end{proof}
When \eqref{eq:cond_omd1} holds, this theorem implies that the predictions of the algorithm always stay in the interior of $\mathcal{X}$ without the need for any projection. If in addition $\mathcal{V}=\mathcal{X}$, the update of the algorithm is the solution of an unconstrained problem because the feasible set is implicit in the Bregman divergence\index{Bregman divergence}.

Now we have a well-defined algorithm, but does it guarantee a sublinear regret? We know that at least in one case, it recovers the \ac{OSD} algorithm, which does work. So, from an intuitive point of view, how well the algorithm works should depend on some characteristic of $\psi$. In particular, a sufficient property will be the \emph{strong convexity} of $\psi$. The strong convexity also takes care of the existence of the argmin in the algorithm, thanks to the next theorem.
\begin{theorem}
\label{thm:min_strongly_convex}
Let $\lambda>0$ and $f:\R^d\to (-\infty, +\infty]$ proper, closed, and $\lambda$-strongly convex\index{function!strongly convex} with respect to $\|\cdot\|$ on its domain. Assume $\dom \partial f \neq \emptyset$. Then, $f$ has exactly one minimizer.
\end{theorem}
\begin{proof}
Let $\by \in  \dom \partial f$ and $\bg \in \partial f(\by)$.
From Lemma~\ref{lemma:strong_convexity}, for any $\bx \in \R^d$, we have
\begin{align*}
f(\bx)
&\geq f(\by) + \langle \bg, \bx-\by\rangle + \frac{\lambda}{2}\|\bx-\by\|^2 \\
&\geq f(\by) - \|\bg\|_\star \|\bx\|-\langle \bg, \by\rangle + \frac{\lambda}{2}(\|\bx\|-\|\by\|)^2 \\
&= f(\by) - \|\bg\|_\star \|\bx\|-\langle \bg, \by\rangle + \frac{\lambda}{2}(\|\bx\|^2+\|\by\|^2-2 \|\bx\| \|\by\|),
\end{align*}
where in the second inequality we used the reverse triangle inequality and the definition of dual norms.
From the above, we have that $\lim_{\|\bx\|\to \infty} f(\bx) = +\infty$. In turn, this implies that the level sets of $f$ are bounded. From the assumption that $f$ is closed, we get that the level sets are compact. Hence, for any $\by$ in $\dom f$, the minimum of $f$ is the same as the minimum of $f$ over the set $\{\bx : f(\bx)\leq f(\by)\}$, that is, the minimum over a compact set, that exists by the Weierstrass theorem, Theorem~\ref{thm:weierstrass}\index{Weierstrass theorem for extended-real-valued functions}. The uniqueness is given by the fact that strongly convex functions are strictly convex\index{function!strictly convex}.
\end{proof}

\index{Online Mirror Descent algorithm!one-step lemma|(}
To analyze \ac{OMD}, we first prove a one-step relationship, similar to the ones we proved for \ac{OGD} and \ac{OSD}.
Note how in this lemma, we will use a lot of the concepts we introduced till now: strong convexity, dual norms, subgradients, etc. In a way, over the past sections, I slowly prepared you to be able to prove this lemma.
\begin{lemma}
\label{lemma:omd_one_step}
Let $B_\psi$ be the Bregman divergence\index{Bregman divergence} with respect to $\psi: \mathcal{X} \to \R$ and assume $\psi$ to be proper, closed, and $\lambda$-strongly convex with respect to $\|\cdot\|$ in $\mathcal{V}$, where $\mathcal{V} \subseteq \mathcal{X}$ is a non-empty closed convex set. Assume \eqref{eq:cond_omd1} or \eqref{eq:cond_omd2} hold. Then, with the notation and initialization in Algorithm~\ref{alg:omd}, for all $t$ we have that $\bx_{t+1}$ exists, it is unique, and it is in the interior of $\mathcal{X}$. Moreover, $\forall \bu \in \mathcal{V}$, the following inequality holds
\begin{align*}
\eta_t (\ell_t(\bx_t) - \ell_t(\bu) )
&\leq \eta_t \langle \bg_t, \bx_t -\bu\rangle\\
&\leq B_\psi(\bu;\bx_t) - B_\psi(\bu;\bx_{t+1}) - B_\psi(\bx_{t+1};\bx_t) + \langle \eta_t \bg_t, \bx_t - \bx_{t+1} \rangle \\
&\leq B_\psi(\bu;\bx_t) - B_\psi(\bu;\bx_{t+1}) + \frac{\eta_t^2}{2\lambda}\|\bg_t\|_\star^2~.
\end{align*}
\end{lemma}
\begin{proof}
First of all, in each round $\bx_{t+1}$ exists using Theorem~\ref{thm:min_strongly_convex} and the fact that $B_{\psi}(\cdot; \bx_t)$ is proper, closed, and strongly convex. Moreover, from Theorem~\ref{thm:omd_well_defined}, $\bx_t \in \interior \mathcal{X}$ for all $t$.

Now, from the optimality condition of Theorem~\ref{thm:constr_opt_condition} for the update of \ac{OMD}, we have
\begin{equation}
\label{eq:proof_regret_md_eq1}
\langle \eta_t \bg_t + \nabla \psi (\bx_{t+1})-\nabla \psi(\bx_t), \bu - \bx_{t+1}\rangle \geq 0, \quad \forall \bu \in \mathcal{V}~.
\end{equation}
Hence, we have that
\begin{align*}
\langle \eta_t \bg_t, \bx_t - \bu\rangle
&\leq \langle \nabla \psi(\bx_{t+1}) - \nabla \psi(\bx_t), \bu - \bx_{t+1} \rangle + \langle \eta_t \bg_t, \bx_t - \bx_{t+1} \rangle \\
&= B_\psi(\bu;\bx_t) - B_\psi(\bu;\bx_{t+1}) - B_\psi(\bx_{t+1};\bx_t) + \langle \eta_t \bg_t, \bx_t - \bx_{t+1} \rangle \\
&\leq B_\psi(\bu;\bx_t) - B_\psi(\bu;\bx_{t+1}) - \frac{\lambda}{2}\|\bx_t - \bx_{t+1} \|^2  + \eta_t\|\bg_t\|_\star \|\bx_t - \bx_{t+1} \|\\
&\leq B_\psi(\bu;\bx_t) - B_\psi(\bu;\bx_{t+1}) + \frac{\eta_t^2}{2\lambda}\|\bg_t\|_\star^2,
\end{align*}
where in the first inequality we used \eqref{eq:proof_regret_md_eq1}, in the equality we used Lemma~\ref{lemma:bregman_3_points}, in the second inequality we used the definition of dual norm, \eqref{eq:bregman_strongly_convex} (because $\psi$ is $\lambda$-strong convex with respect to $\|\cdot\|$), and Cauchy--Schwarz inequality, finally in the last inequality we used the fact that $a x - \frac{b}{2} x^2 \leq \frac{a^2}{2b}$ for all $x \in \R$ and $a,b>0$.

The lower bound with the function values is due to the definition of subgradients.
\end{proof}
\index{Online Mirror Descent algorithm!one-step lemma|)}

We now see how to use this one-step relationship to prove a regret bound that will finally show us if and when this entire construction is a good idea.

We can now prove a regret bound for \ac{OMD}.
\begin{theorem}
\label{thm:md_online}
Assume $0 < \eta_{t+1}\leq \eta_{t}, \ t=1, \dots, T-1$. Then, under the assumptions of Lemma~\ref{lemma:omd_one_step} and for all $\bu \in \mathcal{V}$, the following regret bound holds
\[
\sum_{t=1}^T (\ell_t(\bx_t)- \ell_t(\bu))
\leq \max_{1\leq t \leq T} \frac{B_\psi(\bu;\bx_t)}{\eta_{T}} + \frac{1}{2\lambda}\sum_{t=1}^T \eta_t \|\bg_t\|^2_\star~.
\]
Moreover, if $\eta_t$ is constant, i.e., $\eta_t=\eta \ \forall t=1,\dots,T$, we also have
\[
\sum_{t=1}^T (\ell_t(\bx_t)- \ell_t(\bu))
\leq \frac{B_\psi(\bu;\bx_1)}{\eta} + \frac{\eta}{2\lambda}\sum_{t=1}^T \|\bg_t\|^2_\star~.
\]
\end{theorem}
\begin{proof}
Fix $\bu \in \mathcal{V}$.
As in the proof of \ac{OGD}, dividing the inequality in Lemma~\ref{lemma:omd_one_step} by $\eta_t$ and summing from $t=1,\dots,T$, we get
\begin{align*}
\sum_{t=1}^T &(\ell_t(\bx_t) - \ell_t(\bu))
\leq \sum_{t=1}^T \left(\frac{B_\psi(\bu;\bx_t)}{\eta_t} - \frac{B_\psi(\bu;\bx_{t+1})}{\eta_t}\right) + \sum_{t=1}^T \frac{\eta_t \|\bg_t\|^2_\star}{2\lambda}  \\
&= \frac{B_\psi(\bu;\bx_{1})}{\eta_1} - \frac{B_\psi(\bu;\bx_{T+1})}{\eta_T}   + \sum_{t=1}^{T-1} \left(\frac{1}{\eta_{t+1}}-\frac{1}{\eta_t}\right)B_\psi(\bu;\bx_{t+1})  + \sum_{t=1}^T \frac{\eta_t \|\bg_t\|^2_\star}{2\lambda}  \\
&\leq \frac{D^2}{\eta_1} + D^2 \sum_{t=1}^{T-1} \left(\frac{1}{\eta_{t+1}}-\frac{1}{\eta_{t}}\right) + \sum_{t=1}^T \frac{\eta_t \|\bg_t\|^2_\star}{2\lambda}  \\
&= \frac{D^2}{\eta_1}  + D^2 \left(\frac{1}{\eta_{T}}-\frac{1}{\eta_1}\right) + \sum_{t=1}^T \frac{\eta_t \|\bg_t\|^2_\star}{2\lambda}
= \frac{D^2}{\eta_{T}} + \sum_{t=1}^T \frac{\eta_t \|\bg_t\|^2_\star}{2\lambda},
\end{align*}
where we denoted by $D^2=\max_{1\leq t\leq T} B_\psi(\bu;\bx_t)$.

The second statement is left as an exercise.
\end{proof}

In words, \ac{OMD} allows us to prove regret guarantees that depend on an arbitrary pairs of dual norms $\|\cdot\|$ and $\|\cdot\|_\star$. In particular, the primal norm will be used to measure the feasible set $\mathcal{V}$ or the distance between the competitor and the initial point, and the dual norm will be used to measure the gradients. If you happen to know something about these quantities, we can choose the most appropriate pair of norms to guarantee a small regret. The only thing you need is a function $\psi$ that is strongly convex with respect to the primal norm you have chosen.


Overall, the regret bound is still of the order of $\sqrt{T}$ for Lipschitz functions, and the only difference is that now the Lipschitz constant is measured with respect to a different norm.
Also, everything we did for \ac{OSD} can be trivially used here. For example, we can slightly generalize the stepsizes we saw in Section~\ref{sec:lstar}. Assuming $D^2=\max_{\bx, \by\in \mathcal{V}} \ B_\psi(\bx;\by)<\infty$, set
\[
\eta_t=\frac{D \sqrt{\lambda}}{\sqrt{\sum_{i=1}^t \|\bg_i\|_\star^2}}
\]
to achieve a regret upper bound of $2\frac{D}{\sqrt{\lambda}} \sqrt{\sum_{t=1}^T \|\bg_t\|^2_\star}$.

%

In Sections~\ref{sec:omd_eg} and \ref{sec:omd_pnorm}, we will see practical examples of \ac{OMD} that guarantee strictly better regret than \ac{OSD}. As we did in the case of AdaGrad, the better guarantee will depend on the shape of the domain and the characteristics of the subgradients.

Next, we see the meaning of the ``Mirror'', but first we need another mathematical tool: \emph{Fenchel conjugates}.

\subsection{Convex Analysis Bits: Fenchel Conjugate}
\label{sec:fenchel}

\begin{definition}
\index{Fenchel conjugate|(textbf}
For a function $f: \R^d\to [-\infty,\infty]$, we define the \textbf{Fenchel conjugate} $f^\star:\R^d \to [-\infty,\infty]$ as
\[
f^\star(\btheta) = \sup_{\bx \in \R^d} \ \langle \btheta, \bx\rangle - f(\bx)~.
\]
\end{definition}
From the definition we immediately obtain the \textbf{Fenchel--Young's inequality}\index{inequality!Fenchel--Young's|textbf} for proper functions:
\[
\langle \btheta, \bx \rangle \leq f(\bx)+f^\star(\btheta), \quad  \forall \bx, \btheta \in \R^d~.
\]
Moreover, $f^\star$ is always convex and closed, regardless of the convexity of $f$~\citep[Proposition 13.13]{BauschkeC17}.

We have the following useful properties for the Fenchel conjugate.
\begin{theorem}[{\citealp[Theorem 12.2]{Rockafellar70}}]
\label{thm:dual_of_dual}
Let $f$ be a convex function. Then, $f^\star$ is a closed convex function, proper\index{function!proper} iff $f$ is proper. Moreover, if $f$ is also closed then $f^{\star\star}=f$.
\end{theorem}

\begin{theorem}
\label{thm:props_fenchel}
Let $f:\R^d \to (-\infty,+\infty]$ be proper. Then, the following conditions are equivalent:
\begin{enumerate}[(a)]
\item $\btheta \in \partial f(\bx)$.
\item $\langle \btheta, \by\rangle - f(\by)$ achieves its supremum in $\by$ at $\by=\bx$.
\item $f(\bx)+f^\star(\btheta)=\langle \btheta,\bx\rangle$.
\end{enumerate}
Moreover, if $f$ is also convex and closed, we have an additional equivalent condition
\begin{enumerate}[(a)]
\setcounter{enumi}{3}
\item $\bx \in \partial f^\star(\btheta)$.
\end{enumerate}
\end{theorem}
\begin{proof}
Let's prove (a)$\Leftrightarrow$(b).
From the definition of subgradient, we immediately have
\[
\langle \btheta, \bx\rangle - f(\bx)
\geq \langle \btheta, \by\rangle - f(\by), \quad \forall \by~.
\]
Then, (b)$\Leftrightarrow$(c) by definition of $f^\star(\btheta)$.

If $f$ is also convex and closed, then $f^{\star\star}=f$ is proper by Theorem~\ref{thm:dual_of_dual}. Hence, (c) is equivalent to $f^{\star \star}(\bx)+ f^\star(\btheta)=\langle \btheta, \bx\rangle$, which is equivalent to (d) by following the same reasoning as above.
\end{proof}


\begin{figure}[t]
\centering
\begin{tikzpicture}
\clip (-3,-1) rectangle (3,3.6);

\draw[thick, domain=-2.5:2.5] plot (\x, {\x*\x/2});

\draw[gray, dashed] (-3,{1*(-3)-(1*1)/2}) -- (3,{1*(3)-(1*1)/2});

\draw[->] (-3,0) -- (3,0) node[right] {$x$};
\draw[->] (0,-1) -- (0,3.6) node[above] {$y$};

\filldraw[black] (0,-1/2) circle (2pt);

\node[left] at (0,-1/2) {$-f^\star(\bg_{\bx})$};

\draw[black, dotted] (1,0.5) -- (1,0);

\node[below] at (1,0) {$\bx$};

\draw[black, dotted] (1,0.5) -- (0,0.5);

\node[left] at (0,0.5) {$f(\bx)$};
\end{tikzpicture}
\hspace{.5cm}
\begin{tikzpicture}

\clip (-3,-1) rectangle (3,3.6);

\draw[thick, domain=-2.5:2.5] plot (\x, {\x*\x/2});

\foreach \x in {-2.5,-2.25,...,2.5} {
    \draw[gray, dashed] (-3,{\x*(-3)-(\x*\x)/2}) -- (3,{\x*(3)-(\x*\x)/2});
}

\draw[->] (-3,0) -- (3,0) node[right] {$x$};
\draw[->] (0,-1) -- (0,3.6) node[above] {$y$};
\end{tikzpicture}
\caption{A geometric interpretation of the Fenchel conjugate.}
\label{fig:fenchel}
\commentAlt{Figure~\ref{fig:fenchel}. Geometric view of the Fenchel conjugate for a parabola. The left panel shows one tangent line touching the parabola at x, with vertical and horizontal guides marking f(x) and -f^*(g_x). The right panel shows many tangent lines to the same parabola.}
\end{figure}

In the following, we will not need a geometric intuition of the concept of Fenchel conjugates in order to use them. However, for some people, geometric intuitions help them remember better, so let's briefly discuss it.
Let $f$ be convex, closed, and proper. Let $\bx \in \dom \partial f$ and use a subgradient $\bg_{\bx} \in \partial f(\bx)$ to construct a linear lower bound around $\bx$ passing through $f(\bx)$ as
\[
\tilde{f}(\by):=f(\bx)+\langle \bg_{\bx}, \by-\bx\rangle~.
\]
Moreover, from Theorem~\ref{thm:props_fenchel}(c), we have
\[
\tilde{f}(\boldsymbol{0})
= f(\bx) - \langle \bg_{\bx},\bx\rangle
= - f^\star(\bg_{\bx})~.
\]
Hence, we have that $-f^\star(\bg_{\bx})$ is the value of $\tilde{f}$ at $\by=\boldsymbol{0}$, see Figure~\ref{fig:fenchel} (left).
This is sometimes mentioned as a geometric intuition for Fenchel conjugates, but I do not find it very illuminating, so let's dig deeper.

We can show that we can recover $f$ as the point-wise maximum of all its tangents, see Figure~\ref{fig:fenchel} (right). In turn, the tangents can be expressed using only the knowledge of $f^\star$.
Let's consider the family of linear functions $\tilde{f}_{\bg}(\by)=-f^\star(\bg)+\langle \bg, \by\rangle$ parametrized by a generic $\bg$. By Fenchel--Young's inequality\index{inequality!Fenchel--Young's}, we immediately have $\tilde{f}_{\bg}(\by)\leq f(\by)$. Moreover, from Theorem~\ref{thm:props_fenchel}(c), if $\bg_{\by} \in \partial f(\by)$, then $\tilde{f}_{\bg_{\by}}(\by)=f(\by)$. This means that this family of functions lower bounds $f$ and it contains the tangents to $f$.
Now, observe that for any $\by$ we obtain
\[
\sup_{\bg \in \R^d} \ \tilde{f}_{\bg}(\by)
= \sup_{\bg \in \R^d} \ -f^\star(\bg)+\langle \bg, \by\rangle
= f^{\star \star}(\by)
= f(\by),
\]
where in last equality we used Theorem~\ref{thm:dual_of_dual}.
So, the supremum over this family of linear functions is equal to the function $f$.
Overall, this means that the Fenchel conjugate allows us to reason about functions using their tangent hyperplanes, whose information is in the Fenchel conjugate, without losing anything.

Another way to quantify the above is the fact that the domain of $f^\star$ contains all the possible subgradients of $f$.
\begin{corollary}
Let $f: \R^d \to (-\infty, +\infty]$ be proper. Then, $\{\bg \in \partial f(\bx): \bx \in \R^d\} \subseteq \dom f^\star$.
\end{corollary}
\begin{proof}
From Theorem~\ref{thm:props_fenchel} (a) $\Leftrightarrow$ (c) and the fact that $\dom \partial f \subseteq \dom f$, we have
\[
\bg \in \partial f(\bx)
\Leftrightarrow f(\bx) + f^\star(\bg) = \langle \bg, \bx\rangle
\Rightarrow f^\star(\bg) <+\infty~. \qedhere
\]
\end{proof}

Note that there are cases where the domain of $f^\star$ contains vectors that are not subgradients, as shown in the next example.
\begin{example}
Let $f(x)=\exp(x)$, hence we have $f^\star(\theta)=\sup_x \ x\theta - \exp(x)$. Solving it, we have $x^\star=\ln(\theta)$ if $\theta>0$. Hence, $f^\star(\theta)=\begin{cases}\theta \ln \theta - \theta, & \text{ if } \theta>0\\ 0, & \text{ if } \theta=0\\ +\infty, & \text{ if }\theta<0\end{cases}$.
So, $f^\star(0)=\sup_{x} \ -\exp(x) = 0 <\infty$, but $0$ is not a derivative of $f$.
\end{example}

\begin{example}[Conjugate of the inner product]
Let $f(\bx)=\langle \bz, \bx\rangle$ where $\bz \neq \boldsymbol{0} \in \R^d$. Then
\[
f^\star(\btheta)
=\sup_{\bx \in \R^d}\ \langle \btheta-\bz,\bx\rangle
= \begin{cases}
0, & \btheta=\bz\\
+\infty, & \text{otherwise.}
\end{cases}
\]
\end{example}

\begin{example}[Conjugate of the hinge loss]
Let $f(\bx)=\max(1-\langle \bz,\bx\rangle,0)$\index{hinge loss!conjugate} where $\bz \in \R^d \setminus \{\boldsymbol{0}\}$ and let's calculate $f^\star(\btheta)$. If $\btheta$ has a component orthogonal to $\bz$, I can choose $\bx$ along that component, and the supremum in the definition of $f^\star$ is $+\infty$. Hence, let's consider the case that $\btheta=\alpha \bz$. In this case, we have
\[
f^\star(\btheta)
= \sup_{\bx} \ \alpha\langle \bz,\bx\rangle - \max(1-\langle \bz,\bx\rangle,0)
= \sup_{u} \ \alpha u - \max(1-u,0)~.
\]
If $\alpha>0$ or $\alpha<-1$, again the supremum is $+\infty$. Hence, we only need to consider the case that $-1 \leq \alpha\leq 0$. From a case analysis on $u$, it is easy to see that in this case the supremum is attained in $u=1$. Putting everything together, we have
\[
f^\star(\btheta)
= \begin{cases}
\alpha, & \text{ if } \btheta = \alpha \bz, \alpha \in [-1,0],\\
+\infty, & \text{ otherwise.}
\end{cases}
\]
\end{example}

\begin{example}[Conjugate of squared norms]
\label{example:conj_squared_norm}
Consider the function $f(\bx) = \frac{1}{2}\|\bx\|^2$, where $\|\cdot\|$ is a norm in $\R^d$, with dual norm $\|\cdot\|_\star$. We can show that its conjugate is $f^\star(\btheta)=\frac{1}{2}\|\btheta\|^2_\star$. Let's see how. First, we have
\[
\langle \btheta , \bx \rangle - \frac{1}{2}\|\bx\|^2
\leq \| \btheta\|_\star \|\bx\| - \frac{1}{2}\|\bx\|^2
\]
for all $\bx$. The r.h.s. is a quadratic function of $\|\bx\|$, which has maximum value $\frac{1}{2}\|\btheta\|_\star^2$. Therefore for all $\bx$, we have
\[
\langle \btheta , \bx \rangle - \frac{1}{2}\|\bx\|^2
\leq \frac{1}{2}\|\btheta\|_\star^2,
\]
which shows that $f^\star(\btheta) \leq \frac{1}{2}\|\btheta\|^2_\star$. To show the other inequality, by Lemma~\ref{lemma:dual_norm_attained}, let $\bx$ be any vector with $\langle \btheta, \bx \rangle = \|\btheta\|_\star \|\bx\|$ and $\|\bx\| = \|\btheta\|_\star$. Then, for this $\bx$, we have
\[
\langle \btheta , \bx \rangle - \frac{1}{2}\|\bx\|^2
= \frac{1}{2}\|\btheta\|_\star^2,
\]
which shows that $f^\star(\btheta) \geq \frac{1}{2}\|\btheta\|^2_\star$.
\end{example}

\begin{example}[Young's inequality]
\label{example:young}
\index{inequality!Young's|(textbf}
Let $p>1$ and $f:\R_{\geq0}\to \R_{\geq0}$ defined as $f(x)=\frac1p x^p$. Let's calculate the conjugate. We have
\[
f^\star(\theta)
=\sup_{x\geq 0} \ \theta x - \frac{1}{p}x^p~.
\]
If $\theta< 0$, then the supremum is 0.
For $\theta\geq 0$, by differentiation, we have that $x^\star = \theta^\frac{1}{p-1}$. So, we have
\[
f^\star(\theta)
=
\begin{cases}
0, & \text{if } \theta<0,\\
\frac{p-1}{p} \theta^\frac{p}{p-1}, & \text{if }\theta\geq 0~.
\end{cases}
\]
Denoting by $q>1$ the positive number such that $\frac{1}{p}+\frac{1}{q}=1$, we can rewrite it as
\[
f^\star(\theta)
= \begin{cases}
0, & \text{if } \theta<0,\\
\frac{1}{q} \theta^q, & \text{if }\theta\geq 0~.
\end{cases}
\]
Using the Fenchel--Young's inequality\index{inequality!Fenchel--Young's}, for $x,y\geq 0$ and $p,q>1$ such that $1/p+1/q=1$, we have
\[
x y \leq \frac{1}{p}x^p + \frac{1}{q}y^q,
\]
that is called \emph{Young's inequality}.
\index{inequality!Young's|)textbf}
\end{example}

\begin{example}[Conjugate of norms]
\label{example:conj_norm}
Consider the function $f(\bx) = \|\bx\|$, where $\|\cdot\|$ is a norm in $\R^d$, with dual norm $\|\cdot\|_\star$. Then, we have
\[
f^\star(\btheta)
=\sup_{\bx \in \R^d} \ \langle \btheta , \bx \rangle - \|\bx\|
\leq \sup_{\bx \in \R^d} \ \|\btheta\|_\star \|\bx\| - \|\bx\|
=\begin{cases}
0, & \|\btheta\|_\star\leq 1,\\
+\infty, & \|\btheta\|_\star> 1~.
\end{cases}
\]
To show the other inequality, by Lemma~\ref{lemma:dual_norm_attained}, let $\bx$ be any vector with $\langle \btheta, \bx \rangle = \|\btheta\|_\star \|\bx\|$ and $\|\bx\| = \alpha$. Then, we have
\[
f^\star(\btheta)
= \sup_{\bx \in \R^d} \langle \btheta , \bx \rangle - \|\bx\|
\geq \sup_{\alpha \in\R_{\geq 0}} \alpha\|\btheta\|_\star - \alpha
= \begin{cases}
0, & \|\btheta\|_\star\leq 1,\\
+\infty, & \|\btheta\|_\star> 1~.
\end{cases}
\]
\end{example}

\begin{lemma}
\label{lemma:scaling_conj}
Let $f$ be a function and let $f^\star$ be its Fenchel conjugate. For $a > 0$ and $b \in \R$, the Fenchel conjugate of $h(\bx) = a f(\bx) + b +\langle \bg, \bx\rangle$ is $h^\star(\btheta) = a f^\star((\btheta-\bg)/a) - b$.
\end{lemma}
\begin{proof}
From the definition of conjugate function, we have
\begin{align*}
h^\star(\btheta)
&= \sup_{\bx \in \R^d} \ \langle \btheta-\bg, \bx\rangle - a f(\bx) - b
= -b + a \sup_{\bx \in \R^d} \ \left\langle \frac{\btheta-\bg}{a}, \bx\right\rangle - f(\bx)\\
&= -b + a f^\star\left(\frac{\btheta-\bg}{a}\right)~. \qedhere
\end{align*}
\end{proof}

\begin{lemma}
\label{lemma:inversion_inequality_conjugate}
Let $f_1$ and $f_2$ such that $f_1(\bx) \leq f_2(\bx)$ for all $\bx$. Then, $f_1^\star(\btheta) \geq f^\star_2(\btheta)$ for all $\btheta$.
\end{lemma}
\begin{proof}
\[
f_2^\star(\btheta)
= \sup_{\bx} \ \langle \btheta, \bx\rangle - f_2(\bx)
\leq \sup_{\bx} \ \langle \btheta, \bx\rangle - f_1(\bx)
= f_1^\star(\btheta)~. \qedhere
\]
\end{proof}

\begin{lemma}[{\citealp[Example 13.8]{BauschkeC17}}]
Let $f:\R \to (-\infty, +\infty]$ be even, i.e., $f(x)=f(-x)$. Then $(f \circ \|\cdot\|_2)^\star= f^\star \circ \|\cdot\|_2$.
\end{lemma}

\begin{remark}
There is a tight connection between the dual function in convex optimization and the Fenchel conjugate.
Indeed, let $f:\R^d\to (-\infty,+\infty]$ and consider the constrained optimization problem
\begin{align*}
\min_{\bx}  &\quad f(\bx)\\
\text{s.t.} & \quad \bA \bx \leq \bb\\
            & \quad \bC \bx = \bd,
\end{align*}
where the inequality is component-wise and $\bA, \bC$ are matrices.
The \emph{dual} function is defined as
\begin{align*}
g(\blambda, \bnu)
&= \inf_{\bx} \ (f(\bx)+\blambda^\top (\bA \bx -\bb) +\bnu^\top (\bC \bx-\bd)) \\
&= -\bb^\top \blambda - \bd^\top \bnu + \inf_{\bx} \ (f(\bx)+(\bA^\top \blambda+\bC^\top \bnu)^\top \bx)\\
&= -\bb^\top \blambda - \bd^\top \bnu - f^\star(-\bA^\top \blambda - \bC^\top \bnu)~.
\end{align*}
\end{remark}

For closed, convex, and proper functions, Theorem~\ref{thm:props_fenchel} implies that $\bx \in \partial f^\star(\btheta)$ iff $\btheta \in \partial f(\bx)$, that in words means that $(\partial f)^{-1} = \partial f^\star$ in the sense of multivalued mappings. Now, we show that for strongly convex functions, the Fenchel conjugate is smooth and hence differentiable.
\begin{theorem}[Duality Strong Convexity/Smoothness]
\label{thm:prop_fenchel_sc}
\index{duality strong convexity-smoothness|(textbf}
Let $f:\R^d \to (-\infty, +\infty]$ be a proper, closed, convex function, and $\dom \partial f$ be non-empty. Then, $f$ is $\lambda>0$ strongly convex\index{function!strongly convex} with respect to $\|\cdot\|$ on $\dom \partial f$ iff $f^\star$ is $\frac{1}{\lambda}$-smooth\index{function!smooth} with respect to $\|\cdot\|_\star$ on $\R^d$.
\end{theorem}
\begin{proof}
Let's first prove the implication from left to right.
First, let's show that $f^\star$ is differentiable.
Since $f$ is proper, closed, and strongly convex, the maximizer of $\max_{\bx} \langle \btheta, \bx\rangle -f(\bx)$ exists, and it is unique by Theorem~\ref{thm:min_strongly_convex}.
Denote by $\bx^\star$ the argmax of this expression. Hence, from Theorem~\ref{thm:props_fenchel}, we have $\bx^\star \in \partial f^\star(\btheta)$. Let's now show that this is the only element in the subdifferential.
Assume there exists $\bx' \in \partial f^\star(\btheta)$, then from Theorem~\ref{thm:props_fenchel}, we have $f^\star(\btheta) = \langle\btheta,\bx'\rangle-f(\bx')$ but from the uniqueness of the maximizer we have that $\bx^\star=\bx'$.

Now, let's prove that the gradient of $f^\star$ is $\frac{1}{\lambda}$-Lipschitz with respect to $\|\cdot\|_\star$. For any $\btheta_1$ and $\btheta_2$, set $\bx_1 = \nabla f^\star(\btheta_1)$ and $\bx_2 = \nabla f^\star(\btheta_2)$. From Theorem~\ref{thm:props_fenchel}, we have that $\btheta_1 \in \partial f(\bx_1)$ and $\btheta_2 \in \partial f(\bx_2)$.
Hence, by Lemma~\ref{lemma:strong_convexity}, we have
\begin{align*}
f(\bx_2)&\geq f(\bx_1)+\langle \btheta_1, \bx_2-\bx_1\rangle + \frac{\lambda}{2}\|\bx_1-\bx_2\|^2,\\
f(\bx_1)&\geq f(\bx_2)+\langle \btheta_2, \bx_1-\bx_2\rangle + \frac{\lambda}{2}\|\bx_1-\bx_2\|^2~.
\end{align*}
Summing these two inequalities, we have
\[
\|\btheta_1-\btheta_2\|_\star \|\bx_1-\bx_2\|
\geq \langle \btheta_2-\btheta_1, \bx_1-\bx_2\rangle
\geq \lambda \|\bx_1-\bx_2\|^2,
\]
where in the first inequality we used the definition of dual norms. Solving the inequality, we get that
\[
\|\btheta_1-\btheta_2\|_\star
\geq \lambda \|\bx_1-\bx_2\|
= \lambda \|\nabla f^\star(\btheta_1)-\nabla f^\star(\btheta_2)\|~.
\]

Let's now prove the other direction.
Assume that $f^\star$ is $\frac{1}{\lambda}$-smooth with respect to $\|\cdot\|_\star$ on $\R^d$.
Set $\by \in \dom \partial f$ and $\bu \in \partial f(\by)$. Hence, by Theorem~\ref{thm:props_fenchel} and the differentiability of $f^\star$, we also have $\by = \nabla f^\star(\bu)$.
Define $\phi(\btheta):=f^\star(\btheta+\bu)-f^\star(\bu)-\langle \btheta,\nabla f^\star(\bu)\rangle$. From the $\frac{1}{\lambda}$-smoothness and Lemma~\ref{lemma:smooth_quadratic_upper_lower_bound}, we have $\phi(\btheta)\leq \frac{1}{2\lambda}\|\btheta\|^2_\star$. From Lemma~\ref{lemma:inversion_inequality_conjugate} and Example~\ref{example:conj_squared_norm}, we have that $\phi^\star(\bx) \geq \frac{\lambda}{2}\|\bx\|^2$.
Let's now calculate $\phi^\star(\bx)$.
\begin{align*}
\phi^\star(\bx)
&= \sup_{\btheta} \ \langle \btheta, \bx\rangle - f^\star(\btheta+\bu)+f^\star(\bu)+\langle \btheta, \nabla f^\star(\bu)\rangle \\
&= f^\star(\bu) - \langle \bu, \bx + \nabla f^\star(\bu)\rangle + \sup_{\bv} \ \langle \bv, \bx+\nabla f^\star(\bu)\rangle - f^\star(\bv) \\
&= f^\star(\bu) - \langle \bu, \bx + \nabla f^\star(\bu)\rangle + f(\bx+\nabla f^\star(\bu)) \\
&= - \langle \bu, \bx \rangle  - f(\nabla f^\star(\bu)) + f(\bx+\nabla f^\star(\bu)),
\end{align*}
where we used Theorem~\ref{thm:dual_of_dual} in the third equality and Theorem~\ref{thm:props_fenchel} in the last one.
Putting everything together, we have
\[
f(\bx+\by) - f(\by) - \langle \bu, \bx\rangle \geq \frac{\lambda}{2}\|\bx\|^2, \quad \forall \bu \in \partial f(\by)~. \qedhere
\]
\index{duality strong convexity-smoothness|)textbf}
\end{proof}



We will also use the following theorem on the first-order optimality condition.
\begin{theorem}
\label{thm:first_order_subdiff}
\index{first-order optimality condition|textbf}
Let $f: \R^d \to (-\infty, +\infty]$ be proper.
Then $\bx^\star \in \argmin_{\bx \in \R^d} \ f(\bx)$ iff $\boldsymbol{0} \in \partial f(\bx^\star)$.
\end{theorem}
\begin{proof}
We have that
\begin{align*}
\bx^\star \in  \argmin_{\bx \in \R^d} \ f(\bx)
&\Leftrightarrow \forall \by \in \R^d, \ f(\by)\geq f(\bx^\star) =  f(\bx^\star) + \langle \boldsymbol{0}, \by-\bx^\star\rangle\\
&\Leftrightarrow \boldsymbol{0} \in \partial f(\bx^\star)~. \qedhere
\end{align*}
\end{proof}
\index{Fenchel conjugate|)textbf}

\subsection{The ``Mirror'' Interpretation}

Here, we explain the ``mirror'' interpretation of \ac{OMD}, using the following theorem.
\begin{theorem}
Let $B_\psi$ be the Bregman divergence\index{Bregman divergence} with respect to $\psi: \mathcal{X} \to \R$, where $\psi$ is $\lambda>0$ strongly convex and closed. Let $\mathcal{V} \subseteq \mathcal{X}$ be a non-empty closed convex set and $\bx_t \in \mathcal{V}$. Define
\[
\bx_{t+1} = \argmin_{\bx \in \mathcal{V}} \ \langle \bg_t, \bx\rangle + \frac{1}{\eta_t}B_\psi(\bx; \bx_{t}),
\]
and assume $\psi$ to be differentiable in $\bx_t$ and $\bx_{t+1}$.
Then, for any $\bg_t \in \R^d$, we have
\begin{equation}
\label{eq:mirror_update}
\bx_{t+1}
= \nabla \psi_{\mathcal{V}}^\star( \nabla \psi(\bx_t) - \eta_t \bg_t),
\end{equation}
where $\psi_{\mathcal{V}}$ is the restriction of $\psi$ to $\mathcal{V}$, that is, $\psi_{\mathcal{V}}:=\psi + \indicator_{\mathcal{V}}$.
\end{theorem}
\begin{proof}
We have that
\begin{align*}
\bx_{t+1}
&= \argmin_{\bx \in \mathcal{V}} \ \langle \bg_t, \bx\rangle + \frac{1}{\eta_t}B_\psi(\bx; \bx_{t}) \\
&= \argmin_{\bx \in \mathcal{V}} \ \eta_t \langle \bg_t, \bx\rangle + B_\psi(\bx; \bx_{t}) \\
&= \argmin_{\bx \in \mathcal{V}} \ \eta_t \langle \bg_t, \bx\rangle + \psi(\bx) - \psi(\bx_{t}) -\langle \nabla \psi(\bx_t), \bx-\bx_t\rangle \\
&= \argmin_{\bx \in \mathcal{V}} \ \langle \eta_t \bg_t-\nabla \psi(\bx_t), \bx\rangle + \psi(\bx)~.
\end{align*}

Now, we use the first-order optimality condition in Theorem~\ref{thm:first_order_subdiff} to have
\[
\boldsymbol{0} \in \eta_t \bg_t + \nabla \psi (\bx_{t+1}) - \nabla \psi(\bx_t) + \partial \indicator_{\mathcal{V}}(\bx_{t+1}),
\]
that is
\[
\nabla \psi(\bx_t) - \eta_t \bg_t
\in (\nabla \psi + \partial \indicator_{\mathcal{V}})(\bx_{t+1})
\subseteq \partial \psi_{\mathcal{V}}(\bx_{t+1}),
\]
where in the last inclusion we used Theorem~\ref{thm:sum_subgradients}.
Hence, from Theorem~\ref{thm:props_fenchel}, we have
\[
\bx_{t+1} \in \partial \psi^\star_{\mathcal{V}}(\nabla \psi(\bx_t) - \eta_t \bg_t)~.
\]
Using the fact that $\psi_\mathcal{V} :=\psi + \indicator_{\mathcal{V}}$ is $\lambda$-strongly convex, proper, and closed, from Theorem~\ref{thm:prop_fenchel_sc}\index{duality strong convexity-smoothness} we have that $\partial \psi_{\mathcal{V}}^\star=\{\nabla \psi_{\mathcal{V}}^\star\}$.
Hence,
\[
\bx_{t+1} = \nabla \psi_{\mathcal{V}}^\star( \nabla \psi(\bx_t) - \eta_t \bg_t)~. \qedhere
\]
\end{proof}


\begin{figure}
\centering
\begin{tikzpicture}
    \draw[thick] (0,2.5) -- (2.5,0) -- (0,-2.5) -- (-2.5,0) -- cycle;
    \node at (-1.5,1.5) {$\mathcal{V}$}; 

     \draw[thick, rotate around={15:(7,0)}] (7,0) ellipse (3.5 and 2);

    \filldraw (0,1) circle (0pt) node[left] {$\bx_{t}$}; 
    \filldraw (7,1) circle (0pt) node[right] {$\nabla \psi(\bx_t)$}; 
    \draw[->,>=stealth] (0,1) .. controls (2,2) and (5,2) .. (7,1) node[midway, above] {$\nabla \psi$};

    \filldraw (0,-1) circle (0pt) node[above] {$\bx_{t+1}$}; 

    \node at (5,-0.5) {$\nabla \psi(\bx_t) - \eta_t \bg_t$};
    \draw[->,>=stealth] (7,.8) -- (5.2,-0.2);  

    \draw[->,>=stealth] (4.8,-0.7) .. controls (3,-1.5) and (2,-1.5) .. (0,-0.5) node[midway, below] {$\nabla \psi_{\mathcal{V}}^\star$};
\end{tikzpicture}
\caption{OMD update in terms of duality maps.}
\label{fig:duality_maps}
\commentAlt{Figure~\ref{fig:duality_maps}. Diagram of the OMD update through duality mappings. A point x_t in the primal set V is mapped by nabla psi to nabla psi(x_t) in the dual space, shifted by -eta_t g_t, then mapped back by nabla psi_V^* to x_{t+1}.}
\end{figure}

Let's explain what this theorem says. We said that \ac{OMD} extends the \ac{OSD} method to non-Euclidean norms. Hence, the regret bound we proved contains a pair of dual norms: one for distances between iterates and one for gradients. We also said that it makes sense to use a dual norm to measure a gradient, because it is a natural way to measure how ``big'' the linear functional $\bx \mapsto \langle \nabla f(\by), \bx\rangle$ is. More precisely, the gradient is naturally an element of the \emph{dual space}, that is, a different space from the one where the predictions live. Hence, in general, summing iterates and gradients is not the right geometric operation, in the same way in which we cannot sum pears and apples together. In the Euclidean case, the dual of the L$_2$ norm is again the L$_2$ norm, so the distinction between the geometry used for predictions and the geometry used for gradients is hidden. For a general norm, predictions and gradients are measured with different, dual norms, and \ac{OMD} uses the mirror maps to move between these two geometries.

So, in \ac{OMD} we need a way to go from one space to the other. This is exactly the role of $\nabla \psi$ and $\nabla \psi^\star_{\mathcal{V}}$, which are called \textbf{duality maps}\index{duality map|textbf} or \textbf{mirror}\index{mirror map|textbf} and \textbf{inverse mirror} maps. We can now understand that the theorem tells us that \ac{OMD} takes the primal vector $\bx_t$, transforms it into a dual vector through $\nabla \psi$, does a subgradient descent step in the dual space, and finally transforms the vector back to the primal space through $\nabla \psi_{\mathcal{V}}^\star$. This reasoning is summarized in Figure~\ref{fig:duality_maps}.

\begin{example}
Let $\psi:\R^d \to \R$ equal to $\psi(\bx)=\frac{1}{2}\|\bx\|_2^2$ and $\mathcal{V}=\{\bx \in \R^d: \|\bx\|_2\leq 1\}$. Define $\psi_{\mathcal{V}}=\psi+\indicator_{\mathcal{V}}$.
Then, we have
\[
\psi^\star_{\mathcal{V}}(\btheta)
= \sup_{\bx \in \mathcal{V}} \ \langle \btheta, \bx\rangle - \frac12 \|\bx\|_2^2~.
\]
Let's compute this conjugate.
First of all, if $\btheta=\boldsymbol{0}$ we have that $\psi^\star_{\mathcal{V}}(\btheta)=0$. So, in the following, we assume $\btheta\neq \boldsymbol{0}$.
For any $\bx \in \mathcal{V}$ there exist $\bq$ and $\alpha$ such that $\bx=\alpha \frac{\btheta}{\|\btheta\|_2} + \bq$ and $\langle\bq, \btheta\rangle=0$. Hence, we have
\[
\sup_{\bx \in \mathcal{V}} \ \langle \btheta, \bx\rangle - \frac12 \|\bx\|_2^2
= \sup_{\alpha, \bq : \alpha \frac{\btheta}{\|\btheta\|_2} + \bq \in \mathcal{V}, \langle\bq, \btheta\rangle=0} \ \alpha \|\btheta\|_2 - \frac{\alpha^2}{2} - \frac12\|\bq\|_2^2
= \sup_{-1 \leq \alpha\leq 1} \ \alpha \|\btheta\|_2 - \frac{\alpha^2}{2}.
\]
Solving the constrained optimization problem, we have $\alpha^\star = \min(1,\|\btheta\|_2)$. Hence, we have
\[
\psi^\star_{\mathcal{V}}(\btheta)
= \begin{cases}
\frac{1}{2}\|\btheta\|^2_2, & \|\btheta\|_2\leq 1\\
\|\btheta\|_2-\frac{1}{2}, & \|\btheta\|_2> 1
\end{cases},
\]
that is finite everywhere and differentiable.

So, the two duality maps\index{duality map} are $\nabla \psi(\bx)=\bx$ and
\[
\nabla \psi^\star_{\mathcal{V}}(\btheta) =
\begin{cases}
\btheta, & \|\btheta\|_2\leq 1\\
\frac{\btheta}{\|\btheta\|_2}, & \|\btheta\|_2> 1\\
\end{cases}
= \Pi_{\mathcal{V}}(\btheta)~.
\]
Using \eqref{eq:mirror_update}, we obtain exactly the update of projected online subgradient descent.
\end{example}

\subsection{Yet Another Way to Write the Online Mirror Descent Update}

There exists yet another way to write the update of \ac{OMD}. This third method uses the concept of \emph{Bregman projections}.\index{projection!Bregman|(textbf}
Extending the definition of Euclidean projections, we can define the projection with respect to a Bregman divergence\index{Bregman divergence}.
\begin{definition}
Let $\psi:\mathcal{X} \to \R$ be strictly convex and differentiable in the interior of $\mathcal{X}$. Assume that $\mathcal{V} \subseteq \mathcal{X}$ is convex and closed. Then, for any $\bx \in \interior \mathcal{X}$, we define the \textbf{Bregman projection} as
\[
\Pi_{\mathcal{V},B_\psi}(\bx) := \argmin_{\by \in \mathcal{V}} \ B_{\psi}(\by; \bx),
\]
when the argmin exists.
\end{definition}
\index{projection!Bregman|)textbf}

In the online learning literature, the \ac{OMD} algorithm is typically presented with a two-step update: first, solving the argmin over the entire space and then projecting back over $\mathcal{V}$ with respect to the Bregman divergence\index{Bregman divergence}. In the following, we show that most of the time the two-step update is equivalent to the one-step update in \eqref{eq:mirror_update}.

First, we prove a general theorem that allows us to break the constrained minimization of functions in the minimization over the entire space, plus a Bregman projection step.

\begin{theorem}
\label{thm:two_steps_omd}
Let $f : \R^d \to (-\infty, +\infty]$ be proper, closed, strictly convex\index{function!strictly convex}, and differentiable in $\interior \dom f$. Also, let $\mathcal{V} \subset \R^d$ be a non-empty, closed convex set with $\mathcal{V} \cap \dom f \neq \emptyset$ and assume that $\tilde{\by} = \argmin_{\bz \in \R^d} \ f(\bz)$ exists and $\tilde{\by} \in \interior \dom f$. Then, we have
\begin{enumerate}
\item $\argmin_{\bz \in \mathcal{V}} \ f(\bz)$ contains exactly one element.
\item $\argmin_{\bz \in \mathcal{V}} \ f(\bz) = \argmin_{\bz \in \mathcal{V}} \ B_f(\bz; \tilde{\by})$\index{projection!Bregman}.
\end{enumerate}
\end{theorem}
\begin{proof}
For the first point, from \citet[Proposition 11.13]{BauschkeC17} and the existence of $\tilde{\by}$, we have that $f$ is coercive. So, from \citet[Proposition 11.15]{BauschkeC17}, the minimizer of $f$ in $\mathcal{V}$ exists. Given that $f$ is strictly convex, the minimizer must be unique too.

For the second point, denote by $\by'=\argmin_{\bz \in \mathcal{V}} \ B_f(\bz; \tilde{\by})$ and $\by = \argmin_{\bz \in \mathcal{V}} \ f(\bz)$.
From the definition of $\by$, we have $f(\by)\leq f(\by')$.
On the other hand, from the first-order optimality condition, we have $\nabla f(\tilde{\by}) = \boldsymbol{0}$.
So, we have
\begin{align*}
f(\by')-f(\tilde{\by})
=B_f(\by'; \tilde{\by})
\leq B_f(\by; \tilde{\by})
= f(\by) - f(\tilde{\by}),
\end{align*}
that is $f(\by')\leq f(\by)$. Given that $f$ is strictly convex, $\by'=\by$.
\end{proof}

Now, note that, if $\tilde{\psi} (\bx) =\psi(\bx) + \langle \bg, \bx\rangle$, then
\begin{align*}
B_{\tilde{\psi}}(\bx;\by)
&= \tilde{\psi}(\bx) - \tilde{\psi}(\by) -\langle \nabla \tilde{\psi}(\by), \bx-\by\rangle  \\
&= \psi(\bx) - \psi(\by) -\langle \bg + \nabla \psi(\by), \bx-\by\rangle +\langle \bg, \bx-\by\rangle \\
&=\psi(\bx) - \psi(\by) -\langle \nabla \psi(\by), \bx-\by\rangle
= B_{\psi}(\bx;\by)~.
\end{align*}
Now, define $f(\bx)=\langle \eta_t \bg_t, \bx\rangle + B_\psi(\bx; \bx_t)$, so that $f(\bx)=\psi(\bx)+\langle \bz, \bx\rangle + K$ for some $K \in \R$ and $\bz \in \R^d$.
This implies that $B_f(\bx; \by) = B_\psi(\bx; \by)$. Hence, under the assumption of the above theorem, we have that $\bx_{t+1}=\argmin_{\bx \in \mathcal{V}} \ \langle \bg_t, \bx\rangle + \frac{1}{\eta_t} B_\psi(\bx;\bx_t)$ is equivalent to
\begin{align*}
\tilde{\bx}_{t+1} &= \argmin_{\bx \in \mathcal{X}} \ \langle \eta_t \bg_t, \bx\rangle +  B_\psi(\bx;\bx_t),\\
\bx_{t+1} &= \argmin_{\bx \in \mathcal{V}} \ B_{\psi}(\bx; \tilde{\bx}_{t+1})~.
\end{align*}
The advantage of this update is that sometimes it gives two easier problems to solve rather than a single difficult one.

\index{Online Mirror Descent algorithm!with local norms|(textbf}
\section{OMD Regret Bound using Local Norms}
\label{sec:omd_local_norms}

In Lemma~\ref{lemma:omd_one_step}, strong convexity basically tells us that there is some minimum curvature in all the directions, which allows us to upper bound the difference between $\bx_t$ and $\bx_{t+1}$. However, it turns out that we can still get a meaningful regret upper bound without this assumption. In particular, we can get an interesting expression for the regret that involves the use of \textbf{local norms}\index{norm!local}. We will use these ideas with multi-scale experts (Section~\ref{sec:multiscale}) and with multi-armed bandits (Chapter~\ref{ch:bandits}).

\begin{lemma}
\label{lemma:omd_local_norms}
Let $\mathcal{X}\subseteq\R^d$ be closed and convex, $B_\psi$ be the Bregman divergence\index{Bregman divergence} with respect to $\psi: \mathcal{X} \to \R$, and assume $\psi$ twice differentiable and with the Hessian positive definite in $\interior \mathcal{X}$. Let $\mathcal{V} \subseteq \mathcal{X}$ be a non-empty closed convex set. Assume \eqref{eq:cond_omd1} or \eqref{eq:cond_omd2} to hold.
Define $\|\bx\|_{\bA}:= \sqrt{\bx^\top \bA \bx}$.
Also, with the notation and initialization in Algorithm~\ref{alg:omd}, assume $\bx_{t+1}$ and $\tilde{\bx}_{t+1} \in \argmin_{\bx \in \mathcal{X}} \ \langle \bg_t, \bx\rangle + \frac{1}{\eta_t}B_\psi(\bx; \bx_t)$ exist. Then, $\forall \bu \in \mathcal{V}$,
there exists $\bz_t$ on the line segment between $\bx_t$ and $\bx_{t+1}$, and $\bz'_t$ on the line segment between $\bx_t$ and $\tilde{\bx}_{t+1}$, such that the following inequality holds
\[
\eta_t \langle \bg_t, \bx_t -\bu\rangle
\leq B_\psi(\bu;\bx_t) - B_\psi(\bu;\bx_{t+1}) + \frac{\eta_t^2 \min\left(\|\bg_t\|^2_{(\nabla^2 \psi(\bz_t))^{-1}}, \|\bg_t\|^2_{(\nabla^2 \psi(\bz'_t))^{-1}}\right)}{2} ~.
\]
\end{lemma}
\begin{proof}
First of all, from Theorem~\ref{thm:omd_well_defined} used with two different feasible sets, $\bx_{t+1}$ and $\tilde{\bx}_{t+1}$ are in the interior of $\mathcal{X}$ for all $t\geq 1$.
Then, reasoning as in the proof of Lemma~\ref{lemma:omd_one_step}, we have
\begin{equation}
\label{eq:omd_local_norm_1}
\langle \eta_t \bg_t, \bx_t - \bu\rangle
\leq
B_\psi(\bu;\bx_t) - B_\psi(\bu;\bx_{t+1}) - B_\psi(\bx_{t+1};\bx_t) + \langle \eta_t \bg_t, \bx_t - \bx_{t+1} \rangle~.
\end{equation}
From the Taylor's theorem, $B_{\psi}(\bx_{t+1};\bx_t) = \frac{1}{2}(\bx_{t+1}-\bx_t)^\top \nabla^2 \psi(\bz_t) (\bx_{t+1}-\bx_t)$ for some $\bz_t$ on the line segment between $\bx_t$ and $\bx_{t+1}$. Observe that this is $\frac12 \|\bx_{t+1}-\bx_t\|^2_{\nabla^2 \psi(\bz_t)}$ and it is indeed a norm because we assumed the Hessian of $\psi$ to be positive definite. Hence, by Fenchel--Young's inequality\index{inequality!Fenchel--Young's} and Examples~\ref{example:dual_norm_a} and \ref{example:conj_squared_norm}, we have
\begin{align*}
\langle &\eta_t \bg_t, \bx_t - \bx_{t+1} \rangle - B_\psi(\bx_{t+1};\bx_t)\\
&\leq \frac{\eta_t^2}{2}\|\bg_t\|^2_{(\nabla^2 \psi(\bz_t))^{-1}} + \frac{1}{2}(\bx_{t+1}-\bx_t)^\top \nabla^2 \psi(\bz_t) (\bx_{t+1}-\bx_t) - B_{\psi}(\bx_{t+1};\bx_t) \\
&=\frac{\eta_t^2}{2}\|\bg_t\|^2_{(\nabla^2 \psi(\bz_t))^{-1}},
\end{align*}
that gives the first term in the minimum.

For the second term in the minimum, we instead observe that
\begin{align*}
\langle \eta_t \bg_t, \bx_t - \bx_{t+1}\rangle - B_{\psi}(\bx_{t+1};\bx_t)
&\leq \max_{\bx \in \mathcal{X}} \ \langle \eta_t \bg_t, \bx_t - \bx\rangle - B_{\psi}(\bx;\bx_t)\\
&= \langle \eta_t \bg_t, \bx_t - \tilde{\bx}_{t+1}\rangle - B_{\psi}(\tilde{\bx}_{t+1};\bx_t)~.
\end{align*}
Then, we proceed as in the first bound.
\end{proof}
Despite the apparent more difficult formulation, the second term in the minimum is often easier to use, especially in constrained settings, because $\tilde{\bx}_{t+1}$ is defined over $\mathcal{X}$ rather than over $\mathcal{V}$. Also, under the assumptions of Theorem~\ref{thm:two_steps_omd}, it is easy to recognize that $\bx_{t+1}$ is the Bregman projection of $\tilde{\bx}_{t+1}$ onto $\mathcal{V}$.
\index{Online Mirror Descent algorithm!with local norms|)textbf}

\section{Example of OMD: Exponentiated Gradient}
\label{sec:omd_eg}

\index{Exponentiated Gradient algorithm!OMD version|(textbf}

Let $\Delta^{d-1} := \{\bx \in \R^d: x_i\geq0, \|\bx\|_1=1\}$ be the \textbf{probability simplex}\index{probability simplex|textbf} and set $\mathcal{V}=\Delta^{d-1}$. So, in words, we want to output vectors in the probability simplex. Also, let $\mathcal{X} =\R_{\geq 0}^{d}$ and $\psi(\bx):\mathcal{X}\to \R$ defined as $\psi(\bx)=\sum_{i=1}^d x_i \ln x_i$, where we define $0 \ln(0)=0$. Note that the restriction of $\psi$ to $\mathcal{V}$ is the negative Shannon entropy of the discrete distributions in $\Delta^{d-1}$\index{entropy!negative Shannon}. It is possible to verify that $\psi$ satisfies the first condition in Theorem~\ref{thm:omd_well_defined}, hence the update is well defined.

The Fenchel conjugate $\psi_{\mathcal{V}}^\star(\btheta)$ is defined as
\[
\psi_{\mathcal{V}}^\star(\btheta)
= \sup_{\bx \in \mathcal{V}} \ \langle \btheta, \bx\rangle - \psi(\bx)
= \sup_{\bx \in \mathcal{V}} \ \langle \btheta, \bx\rangle - \sum_{i=1}^d x_i \ln x_i~.
\]
It is a constrained optimization problem, we could solve it using the \acl{KKT} conditions\index{Karush--Kuhn--Tucker conditions}.
However, there is a simpler way to do it: we will remove the probability simplex constraint, rephrasing the problem over $d-1$ variables.
In fact, the maximization problem is equivalent to
\[
\min_{\bx \in \R^{d-1}} \ \sum_{i=1}^{d-1} x_i \ln x_i + \left(1-\sum_{i=1}^{d-1} x_i\right)\ln \left(1-\sum_{i=1}^{d-1} x_i\right) - \sum_{i=1}^{d-1} \theta_i x_i - \theta_d \left(1-\sum_{i=1}^{d-1} x_i\right)~.
\]
Note that a constraint on $x_1,\dots, x_{d-1}$ and $1-\sum_{i=1}^{d-1} x_i$ to be non-negative is enforced by the domain of the logarithm.
This is now an unconstrained convex optimization problem, so we can solve it by equating the gradient of the objective function to zero.
Hence, we have
\begin{align*}
\ln \frac{x_i}{1-\sum_{j=1}^{d-1} x_j} = \theta_i -\theta_d, \quad i=1, \dots, d-1~.
\end{align*}
That is
\begin{equation}
\label{eq:omd_eg_1}
x_i = \exp(\theta_i -\theta_d)\left(1-\sum_{j=1}^{d-1} x_j\right), \quad i=1, \dots, d-1~.
\end{equation}
Summing this equality over $i=1,\dots,d-1$, we obtain
\[
\sum_{i=1}^{d-1} x_i = \sum_{i=1}^{d-1} \exp(\theta_i -\theta_d)\left(1-\sum_{j=1}^{d-1} x_j\right),
\]
that can be solved to obtain
\[
1-\sum_{j=1}^{d-1} x_j = \frac{1}{1+\sum_{j=1}^{d-1} \exp(\theta_j-\theta_d)}~.
\]
Substituting it back in \eqref{eq:omd_eg_1}, we have
\begin{equation}
x_i
= \frac{\exp(\theta_i-\theta_d)}{1+\sum_{j=1}^{d-1} \exp(\theta_j-\theta_d)}
= \frac{\exp(\theta_i)}{\sum_{j=1}^{d} \exp(\theta_j)}, \quad i=1, \dots, d~. \label{eq:maximizer_fenchel_eg}
\end{equation}


Denoting with $\alpha=\sum_{i=1}^d \exp(\theta_i)$, and substituting in the definition of the conjugate function, we get
\begin{align*}
\psi_{\mathcal{V}}^\star(\btheta)
&= \sum_{i=1}^d \left(\frac{1}{\alpha} \theta_i\exp(\theta_i) - \frac{1}{\alpha} \exp(\theta_i)(\theta_i-\ln(\alpha))\right)
= \ln(\alpha) \frac{1}{\alpha}\sum_{i=1}^d \exp(\theta_i )
= \ln(\alpha)\\
&= \ln\left(\sum_{i=1}^d \exp(\theta_i)\right)~.
\end{align*}

We also have $(\nabla \psi_{\mathcal{V}}^\star(\btheta))_j= \frac{\exp(\theta_j)}{\sum_{i=1}^d \exp(\theta_i)}$ and $(\nabla \psi(\bx))_j=\ln(x_j)+1$ for $\bx \in \R^d_{>0}$. Note that we could have also derived the gradient of $\psi_{\mathcal{V}}^\star$ directly from~\eqref{eq:maximizer_fenchel_eg}, using Theorem~\ref{thm:props_fenchel}.

\begin{algorithm}[t]
\caption{Exponentiated Gradient (EG)}
\label{alg:eg}
\begin{algorithmic}[1]
{
    \REQUIRE{$\eta>0$}
    \STATE{Set $\bx_1 = [1/d,\dots,1/d]$}
    \FOR{$t=1$ {\bfseries to} $T$}
    \STATE{Output $\bx_t \in \Delta^{d-1}$}
    \STATE{Pay the loss $\ell_t(\bx_t)$, where $\ell_t$ is subdifferentiable on $\Delta^{d-1}$}
    \STATE{Set $\bg_t \in \partial \ell_t(\bx_t)$}
    \STATE{$x_{t+1,j} = \frac{x_{t,j}\exp(-\eta g_{t,j})}{\sum_{i=1}^d x_{t,i} \exp(-\eta g_{t,i})}, \ j=1,\dots,d$}
    \ENDFOR
}
\end{algorithmic}
\end{algorithm}

Putting everything together, we have the online mirror descent update rule for the entropic distance generating function\index{distance generating function}.
\[
x_{t+1,j}
= \frac{\exp(\ln x_{t,j}+1-\eta_t g_{t,j})}{\sum_{i=1}^d \exp(\ln x_{t,i}+1-\eta_{t} g_{t,i})}
= \frac{x_{t,j}\exp(-\eta_t g_{t,j})}{\sum_{i=1}^d x_{t,i} \exp(-\eta_{t} g_{t,i})}~.
\]
The algorithm is summarized in Algorithm~\ref{alg:eg}. This algorithm is called \textbf{\ac{EG}} because in the update rule we take the component-wise exponential of the (sub)gradient vector.

Let's take a look at the regret bound we get. From Example~\ref{example:kl}, for all $\bx \in \Delta^{d-1}$ and $\by \in \{\bx \in \R^d: x_i >0, \|\bx\|_1=1\}$, we have $\KL(\bx;\by):=B_\psi(\bx; \by) = \sum_{i=1}^d x_i \ln \tfrac{x_i}{y_i}$, that is the \ac{KL} divergence between the discrete distributions $\bx$ and $\by$.
Now, we prove the strong convexity of $\psi$ through a slightly more general statement. Setting $c_i =1$ for $i=1, \dots, d$ gives the result we need here.
\begin{lemma}
\label{lemma:entropy_strongly_convex}
Let $c_1, \dots, c_d \in \R_{>0}$.
Then, the negative weighted entropy function $\psi(\bx)=\sum_{i=1}^d c_i x_i \ln x_i$\index{entropy!negative weighted} is 1-strongly convex\index{function!strongly convex} with respect to the L$_1$ norm defined over the set $\mathcal{K}=\{\bx \in \R^d: x_i >0, \sum_{i=1}^d x_i/c_i=1\}$.
\end{lemma}
\begin{proof}

Let $\phi(u)=(u-1) \ln u - \frac{2(u-1)^2}{u+1}$ for $u>0$. Observe that $\phi''(u)> 0$ for $u>0$ so the function is convex. Moreover, $\phi(1)=\phi'(1)=0$. So, we have $\phi(u)\geq \phi(1)+\phi'(1)(u-1)=0$ for all $u>0$.

Using this inequality, we have
\begin{align*}
\langle &\nabla \psi(\bx) -\nabla \psi(\by), \bx -\by\rangle\\
&= \sum_{i=1}^d c_i (x_i-y_i)\ln \frac{x_i}{y_i}
= \sum_{i=1}^d c_i y_i \left(\frac{x_i}{y_i}-1\right)\ln \frac{x_i}{y_i}
\geq \sum_{i=1}^d c_i \frac{2(x_i-y_i)^2}{x_i+y_i}\\
&= \sum_{i=1}^d \frac{x_i+y_i}{2 c_i} \left(\frac{|x_i-y_i|}{\frac{x_i+y_i}{2 c_i}}\right)^2
\geq \left(\sum_{i=1}^d \frac{x_i+y_i}{2 c_i} \frac{|x_i-y_i|}{\frac{x_i+y_i}{2 c_i}}\right)^2
= \|\bx-\by\|_1^2,
\end{align*}
where in the last inequality we used Jensen's inequality\index{inequality!Jensen's} (Theorem~\ref{thm:jensen}) with the distribution $[\frac{x_1+y_1}{2 c_1}, \dots, \frac{x_d+y_d}{2 c_d}]$.
Using Theorem~\ref{thm:hessian_strong_conv} completes the proof.
\end{proof}

Another thing to do is to decide the initial point $\bx_1$. A reasonable choice is to set $\bx_1$ to be the minimizer of $\psi$ in $\mathcal{V}$. Hence, we set $\bx_1=[1/d, \dots, 1/d] \in \R^d$, because the uniform distribution minimizes the negative entropy. Hence, we have that $B_\psi(\bu; \bx_1)$ is equal to $\psi(\bu)-\min_{\bx \in \mathcal{V}} \psi(\bx)$.
So, we have $B_\psi(\bu;\bx_1) = \sum_{i=1}^d u_i \ln u_i + \ln{d} \leq \ln d$.

Putting everything together, we have
\[
\sum_{t=1}^T (\ell_t(\bx_t) - \ell_t(\bu))
\leq \frac{\ln d}{\eta} + \frac{\eta}{2} \sum_{t=1}^T \|\bg_t\|_\infty^2, \quad \forall \bu \in \Delta^{d-1}~.
\]
Assuming $\|\bg_t\|_\infty\leq L_\infty$, we can set $\eta=\sqrt{\frac{2\ln d}{L_\infty^2 T}}$, to obtain an upper bound on the regret of $L_\infty \sqrt{2 T \ln d}$.

\begin{remark}
Note that the time-varying version of \ac{OMD} with entropic distance generating function\index{distance generating function} would give rise to a vacuous bound. Can you see why? In Chapter~\ref{ch:ftrl}, we will see how \ac{FTRL} overcomes this issue using a time-varying regularizer rather than a time-varying learning rate.
\end{remark}

\index{Online Mirror Descent algorithm!with local norms|(}
We can also get a \emph{tighter} bound using the local norms. Let's use the additional assumption that $g_{t,i}\geq0$, for all $t=1, \dots,T$ and $i=1, \dots,d$. Summing the inequality of Lemma~\ref{lemma:omd_local_norms} from $t=1$ to $T$, we have for all $\bu \in \mathcal{V}$ that
\begin{align*}
\sum_{t=1}^T \ell_t(\bx_t) - \sum_{t=1}^T \ell_t(\bu)
&\leq \frac{\ln d}{\eta} + \frac{\eta}{2} \sum_{t=1}^T \|\bg_t\|^2_{(\nabla^2 \psi(\bz'_t))^{-1}},
\end{align*}
where $\bz'_t$ is on the line segment between $\bx_t$ and $\tilde{\bx}_{t+1}$. In this case, it is easy to calculate $\tilde{x}_{t+1,i}$ as $x_{t,i}\exp(-\eta g_{t,i})$ for $i=1, \dots,d$. Moreover, $\nabla^2 \psi(\bz'_t)$ is a diagonal matrix whose elements on the diagonal are $\frac{1}{z'_{t,i}}, i=1, \dots, d$. Hence, we have that
\[
\|\bg_t\|^2_{(\nabla^2 \psi(\bz'_t))^{-1}}
= \sum_{i=1}^d g_{t,i}^2 z'_{t,i}
\leq \sum_{i=1}^d g_{t,i}^2 x_{t,i}~.
\]
Putting everything together, the final bound would be
\begin{equation}
\sum_{t=1}^T \ell_t(\bx_t) - \sum_{t=1}^T \ell_t(\bu)
\leq \frac{\ln d}{\eta} + \frac{\eta}{2} \sum_{t=1}^T \sum_{i=1}^d g_{t,i}^2 x_{t,i}, \quad \forall \bu \in \Delta^{d-1}~. \label{eq:eg_localnorm}
\end{equation}
This is indeed a tighter bound because $\sum_{i=1}^d g_{t,i}^2 x_{t,i}\leq \|\bg_t\|^2_\infty$.

One can also weaken the assumption on the subgradients at the cost of a slight increase in the constant in front of them. In fact, for example, assume that $\|\bg_t\|_\infty\leq L_\infty$ for all $t=1, \dots, T$ and $\eta\leq \frac{1}{L_\infty}$. In this case, we can upper bound $\tilde{x}_{t+1,i}$ as $x_{t,i}\exp(-\eta g_{t,i}) \leq e x_{t,i}$. Hence, the final bound in this case is
\begin{align*}
\sum_{t=1}^T \ell_t(\bx_t) - \sum_{t=1}^T \ell_t(\bu)
&\leq \frac{\ln d}{\eta} + \frac{e \eta}{2} \sum_{t=1}^T \sum_{i=1}^d g_{t,i}^2 x_{t,i}, \quad \forall \bu \in \Delta^{d-1}~.
\end{align*}
\index{Online Mirror Descent algorithm!with local norms|)}

How would \ac{OSD} work on the same problem? First, it is important to realize that nothing prevents us from using \ac{OSD} on this problem. We just have to implement the Euclidean projection onto the probability simplex\index{probability simplex}, which does not have a closed form, but it can be implemented in $\mathcal{O}(d \ln d)$ time, see \citet{Condat16}.
The regret bound we would get from \ac{OSD} is
\[
\sum_{t=1}^T (\ell_t(\bx_t) - \ell_t(\bu))
\leq \frac{2}{\eta} + \frac{\eta}{2} \sum_{t=1}^T \|\bg_t\|_2^2
\leq \frac{2}{\eta} + \frac{\eta}{2} T L^2_2, \quad \forall \bu \in \Delta^{d-1},
\]
where $L_2 \geq \|\bg_t\|_2$ for all $t$. Optimally tuning the learning rate, we get
\[
\sum_{t=1}^T (\ell_t(\bx_t) - \ell_t(\bu))
\leq 2 L_2 \sqrt{ T }, \quad \forall \bu \in \Delta^{d-1}~.
\]
To compare it with the bound of \ac{EG}, it is sufficient to observe that
\[
\frac{1}{\sqrt{d}}\|\bg_t\|_{2}
\leq \|\bg_t\|_\infty
\leq \|\bg_t\|_2,
\]
and the inequalities are tight.
Hence, ignoring the numerical constants, the bound of \ac{EG} is between $\sqrt{\frac{\ln d}{d}}$ and $\sqrt{\ln d}$ times the one of \ac{OSD}, depending on the structure of the subgradients $\bg_t$.
Hence, in a worst-case sense with respect to the infinity norm of the subgradients $\bg_t$, using an entropic distance generating function\index{distance generating function} can transform a dependence on the dimension from $\sqrt{d}$ to $\sqrt{\ln d}$ for \ac{OCO} over the probability simplex.
So, as we already saw analyzing AdaGrad (Section~\ref{sec:adagrad}), the shape of the domain and the structure of the subgradients are the important ingredients when we move from the Euclidean norm to other norms.

\index{Exponentiated Gradient algorithm!OMD version|)textbf}

%
%

\section{Example of OMD: $p$-norm Algorithms}
\label{sec:omd_pnorm}

We now present a class of algorithms that interpolates between the regret guarantee of \ac{OSD} and that of \ac{EG}.

Consider the distance generating function\index{distance generating function} $\psi(\bx)=\frac{1}{2}\|\bx\|^2_p$, for $1 < p\leq2$ over $\mathcal{X} =\mathcal{V}=\R^d$. Let's remind the reader that the $p$-norm of a vector $\bx$ is defined as $(\sum_{i=1}^d |x_i|^p)^{1/p}$.
From Examples~\ref{example:dual_lp} and~\ref{example:conj_squared_norm}, we have that $\psi^\star_{\mathcal{V}}(\btheta) = \frac{1}{2}\|\btheta\|_q^2$, where $\frac{1}{p}+\frac{1}{q}=1$\index{norm!$p$-}, so that $q\geq 2$.
Let's calculate the dual maps: $(\nabla \psi(\bx))_j = \sign(x_j) |x_j|^{p-1} \|\bx\|^{2-p}_p$ and, for $\bx\neq \boldsymbol{0}$, $(\nabla \psi^\star_{\mathcal{V}}(\bx))_j = \sign(x_j) |x_j|^{q-1} \|\bx\|^{2-q}_q$, while $\nabla \psi^\star_{\mathcal{V}}(\boldsymbol{0})=\boldsymbol{0}$.
Hence, we can write the update rule as
\begin{align*}
\tilde{x}_{t+1,j} &= \sign(x_{t,j}) |x_{t,j}|^{p-1}\|\bx_t\|^{2-p}_p - \eta g_{t,j}, \quad j=1,\dots,d, \\
x_{t+1,j} &= \begin{cases} \sign(\tilde{x}_{t+1,j})|\tilde{x}_{t+1,j}|^{q-1} \left\| \tilde{\bx}_{t+1} \right\|^{2-q}_q, & \tilde{\bx}_{t+1}\neq \boldsymbol{0},\\
0, & \tilde{\bx}_{t+1}= \boldsymbol{0},
\end{cases}
\quad j=1,\dots,d,
\end{align*}
where we broke the update into two steps to simplify the notation (and the implementation).
Starting from $\bx_1=\boldsymbol{0}$, we have that
\[
B_\psi(\bu; \bx_1)
= \psi(\bu) - \psi(\bx_1) - \langle \nabla \psi(\bx_1), \bu - \bx_1 \rangle
= \psi(\bu)~.
\]

The last ingredient is the fact that $\psi(\bx)$ is $p-1$ strongly convex\index{function!strongly convex} with respect to $\|\cdot\|_p$.
\begin{lemma}
\label{lemma:strong_convexity_pnorm}
$\psi(\bx)=\frac{1}{2}\|\bx\|^2_p$ is $(p-1)$-strongly convex with respect to $\|\cdot\|_p$, for $1 \leq p\leq 2$.
\end{lemma}
\begin{proof}
Observe that for $p=1$ the function is convex, hence it is $0$-strongly convex with respect to any norm. So, we now assume $p>1$.

We want to show that
\begin{equation}
\label{eq:p_norm_strongly_convex_1}
\psi(\by)
\geq \psi(\bx) + \langle \nabla \psi(\bx), \by-\bx\rangle + \frac{p-1}{2}\|\bx-\by\|^2_p, \quad \forall \bx,\by \in \R^d~.
\end{equation}

Rather than studying $\psi$ directly, we study the Hessian of a surrogate function: $\psi_a(\bx)=\frac12 \left(\sum_{i=1}^d (x_i^2+a)^{\frac{p}{2}}\right)^{2/p}$ where $a>0$. Observe that $\psi_a$ is continuously twice differentiable for $a>0$.

Define the auxiliary functions $f:\R \to \R$ convex and twice differentiable and $h:\R^d \to \R$ differentiable, and define $\phi(\bx)=f(h(\bx))$.
So, we have
\[
\nabla^2 \phi(\bx)
=f''(h(\bx))\nabla h(\bx) \nabla h(\bx)^\top + f'(h(\bx)) \nabla^2 h(\bx)
\succeq f'(h(\bx)) \nabla^2 h(\bx),
\]
where in the inequality we used the fact that $f$ is convex, so $f''(x)\geq 0$ for all $x \in \R$.

Now, let's specialize the above result to $h(\bx)=\sum_{i=1}^d (x_i^2+a)^{\frac{p}{2}}$ and $f(x)=\frac12 x^{\frac{2}{p}}$. So, $\nabla^2 h(\bx)$ is a diagonal matrix and the element $i$ on the diagonal is
\begin{align*}
p(x_i^2+a)^{\frac{p}{2}-1} + p(p-2) x_i^2 (x_i^2+a)^{\frac{p}{2}-2}
&= p (x_i^2+a)^{\frac{p}{2}-1} (1+(p-2) x_i^2 (x_i^2+a)^{-1})\\
&\geq p (x_i^2+a)^{\frac{p}{2}-1} (1+(p-2))\\
&= p (p-1) (x_i^2+a)^{\frac{p}{2}-1},
\end{align*}
where in the inequality we have used the fact that $a>0$ and $p\leq 2$.
Hence, we have
\[
\nabla^2 \psi_a(\bx)
\succeq \frac1p \left(\sum_{i=1}^d (x_i^2+a)^{\frac{p}{2}}\right)^{\frac{2}{p}-1} p(p-1)\diag{(x^2_1+a)^{\frac{p}{2}-1}, \dots, (x_d^2+a)^{\frac{p}{2}-1}}~.
\]
Hence, denoting $w_i=(x^2_i+a)^{(2-p)\frac{p}{4}}>0$, we have
\begin{align*}
\langle \nabla^2 \psi_a(\bx)\by, \by\rangle
&\geq (p-1) \left(\sum_{i=1}^d (x_i^2+a)^{\frac{p}{2}}\right)^{\frac{2-p}{p}}\sum_{i=1}^d  (x^2_i+a)^{\frac{p}{2}-1} y_i^2\\
&=(p-1) \left[\left(\sum_{i=1}^d (x_i^2+a)^{\frac{p}{2}} \right)^{\frac{2-p}{2}} \left(\sum_{i=1}^d  (x^2_i+a)^{\frac{p}{2}-1} y_i^2\right)^\frac{p}{2}\right]^\frac{2}{p} \\
&=(p-1) \left[\left(\sum_{i=1}^d w_i^\frac{2}{2-p} \right)^{\frac{2-p}{2}} \left(\sum_{i=1}^d \frac{y_i^2}{w_i^\frac{2}{p}}\right)^\frac{p}{2}\right]^\frac{2}{p} \\
&\geq (p-1) \left(\sum_{i=1}^d w_i \frac{y_i^p}{w_i}\right)^{\frac{2}{p}}
= (p-1) \|\by\|^2_p,
\end{align*}
where we used H\"older inequality with dual norms $\|\cdot\|_\frac{2}{p}$ and $\|\cdot\|_\frac{2}{2-p}$.

Hence, we have that
\[
\psi_a(\by)
\geq \psi_a(\bx) + \langle \nabla \psi_a(\bx), \by-\bx\rangle + \frac{p-1}{2}\|\bx-\by\|^2_p, \quad \forall \bx,\by \in \R^d~.
\]
Now, given that $\psi_a$ and $\nabla \psi_a$ are continuous in $a$, taking the limit for $a\to0^+$, we get \eqref{eq:p_norm_strongly_convex_1}.
By Lemma~\ref{lemma:strong_convexity}, this implies the strong convexity of $\psi$.
\end{proof}

Hence, the regret bound will be
\[
\sum_{t=1}^T (\ell_t(\bx_t) - \ell_t(\bu))
\leq \frac{\|\bu\|_p^2}{2 \eta} + \frac{\eta}{2(p-1)}\sum_{t=1}^T \|\bg_t\|_q^2~.
\]

Setting $p=2$, we get the (unprojected) \ac{OSD}. However, we can set $p$ to achieve a logarithmic dependence in the dimension $d$ as in \ac{EG}. Let's assume again that $\|\bg_t\|_\infty\leq L_\infty$, so we have $
\sum_{t=1}^T \|\bg_t\|_q^2
\leq L_\infty^2 d^{2/q} T$.
Also, note that $\|\bu\|_p\leq \|\bu\|_1$, so we have an upper bound on the regret of
\[
\Regret_T(\bu) \leq \frac{\|\bu\|^2_1}{2 \eta} + \frac{L_\infty^2 d^{2/q} T \eta}{2(p-1)}, \quad \forall \bu \in \R^d~.
\]
Setting $\eta=\frac{\alpha\sqrt{p-1}}{L_\infty d^{1/q} \sqrt{T}}$, we get an upper bound on the regret of
\begin{align*}
\frac{1}{2}\left(\frac{\|\bu\|^2_1}{\alpha} + \alpha\right) L_\infty \sqrt{T} \frac{d^{1/q}}{\sqrt{p-1}}
&= \frac{1}{2}\left(\frac{\|\bu\|^2_1}{\alpha} + \alpha\right) L_\infty \sqrt{T} \sqrt{q-1}d^{1/q}\\
&\leq \frac{1}{2}\left(\frac{\|\bu\|^2_1}{\alpha} + \alpha\right) L_\infty \sqrt{T} \sqrt{q}d^{1/q}~.
\end{align*}
Assuming $d\geq 3$, the choice of $q$ that minimizes the last term is $q=2 \ln d$ that makes the term $\sqrt{q}d^{1/q}=\sqrt{2 e \ln d}$. Hence, we have regret bound of $\mathcal{O}(\sqrt{T \ln d})$ as $T \to \infty$.

Note that here the set $\mathcal{V}$ is the entire space, however we could still set $\mathcal{V}=\{\bx \in\R^d: x_i\geq0, \|\bx\|_1=1\}$. While this would allow us to get the same asymptotic bound of \ac{EG}, the update would not be in a closed form anymore.

\section{Application: Learning with Expert Advice}
\label{sec:lea}

Let's introduce a particular \ac{OCO} game called \textbf{\ac{LEA}}\index{learning with expert advice|(textbf}.

In this setting, on each round, we have $d$ experts who give us some advice, and we have to decide which expert we want to follow. After we made our choice, the losses associated with each expert are revealed, and we pay the loss associated with the expert we picked. The aim of the game is to minimize the losses we suffer compared to the cumulative losses of the best expert.
This is a general setting that allows us to model many interesting cases. For example, we have a number of different online learning algorithms, and we would like to choose the best among them.

Is this problem solvable? If we put ourselves in the adversarial setting, unfortunately, it cannot be solved!
Indeed, even with 2 experts, the adversary can force the algorithm to incur linear regret. Let's see how. In each round, we have to pick expert 1 or expert 2. In each round, the adversary can decide that the expert we pick has a loss of 1 and the other one has a loss of 0. This means that the cumulative loss of the algorithm over $T$ rounds is $T$. On the other hand, the best cumulative loss over experts 1 and 2 is less than $T/2$. This means that our regret, no matter what we do, can be as big as $T/2$.

The problem above is due to the fact that the adversary has too much power. One way to reduce its power is by using \emph{randomization}. We can allow the algorithm to be randomized \emph{and} force the adversary to decide the losses at time $t$ without knowing the outcome of the randomization of the algorithm at time $t$ (but it can depend on the past randomization). This is enough to make the problem solvable. Another view to look at it is that randomization makes the problem convex, allowing us to use any \ac{OCO} algorithm on it.

First, let's write the problem in the original formulation.
We set a discrete feasible set $\mathcal{V}=\{\be_i\}_{i=1}^d$, where $\be_i$ is the vector with all zeros but a $1$ in the coordinate $i$.
Our predictions and the competitor are from $\mathcal{V}$.
The losses are linear losses: $\ell_t(\bx) = \langle \bg_t, \bx\rangle$, for $t=1,\dots,T$.
The regret is
\begin{equation}
\label{eq:regret_lea}
\Regret_T(\be_i) = \sum_{t=1}^T \langle \bg_{t}, \bx_t\rangle - \sum_{t=1}^T \langle \bg_t, \be_i\rangle , \quad i=1, \dots, d~.
\end{equation}
The only thing that makes this problem non-convex is the feasibility set, which is clearly a non-convex one.

Let's now see how the randomization makes this problem convex.
Let's extend the feasible set to $\Delta^{d-1}:=\{\bx \in \R^d : x_i \geq 0, \|\bx\|_1=1\}$. Note that $\be_i \in \Delta^{d-1}$. For this problem, we can use an \ac{OCO} algorithm to minimize the regret
\[
\Regret_T(\bu)
= \sum_{t=1}^T \langle \bg_t, \bx_t\rangle - \sum_{t=1}^T \langle \bg_t, \bu\rangle, \quad \forall \bu \in \Delta^{d-1}~.
\]
Can we find a way to transform an upper bound on this regret to the one we care about in \eqref{eq:regret_lea}?
One way is the following one: on each time step, construct a random variable $A_t$ that is equal to $i$ with probability $x_{t,i}$ for $i=1,\dots,d$.
Then, select the expert according to the outcome of $A_t$. Now, using the law of total expectation, we have
\[
\E[g_{t,A_t}]
= \E_{A_1, \dots, A_{t-1}}[\E_{A_t}[g_{t,A_t}|A_1, \dots, A_{t-1}]]
= \E[\langle \bg_t,\bx_t\rangle],
\]
and
\[
\E\left[\sum_{t=1}^T (g_{t,A_t}-g_{t,i})\right]
= \E\left[\sum_{t=1}^T \langle \bg_{t}, \bx_t\rangle - \sum_{t=1}^T \langle \bg_t, \be_i\rangle\right], \quad \forall i \in \{1, \dots, d\}~.
\]
This means that we can minimize in expectation the non-convex regret with a randomized \ac{OCO} algorithm.
We can summarize this reasoning in Algorithm~\ref{alg:lea}.

\begin{algorithm}[t]
\caption{Learning with Expert Advice through Randomization}
\label{alg:lea}
\begin{algorithmic}[1]
{
    \REQUIRE{$\bx_1 \in \Delta^{d-1}$}
    \FOR{$t=1$ {\bfseries to} $T$}
    \STATE{Draw $A_t$ according to $\Pr\{A_t=i\}=x_{t,i}$}
    \STATE{Select expert $A_t$}
    \STATE{Observe all the experts' losses $\bg_t \in \R^d$ and pay the loss $g_{t,A_t}$}
    \STATE{Get $\bx_{t+1}$ from an \ac{OCO} algorithm with feasible set $\Delta^{d-1}$}
    \ENDFOR
}
\end{algorithmic}
\end{algorithm}

For example, assume that $\|\bg_t\|_\infty \leq L_\infty$ for all $t$.
Then, using the \ac{EG}\index{Exponentiated Gradient algorithm!OMD version} algorithm from Section~\ref{sec:omd_eg}, we obtain the following update rule
\[
x_{t+1,j} = \frac{x_{t,j}\exp(-\eta g_{t,j})}{\sum_{i=1}^d x_{t,i} \exp(-\eta g_{t,i})}, \quad j=1,\dots,d,
\]
setting $\bx_1=[1/d, \dots, 1/d]$ and $\eta = \frac{\sqrt{2\ln d}}{L_\infty \sqrt{T}}$. For such an algorithm, the regret will be
\[
\E[\Regret_T(\be_i)] \leq L_\infty \sqrt{2 T \ln d}, \quad \forall i\in \{1, \dots, d\}~.
\]
It is worth stressing the importance of the result just obtained: we can design an algorithm that in expectation, is close to the best expert in a set, \emph{paying only a logarithmic penalty in the size of the set}. The $p$-norm Algorithm in Section~\ref{sec:omd_pnorm} would give a similar guarantee.

Later, in Section~\ref{sec:shifted_kt}, we will see algorithms that achieve the even better regret guarantee of $\mathcal{O}(\sqrt{T\cdot \KL(\bu;\bx_1)})$ as $T\to\infty$, for any $\bu$ in the probability simplex. You should be able to convince yourself that no setting of $\eta$ in \ac{EG} allows us to achieve such a regret guarantee. Indeed, these algorithms will be based on a very different strategy.
\index{learning with expert advice|)textbf}

\index{Multi-Scale Expert algorithm|(textbf}
\section{Example of OMD: Multi-Scale Expert Algorithm}
\label{sec:multiscale}

We now consider an application of \ac{OMD} to the problem of prediction with multi-scale expert advice. In this setting, at each time step $t$, an algorithm selects a probability distribution $\bx_t$ over a set of $d$ experts (or actions). After observing a loss vector $\bg_t \in \R^d$, the algorithm incurs a loss of $\langle \bg_t, \bx_t \rangle$. The goal is to minimize the regret with respect to the best single expert in hindsight. However, unlike in the usual learning with expert advice setting, here the losses associated with different experts can have vastly different scales. For example, one expert might correspond to a low-risk, low-reward strategy with losses in $[-0.1, 0.1]$, while another might be a high-risk, high-reward strategy with losses in $[-100, 100]$.
Thus, we would like to prove a regret guarantee that depends on the range of the best coordinate.

One might be tempted to use the \ac{EG} algorithm, but it has a regret bound that scales with the largest range of the loss coordinates, so it is not suitable for this problem. We need a different algorithm.

Here, we assume that we have prior knowledge of the approximate scale of losses for each expert. Specifically, for each expert $i \in \{1, \dots, d\}$, we know a value $c_i > 0$ such that the loss $g_{t,i}$ is guaranteed to be in the range $[-c_i, c_i]$ for all $t$. The goal is to leverage this scale information to achieve better regret bounds.

\begin{algorithm}[t]
\caption{OMD for Multi-Scale Experts}
\label{alg:MSE}
\begin{algorithmic}[1]
{
    \REQUIRE{Per-expert scales $c_1, \dots, c_d > 0$, learning rate $\eta>0$, $\bx_1 \in \Delta^{d-1}$}
    \STATE{Set $\psi(\bx)=\sum_{i=1}^d c_i x_i \ln x_i$}
    \FOR{$t=1$ {\bfseries to} $T$}
    \STATE{Output $\bx_t \in \Delta^{d-1}$}
    \STATE{Receive $\bg_t \in \R^d$ where $|g_{t,i}| \leq c_i$ for $i=1,\dots,d$}
    \STATE{Pay the loss $\langle \bg_t, \bx_t\rangle$}
    \STATE{Set $\tilde{g}_{t,i}=g_{t,i}+\eta \frac{g_{t,i}^2}{c_i}$ for $i=1, \dots, d$}
    \STATE{Set $\bx_{t+1} = \argmin_{\bx \in \Delta^{d-1}} \ B_{\psi}(\bx;\bx_t) + \eta \langle \tilde{\bg}_t,\bx\rangle$}
    \ENDFOR
}
\end{algorithmic}
\end{algorithm}

Here, we derive an algorithm for this setting as an instance of \ac{OMD} with a specific choice of surrogate losses.
In particular, we run \ac{OMD} with distance generating function $\psi(\bx)=\sum_{i=1}^d c_i x_i \ln x_i$ on the shifted losses $\tilde{g}_{t,i}=g_{t,i}+\eta\frac{g_{t,i}^2}{c_i}$, see Algorithm~\ref{alg:MSE}.

\begin{theorem}
\label{thm:multi-scale}
Let $c_1, \dots, c_d>0$ such that $|g_{t,i}|\leq c_i$ for all $t=1, \dots T$ and $i=1, \dots d$. In Algorithm~\ref{alg:MSE}, set $\eta\leq \frac{1}{5}$. Then, we have
\[
\sum_{t=1}^T \langle \bg_t, \bx_t- \bu\rangle
\leq \frac{B_\psi(\bu;\bx_1)}{\eta} + \eta \sum_{t=1}^T \sum_{i=1}^d \frac{u_{i} g_{t,i}^2}{c_i}, \quad \forall \bu \in \Delta^{d-1}~.
\]
\end{theorem}
\begin{proof}
First, observe that
\begin{equation}
|\tilde{g}_{t,i}|
\leq |g_{t,i}| \left(1+ \eta \right)~.
\label{eq:proof_multi-scale_eq1}
\end{equation}

Using the local norm regret upper bound for OMD in Lemma~\ref{lemma:omd_local_norms}, we obtain
\[
\sum_{t=1}^T \langle \tilde{\bg}_t, \bx_t- \bu\rangle
\leq \frac{B_\psi(\bu;\bx_1)}{\eta} + \frac{\eta}{2} \sum_{t=1}^T \sum_{i=1}^d \frac{z_{t,i} \tilde{g}_{t,i}^2}{c_i}, \quad \forall \bu \in \Delta^{d-1},
\]
where $\bz_{t}$ is between $\bx_t$ and $\bx'_{t+1} = \argmin_{\bx \in \R^d_{\geq 0}} \ B_\psi(\bx; \bx_t)+ \eta\langle \tilde{\bg}_t, \bx\rangle$.
Using \eqref{eq:proof_multi-scale_eq1}, we have that
\[
x'_{t+1,i}
= x_{t,i} \exp\left(-\eta\frac{\tilde{g}_{t,i}}{c_{i}}\right)
\leq x_{t,i} \exp\left(\eta+\eta^2\right)~.
\]
So, we can upper bound each $z_{t,i}$ with $\exp\left(\eta+\eta^2\right) x_{t,i}$.

Using again~\eqref{eq:proof_multi-scale_eq1}, we have that the right-hand side of the regret upper bound becomes
\[
\frac{B_\psi(\bu;\bx_1)}{\eta} + \frac12 \exp\left(\eta+\eta^2\right)\left(1+\eta\right)^2 \eta \sum_{t=1}^T \sum_{i=1}^d \frac{x_{t,i} g_{t,i}^2}{c_i}~.
\]

Observe that $\eta\leq \frac{1}{5}$, so $\frac{1}{2} \exp\left(\eta+\eta^2\right)\left(1+\eta\right)^2 \leq 1$.
Hence, for all $\bu \in \Delta^{d-1}$, we obtain
\begin{align*}
\sum_{t=1}^T \langle \bg_t, \bx_t- \bu\rangle + \sum_{t=1}^T \sum_{i=1}^d \eta \frac{g_{t,i}^2}{c_i} (x_{t,i}-u_i)
&= \sum_{t=1}^T \langle \tilde{\bg}_t, \bx_t- \bu\rangle\\
&\leq \frac{B_\psi(\bu;\bx_1)}{\eta} + \eta \sum_{t=1}^T \sum_{i=1}^d \frac{x_{t,i} g_{t,i}^2}{c_i}~.
\end{align*}
Simplifying, we have the stated result.
\end{proof}

Let's now see how we can choose the prior $\pi$ and the initial point $\bx_1 \in \Delta^{d-1}$.
First, let's calculate the value of the Bregman divergence:
\begin{align*}
B_\psi(\bu;\bx_1)
&=\sum_{i=1}^d c_i u_i \ln u_i-\sum_{i=1}^d c_i x_{1,i} \ln x_{1,i} - \sum_{i=1}^d c_i (\ln x_{1,i}+1)(u_i-x_{1,i})\\
&=\sum_{i=1}^d c_i u_i \ln \frac{u_i}{x_{1,i}} - \sum_{i=1}^d c_i u_i + \sum_{i=1}^d c_i x_{1,i}~.
\end{align*}

Define $c_{\min}=\min_i \ c_i$, $c_{\max}=\max_i \ c_i$, and $i_{\min}$ to be any index in $\argmin_i \ c_i$.
If we set $\bu=\be_j$, we have
\[
B_\psi(\bu;\bx_1)
= c_j \ln \frac{1}{x_{1,j}} + \sum_{i=1}^d c_i x_{1,i} - c_j~.
\]

Now, set $x_{1,i} = \frac{c_{\min}}{c_i} \pi_i$ for $i\neq i_{\min}$ and $x_{1,i_{\min}}=1-\sum_{i\neq i_{\min}} x_{1,i}$.
Also, set $\pi_i = \frac{c_i}{\sum_{j=1}^d c_j}$.
In this way, for $j\neq i_{\min}$, we have
\begin{align*}
B_\psi(\bu;\bx_1)
&= c_j \ln \frac{1}{x_{1,j}} + \sum_{i=1}^d c_i x_{1,i} - c_j\\
&= c_j \ln \frac{c_j}{ c_{\min} \pi_j} + c_{\min} (1-\pi_{i_{\min}}) + c_{\min} x_{1,i_{\min}}- c_j\\
&\leq c_j \ln \frac{c_j}{ c_{\min} \pi_j} + 2 c_{\min}  - c_j\\
&\leq c_j \ln \frac{c_j}{ c_{\min} \pi_j} + c_j
= c_j \ln \frac{\sum_{k=1}^d c_k}{ c_{\min} } + c_j~.
\end{align*}
On the other hand, for $j=i_{\min}$, we have
\[
x_{1,i_{\min}}
= 1-\sum_{i\neq i_{\min}} x_{1,i}
\geq 1-\sum_{i\neq i_{\min}} \pi_{i}
= \pi_{i_{\min}}
= \frac{ c_{\min} }{\sum_{k=1}^d c_k}~.
\]
Hence, the previous bound holds for all $j$.
In this case, choosing $\eta=\frac{1}{5+\frac{\sqrt{T}}{\sqrt{1+\ln \frac{\sum_{i=1}^d c_i}{c_{\min}}}}}\leq \frac{1}{5}$, for all $j=1, \dots, d$, we have
\[
\sum_{t=1}^T \langle \bg_t, \bx_t- \be_j\rangle
\leq 2 c_j \sqrt{T \left(1+\ln \frac{\sum_{i=1}^d c_{i}}{ c_{\min} }\right)} + 5 c_j \left(1+\ln \frac{\sum_{i=1}^d c_{i}}{ c_{\min} }\right)~.
\]
Note that if $c_1=c_2=\dots=c_d$ then $\ln \frac{\sum_{i=1}^d c_{i}}{ c_{\min} }=\ln d$, so we recover the bound of the \ac{EG} algorithm, up to constant factors. On the other hand, if the scales $c_i$ are different, then the bounds depend on the range of the expert with whom we compete and only logarithmically on the ratio between the largest and smallest scale.

\section{Application: Combining Online Algorithms to Adapt to the Learning Rate}
\label{sec:easy_metagrad}

We now show an interesting application of the multi-scale expert algorithm.

In our analysis of \ac{OSD}, we have seen that the choice of the learning rate $\eta_t$ is critical for performance. For instance, with a constant learning rate $\eta$, the optimal choice that minimizes the regret bound is $\eta^\star = \frac{\|\bu-\bx_1\|_2}{\sqrt{\sum_{t=1}^T \|\bg_t\|_2^2}}$, which unfortunately depends on the competitor $\bu$ and the entire sequence of future gradients. One might be tempted to use a grid of learning rates and select the best one in hindsight, but unfortunately, this is not a valid online learning procedure.

In this section, we demonstrate how to use the \ac{LEA}\index{learning with expert advice} framework to design a meta-algorithm that automatically adapts to the best learning rate from a given set, paying only a small price in the regret. The core idea is to treat each instance of an online learning algorithm with a fixed learning rate as an ``expert''. We then use a multi-scale expert algorithm to combine the predictions of these experts. The resulting ensemble algorithm will have a regret guarantee that is close to the regret of the best expert---and thus the best learning rate---in hindsight.

Let us consider running $N$ parallel instances of the \ac{OSD} algorithm, each with a different learning rate $\eta^{(i)}$ for $i=1, \dots, N$. At each round $t$, each \ac{OSD} instance $i$ produces a prediction $\bx_t^{(i)}$. We can view these $N$ predictions as advice from $N$ experts. Our goal is to combine them into a single prediction $\bx_t$ that performs nearly as well as the best prediction $\bx_t^{(i^\star)}$ from the best \ac{OSD} instance $i^\star$.

A straightforward approach would be to compute the loss $\ell_t(\bx_t^{(i)})$ for each expert $i$ and use this as the loss vector for the multi-scale expert algorithm.
However, this would require computing $N$ separate subgradients $\bg_t^{(i)} \in \partial\ell_t(\bx_t^{(i)})$ at each round, which can be computationally expensive.

To create an efficient algorithm that requires only one subgradient evaluation per round, we can use the linearization technique. The controller algorithm forms its combined prediction $\bx_t = \sum_{i=1}^N p_{t,i} \bx_t^{(i)}$. We then receive a single subgradient $\bg_t \in \partial\ell_t(\bx_t)$ at this combined point. This single subgradient is then used to define a linear surrogate loss, $\tilde{\ell}_t(\bx) = \langle \bg_t, \bx \rangle$, which is passed to all expert algorithms and to the controller algorithm.
Because all experts receive the same linear loss function, they all use the same subgradient $\bg_t$ for their updates. The loss for the $i$-th expert, used by the multi-scale expert algorithm, is simply $\langle \bg_t, \bx_t^{(i)} \rangle$. This procedure is summarized in Algorithm~\ref{alg:ms_osd}.

\begin{algorithm}[t]
\caption{Combining OSD Instances with a Multi-Scale Expert Algorithm}
\label{alg:ms_osd}
\begin{algorithmic}[1]
{
    \REQUIRE{Number of OSD instances $N$, number of rounds $T$, sequence of learning rates $\eta_t^{(1)}, \dots, \eta_t^{(N)}$, initial radius $R$}
    \STATE{Initialize $N$ copies of OSD: $\bx_1^{(i)} = \boldsymbol{0}$ for $i=1, \dots, N$}
    \STATE{Set $\mathcal{V}^{(i)} = \{\bx \in \R^d: \|\bx\|_2\leq R 2^{i-1}\}$}
    \STATE{Set $c_i = L R 2^{i-1}$ for $i=1, \dots, N$}
    \STATE{Set the prior in the Multi-scale Expert algorithm: $\pi_i=\frac{c_i}{\sum_{j=1}^N c_j}$ for $i=1, \dots, N$}
    \STATE{Initialize the Multi-scale Expert algorithm: $p_{1,i} = \frac{c_1}{c_i} \pi_i$ for $i=2, \dots, N$ and $p_{1,1}=1-\sum_{j=2}^N p_{1,j}$}
    \FOR{$t=1$ {\bfseries to} $T$}
    \STATE{Get $\bp_t \in \Delta^{N-1}$ from the Multi-scale Expert algorithm}
    \STATE{Output combined prediction $\bx_t = \sum_{i=1}^N p_{t,i} \bx_t^{(i)}$}
    \STATE{Pay the loss $\ell_t(\bx_t)$, where $\ell_t$ is subdifferentiable on $\R^d$}
    \STATE{Set $\bg_t \in \partial \ell_t(\bx_t)$}
    \STATE{Define loss vector: $z_{t,i} = \langle \bg_t, \bx_t^{(i)} \rangle$ for $i=1, \dots, N$}
    \STATE{Pass $\bz_t$ to the Multi-scale Expert algorithm}
    \STATE{Update each OSD instance: $\bx_{t+1}^{(i)} = \Pi_{\mathcal{V}^{(i)}}(\bx_t^{(i)} - \eta^{(i)}_t \bg_t)$ for $i=1, \dots, N$}
    \ENDFOR
}
\end{algorithmic}
\end{algorithm}

We can now prove a regret bound for this ensemble algorithm.
\begin{theorem}
Let $R>0$ and fix $T$, the number of rounds. With the notation in Algorithm~\ref{alg:ms_osd}, assume that the losses $\ell_t$ are $L$-Lipschitz with respect to $\|\cdot\|_2$. Let $\eta^{(i)}_t=\frac{R 2^{i-1}}{L \sqrt{T}}$ be the set of learning rates for the \ac{OSD} experts, $N = \lceil \log_2\sqrt{T}\rceil+1$, and the learning rate of the multi-scale expert algorithm in Algorithm~\ref{alg:MSE} to $\alpha=\frac{\sqrt{1+\ln(2^N-1)}}{5\sqrt{1+\ln(2^N-1)}+\sqrt{T}}$. Then, for all $\bu$ such that $\|\bu\|_2 \leq R \sqrt{T}$, Algorithm~\ref{alg:ms_osd} satisfies
\[
\sum_{t=1}^T (\ell_t(\bx_t) - \ell_t(\bu))
\leq \max(2\|\bu\|_2, R) L \left[\sqrt{T}\left(1 + 2\sqrt{3+\ln\sqrt{T}}\right)+ 15+5\ln \sqrt{T}\right].
\]
\end{theorem}
\begin{proof}
The choice of $N$ ensures that for any $\bu$ such that $\|\bu\|_2\leq R \sqrt{T}$ there exists an expert $j_{\bu}\in\{1,\dots,N\}$ such that $\|\bu\|_2 \leq R2^{j_{\bu}-1} \leq \max(2\|\bu\|_2,R)$.
Indeed, if $\|\bu\|_2\leq R$ we take $j_{\bu}=1$, while otherwise we take $j_{\bu}=\left\lceil\log_2 \frac{\|\bu\|_2}{R}\right\rceil+1$.

By the definition of subgradient, the regret of the ensemble algorithm can be bounded by the regret on the linearized losses:
\[
\sum_{t=1}^T (\ell_t(\bx_t) - \ell_t(\bu))
\leq \sum_{t=1}^T \langle \bg_t, \bx_t - \bu \rangle~.
\]
We can now decompose the regret as
\begin{equation}
\label{eq:multiscale_combine_eq1}
\sum_{t=1}^T \langle \bg_t, \bx_t - \bu \rangle
= \sum_{t=1}^T \langle \bg_t, \bx_t - \bx_t^{(j_{\bu})} \rangle + \sum_{t=1}^T \langle \bg_t, \bx_t^{(j_{\bu})} - \bu \rangle~.
\end{equation}

The second sum in r.h.s. of \eqref{eq:multiscale_combine_eq1} is the regret of the best \ac{OSD} expert on the sequence of linear losses $\langle \bg_t, \cdot \rangle$. From the regret bound of \ac{OSD} with constant learning rate $\eta^{(j_{\bu})}=\frac{R 2^{j_{\bu}-1}}{L \sqrt{T}}$ and initial point equal to zero, we have
\[
\sum_{t=1}^T \langle \bg_t, \bx_t^{(j_{\bu})} - \bu \rangle
\leq R 2^{j_{\bu}-1} L \sqrt{T}
\leq \max(2\|\bu\|_2, R) L \sqrt{T}~.
\]
For the first sum in the r.h.s. of \eqref{eq:multiscale_combine_eq1}, we have
\begin{align*}
\sum_{t=1}^T \langle \bg_t, \bx_t - \bx_t^{(j_{\bu})} \rangle
&= \sum_{t=1}^T \left\langle \bg_t, \sum_{i=1}^N p_{t,i} \bx_t^{(i)} - \bx_t^{(j_{\bu})} \right\rangle\\
&= \sum_{t=1}^T \left(\sum_{i=1}^N p_{t,i} \langle \bg_t, \bx_t^{(i)} \rangle - \langle \bg_t, \bx_t^{(j_{\bu})} \rangle \right)~.
\end{align*}
This is the regret of the multi-scale expert algorithm against the expert $j_{\bu}$ on the sequence of loss vectors $\bz_t$ where $|z_{t,i}| = |\langle \bg_t, \bx_t^{(i)} \rangle| \leq \|\bg_t\|_2 \|\bx_t^{(i)}\|_2 \leq L 2^{i-1} R$ for all $i=1, \dots, N$, where we used that $\ell_t$ is $L$-Lipschitz and $\mathcal{V}^{(i)}$ has radius $R 2^{i-1}$. This is exactly the setting we used for the scales $c_i$.
So, using the regret bound for the multi-scale expert algorithm, we get
\[
\sum_{t=1}^T \langle \bg_t, \bx_t - \bx_t^{(j_{\bu})} \rangle
\leq 2 L R 2^{j_{\bu}-1} \sqrt{T (1+\ln (2^{N}-1))} + 5 L R 2^{j_{\bu}-1} \left(1+\ln (2^{N}-1)\right)~.
\]
Using the fact that $R 2^{j_{\bu}-1}\leq \max(2\|\bu\|_2, R)$, the expression of $N$, and combining the bounds for both terms completes the proof.
\end{proof}

This result shows that by combining \ac{OSD} instances, we can achieve a regret that is close to the one obtained by the best learning rate in the chosen set.
The constraint on $\bu$ is motivated by the fact that for larger competitors, the regret would not be sublinear in $T$, even with prior knowledge of the norm of the competitor. Hence, it is impossible to compete with $\bu$ that are too large.
The additional multiplicative price for the adaptivity is $\mathcal{O}(\sqrt{\ln T})$, which we know to be necessary by the lower bound for unconstrained \ac{OCO} in Section~\ref{sec:lower_unconstrained_olo}.
\index{Multi-Scale Expert algorithm|)textbf}

This technique provides a principled method for automating the selection of learning rates in an online fashion. However, its interest is mostly theoretical: it shows that such adaptivity to $\bu$ is possible, but one should not expect a strong empirical performance. In Chapter~\ref{ch:parameterfree}, we will see how to obtain learning-rate-free algorithms for unbounded domains, without the need to combine several base online learners.

\section{Optimistic Online Mirror Descent}
\label{sec:optimistic_omd}

\begin{algorithm}[h]
\caption{Optimistic Online Mirror Descent (Optimistic OMD)}
\label{alg:optimistic_omd}
\begin{algorithmic}[1]
{
    \REQUIRE{Non-empty closed convex $\mathcal{V} \subset \mathcal{X}\subseteq \R^d$, $\psi: \mathcal{X} \to \R$ strictly convex and differentiable on $\interior \mathcal{X}$, $\bx_1 \in \interior \mathcal{X} \cap \mathcal{V}$, $\eta_1,\dots,\eta_T>0$}
    \STATE{$\tilde{\bg}_1=\boldsymbol{0}$}
    \FOR{$t=1$ {\bfseries to} $T$}
    \STATE{Output $\bx_t \in \mathcal{V}$}
    \STATE{Pay the loss $\ell_t(\bx_t)$, where $\ell_t$ is subdifferentiable on $\mathcal{V}$}
    \STATE{Set $\bg_t \in \partial \ell_t(\bx_t)$}
    \IF{$t < T$}
    \STATE{Get hint $\tilde{\bg}_{t+1} \in \R^d$ on the subgradient}
    \ELSE
    \STATE{Get hint $\tilde{\bg}_{T+1} =\boldsymbol{0}$}
    \ENDIF
    \STATE{$\bx_{t+1} \in \argmin_{\bx \in \mathcal{V}} \ \eta_t \langle \bg_t-\tilde{\bg}_t+\tilde{\bg}_{t+1}, \bx\rangle + B_\psi(\bx; \bx_{t})$}
    \ENDFOR
}
\end{algorithmic}
\end{algorithm}

Till now, we have mainly considered the adversarial model as our model of the environment. However, the world is never completely adversarial. So, we might be tempted to model the environment in some way, but that would leave our algorithm vulnerable to attacks. An alternative is to consider the data as generated by some \emph{predictable process plus an adversarial signal}. In this view, it might be beneficial to try to model the predictable part, without compromising the robustness to the adversarial signal.

In this section, we will explore this possibility through a particular version of \ac{OMD}, where we \emph{predict} the next gradient. In very intuitive terms, if our predicted gradient is correct, we can expect the regret to decrease. However, if our prediction is wrong, we still want to recover the worst-case guarantee.
Such an algorithm is called \textbf{Optimistic \ac{OMD}}.\index{Online Mirror Descent algorithm!optimistic|(textbf}

The pseudo-code of Optimistic \ac{OMD} is summarized in Algorithm~\ref{alg:optimistic_omd}. At round $t$, the algorithm receives a hint $\tilde{\bg}_{t+1}$ on the next subgradient $\bg_{t+1}$ and uses it to construct the update. At the same time, you have to remove the hint you used at the previous time step, $\tilde{\bg}_{t}$. Note that for the sake of the analysis, it does not matter how the prediction is generated. It can even be generated by another online learning procedure!

To gain some intuition on why this update makes sense, consider the case that $\psi(\bx)=\frac{1}{2}\|\bx\|_2^2$, $\eta_t=\eta$, and $\mathcal{V}=\R^d$. In this case, $\bx_{t+1} = \bx_{t} + \eta \tilde{\bg}_{t} - \eta \bg_t - \eta \tilde{\bg}_{t+1}$. Unrolling the update, we get $\bx_{t+1} = \bx_1 - \eta (\tilde{\bg}_{t+1}+\sum_{i=1}^t \bg_i)$.
Without hints, that is in plain \ac{OMD}, under the same assumptions the unrolled update would be $\bx_{t+1} = \bx_1 - \eta \sum_{i=1}^t \bg_i$ and $\bx_{t+2} = \bx_1 - \eta \sum_{i=1}^{t+1} \bg_i$. Hence, $\tilde{\bg}_{t+1}$ is a proxy for the next (unknown) subgradient $\bg_{t+1}$.

One might be tempted to multiply $\tilde{\bg}_t$ by $\eta_{t-1}$, because in the previous iteration we used the learning rate $\eta_{t-1}$. However, the \ac{OMD} proof reveals that the correct approach is to think of the learning rate as attached to the Bregman divergence\index{Bregman divergence} rather than to the subgradients.

One might also be tempted to find a way to study this algorithm with a special proof. However, the one-step lemma we proved for \ac{OMD} is essentially tight: we only used two inequalities, one to deal with the set $\mathcal{V}$ and the other one to linearize the losses. But, both steps can be made tight, considering $\mathcal{V}=\R^d$ and linear losses. Hence, if the update is just \ac{OMD} with a different sequence of subgradients, the proof \emph{must} follow from that of \ac{OMD} with a different set of subgradients.

Observe that setting $\tilde{\bg}_1=\boldsymbol{0}$ is not a limitation because setting it to any other value would be equivalent to changing the initial point $\bx_1$. In the same way, setting $\tilde{\bg}_{T+1}=\boldsymbol{0}$ does not change the behaviour of the algorithm in rounds $t=1, \dots, T$, but it simplifies the analysis.

\begin{theorem}
\label{thm:oomd}
With the notation and initialization in Algorithm~\ref{alg:optimistic_omd},
assume $\eta_{t+1}\leq \eta_{t}, \ t=1, \dots, T-1$.
Define $\delta_t=\max(\|\bg_t-\tilde{\bg}_t\|_\star-\|\bg_t-\tilde{\bg}_t+\tilde{\bg}_{t+1}\|_\star,0)$.
Then, under the assumptions of Lemma~\ref{lemma:omd_one_step}, for all $\bu \in \mathcal{V}$, the following regret bounds hold
\begin{align*}
&\sum_{t=1}^T (\ell_t(\bx_t)- \ell_t(\bu))\\
&\leq \frac{ \max_{1\leq t \leq T} \ B_\psi(\bu;\bx_t)}{\eta_{T}} + \sum_{t=1}^T \left(\langle \bg_t-\tilde{\bg}_t,\bx_{t}-\bx_{t+1}\rangle- \frac{1}{\eta_t} B_\psi(\bx_{t+1};\bx_t)\right).
\end{align*}
Moreover, if $\eta_t$ is constant, i.e., $\eta_t=\eta \ \forall t=1,\dots,T$, we have
\begin{align*}
\sum_{t=1}^T (\ell_t(\bx_t)- \ell_t(\bu))
\leq \frac{B_\psi(\bu;\bx_1)}{\eta} + \sum_{t=1}^T \left(\langle \bg_t-\tilde{\bg}_t,\bx_{t}-\bx_{t+1}\rangle- \frac{1}{\eta} B_\psi(\bx_{t+1};\bx_t)\right).
\end{align*}
Moreover, in both cases for all $t$ we have
\begin{align*}
&\langle \bg_t-\tilde{\bg}_t,\bx_{t}-\bx_{t+1}\rangle- \frac{1}{\eta_t} B_\psi(\bx_{t+1};\bx_t)\\
&\quad\leq \frac{\eta_t (\|\bg_t-\tilde{\bg}_t\|^2_\star-\delta^2_t)}{2\lambda}\\
&\quad\leq \frac{\eta_t \|\bg_t-\tilde{\bg}_t\|_\star\min(\|\bg_t-\tilde{\bg}_t\|_\star, 2 \|\bg_t-\tilde{\bg}_t+\tilde{\bg}_{t+1}\|_\star)}{2\lambda}~.
\end{align*}
\end{theorem}

To prove this theorem, we will use the following technical lemmas.
\begin{lemma}
\label{lemma:diff_iterate_omd}
Assume $\psi:\mathcal{X}\to \R$ to be $\lambda$-strongly convex\index{function!strongly convex} with respect to $\|\cdot\|$ and $\bu \in \mathcal{V} \cap \interior \mathcal{X}$. Let $\mathcal{V}\subseteq \mathcal{X}$ be a non-empty convex closed set.
Assume that $\bx'=\argmin_{\bx \in \mathcal{V}} \ B_\psi(\bx;\bu) + \eta\langle \bg,\bx\rangle$ exists and belongs to $\interior\mathcal{X}$.
Then, we have
\[
\|\bx'-\bu\|
\leq \frac{\eta}{\lambda} \|\bg\|_\star~.
\]
\end{lemma}
\begin{proof}
Observe that from the optimality condition of Theorem~\ref{thm:constr_opt_condition} for $\bx'$, we have
\begin{equation}
\label{eq:lemma:diff_iterate_omd_eq1}
\langle \eta \bg+\nabla \psi(\bx')-\nabla \psi(\bu), \bv-\bx'\rangle\geq 0, \quad \forall \bv \in \mathcal{V}~.
\end{equation}
Hence, we have
\begin{align*}
\lambda \|\bx'-\bu\|^2
&\leq B_\psi(\bu;\bx')+B_\psi(\bx';\bu)
= \langle \nabla \psi(\bx') - \nabla \psi(\bu), \bx'-\bu\rangle
\leq \langle \eta \bg, \bx' - \bu\rangle\\
&\leq \eta \|\bg\|_\star \|\bx'-\bu\|~.
\end{align*}
where in the first inequality we used the strong convexity of $\psi$, \eqref{eq:lemma:diff_iterate_omd_eq1} in the second one, and the definition of dual norms in the last one.
Reordering, we have the stated bound.
\end{proof}

\begin{lemma}
\label{lemma:constrained_bound}
Let $\bg \in \R^d$ and $a,c>0$. Then, we have
\begin{align*}
\sup_{\bv \in \R^d: \|\bv\|\leq c/a} \ \langle \bg,\bv\rangle - \frac{a}{2}\|\bv\|^2
= \frac{\|\bg\|^2_\star-\max(\|\bg\|_\star-c,0)^2}{2 a }
\leq \frac{\|\bg\|_\star}{2 a } \min\left(\|\bg\|_\star, 2 c \right).
\end{align*}
\end{lemma}
\begin{proof}
\begin{align*}
\sup_{\bv \in \R^d: \|\bv\|\leq c/a} \ \langle \bg,\bv\rangle - \frac{a}{2}\|\bv\|^2
&=\sup_{0\leq x\leq c/a} \ \sup_{\bv \in \R^d: \|\bv\|= x} \ \langle \bg,\bv\rangle - \frac{a}{2}\|\bv\|^2\\
&=\sup_{0\leq x\leq c/a} \  x\|\bg\|_\star - \frac{a}{2}x^2\\
&=\frac{1}{a} \min(\|\bg\|_\star,c)\left[\|\bg\|_\star-\frac{1}{2}\min(\|\bg\|_\star,c)\right]\\
&=\frac{1}{2a} \left[\|\bg\|^2_\star-\max(\|\bg\|_\star-c,0)^2\right]~.
\end{align*}
The second inequality is obtained by lower-bounding the max with its two possible values and overapproximating.
\end{proof}


We can now prove the theorem.
\begin{proof}[Proof of Theorem~\ref{thm:oomd}]
We use Lemma~\ref{lemma:omd_one_step} with $\bg_t \to \bg_t - \tilde{\bg}_t + \tilde{\bg}_{t+1}$, that gives us that $\bx_{t+1}$ exists, and it is in $\interior \mathcal{X}$.
Moreover, we also have
\begin{align*}
\langle &\bg_t - \tilde{\bg}_t + \tilde{\bg}_{t+1}, \bx_t - \bu\rangle\\
&\leq \frac{1}{\eta_t}\left(B_\psi(\bu;\bx_t) - B_\psi(\bu;\bx_{t+1}) - B_\psi(\bx_{t+1};\bx_t)\right) + \langle \bg_t - \tilde{\bg}_t + \tilde{\bg}_{t+1}, \bx_t - \bx_{t+1} \rangle~.
\end{align*}
Summing over $t=1,\dots,T$ the l.h.s., we obtain
\begin{align*}
\sum_{t=1}^T &\langle \bg_t - \tilde{\bg}_t + \tilde{\bg}_{t+1}, \bx_t - \bu\rangle\\
&= \sum_{t=1}^T \langle \bg_t, \bx_t - \bu\rangle + \sum_{t=1}^T \langle \tilde{\bg}_{t+1} -\tilde{\bg}_t, \bx_t-\bu\rangle \\
&= \sum_{t=1}^T (\langle \bg_t, \bx_t - \bu\rangle + \langle  \tilde{\bg}_{t+1} - \tilde{\bg}_t, \bx_t\rangle)+\langle \tilde{\bg}_{1} - \tilde{\bg}_{T+1}, \bu\rangle ~.
\end{align*}
Summing the r.h.s., we have that
\begin{align*}
\sum_{t=1}^T \langle \bg_t - \tilde{\bg}_t + \tilde{\bg}_{t+1}, \bx_t - \bx_{t+1} \rangle
&= \sum_{t=1}^T \langle \bg_t-\tilde{\bg}_{t}, \bx_t - \bx_{t+1} \rangle + \sum_{t=1}^T \langle \tilde{\bg}_{t+1}, \bx_t - \bx_{t+1} \rangle~.
\end{align*}
Finally, observe that
\begin{align*}
\sum_{t=1}^T \langle \tilde{\bg}_{t+1} - \tilde{\bg}_t, \bx_t\rangle - \sum_{t=1}^T \langle \tilde{\bg}_{t+1}, \bx_t - \bx_{t+1} \rangle
&= \sum_{t=1}^T (\langle \tilde{\bg}_{t+1},\bx_{t+1}\rangle -\langle \tilde{\bg}_t,\bx_t\rangle)\\
&= \langle \tilde{\bg}_{T+1},\bx_{T+1}\rangle - \langle \tilde{\bg}_1,\bx_1\rangle~.
\end{align*}

Telescoping the terms with the Bregman divergences\index{Bregman divergence} gives the first bound. The bound for fixed $\eta$ is proved similarly.

Now, using the strong convexity of $\psi$, we have
\[
\langle \bg_t-\tilde{\bg}_t,\bx_{t}-\bx_{t+1}\rangle- \frac{1}{\eta_t} B_\psi(\bx_{t+1};\bx_t)
\leq \langle \bg_t-\tilde{\bg}_t,\bx_{t}-\bx_{t+1}\rangle- \frac{\lambda}{2 \eta_t} \|\bx_t-\bx_{t+1}\|^2~.
\]
From Lemma~\ref{lemma:diff_iterate_omd}, we know that $\|\bx_{t}-\bx_{t+1}\|$ is upper bounded by $\frac{\eta_t}{\lambda}\|\bg_t-\tilde{\bg}_t+\tilde{\bg}_{t+1}\|_\star$. Hence, we have
\[
\langle \bg_t-\tilde{\bg}_t,\bx_{t}-\bx_{t+1}\rangle- \frac{\lambda}{2 \eta_t} \|\bx_t-\bx_{t+1}\|^2
\leq \max_{\bv\in \R^d: \|\bv\|\leq \frac{\eta_t}{\lambda}\|\bg_t-\tilde{\bg}_t+\tilde{\bg}_{t+1}\|_\star} \ \langle \bg_t-\tilde{\bg}_t,\bv\rangle- \frac{\lambda}{2 \eta_t} \|\bv\|^2~.
\]
Using Lemma~\ref{lemma:constrained_bound} finishes the proof.
\end{proof}
\index{Online Mirror Descent algorithm!optimistic|)textbf}

In the next section, we show an immediate application of Optimistic \ac{OMD} to study \ac{OMD} with delayed feedback. In Section~\ref{sec:saddle_point_optimism}, we describe the applications of optimistic algorithms for saddle-point optimization.

\section{Application: OMD with Delayed Feedback}
\label{sec:delayed_omd}

\index{Online Mirror Descent algorithm!optimistic|(}
\index{delayed feedback|(textbf}
We now consider online learning with \emph{delayed feedback}. For simplicity, we consider a constant delay of length $\tau$, where at time $t$ the learner has only observed $\ell_1, \dots, \ell_{t-1-\tau}$ before producing $\bx_t$. In other words, the learner observes $\ell_t$ at time $t+\tau$. This means that the algorithm receives its first feedback at round $\tau+1$ and it receives no feedback before that round.

Instead of designing online algorithms specifically for the case of delayed feedback, we will reduce the setting of online learning with delays to that of optimistic online learning.

In \ac{OMD} with delays we predict the same $\bx_1$ for the first $\tau+1$ rounds and
\[
\bx_{t+1}
= \argmin_{\bx \in \mathcal{V}} \ B_\psi(\bx;\bx_t) + \eta_{t} \langle \bg_{t-\tau}, \bx\rangle, \quad \forall t\geq \tau+1~.
\]
On the other hand, Optimistic \ac{OMD} without delays updates with
\[
\bx_{t+1}
= \argmin_{\bx \in \mathcal{V}} \ B_\psi(\bx;\bx_t) + \eta_{t} \langle \bg_{t}+\tilde{\bg}_{t+1}-\tilde{\bg}_{t}, \bx\rangle,
\]
where $\tilde{\bg}_1=\boldsymbol{0}$.
So, the two are \emph{equivalent} by setting $\tilde{\bg}_{t+1}-\tilde{\bg}_{t}=\bg_{t-\tau}-\bg_t$. So, unrolling, we have $\tilde{\bg}_{t+1}=-\sum_{i=\max(1,t-\tau+1)}^{t} \bg_i$.
In other words, \emph{we are receiving a very poor hint that corresponds to the delayed feedback}.

The above observation is very powerful because it immediately gives us the regret guarantee of \ac{OMD} with delays. Let's consider \ac{OMD} with delayed feedback and constant learning rate.
Following the delay-to-optimism conversion in the previous section, we set $\tilde{\bg}_{t+1}=-\sum_{i=\max(1,t-\tau+1)}^{t} \bg_i$ for $t=1, \dots, T-1$, yet we must also set $\tilde{\bg}_{T+1}=\boldsymbol{0}$. Hence, for all $\bu \in \mathcal{V}$, we obtain a regret upper bound with reduced terms in the sum for the first $T-1$ rounds:
\begin{align*}
\sum_{t=1}^T &\langle \bg_t, \bx_t-\bu\rangle\\
&\leq \frac{B_\psi(\bu;\bx_1)}{\eta} + \frac{\eta}{\lambda} \sum_{t=\tau+1}^{T-1} \left\|\sum_{i=\max(1,t-\tau)}^{t} \bg_i\right\|_\star \|\bg_{t-\tau}\|_\star + \frac{\eta}{2\lambda} \left\|\sum_{i=\max(1,T-\tau)}^{T} \bg_i\right\|^2_\star~.
\end{align*}
That is, assuming $\bg_t$ is bounded for all $t$, we improved the worst-case dependence in the dominant term of the bound from $\tau^2$ to $\tau$. Optimally choosing the learning rate, we obtain the optimal regret guarantee of $\mathcal{O}(\sqrt{\tau T})$.
\index{delayed feedback|)textbf}
\index{Online Mirror Descent algorithm!optimistic|)}

\section{History Bits}

The Bregman divergence\index{Bregman divergence} was introduced by \citet{Bregman67} as a particular example of a distance-like function satisfying certain properties, to generalize the cyclic projection algorithm to general topological vector spaces. Often people drop the condition on the strict convexity \citep[e.g.,][]{BauschkeBC03}, but in reality it is part of the original definition by \citet{Bregman67}.

Mirror Descent (MD)\index{Mirror Descent algorithm} was introduced by \citet{NemirovskijY83} in the \emph{offline} setting. The description of MD with Bregman divergence\index{Bregman divergence} that I described here (with minor changes) was done by \citet{BeckT03}. The minor changes are in decoupling the domain $\mathcal{X}$ of $\psi$ from the feasibility set $\mathcal{V}$. This allows us to use functions $\psi$ that do not satisfy the condition \eqref{eq:cond_omd1}, but they satisfy \eqref{eq:cond_omd2}.
In the online setting, the mirror descent scheme was used for the first time by \citet{Warmuth97}.

Most of the online learning literature for \ac{OMD} assumes $\psi$ to be \emph{Legendre}~\citep[see, e.g.,][]{Cesa-BianchiL06}\index{function!Legendre} that corresponds to assuming \eqref{eq:cond_omd1} (or $\lim_{\bx \to \bdry \mathcal{X}} \|\nabla \psi(\bx)\|_2 = +\infty$, see \citep[Theorem 26.1 and Lemma 26.2]{Rockafellar70}). This condition allows us to prove that $\nabla \psi_{\mathcal{V}}^\star=(\nabla \psi_{\mathcal{V}})^{-1}$. However, it turns out that the Legendre condition is not necessary, and we only need the function $\psi$ to be differentiable on the predictions $\bx_t$. For example, we only need one of the two conditions in \eqref{eq:cond_omd1} or \eqref{eq:cond_omd2} to hold. Removing the Legendre assumption makes it easier to use \ac{OMD} with different combinations of feasibility sets/Bregman divergences\index{Bregman divergence}. So, I did not introduce the concept of Legendre functions at all, relying instead on (a minor modification of) \ac{OMD} as described by \citet{BeckT03}.
Theorem~\ref{thm:omd_well_defined} is derived from \citep[Theorem 3.12]{BauschkeB97}.

The proof of Theorem~\ref{thm:prop_fenchel_sc}\index{duality strong convexity-smoothness} is based on the one in \citet{KakadeSST09}.

The concept of local norms\index{norm!local} was introduced in \citet{AbernethyHR08} for \ac{FTRL} with self-concordant regularizers\index{regularizer!self-concordant}.

The \ac{EG}\index{Exponentiated Gradient algorithm!OMD version} algorithm was introduced by \citet{KivinenW97}, but not as a specific instantiation of \ac{OMD}.
\citet{BeckT03} rediscover \ac{EG} for the offline case as an example of MD. Later, \citet{Cesa-BianchiL06} shows that \ac{EG} is just an instantiation of \ac{OMD}. The $p$-norm algorithms for online prediction were originally introduced by \citet{GroveLS97,GroveLS01}. Lemma~\ref{lemma:strong_convexity_pnorm} is well-known, but I could not find a good proof for it,\footnote{People often cite \citet[Lemma 17]{Shalev-Shwartz07}, but it has a wrong proof because it ignores the fact that the function is not twice differentiable.} so I wrote one. The trick to set $q=2\ln d$ is from \citet{GentileL99,Gentile03} (online learning) and apparently rediscovered in \citet{Ben-TalMN01} (optimization). \index{learning with expert advice|(}The \ac{LEA} setting was introduced by \citet{LittlestoneW94} and \citet{Vovk90}. The ideas in Algorithm~\ref{alg:lea} are based on the Multiplicative Weights algorithm\index{Multiplicative Weights algorithm}~\citep{LittlestoneW94} and the Hedge algorithm\index{Hedge algorithm}~\citep{FreundS95,FreundS97}, both being just \ac{OMD}/\ac{FTRL} with the entropic distance generating function/regularizer. As a side note, the weighted majority algorithm was also discovered independently in the game theory literature by \citet{FudenbergL95}.
For two experts with losses in $[0,1]$, \citet{Cover65} showed that the asymptotic minimax regret is $\sqrt{\frac{T}{2\pi}}$ and proposed an algorithm achieving it. Notably, the approach in \citet{Cover65} is based on online betting. On the other hand, for more than 2 experts and losses in $[0,1]$, the minimax regret is $(1+o(1)) \sqrt{\frac{T \ln d}{2}}$, where $o(1)\to 0$ when $d, T\to \infty$~\citep{Cesa-BianchiFHHSW93,Cesa-BianchiFHHSW97}.
By now, the literature on \ac{LEA} is huge, with tons of variations over algorithms and settings.\index{learning with expert advice|)}


\index{Multi-Scale Expert algorithm|(}
The multi-scale setting was introduced independently\footnote{Dylan Foster, personal communication, 2026} and around the same time by \citet{BubeckDHN17,BubeckDHN19} and \citet{FosterKMS17}. \citet{BubeckDHN17,BubeckDHN19} propose two algorithms, one for non-positive losses and one for generic losses. Unfortunately, their proof for the generic losses appears to be wrong and not easily fixable. Their bound is similar to the one we proved here.
\citet{FosterKMS17} proved a stronger result, where the term in the logarithm depends on $c_i/\pi_i$ rather than on the sum of the scales and the minimal value. However, this algorithm has a computational complexity at step $t$ of $\mathcal{O}(t)$.
\citet{CutkoskyO18} proposed another solution that achieves the same bound of \citet{FosterKMS17}, but with constant computational complexity.
The algorithm I describe here is a simplification of the one in \citet{ChenLW21}. The bound is worse than the one in \citet{CutkoskyO18} and \citet{FosterKMS17}, but it is simpler to describe, and it also allowed me to describe the method of the shifted surrogate loss.
The use of shifted surrogate losses in \ac{LEA} setting was proposed in \citet{HazanK08,HazanK10} and later used in \citet{Steinhardt14}. Interestingly, these shifted surrogates can be understood as an approximation of the coin-betting instantaneous log wealth, as one can see by comparing the update of Squint~\citep{KoolenVE15} with the update of the KT-based parameter-free algorithm for \ac{LEA}.
\index{Multi-Scale Expert algorithm|)}

The construction in Section~\ref{sec:easy_metagrad} is from \citet{FosterKMS17}, but they require $N$ function values per step, while we use only one subgradient per step using the linearization idea proposed in MetaGrad~\citep{vanErvenK16,VanErvenKV21}. The idea of restricting the norm of the competitor to the ``meaningful ones'' is not necessary in the original formulation by \citet{FosterKMS17}, but it is standard, and it has been rediscovered many times in different forms, e.g., in a method privately proposed by Nemirovski in 2013 and described in the ICML 2020 tutorial\footnote{\url{https://parameterfree.com/icml-tutorial/}} on ``Parameter-free Online Optimization'' by Orabona and Cutkosky, in \citet[Theorem 2]{Orabona14}, in \citet[Section 5]{Cutkosky19c}.

\index{Online Mirror Descent algorithm!optimistic|(}
The idea of ``hallucinating'' future losses used in Optimistic \ac{OMD} is originally from \citet{AzouryW01} in the Forward Algorithm\index{Forward Algorithm}. Apparently, this idea was forgotten and rediscovered by \citet{Chiang12} that used the previous loss function as an estimate of the next one, showing smaller regret in the case that the losses have small temporal variation.
Later, \citet{RakhlinS13b} generalized this idea in the Optimistic \ac{OMD} algorithm. Surprisingly enough, the procedure using two Optimistic \ac{OGD} algorithms to solve saddle-point problems was already proposed by \citet{Popov80}, see also Section~\ref{sec:saddle-point_history}. Optimistic \ac{OMD} was proposed in \citet{RakhlinS13b} with a two-step update. It was then simplified to the one-step updates I presented here by \citet{JoulaniGS17}. However, \citet{Malitsky15} presented a version of Popov's algorithm for variational inequalities with only one projection that is essentially Optimistic \ac{OGD} with one projection. The proof I present here is based on the one I proposed for Optimistic \ac{FTRL} in Section~\ref{sec:optimistic_ftrl}.
\index{Online Mirror Descent algorithm!optimistic|)}

\index{delayed feedback|(}
\citet{WeinbergerO02} were the first to analyze the delayed feedback problem. They considered the adversarial full information setting with a fixed, known delay $\tau$. They achieved the optimal rate of $\mathcal{O}(\sqrt{\tau T})$ with a black-box reduction, by running $\tau+1$ online algorithms on subsampled sequences.
\citet{ZinkevichLS09} analyzed \ac{OMD} with delayed feedback and obtained the optimal regret. \citet{JoulaniGS13} unified most of the prior research on the effect of delay in adversarial and stochastic problems.
\citet{McMahanS14} studied a variant of AdaGrad with delays\index{AdaGrad algorithm}. \citet{QuanrudK15} studied \ac{OGD} and \ac{FTRL} with adversarially chosen delays, while \citet{JoulaniGS16} studied \ac{OMD} and \ac{FTRL}-Proximal with adaptive learning rates.

The equivalence of delays/bad hints and the improved regret guarantees for Optimistic \ac{OMD} and Optimistic \ac{FTRL} are due to \citet{FlaspohlerOCMOOM21}. However, their OMD bound contains a small mistake: they are missing the last terms in the bound, because we set $\tilde{\bg}_{T+1}=\boldsymbol{0}$, breaking the equivalence between optimistic and delayed updates for $\bx_{T+1}$. Note that while the choice of $\bx_{T+1}$ does not influence the algorithm in the rounds $1, \dots, T$, its value appears in the regret upper bound. \citet{FlaspohlerOCMOOM21} also show that \acl{RM} and \acl{RM+} algorithms automatically adapts to the delay $\tau$.
\index{delayed feedback|)}

Lemma~\ref{lemma:diff_iterate_omd} is from \citet{JoulaniGS16}.

\section{Exercises}


\begin{exer}
Prove the three-points equality for Bregman divergences\index{Bregman divergence} in Lemma~\ref{lemma:bregman_3_points}.
\end{exer}

\begin{exer}
Let $\bA \in \R^{d \times d}$ be a positive definite matrix. Define $\|\bx\|^2_{\bA}=\bx^\top \bA \bx$. Prove that $\frac12 \|\bx-\by\|^2_{\bA}$ is the Bregman divergence $B_\psi(\bx;\by)$ associated with $\psi(\bx)=\frac{1}{2}\|\bx\|_{\bA}^2$.
\end{exer}


\begin{exer}
Let $\psi: \mathcal{X} \to \R$ be a distance generating function\index{distance generating function} and $B_\psi$ its associated Bregman divergence\index{Bregman divergence}. Fix $\by \in \interior\mathcal{X}$ and define $f(\bx)= B_\psi(\bx;\by)$. For any $\bv \in \interior \mathcal{X}$ and $\bx \in \mathcal{X}$, prove that $B_f(\bx;\bv)=B_\psi(\bx;\bv)$.
\end{exer}

\begin{exer}
We want to show an equality quantifying the gap in Fenchel--Young's inequality\index{inequality!Fenchel--Young's} with a Bregman divergence\index{Bregman divergence} term.
Assume that $f$ and $f^\star$ are differentiable, $f$ strictly convex, and $\dom f =\R^d$. Prove that
\[
f(\bx)+ f^\star(\btheta) = \langle \btheta, \bx\rangle + B_f\left(\bx;\nabla f^\star(\btheta)\right)~.
\]
\end{exer}

\begin{exer}
Let $f : \R^d \to (-\infty, +\infty]$ be even. Prove that $f^\star$ is even.
\end{exer}

\begin{exer}
In the proof of \ac{OMD}, we have the terms $-B_\psi(\bx_{t+1};\bx_t)+\langle \eta_t \bg_t, \bx_t - \bx_{t+1}\rangle$.
Prove that they can be lower bounded by $B_\psi(\bx_t;\bx_{t+1})$.
\end{exer}

\begin{exer}
\label{exercise:relative_strong_convexity}
Generalize the concept of strong convexity to Bregman divergences\index{Bregman divergence}, instead of norms, and prove a logarithmic regret guarantee for such functions using \ac{OMD}.
\end{exer}

\begin{exer}
\label{exercise:doubling_trick_omd}
Consider the setting of Theorem~\ref{thm:md_online}, assume that
$\sup_{\bx,\by \in \mathcal{V}} B_\psi(\bx;\by) \leq D^2$, and that $\|\bg_t\|_\star \leq L$ for all $t$.
Suppose that the time horizon $T$ is not known in advance.

Use the following \emph{doubling trick}\index{doubling trick}. Split time into epochs of lengths $1,2,4,\dots$. At the beginning of each epoch of length $2^k$, restart
\ac{OMD} from an arbitrary point in $\interior \mathcal{X}\cap \mathcal{V}$ and
run it with the constant learning rate $\eta_k=\frac{D\sqrt{2\lambda}}{L\sqrt{2^k}}$.
Prove that, for every $T\geq 1$ and every $\bu \in \mathcal{V}$, this algorithm
satisfies
\[
\sum_{t=1}^T \left(\ell_t(\bx_t)-\ell_t(\bu)\right)
\leq \frac{2}{\sqrt{2}-1} \frac{D L}{\sqrt{\lambda}}\sqrt{T}~.
\]
\end{exer}

\begin{exer}
\index{Exponentiated Gradient algorithm!OMD version}
Derive the \ac{EG} update rule and regret bound in the case that the algorithm starts from an arbitrary vector $\bx_1$ in the probability simplex.
\end{exer}

\begin{exer}
\label{exercise:eg_invariant_constant}
\index{Exponentiated Gradient algorithm!OMD version}
Show that \ac{EG} is invariant to additive constants added to the loss vectors. Use this observation to show that the terms $\sum_{i=1}^d g_{t,i}^2 x_{t,i}$ for $t=1,\dots,T$ in the regret upper bound can be tightened to $\sum_{i=1}^d (g_{t,i}-m_t)^2 x_{t,i}$ for any $m_t \in \R$.
\end{exer}

\begin{exer}
Extend Theorem~\ref{thm:lower_bound_constr} to arbitrary norms, measuring the diameter with respect to a norm $\|\cdot\|$ and considering losses $L$-Lipschitz with respect to $\|\cdot\|$.
\end{exer}

\begin{exer}
In this problem, we will tackle Online Non-Convex Optimization\index{online non-convex optimization}.
Assume that $\mathcal{V} \subset \R^d$ is the feasible set and it is convex and bounded.
The losses $\ell_t:\R^d \to [0,1]$ are non-convex and $1$-Lipschitz with respect to $\|\cdot\|_2$.
Prove that there exists a randomized algorithm that achieves sublinear regret on this problem, assuming knowledge of the total number of rounds $T$. Hint: aim for something like $\E[\Regret_T(\bu)] = \mathcal{O}(\sqrt{d T \ln T})$ and do not worry about the efficiency of the algorithm.
\end{exer}

\index{Online Mirror Descent algorithm|)textbf}

\acresetall

\chapter{Follow-the-Regularized-Leader}
\label{ch:ftrl}

\index{Follow-the-Regularized-Leader algorithm|(textbf}
Until now, we focused only on Online Subgradient Descent and its generalization, Online Mirror Descent, with a brief ad-hoc analysis of \ac{FTL} in the first chapter.
In this chapter, we will extend \ac{FTL} to a powerful and generic algorithm to do \ac{OCO}: \ac{FTRL}.

\ac{FTRL} is a very intuitive algorithm: at each time step, it will play the minimizer of the sum of the past losses \emph{plus} a time-varying \emph{regularizer}, $\psi_t$\index{regularizer}.
We will see that the regularization is needed to make the algorithm ``more stable'' with linear losses and avoid the jumping back and forth that we saw in Example~\ref{ex:failure_ftl}.

\acresetall

\section{The Follow-the-Regularized-Leader Algorithm}

\begin{algorithm}[h]
\caption{Follow-the-Regularized-Leader (FTRL)}
\label{alg:ftrl}
\begin{algorithmic}[1]
{
    \REQUIRE{A sequence of regularizers $\psi_1, \dots, \psi_T :\mathcal{X} \to \R$, closed non-empty set $\mathcal{V} \subseteq \mathcal{X} \subseteq \R^d$}
    \FOR{$t=1$ {\bfseries to} $T$}
    \STATE{Output $\bx_t \in \argmin_{\bx \in \mathcal{V}} \ \psi_t(\bx) + \sum_{i=1}^{t-1} \ell_i(\bx)$}
    \STATE{Receive $\ell_t:\R^d \to (-\infty, +\infty]$ and pay the loss $\ell_t(\bx_t)$}
    \ENDFOR
}
\end{algorithmic}
\end{algorithm}

In \textbf{\ac{FTRL}}, in each round we output the minimizer of the regularized cumulative past losses, see Algorithm~\ref{alg:ftrl}. It should be clear that \ac{FTRL} is not an algorithm, but a family of algorithms, in the same way that \ac{OMD} is a family of algorithms.

Before analyzing \ac{FTRL}, let's get some intuition on it. In \ac{OMD}, we saw that the ``state'' of the algorithm is stored in the current iterate $\bx_t$, in the sense that the next iterate $\bx_{t+1}$ depends on $\bx_t$ and the loss received at time $t$, plus obviously the choice of the learning rate. Instead, in \ac{FTRL}, the next iterate $\bx_{t+1}$ depends on the entire history of losses received up to time $t$ and on the regularizer used at time $t$. This has an immediate consequence: in the case that $\mathcal{V}$ is bounded, \ac{OMD} will only ``remember'' the last $\bx_t$ and nothing else. On the other hand, \ac{FTRL} keeps in memory the entire history of the past, which in principle allows us to recover the iterates before the projection in $\mathcal{V}$.

This difference in behavior might make the reader think that \ac{FTRL} is more computationally and memory-expensive. And indeed it is! But, in Section~\ref{sec:ftrl_linear_losses} we will also see that there is a way to consider approximate losses that makes the algorithm as expensive as \ac{OMD}, yet retaining the same or more information than \ac{OMD}.

\index{Follow-the-Regularized-Leader algorithm!regret equality|(textbf}
For \ac{FTRL}, we prove an \emph{equality} for the regret. This equality factors the regret in three terms that have precise meanings and can be easily upper-bounded with some familiar quantities.
\begin{lemma}
\label{lemma:ftrl_equality}
Let $\psi_1, \dots, \psi_{T+1} :\mathcal{X} \to \R$ be a sequence of regularization functions and $\mathcal{V} \subseteq \mathcal{X} \subseteq \R^d$ a closed and non-empty set.
Denote by $F_t(\bx) = \psi_{t}(\bx) + \sum_{i=1}^{t-1} \ell_i(\bx)$, where $\ell_t:\R^d \to (-\infty, +\infty]$ for all $t$.
Assume that $\argmin_{\bx \in \mathcal{V}} \ F_{t}(\bx)$ is not empty and set $\bx_t \in \argmin_{\bx \in \mathcal{V}} \ F_{t}(\bx)$. Then, for all $\bu \in \mathcal{X}$, we have
\begin{align*}
\sum_{t=1}^T ( \ell_t(\bx_t) - \ell_t(\bu) )
&= \psi_{T+1}(\bu) - \min_{\bx \in \mathcal{V}} \ \psi_{1}(\bx) + \sum_{t=1}^T [F_t(\bx_t) - F_{t+1}(\bx_{t+1}) + \ell_t(\bx_t)] \\
&\quad + F_{T+1}(\bx_{T+1}) - F_{T+1}(\bu)~.
\end{align*}
In addition, the choice of $\psi_{T+1}$ does not change the iterates $\bx_1, \dots, \bx_T$ of the algorithm.
\end{lemma}
\begin{proof}
Given that the terms $\ell_t(\bx_t)$ appear on both sides of the equality, we just have to verify that
\begin{align*}
- \sum_{t=1}^T \ell_t(\bu)
&= \psi_{T+1}(\bu) - \min_{\bx \in \mathcal{V}} \ \psi_{1}(\bx) + \sum_{t=1}^T [F_t(\bx_t) - F_{t+1}(\bx_{t+1})]\\
&\quad + F_{T+1}(\bx_{T+1}) - F_{T+1}(\bu)~.
\end{align*}
Remembering that $F_1(\bx_1) = \min_{\bx \in \mathcal{V}} \ \psi_1(\bx)$ and using the fact that the sum with $F_t$ is telescopic, we have
\begin{align*}
- \sum_{t=1}^T \ell_t(\bu)
&= \psi_{T+1}(\bu) - F_1(\bx_1) + F_1(\bx_1) - F_{T+1}(\bx_{T+1}) + F_{T+1}(\bx_{T+1}) - F_{T+1}(\bu)\\
&= \psi_{T+1}(\bu) - F_{T+1}(\bu),
\end{align*}
that is true by the definition of $F_{T+1}$.
The second statement is immediate.
\end{proof}

\begin{remark}
We basically did not assume anything on $\ell_t$ nor on $\psi_t$, hence the above equality holds even for non-convex losses and regularizers. Yet, solving the minimization problem at each step might be computationally infeasible.
\end{remark}


\begin{remark}
The \ac{FTRL} algorithm is invariant to any constant added to the regularizers, hence we can always state the regret guarantee with $\psi_t(\bu)-\min_{\bx \in \mathcal{V}}\psi_t(\bx)$ instead of $\psi_t(\bu)$. However, for clarity, we will sometimes explicitly choose the regularizers such that their minimum in $\mathcal{V}$ is 0.
\end{remark}

\begin{remark}
\label{remark:ftrl_any_prediction}
The terms $\ell_t(\bx_t)$ appear on the left and right of the equality. Hence, in some proofs we can substitute them with some other terms, for example $\ell_t(\tilde{\bx}_t)$ where $\tilde{\bx}_t$ is \emph{not} the \ac{FTRL} prediction. In other words, we can write the same equality for the prediction of a generic algorithm whose regret will depend on the prediction of the \ac{FTRL} algorithm.
\end{remark}

Let's take a closer look at the equality. If $\bu \in \mathcal{V}$, we have that the sum of the last two terms on the r.h.s. is negative.
On the other hand, the first two terms on the r.h.s. are similar to what we got in \ac{OMD}. The interesting part is the sum of the terms $F_t(\bx_t) - F_{t+1}(\bx_{t+1}) + \ell_t(\bx_t)$. To give an intuition of what is going on, let's consider the case when the regularizer is constant over time, i.e., $\psi_t=\psi$. Hence, the terms in the sum can be rewritten as
\[
F_t(\bx_t) - F_{t+1}(\bx_{t+1}) + \ell_t(\bx_t)=
\psi(\bx_t) + \sum_{i=1}^{t} \ell_i(\bx_t) - \left( \psi(\bx_{t+1}) + \sum_{i=1}^{t} \ell_i(\bx_{t+1})\right)~.
\]
Hence, we are measuring the distance between the value of the regularized losses in two consecutive predictions of the algorithms. Roughly speaking, if $\bx_{t}\approx \bx_{t+1}$ and the losses plus regularization are ``nice'' then this term will be small. This should remind you exactly of the \ac{OMD} update, where we \emph{constrain} $\bx_{t+1}$ to be close to $\bx_{t}$.
Instead, here the two predictions will be close one to the other if the minimizer of the regularized losses up to time $t$ is close to the minimizer of the losses up to time $t+1$. So, like in \ac{OMD}, the regularizer here will play the critical role of \emph{stabilizing} the predictions, if the losses do not possess enough curvature.

However, while surprising, the above equality is not yet a regret bound, because it is ``implicit''. In fact, the losses are appearing on both sides of the equation.
So, in the following, we will see different ways to get an explicit upper bound from Lemma~\ref{lemma:ftrl_equality}.
\index{Follow-the-Regularized-Leader algorithm!regret equality|)textbf}

\section{FTRL Regret Bound using Strong Convexity}

An easy case to get a regret upper bound for \ac{FTRL} is when the losses plus the regularizer are strongly convex, as we prove in the next lemma.
\begin{lemma}
\label{lemma:ftrl_strongly_1}
With the notation in Algorithm~\ref{alg:ftrl}, denote by $F_t(\bx) = \psi_{t}(\bx) + \sum_{i=1}^{t-1} \ell_i(\bx)$. Assume that $\mathcal{V} \subseteq \mathcal{X}$ is non-empty, closed, and convex. If $F_t$ is closed, subdifferentiable, and strongly convex\index{function!strongly convex} in $\mathcal{V}$, then $\bx_t$ exists and is unique. In addition, assume $\partial \ell_t(\bx_t)$ to be non-empty and $F_t+\ell_t$ to be closed, subdifferentiable, and $\lambda_t$-strongly convex with respect to $\|\cdot\|$ in $\mathcal{V}$. Then, we have
\begin{align*}
F_t&(\bx_t) - F_{t+1}(\bx_{t+1}) + \ell_t(\bx_t)\\
&\leq \langle \bg_t, \bx_{t}-\bx_{t+1}\rangle - \frac{\lambda_t}{2}\|\bx_t-\bx_{t+1}\|^2+ \psi_t(\bx_{t+1}) - \psi_{t+1}(\bx_{t+1}) \\
&\leq \frac{\|\bg_t\|_\star^2}{2\lambda_t} + \psi_t(\bx_{t+1}) - \psi_{t+1}(\bx_{t+1}), \quad \forall \bg_t \in \partial \ell_t(\bx_t)~.
\end{align*}
\end{lemma}
\begin{proof}
The existence and uniqueness are given by Theorem~\ref{thm:min_strongly_convex}.
Then, we have
\begin{align*}
F_t&(\bx_t) - F_{t+1}(\bx_{t+1}) + \ell_t(\bx_t) \\
&= (F_t(\bx_t) + \ell_t(\bx_t)) - (F_{t}(\bx_{t+1}) + \ell_t(\bx_{t+1})) + \psi_t(\bx_{t+1}) - \psi_{t+1}(\bx_{t+1}) \\
&= (F_t(\bx_t) + \ell_t(\bx_t)+\indicator_\mathcal{V}(\bx_t)) - (F_{t}(\bx_{t+1}) + \ell_t(\bx_{t+1})+\indicator_\mathcal{V}(\bx_{t+1})) \\
&\quad + \psi_t(\bx_{t+1}) - \psi_{t+1}(\bx_{t+1}) \\
&\leq \langle \bg_t, \bx_{t}-\bx_{t+1}\rangle - \frac{\lambda_t}{2}\|\bx_t-\bx_{t+1}\|^2+ \psi_t(\bx_{t+1}) - \psi_{t+1}(\bx_{t+1}),
\end{align*}
where in the second inequality we used Lemma~\ref{lemma:strong_convexity} for the function $F_t+\ell_t+\indicator_\mathcal{V}$
with $\bg_t \in \partial (F_t + \ell_t+\indicator_{\mathcal{V}})(\bx_t)$. Observing that $\bx_t = \argmin_{\bx \in \mathcal{V}} \ F_t(\bx)$, we have $\boldsymbol{0} \in \partial (F_t+\indicator_{\mathcal{V}})(\bx_t)$ by Theorem~\ref{thm:first_order_subdiff}. Hence, using Theorem~\ref{thm:sum_subgradients}, we have $\partial \ell_t(\bx_t)\subseteq \partial (F_t + \ell_t+\indicator_{\mathcal{V}})(\bx_t)$.

The second inequality is obtained observing that $\langle \bg_t, \bx_{t}-\bx_{t+1}\rangle \leq \|\bg_t\|_\star \|\bx_t-\bx_{t+1}\|$ and using the elementary inequality $a x - \frac{b}{2} x^2 \leq \frac{a^2}{2b}$ for $x \in \R$ and $a,b>0$.
\end{proof}

Let's see an immediate application of this lemma.
\begin{corollary}
\label{cor:ftrl_lip}
With the notation in Algorithm~\ref{alg:ftrl}, let $\mathcal{V}$ be non-empty, closed, and convex. Let $\psi: \mathcal{V} \to \R$ be a closed and $\mu$-strongly convex\index{function!strongly convex} function with respect to $\|\cdot\|$. Set the sequence of regularizers as $\psi_t(\bx)=\frac{1}{\eta_{t-1}} (\psi(\bx)-\min_{\bz \in \mathcal{V}} \psi(\bz))$ for $t=1, \dots, T$ and $\psi_{T+1}=\psi_T$. Assume each $\ell_t$ subdifferentiable in $\mathcal{V}$. Then, for all $\bu \in \mathcal{V}$, \ac{FTRL} guarantees
\begin{align*}
\sum_{t=1}^T (\ell_t(\bx_t) -\ell_t(\bu))
&\leq \frac{\psi(\bu) - \min_{\bx \in \mathcal{V}} \psi(\bx)}{\eta_{T-1}} +  \sum_{t=1}^T \frac{\eta_{t-1}}{2 \mu} \|\bg_t\|^2_\star \\
&\quad + \sum_{t=1}^{T-1} \left(\frac{1}{\eta_{t-1}}-\frac{1}{\eta_t}\right) (\psi(\bx_{t+1})-\min_{\bz \in \mathcal{V}} \psi(\bz)), \quad \forall \bg_t \in \partial \ell_t(\bx_t)~.
\end{align*}
Moreover, if the functions $\ell_t$ are $L$-Lipschitz on an open set containing $\mathcal{V}$, setting $\eta_{t-1}=\frac{\alpha \sqrt{\mu}}{L\sqrt{t}}$ we get
\[
\sum_{t=1}^T \ell_t(\bx_t) - \sum_{t=1}^T \ell_t(\bu)
\leq \left(\frac{\psi(\bu) - \min_{\bx} \psi(\bx)}{\alpha} +\alpha \right) \frac{L \sqrt{T}}{\sqrt{\mu}}, \quad \forall \bu \in \mathcal{V}~.
\]
\end{corollary}
\begin{proof}
The corollary is immediate from Lemma~\ref{lemma:ftrl_equality}, Lemma~\ref{lemma:ftrl_strongly_1}, and the observation that from the assumptions we have $\psi_t(\bx) - \psi_{t+1}(\bx)\leq0, \  \forall \bx \in \mathcal{V}$.
\end{proof}

This might look like the same regret guarantee of \ac{OMD}; however, here there is a very important difference: the last term contains a time-varying element ($\eta_t$), but the domain does not have to be bounded!
Also, I used the regularizer $\frac{1}{\eta_{t-1}} \psi(\bx)$ and not $\frac{1}{\eta_{t}} \psi(\bx)$ to remind you of another important difference: In \ac{OMD} the learning rate $\eta_t$ is chosen after receiving information about $\ell_t$, through the subgradient $\bg_t$, while here you have to choose it \emph{before} observing $\ell_t$.

\index{regularizer!proximal|(textbf}
\subsection{Proximal Regularizers}

We can obtain a slightly stronger guarantee in the case that the regularizer is \textbf{proximal}, that is, it satisfies that $\bx_t \in \argmin_{\bx \in \mathcal{V}} \ \psi_{t+1}(\bx)-\psi_t(\bx)$ for all $t$.
\begin{lemma}
\label{lemma:proximal}
With the notation in Algorithm~\ref{alg:ftrl}, denote by $F_t(\bx) = \psi_{t}(\bx) + \sum_{i=1}^{t-1} \ell_i(\bx)$ for all $t$. Assume that $\mathcal{V} \subseteq \mathcal{X}$ is non-empty, closed, and convex. Also, assume that $F_{t+1}$ is closed, subdifferentiable, and $\lambda_{t+1}$-strongly convex\index{function!strongly convex} with respect to $\|\cdot\|$ in $\mathcal{V}$, the regularizer is such that $\bx_{t} \in \argmin_{\bx \in \mathcal{V}} \ \psi_{t+1}(\bx)-\psi_t(\bx)$. Then, $\bx_{t+1}$ exists and is unique. Moreover, if $\partial \ell_t(\bx_t)$ is non-empty, we have
\begin{align*}
F_t&(\bx_t) - F_{t+1}(\bx_{t+1}) + \ell_t(\bx_t)\\
&\leq \langle \bg_t, \bx_t -\bx_{t+1}\rangle -\frac{\lambda_{t+1}}{2}\|\bx_{t}-\bx_{t+1}\|^2 + \psi_t(\bx_t) - \psi_{t+1}(\bx_t) \\
&\leq \frac{\|\bg_t\|_\star^2}{2\lambda_{t+1}} + \psi_t(\bx_t) - \psi_{t+1}(\bx_t), \quad \forall \bg_t \in \partial \ell_t(\bx_t)~.
\end{align*}
\end{lemma}
\begin{proof}
The existence and uniqueness are given by Theorem~\ref{thm:min_strongly_convex}.
We have
\begin{align*}
F_t&(\bx_t) - F_{t+1}(\bx_{t+1}) + \ell_t(\bx_t) \\
&= (F_t(\bx_t) + \ell_t(\bx_t) + \psi_{t+1}(\bx_t) - \psi_t(\bx_t)) - F_{t+1}(\bx_{t+1}) - \psi_{t+1}(\bx_t) + \psi_t(\bx_t)\\
&= F_{t+1}(\bx_t) + \indicator_\mathcal{V}(\bx_t)- (F_{t+1}(\bx_{t+1})+\indicator_\mathcal{V}(\bx_{t+1})) - \psi_{t+1}(\bx_t) + \psi_t(\bx_t)\\
&\leq \langle \bg_t, \bx_t -\bx_{t+1}\rangle -\frac{\lambda_{t+1}}{2}\|\bx_{t}-\bx_{t+1}\|^2 - \psi_{t+1}(\bx_t) + \psi_t(\bx_t),
\end{align*}
where in the second inequality we used Lemma~\ref{lemma:strong_convexity} on the strongly convex function $F_{t+1}+\indicator_\mathcal{V}$
with $\bg_t \in \partial (F_{t+1}+\indicator_{\mathcal{V}})(\bx_t)$. Observing that, from the proximal property, we have that $\bx_t = \argmin_{\bx \in \mathcal{V}} \ F_t(\bx) + \psi_{t+1}(\bx) - \psi_t(\bx)$, from Theorem~\ref{thm:first_order_subdiff} we have $\boldsymbol{0} \in \partial (F_t + \indicator_\mathcal{V} + \psi_{t+1} - \psi_t)(\bx_t)$. Hence, using Theorem~\ref{thm:sum_subgradients}, and remembering that $F_{t+1}(\bx) = F_{t}(\bx) + \ell_t(\bx) + \psi_{t+1}(\bx) - \psi_t(\bx)$, we have that $\partial \ell_t(\bx_t) \subseteq \partial (F_{t+1}+\indicator_{\mathcal{V}})(\bx_t)$.

The second inequality is obtained observing that $\langle \bg_t, \bx_{t}-\bx_{t+1}\rangle \leq \|\bg_t\|_\star \|\bx_t-\bx_{t+1}\|$ and using the elementary inequality $a x - \frac{b}{2} x^2 \leq \frac{a^2}{2b}$ for $x \in \R$ and $a,b>0$.
\end{proof}

\begin{remark}
Note that a constant regularizer is proximal because any point is the minimizer of the zero function. On the other hand, a constant regularizer makes Lemmas~\ref{lemma:ftrl_strongly_1} and \ref{lemma:proximal} the same.
\end{remark}

\begin{example}
An example of proximal regularization is $\psi_t(\bx)=\frac12 \sum_{i=1}^{t-1} \|\bx_i-\bx\|_2^2$. It is easy to see that this is $t-1$ strongly convex with respect to the L$_2$ norm. Using this regularizer with linear losses gives an algorithm known as \textbf{FTRL-Proximal}\index{FTRL-Proximal algorithm}. This algorithm lies in between \ac{FTRL} and \ac{OMD}, indeed it solves the issue of being ``one-step-behind'' compared to \ac{OMD} explained in the previous section, but now its regret depends on $\frac12 \sum_{t=1}^{T} \|\bx_t-\bu\|_2^2$ that is not easily controlled, unless $\mathcal{V}$ has a bounded diameter. However, the main advantage of \ac{FTRL}-Proximal is the possibility to deal with non-smooth regularizers (see Section~\ref{sec:ftrl_composite}), which is not possible in \ac{OMD}.
\end{example}
\index{regularizer!proximal|)textbf}

\index{Follow-the-Regularized-Leader algorithm!with linear losses|(textbf}
\section{FTRL with Linearized Losses}
\label{sec:ftrl_linear_losses}

An important difference between \ac{OMD} and \ac{FTRL} is that here the update rule seems way more expensive than in \ac{OMD}, because we need to solve a convex optimization problem at each step. However, it turns out we can use \ac{FTRL} on \emph{linearized losses} and obtain the same bound with the same computational complexity as that of \ac{OMD}.

Consider the case in which the losses are linear, i.e., $\ell_t(\bx) =\langle \bg_t, \bx\rangle, \ t=1,\dots,T$, we have that the prediction of \ac{FTRL} is
\[
\bx_{t+1} \in \argmin_{\bx \in \mathcal{V}} \ \psi_{t+1}(\bx) + \sum_{i=1}^{t} \langle \bg_i, \bx\rangle
= \argmax_{\bx \in \mathcal{V}} \ \left\langle -\sum_{i=1}^{t} \bg_i, \bx\right\rangle -\psi_{t+1}(\bx)~.
\]
Denote by $\psi_{\mathcal{V},t} (\bx):=\psi_t(\bx)+\indicator_{\mathcal{V}}(\bx)$. Now, if we assume $\psi_{\mathcal{V},t}$ to be proper, convex, and closed, using Theorem~\ref{thm:props_fenchel}, we have that $\bx_{t+1} \in \partial \psi_{\mathcal{V},t+1}^\star(-\sum_{i=1}^{t} \bg_i)$. Moreover, if $\psi_{\mathcal{V},t+1}$ is also strongly convex, by Theorem~\ref{thm:prop_fenchel_sc}\index{duality strong convexity-smoothness} we know that $\psi_{\mathcal{V},t+1}^\star$ is differentiable and we get
\begin{equation}
\label{eq:ftrl_update_linear}
\bx_{t+1} = \nabla \psi_{\mathcal{V},t+1}^\star\left(-\sum_{i=1}^{t} \bg_i\right)~.
\end{equation}
In turn, this update can be written in the following way
\begin{align*}
\btheta_{t+1} &= \btheta_{t} - \bg_{t}, \\
\bx_{t+1} &= \nabla \psi_{\mathcal{V},t+1}^\star (\btheta_{t+1})~.
\end{align*}
This corresponds to Figure~\ref{fig:dual_mapping_ftrl}.


\begin{figure}
\centering
\begin{tikzpicture}
    \draw[thick] (0,2.5) -- (2.5,0) -- (0,-2.5) -- (-2.5,0) -- cycle;
    \node at (-1.5,1.5) {$\mathcal{V}$}; 

     \draw[thick, rotate around={15:(7,0)}] (7,0) ellipse (3.5 and 2);

    \filldraw (0,-1) circle (0pt) node[above] {$\bx_{t+1}$}; 
    \filldraw (7,0.5) circle (0pt) node[right] {$\btheta_t$}; 

    \node at (5,-0.5) {$\btheta_{t+1} = \btheta_t - \bg_t$};
    \draw[->,>=stealth] (7,0.5) -- (4.8,-0.2);  

    \draw[->,>=stealth] (4.8,-0.7) .. controls (3,-1.5) and (2,-1.5) .. (0,-0.5) node[midway, below] {$\nabla \psi_{\mathcal{V},t+1}^\star$};
\end{tikzpicture}
\caption{Dual map for \ac{FTRL} with linear losses.}
\label{fig:dual_mapping_ftrl}
\commentAlt{Figure~\ref{fig:dual_mapping_ftrl}. Diagram of the FTRL dual update. A dual point theta_t is shifted to theta_{t+1}=theta_t-g_t, then mapped back to the primal set V by nabla psi_{V,t+1}^* to produce x_{t+1}.}
\end{figure}

Compare it to the mirror update of \ac{OMD}, rewritten similarly:
\begin{align*}
\btheta_{t+1} &= \nabla \psi (\bx_{t}) - \eta_t \bg_{t}, \\
\bx_{t+1} &= \nabla \psi_{\mathcal{V}}^\star (\btheta_{t+1})~.
\end{align*}
They are very similar, but with important differences:
\begin{itemize}
\item In \ac{OMD}, the state is kept in $\bx_t$, so we need to use a duality map\index{duality map} to transform it into a dual variable before making the update, and then back to the primal variable.
\item In \ac{FTRL} with linear losses, the state is kept directly in the dual space, updated, and then transformed into the primal variable. The primal variable is only used to predict, but not directly in the update.
\item In \ac{OMD}, the samples are weighted by the learning rates, which typically decrease.
\item In \ac{FTRL} with linear losses, all the subgradients have the same weight, but the regularizer is typically increasing over time.
\end{itemize}

Also, we will not lose anything in the worst-case regret bound! Indeed, we can run \textbf{\ac{FTRL} on linearized losses} $\tilde{\ell}_t(\bx)=\langle \bg_t, \bx\rangle$, where $\bg_t \in \partial \ell_t(\bx_t)$, guaranteeing exactly the same worst-case upper bound on the regret with the losses $\ell_t$.
The algorithm for such a procedure is in Algorithm~\ref{alg:ftrl_linear}.

\begin{algorithm}[t]
\caption{Follow-the-Regularized-Leader (FTRL) on Linearized Losses}
\label{alg:ftrl_linear}
\begin{algorithmic}[1]
{
    \REQUIRE{A sequence of regularizers $\psi_1, \dots, \psi_T :\mathcal{X} \to \R$, closed non-empty convex set $\mathcal{V} \subseteq \mathcal{X} \subseteq \R^d$}
    \FOR{$t=1$ {\bfseries to} $T$}
    \STATE{Output $\bx_t \in \argmin_{\bx \in \mathcal{V}} \ \psi_t(\bx) + \sum_{i=1}^{t-1} \langle \bg_i,\bx\rangle$}
    \STATE{Pay the loss $\ell_t(\bx_t)$, where $\ell_t$ is subdifferentiable on $\mathcal{V}$}
    \STATE{Set $\bg_t \in \partial \ell_t(\bx_t)$}
    \ENDFOR
}
\end{algorithmic}
\end{algorithm}

In fact, using the definition of the subgradients and the assumptions of Corollary~\ref{cor:ftrl_lip}, we have
\begin{align*}
\Regret_T(\bu)
&= \sum_{t=1}^T (\ell_t(\bx_t) -\ell_t(\bu))
\leq \sum_{t=1}^T (\tilde{\ell}_t(\bx_t) - \tilde{\ell}_t(\bu))\\
&\leq \frac{\psi(\bu) - \min_{\bx \in \mathcal{V}} \psi(\bx)}{\eta_{T-1}} + \frac{1}{2 \mu} \sum_{t=1}^T \eta_{t-1} \|\bg_t\|^2_\star, \quad \forall \bu \in \mathcal{V}~.
\end{align*}
The only difference with respect to Corollary~\ref{cor:ftrl_lip} is that here the $\bg_t$ are the specific ones we use in the algorithm, while in Corollary~\ref{cor:ftrl_lip} the statement holds for any choice of the $\bg_t \in \partial \ell_t(\bx_t)$. Moreover, \ac{FTRL} with exact losses can adapt to any ``nice'' properties of the losses, for example, strong convexity, while this information is lost in the linearized case.
We also have to remember that this is just a worst-case guarantee: on real problems \ac{FTRL} with full losses typically performs better than both \ac{FTRL} with linearized losses and \ac{OMD}.

More generally, we can also specialize the \ac{FTRL} regret equality to this case, obtaining the following lemma, which is usually considered a ``dual analysis'' of \ac{FTRL} for linear losses.
\begin{lemma}
Under the assumptions of Lemma~\ref{lemma:ftrl_equality}, further assume for all $t$ that $\ell_t(\bx)=\langle \bg_t,\bx\rangle$. Then, for any $\bu \in \mathcal{V}$, we have
\begin{align}
&\sum_{t=1}^T \langle \bg_t,\bx_t-\bu\rangle \nonumber\\
&\quad= \psi_{\mathcal{V},T+1}(\bu) + \psi^\star_{\mathcal{V},1}(\boldsymbol{0}) + \sum_{t=1}^T \left[\psi_{\mathcal{V},t+1}^\star\left(-\sum_{i=1}^{t} \bg_i\right) - \psi_{\mathcal{V},t}^\star\left(-\sum_{i=1}^{t-1} \bg_i\right)+ \langle\bg_t,\bx_t\rangle\right] \nonumber \\
&\qquad -\psi_{\mathcal{V},T+1}^\star\left(-\sum_{i=1}^{T} \bg_i\right) - \psi_{\mathcal{V},T+1}(\bu) - \left\langle \sum_{i=1}^{T} \bg_i,\bu\right\rangle, \label{eq:dual_analysis_ftrl}
\end{align}
where $\psi^\star_{\mathcal{V},t}$ is the Fenchel conjugate of $\psi_{\mathcal{V},t}=\psi_t+\indicator_{\mathcal{V}}$. Also, by Fenchel--Young's inequality\index{inequality!Fenchel--Young's}, the sum of the last three terms is non-positive. In addition, the choice of $\psi_{T+1}$ does not change the iterates $\bx_1, \dots, \bx_T$ of the algorithm.
\end{lemma}
\begin{proof}
The proof is immediate by just summing up the terms on both sides.

We can also obtain it from Lemma~\ref{lemma:ftrl_equality}, observing that
\begin{align*}
F_t(\bx_t)
&= \min_{\bx \in \mathcal{V}} \ \psi_{t}(\bx) + \sum_{i=1}^{t-1} \ell_i(\bx)
= -\max_{\bx \in \mathcal{V}} \ \left\langle -\sum_{i=1}^{t-1} \bg_i,\bx\right\rangle -\psi_{t}(\bx)\\
&= - \psi_{\mathcal{V},t}^\star\left(-\sum_{i=1}^{t-1} \bg_i\right),
\end{align*}
and using the fact that $\psi^\star_{\mathcal{V},1}(\boldsymbol{0})=-\inf_{\bx \in \R^d} \ \psi_{\mathcal{V},1}(\bx)$.
\end{proof}
The function $\psi_{\mathcal{V},t}^\star$ in the above lemma is also known as the ``potential function''\index{potential function} in the online learning literature.

\begin{remark}
\label{remark:generalized_bregman}
While the regularizer $\psi_t$ in \ac{FTRL} plays the same role as the distance generating function\index{distance generating function} in \ac{OMD}, it is important to note that $\psi$ must be differentiable in \ac{OMD}. Indeed, it is necessary for the existence of the Bregman divergence\index{Bregman divergence}, and for the uniqueness of the dual map $\nabla \psi$\index{dual map}. Instead, in \ac{FTRL} the dual map $\nabla \psi$ is not used, as it can be seen in Figure~\ref{fig:dual_mapping_ftrl}. We will use this property of \ac{FTRL} in Section~\ref{sec:ftrl_composite}, when we will use regularizers that contain a non-differentiable component.
Instead, we can express the above bound using the notion of \textbf{generalized Bregman divergence}\index{Bregman divergence!generalized|textbf} $\bar{B}_{\psi}: \R^d \times \R^d \to [0,+\infty]$ defined as $\bar{B}_\psi(\bx; \btheta) = \psi(\bx)+\psi^\star(\btheta)-\langle \btheta,\bx\rangle$. By the Fenchel--Young's inequality\index{inequality!Fenchel--Young's}, $\bar{B}_\psi(\bx; \btheta)\geq 0$. Observe that we do not require $\psi$ to be differentiable. However, if $\psi$ is differentiable, then from Theorem~\ref{thm:props_fenchel} we have $B_\psi(\bx; \by) = \bar{B}_\psi(\bx;\nabla \psi(\by))$. With this definition and assuming that $\inf_{\bx \in \mathcal{V}} \psi_1(\bx) = \inf_{\bx \in \mathcal{V}} \psi_{T+1}(\bx)$, \eqref{eq:dual_analysis_ftrl} becomes
\begin{align*}
\sum_{t=1}^T \langle \bg_t,\bx_t-\bu\rangle
&= \bar{B}_{\psi_{\mathcal{V},T+1}}(\bu;\boldsymbol{0}) + \sum_{t=1}^T \left[\bar{B}_{\psi_{\mathcal{V},t+1}}\left(\bu;\btheta_{t+1}\right) - \bar{B}_{\psi_{\mathcal{V},t}}\left(\bu;\btheta_t\right) + \langle\bg_t,\bx_t\rangle\right]  \\
&\quad -\bar{B}_{\psi_{\mathcal{V},T+1}}\left(\bu; \btheta_{T+1}\right),
\end{align*}
where $\btheta_t= - \sum_{i=1}^{t-1} \bg_i$.
\end{remark}

In the next example, we can see the different behavior of \ac{FTRL} and \ac{OMD}.
\begin{example}
\label{example:omd_vs_ftrl}
Consider $\mathcal{V}=\{\bx \in \R^d: \|\bx\|_2\leq 1\}$. With \ac{OSD} with learning rate $\eta_t=\frac{1}{\sqrt{t}}$ and $\bx_1=\boldsymbol{0}$, the update is
\begin{align*}
\tilde{\bx}_{t+1} &= \bx_t -\frac{1}{\sqrt{t}} \bg_t, \\
\bx_{t+1} &= \tilde{\bx}_{t+1} \, \min\left(\frac{1}{\|\tilde{\bx}_{t+1}\|_2},1\right)~.
\end{align*}
On the other hand in \ac{FTRL} with linearized losses, we can use $\psi_t(\bx) = \frac{\sqrt{t}}{2}\|\bx\|_2^2$ and it is easy to verify that the update in \eqref{eq:ftrl_update_linear} becomes
\begin{align*}
\tilde{\bx}_{t+1} &= \frac{-\sum_{i=1}^t \bg_i}{\sqrt{t}},\\
\bx_{t+1} &= \tilde{\bx}_{t+1} \, \min\left(\frac{1}{\|\tilde{\bx}_{t+1}\|_2},1\right)~.
\end{align*}
While the regret guarantee would be the same for these two updates, from an intuitive point of view, \ac{OMD} seems to be losing a lot of potential information due to the fact that we only memorize the projected iterate.
\end{example}

\subsection{FTRL with Linearized Losses Can Be Equivalent to OMD}

Even if \ac{FTRL} and \ac{OMD} seem very different, in certain cases they are actually equivalent. Let's consider an example.

Let $\psi:\mathcal{X}\to \R$ and consider that case that $\mathcal{V}=\mathcal{X} =\dom \psi$. The output of \ac{OMD} is
\[
\bx_{t+1} = \argmin_{\bx} \ \langle \eta \bg_t, \bx\rangle + B_\psi(\bx;\bx_t)~.
\]
Assume that $\bx_{t+1} \in \interior \dom \psi$ for all $t=1, \dots, T$. This implies that $\eta \bg_t + \nabla \psi(\bx_{t+1}) - \nabla \psi(\bx_t)=\boldsymbol{0}$, that is $ \nabla \psi(\bx_{t+1}) =  \nabla \psi(\bx_{t}) -\eta \bg_t$. Assuming $\bx_1 \in \argmin_{\bx \in \mathcal{V}} \psi(\bx)$, we have
\[
\nabla \psi(\bx_{t+1}) = - \eta \sum_{i=1}^t \bg_i~.
\]
On the other hand, consider \ac{FTRL} with linearized losses with regularizers $\psi_t= \frac{1}{\eta} \psi$, then
\[
\bx_{t+1}
= \argmin_{\bx} \ \frac{1}{\eta} \psi(\bx) + \sum_{i=1}^t \langle \bg_i, \bx\rangle
= \argmin_{\bx} \ \psi(\bx) + \eta \sum_{i=1}^t \langle \bg_i, \bx\rangle~.
\]
Assuming that $\bx_{t+1} \in \interior \dom \psi$, this implies that $\nabla \psi(\bx_{t+1})=-\eta \sum_{i=1}^t \bg_i$.
Further, assuming that $\nabla \psi$ is invertible implies that the predictions of \ac{FTRL} and \ac{OMD} are the same.

This equivalence immediately gives us some intuition on the role of $\psi$ in both algorithms: the same function is inducing the Bregman divergence\index{Bregman divergence}, that is, our similarity measure, and is the regularizer in \ac{FTRL}. Moreover, the inverse of the growth rate over time of the regularizers in \ac{FTRL} takes the role of the learning rate in \ac{OMD}.

\begin{example}
Consider $\psi(\bx)=\frac{1}{2}\|\bx\|_2^2$ and $\mathcal{V}=\R^d$, then it satisfies the conditions above to have the predictions of \ac{OMD} equal to the ones of \ac{FTRL}.
\end{example}

\begin{remark}
Based on the above observation and on the observations in Section~\ref{sec:ftrl_linear_losses}, a common misunderstanding even among experts is that \ac{FTRL} and \ac{OMD} differ \emph{only} in the constrained setting. This is clearly false: the regularizer of \ac{FTRL} can vary over time in arbitrary ways, not just as a scaled version of a fixed regularizer. Indeed, general time-varying regularizers are used, for example, in parameter-free algorithms (Chapter~\ref{ch:parameterfree}) and in Vovk--Azoury--Warmuth forecaster (Section~\ref{sec:vaw}).
\end{remark}
\index{Follow-the-Regularized-Leader algorithm!with linear losses|)textbf}

\index{Follow-the-Regularized-Leader algorithm!with local norms|(textbf}
\section{FTRL Regret Bound using Local Norms}
\label{sec:ftrl_local_norms}

In Lemma~\ref{lemma:ftrl_strongly_1}, strong convexity basically tells us that the losses plus regularizer have some minimum curvature in all directions. However, as in the \ac{OMD} case, it turns out that we can get a regret upper bound using again the notion of local norms.\index{norm!local}

\begin{lemma}
\label{lemma:ftrl_local_norms}
Under the same assumptions of Lemma~\ref{lemma:ftrl_equality}, assume $\psi_1, \dots, \psi_T$ twice differentiable and with the Hessian positive definite in the interior of their domains. Also, assume $\ell_t(\bx)=\langle \bg_t, \bx \rangle, \ t=1, \dots, T$, for arbitrary vectors $\bg_t$. Define $\|\bx\|_{\bA}:= \sqrt{\bx^\top \bA \bx}$.
For $t=1,\dots,T$, assume that $\bx_t \in \interior \dom \psi_t$ and that $\tilde{\bx}_{t+1} := \argmin_{\bx \in \R^d} \ \langle \bg_t, \bx\rangle + B_{\psi_t}(\bx; \bx_t)$ exists and $\tilde{\bx}_{t+1} \in \interior \dom \psi_t$. Then, there exists $\bz_t$ on the line segment between $\bx_t$ and $\bx_{t+1}$ and  $\bz'_t$ on the line segment between $\bx_t$ and $\tilde{\bx}_{t+1}$, such that the following inequality holds for any $\bu \in \mathcal{V}$
\begin{align*}
\sum_{t=1}^T \langle \bg_t,\bx_t - \bu\rangle
&\leq \psi_{T}(\bu) - \min_{\bx \in \mathcal{V}} \ \psi_{1}(\bx)
+ \sum_{t=1}^{T-1} \left[\psi_t(\bx_{t+1}) - \psi_{t+1}(\bx_{t+1})\right]\\
&\quad +\sum_{t=1}^T \min\left(\frac{\|\bg_t\|^2_{(\nabla^2 \psi_t(\bz_t))^{-1}}}{2}, \frac{\|\bg_t\|^2_{(\nabla^2 \psi_t(\bz'_t))^{-1}}}{2}\right)
\end{align*}
\end{lemma}
\begin{proof}
We start from Lemma~\ref{lemma:ftrl_equality} with $\psi_{T+1}=\psi_T$ and observe that
\begin{align*}
F_t&(\bx_t) - F_{t+1}(\bx_{t+1}) + \ell_t(\bx_t)\\
&= F_t(\bx_t) - F_{t}(\bx_{t+1}) + \ell_t(\bx_t) - \ell_t(\bx_{t+1}) + \psi_t(\bx_{t+1}) - \psi_{t+1}(\bx_{t+1})~.
\end{align*}

Then, observe that $\psi_t$ are strictly convex\index{function!strictly convex} because the Hessians are positive definite. Hence, they can be used to define Bregman divergences\index{Bregman divergence}.
Moreover, from the optimality condition of Theorem~\ref{thm:constr_opt_condition} for $\bx_t$, we have
\[
\langle \nabla F_t(\bx_t), \bv - \bx_t\rangle \geq 0, \quad \forall \bv \in \mathcal{V}~.
\]
Hence, in particular, we have
\[
\langle \nabla F_t(\bx_t), \bx_{t+1} - \bx_t\rangle \geq 0~.
\]
Using this inequality, we have
\[
B_{F_t}(\bx_{t+1};\bx_t)
= F_t(\bx_{t+1}) - F_t(\bx_t) - \langle \nabla F_t(\bx_t), \bx_{t+1} - \bx_t\rangle
\leq F_t(\bx_{t+1}) - F_t(\bx_t)~.
\]
This last inequality implies that
\begin{align*}
F_t(\bx_t) - F_{t}(\bx_{t+1}) + \ell_t(\bx_t) - \ell_t(\bx_{t+1})
&= F_t(\bx_t) - F_{t}(\bx_{t+1}) + \langle \bg_t, \bx_t - \bx_{t+1}\rangle \\
&\leq \langle \bg_t, \bx_t - \bx_{t+1}\rangle - B_{F_t}(\bx_{t+1};\bx_t)~.
\end{align*}
Bregman divergences\index{Bregman divergence} are independent of linear terms, so $B_{F_t}(\bx_{t+1};\bx_t)=B_{\psi_t}(\bx_{t+1};\bx_t)$.
From the Taylor's theorem, $B_{\psi_t}(\bx_{t+1};\bx_t) = \frac{1}{2}(\bx_{t+1}-\bx_t)^\top \nabla^2 \psi_t(\bz_t) (\bx_{t+1}-\bx_t)$, where $\bz_t$ is on the line segment between $\bx_t$ and $\bx_{t+1}$. Observe that this is $\frac12 \|\bx_{t+1}-\bx_t\|^2_{\nabla^2 \psi_t(\bz_t)}$ and it is indeed a norm because we assumed the Hessian of $\psi_t$ to be positive definite. Hence, by Fenchel--Young's inequality\index{inequality!Fenchel--Young's} and Examples~\ref{example:dual_norm_a} and \ref{example:conj_squared_norm}, we have
\begin{align}
F_t&(\bx_t) - F_{t}(\bx_{t+1}) + \ell_t(\bx_t) - \ell_t(\bx_{t+1}) \nonumber \\
&\leq \langle \bg_t, \bx_t - \bx_{t+1}\rangle - B_{\psi_t}(\bx_{t+1};\bx_t) \label{eq:local_norm_1}\\
&\leq \frac{1}{2}\|\bg_t\|^2_{(\nabla^2 \psi_t(\bz_t))^{-1}} + \frac{1}{2}(\bx_{t+1}-\bx_t)^\top \nabla^2 \psi_t(\bz_t) (\bx_{t+1}-\bx_t) - B_{\psi_t}(\bx_{t+1};\bx_t) \nonumber \\
&=\frac{1}{2}\|\bg_t\|^2_{(\nabla^2 \psi_t(\bz_t))^{-1}}, \nonumber
\end{align}
that gives the first term in the minimum.

For the second term in the minimum, we start from \eqref{eq:local_norm_1} to get
\begin{align*}
\langle \bg_t, \bx_t - \bx_{t+1}\rangle - B_{\psi_t}(\bx_{t+1};\bx_t)
&\leq \max_{\bx \in \R^d} \ \langle \bg_t, \bx_t - \bx\rangle - B_{\psi_t}(\bx;\bx_t) \\
&= \langle \bg_t, \bx_t - \tilde{\bx}_{t+1}\rangle - B_{\psi_t}(\tilde{\bx}_{t+1};\bx_t)~.
\end{align*}
Then, we proceed as in the first bound.
\end{proof}

Observe that $\bz'_t$ is defined as a sort of \ac{OMD} update using $\psi_t$ as distance generating function\index{distance generating function}.
\index{Follow-the-Regularized-Leader algorithm!with local norms|)textbf}

\section{Example of FTRL: Exponentiated Gradient without Knowing $T$}
\label{sec:ftrl_eg}

\index{Exponentiated Gradient algorithm!FTRL version|(textbf}
As we did in Section~\ref{sec:omd_eg} for \ac{OMD}, let's see an example of an instantiation of \ac{FTRL} with linearized losses to have the \ac{FTRL} version of \ac{EG}.

Let $\mathcal{V}=\{\bx \in \R^d: \|\bx\|_1=1, x_i\geq0\}$ and the sequence of loss functions $\ell_t:\R^d \to (-\infty, +\infty]$ be convex, subdifferentiable on $\mathcal{V}$, and $L_\infty$-Lipschitz with respect to the L$_1$ norm.
Let $\psi:R^d_{\geq 0}\to \R$ be defined as $\psi(\bx)=\sum_{i=1}^d x_i \ln x_i$, where we define $0\ln0=0$. Set $\psi_t(\bx)=\alpha L_\infty \sqrt{t} \psi(\bx)$, that is $\alpha L_\infty \sqrt{t}$-strongly convex with respect to the $L_1$ norm by Lemma~\ref{lemma:entropy_strongly_convex}, where $\alpha>0$ is a parameter of the algorithm.

Given that the regularizers are strongly convex and defining $\psi_{\mathcal{V},t}=\psi_t+\indicator_{\mathcal{V}}$, from \eqref{eq:ftrl_update_linear} we have
\[
\bx_t = \nabla \psi_{\mathcal{V},t}^\star \left(-\sum_{i=1}^{t-1} \bg_i\right)~.
\]
We already saw in Section~\ref{sec:omd_eg} that $\psi_{\mathcal{V}}^\star(\btheta) = \ln \left(\sum_{i=1}^d \exp(\theta_i)\right)$, that implies that $\psi^\star_{\mathcal{V},t}(\btheta)=\alpha L_\infty \sqrt{t} \ln\left(\sum_{i=1}^d \exp\left(\frac{\theta_i}{\alpha L_\infty \sqrt{t}}\right)\right)$. So, running \ac{FTRL} with linearized losses, we have that
\[
x_{t,j} = \frac{\exp\left(-\frac{\sum_{k=1}^{t-1} g_{k,j}}{\alpha L_\infty \sqrt{t}}\right)}{\sum_{i=1}^d \exp\left(-\frac{\sum_{k=1}^{t-1} g_{k,i}}{\alpha L_\infty \sqrt{t}}\right)}, \quad \forall j=1, \dots, d,
\]
where $\bg_t \in \partial \ell_t(\bx_t)$.
Note that this is exactly the same update of \ac{EG} based on \ac{OMD}, but here we are effectively using time-varying learning rates.

We also get that the regret guarantee is
\begin{align}
\sum_{t=1}^T \ell_t(\bx_t) - \sum_{t=1}^T \ell_t(\bu)
&\leq L_\infty \sqrt{T} \alpha\left(\sum_{i=1}^d u_i \ln u_i + \ln d\right) + \frac{1}{2\alpha L_\infty} \sum_{t=1}^T \frac{\|\bg_t\|^2_\infty}{\sqrt{t}} \label{eq:example_ftrl_eg_1} \\
&\leq L_\infty \sqrt{T}\left(\alpha\left(\sum_{i=1}^d u_i \ln u_i + \ln d\right)+\frac{1}{\alpha}\right) \nonumber \\
&\leq L_\infty \sqrt{T}\left(\alpha \ln d +\frac{1}{\alpha}\right), \quad \forall \bu \in \mathcal{V}, \nonumber
\end{align}
where we used the fact that using $\psi_t(\bx)=\alpha L_\infty \sqrt{t} \psi(\bx)$ and $\psi_t(\bx)=\alpha L_\infty \sqrt{t} (\psi(\bx)-\min_{\bx} \psi(\bx))$ give the same iterates in \ac{FTRL}. The optimal choice of $\alpha$ to minimize the regret upper bound is $\frac{1}{\sqrt{\ln d}}$.
This regret guarantee is similar to the one we proved for \ac{OMD}, but with an important difference: we do not have to know in advance the number of rounds $T$. In \ac{OMD}, a similar bound would be vacuous because it would depend on the $\max_{\bu,\bx \in \mathcal{V}} B_\psi(\bu;\bx)$ that is infinite.

\index{Follow-the-Regularized-Leader algorithm!with local norms|(textbf}
As we did in the \ac{OMD} case, we can also get a bound using the local norms.\index{norm!local} Let's use the additional assumption that $g_{t,i}\geq0$, for all $t=1, \dots,T$ and $i=1, \dots,d$. Using Lemma~\ref{lemma:ftrl_local_norms}, we have for all $\bu \in \mathcal{V}$ that
\begin{align*}
\sum_{t=1}^T \ell_t(\bx_t) - \sum_{t=1}^T \ell_t(\bu)
&\leq L_\infty \sqrt{T} \alpha\left(\sum_{i=1}^d u_i \ln u_i + \ln d\right) + \frac12 \sum_{t=1}^T \|\bg_t\|^2_{(\nabla^2 \psi_t(\bz'_t))^{-1}} \\
&= L_\infty \sqrt{T} \alpha\left(\sum_{i=1}^d u_i \ln u_i + \ln d\right) + \frac{1}{2\alpha L_\infty} \sum_{t=1}^T \frac{\|\bg_t\|^2_{(\nabla^2 \psi(\bz'_t))^{-1}}}{\sqrt{t}},
\end{align*}
where $\bz'_t$ is on the line segment between $\bx_t$ and $\tilde{\bx}_{t+1}$. In this case, it is easy to calculate $\tilde{x}_{t+1,i}$ as $x_{t,i}\exp(-\tfrac{g_{t,i}}{\alpha L_{\infty} \sqrt{t}})$ for $i=1, \dots,d$. Moreover, $\nabla^2 \psi(\bz'_t)$ is a diagonal matrix whose elements on the diagonal are $\frac{1}{z'_{t,i}}, i=1, \dots, d$. Hence, we have that
\begin{equation}
\label{eq:eg_ftrl_local_norm}
\|\bg_t\|^2_{(\nabla^2 \psi(\bz'_t))^{-1}}
= \sum_{i=1}^d g_{t,i}^2 z'_{t,i}
\leq \sum_{i=1}^d g_{t,i}^2 x_{t,i},
\end{equation}
that is less than or equal to the terms $\|\bg_t\|^2_\infty$ we have in \eqref{eq:example_ftrl_eg_1}.
\index{Follow-the-Regularized-Leader algorithm!with local norms|)textbf}
\index{Exponentiated Gradient algorithm!FTRL version|)textbf}

\section{Example of FTRL: AdaHedge}

\index{AdaHedge algorithm|(textbf}
In this section, we explain a variation of the \ac{EG} algorithm, called \emph{AdaHedge}.
The basic idea is to design an algorithm that is adaptive to the sum of the squared L$_\infty$ norm of the losses, without any prior information on the range of the losses.

First, consider the case in which we use as a constant regularizer the negative Shannon entropy\index{entropy!negative Shannon} $\psi_t(\bx) = \lambda \sum_{i=1}^d x_i \ln x_i$ defined over $\R^d_{\geq0}$, where $\lambda>0$ will be determined in the following. Set $\mathcal{V}=\Delta^{d-1}$, the probability simplex in $\R^d$\index{probability simplex}.
Using \ac{FTRL} with linear losses with this regularizer, for any $\bu \in \Delta^{d-1}$, we immediately obtain
\begin{align*}
\Regret_T(\bu)
&\leq \lambda \left(\ln d + \sum_{i=1}^d u_i \ln u_i\right) + \sum_{t=1}^T \left(F_t(\bx_t) - F_{t+1}(\bx_{t+1}) + \langle \bg_t, \bx_t\rangle\right) \\
&\leq \lambda \ln d + \sum_{t=1}^T \left(F_t(\bx_t) - F_{t+1}(\bx_{t+1}) + \langle \bg_t, \bx_t\rangle\right),
\end{align*}
where we upper bounded the negative entropy of $\bu$ with 0.
Using the strong convexity of the regularizer with respect to the L$_1$ norm and Lemma~\ref{lemma:ftrl_strongly_1}, we further upper bound it as
\[
\Regret_T(\bu)
\leq \lambda \ln d + \sum_{t=1}^T \frac{\|\bg_t\|^2_\infty}{2\lambda}~.
\]
This suggests that the optimal $\lambda$ should be $\lambda = \sqrt{\frac{\sum_{t=1}^T \|\bg_t\|^2_\infty}{2\ln d}}$. However, as we have seen in Section~\ref{sec:lstar}, this choice is not feasible. Hence, exactly as we did in Section~\ref{sec:lstar}, we might think of using an online version of this choice
\begin{equation}
\label{eq:adaftrl}
\psi_t(\bx)= \lambda_t \sum_{i=1}^d x_i \ln x_i
\quad \text{where} \quad
\lambda_t = \frac{1}{\alpha}\sqrt{\sum_{i=1}^{t-1} \|\bg_i\|^2_\infty},
\end{equation}
where $\alpha>0$ is a constant that will be determined later. An important property of such a choice is that it gives rise to an algorithm that is scale-free (Definition~\ref{def:scale-free}), that is, its predictions $\bx_t$ are invariant to the scaling of the losses by any constant factor\index{algorithm!scale-free}. This is easy to see because
\[
x_{t,j} \propto \exp\left(- \frac{\alpha \sum_{i=1}^{t-1} g_{i,j}}{\sqrt{\sum_{i=1}^{t-1} \|\bg_i\|^2_\infty}}\right), \quad \forall j=1,\dots, d~.
\]

This choice makes the regularizer non-decreasing over time and immediately gives us
\[
\Regret_T(\bu)
\leq \lambda_T \ln d + \sum_{t=1}^T \frac{\|\bg_t\|^2_\infty}{2\lambda_{t}}
=  \frac{\ln d}{\alpha} \sqrt{\sum_{t=1}^{T} \|\bg_t\|^2_\infty} + \alpha \sum_{t=1}^T \frac{\|\bg_t\|^2_\infty}{2 \sqrt{\sum_{i=1}^{t-1} \|\bg_i\|^2_\infty}}~.
\]
At this point, we might be tempted to use Lemma~\ref{lemma:sum_integral_bounds} to upper bound the last sum in the upper bound, but unfortunately, we cannot! Indeed, the denominator does not contain the term $\|\bg_t\|_\infty^2$. We might add a constant to $\lambda_t$, but that would destroy the scale-freeness of the algorithm. However, it turns out that we can still prove our bound without any change to the regularizer. The key observation is that we can bound the term $F_t(\bx_t) - F_{t+1}(\bx_{t+1}) + \langle \bg_t, \bx_t\rangle$ in two different ways. One way is using Lemma~\ref{lemma:ftrl_strongly_1}. The other one is
\begin{align*}
F_t(\bx_t) - F_{t+1}(\bx_{t+1}) + \langle \bg_t, \bx_t\rangle
&\leq F_t(\bx_{t+1}) - F_{t+1}(\bx_{t+1}) + \langle \bg_t, \bx_t\rangle\\
&=\psi_t(\bx_{t+1}) - \psi_{t+1}(\bx_{t+1}) - \langle \bg_t, \bx_{t+1}\rangle +\langle \bg_t, \bx_t\rangle\\
&\leq -\langle \bg_t, \bx_{t+1}\rangle + \langle \bg_t, \bx_t\rangle
\leq  2\|\bg_t\|_\infty,
\end{align*}
where we used the definition of $\bx_{t+1}$ and the fact that the regularizer is non-decreasing over time.
So, we can now write
\begin{align*}
\sum_{t=1}^T F_t(\bx_t) - F_{t+1}(\bx_{t+1}) + \langle \bg_t, \bx_t\rangle
&\leq \sum_{t=1}^T \min\left( \frac{\alpha \|\bg_t\|_\infty^2}{2 \sqrt{\sum_{i=1}^{t-1} \|\bg_i\|_\infty^2}}, 2\|\bg_t\|_\infty \right) \\
&= 2 \sum_{t=1}^T \sqrt{\min\left( \frac{\alpha^2 \|\bg_t\|_\infty^4}{16\sum_{i=1}^{t-1} \|\bg_i\|_\infty^2}, \|\bg_t\|^2_\infty \right)} \\
&\leq 2\sum_{t=1}^T \sqrt{\frac{2}{\frac{16\sum_{i=1}^{t-1} \|\bg_i\|_\infty^2}{\alpha^2 \|\bg_t\|_\infty^4} + \frac{1}{\|\bg_t\|^2_\infty}}   } \\
&=2 \sum_{t=1}^T \sqrt{2}\frac{\alpha \|\bg_t\|_\infty^2}{\sqrt{\alpha^2 \|\bg_t\|^2_\infty+16\sum_{i=1}^{t-1} \|\bg_i\|_\infty^2}},
\end{align*}
where we used the fact that the minimum between two numbers is less than their harmonic mean.
Assuming $\alpha \geq 4$ and using Lemma~\ref{lemma:sum_integral_bounds}, we have
\[
\sum_{t=1}^T F_t(\bx_t) - F_{t+1}(\bx_{t+1}) + \langle \bg_t, \bx_t\rangle
\leq \frac{\sqrt{2}}{2}\sum_{t=1}^T \frac{\alpha \|\bg_t\|_\infty^2}{\sqrt{\sum_{i=1}^{t} \|\bg_i\|_\infty^2}}
\leq \alpha\sqrt{2 \sum_{t=1}^{T} \|\bg_t\|_\infty^2}
\]
and
\begin{equation}
\label{eq:adahedge1}
\Regret_T(\bu)
\leq  \left(\frac{\ln d}{\alpha} + \alpha \sqrt{2}\right)\sqrt{\sum_{t=1}^{T} \|\bg_t\|^2_\infty}, \quad \forall \bu \in \Delta^{d-1}~.
\end{equation}
The bound and the assumption on $\alpha$ suggest to set $\alpha = \max(4, 2^{-1/4} \sqrt{\ln d})$.

We might consider ourselves happy, but there is a clear problem in the above algorithm: the choice of $\lambda_t$ in the time-varying regularizer strictly depends on our upper bound. So, a loose bound will result in a poor choice of the regularization! In general, every time we use a part of the proof in the design of an algorithm, we cannot expect an exciting empirical performance, unless our upper bound is really tight. So, can we design a better regularizer? Well, we need a better upper bound!

Let's consider a generic regularizer $\psi_t(\bx)=\lambda_t \psi(\bx)$ and its corresponding \ac{FTRL} with linear losses regret upper bound
\[
\Regret_T(\bu)
\leq \lambda_T (\psi(\bu) -\inf_{\bx \in \mathcal{V}} \psi(\bx))+ \sum_{t=1}^T \left(F_t(\bx_t) - F_{t+1}(\bx_{t+1}) + \langle \bg_t, \bx_t\rangle\right),
\]
where we assume $\lambda_t$ to be non-decreasing in time.

Now, observe that the sum is unlikely to disappear for this kind of algorithm, so we could try to make the term $\lambda_T (\psi(\bu) -\inf_{\bx \in \mathcal{V}} \psi(\bx))$ of the same order as the sum. So, we would like to set $\lambda_t$ of the same order of $\sum_{i=1}^t\left(F_i(\bx_i) - F_{i+1}(\bx_{i+1}) + \langle \bg_i, \bx_i\rangle\right)$. However, this approach would cause an annoying recurrence. So, using the fact that $\lambda_t$ is non-decreasing, let's upper bound the terms in the sum just a little bit:
\begin{align*}
F_t&(\bx_t) - F_{t+1}(\bx_{t+1}) + \langle \bg_t, \bx_t\rangle\\
&= F_t(\bx_t)- \lambda_{t+1} \psi(\bx_{t+1}) - \sum_{i=1}^t \langle \bg_i, \bx_{t+1}\rangle + \langle \bg_t, \bx_t\rangle\\
&\leq F_t(\bx_t) - \lambda_{t} \psi(\bx_{t+1}) - \sum_{i=1}^t \langle \bg_i, \bx_{t+1}\rangle + \langle \bg_t, \bx_t\rangle\\
&\leq F_t(\bx_t) - \min_{\bx \in \mathcal{V}} \left(\lambda_{t} \psi(\bx) + \sum_{i=1}^t \langle \bg_i, \bx\rangle\right) + \langle \bg_t, \bx_t\rangle
:=\delta_t~.
\end{align*}
Now, we can set $\lambda_t = \frac{1}{\alpha^2} \sum_{i=1}^{t-1} \delta_i$ for $t\geq 2$, $\lambda_1=0$, and $\bx_1 = \argmin_{\bx \in \mathcal{V}} \ \psi(\bx)$.
This immediately implies that
\[
\Regret_T(\bu)
\leq \left(\psi(\bu) -\inf_{\bx \in \mathcal{V}} \psi(\bx) + \alpha^2\right)\lambda_{T+1}~.
\]
Setting $\psi$ to be equal to the negative entropy\index{entropy!negative Shannon}, we get an algorithm known as \textbf{AdaHedge}.

With this choice of the regularizer, we can simplify a bit the expression of $\delta_t$. For $\lambda_t=0$, we have $\delta_t = -\max_{j=1,\dots,d} \theta_{t,j} + \max_{j=1,\dots,d} \theta_{t+1,j} +\langle \bg_t,\bx_t\rangle$. Instead, for $\lambda_t>0$, using the properties of the Fenchel conjugates, we have that
\[
\delta_t
=\lambda_t \ln \frac{\sum_{j=1}^d \exp\left(\theta_{t+1,j}/\lambda_t\right)}{\sum_{j=1}^d \exp\left(\theta_{t,j}/\lambda_t\right)}+\langle \bg_t,\bx_t\rangle
=\lambda_t \ln \left(\sum_{j=1}^d x_{t,j}\exp\left(-g_{t,j}/\lambda_t\right)\right)+\langle \bg_t,\bx_t\rangle~.
\]
Overall, we get the pseudo-code of AdaHedge in Algorithm~\ref{alg:adahedge}.

\begin{algorithm}[t]
\caption{AdaHedge}
\label{alg:adahedge}
\begin{algorithmic}[1]
{
    \REQUIRE{$\alpha>0$}
    \STATE{$\lambda_1=0$}
    \STATE{$\bx_1=[1/d, \dots, 1/d] \in \Delta^{d-1}$}
    \STATE{$\btheta_{1}=\boldsymbol{0} \in \R^d$}
    \FOR{$t=1$ {\bfseries to} $T$}
    \STATE{Output $\bx_t$}
    \STATE{Receive $\bg_t \in \R^d$ and pay $\langle \bg_t,\bx_t\rangle$}
    \STATE{Update $\btheta_{t+1} = \btheta_{t} - \bg_{t}$}
    \STATE{Set $\delta_t= \begin{cases}
    -\max_{j=1,\dots,d} \theta_{t,j} + \max_{j=1,\dots,d} \theta_{t+1,j} +\langle \bg_t,\bx_t\rangle, & \lambda_t=0\\
    \lambda_t \ln \left(\sum_{j=1}^d x_{t,j}\exp\left(-g_{t,j}/\lambda_t\right)\right)+\langle \bg_t,\bx_t\rangle, & \text{otherwise}
    \end{cases}$}
    \STATE{Update $\lambda_{t+1}=\lambda_t + \frac{1}{\alpha^2} \delta_t$}
    \STATE{Set $\bx_{t+1}=\begin{cases}
    [1/d,\dots,1/d], & \lambda_{t+1}=0\\
    x_{t+1,j} \propto \exp\left(\frac{\theta_{t+1,j}}{\lambda_{t+1}}\right), j=1, \dots, d, & \text{otherwise}
    \end{cases}$}
    \ENDFOR
}
\end{algorithmic}
\end{algorithm}

So, now we need an upper bound for $\lambda_{T+1}$.
Observe that $\lambda_{t+1}=\lambda_{t} + \frac{1}{\alpha^2} \delta_t$. Moreover, as we have done before, we can upper bound $\delta_t$ in two different ways. In fact, from Lemma~\ref{lemma:ftrl_strongly_1}, we have $\delta_t \leq \frac{\|\bg_t\|^2_\infty}{2\lambda_t}$ whenever $\lambda_t>0$. Also, denoting by $\tilde{\bx}_{t+1} = \argmin_{\bx \in \mathcal{V}} \ \lambda_{t} \psi(\bx) + \sum_{i=1}^t \langle \bg_i, \bx\rangle$, we have
\begin{align*}
\delta_t
&= F_t(\bx_t) - \min_{\bx \in \mathcal{V}} \left(\lambda_{t} \psi(\bx) + \sum_{i=1}^t \langle \bg_i, \bx\rangle\right) + \langle \bg_t, \bx_t\rangle\\
&= F_t(\bx_t) - \lambda_{t} \psi(\tilde{\bx}_{t+1}) - \sum_{i=1}^t \langle \bg_i, \tilde{\bx}_{t+1}\rangle + \langle \bg_t, \bx_t\rangle \\
&\leq F_t(\tilde{\bx}_{t+1}) - \lambda_{t} \psi(\tilde{\bx}_{t+1}) - \sum_{i=1}^t \langle \bg_i, \tilde{\bx}_{t+1}\rangle + \langle \bg_t, \bx_t\rangle \\
&= -\langle \bg_t, \tilde{\bx}_{t+1}\rangle + \langle \bg_t, \bx_t\rangle
\leq 2\|\bg_t\|_\infty~.
\end{align*}
Hence, we have
\begin{align*}
\lambda_1 &= 0, \\
\lambda_2 &\leq \frac{2 \|\bg_1\|_\infty}{\alpha^2},\\
\lambda_{t+1} &= \lambda_t + \frac{\delta_t}{\alpha^2} \leq \lambda_t + \frac{1}{\alpha^2} \min\left(2\|\bg_t\|_\infty,\frac{\|\bg_t\|^2_\infty}{2\lambda_t} \right), \quad \forall t\geq 1,
\end{align*}
where we use the convention that $\frac{\|\bg_t\|^2_\infty}{2\lambda_t}=+\infty$ when $\lambda_t=0$.
We can solve this recurrence using the following lemma, where $\Delta_t=\lambda_{t+1}$ and $a_t=\|\bg_t\|_\infty$.
\begin{lemma}
\label{lemma:recurrence-solution}
Let $\{a_t\}_{t=1}^\infty$ be any sequence of non-negative real numbers.
Suppose that $\{\Delta_t\}_{t=0}^\infty$ is a non-decreasing sequence of non-negative real numbers satisfying
\[
\Delta_0 = 0 \qquad \text{and} \qquad
\Delta_t \le \Delta_{t-1} + \min \left\{ b a_t, \ c \frac{a_t^2}{2\Delta_{t-1}} \right\} \quad \text{for any $t \ge 1$}~.
\]
When $\Delta_{t-1}=0$, we use the convention $\frac{a_t^2}{2\Delta_{t-1}}=+\infty$.
Then, for any $T \ge 0$, $\Delta_T \le \sqrt{(b^2 + c)\sum_{t=1}^T a_t^2}$.
\end{lemma}
\begin{proof}
Observe that
\[
\Delta_T^2
= \sum_{t=1}^T (\Delta_t^2 - \Delta_{t-1}^2)
= \sum_{t=1}^T \left[(\Delta_t - \Delta_{t-1})^2 + 2 (\Delta_t - \Delta_{t-1}) \Delta_{t-1}\right]~.
\]
We bound each term in the sum separately.
The left term of the minimum inequality in the definition of $\Delta_t$ gives
$ (\Delta_t - \Delta_{t-1})^2 \le b^2 a_t^2$, while the right term gives $2 (\Delta_t - \Delta_{t-1}) \Delta_{t-1} \le c a_t^2$.
The latter inequality is also true when $\Delta_{t-1}=0$.
So, we conclude $\Delta_T^2 \leq (b^2+c) \sum_{t=1}^T a_t^2$.
\end{proof}

Overall, for any $\bu \in \Delta^{d-1}$, we got
\begin{align*}
\Regret_T(\bu)
&\leq \left(\psi(\bu) -\inf_{\bx \in \mathcal{V}} \psi(\bx) + \alpha^2\right)\lambda_{T+1}\\
&\leq \left(\frac{\psi(\bu) -\inf_{\bx \in \mathcal{V}} \psi(\bx)}{\alpha^2} + 1\right)\sqrt{\left(4+\alpha^2\right) \sum_{t=1}^T \|\bg_t\|_\infty^2},
\end{align*}
and setting $\alpha=\sqrt{\ln d}$, we have
\[
\Regret_T(\bu)
\leq 2\sqrt{\left(4+\ln d\right) \sum_{t=1}^T \|\bg_t\|_\infty^2}~.
\]

Note that this is roughly the same regret in \eqref{eq:adahedge1}, but the very important difference is that this new regret bound depends on the much tighter quantity $\lambda_{T+1}$, that we upper bounded with a term proportional to $\sqrt{\sum_{t=1}^T \|\bg_t\|_\infty^2}$, but it can be much smaller than that. For example, $\lambda_{T+1}$ could be upper bounded using the tighter local norms, see Section~\ref{sec:ftrl_eg}. Instead, in the first solution, the regret will always be dominated by the term $\sqrt{\sum_{t=1}^T \|\bg_t\|_\infty^2}$ because we explicitly use it in the regularizer.

\begin{remark}
There is an important lesson to be learned from AdaHedge: the regret is not the full story, and algorithms with the same worst-case guarantee can exhibit vastly different empirical behaviours. Unfortunately, this message is rarely heard, and sometimes one focuses too much on the worst-case guarantee rather than on the empirical performance. Even worse, sometimes people favor algorithms with a ``more elegant analysis'' completely ignoring the likely worse empirical performance.
\end{remark}
\index{AdaHedge algorithm|)textbf}

\section{Example of FTRL: Group Norms}

Suppose you want to do linear regression and you have $n$ different groups of features in $\R^d$ associated with each sample. For example, they could represent different modalities (audio, video, etc.) associated with the same sample. You could just concatenate all the features in a long vector and learn with them using \ac{FTRL} with a squared L$_2$ regularizer. Alternatively, we could use \ac{FTRL} with $p$-norm (see Problem~\ref{exercise:pnorm_ftrl}), and tune $p$ to have a logarithmic dependence on $n\, d$. However, there is something in between: we may want to learn a linear combination of predictors, one for each group, trying to use a very large number of groups and paying a small price for it. In particular, we might not want $d$ to appear in the bound because it might be arbitrarily large, potentially infinity. This is the equivalent of the \ac{LEA} setting, where each expert is now a linear predictor. In fact, one might be tempted to use one online learning algorithm on each group and a \ac{LEA} algorithm on top that combines the predictions. However, this approach is very clunky because the \ac{LEA} algorithm and the linear predictors are learning somehow independently, hence we can do much better. In fact, we can use an L$_2$ norm over each group and a $p$-norm over the groups with $p$ very close to 1. In this way, we expect a regret upper bound that depends on the maximum L$_2$ norm of the subgradients with respect to each single group, but no explicit dependence on $d$. Let's see how this would work.

First, we introduce the definition of \emph{group norms}.
\begin{definition}
Let $\bX = \begin{bmatrix}\bx_1 &\bx_2 &\dots &\bx_n\end{bmatrix}$ be a $d \times n$ real matrix with columns $\bx_i \in \R^d$. Given norms $\Psi$ and $\Phi$ on $\R^d$ and $\R^n$, we define the \textbf{group norm}\index{norm!group|textbf} $\|\bX\|_{\Psi,\Phi}$ as
\[
\|\bX\|_{\Psi,\Phi}
:= \Phi([\Psi(\bx_1), \dots, \Psi(\bx_n)])~.
\]
\end{definition}

It is possible to check that this is indeed a norm, and we can also calculate the dual norm.
We will need an additional definition.
\begin{definition}
We say that $f:\R^d \to \R$ is \textbf{absolutely symmetric}\index{function!absolute symmetric|textbf} if it is invariant under permutations and sign changes of the components of its input.
\end{definition}

\begin{lemma}[{\citealp[Lemma 17]{KakadeSST09}}]
\label{lemma:dual_group_norm}
Let $\Phi$ be an absolutely symmetric norm on $\R^n$. Then,
\[
(\| \cdot \|_{\Psi,\Phi})_\star
= \| \cdot \|_{\Psi^\star , \Phi^\star}~.
\]
\end{lemma}
From this lemma and Example~\ref{example:conj_squared_norm}, we immediately have that the Fenchel conjugate of $\frac{1}{2}\|\cdot\|^2_{\Psi,\Phi}$ is $\frac{1}{2}\|\cdot\|^2_{\Psi^\star,\Phi^\star}$.

We can also construct strongly convex functions based on group norms.
\begin{theorem}[{\citealp[Theorem 18]{KakadeSST09}}]
Let $\Psi, \Phi$ be absolutely symmetric norms on $\R^d$, $\R^n$. Let $\Phi^2 \circ \sqrt{} : \R^n \to (-\infty, +\infty]$ denote
the following function,
\[
(\Phi^2 \circ \sqrt{})(\bx) := \Phi^2(\sqrt{x_1}, \dots , \sqrt{x_n})~.
\]
Suppose, $\Phi^2 \circ \sqrt{}$ satisfies the properties of a norm on $\R_{\geq 0}^n$. Further, let the functions $\Psi^2$ and $\Phi^2$ be $\lambda_1$- and $\lambda_2$-smooth\index{function!smooth} with respect to $\Psi$ and $\Phi$ respectively. Then, $\| \cdot \|^2_{\Psi, \Phi}$ is $(\lambda_1 + \lambda_2)$-smooth\index{function!smooth} with respect to $\| \cdot \|_{\Psi, \Phi}$.
\end{theorem}

We can now specialize this result to \emph{group $p$-norms}.
\begin{corollary}
Let $q, s \geq 2$. The function $\frac12 \| \cdot \|^2_{q,s}$ is $(q + s - 2)$-smooth\index{function!smooth} with respect to $\| \cdot \|_{q,s}$ on $\R^{d \times n}$.
\end{corollary}

Assuming $n\geq 2$, the above suggests to use \ac{FTRL} with linearized losses and regularizer $\psi_t(\bX)=\frac{\alpha \sqrt{t}}{2}\|\bX\|^2_{2,p}$, where $1 < p\leq 2$ and $\alpha>0$ will be decided in the following. From the above and Example~\ref{example:conj_squared_norm}, the dual is $\psi_t^\star(\bTheta)=\frac{1}{2\alpha \sqrt{t}}\|\bTheta\|^2_{2,q}$ and it is $\frac{q}{\alpha \sqrt{t}}$-smooth\index{function!smooth}, where $\frac{1}{p}+\frac{1}{q}=1$. Moreover,
\[
[\nabla \psi_t^\star(\bTheta)]_i
=\frac{\btheta_i \|\btheta_i\|_2^{q-2}}{\alpha \sqrt{t} (\sum_{j=1}^n \|\btheta_j\|_2^q)^{1-2/q}}, \quad i=1, \dots, n~.
\]

For simplicity, we will use the absolute loss as a loss function.
We are now ready to introduce Algorithm~\ref{alg:ftrl_groupnorm}.

\begin{algorithm}[t]
\caption{FTRL with Group Norms for Linear Regression}
\label{alg:ftrl_groupnorm}
\begin{algorithmic}[1]
{
    \REQUIRE{$q\geq 2, \alpha>0$}
    \STATE{$\btheta_{1,i}=\bx_{1,i}=\boldsymbol{0} \in \R^d, \ i=1, \dots, n$}
    \FOR{$t=1$ {\bfseries to} $T$}
    \STATE{Receive matrix of features $\bZ_t=\begin{bmatrix} \bz_{t,1} & \dots &\bz_{t,n}\end{bmatrix}$}
    \STATE{Predict $\hat{y}_t = \sum_{i=1}^n \bz_{t,i}^\top\bx_{t,i}$}
    \STATE{Receive label $y_t$ and pay $|\hat{y}_t-y_t|$}
    \STATE{Update $\btheta_{t+1,i} = \btheta_{t,i} - \bz_{t,i}\sign(\hat{y}_t-y_t), \ i=1, \dots, n$}
    \STATE{Update $\bx_{t+1,i} =
    \begin{cases}
    \boldsymbol{0}, & \btheta_{t+1,j}=\boldsymbol{0}, \ \forall j,\\
    \frac{\btheta_{t+1,i} \|\btheta_{t+1,i}\|_2^{q-2}}
    {\alpha \sqrt{t+1} (\sum_{j=1}^n \|\btheta_{t+1,j}\|_2^q)^{1-2/q}}, & \text{otherwise},
    \end{cases}
    \quad i=1,\dots,n$}
    \ENDFOR
}
\end{algorithmic}
\end{algorithm}

From the analysis of \ac{FTRL} with linearized losses, the properties of the group norms above, and Theorem~\ref{thm:prop_fenchel_sc}\index{duality strong convexity-smoothness}, it is immediate to state a regret upper bound. With the notation in Algorithm~\ref{alg:ftrl_groupnorm}, for any $\bU \in \R^{d \times n}$ we have
\begin{align*}
\sum_{t=1}^T |\hat{y}_t-y_t| - \sum_{t=1}^T \left|\sum_{i=1}^n \bz_{t,i}^\top \bu_i -y_t\right|
&\leq \frac{\alpha \sqrt{T}}{2}\|\bU\|^2_{2,p} + \frac{q}{2\alpha} \sum_{t=1}^T \frac{\|\bZ_t\|^2_{2,q}}{\sqrt{t}}\\
&\leq \frac{\alpha \sqrt{T}}{2}\|\bU\|^2_{2,p} + \frac{q \sqrt{T}}{\alpha} \max_{t\leq T} \ \|\bZ_t\|^2_{2,q}~.
\end{align*}
To obtain a logarithmic dependence in $n$, we can now proceed as in Section~\ref{sec:omd_eg}. Assuming $\|\bz_{t,i}\|_2\leq L$, we have that $\|\bZ_t\|^2_{2,q}\leq L^2 n^{2/q}$. Hence, setting $\alpha = q L n^{1/q}$ and $q = \max(2 \ln n,2)$ we have that $\frac{q}{\alpha}\|\bZ_t\|^2_{2,q}\leq L \max(\sqrt{2 e \ln n},2)$. Moreover, we can upper bound $\|\bU\|^2_{2,p}$ with $\|\bU\|^2_{2,1}$. Hence, the final regret upper bound is logarithmic in the number of groups: $L \max(\sqrt{2 e \ln n},2) \left(\frac{\|\bU\|^2_{2,1}}{2} + 1\right) \sqrt{T}$.

\section{Composite Losses and Proximal Operators}
\label{sec:ftrl_composite}
\index{composite loss|(textbf}

Let's now see a variant of the linearization of the losses: \emph{partial linearization of composite losses}.

Suppose that the losses we receive are composed of two terms: one convex function changing over time, and another part is fixed and known. These losses are called \textbf{composite}.
For example, we might have $\ell_t(\bx)= \tilde{\ell}_t(\bx)+\lambda \|\bx\|_1$.
Using the linearization, we might just take the subgradient of $\ell_t$. However, in this particular case, we might lose the ability of the L$_1$ norm to produce sparse solutions.

There is a better way to deal with these kinds of losses: move the constant part of the loss inside the regularization term. In this way, that part will not be linearized but used exactly in the argmin of the update. Moreover, the regret of the algorithm will depend only on the subgradient of  $\tilde{\ell}_t$ rather than subgradients of the full loss $\ell_t$. Assuming that the argmin is still easily computable, you can always expect better performance from this approach. In particular, in the case of adding an L$_1$ norm to the losses, you will be predicting in each step with the solution of an L$_1$ regularized optimization problem.

\begin{remark}
Given that we do not pay for the gradient of the regularizer in the regret upper bound of \ac{FTRL}, one may wonder why we do not put the entire loss function in the regularizer. This is not possible because it would require knowledge of the future: at time $t$, we construct the regularizer $\psi_t$ when we do not have access to the future loss $\ell_t$.
\end{remark}

Practically speaking, in the example above, we will define $\psi_t(\bx) = \frac{L}{\alpha} \sqrt{t}(\psi(\bx) - \min_{\by} \ \psi(\by))+ \lambda t \|\bx\|_1$, where we assume $\psi$ to be 1-strongly convex and the losses $\tilde{\ell}_t$ to be $L$-Lipschitz. Note that we use at time $t$ a term $\lambda t \|\bx\|_1$ because we anticipate the next term in the next round.
Given that $\psi_t(\bx)+\sum_{i=1}^{t-1} \ell_i(\bx)$ is $\frac{L}{\alpha} \sqrt{t}$-strongly convex, using Lemma~\ref{lemma:ftrl_strongly_1} and setting $\psi_{T+1}=\psi_T$, we have
\begin{align*}
&\sum_{t=1}^T \tilde{\ell}_t(\bx_t) - \sum_{t=1}^T \tilde{\ell}_t(\bu)
\leq \sum_{t=1}^T \langle \bg_t,\bx_t\rangle - \sum_{t=1}^T \langle \bg_t,\bu\rangle \\
&\leq \frac{L}{\alpha} \sqrt{T}\left(\psi(\bu) - \min_{\bx \in \mathcal{V}} \ \psi(\bx)\right) + \lambda T \|\bu\|_1
- \min_{\bx \in \mathcal{V}} \ \left[\lambda \|\bx\|_1 + \frac{L}{\alpha}(\psi(\bx)- \min_{\by \in \mathcal{V}} \ \psi(\by))\right] \\
&\quad + \frac{\alpha}{2}\sum_{t=1}^T \frac{\|\bg_t\|^2_\star}{L \sqrt{t}}
- \lambda \sum_{t=1}^{T-1} \|\bx_{t+1}\|_1 \\
&\leq L \sqrt{T}\left(\alpha+\frac{\psi(\bu) - \min_{\bx \in \mathcal{V}} \ \psi(\bx)}{\alpha} \right) + \lambda T \|\bu\|_1 - \lambda \|\bx_1\|_1  - \lambda \sum_{t=1}^{T-1} \|\bx_{t+1}\|_1,
\end{align*}
where $\bg_t \in \partial \tilde{\ell}_t(\bx_t)$.
Reordering the terms, we have
\begin{align*}
\sum_{t=1}^T \ell_t(\bx_t) - \sum_{t=1}^T \ell_t(\bu)
&= \sum_{t=1}^T (\lambda \|\bx_t\|_1+\tilde{\ell}_t(\bx_t)) - \sum_{t=1}^T (\lambda \|\bu\|_1+\tilde{\ell}_t(\bu))\\
&\leq L \sqrt{T}\left(\frac{\psi(\bu) - \min_{\bx \in \mathcal{V}} \psi(\bx)}{\alpha}+ \alpha\right)~.
\end{align*}

\begin{example}
Let's also take a look at the update rule in that case, that $\psi(\bx)=\frac{1}{2}\|\bx\|_2^2$ and we get composite losses with the L$_1$ norm. Also, assume that $\mathcal{V}=\R^d$. So, we have
\[
\bx_{t}
= \argmin_{\bx \in \R^d} \ \frac{L\sqrt{t}}{2\alpha}\|\bx\|_2^2 + \lambda t \|\bx\|_1 + \sum_{i=1}^{t-1} \langle \bg_i, \bx\rangle~.
\]
We can solve this problem by observing that the minimization decomposes over each coordinate of $\bx$. Denote by $\btheta_t=\sum_{i=1}^{t-1} \bg_i$. Hence, we know from first-order optimality condition in Theorem~\ref{thm:first_order_subdiff} that $x_{t,i}$ is the solution for the coordinate $i$ iff there exists $v_i \in \partial |x_{t,i}|$ such that
\[
\frac{L}{\alpha} \sqrt{t} x_{t,i}  + \lambda t v_i + \theta_{t,i} = 0~.
\]
The difficulty comes from the fact that $\bv$ is a subgradient in the next point $\bx_t$, which we do not know because we do not know $\bv$.
Unfortunately, there is no automatic way to find $v_i$, so we can ``guess'' the correct expression of $v_i$ and verify its correctness in the above expression.
In particular, considering 3 different cases, we have
\begin{itemize}
\item $|\theta_{t,i}|\leq \lambda t$, then $x_{t,i}=0$ and $v_i = - \tfrac{\theta_{t,i}}{\lambda t}$.
\item $\theta_{t,i}> \lambda t$, then $x_{t,i}=-\alpha\frac{\theta_{t,i}-\lambda t}{L \sqrt{t}}$ and $v_i = - 1$.
\item $\theta_{t,i}< - \lambda t$, then $x_{t,i}=-\alpha\frac{\theta_{t,i}+\lambda t}{L \sqrt{t}}$ and $v_i = 1$.
\end{itemize}
So, overall, we have
\[
x_{t,i} = -\alpha\frac{\sign(\theta_{t,i})\max(|\theta_{t,i}|-\lambda t,0)}{L \sqrt{t}}~.
\]
Observe that this update is a soft-thresholding operator\index{soft-thresholding operator} and it will produce sparse solutions, while just taking the subgradient of the L$_1$ norm would have never produced sparse predictions.
\end{example}

In the example above, we calculated something like
\begin{align*}
\argmin_{\bx \in \R^d} \ \|\bx\|_1 - \langle\bv, \bx\rangle + \frac{1}{2}\|\bx\|_2^2
&=\argmin_{\bx \in \R^d} \ \|\bx\|_1 - \langle\bv, \bx\rangle + \frac{1}{2}\|\bx\|_2^2 +\frac{1}{2}\|\bv\|_2^2\\
&=\argmin_{\bx \in \R^d} \ \|\bx\|_1 +\frac{1}{2}\|\bx-\bv\|_2^2~.
\end{align*}
This operation is known in the optimization literature as \textbf{Proximal Operator}\index{proximal operator|textbf} of the L$_1$ norm. In general, a proximal operator of a convex, proper, and closed function $f:\R^d \to (-\infty,+\infty]$ is defined as
\begin{equation}
\label{eq:prox}
\Prox_f(\bv) = \argmin_{\bx \in \R^d} \ f(\bx) +\frac{1}{2}\|\bx-\bv\|_2^2~.
\end{equation}
Proximal operators are used in optimization in the same way as we used them: they allow us to minimize the entire function rather than a linear approximation of it. Also, proximal operators generalize the concept of Euclidean projection. Indeed, $\Prox_{\indicator_{\mathcal{V}}}(\bv)= \Pi_{\mathcal{V}}(\bv)$.

From the definition, we also derive the important property of proximal operators that they minimize the function in each update. In other words, if $\bx_{t+1}=\Prox_f(\bx_t)$, then $f(\bx_{t+1})\leq f(\bx_t)$. Another interesting property is the fact that we can write the proximal operator as an \emph{implicit} subgradient descent update. Let $\eta>0$ and consider the following update
\[
\bx_{t+1}
=\Prox_{\eta f}(\bx_t)
= \argmin_{\bx \in \R^d} \ \eta f(\bx) +\frac{1}{2}\|\bx-\bx_t\|_2^2~.
\]
Then, from Theorem~\ref{thm:first_order_subdiff} we have
\[
\bx_{t+1} = \bx_t - \eta \bg_t,
\]
where $\bg_t \in \partial f(\bx_{t+1})$. Hence, we do an update step using the subgradient in the \emph{next} iteration.
This also reveals a connection between \ac{OMD} and proximal updates. Indeed, the update in \eqref{eq:prox} resembles the one in \eqref{eq:osd_as_omd}, where the function is not linearized. We will use this kind of update in Section~\ref{sec:prescient_omd}.

%
\index{composite loss|)textbf}

\section{FTRL on Strongly Convex Functions}
\label{sec:ftrl_strongly}
\index{function!strongly convex|(}

We will now use Lemma~\ref{lemma:ftrl_strongly_1} to prove a logarithmic regret bound for \ac{FTL} with strongly convex losses. Note that the existence and uniqueness (for $t\geq 2$) of $\bx_t$ is guaranteed by Theorem~\ref{thm:min_strongly_convex}.
\begin{corollary}
\label{cor:ftrl_strongly}
Let $\mathcal{V} \subseteq \R^d$ be a non-empty convex closed set.
Let $\ell_t:\R^d \to (-\infty, +\infty]$ be $\mu_t$-strongly convex with respect to $\|\cdot\|$, closed, and subdifferentiable on $\mathcal{V}$ for $t=1,\dots,T$. Set the sequence of regularizers to zero.
Then, \ac{FTL} guarantees a regret of
\[
\sum_{t=1}^T \ell_t(\bx_t) - \sum_{t=1}^T \ell_t(\bu)
\leq \frac{1}{2}\sum_{t=1}^T \frac{\|\bg_t\|_\star^2}{\sum_{i=1}^t \mu_i}, \quad \forall \bu \in \mathcal{V},\  \forall \bg_t \in \partial \ell_t(\bx_t)~.
\]
\end{corollary}
The above regret guarantee is the same as the one of \ac{OMD} over strongly convex losses, but here we do not need to know the strong convexity of the losses. In fact, we just need to output the minimizer over the past losses. However, this might be undesirable because now each update requires the minimization of a convex function.

Hence, we can again use the idea of replacing the losses with simpler \emph{surrogates}. In the Lipschitz case, it made sense to use linear losses because they were used to upper bound the regret of the true losses. However, here you can do better and use \emph{quadratic} losses, because the losses are strongly convex. In fact, assuming $\ell_t$ to be $\mu_t$ strongly convex with respect to $\|\cdot\|$, we have
\[
\ell_t(\bx_t) - \ell_t(\bu)
\leq \langle \bg_t, \bx_t-\bu\rangle - \frac{\mu_t}{2}\|\bx_t-\bu\|^2
= \tilde{\ell}_t(\bx_t) - \tilde{\ell}_t(\bu), \quad \forall \bg_t \in \partial \ell_t(\bx_t)~.
\]
So, we could run \ac{FTRL} on the quadratic losses $\tilde{\ell}_t(\bx)=\langle \bg_t, \bx\rangle + \frac{\mu_t}{2}\|\bx-\bx_t\|^2$, where $\bg_t \in \partial \ell_t(\bx_t)$. Depending on the norm, these quadratic losses might be strongly convex, allowing us to use Corollary~\ref{cor:ftrl_strongly} and obtain Algorithm~\ref{alg:ftrl_quad}.

\begin{algorithm}[t]
\caption{Follow-the-Regularized-Leader with ``Quadratized'' Losses}
\label{alg:ftrl_quad}
\begin{algorithmic}[1]
{
    \REQUIRE{A sequence of regularizers $\psi_1, \dots, \psi_T :\mathcal{X} \to \R$, closed non-empty convex set $\mathcal{V} \subseteq \mathcal{X} \subseteq \R^d$}
    \FOR{$t=1$ {\bfseries to} $T$}
    \STATE{Output $\bx_t \in \argmin_{\bx \in \mathcal{V}} \ \psi_t(\bx) + \sum_{i=1}^{t-1} \left(\langle \bg_i,\bx\rangle + \frac{\mu_i}{2} \|\bx-\bx_i\|^2\right)$}
    \STATE{Pay the loss $\ell_t(\bx_t)$, where $\ell_t$ is subdifferentiable on $\mathcal{V}$ and $\mu_t$ strongly convex in $\mathcal{V}$}
    \STATE{Set $\bg_t \in \partial \ell_t(\bx_t)$}
    \ENDFOR
}
\end{algorithmic}
\end{algorithm}

For example, consider the case that the losses are strongly convex with respect to $\|\cdot\|_2$ and there is no regularizer. Using Theorem~\ref{thm:two_steps_omd}, the update in Algorithm~\ref{alg:ftrl_quad} for $t\geq 2$ becomes
\begin{equation}
\label{eq:ftrl_upd_strongly}
\bx_t
= \argmin_{\bx \in \mathcal{V}} \ \sum_{i=1}^{t-1} \left(\langle \bg_i, \bx\rangle + \frac{\mu_i}{2}\|\bx-\bx_i\|_2^2\right)
= \Pi_{\mathcal{V}}\left(\frac{\sum_{i=1}^{t-1} (\mu_i \bx_i - \bg_i)}{\sum_{i=1}^{t-1} \mu_i}\right),
\end{equation}
where the only expensive operation is the projection, which can be computed in polynomial time.
Moreover, we will get exactly the same regret bound as in Corollary~\ref{cor:ftrl_strongly} because $\bg_t \in \partial \tilde{\ell}_t(\bx_t)$. However, there are two important differences: 1) here the guarantee holds for a specific choice of the $\bg_t$ rather than for any subgradient in $\partial \ell_t(\bx_t)$; 2) we need to know the strong convexity parameters $\mu_i$, while it was not needed for \ac{FTL} with exact losses.

\begin{example}
Going back to the example in the first chapter, where $\mathcal{V}=[0,1]$ and $\ell_t(\bx)=(x-y_t)^2$ are strongly convex, we now see immediately that \ac{FTRL} without a regularizer, that is \ac{FTL}, gives logarithmic regret.
\end{example}
\index{function!strongly convex|)}

\section{A Weaker Notion of Strong Convexity and the Online Newton Step}
\label{sec:ons}

In the previous section, we saw that the notion of strong convexity allows us to build quadratic surrogate loss functions, on which \ac{FTRL} has smaller regret.
Can we find a more general notion of strong convexity that allows us to get a small regret for a larger class of functions?
We can start from strong convexity and try to generalize it.
So, instead of asking that the function is strongly convex with respect to a norm, we might be happy requiring that strong convexity holds at a particular point with respect to a seminorm that depends on the points themselves.

In particular, we can require that for a loss $\ell:\R^d \to (-\infty, +\infty]$ subdifferentiable on $\mathcal{V}$, for a $\bx \in \mathcal{V}$, and for a $\bg \in \partial \ell(\bx)$, there exists $\bA \in \R^{d \times d}$ such that $\bA \succeq 0$ and
\begin{equation}
\label{eq:restriced_sc}
\ell(\by) \geq \ell(\bx) + \langle \bg , \by - \bx \rangle + \frac{1}{2} \| \bx - \by \|_{\bA}^2, \quad \forall \by \in \mathcal{V},
\end{equation}
where $\|\bx\|_{\bA}$ is defined as $\sqrt{\bx^\top \bA\bx}$.
Note that this is a weaker property than strong convexity because $\bA$ depends on $\bx$ and $\bg$. On the other hand, in the characterization of strong convexity in Lemma~\ref{lemma:strong_convexity}, we want the last term to be the same norm everywhere in the space. See also Problem~\ref{exercise:relative_strong_convexity} for a similar generalization of strong convexity.

The rationale of this new definition is that it still allows us to build surrogate loss functions, but without requiring strong convexity over the entire space.
In particular, for a suitable choice of $\bA_t$, we can think to use \ac{FTRL} on the surrogate losses
\[
\tilde{\ell}_t(\bx)=\langle \bg_t, \bx\rangle + \frac12 \|\bx_t-\bx\|^2_{\bA_t}
\]
and the regularizers $\psi_t(\bx) = \frac{\lambda}{2} \|\bx\|^2_2$, where $\lambda>0$.
Observe that, as in the surrogate losses for the strongly convex case, $\bg_t \in \partial \tilde{\ell}_t(\bx_t)$.
We will denote by $\bS_t=\lambda \bI_d + \sum_{i=1}^t \bA_i$.

\begin{remark}
Note that $\|\bx\|_{\bS_t}:=\sqrt{\bx^\top \bS_t \bx}$ is a norm because $\bS_t$ is positive definite. Also, $f(\bx) = \frac{1}{2} \bx^\top \bS_t \bx$ is $1$-strongly convex\index{function!strongly convex} with respect to $\|\cdot\|_{\bS_t}$ (from Theorem~\ref{thm:hessian_strong_conv} given that the Hessian is $\bS_t$ and $\bx^\top \nabla^2 f(\by) \bx = \|\bx\|^2_{\bS_t}$).
Also, from Example~\ref{example:dual_norm_a}, the dual norm of $\|\cdot\|_{\bS_t}$ is $\|\cdot\|_{\bS^{-1}_t}$.
\end{remark}

From the above remark, we have that the $\psi_t(\bx)+\sum_{i=1}^t \tilde{\ell}_i(\bx)$ is 1-strongly convex with respect to $\|\cdot\|_{\bS_{t}}$. Hence, using the \ac{FTRL} regret equality in Lemma~\ref{lemma:ftrl_equality} and Lemma~\ref{lemma:ftrl_strongly_1}, we immediately get the following guarantee
\begin{align}
\sum_{t=1}^T \ell_t(\bx_t) - \sum_{t=1}^T \ell_t(\bu)
&\leq \sum_{t=1}^T \left(\langle \bg_t, \bx_t-\bu\rangle  - \frac{1}{2}\|\bu-\bx_t\|^2_{\bA_t}\right)
= \sum_{t=1}^T (\tilde{\ell}_t(\bx_t)-\tilde{\ell}_t(\bu)) \nonumber \\
&\leq \frac{\lambda}{2} \|\bu\|^2_2 -\min_{\bx \in \mathcal{V}} \ \frac{\lambda}{2} \|\bx\|^2_2 + \frac12 \sum_{t=1}^T \|\bg_t\|^2_{\bS^{-1}_t}, \quad \forall \bu \in \mathcal{V}~. \label{eq:ons_regret}
\end{align}
Note how the proof and the algorithm mirror what we did for \ac{FTRL} with strongly convex\index{function!strongly convex} losses in Section~\ref{sec:ftrl_strongly}.

\begin{remark}
It is possible to generalize Lemma~\ref{lemma:ftrl_strongly_1} to hold in this generalized notion of strong convexity. This would allow us to get exactly the same bound running \ac{FTRL} over the original losses with regularizer $\psi(\bx)=\frac{\lambda}{2}\|\bx\|_2^2$.
\end{remark}

Let's now see a practical instantiation of this idea.
Consider the case that the sequence of loss functions we receive satisfies
\begin{equation}
\label{eq:exp_concave}
\ell_t(\bx) - \ell_t(\bu)
\leq \langle \bg, \bx-\bu\rangle -\frac{\mu}{2}(\langle \bg, \bx-\bu\rangle)^2, \quad \forall \bx, \bu \in \mathcal{V}, \forall \bg \in \partial \ell_t(\bx)~.
\end{equation}
In words, \emph{we assume to have a class of functions that can be lower bounded by a quadratic that depends on the current subgradient}.
In particular, these functions guarantee a minimal curvature only in the direction of (any of) their subgradients.
Denoting by $\bA_t= \mu \bg_t \bg_t^\top$, we can use the above idea using
\[
\tilde{\ell}_t(\bx)
= \langle \bg_t, \bx\rangle + \frac12 \|\bx_t -\bx\|_{\bA_t}^2
= \langle \bg_t, \bx\rangle + \frac{\mu}{2} (\langle\bg_t,\bx-\bx_t\rangle)^2~.
\]
Hence, the update rule would be
\[
\bx_{t}
= \argmin_{\bx \in \mathcal{V}} \ \frac{\lambda }{2}\|\bx\|_2^2 + \sum_{i=1}^{t-1} \left(\langle \bg_i, \bx\rangle  + \frac{\mu}{2}  (\langle\bg_i,\bx-\bx_i\rangle)^2\right)~.
\]
This update is called the \textbf{\ac{ONS}}\index{Online Newton Step algorithm|(textbf}, see Algorithm~\ref{alg:ons}.

\begin{algorithm}[t]
\caption{Online Newton Step (ONS)}
\label{alg:ons}
\begin{algorithmic}[1]
{
    \REQUIRE{$\mathcal{V}\subset\R^d$ closed non-empty convex set, $\lambda,\mu>0$}
    \FOR{$t=1$ {\bfseries to} $T$}
    \STATE{Set $\bx_t = \argmin_{\bx \in \mathcal{V}} \ \frac{\lambda}{2}\|\bx\|_2^2 + \sum_{i=1}^{t-1} \left(\langle \bg_i, \bx\rangle  + \frac{\mu}{2} (\langle\bg_i,\bx-\bx_i\rangle)^2\right)$}
    \STATE{Pay the loss $\ell_t(\bx_t)$, where $\ell_t$ is subdifferentiable on $\mathcal{V}$}
    \STATE{Set $\bg_t \in \partial \ell_t(\bx_t)$}
    \ENDFOR
}
\end{algorithmic}
\end{algorithm}

Denoting by $\bS_t=\lambda \bI_d + \sum_{i=1}^{t} \bA_i$ and using \eqref{eq:ons_regret}, we have
\[
\sum_{t=1}^T \ell_t(\bx_t) - \sum_{t=1}^T \ell_t(\bu)
\leq \frac{\lambda }{2} \|\bu\|_2^2 -\min_{\bx \in \mathcal{V}} \ \frac{\lambda}{2} \|\bx\|^2_2 + \frac12 \sum_{t=1}^T \|\bg_t\|^2_{\bS^{-1}_t}, \quad \forall \bu \in \mathcal{V}~.
\]
To bound the last term, we will use the following lemma.
%
\begin{lemma}
\label{lemma:log_eigen}
Let $\bg_1,\dots, \bg_T$ be an arbitrary sequence of vectors in $\R^d$ and $\mu,\lambda>0$. Define $\bS_t=\lambda \bI_d + \mu \sum_{i=1}^t \bg_i \bg_i^\top$. Then, the following holds
\[
\sum_{t=1}^T \bg_t^\top \bS_t^{-1} \bg_t
\leq \frac{1}{\mu}\sum_{i=1}^d \ln \left(1 + \frac{\mu \lambda_i}{\lambda}\right),
\]
where $\lambda_1, \dots, \lambda_d$ are the eigenvalues of $\sum_{i=1}^T \bg_i \bg_i^\top$.
\end{lemma}
\begin{proof}
First, we will prove the so-called \emph{matrix determinant lemma}\index{matrix determinant lemma} for any invertible $\bB \in \R^{d \times d}$:
\[
\det(\bB+\bu\bv^\top)=\det(\bB)(1+\bv^\top \bB^{-1}\bu), \quad \forall \bu,\bv \in \R^d~.
\]
Observe that the matrix $\bu \bv^\top$ has one eigenvalue equal to $\bv^\top \bu$ and the remaining $d-1$ are 0. Hence, we have that
\[
\det(\bI+\bu \bv^\top) = 1+ \bv^\top \bu~.
\]
Using this equality, we have the matrix determinant lemma:
\[
\det(\bB+\bu\bv^\top)
= \det(\bB)\det(\bI + \bB^{-1}\bu \bv^\top)
= \det(\bB)(1 + \bv^\top \bB^{-1}\bu)~.
\]

Now, setting $\bu=-\sqrt{\mu}\bg_t$, $\bv=\sqrt{\mu}\bg_t$ and $\bB=\bS_t$, we have
\[
\det(\bS_{t-1})
= \det(\bS_t-\mu \bg_t\bg_t^\top)
=\det(\bS_t)(1- \mu \bg_t^\top \bS^{-1}_{t}\bg_t)~.
\]
Reordering the terms and summing over $t=1,\dots, T$, we have
\begin{align*}
\sum_{t=1}^T \bg_t^\top \bS_t^{-1} \bg_t
&= \frac{1}{\mu}\sum_{t=1}^T \left(1-\frac{\det(\bS_{t-1})}{\det(\bS_{t})}\right)
\leq \frac{1}{\mu}\sum_{t=1}^T \ln \frac{\det(\bS_{t})}{\det(\bS_{t-1})}
= \frac{1}{\mu}\ln \frac{\det(\bS_{T})}{\det(\bS_{0})}\\
&= \frac{1}{\mu}\sum_{i=1}^d \ln \left(1 + \frac{\mu \lambda_i}{\lambda}\right),
\end{align*}
where in the inequality we used the elementary inequality $1-x\leq \ln 1/x$ for $x>0$.
\end{proof}

Using this lemma and assuming $\|\bg_t\|_2\leq L$ and \eqref{eq:exp_concave} holds for the losses, then \ac{ONS} satisfies the following regret
\begin{align*}
\sum_{t=1}^T \ell_t(\bx_t) - \sum_{t=1}^T \ell_t(\bu)
&\leq \frac{\lambda }{2} \|\bu\|_2^2 -\min_{\bx \in \mathcal{V}}\ \frac{\lambda}{2} \|\bx\|^2_2+ \frac{1}{2\mu} \sum_{i=1}^d \ln \left(1+\frac{\mu \lambda_i}{\lambda}\right)\\
&\leq \frac{\lambda }{2} \|\bu\|_2^2 -\min_{\bx \in \mathcal{V}} \ \frac{\lambda}{2} \|\bx\|^2_2+ \frac{d}{2\mu}\ln \left(1 + \frac{\mu T L^2}{d \lambda }\right),
\end{align*}
where in the second inequality we used the inequality of arithmetic and geometric means, $(\prod_{i=1}^d x_i)^{1/d} \leq \frac{1}{d} \sum_{i=1}^d x_i$, and the fact that $\sum_{i=1}^d \lambda_i \leq T L^2$.

Hence, if the losses satisfy \eqref{eq:exp_concave}, we can guarantee a logarithmic regret. However, unlike the strongly convex case, here the complexity of the update is at least quadratic in the number of dimensions. Moreover, the regret also depends linearly on the number of dimensions.

\begin{remark}
Despite the name, the \ac{ONS} algorithm should not be confused with the Newton algorithm. They are similar in spirit because they both construct quadratic approximations to the function, but the Newton algorithm uses the exact Hessian while the \ac{ONS} uses an approximation that works only for a restricted class of functions. In this view, the \ac{ONS} algorithm is more similar to Quasi-Newton methods. However, the best lens to understand the \ac{ONS} is still through the generalized concept of strong convexity.
\end{remark}
\index{Online Newton Step algorithm|)textbf}

In the next section, we introduce the \emph{exp-concave} functions that satisfy \eqref{eq:exp_concave}.


\subsection{Convex Analysis Bits: Exp-concavity}

\begin{definition}
\index{function!exp-concave|(textbf}
Let $\mathcal{V} \subseteq \R^d$ be convex and $\alpha>0$. A function $f:\R^d\to(-\infty,+\infty]$ is called \textbf{$\alpha$-exp-concave} on $\mathcal{V}$ if $\bx\mapsto \exp(-\alpha f(\bx))$, with the convention $\exp(-\alpha\cdot(+\infty))=0$, is concave on $\mathcal{V}$.
\end{definition}

\begin{proposition}
Let $\alpha >0$. If $f:\R^d \to (-\infty,+\infty]$ is $\alpha$-exp-concave on $\mathcal{V}$, then $f$ is $\beta$-exp-concave on $\mathcal{V}$ for any $0<\beta<\alpha$.
\end{proposition}
\begin{proof}
Observe that $h(\bx)=\exp(-\beta f(\bx))=(\exp(-\alpha f(\bx)))^{\beta/\alpha}$, where $0^{\beta/\alpha}=0$. Since $x\mapsto x^{\beta/\alpha}$ is increasing and concave on $\R_+$, $h$ is concave.
\end{proof}

\begin{proposition}
Assume $f:\R^n \to(-\infty,+\infty]$ is $\alpha$-exp-concave on $\mathcal{V}\subseteq \R^n$ and let $h(\bx)=f(\bA \bx+\bu)$, where $\bA\in\R^{n\times d}$ and $\bu\in\R^n$. Then, $h$ is $\alpha$-exp-concave on $\{\bx:\bA\bx+\bu\in\mathcal{V}\}$.
\end{proposition}
\begin{proof}
Since $\exp(-\alpha f(\bx))$ is concave, then $\exp(-\alpha h(\bx)) = \exp(-\alpha f(\bA \bx+\bu))$ is also concave because it is the composition of a concave function with an affine transformation.
\end{proof}

The next proposition shows that exp-concavity is a stronger property than convexity.
\begin{proposition}
\label{prop:exp-concave-convex}
Let $f:\R^d\to(-\infty,+\infty]$ be $\alpha$-exp-concave on $\mathcal{V}$. Then, $f$ is convex on $\mathcal{V}$.
\end{proposition}
\begin{proof}
Let $g(\bx)=\exp(-\alpha f(\bx))$. By hypothesis, $g$ is concave and non-negative on $\mathcal V$. Fix $\bx,\by\in\mathcal{V}$ and $t\in[0,1]$. If $f(\bx)=+\infty$ or $f(\by)=+\infty$, the convexity inequality is trivial. Otherwise, $g(\bx),g(\by)>0$. Concavity of $g$ and the weighted AM-GM inequality give
\[
g(t\bx+(1-t)\by)
\geq t g(\bx)+(1-t)g(\by)
\geq g(\bx)^t g(\by)^{1-t}~.
\]
Hence,
\[
\exp(-\alpha f(t\bx+(1-t)\by))
\geq \exp\left[-\alpha\left(t f(\bx)+(1-t)f(\by)\right)\right]~.
\]
Taking $-\frac{1}{\alpha}\ln$ of both sides gives
\[
f(t\bx+(1-t)\by)
\leq t f(\bx)+(1-t)f(\by)~.
\]
So, $f$ is convex on $\mathcal{V}$.
\end{proof}

We can also characterize the exp-concavity in terms of the curvature of the function: twice differentiable exp-concave functions have a minimal curvature in the direction of the gradient.
\begin{theorem}
\label{thm:exp_concave_twice_diff}
Let $\mathcal{V} \subset \R^d$ be convex and $f$ be twice-differentiable and finite in an open set containing $\mathcal{V}$. Then, $f$ is $\alpha$-exp-concave on $\mathcal{V}$ iff
\[
\nabla^2 f(\bx)
\succeq \alpha \nabla f(\bx)\,\nabla f(\bx)^\top, \quad  \forall \bx\in\mathcal{V}~.
\]
\end{theorem}
\begin{proof}
The concavity of $g(\bx)=\exp(-\alpha f(\bx))$ is equivalent to $\nabla^2 g(\bx)\preceq 0$ for all $\bx \in \mathcal{V}$.

Let's compute $\nabla g$ and $\nabla^2 g$.  Denote $g(\bx)=\exp(-\alpha f(\bx))$.
The gradient is
\[
\nabla g(\bx)
=\nabla\left(\exp(-\alpha f(\bx))\right)
=-\alpha\,\exp(-\alpha f(\bx))\,\nabla f(\bx)~.
\]
The Hessian is
\begin{align*}
\nabla^2 g(\bx)
&=\nabla\left(-\alpha \exp(-\alpha f(\bx))\,\nabla f(\bx)\right)\\
&=-\alpha\,\nabla\left(\exp(-\alpha f(\bx))\right)\,\nabla f(\bx)^\top - \alpha\,\exp(-\alpha f(\bx))\,\nabla^2 f(\bx)\\
&= \alpha^2 \exp(-\alpha f(\bx))\,\left(\nabla f(\bx)\nabla f(\bx)^\top\right) -\alpha \exp(-\alpha f(\bx))\,\nabla^2 f(\bx)~.
\end{align*}

Re-ordering,
\[
\nabla^2 g(\bx)
=\exp(-\alpha f(\bx))\left(\alpha^2\,\nabla f(\bx)\nabla f(\bx)^\top - \alpha\,\nabla^2 f(\bx)\right)~.
\]
Since $\exp(-\alpha f(\bx))>0$ and $\alpha>0$, the concavity condition $\nabla^2 g(\bx)\preceq0$ is equivalent to $\alpha\,\nabla f(\bx)\nabla f(\bx)^\top  \preceq \nabla^2 f(\bx)$.
\end{proof}

\begin{example}
\label{example:exp_concave_l2}
Consider $\mathcal{V}=\{\bx: \|\bx\|_2\leq R\}$, $\mathcal{Y}=\{\by: \|\by\|_2\leq R\}$, and $f:\R^d \to \R$ defined as $f(\bx)=\|\bx-\by\|_2^2$, where $\by \in \mathcal{Y}$.
We have
\[
\nabla^2 f(\bx)=2 \bI_d
\]
and
\[
\nabla f(\bx)=2 (\bx-\by)~.
\]
Using Theorem~\ref{thm:exp_concave_twice_diff}, $f$ is $\alpha$-exp-concave on $\mathcal{V}$ iff
\[
\bI_d \succeq 2 \alpha (\bx-\by)(\bx-\by)^\top, \ \forall \bx \in \mathcal{V}, \by \in \mathcal{Y}~.
\]
Hence, we have that $\alpha=\frac{1}{2(2R)^2}$.
\end{example}

\begin{example}
\label{ex:logistic_exp_concave}
Let $\mathcal{V}=\{\bx \in \R^d: \|\bx\|_2\leq U\}$. The logistic loss\index{logistic loss} of a linear predictor $\ell(\bx)= \ln (1+\exp(-\langle \bz,\bx\rangle))$, where $\|\bz\|_2\leq 1$ is $\exp(-2U)/2$-exp-concave on $\mathcal{V}$. The proof is left to the reader, see Problem~\ref{exercise:logistic_exp_concave}.
\end{example}

Finally, we show that exp-concavity holds on the entire $\R^d$ only for constant functions.
\begin{proposition}
\label{prop:constant_exp_concave}
Let $f:\R^d \to (-\infty,+\infty]$ be $\alpha$-exp-concave on $\R^d$. Then, $f$ is constant, possibly equal to $+\infty$ everywhere.
\end{proposition}
\begin{proof}
Denote by $h(\bx)=-\exp(-\alpha f(\bx))$, which is convex by definition, and it is upper bounded by 0.
Suppose that $h$ is not constant, this means that there exist $\bx, \by \in \R^d$ such that $h(\bx)>h(\by)$. Since $h$ is convex, we have
\[
h(\bx) \leq
\lambda h\left(\frac{\bx-(1-\lambda)\by}{\lambda}\right) + (1-\lambda)h(\by), \quad \forall \lambda \in (0,1)~.
\]
Hence, we have
\[
\frac{h(\bx)-h(\by)}{\lambda}+ h(\by)
=\frac{h(\bx)-(1-\lambda)h(\by)}{\lambda}
\leq h\left(\frac{\bx-(1-\lambda)\by}{\lambda}\right)
\leq 0~.
\]
Since $h(\bx)> h(\by)$, we have that the l.h.s. of the inequality goes to infinity as $\lambda \to 0^+$, giving a contradiction. Hence, $h$ is constant. Therefore $\exp(-\alpha f)$ is constant, which implies that $f$ is constant, with the possible constant value $+\infty$.
\end{proof}

\begin{example}[ONS on exp-concave losses]
\label{ex:exp_concave}
\index{Online Newton Step algorithm}
Assume a finite-valued function $f$ to be $\alpha$-exp-concave on a set $\mathcal{X}$.

Choose $\beta \leq \frac{\alpha}{2}$ such that $|\beta \langle \bg, \by -\bx\rangle| \leq \frac{1}{2}$ for all $\bx, \by \in \mathcal{X}$ and $\bg \in \partial f(\bx)$. For $\beta$ to exist, we need $\mathcal{X}$ to be a bounded domain.
Then, this class of functions satisfies the property \eqref{eq:exp_concave}.
In fact, given that $f$ is $\alpha$-exp-concave then it is also $2 \beta$-exp-concave. Hence, we have
\[
\exp(-2 \beta f(\by))
\leq \exp(-2 \beta f(\bx)) - 2\beta \exp(-2 \beta f(\bx)) \langle \bg, \by -\bx\rangle,
\]
that is
\[
\exp(-2 \beta f(\by)+ 2\beta f(\bx))
\leq 1 - 2 \beta \langle \bg, \by -\bx\rangle,
\]
that implies
\[
f(\bx) - f(\by)
\leq \frac{1}{2\beta}\ln \left( 1+ 2 \beta \langle \bg, \bx- \by\rangle\right)
\leq \langle \bg, \bx- \by\rangle - \frac{\beta}{2} (\langle \bg, \by -\bx\rangle)^2,
\]
where we used the elementary inequality $\ln(1+x)\leq x - x^2/4$, for $|x|\leq 1$.

For example, assuming the losses to be finite-valued and $\alpha$-exp-concave, the L$_2$ diameter of the feasible set $\mathcal{V}$ to be bounded by $D$, and the L$_2$ norm of the subgradient to be bounded by $G$, we can set $\beta=\frac{1}{2}\min(\frac{1}{GD},\alpha)$.
\end{example}

\begin{remark}
Proposition~\ref{prop:constant_exp_concave} and the fact that exp-concave functions are often used only on bounded domains might induce someone to think that exp-concave functions only exist on bounded domains. This is not true: there exists $f:\R^d\to \R$ finite-valued and exp-concave on $\mathcal{V}$, where $\mathcal{V}\subset \R^d$ is not bounded. For example, let $\mathcal{V}=\{x: x\geq 0\}$ and $f(x)= -\ln(x+1)$ that is 1-exp-concave on $\mathcal{V}$.
\end{remark}

\begin{remark}
Using the online-to-batch conversion in Theorem~\ref{thm:o2b}, the above theorem says that we can minimize stochastic exp-concave losses with a convergence rate in expectation of $\mathcal{O}(\frac{d \ln T}{T})$. Hence, up to logarithmic terms, we obtain a $\mathcal{O}(1/T)$ convergence rate without strong convexity.
\end{remark}
\index{function!exp-concave|)textbf}

\section{Online Linear Regression: Vovk--Azoury--Warmuth Forecaster}
\label{sec:vaw}

\index{Vovk--Azoury--Warmuth forecaster|(textbf}

\begin{algorithm}[h]
\caption{Vovk--Azoury--Warmuth Forecaster}
\label{alg:vaw}
\begin{algorithmic}[1]
{
    \REQUIRE{$\lambda>0$}
    \FOR{$t=1$ {\bfseries to} $T$}
    \STATE{Receive $\bz_t \in \R^d$}
    \STATE{Set $\bx_t = \argmin_{\bx} \ \frac{\lambda}{2}\|\bx\|_2^2 + \frac12 \sum_{i=1}^{t-1} (\langle \bz_i, \bx\rangle - y_i)^2 + \frac12 (\langle \bz_t,\bx\rangle)^2$}
    \STATE{Receive $y_t \in \R$ and pay the loss $\ell_t(\bx_t)=\frac12 (\langle \bz_t, \bx_t\rangle - y_t)^2$}
    \ENDFOR
}
\end{algorithmic}
\end{algorithm}

Let's now consider the specific case that $\ell_t(\bx)=\frac12 (\langle \bz_t, \bx\rangle - y_t)^2$ and $\mathcal{V}=\R^d$, which is the setting of \emph{unconstrained online linear regression with square loss}.
These losses are not strongly convex with respect to $\bx$, but they are exp-concave\index{function!exp-concave} when the domain is bounded. We could use the \ac{ONS} algorithm, but it would not work in the unbounded case.
Another possibility would be to run \ac{FTRL}, but those losses are not strongly convex, and we would get only a $\mathcal{O}(\sqrt{T})$ regret.

It turns out we can still get a logarithmic regret if we make an additional assumption.
We will assume to have access to $\bz_t$ before predicting $\bx_t$. Note that this is a reasonable assumption in many interesting applications in machine learning.
Then, we will use the same strategy we used for the composite losses\index{composite loss} in Section~\ref{sec:ftrl_composite}, that is, we put in the regularizer the future part of the loss we know we will receive, that is $\frac{1}{2}(\langle \bz_t, \bx\rangle)^2$. This also has another equivalent interpretation as running \ac{FTRL} over the past losses plus the loss on the received $\bz_t$ hallucinating a label of $0$. We will call this algorithm \textbf{Vovk--Azoury--Warmuth forecaster}, from the names of the inventors. The details are in Algorithm~\ref{alg:vaw}.

As we did for composite losses\index{composite loss}, we look closely at the loss functions to see if there are terms that we might move inside the regularizer. The motivation would be the same as in the composite losses case: the bound will depend only on the subgradients of the part of the losses that are outside of the regularizer.

So, observe that
\[
\ell_t(\bx)= \frac12 (\langle \bz_t, \bx\rangle)^2 - y_t \langle \bz_t, \bx\rangle + \frac12 y_t^2~.
\]
From the above, we see that we could think of moving the terms $\frac12(\langle \bz_t, \bx\rangle)^2$ in the regularizer and leaving the linear terms in the loss: $\tilde{\ell}_t(\bx)=- y_t \langle \bz_t, \bx\rangle$.
Hence, we will use
\[
\psi_t(\bx)
= \frac{1}{2}\bx^\top \left(\sum_{i=1}^t \bz_i \bz_i^\top\right)\bx + \frac{\lambda}{2}\|\bx\|_2^2~.
\]
Note that the regularizer at time $t$ contains the $\bz_t$ that is revealed to the algorithm before it makes its prediction.
For simplicity of notation, denote by $\bS_t=\lambda \bI_d +\sum_{i=1}^t \bz_i \bz_i^\top$.

\begin{remark}
This splitting of the loss between the regularizer and the surrogate linear losses is important because now the gradient of the surrogate is just $-y_t \bz_t$, which is controlled if $\bz_t$ is controlled; instead, the gradient of the original losses is $y_t \bz_t (\langle \bz_t, \bx\rangle -y_t)$ that is not controlled for unbounded domains. However, now Lemma~\ref{lemma:ftrl_strongly_1} will give us a strong convexity that depends on the strong convexity of $\psi_t$, and this is exactly the reason why we need $(\langle\bz_t, \bx\rangle)^2$ in the regularizer.
\end{remark}

Using such a procedure, the prediction can be written in a closed form:
\begin{align*}
\bx_t
&= \argmin_{\bx} \ \frac{\lambda}{2}\|\bx\|^2_2 + \frac12 \sum_{i=1}^{t-1} (\langle \bz_i,\bx\rangle - y_i)^2 + \frac12 (\langle \bz_t,\bx\rangle)^2 \\
&= \left(\lambda \bI_d + \sum_{i=1}^t \bz_i \bz_i^\top\right)^{-1} \sum_{i=1}^{t-1} y_i \bz_i ~.
\end{align*}
Using an incremental update of the inverse using the Sherman--Morrison formula\index{Sherman--Morrison formula}, the computational complexity of this update is $\mathcal{O}(d^2)$.

Hence, using the regret we proved for \ac{FTRL} with strongly convex regularizers and $\psi_{T+1}=\psi_T$, we get the following guarantee
\begin{align*}
\sum_{t=1}^T &\left(\tilde{\ell}_t(\bx_t) - \tilde{\ell}_t(\bu)\right) \\
&= \sum_{t=1}^T \left(-y_t \langle \bz_t, \bx_t\rangle + y_t \langle \bz_t, \bu\rangle\right) \\
&\leq \psi_T(\bu) - \min_{\bx} \ \psi_1(\bx) + \frac12 \sum_{t=1}^T y_t^2 \bz_t^\top \bS^{-1}_t \bz_t + \sum_{t=1}^{T-1}\left(\psi_t(\bx_{t+1})-\psi_{t+1}(\bx_{t+1})\right) \\
&= \frac{1}{2}\bu^\top \bS_T \bu + \frac{1}{2}\sum_{t=1}^{T} y_t^2 \bz_t^\top \bS^{-1}_t \bz_t - \frac12 \sum_{t=1}^{T-1}(\langle\bz_{t+1}, \bx_{t+1} \rangle)^2 ~.
\end{align*}
Noting that $\bx_1=\boldsymbol{0}$ and reordering the terms, we have
\begin{align*}
\sum_{t=1}^T &\frac12(\langle \bz_t, \bx_t\rangle - y_t)^2 - \sum_{t=1}^T \frac12 (\langle \bz_t, \bu\rangle - y_t)^2 \\
&= \frac12 \sum_{t=1}^T (\langle\bz_{t}, \bx_{t} \rangle)^2 + \sum_{t=1}^T \left(-y_t \langle \bz_t, \bx_t\rangle + y_t \langle \bz_t, \bu\rangle\right) - \frac{1}{2}\sum_{t=1}^T (\langle\bz_{t}, \bu \rangle)^2 \\
&\leq \frac{\lambda}{2}\|\bu\|_2^2 + \frac{1}{2}\sum_{t=1}^{T} y_t^2 \bz_t^\top \bS^{-1}_t \bz_t~.
\end{align*}
\begin{remark}
Note that, differently from the \ac{ONS} algorithm, the strong convexity comes from the regularizer. Yet, the bound depends on $\bS^{-1}_t$ instead of $\bS^{-1}_{t-1}$ because the current sample $\bz_t$ is used in the regularizer. Without the knowledge of $\bz_t$ the last term would be $\frac{1}{2}\sum_{t=1}^{T} y_t^2 \bz_t^\top \bS^{-1}_{t-1} \bz_t$ that would have to be controlled using $\lambda$ and a bound on the norms of $\bz_t$.
\end{remark}

So, using again Lemma~\ref{lemma:log_eigen} and assuming $|y_t|\leq Y, \ t=1,\dots,T$, we have
\[
\sum_{t=1}^T \frac12 (\langle \bz_t, \bx_t\rangle - y_t)^2 - \sum_{t=1}^T \frac12  (\langle \bz_t, \bu\rangle - y_t)^2
\leq \frac{\lambda}{2}\|\bu\|_2^2 + \frac{Y^2}{2} \sum_{i=1}^d \ln \left(1 + \frac{\lambda_i}{\lambda}\right),
\]
where $\lambda_1, \dots, \lambda_d$ are the eigenvalues of $\sum_{i=1}^T \bz_i \bz_i^\top$.

If we assume that $\|\bz_t\|_2\leq R$, we can reason as we did for the similar term in \ac{ONS}, to have
\[
\sum_{i=1}^d \ln \left(1 + \frac{\lambda_i}{\lambda}\right)
\leq d \ln \left(1+\frac{R^2 T}{\lambda d}\right)~.
\]

Putting everything together, we have the following theorem.
\begin{theorem}
Let $\lambda>0$ and assume $\|\bz_t\|_2\leq R$ and $|y_t|\leq Y$ for $t=1, \dots, T$. Then, using the prediction strategy in Algorithm~\ref{alg:vaw}, for all $\bu \in \R^d$, we have
\[
\sum_{t=1}^T \frac12 (\langle \bz_t, \bx_t\rangle - y_t)^2 - \sum_{t=1}^T \frac12 (\langle \bz_t, \bu\rangle - y_t)^2
\leq \frac{\lambda}{2}\|\bu\|_2^2 + \frac{d Y^2 }{2} \ln \left(1+\frac{R^2 T}{\lambda d}\right)~.
\]
\end{theorem}

\begin{remark}
It is possible to show that the regret of the Vovk--Azoury--Warmuth forecaster is optimal up to multiplicative factors~\citep[Theorem 11.9]{Cesa-BianchiL06}.
\end{remark}

\begin{remark}
We say that an online learning algorithm is \textbf{improper}\index{algorithm!improper|textbf} if the prediction strategy can be chosen from a larger class than the competitor one. So, the Vovk--Azoury--Warmuth forecaster is improper because we measure its regret with respect to a linear predictor $\bu$, but the Vovk--Azoury--Warmuth prediction depends in a non-linear way on $\bz_t$.
\end{remark}
\index{Vovk--Azoury--Warmuth forecaster|)textbf}

\section{Optimistic Follow-the-Regularized-Leader}
\label{sec:optimistic_ftrl}

\begin{algorithm}[h]
\caption{Optimistic Follow-the-Regularized-Leader (Optimistic FTRL)}
\label{alg:ftrl_optimistic}
\begin{algorithmic}[1]
{
    \REQUIRE{A sequence of regularizers $\psi_1, \dots, \psi_T :\mathcal{X} \to \R$, closed non-empty convex set $\mathcal{V} \subseteq \mathcal{X} \subseteq \R^d$}
    \FOR{$t=1$ {\bfseries to} $T$}
    \STATE{Predict next loss $\tilde{\ell}_t:\mathcal{V}\to \R$}
    \STATE{Output $\bx_t \in \argmin_{\bx \in \mathcal{V}} \ \psi_t(\bx) + \tilde{\ell}_t(\bx) + \sum_{i=1}^{t-1} \ell_i(\bx)$}
    \STATE{Receive $\ell_t:\R^d \to (-\infty,+\infty]$ such that $\bx_t \in \dom \ell_t$, and pay the loss $\ell_t(\bx_t)$}
    \ENDFOR
}
\end{algorithmic}
\end{algorithm}

In Section~\ref{sec:optimistic_omd}, we have seen that it is possible to achieve better regret guarantees through the so-called Optimistic \ac{OMD}, which uses a prediction of the next gradient. Here, we will extend the same idea to \ac{FTRL}, where we predict the next loss function.
If our predicted loss is correct, we can expect the regret to decrease. However, if our prediction is wrong, we still want to recover the worst-case guarantee.
Such an algorithm is called \textbf{Optimistic \ac{FTRL}}\index{Follow-the-Regularized-Leader algorithm!optimistic|(textbf}, see Algorithm~\ref{alg:ftrl_optimistic}.
As in Optimistic \ac{OMD}, we do not make any assumption on how the prediction of the next loss is generated.

Let's see why this is a good idea. Remember that \ac{FTRL} simply predicts with the minimizer of the previous losses plus a time-varying regularizer. Let's assume for a moment that instead we have the gift of predicting the future, so we do know the next loss ahead of time. Then, we could predict with its minimizer and suffer a negative regret. However, probably our foresight abilities are not so powerful, so our prediction of the next loss might be inaccurate. In this case, a better idea might be just to add our predicted loss to the previous ones and minimize the regularized sum. We would expect the regret guarantee to improve if our prediction of the future loss is precise. At the same time, if the prediction is wrong, we expect its influence to be limited, given that we use it together with all the past losses.

All these intuitions can be formalized in the following theorem.
\begin{theorem}
\label{thm:ftrl_optimistic}
With the notation in Algorithm~\ref{alg:ftrl_optimistic}, let $\mathcal{V}$ be convex, closed, non-empty.
Denote by $F_t(\bx) = \psi_{t}(\bx) + \sum_{i=1}^{t-1} \ell_i(\bx)$.
For $t=1, \dots, T$, if $F_t+\tilde{\ell}_t$ is closed, subdifferentiable, and strongly convex in $\mathcal{V}$, then $\bx_t$ exists and is unique.
In addition, for $t=1, \dots, T$, assume $\partial (\ell_t-\tilde{\ell}_t)(\bx_t)$ to be non-empty and $F_t+\ell_t$ to be closed, subdifferentiable, and $\lambda_t$-strongly convex with respect to $\|\cdot\|$ in $\mathcal{V}$.
Then, we have
\begin{align*}
\sum_{t=1}^T (\ell_t(\bx_t) - \ell_t(\bu))
&\leq \psi_{T}(\bu) - \psi_{1}(\bx_1) + \sum_{t=1}^T \left[\langle \hat{\bg}_t,\bx_t-\bx_{t+1}\rangle -\frac{\lambda_t}{2} \|\bx_t-\bx_{t+1}\|^2 \right]\\
&\quad +\sum_{t=1}^{T-1}\left[\psi_t(\bx_{t+1}) - \psi_{t+1}(\bx_{t+1})\right] \\
&\leq \psi_{T}(\bu) - \psi_{1}(\bx_1) + \sum_{t=1}^T \frac{\| \hat{\bg}_t\|_\star^2}{2\lambda_t} + \sum_{t=1}^{T-1}\left[\psi_t(\bx_{t+1}) - \psi_{t+1}(\bx_{t+1})\right],
\end{align*}
for all $\bu \in \mathcal{V}$ and all $\hat{\bg}_t \in \partial (\ell_t-\tilde{\ell}_t)(\bx_t)$ for $t=1, \dots, T$.
\end{theorem}
\begin{proof}
We can interpret the Optimistic FTRL as \ac{FTRL} with a regularizer $\tilde{\psi}_t(\bx)=\psi_t(\bx)+\tilde{\ell}_t(\bx)$. So, as in Lemma~\ref{lemma:ftrl_strongly_1}, the existence and uniqueness of $\bx_t$ is given by Theorem~\ref{thm:min_strongly_convex}.
Also, note that $\tilde{\ell}_{T+1}$ does not influence the regret of the algorithm over the $T$ rounds, so we can set $\tilde{\ell}_{T+1}=\tilde{\ell}_T$.

Define $F_t(\bx)=\psi_t(\bx)+\sum_{i=1}^{t-1} \ell_i(\bx)$.
We can use Lemma~\ref{lemma:ftrl_equality} with $\psi_{T+1}=\psi_T$. Given that $\bu \in \mathcal{V}$, we can discard the last term in the lemma because non-positive, to obtain
\begin{align*}
&\sum_{t=1}^T \ell_t(\bx_t) - \sum_{t=1}^T \ell_t(\bu)\\
&\quad\leq \tilde{\ell}_{T}(\bu) + \psi_{T}(\bu) - \min_{\bx \in \mathcal{V}}\   (\tilde{\ell}_1(\bx)+\psi_{1}(\bx))  \\
&\qquad +\sum_{t=1}^T [F_t(\bx_t) - F_{t+1}(\bx_{t+1}) + \ell_t(\bx_t)]+\sum_{t=1}^{T-1}[\tilde{\ell}_t(\bx_t)-\tilde{\ell}_{t+1}(\bx_{t+1})] \\
&\quad= \psi_{T}(\bu) - \psi_{1}(\bx_1) + \sum_{t=1}^T [F_t(\bx_t) - F_{t+1}(\bx_{t+1}) + \ell_t(\bx_t)],
\end{align*}
where we used the fact that the terms $\tilde{\ell}_t(\bx_t)-\tilde{\ell}_{t+1}(\bx_{t+1})$ form a telescopic sum and $\min_{\bx \in \mathcal{V}}\   (\tilde{\ell}_1(\bx)+\psi_{1}(\bx)) = \tilde{\ell}_1(\bx_1)+\psi_{1}(\bx_1)$.
Now focus on the terms $F_t(\bx_t) - F_{t+1}(\bx_{t+1}) + \ell_t(\bx_t)$. Observe that $F_t(\bx)+\ell_t(\bx)+\indicator_{\mathcal{V}}(\bx)$ is $\lambda_t$-strongly convex with respect to $\|\cdot\|$, hence from Lemma~\ref{lemma:strong_convexity} we have
\begin{align*}
F_t&(\bx_t) - F_{t+1}(\bx_{t+1}) + \ell_t(\bx_t) \\
&= (F_t(\bx_t) + \ell_t(\bx_t)) - (F_{t}(\bx_{t+1}) + \ell_t(\bx_{t+1})) + \psi_t(\bx_{t+1}) - \psi_{t+1}(\bx_{t+1}) \\
&\leq \langle \hat{\bg}_t,\bx_t-\bx_{t+1}\rangle -\frac{\lambda_t}{2} \|\bx_t-\bx_{t+1}\|^2 +\psi_t(\bx_{t+1}) - \psi_{t+1}(\bx_{t+1}),
\end{align*}
for all $\hat{\bg}_t \in \partial (F_t + \ell_t + \indicator_{\mathcal{V}})(\bx_t)$. Observing that $\bx_t = \argmin_{\bx \in \mathcal{V}} \ F_t(\bx) + \tilde{\ell}_t(\bx)$, we have $\boldsymbol{0} \in \partial (F_t + \tilde{\ell}_t + \indicator_{\mathcal{V}})(\bx_t)$. Hence, using Theorem~\ref{thm:sum_subgradients}, $\partial (\ell_t-\tilde{\ell}_t)(\bx_t) \subseteq \partial (F_t + \ell_t + \indicator_{\mathcal{V}})(\bx_t)$.

For the second inequality, by the definition of dual norms, we have that
\[
\langle \hat{\bg}_t,\bx_t-\bx_{t+1}\rangle -\frac{\lambda_t \|\bx_t-\bx_{t+1}\|^2}{2}
\leq \| \hat{\bg}_t\|_\star \|\bx_t-\bx_{t+1}\| -\frac{\lambda_t \|\bx_t-\bx_{t+1}\|^2}{2}
\leq \frac{\|\hat{\bg}_t\|^2_\star}{2 \lambda_t}~. \qedhere
\]
\end{proof}

\subsection{Improved Optimistic FTRL Bound for Linear Losses}

We now prove a similar improved regret guarantee for Optimistic FTRL with linearized losses. The analysis will use an \emph{auxiliary} sequence of predictions, and then relate these predictions to the ones of Optimistic \ac{FTRL}.

Define $\bx^{(a)}_t$ for $t=1, \dots, T$ the auxiliary sequence of predictions obtained with Optimistic \ac{FTRL} with linearized losses with losses $\ell_t(\bx)=\langle \bg_t, \bx\rangle$ and linear hints $\tilde{\ell}_t(\bx)=\langle \tilde{\bg}^{(a)}_t, \bx\rangle$. Given that this auxiliary sequence is only used in the analysis, we do not have to specify the choice of the hints $\tilde{\bg}^{(a)}_t$. Instead, we will leave them free, and state the final bound for any of these choices.
Formally, defining $F^{(a)}_t(\bx):=\psi_t(\bx) + \langle \tilde{\bg}^{(a)}_t, \bx\rangle + \sum_{i=1}^{t-1} \langle \bg_i, \bx\rangle$ we have $\bx^{(a)}_t = \argmin_{\bx \in \mathcal{V}} \ F^{(a)}_t(\bx)$.

\begin{theorem}
\label{thm:improved_ftrl_optimistic}
With the notation in Algorithm~\ref{alg:ftrl_optimistic}, let $\mathcal{V}$ be convex, closed, non-empty. Assume $\ell_t(\bx)=\langle \bg_t, \bx\rangle$ and $\tilde{\ell}_t(\bx)=\langle \tilde{\bg}_t, \bx\rangle$ for $t=1,\dots, T$.
For $t=1, \dots, T$, if $\psi_t$ is closed, subdifferentiable, and $\lambda_t$-strongly convex with respect to $\|\cdot\|$ in $\mathcal{V}$, then $\bx_t$ exists and is unique.
Moreover, if in addition $\psi_{t+1}\geq \psi_t$ pointwise for all $t$, we have
\begin{align*}
\sum_{t=1}^T \langle\bg_t, \bx_t - \bu\rangle
\leq \psi_{T}(\bu) - \min_{\bx \in \mathcal{V}} \psi_{1}(\bx) + \sum_{t=1}^T \left[\frac{\|\bg_t-\tilde{\bg}^{(a)}_t\|_\star^2}{2\lambda_t}  + \frac{\|\bg_t\|_\star \|\tilde{\bg}_t - \tilde{\bg}^{(a)}_t\|_\star}{\lambda_t} \right]\!,
\end{align*}
for all $\bu \in \mathcal{V}$ and all $\tilde{\bg}^{(a)}_t \in \R^d$ for $t=1, \dots, T$.
\end{theorem}

To prove this theorem, we will need the following stability lemma for \ac{FTRL}.
\begin{lemma}
\label{lemma:stability_ftrl}
Assume $\psi:\mathcal{V} \to \R$ to be $\lambda$-strongly convex\index{function!strongly convex} with respect to $\|\cdot\|$.
Define $\bx^{(1)} := \argmin_{\bx \in \mathcal{V}} \ \psi(\bx)+\langle \btheta^{(1)},\bx\rangle$ and $\bx^{(2)} := \argmin_{\bx \in \mathcal{V}} \ \psi(\bx)+\langle \btheta^{(2)},\bx\rangle$.
Then, $\|\bx^{(1)}-\bx^{(2)}\|\leq \frac{1}{\lambda} \|\btheta^{(1)}-\btheta^{(2)}\|_\star$.
\end{lemma}
\begin{proof}
Define $F^{(1)}(\bx):=\psi(\bx)+\langle \btheta^{(1)},\bx\rangle$ and $F^{(2)}(\bx):=\psi(\bx)+\langle \btheta^{(2)},\bx\rangle$.
From the strong convexity of $\psi$, we have
\begin{align*}
F^{(1)}(\bx^{(2)})-F^{(1)}(\bx^{(1)})
&\geq \frac{\lambda}{2}\|\bx^{(1)}-\bx^{(2)}\|^2\\
F^{(2)}(\bx^{(1)})-F^{(2)}(\bx^{(2)})
&\geq \frac{\lambda}{2}\|\bx^{(1)}-\bx^{(2)}\|^2~.
\end{align*}
Hence, summing these two inequalities, we have
\[
\|\btheta^{(1)}-\btheta^{(2)}\|_\star \, \|\bx^{(2)}-\bx^{(1)}\|
\geq \langle \btheta^{(1)}-\btheta^{(2)}, \bx^{(2)}-\bx^{(1)}\rangle
\geq \lambda\|\bx^{(1)}-\bx^{(2)}\|^2,
\]
that completes the proof.
\end{proof}

We can now prove Theorem~\ref{thm:improved_ftrl_optimistic}.
\begin{proof}[Proof of Theorem~\ref{thm:improved_ftrl_optimistic}]
Using Theorem~\ref{thm:ftrl_optimistic} on the predictions $\bx^{(a)}_t$, for all $\bu \in \mathcal{V}$, we can upper bound the regret of $\bx_t$ as
\begin{align*}
\sum_{t=1}^T \langle \bg_t, \bx_t-\bu\rangle
&= \sum_{t=1}^T \langle \bg_t, \bx^{(a)}_t-\bu\rangle + \sum_{t=1}^T \langle \bg_t, \bx_t-\bx^{(a)}_t\rangle\\
&\leq \psi_{T}(\bu) - \psi_1(\bx^{(a)}_1)
+ \sum_{t=1}^T \left(\frac{\|\bg_t -\tilde{\bg}^{(a)}_t\|^2_\star}{2\lambda_t}
+ \|\bg_t\|_\star \|\bx_t - \bx^{(a)}_t\|\right)\\
&\leq \psi_{T}(\bu) - \min_{\bx \in \mathcal{V}} \psi_1(\bx)
+ \sum_{t=1}^T \left(\frac{\|\bg_t -\tilde{\bg}^{(a)}_t\|^2_\star}{2\lambda_t}
+ \|\bg_t\|_\star \|\bx_t - \bx^{(a)}_t\|\right).
\end{align*}
Now, we focus our attention on $\|\bx_t - \bx^{(a)}_t\|$. From Lemma~\ref{lemma:stability_ftrl}, we have $\|\bx_t^{(a)}-\bx_t\|\leq \frac{1}{\lambda_t} \|\tilde{\bg}_t-\tilde{\bg}^{(a)}_t\|_\star$.

Using this inequality in the regret bound above, for any $\tilde{\bg}_1^{(a)}, \dots, \tilde{\bg}^{(a)}_T$ and any $\bu \in \mathcal{V}$, we have
\begin{align*}
\sum_{t=1}^T \langle \bg_t, \bx_t-\bu\rangle
&= \sum_{t=1}^T \langle \bg_t, \bx^{(a)}_t-\bu\rangle + \sum_{t=1}^T \langle \bg_t, \bx_t-\bx^{(a)}_t\rangle\\
&\leq \psi_{T}(\bu) - \min_{\bx \in \mathcal{V}} \psi_1(\bx)
+ \sum_{t=1}^T \left(\frac{\|\bg_t -\tilde{\bg}^{(a)}_t\|^2_\star}{2\lambda_t} + \frac{\|\bg_t\|_\star \|\tilde{\bg}_t - \tilde{\bg}^{(a)}_t\|_\star}{\lambda_t}\right).
\end{align*}
Remember that $\tilde{\bg}^{(a)}_t$ is only used in the analysis. So, this upper bound holds for any choice of $\tilde{\bg}^{(a)}_t$.
\end{proof}

Note that the minimization with respect to $\tilde{\bg}_t^{(a)}$ can be solved only with the knowledge of the dual norm used. However, we have the following upper bound (the proof of the first inequality is left as an exercise, see Problem~\ref{exercise:best_surrogate_hint}):
\begin{align*}
&\min_{\tilde{\bg}_t^{(a)}} \ \frac{\|\bg_t -\tilde{\bg}_t^{(a)}\|^2_\star}{2} +\|\bg_t\|_\star \|\tilde{\bg}_t - \tilde{\bg}_t^{(a)}\|_\star\\
&\quad\leq \frac{1}{2} \left[\|\bg_t -\tilde{\bg}_t\|^2_\star - \max(\|\bg_t -\tilde{\bg}_t\|_\star-\|\bg_t\|_\star,0)^2\right]\\
&\quad \leq \frac12 \|\bg_t -\tilde{\bg}_t\|_\star \min(\|\bg_t -\tilde{\bg}_t\|_\star, 2\|\bg_t\|_\star)~.
\end{align*}
Let's compare this last expression to the similar one in the regret bound for \ac{FTRL}. If the prediction of the subgradient of the next loss is good, that is, $\bg_t \approx \tilde{\bg}_t$, that term can become smaller and possibly even zero! On the other hand, if the predictions are bad, we only pay linearly in $\|\bg_t -\tilde{\bg}_t\|_\star$. Overall, in the best case, we can gain a lot; in the worst case, we do not lose that much.

\begin{example}
In the case of linear losses $\ell_t(\bx)=\langle \bg_t,\bx\rangle$ and strongly convex regularizers, the update of optimistic \ac{FTRL} becomes
\[
\bx_t = \nabla \psi^\star_{\mathcal{V},t} \left(- \tilde{\bg}_t -\sum_{i=1}^{t-1} \bg_i \right),
\]
where $\psi_{\mathcal{V},t}=\psi_t+\indicator_{\mathcal{V}}$ and we used the hint $\tilde{\ell}_t(\bx)=\langle \tilde{\bg}_t, \bx\rangle$.
For example, if $\psi_{t}(\bx)=\frac{\sqrt{t}}{2} \|\bx\|_2^2$, we have
\[
\bx_t = \Pi_{\mathcal{V}}\left( \frac{- \tilde{\bg}_t -\sum_{i=1}^{t-1} \bg_i}{\sqrt{t}}\right),
\]
where $\Pi_{\mathcal{V}}$ is the projection onto $\mathcal{V}$.
\end{example}

Despite the simplicity of the algorithm and its analysis, there are many applications of this principle. Here, we will only describe a few of them, while in Section~\ref{sec:saddle_point_optimism} we describe the applications of optimistic algorithms for saddle-point optimization.

\section{Application: FTRL with Delayed Feedback}
\label{sec:delayed_ftrl}

In Section~\ref{sec:delayed_omd}, we saw that it is possible to reduce \ac{OMD} with delayed feedback to Optimistic \ac{OMD}. Here, we will do the same for \ac{FTRL}, obtaining a slightly different bound.

Similarly to the derivation in Section~\ref{sec:delayed_omd}, we assume to receive the feedback with a delay of $\tau$ rounds.
So, using FTRL with linearized losses with delays means that we predict with $\bx_t=\argmin_{\bx \in \mathcal{V}} \ \psi_t(\bx)$ for $t\leq \tau+1$ and
\[
\bx_t
= \argmin_{\bx \in \mathcal{V}} \ \psi_t(\bx) + \sum_{i=1}^{t-1-\tau} \langle \bg_i, \bx\rangle, \quad t \geq \tau+2~.
\]
On the other hand, in Optimistic FTRL without delays and linear losses, we predict with
\[
\bx_t
= \argmin_{\bx \in \mathcal{V}} \ \psi_t(\bx) + \langle \tilde{\bg}_t,\bx\rangle +\sum_{i=1}^{t-1} \langle \bg_i, \bx\rangle~.
\]
Hence, the two updates are equivalent if we set $\tilde{\bg}_t= - \sum_{i=\max(1,t-\tau)}^{t-1} \bg_i$.

So, we can choose $\tilde{\bg}^{(a)}_t=\bg_t$ in Theorem~\ref{thm:improved_ftrl_optimistic}. Then, using the delay-to-optimism conversion above, we have that \ac{FTRL} with delayed linearized losses and increasing regularizer satisfies
\begin{equation}
\label{eq:delayed_ftrl}
\sum_{t=1}^T \langle \bg_t, \bx_t-\bu\rangle
\leq \psi_{T}(\bu) - \min_{\bx \in \mathcal{V}} \ \psi_1(\bx)
+ \sum_{t=1}^T \frac{\|\bg_t\|_\star \|\sum_{i=\max(1,t-\tau)}^{t} \bg_i\|_\star}{\lambda_t}, \quad \forall \bu \in \mathcal{V}~.
\end{equation}
That is, assuming that the $\bg_t$ are bounded, we improved the dependence in the sum from $\mathcal{O}(\tau^2)$ to $\mathcal{O}(\tau)$. Choosing the strong convexity of the regularizer optimally, we obtain the optimal regret guarantee of $\mathcal{O}(\sqrt{\tau T})$.

\section{Application: Regret that Depends on the Variance of the Subgradients}
\label{sec:oftrl_variance}

Consider running Optimistic-FTRL on the linearized losses $\ell_t(\bx)=\langle \bg_t, \bx\rangle$.
We can gain something out of the Optimistic-FTRL compared to plain \ac{FTRL} if we can predict the next $\bg_t$.
A simple possibility is to predict the average of the past values, $\bar{\bg}_t=\frac{1}{t-1}\sum_{i=1}^{t-1} \bg_i$. Indeed, from the first chapter, we know that such a strategy is itself an online learning procedure! In particular, it corresponds to a \ac{FTL} algorithm on the losses $\bx \mapsto \|\bx-\bg_t\|_2^2$. Hence, from the strong convexity of these losses and assuming $\|\bg_t\|_2\leq 1$, we know that
\[
\sum_{t=1}^T \|\bar{\bg}_t - \bg_t\|^2_2 - \sum_{t=1}^T \|\bu - \bg_t\|^2_2 \leq 4(1+\ln T), \quad \forall \bu \in \R^d~.
\]
This implies
\[
\sum_{t=1}^T \|\bar{\bg}_t - \bg_t\|^2_2 \leq 4(1 + \ln T) + \min_{\bu} \sum_{t=1}^T \|\bu - \bg_t\|^2_2~.
\]
It is immediate to see that the minimizer is $\bu=\frac{1}{T}\sum_{t=1}^T \bg_t$, which results in $T$ times the empirical ``variance'' of the subgradients, that is, the sum of the empirical variances of the coordinates of the subgradients.
Plugging it into the Optimistic-FTRL regret, with $\psi_t=\psi$, we have
\[
\sum_{t=1}^T \ell_t(\bx_t) - \sum_{t=1}^T \ell_t(\bu)
\leq \psi(\bu) - \psi(\bx_1) + \frac{4+4\ln T+\min_{\bg} \sum_{t=1}^T \|\bg_t-\bg\|_2^2}{2\lambda}~.
\]


\begin{remark}
Instead of using the mean of the past subgradients, we could use any other strategy or even a mix of different strategies.
For example, assuming the subgradients are bounded, we could use an algorithm to solve the \ac{LEA} problem, where each expert is a strategy. Then, we would obtain a bound that depends on the predictions of the best strategy, plus the regret of the expert algorithm.
\end{remark}

\section{Application: Online Convex Optimization with Gradual Variations}
\label{sec:oftrl_gradual_variations}

In this section, we consider the case where the losses we receive have small variations over time. We will show that in this case, it is possible to get constant regret in the case that the losses are equal.

In this case, the simple strategy we can use to predict the next subgradient is to use the previous one, that is $\tilde{\ell}_t(\bx)=\langle \bg_{t-1}, \bx\rangle$ for $t\geq2$ and $\tilde{\ell}_1(\bx)=0$.
\begin{corollary}
\label{cor:ftrl_optimistic_gradual}
Under the assumptions of Theorem~\ref{thm:ftrl_optimistic}, also assume that the losses $\ell_t$ are $m$-smooth\index{function!smooth} with respect to $\|\cdot\|$.
Define $\tilde{\ell}_t(\bx)=\langle \nabla \ell_{t-1}(\bx_{t-1}), \bx\rangle$ for $t\geq2$ and $\tilde{\ell}_1(\bx)=0$ for all $\bx \in \mathcal{V}$. Set $\psi_t(\bx)=\lambda_t \psi(\bx)$ where $\psi$ is 1-strongly convex with respect to $\|\cdot\|$ and $\lambda_t>0$ satisfies $\lambda_t \lambda_{t-1}\geq 8m^2$ and $\lambda_t\geq \lambda_{t-1}$ for $t=2, \dots, T$. Then, $\forall \bu \in \mathcal{V}$, we have
\begin{align*}
\sum_{t=1}^T \ell_t(\bx_t) - \sum_{t=1}^T \ell_t(\bu)
&\leq \lambda_{T} \psi(\bu) - \lambda_1 \psi(\bx_1) + \frac{1}{\lambda_1} \|\nabla \ell_1(\bx_1)\|_\star^2 \\
&\quad + \sum_{t=2}^T \frac{2}{\lambda_t} \|\nabla \ell_t(\bx_{t-1})-\nabla \ell_{t-1}(\bx_{t-1})\|^2_\star~.
\end{align*}

In addition, assume that in each round $t$ we can also query the gradient of $\ell_t$ in $\bx_{t-1}$, fix $\alpha>0$, assume $\|\nabla \ell_t(\bx)\|_\star \leq L$ for all $\bx \in \mathcal{V}$, and set
\[
\lambda_t
= \frac{1}{\alpha}\sqrt{\max(8 \alpha^2 m^2 ,4L^2) + \sum_{i=2}^{t-1} \|\nabla \ell_i(\bx_{i-1})-\nabla \ell_{i-1}(\bx_{i-1})\|^2_\star}~.
\]
Then, we have
\begin{align*}
&\sum_{t=1}^T \ell_t(\bx_t) - \sum_{t=1}^T \ell_t(\bu)\\
&\ \leq \left(\frac{\psi(\bu) - \psi(\bx_1)}{\alpha}+ 4\alpha\right)\sqrt{\max(8 \alpha^2 m^2 , 4L^2) + \sum_{t=2}^T \|\nabla \ell_t(\bx_{t-1})-\nabla \ell_{t-1}(\bx_{t-1})\|^2_\star}\\
& \ \quad + \frac{\alpha L}{2}~.
\end{align*}
\end{corollary}
\begin{proof}
From the Optimistic-FTRL bound, we immediately get
\begin{align*}
&\sum_{t=1}^T (\ell_t(\bx_t) - \ell_t(\bu))\\
&\quad \leq \lambda_{T}\psi(\bu) - \lambda_1 \psi(\bx_1) + \sum_{t=1}^T \left[\langle \bg_t-\tilde{\bg}_t,\bx_t-\bx_{t+1}\rangle -\frac{\lambda_t \|\bx_t-\bx_{t+1}\|^2}{2} \right],
\end{align*}
where $\bg_t=\nabla \ell_t(\bx_t)$.
Now, consider the case that the losses $\ell_t$ are $m$-smooth.
So, for any $\beta_t>0$, we have
\begin{align*}
\langle &\bg_t-\tilde{\bg}_t,\bx_t-\bx_{t+1}\rangle -\frac{\lambda_t}{2} \|\bx_t-\bx_{t+1}\|^2\\
&= \langle \nabla \ell_t(\bx_t)-\nabla \ell_{t-1}(\bx_{t-1}),\bx_t-\bx_{t+1}\rangle -\frac{\lambda_t}{2} \|\bx_t-\bx_{t+1}\|^2 \\
&\leq \frac{\beta_t}{2} \|\nabla \ell_t(\bx_t)-\nabla \ell_{t-1}(\bx_{t-1})\|^2_\star +\frac{1}{2\beta_t}\|\bx_t-\bx_{t+1}\|^2-\frac{\lambda_t}{2} \|\bx_t-\bx_{t+1}\|^2~.
\end{align*}
Focusing on the first term, for $t=2, \dots, T$, we have
\begin{align*}
\frac{\beta_t}{2} &\|\nabla \ell_t(\bx_t)-\nabla \ell_{t-1}(\bx_{t-1})\|^2_\star\\
&= \frac{\beta_t}{2} \|\nabla \ell_t(\bx_t)-\nabla \ell_{t-1}(\bx_{t-1}) - \nabla \ell_t(\bx_{t-1}) + \nabla \ell_t(\bx_{t-1})\|^2_\star \\
&\leq \beta_t \|\nabla \ell_t(\bx_t)-\nabla \ell_{t}(\bx_{t-1})\|^2_\star +\beta_t \|\nabla \ell_t(\bx_{t-1})-\nabla \ell_{t-1}(\bx_{t-1})\|^2_\star \\
&\leq \beta_t m^2 \|\bx_t-\bx_{t-1}\|^2 + \beta_t \|\nabla \ell_t(\bx_{t-1})-\nabla \ell_{t-1}(\bx_{t-1})\|^2_\star~.
\end{align*}
Choose $\beta_t=\frac{2}{\lambda_t}$. We have for $t=2,\dots, T$
\begin{align*}
&\langle \bg_t-\tilde{\bg}_t,\bx_t-\bx_{t+1}\rangle -\frac{\lambda_t}{2} \|\bx_t-\bx_{t+1}\|^2\\
&\quad \leq \frac{2m^2}{\lambda_t} \|\bx_t-\bx_{t-1}\|^2 + \frac{2}{\lambda_t} \|\nabla \ell_t(\bx_{t-1})-\nabla \ell_{t-1}(\bx_{t-1})\|^2_\star -\frac{\lambda_t}{4} \|\bx_t-\bx_{t+1}\|^2~.
\end{align*}
For $t=1$, using the fact that $\tilde{\ell}_1$ is the null function, we have
\[
\langle \bg_1-\tilde{\bg}_1,\bx_1-\bx_{2}\rangle -\frac{\lambda_1}{2} \|\bx_1-\bx_{2}\|^2 \\
\leq \frac{1}{\lambda_1}\|\nabla \ell_1(\bx_1)\|^2_\star  -\frac{\lambda_1}{4} \|\bx_1-\bx_{2}\|^2~.
\]
Now, observe that the assumption on $\lambda_t$ implies $\frac{2m^2}{\lambda_t}\leq \frac{\lambda_{t-1}}{4}$ for $t=2, \dots,T$. So, summing for $t=1,\dots,T$, we have
\begin{align*}
&\sum_{t=1}^T \left( \langle \bg_t-\tilde{\bg}_t,\bx_t-\bx_{t+1}\rangle - \frac{\lambda_t}{2} \|\bx_t-\bx_{t+1}\|^2 \right)\\
&\quad \leq \frac{1}{\lambda_1} \|\nabla \ell_1(\bx_1)\|_\star^2 + \sum_{t=2}^T \frac{2}{\lambda_t} \|\nabla \ell_t(\bx_{t-1})-\nabla \ell_{t-1}(\bx_{t-1})\|^2_\star~.
\end{align*}
Putting everything together, we have the first stated bound.

The second one is obtained by observing that
\begin{align*}
&\sum_{t=2}^T \frac{\|\nabla \ell_t(\bx_{t-1})-\nabla \ell_{t-1}(\bx_{t-1})\|^2_\star}{\lambda_t}\\
&\quad= \alpha \sum_{t=2}^T \frac{\|\nabla \ell_t(\bx_{t-1})-\nabla \ell_{t-1}(\bx_{t-1})\|^2_\star}{\sqrt{\max(8 \alpha^2 m^2 ,4L^2) + \sum_{i=2}^{t-1} \|\nabla \ell_i(\bx_{i-1})-\nabla \ell_{i-1}(\bx_{i-1})\|^2_\star}} \\
&\quad \leq \alpha \sum_{t=2}^T \frac{\|\nabla \ell_t(\bx_{t-1})-\nabla \ell_{t-1}(\bx_{t-1})\|^2_\star}{\sqrt{\sum_{i=2}^{t} \|\nabla \ell_i(\bx_{i-1})-\nabla \ell_{i-1}(\bx_{i-1})\|^2_\star}} \\
&\quad \leq 2\alpha \sqrt{\sum_{i=2}^{T} \|\nabla \ell_i(\bx_{i-1})-\nabla \ell_{i-1}(\bx_{i-1})\|^2_\star}~. \qedhere
\end{align*}
\end{proof}

Note that if the losses are all the same, the regret becomes a constant! This is not surprising because the prediction of the next loss is a linear approximation of the previous loss. Indeed, looking back at the proof, the key idea is to use the smoothness to argue that, even if the past subgradient was taken in a different point than the current one, it is still a good prediction of the current subgradient.

\begin{remark}
Note that the assumption of smoothness\index{function!smooth} is necessary. Indeed, always passing the same function and using the online-to-batch conversion would result in a convergence rate of $\mathcal{O}(1/T)$ for a Lipschitz function, which is impossible by the lower bound in offline convex optimization.
\end{remark}
\index{Follow-the-Regularized-Leader algorithm!optimistic|)textbf}

\section{History Bits}
The \ac{FTRL} algorithm has a peculiar story. It was introduced by \emph{four} different groups roughly around 1999-2006, following four different points of view.
The algorithm was officially introduced in \citet{AbernethyHR08,HazanK08} where at each step the prediction is computed as the minimizer of a regularization term plus the sum of losses on all past rounds. However, the key ideas were already presented years before. \citet{Gordon99, Gordon99b} appear to have proposed it for the first time, naming it Maximum a Posteriori with a clear inspiration to Bayesian methods. Also, the analysis in \citet{Gordon99, Gordon99b} is the first one to use concepts from convex analysis.
Later, \citet{AgarwalH05} proposed to add a regularizer to stabilize the \ac{FTL} algorithm~\citep{Hannan57} in an algorithm named Smooth Prediction for the specific application of Portfolio Selection. A little bit later, Shai Shalev-Shwartz and Yoram Singer proposed a very general analysis of \ac{FTRL} using a dual perspective~\citep{ShalevS06,ShalevS06b}. In particular, the PhD thesis of Shai Shalev-Shwartz~\citep{Shalev-Shwartz07} contains the most precise dual analysis of \ac{FTRL}, which also allows multiple dual updates and regularizers with a time-varying weight. However, he called it ``online mirror descent'' because the name \ac{FTRL} was only invented later! (I also contributed to the confusion, naming a general analysis of \ac{FTRL} with time-varying regularizer and linear losses ``generalized online mirror descent''~\citep{OrabonaCCB13}. So, now I am trying to set the record straight!)
The offline optimization community also knows \ac{FTRL}, but only with linear losses, and they call it Dual Averaging~\citep{Nesterov09}\index{Dual Averaging algorithm|(}. Note that the paper \citet{Nesterov09} is actually from 2005\footnote{\url{https://papers.ssrn.com/sol3/papers.cfm?abstract_id=912637}}, and in the paper Nesterov claims that these ideas are from 2001-2002, but he decided not to publish them for a while because he was convinced that ``the importance of black-box approach in Convex Optimization will be irreversibly vanishing, and, finally, this approach will be completely replaced by other ones based on a clever use of problem's structure (interior-point methods, smoothing, etc.).'' As Shai Shalev-Shwartz, Nesterov looks at \ac{FTRL} from the dual point of view, but he focuses only on linear losses and only in the offline case.\index{Dual Averaging algorithm|)}

Finally, \citet{McMahan17} gives the elegant equality result that I presented here (with minor improvements) that holds for general loss functions and regularizers. A similar equality, setting $\bu=\bx_{T+1}$ and specialized only to linear losses, was also proven in \citet{AbernethyLST14}, essentially matching the inequality in \citet{OrabonaCCB13}. Note that Dual Averaging stems from the dual interpretation of \ac{FTRL} for linear losses, but Lemma~\ref{lemma:ftrl_equality} underlines the fact that \ac{FTRL} is actually more general.
People used to prove the regret bound for \ac{FTRL} using the so-called Be-the-Leader-Follow-the-Leader lemma (see Lemma~\ref{lemma:be_leader}\index{Be-the-Leader!lemma} and Problem~\ref{exercise:ftl-btl}). However, the proof is off by a factor of 2 in the case of fixed regularizers~\citep{McMahan17}. Moreover, it seems to fail in the case of generic strongly convex increasing regularizers, while it works for the particular case of proximal regularizers~\citep{McMahan17}.\index{regularizer!proximal}

Another source of confusion stems from the fact that some people differentiate between a ``lazy'' and ``greedy'' version of \ac{OMD}. \citet{Zinkevich03} introduced both ``greedy projection'' and ``lazy projection'' algorithms. The greedy projection algorithm was later renamed projected online gradient descent. We now know that the lazy projection algorithm is \ac{FTRL} with linearized losses in the special case of $\psi_t(\bx)=\frac{\eta}{2}\|\bx\|_2^2$. More generally, as proved in \citet{McMahan17}, what people call the lazy version of \ac{OMD} is just \ac{FTRL} with linearized losses, and the greedy one is just \ac{OMD}.

The concept of generalized Bregman divergence\index{Bregman divergence!generalized} has been rediscovered many times in the literature, with different names too. It was first introduced as the canonical divergence of dually flat spaces in information geometry in \citet[Eq. 3.21]{Amari85}. Independently, it was proposed in (online) optimization by \citet{Gordon99, Gordon99b} to analyze \ac{FTRL} for fixed regularizers. The equality in Remark~\ref{remark:generalized_bregman} might be new, but immediate from the regret equality of \ac{FTRL}.

The second result in Lemma~\ref{lemma:ftrl_local_norms} is a distillation of the reasoning used in \citet{ZimmertS19,ZimmertS21}, but without using duality concepts.

The use of \ac{FTRL} with the regularizer in~\eqref{eq:adaftrl} was proposed in \citet{OrabonaP15, OrabonaP18}, while here I presented a simpler version of their proof that does not require Fenchel conjugates. \index{AdaHedge algorithm|(}The AdaHedge algorithm was introduced in \citet{ErvenKRG11} and refined in \citet{RooijEGK14}. The analysis reported here is new, but inspired by the one in \citet{OrabonaP15, OrabonaP18}, which generalized AdaHedge to arbitrary regularizers in AdaFTRL. Additional properties of AdaHedge for the stochastic case were proven in \citet{RooijEGK14}.
\index{AdaHedge algorithm|)}

The first implicit use of the $(2,1)$ group norm I could find is in \citet{Bakin99}, where it is used as a constraint in the first proposed formulation of Group Lasso. However, there is no mention of the fact that this constraint is a norm. \citet{DingZHZ06} propose the use of the $(2,1)$ group norm for Principal Component Analysis, calling it $R_1$ norm. A few months later, the same group norm was proposed independently in \citet{ArgyriouEP06}. \citet{AgarwalRB08} extended the $(2,1)$ group norm to $(p,s)$ and used it in \ac{OMD} for multitask online learning. \citet{KakadeSST09} proved the smoothness\index{function!smooth} and strong convexity of squared group norms. \citet{JieOFCC10} used the $(2,s)$ group norm \ac{FTRL} for online multi-kernel learning, while \citet{OrabonaJC10,OrabonaJC12} used the same norm in a Perceptron-like algorithm.

Nesterov's Dual Averaging was extended to deal with composite losses\index{composite loss} by \citet{Xiao09,Xiao10} in the stochastic and online learning setting, essentially rediscovering the framework of Shalev-Shwartz, and calling the resulting framework Regularized Dual Averaging.
The analysis presented here using the negative terms $\psi_t(\bx_{t+1}) - \psi_{t+1}(\bx_{t+1})$ to easily prove regret bounds for \ac{FTRL} for composite losses\index{composite loss} is from \citet{OrabonaCCB13}.

The first proof of \ac{FTRL} for strongly convex\index{function!strongly convex} losses was in \citet{Shalev-ShwartzS07b} (even if they do not call it \ac{FTRL}). In the same paper, they also define strong convexity with respect to Bregman divergences\index{Bregman divergence}, rather than just norms, and prove the same logarithmic regret bound (see Problem~\ref{exercise:relative_strong_convexity}). Analogously, people later defined the concept of smoothness\index{function!smooth} with respect to a Bregman divergence~\citep{BirnbaumDX10,BirnbaumDX11}, rediscovered in \citet{BauschkeBT17}. Later, both concepts were rebranded as ``relative smoothness''\index{relative smoothness} and ``relative strong convexity''\index{relative strong convexity}~\citep{LuFN18}.

\index{FTRL-Proximal algorithm|(}
FTRL-Proximal was proposed in \citet{McMahanS10} for AdaGrad\index{AdaGrad algorithm} and then further analyzed in \citet{McMahan11}.
There is an interesting bit about \ac{FTRL}-Proximal: it became very famous in internet companies when Google disclosed in a very influential paper that they were using \ac{FTRL}-Proximal to train the classifier for click prediction~\citep{McMahanHSYEGNPDGCLWHBK13}. This generated some confusion because many people started conflating the term \ac{FTRL}-Proximal (a specific algorithm) with \ac{FTRL} (an entire family of algorithms).
\index{FTRL-Proximal algorithm|)}

\index{Online Newton Step algorithm|(}
The \ac{ONS} algorithm was introduced in \citet{HazanKKA06,HazanAK07} and it is described for the particular case that the loss functions are exp-concave\index{function!exp-concave}. The assumption in \eqref{eq:restriced_sc} with a fixed matrix $\bA$ for all points $\bx$ is called $\bA$-restricted strong convexity\index{A-restricted strong convexity@$\bA$-restricted strong convexity} in \citet{RyuB14}, and it was used for the first time in an online learning algorithm in this book.
Exp-concave functions\index{function!exp-concave} were introduced in \citet{KivinenW99}. Here, I have extended them to extended-value functions\index{function!extended-real-valued}. Theorem~\ref{thm:exp_concave_twice_diff} and Example~\ref{example:exp_concave_l2} are from \citet{KivinenW99}.
Here, I described a slight generalization for any sequence of functions that satisfy \eqref{eq:exp_concave}, that in my view better shows the parallel between \ac{FTRL} over strongly convex functions and \ac{ONS}. Note that \citet{HazanKKA06,HazanAK07} also describes a variant of \ac{ONS} based on \ac{OMD}, but I find its analysis less interesting from a didactic point of view. Lemma~\ref{lemma:log_eigen} dates back to \citet[Lemma 2]{LaiW82}. \citet{OrabonaCG12} obtained a logarithmic $\mathscr{L}^\star$ bound, i.e., proportional to $\ln(1+\sum_{t=1}^T \ell_t(\bu))$, for exp-concave\index{function!exp-concave} and smooth losses\index{function!smooth} using \ac{ONS}.
\index{Online Newton Step algorithm|)}

\index{Vovk--Azoury--Warmuth forecaster|(}
The Vovk--Azoury--Warmuth forecaster was introduced independently by \citet{AzouryW01} and \citet{Vovk01}, and its interpretation as hallucinating the future loss is by \citet{AzouryW01}. The proof presented here is from \citet{OrabonaCCB13}.
\index{Vovk--Azoury--Warmuth forecaster|)}

\index{Follow-the-Regularized-Leader algorithm!optimistic|(}
The Optimistic \ac{FTRL} version was proposed in \citet{RakhlinS13b} but analyzed only with fixed self-concordant regularizers\index{regularizer!self-concordant}. \citet{Steinhardt14} proposed Optimistic \ac{FTRL} as we know it, even if it was called ``Optimistic Mirror Descent'' for the misnaming problem I have explained above. Due to the misnaming, \citet{MohriY16} appeared to have rediscovered Optimistic \ac{FTRL} with non-self-concordant regularizers.
The proof of Theorem~\ref{thm:ftrl_optimistic} I present here is new. Theorem~\ref{thm:improved_ftrl_optimistic} is from \citet{FlaspohlerOCMOOM21} as well as the reduction from delayed feedback to optimism in Section~\ref{sec:delayed_ftrl}. See the History Bits of Chapter~\ref{ch:omd} for additional pointers to online learning with delayed feedback.
\index{Follow-the-Regularized-Leader algorithm!optimistic|)}

Corollary~\ref{cor:ftrl_optimistic_gradual} was proved by \citet{Chiang12} for Optimistic \ac{OMD} and presented in a similar form for Optimistic \ac{FTRL} in \citet{JoulaniGS17}, but only for bounded domains.

\section{Exercises}

\begin{exer}
Prove that the update of \ac{FTRL} with linearized loss in Example~\ref{example:omd_vs_ftrl} is correct.
\end{exer}

\begin{exer}
\label{exercise:ftrl_last_bregman}
Consider \ac{FTRL} with linearized losses with regularizers $\psi_t(\bx)$. The regret equality in Lemma~\ref{lemma:ftrl_equality} has a term equal to $F_{T+1}(\bx_{T+1})-F_{T+1}(\bu)$. Assume that $\mathcal{V}$ is convex and non-empty, and $\psi_{T+1}$ is strictly convex\index{function!strictly convex} and differentiable at the point $\bx_{T+1}$. Then, show that $F_{T+1}(\bx_{T+1})-F_{T+1}(\bu)\leq -B_{\psi_{T+1}}(\bu; \bx_{T+1})$ for any $\bu \in \mathcal{V}$.
\end{exer}

\begin{exer}
\label{exercise:ftl-btl}
Consider \ac{FTRL} with losses $\ell_t(\bx)$ and regularizers $\psi_t(\bx)$. Assume that $\psi_{t+1}\geq \psi_{t}$ for $t=1, \dots, T-1$.
Denote by $F_t(\bx):=\psi_t(\bx)+ \sum_{i=1}^{t-1} \ell_i(\bx)$ and show that
\[
F_t(\bx_t) - F_{t+1}(\bx_{t+1}) + \ell_t(\bx_t)
\leq \ell_t(\bx_t) - \ell_t(\bx_{t+1})~.
\]
\end{exer}

\begin{exer}
We can obtain $\mathscr{L}^\star$ bounds for smooth\index{function!smooth} and Lipschitz losses with linearized \ac{FTRL}, even without knowing the Lipschitz constant. In particular, show that under these assumptions the regularizers $\psi_t(\bx)=\frac{\lambda_t}{2}\|\bx\|_2^2$, where $\lambda_t=\sqrt{a + \sum_{i=1}^{t-1} \|\bg_i\|_2^2}$ and $a>0$, would give an $\mathscr{L}^\star$ bound. Hint: use \citet[Lemma 14]{GaillardSV14}.
\end{exer}

\begin{exer}
Define $\lambda_t=a+\sum_{i=1}^{t-1} \frac{b_i}{ \lambda_i}$, where $a>0$. Let $b_t \in [0,m]$ and prove that $\sum_{t=1}^T \frac{b_t}{\lambda_t} \leq \sqrt{2 \sum_{t=1}^T b_t}+\frac{m}{a}$. Use this inequality to design an alternative regularizer that solves the previous exercise. Hint: see \citet[Proof of Theorem 5]{SachsHvEG22}.
\end{exer}

\begin{exer}
Let $\mathcal{V}=\R^d$ and prove that the update in \eqref{eq:ftrl_upd_strongly} is equivalent to the one of \ac{OSD} with learning rate $\eta_t=\frac{1}{\sum_{i=1}^t \mu_i}$, when both algorithms start with $\bx_1=0$.
\end{exer}

\begin{exer}
\label{exercise:logistic_exp_concave}
Prove the statement in Example~\ref{ex:logistic_exp_concave}.
\end{exer}

\begin{exer}
\label{exercise:pnorm_ftrl}
Design and analyze the \ac{FTRL} version of \ac{OMD} with $p$-norms in Section~\ref{sec:omd_pnorm}.
\end{exer}

\begin{exer}
Consider the \ac{LEA} setting and assume $\bg_{t,i} \in [0,1]$ for all $t=1,\dots T$ and $i=1,\dots,d$.
From the analysis of \ac{EG} with local norms\index{norm!local}, see~\eqref{eq:eg_ftrl_local_norm}, derive a time-varying regularizer that gives a regret upper bound proportional to $\ln d + \sqrt{\mathscr{L}^\star(\bu) \ln d}$, where $\mathscr{L}^\star(\bu) := \sum_{t=1}^T \langle \bg_t, \bu\rangle$ for any $\bu \in \Delta^{d-1}$\index{L* bound@$\mathscr{L}^\star$ bound}.
\end{exer}

\begin{exer}
Prove a regret bound for optimistic \ac{FTRL} with proximal regularizers.\index{regularizer!proximal}
\end{exer}

\begin{exer}
Prove that the losses $\ell_t(\bx)=\frac{1}{2}(\langle \bz_t, \bx\rangle-y_t)^2$, where $\|\bz_t\|_2\leq 1$ and $|y_t|\leq1$, are exp-concave\index{function!exp-concave} on $\mathcal{V}=\{\bx\in \R^d: \|\bx\|_2\leq 1\}$, and find the exp-concavity constant.
\end{exer}

\index{Online Newton Step algorithm|(}
\begin{exer}
\label{exercise:ons_two_steps}
Prove that the \ac{ONS} update is equivalent to the following two steps:
\begin{align*}
\tilde{\bx}_t &= S_{t-1}^{-1}\left(\sum_{i=1}^{t-1} \bA_i \bx_i - \sum_{i=1}^{t-1} \bg_i\right) \\
\bx_t &= \argmin_{\bx \in \mathcal{V}} \  \|\bx-\tilde{\bx}_t\|_{\bS_{t-1}},
\end{align*}
where $\bA_t=\mu \bg_t \bg_t^\top$ and $\bS_t=\lambda \bI+\sum_{i=1}^t \bA_i$.
\end{exer}
\index{Online Newton Step algorithm|)}

\begin{exer}
Consider the online learning algorithm for linear regression with square loss that predicts in each round with
\[
\bx_t = \argmin_{\bx} \ \lambda \|\bx\|^2_2 + \sum_{i=1}^{t-1} (\langle \bz_i, \bx\rangle -y_i)^2~.
\]
Prove for it the equality
\[
\sum_{t=1}^T \frac{(\langle \bz_t,\bx_t\rangle-y_t)^2}{1+d_t}
= \min_{\bu} \ \lambda\|\bu\|^2_2 + \sum_{t=1}^{T} (\langle \bz_t, \bu\rangle -y_t)^2,
\]
where $d_t=\bz_t^\top ( \sum_{i=1}^{t-1} \bz_i \bz_i^\top + \lambda \bI)^{-1} \bz_t$.
Hint: see \citet{ZhdanovK10,ZhdanovK13}.
\end{exer}

\begin{exer}
\label{exercise:best_surrogate_hint}
Let $\bg, \tilde{\bg} \in \R^d$ and $\|\cdot\|_\star$ a norm.
Then, show that
\[
\min_{\tilde{\bg}^{(a)}} \ \frac{\|\bg -\tilde{\bg}^{(a)}\|^2_\star}{2} +\|\bg\|_\star \|\tilde{\bg} - \tilde{\bg}^{(a)}\|_\star
\leq \frac{1}{2} \left[\|\bg -\tilde{\bg}\|^2_\star - \max(\|\bg -\tilde{\bg}\|_\star-\|\bg\|_\star,0)^2\right]~.
\]
\end{exer}

\begin{exer}
In Section~\ref{sec:oftrl_variance}, we showed that it is possible to have a regret bound that depends on the sum of the empirical variances of the coordinates of the subgradients, using the average of the past subgradients as a prediction of the next losses. Prove that a similar bound is possible simply using the previous subgradients as predictions of the next losses, i.e., $\tilde{\bg}_{t}=\bg_{t-1}$.
\end{exer}
\index{Follow-the-Regularized-Leader algorithm|)textbf}

\acresetall

\chapter{Online Linear Classification}

In this chapter, we will consider the problem of \emph{online linear classification}.
We consider the following setting:
\begin{itemize}
\item At each time step, we receive a sample $\bz_t \in \R^d$
\item We output a prediction of the binary label $y_t\in \{-1,1\}$ of $\bz_t$
\item We receive the true label, and we see whether we made a mistake or not
\item We update our online classifier
\end{itemize}
Ideally, the aim of the online algorithm is to minimize the number of mistakes it makes compared to some best fixed classifier.

We will focus on linear classifiers\index{linear classifier}, which predict with the sign of the inner product between a vector $\bx_t$ and the input features $\bz_t$.
Hence, $\tilde{y}_t=\sign(\langle \bz_t,\bx_t\rangle)$. However, we will also allow some form of randomization.

This problem could be written as a regret minimization problem:
\[
\Regret_T(\bu)=\sum_{t=1}^T \ell_t(\bx_t) - \sum_{t=1}^T \ell_t(\bu),
\]
where $\ell_t(\bx) = \indevent\{\sign(\langle \bz_t,\bx\rangle) \neq y_t\}$, where $\sign(0):=0$. However, these losses are non-convex. Hence, we need an alternative way to deal with them.
In the following, we will see three possible approaches to this problem.

\acresetall

\section{Online Randomized Classifier}
As we did for the \ac{LEA} framework, we might think of convexifying the losses using randomization.
Hence, on each round we can predict a number $q_t \in [-1,1]$ and output the label $+1$ with probability $\frac{1+q_t}{2}$ and the label $-1$ with probability $\frac{1-q_t}{2}$.
So, define the random variable
\begin{equation}
\label{eq:random_label}
\tilde{Y}(q)
:= \begin{cases}
+1, & \text{ with probability } \frac{1+q}{2},\\
-1, & \text{ with probability } \frac{1-q}{2}~.
\end{cases}
\end{equation}
Now observe that $\E[\indevent\{\tilde{Y}(q_t)\neq y_t\}]=\frac{1}{2}|q_t-y_t|$.
If we consider linear predictors, we can set $q_t=\langle \bz_t, \bx_t\rangle$ and similarly for the competitor $q'_t=\langle \bz_t, \bu\rangle$. Assuming we can constrain both the algorithm and the competitor to the space of vectors where $|\langle \bz_t, \bx\rangle|\leq 1$ for $t=1, \dots,T$, we can write
\begin{align*}
&\E\left[\sum_{t=1}^T \indevent\left\{\tilde{Y}(\langle \bz_t,\bx_t\rangle) \neq y_t\right\} - \sum_{t=1}^T \indevent\left\{\tilde{Y}(\langle \bz_t,\bu\rangle) \neq y_t\right\}\right]\\
&\quad= \E\left[\sum_{t=1}^T \frac12 \left|\langle \bz_t,\bx_t\rangle - y_t\right| - \sum_{t=1}^T \frac12 \left|\langle \bz_t, \bu\rangle - y_t\right|\right]~.
\end{align*}
Hence, the surrogate convex loss becomes $\tilde{\ell}_t(\bx) = \frac{1}{2}|\langle \bz_t, \bx\rangle - y_t|$ and the feasible set is any convex set where we have the property $|\langle \bz_t,\bx\rangle|\leq 1$ for $t=1,\dots,T$.

Given that this problem is convex, assuming $\bz_t$ to be bounded with respect to some norm, we can use almost any of the algorithms we have seen till now, from \ac{OMD} to \ac{FTRL}. All of them would result in $\mathcal{O}(\sqrt{T})$ regret upper bounds.
The only caveat is to restrict $\langle \bz_t,\bx_t\rangle$ in $[-1,1]$. One way to ensure this is to assume $\|\bz_t\|_\star\leq R$ and choose the feasible set $\mathcal{V}=\{\bx \in\R^d : \|\bx\|\leq \frac{1}{R}\}$.


Putting everything together, we can have, for example, using \ac{FTRL}, Algorithm~\ref{alg:rolc}.

\begin{algorithm}[t]
\caption{Randomized Online Linear Classifier through \ac{FTRL}}
\label{alg:rolc}
\begin{algorithmic}[1]
{
    \REQUIRE{$R>0$ such that $\|\bz_t\|_2\leq R$ for $t=1, \dots,T$, $\lambda_t>0 $ for $t=1, \dots,T$}
    \STATE{Set $\btheta_1=\boldsymbol{0} \in \R^d$}
    \FOR{$t=1$ {\bfseries to} $T$}
    \STATE{$\bx_{t}= \frac{\frac{1}{\lambda_t}\btheta_{t}}{\max(R \frac{1}{\lambda_t} \|\btheta_t\|_2,1)}$}
    \STATE{Receive $\bz_t \in \mathcal{X}$}
    \STATE{Predict $\tilde{y}_t = \tilde{Y}(\langle \bz_t,\bx_t\rangle)$, where $\tilde{Y}$ is defined in \eqref{eq:random_label}}
    \STATE{Receive $y_t$ and pay the loss $\indevent\{y_t\neq \tilde{y}_t\}$}
    \STATE{$\btheta_{t+1}=\btheta_t-\frac{1}{2}\sign(\langle \bz_t,\bx_t \rangle- y_t )\bz_t$ where $\sign(0):=0$}
    \ENDFOR
}
\end{algorithmic}
\end{algorithm}

\begin{theorem}
Let $\mathcal{X}\subset \R^d$ and $(\bz_t, y_t)_{t=1}^T$ be an arbitrary sequence of sample-label pairs where $(\bz_t, y_t) \in \mathcal{X} \times \{-1,1\}$.
Assume $\|\bz_t\|_2\leq R$, for $t=1, \dots,T$. Then, running Algorithm~\ref{alg:rolc} with $\lambda_t=\frac{R^2 \sqrt{2 t}}{2}$, for any $\bu \in \{\bx \in \R^d : \|\bx\|_2\leq \frac{1}{R}\}$ we have the following guarantee:
\[
\E\left[\sum_{t=1}^T \indevent\{\tilde{Y}(\langle\bz_t,\bx_t\rangle) \neq y_t\}\right] - \E\left[\sum_{t=1}^T  \indevent\{\tilde{Y}(\langle\bz_t,\bu\rangle) \neq y_t\} \right]
\leq \frac{\sqrt{2 T}}{2}~.
\]
\end{theorem}
The proof follows from the \ac{FTRL} regret bound with the chosen increasing regularizer, and it is left as an exercise.

\section{The Perceptron Algorithm}

\index{Perceptron algorithm|(textbf}
The above strategy has the shortcoming of restricting the feasible vectors to a possibly very small set. In turn, this could make the best competitor's performance poor, and the performance of the online algorithm is only guaranteed to be close to that of the best competitor.

Another way to deal with the non-convexity is to compare the number of mistakes that the algorithm makes with a convex cumulative loss of the competitor. That is, we can try to prove a weaker regret guarantee, called \textbf{mistake bound}\index{mistake bound|textbf}:
\begin{equation}
\label{eq:weak_regret}
\sum_{t=1}^T \indevent{\{y_t \neq \tilde{y}_t\}} - \sum_{t=1}^T \ell(\langle \bz_t, \bu\rangle, y_t)
= \mathcal{O}\left(\sqrt{T}\right)~.
\end{equation}
In particular, the convex loss we consider is a \emph{power} of the \textbf{hinge loss}\index{hinge loss!powers}: $\ell^q(m, y):=\max(1-y m,0)^q$.
The hinge loss is a convex upper bound to the 0/1 loss, when the prediction $\tilde{y}=\sign(m)$, and it achieves the value of zero when the sign of the prediction is correct \emph{and} the magnitude of the inner product is sufficiently large.
Moreover, by taking powers, we obtain a family of functions that trades off the loss on misclassified samples against the loss on correctly classified samples whose prediction satisfies $|\langle \bz_t,\bx\rangle|\leq 1$, see Figure~\ref{fig:powers_hinge}.


\begin{figure}[t]
\centering
\begin{tikzpicture}
\begin{axis}[
    width=7cm,
    xmin=-2, xmax=2,
    ymin=0,  ymax=3.5,
    xlabel={$y m$},
    ylabel={$\ell^q(m,y)$},
    samples=800,
    domain=-2:2,
    axis lines=left,
    xtick={-2,-1,0,1,2},
    ytick={0,1,2,3},
]

\addplot[thick, black] { (x<=1) * (1-x)^(1) };
\addlegendentry{$q = 1$}

\addplot[thick, gray, dashed] { (x<=1) * (1-x)^(1.5) };
\addlegendentry{$q = 1.5$}

\addplot[thick, gray, dotted] { (x<=1) * (1-x)^(2) };
\addlegendentry{$q = 2$}

\addplot[thick, gray, dashdotted] { (x<=1) * (1-x)^(4) };
\addlegendentry{$q = 4$}

\addplot[thick, gray] { (x<=0) };
\addlegendentry{0/1 loss}
\end{axis}
\end{tikzpicture}
\caption{Powers of the hinge loss.}
\label{fig:powers_hinge}
\commentAlt{Figure~\ref{fig:powers_hinge}. Plot of powers of the hinge loss as functions of the signed margin. Curves for q=1, 1.5, 2, and 4 are shown, along with the step-shaped zero-one loss.}
\end{figure}

\begin{algorithm}[t]
\caption{Perceptron}
\label{alg:perc}
\begin{algorithmic}[1]
{
    \STATE{Set $\bx_1=\boldsymbol{0} \in \R^d$}
    \FOR{$t=1$ {\bfseries to} $T$}
    \STATE{Receive $\bz_t \in \mathcal{X}$}
    \STATE{Predict $\tilde{y}_t=\sign(\langle \bz_t,\bx_t\rangle)$ where $\sign(0):=0$}
    \STATE{Receive $y_t$ and pay the loss $\indevent{\{y_t \tilde{y}_t \leq 0\}}$}
    \STATE{$\bx_{t+1}=\bx_t + \indevent{\{y_t \tilde{y}_t \leq 0\}} y_t\bz_t$}
    \ENDFOR
}
\end{algorithmic}
\end{algorithm}

The oldest known algorithm to minimize the modified regret in \eqref{eq:weak_regret} is the \textbf{Perceptron} algorithm, in Algorithm~\ref{alg:perc}.
The Perceptron algorithm updates the current prediction $\bx_t$ by moving in the direction of the current sample multiplied by its label. Let's see why this is a good idea. Assume that $y_t=1$ and the algorithm made a mistake. Then, the updated prediction $\bx_{t+1}$ would predict a more positive number on the same sample $\bz_t$. In fact, assuming $\|\bz_t\|_2>0$, we have
\[
\langle \bz_t, \bx_{t+1}\rangle
= \langle \bz_t, \bx_t + y_t \bz_t\rangle
= \langle \bz_t, \bx_t\rangle + \|\bz_t\|^2_2
> \langle \bz_t, \bx_t\rangle~.
\]
In the same way, if $y_t=-1$ and the algorithm made a mistake, the update would result in a more negative prediction on the same sample.

We can formalize the above intuition in the following guarantee.
\begin{theorem}
\label{thm:perc}
Let $\mathcal{X}\subset \R^d$ and $(\bz_t, y_t)_{t=1}^T$ be an arbitrary sequence of sample-label pairs where $(\bz_t, y_t) \in \mathcal{X} \times \{-1,1\}$.
Assume $\|\bz_t\|_2\leq R$, for $t=1, \dots,T$. Then, running the Perceptron algorithm, for all $\bu \in \R^d$ and $q\geq 1$, we have the following guarantee
\begin{align*}
&\sum_{t=1}^T \indevent{\{y_t \tilde{y}_t \leq 0\}} - \sum_{t=1}^T \ell^q(\langle \bz_t, \bu\rangle, y_t)\\
&\quad \leq \frac{q^2 R^2 \|\bu\|^2_2}{2} + q R \|\bu\|_2\sqrt{\frac{q^2 R^2 \|\bu\|^2_2}{4} + \sum_{t=1}^T \ell^q(\langle \bz_t, \bu\rangle, y_t)}~.
\end{align*}
\end{theorem}

Before proving the theorem, let's take a look at its meaning.
If there exists a $\bu \in \R^d$ such that $\sum_{t=1}^T \ell^q(\langle \bz_t, \bu\rangle, y_t)=0$, then the Perceptron algorithm makes a \emph{finite} number of mistakes upper bounded by $R^2 \|\bu\|^2_2$, by choosing $q=1$. If there are many vectors $\bu$ that achieve $\sum_{t=1}^T \ell^q(\langle \bz_t, \bu\rangle, y_t)=0$, we have that the finite number of mistakes is bounded by the norm of the smallest $\bu$ among them.
What is the meaning of this quantity?

Remember that a hyperplane represented by its normal vector $\bu$ divides the space into two half-spaces: one with the points $\bz$ that give a positive value for the inner product $\langle \bz, \bu\rangle$ and the other one where the same inner product is negative.
Now, we have that the distance of a sample $\bz_t$ from the hyperplane whose normal is $\bu$ is
\[
\frac{|\langle \bz_t, \bu\rangle|}{\|\bu\|_2}~.
\]
Also, given that we are considering a $\bu$ that gives cumulative hinge loss zero, we have that this quantity is at least $\frac{1}{\|\bu\|_2}$.
So, \emph{the norm of the minimal $\bu$ that has cumulative hinge loss equal to zero is inversely proportional to the minimum distance between the points and the separating hyperplane}. This distance is called the \textbf{margin}\index{margin!in linear classification|textbf} of the samples $(\bz_t,y_t)_{t=1}^T$.
So, if the margin is small, the Perceptron algorithm can make more mistakes than when the margin is large.

If the problem is not linearly separable, the Perceptron satisfies a mistake bound of $\mathscr{L}^{(q)}+\mathcal{O}(q\sqrt{\mathscr{L}^{(q)}})$\index{mistake bound} when $T\to \infty$, uniformly for all $q\geq 1$, where $\mathscr{L}^{(q)}$ is the cumulative loss of the competitor, $\sum_{t=1}^T\ell^q(\langle \bz_t,\bu\rangle,y_t)$. Moreover, we measure the competitor with a \emph{family of loss functions} and compete with the best $\bu$ measured with the best loss. This adaptivity is achieved through two basic ingredients:
\begin{itemize}
\item \emph{The Perceptron is independent of scaling of the update by a hypothetical learning rate $\eta$}, in the sense that the mistakes it makes are independent of the scaling. That is, we could update with $\bx_{t+1}=\bx_t + \eta \indevent{\{y_t \tilde{y}_t \leq 0\}} y_t\bz_t$ and have the same mistakes and update on the same samples because they only depend on the sign of $\tilde{y}_t$. Hence, we can think of it as always using the best possible learning rate $\eta$.
\item The weakened definition of regret allows us to consider a family of loss functions, because \emph{the Perceptron is not using any of them in the update.} In this sense, it is worth stressing that the Perceptron algorithm is \emph{not} simply \ac{OSD} on some sequence of loss functions.
\end{itemize}

Let's now prove the regret guarantee.
%
%
\begin{proof}[Proof of Theorem~\ref{thm:perc}]
Denote by $\tau_t:=\indevent{\{y_t \tilde{y}_t \leq 0\}}$, so the total number of mistakes of the Perceptron algorithm is $M:=\sum_{t=1}^T \tau_t$.

First, note that the Perceptron algorithm can be thought of as running \ac{OSD} with a fixed stepsize $\eta>0$ over the surrogate losses $\tilde{\ell}_t(\bx):=-\langle \tau_t y_t \bz_t, \bx\rangle$ over $\mathcal{V}=\R^d$.
Indeed, \ac{OSD} over such losses would update
\begin{equation}
\label{eq:proof_perc_eq1}
\bx_{t+1}=\bx_t + \eta \tau_t y_t\bz_t~.
\end{equation}
Now, as said above, $\eta$ does not affect in any way the sign of the predictions, hence the Perceptron algorithm could be run with \eqref{eq:proof_perc_eq1} and its predictions would be exactly the same. Hence, for all $\eta>0$, we have
\[
\sum_{t=1}^T -\langle \tau_t y_t \bz_t, \bx_t\rangle + \sum_{t=1}^T \langle \tau_t y_t \bz_t, \bu\rangle\\
\leq \frac{\|\bu\|^2_2}{2\eta} + \frac{\eta}{2} \sum_{t=1}^T \tau_t \|\bz_t\|^2_2~.
\]
Given that this inequality holds for any $\eta$, we can choose the one that minimizes the r.h.s., to have
\begin{align}
\sum_{t=1}^T ( -\langle \tau_t y_t \bz_t, \bx_t\rangle +  \langle \tau_t y_t \bz_t, \bu\rangle)
&\leq \|\bu\|_2 \sqrt{\sum_{t=1}^T \tau_t \|\bz_t\|^2_2}
\leq \|\bu\|_2 R\sqrt{\sum_{t=1}^T \tau_t} \nonumber \\
&\leq \|\bu\|_2 R\sqrt{M}~. \label{eq:proof_perc_eq2}
\end{align}
Note that $-\tau_t y_t\langle \bz_t, \bx_t\rangle\geq 0$.
Also, we have
\[
\langle y_t \bz_t, \bu\rangle
=1-(1-\langle y_t \bz_t, \bu\rangle)
\geq 1 - \max(1-\langle y_t \bz_t, \bu\rangle,0)
= 1 - \ell^1(\langle \bz_t,\bu\rangle,y_t)~.
\]
Denote by $\mathscr{L}^{(q)}:=\sum_{t=1}^T\ell^q(\langle \bz_t,\bu\rangle,y_t)$, where for shortness of notation we hide the dependency on $\bu$.
We now rewrite \eqref{eq:proof_perc_eq2} as
\begin{align}
M
&= \sum_{t=1}^T \tau_t
\leq \|\bu\|_2 R\sqrt{M } + \sum_{t=1}^T \tau_t \ell^1(\langle \bz_t,\bu\rangle,y_t) \label{eq:proof_perc_eq3}\\
&\leq \|\bu\|_2 R\sqrt{M } + \mathscr{L}^{(1)} \nonumber~.
\end{align}
Using Lemma~\ref{lemma:sqrt}, we have the stated bound for $q=1$.

For $q>1$, using Holder's inequality and $\frac{1}{p}+\frac{1}{q}=1$ in \eqref{eq:proof_perc_eq3}, we have
\begin{align*}
M
&\leq \|\bu\|_2 R\sqrt{M } + \left(\sum_{t=1}^T \tau_t^p\right)^{1/p} \left(\sum_{t=1}^T\ell^q(\langle \bz_t,\bu\rangle,y_t)\right)^{1/q} \\
&= \|\bu\|_2 R\sqrt{M} + M^{1/p} (\mathscr{L}^{(q)})^{1/q}~.
\end{align*}
Using Young's inequality\index{inequality!Young's} (Example~\ref{example:young}), we have
\begin{align*}
M
&\leq \|\bu\|_2 R\sqrt{M} + M^{1/p} (\mathscr{L}^{(q)})^{1/q}
\leq \|\bu\|_2 R\sqrt{M} + \frac{1}{p} M + \frac{1}{q} \mathscr{L}^{(q)},
\end{align*}
that implies
\[
M\left(1-\frac{1}{p} \right)\leq \|\bu\|_2 R\sqrt{M} + \frac{1}{q} \mathscr{L}^{(q)}~.
\]
Using the fact that $1-\frac{1}{p}=\frac{1}{q}$, we have
\[
M \leq q \|\bu\|_2 R\sqrt{M} + \mathscr{L}^{(q)}~.
\]
Finally, using Lemma~\ref{lemma:sqrt}, we have the stated bound.
\end{proof}

\begin{remark}
The Perceptron algorithm can also be seen as a way to solve the \emph{feasibility problem}\index{feasibility problem}.
In particular, consider the problem
\begin{align*}
\text{Find} & \quad \bx\\
\text{s.t.} & \quad \langle \bq_i,\bx\rangle > 0, \ i=1, \dots, N
\end{align*}
where $\bq_i \in \R^d$.
Essentially, we have a set $\mathcal{V}$ defined by the intersection of the open half-spaces $\langle \bq_i,\bx\rangle > 0$ and we want to find any point in it. This is useful, for example, when we have a constrained convex optimization problem with feasible set $\mathcal{V}$ and we need a way to initialize our optimization algorithm.

To use the Perceptron to solve this problem, in each round $t$, we can find one constraint $\bq_i$ that is violated and update
\[
\bx_{t+1} = \bx_t + \bq_i~.
\]
The algorithm will find a point in $\mathcal{V}$ after $R^2 \|\bu\|^2_2$ iterations, where $\|\bq_i\|_2\leq R$ and $\bu$ is given by the following optimization problem
\begin{align*}
\bu
= &\argmin_{\bx \in \R^d}  \quad \|\bx\|^2_2\\
&\text{s.t.}  \quad \langle \bq_i,\bx\rangle \geq 1, \ i=1, \dots, N~.
\end{align*}
Given that each iteration has computational complexity proportional to $d N$, the total computational complexity is $\mathcal{O}(d N R^2 \|\bu\|^2_2)$.
\end{remark}
\index{Perceptron algorithm|)textbf}

\section{The Gaptron Algorithm}

\index{Gaptron algorithm|(textbf}
In this section, we introduce an algorithm that merges ideas from the randomized strategy and the Perceptron algorithm: the \textbf{Gaptron} algorithm.

The Perceptron analysis bounds the regret with respect to surrogate losses. However, this can be a loose upper bound on the actual number of mistakes. The Gaptron algorithm is designed to tighten this analysis by actively managing the \emph{surrogate gap}\index{surrogate gap}, which is the difference between the true zero-one loss and the convex surrogate loss function used for optimization, leading to better mistake bounds.
For pedagogical reasons, in the following, we will focus on the binary version of the Gaptron algorithm.

The algorithm maintains a weight vector $\bx_t$ and, at each round, computes a measure of confidence using the features $\bz_t$ as $m_t=\langle \bz_t, \bx_t\rangle$. Based on this confidence measure and the loss function used, it decides whether to follow the prediction of the current weight vector or to explore by choosing a label uniformly at random. This exploration is controlled by a \textbf{gap map}\index{gap map|textbf}, $a_t$, which allows us to exploit the surrogate gap. Intuitively, the algorithm will randomize its prediction each time it is unsure, and in this way, the expected number of mistakes approaches the value of the surrogate loss. Finally, the weight vector is updated using \ac{OSD} on the surrogate loss.
Algorithm~\ref{alg:gaptron} summarizes the entire procedure.

\begin{algorithm}[t]
\caption{Gaptron for Binary Classification}
\label{alg:gaptron}
\begin{algorithmic}[1]
{
    \REQUIRE{Initial point $\bx_1$, learning rate $\eta > 0$, loss $\ell:\R\times\{-1,1\}\to \R_{\geq 0}$, gap map $a: \R \times \{-1,1\} \to [0,1]$}
    \FOR{$t=1$ {\bfseries to} $T$}
    \STATE{Receive feature vector $\bz_t \in \R^d$}
    \STATE{Set $m_t=\langle \bz_t, \bx_t\rangle$}
    \STATE{Set $\hat{y}_t = \begin{cases} 1, &\text{ if } m_t>0\\ \text{any label}, & \text{ if } m_t=0\\ -1, & \text{ if } m_t<0 \end{cases}$}
    \STATE{Set $a_t = a(m_t,\hat{y}_t) \in [0,1]$}
    \STATE{Predict $Y'_t = \hat{y}_t$ with probability $1-a_t$, and a uniform random label from $\{-1, 1\}$ with probability $a_t$}
    \STATE{Receive true label $y_t \in \{-1, 1\}$ and suffer loss $\indevent{\{Y'_t \neq y_t\}}$}
    \STATE{Compute subgradient $\bg_t \in \bz_t \cdot \partial (\ell(\cdot,y_t))(m_t)$}
    \STATE{Update $\bx_{t+1} = \bx_t - \eta \bg_t$}
    \ENDFOR
}
\end{algorithmic}
\end{algorithm}

We now present a theorem on the expected number of mistakes\index{mistake bound} for the Gaptron algorithm when using self-bounded losses.
\begin{theorem}
\label{thm:gaptron_mistakes}
Let $\ell_t(\bx) = \ell(\langle \bz_t, \bx \rangle, y_t)$ where $\ell:\R\times \{-1,1 \}\to \R_{\geq0}$. Assume that
\begin{itemize}
\item For each $y\in\{-1,1\}$, the one-dimensional function $x\mapsto \ell(x,y)$ is $s$-self-bounded\index{function!self-bounded};
\item $\ell(x, 1)+\ell(x, -1)\geq 2$ for all $ x\in \R$;
\item $\ell(x, \sign(x)) \in [0,1]$ for $x\neq 0$;
\item $\ell(0, 1) \in [0,1]$ and $\ell(0, -1) \in [0,1]$.
\end{itemize}
Let $\|\bz_t\|_2 \leq R$ for all $t$. Set the learning rate $\eta \leq \frac{1}{2 s R^2}$ and the gap map $a(m,\hat{y})=\ell(m,\hat{y})$.
Then, the expected number of mistakes of the Gaptron algorithm (Algorithm~\ref{alg:gaptron}) is upper bounded as
\[
\E\left[\sum_{t=1}^T \indevent{\{Y'_t \neq y_t\}}\right]
\leq \sum_{t=1}^T \ell_t(\bu) + \frac{\|\bu-\bx_1\|_2^2}{2\eta}, \quad \forall \bu \in \R^d~.
\]
\end{theorem}
\begin{proof}
Using the fact that self-boundedness\index{function!self-bounded} implies convexity, the analysis starts from the standard regret bound for \ac{OSD}, which for any $\bu$ gives
\[
\sum_{t=1}^T \ell_t(\bx_t) - \sum_{t=1}^T \ell_t(\bu)
\leq \frac{\|\bu-\bx_1\|_2^2}{2\eta} + \frac{\eta}{2}\sum_{t=1}^T \|\bg_t\|_2^2~.
\]
The expected number of mistakes in round $t$, conditioned on the history up to round $t-1$, is
\[
\E_t[\indevent{\{Y'_t \neq y_t\}}]
= (1-a_t)\indevent{\{\hat{y}_t \neq y_t\}} + \frac{a_t}{2}~.
\]
Hence, we have
\begin{align*}
&\sum_{t=1}^T \E\left[\indevent{\{Y'_t \neq y_t\}}\right] -\sum_{t=1}^T \ell_t(\bu)\\
&\quad= \sum_{t=1}^T \E\left[\indevent{\{Y'_t \neq y_t\}}\right] - \sum_{t=1}^T \E[\ell_t(\bx_t)] + \sum_{t=1}^T \E[\ell_t(\bx_t)]-\sum_{t=1}^T \ell_t(\bu)\\
&\quad\leq \frac{\|\bu-\bx_1\|_2^2}{2\eta} + \sum_{t=1}^T \E\left[\indevent{\{Y'_t \neq y_t\}}-\ell_t(\bx_t) +\frac{\eta}{2} \|\bg_t\|_2^2\right]\\
&\quad= \frac{\|\bu-\bx_1\|_2^2}{2\eta} + \sum_{t=1}^T \E\left[(1-a_t)\indevent{\{\hat{y}_t \neq y_t\}} + \frac{a_t}{2}-\ell_t(\bx_t) +\frac{\eta}{2} \|\bg_t\|_2^2\right]~.
\end{align*}

Let's analyze the term we call the \textbf{surrogate gap}\index{surrogate gap|textbf}:
\[
\gamma_t := (1-a_t)\indevent{\{\hat{y}_t \neq y_t\}} + \frac{a_t}{2} - \ell_t(\bx_t) + \frac{\eta}{2}\|\bg_t\|_2^2~.
\]
The theorem is proven if we can show that $\gamma_t \leq 0$ for our choice of $a_t$ and $\eta$.

The bound on the norm of the subgradient can be calculated through the self-boundedness\index{function!self-bounded} of $\ell$ in its first argument:
\[
\|\bg_t\|^2_2
\leq 2 s \|\bz_t\|^2_2 \ell(m_t, y_t)
\leq 2 s R^2 \ell(m_t, y_t)~.
\]

Hence, we have two cases.

\textbf{Case $ y_t \neq \hat{y}_t$:}
\begin{align*}
\gamma_t
&=(1-a_t)\indevent{\{\hat{y}_t \neq y_t\}} + \frac{a_t}{2} - \ell_t(\bx_t) + \frac{\eta}{2}\|\bg_t\|_2^2\\
&\leq 1-\ell(m_t, \hat{y}_t) + \frac12 \ell(m_t, \hat{y}_t) - \ell(m_t, y_t) + \eta s R^2 \ell(m_t,y_t)\\
&= 1-\frac12 \ell(m_t, \hat{y}_t) - \frac12 \ell(m_t, y_t) - \frac12 \ell(m_t, y_t) + \eta s R^2 \ell(m_t, y_t)\\
&\leq - \frac12 \ell(m_t, y_t) + \eta s R^2 \ell(m_t, y_t),
\end{align*}
where in the second inequality we used the assumption on the loss.
Hence, we need $\eta\leq \frac{1}{2 s R^2}$ to have $\gamma_t \leq 0$.

\textbf{Case $ y_t = \hat{y}_t$:}
\begin{align*}
\gamma_t
&=(1-a_t)\indevent{\{\hat{y}_t \neq y_t\}} + \frac{a_t}{2} - \ell_t(\bx_t) + \frac{\eta}{2}\|\bg_t\|_2^2\\
&\leq \frac12 \ell(m_t, \hat{y}_t) - \ell(m_t, y_t) + \eta s R^2 \ell(m_t, y_t)
= -\frac12 \ell(m_t, y_t) + \eta s R^2 \ell(m_t, y_t)~.
\end{align*}
Hence, again we need $\eta\leq \frac{1}{2 s R^2}$ to have $\gamma_t \leq 0$.
\end{proof}

Surprisingly, the use of the randomization allows us to obtain a \emph{constant} additive term over the comparator's cumulative surrogate loss.

\begin{remark}
Observe that if $\ell(x,y)=f(x y)$ for some convex function $f:\R \to \R$, then $f(0)\geq 1$ implies the condition $\ell(x, 1)+\ell(x, -1)\geq 2$ for all $ x\in \R$ by Jensen's inequality\index{inequality!Jensen's} (Theorem~\ref{thm:jensen}).
\end{remark}

We now show some examples of binary classification losses that satisfy the assumptions of the theorem.

\begin{example}
We will use the squared hinge loss\index{hinge loss!squared}, $\ell_2(\tilde{y},y)=\max(1 - y \tilde{y},0)^2$, that satisfies all the assumptions in Theorem~\ref{thm:gaptron_mistakes} and it is $2$-self-bounded\index{function!self-bounded} in its first argument. Hence, we have that the Gaptron with the squared hinge loss and $\eta=\frac{1}{4 R^2}$ satisfies
\[
\E\left[\sum_{t=1}^T \indevent{\{Y'_t \neq y_t\}}\right]
\leq \sum_{t=1}^T \ell_2(\langle \bz_t,\bu\rangle, y_t) + 2\|\bu-\bx_1\|_2^2 R^2, \quad \forall \bu \in \R^d~.
\]
So, if the problem is not linearly separable, we can greatly beat (in expectation) the bound on the number of mistakes we gave for the Perceptron in terms of the squared hinge loss.

An even better choice is the \emph{smoothed hinge loss}\index{hinge loss!smoothed}:
\begin{equation}
\label{eq:smooth_hinge}
\ell^\text{sh}(\tilde{y},y)
:=
\begin{cases}
1 - 2 y \tilde{y} , & \text{ if }  y \tilde{y} < 0\\
(1 - y \tilde{y})^2, & \text{ if } 0 \leq y \tilde{y} \leq 1\\
0, & \text{ if } y \tilde{y} > 1
\end{cases}
\end{equation}
This loss is always less than or equal to the squared hinge loss; it still satisfies all the assumptions of Theorem~\ref{thm:gaptron_mistakes}, and it is $2$-self-bounded\index{function!self-bounded} in its first argument.
\end{example}

\begin{example}
One can also use loss functions that give different weights to the errors of the two classes. For example, we can use the following function
\[
\ell(\tilde{y},y)
=
\begin{cases}
1 - 2 w_y y \tilde{y} , & \text{ if }  y \tilde{y} < 0\\
(1 - y \tilde{y})^2, & \text{ if } 0 \leq y \tilde{y} \leq 1\\
0, & \text{ if } y \tilde{y} > 1
\end{cases}
\]
where $w_1=2$ and $w_{-1}=1$. In this case, the loss function is $8$-self-bounded\index{function!self-bounded}, and it still satisfies the assumptions of Theorem~\ref{thm:gaptron_mistakes}.
\end{example}
\index{Gaptron algorithm|)textbf}

\section{History Bits}

The Perceptron\index{Perceptron algorithm|(} was proposed by \citet{Rosenblatt58}. To be more precise, he introduced a \emph{family} of algorithms characterized by a certain architecture. Also, he considered what we call now supervised and unsupervised training procedures.
The particular class of Perceptron that we use nowadays, and that I described, was called \emph{$\alpha$-system}~\citep{Block62}. I hypothesize that the fact that the $\alpha$-system survived the test of time is exactly due to the simple convergence proof in \citet{Block62} and \citet{Novikoff63}. Both proofs are non-stochastic. For the sake of proper credit assignment, it seems that the convergence proof of the Perceptron was proved by many others before Block and Novikoff \citep[see references in][]{Novikoff63}. However, the proof in \citet{Novikoff63} seems to be the cleanest one. \citet{AizermanBR64} (essentially) described for the first time the Kernel Perceptron\index{Perceptron algorithm!with kernels} and proved a finite mistake bound for it.
The proof of convergence in the non-separable case for $q=1$ is by \citet{GentileL99,Gentile03} and for $q=2$ is from \citet{FreundS98,FreundS99}.
Theorem~\ref{thm:perc} is from \citet{BeygelzimerOZ17}, but with a new proof.
\index{Perceptron algorithm|)}

Let us also briefly mention the general combinatorial quantity that characterizes online binary classification in the realizable case.
The \emph{Littlestone dimension}\index{Littlestone dimension|textbf} of a class of classifiers $\mathcal{H}\subseteq \{-1,1\}^{\mathcal{X}}$ is the largest depth of a complete binary tree, whose internal nodes are labeled by points in $\mathcal{X}$, such that every path of labels in $\{-1,1\}$ is realized by some classifier in $\mathcal{H}$. This quantity was introduced by \citet{Littlestone88} and plays, in online learning, a role analogous to the one played by the \ac{VC} dimension\index{VC dimension} in statistical learning: in the realizable case, a class admits an online algorithm with a finite worst-case number of mistakes if and only if its Littlestone dimension is finite, and the optimal mistake bound is exactly its Littlestone dimension.

\index{Gaptron algorithm|(}
The Gaptron was introduced in \citet{vanderHoeven20}, in turn based on some of the ideas in \citet{NeuZ20}. \citet{vanderHoeven20} also described a multiclass variant for different surrogate losses, as well as a bandit variant. The proof here and the conditions in Theorem~\ref{thm:gaptron_mistakes} were developed in collaboration with Dirk van der Hoeven, and he was so kind to allow me to reproduce them here.
Recently, \citet{SakaueBTO24} extended the Gaptron algorithm to the structured prediction case, while \citet{SakaueBC25} extended it to the dynamic setting.
\index{Gaptron algorithm|)}

\section{Exercises}

\begin{exer}
In the Perceptron\index{Perceptron algorithm} mistake upper bound we let the competitor have any norm and we measure the loss of the competitor with the hinge loss\index{hinge loss} $\ell^{\text{hinge}}(\hat y, y)= \max(1-\hat{y} y,0)$, where $\hat{y} \in \R$ and $y \in \{-1,1\}$. Instead, we can equivalently constrain the competitor to have norm equal to 1 and measure the loss with the \emph{hinge loss at margin $\gamma$}\index{hinge loss!at margin $\gamma$}: $\ell^{\text{hinge},\gamma}(\hat y, y) = \frac{1}{\gamma}\max(\gamma-\hat y y,0)$, where $\gamma>0$.
Prove that $\ell^{\text{hinge},\gamma}(\hat y, y)$ is still an upper bound on the zero-one loss over the prediction given by $\sign(\hat{y})$.
\end{exer}

\begin{exer}
Use \ac{FTRL} with an increasing regularizer instead of \ac{OSD} in the Gaptron\index{Gaptron algorithm} to get rid of the need to know $R$ to set the learning rate, while still allowing unbounded feasible sets. Note that we make use of the knowledge of $\bz_t$ to set the regularizer at time $t$.
\end{exer}

\acresetall

\chapter{Multi-Armed Bandit}
\label{ch:bandits}

\index{multi-armed bandit|(}
The multi-armed bandit setting is similar to the \ac{LEA} setting: in each round, we select one expert $A_t\in \{1, \dots, d\}$, where $d\geq 2$, but, differently from the full-information setting, we only observe the loss $g_{t,A_t}$ of that expert. The aim is still to compete with the cumulative loss of the best expert in hindsight. The observed losses can be adversarial or stochastic, giving rise to adversarial and stochastic multi-armed bandits.

\acresetall

\index{multi-armed bandit!adversarial|(}
\section{Adversarial Multi-Armed Bandit}

Problems where we do not receive the full-information feedback, i.e., we do not observe the loss vector, are called \textbf{bandit problems}. The name comes from the problem of a gambler who plays a pool of slot machines, which are often called ``one-armed bandits''. On each round, the gambler places his bet on a slot machine, and his goal is to win almost as much money as if he had known in advance which slot machine would return the maximum total reward.
In this problem, we clearly have an \emph{exploration-exploitation trade-off}\index{exploration-exploitation trade-off}. In fact, on the one hand, we would like to play the slot machine, which, based on previous rounds, we believe will give us the biggest win. On the other hand, we have to explore the slot machines to find the best ones. On each round, we have to solve this trade-off.

Given that we do not observe the whole loss vector, we cannot directly use \ac{OMD} and \ac{FTRL} to solve such a problem, because they both need the loss functions or at least lower bounds to them.
Moreover, the adversarial multi-armed bandit problem is a harder problem than \ac{LEA}, due to the lack of full information. So, as in the \ac{LEA} case (Section~\ref{sec:lea}), randomization will be necessary to have a sublinear regret, but we will also need ways to solve the exploration-exploitation trade-off.

Before showing how to solve this problem, we will weaken the adversary a bit by assuming that it is \textbf{oblivious}\index{oblivious adversary|textbf}. That is, with the knowledge of the online algorithm being used, before the game starts, the adversary decides the losses of all the rounds. This makes the losses deterministic quantities, and it avoids the inadequacy in our definition of regret when the adversary is adaptive to the past algorithm's plays (see \citet{AroraDT12}).

One way to solve the lack of full-information feedback is to construct \emph{stochastic estimates} of the unknown losses. This is a natural choice given that the prediction strategy has to be randomized. So, in each round $t$, we construct a probability distribution over the arms $\bx_t$, and we sample one action $A_t$ according to this probability distribution. Then, we only observe the coordinate $A_t$ of the loss vector $\bg_t \in \R^d$.
One possibility to have a stochastic estimate of the losses is to use an \emph{importance-weighted estimator}\index{importance-weighted estimator}: construct the estimator $\tilde{\bg}_t$ of the unknown vector $\bg_t$ in the following way:
\[
\tilde{g}_{t,i}=\begin{cases} \frac{g_{t,i}}{x_{t,i}}, & i=A_t\\ 0, & \text{ otherwise}\end{cases} \text{ for } i=1,\dots,d~.
\]
Note that this estimator has all the coordinates equal to 0, except for the coordinate corresponding to the arm that was pulled.

This estimator is \emph{unbiased}, that is $\E[\tilde{\bg}_t|A_1, \dots, A_{t-1}]=\bg_t$. To see why, note that we can rewrite $\tilde{g}_{t,i}$ as $\indevent{\{A_t=i\}}\frac{g_{t,i}}{x_{t,i}}$ for $i=1, \dots, d$ and $\E[\indevent{\{A_t=i\}}|A_1, \dots, A_{t-1}]=x_{t,i}$. Hence, we have
\begin{align*}
\E[\tilde{g}_{t,i}|A_1, \dots, A_{t-1}]
&= \E\left[\indevent{\{A_t=i\}}\frac{g_{t,i}}{x_{t,i}}\middle|A_1, \dots, A_{t-1}\right]\\
&= \frac{g_{t,i}}{x_{t,i}} \E[\indevent{\{A_t=i\}}|A_1, \dots, A_{t-1}]
= g_{t,i}, \quad \forall i=1, \dots, d~.
\end{align*}
Let's also calculate the raw second moment of the coordinates of this estimator.
We have
\[
\E[\tilde{g}_{t,i}^2|A_1, \dots, A_{t-1}]
= \E\left[\indevent{\{A_t=i\}}\frac{g^2_{t,i}}{x^2_{t,i}}\middle|A_1, \dots, A_{t-1}\right]
= \frac{g_{t,i}^2}{x_{t,i}}~.
\]

We can now think of using \ac{OMD} with these estimated losses and an entropic regularizer. Of course, a similar reasoning can also be done with \ac{FTRL}. Hence, assume $\|\bg_t\|_\infty\leq L_\infty$ and set $\psi:\R^d_{\geq0} \to \R$ defined as $\psi(\bx)=\sum_{i=1}^d x_i \ln x_i$, that is the unnormalized negative entropy\index{entropy!unnormalized negative}. Also, set $\bx_1=[1/d, \dots, 1/d]$.
Using the \ac{OMD} analysis, we have
\[
\sum_{t=1}^T \langle \tilde{\bg}_t,\bx_t-\bu\rangle
\leq \frac{\ln d}{\eta} + \frac{\eta}{2}\sum_{t=1}^T \|\tilde{\bg}_t\|_\infty^2~.
\]
We can now take the expectation on both sides and use the fact that
\[
\E[\|\tilde{\bg}_t\|_\infty^2]
= \E\left[\frac{g^2_{t,A_t}}{x^2_{t,A_t}}\right]
= \E\left[\E\left[\frac{g^2_{t,A_t}}{x^2_{t,A_t}}\middle|A_1, \dots, A_{t-1}\right]\right]
= \E\left[\sum_{i=1}^d \frac{g^2_{t,i}}{x_{t,i}}\right],
\]
to get
\begin{align}
\E\left[\sum_{t=1}^T g_{t,A_t}\right] - \sum_{t=1}^T \langle \bg_t,\bu\rangle
&=\E\left[\sum_{t=1}^T \langle \bg_t,\bx_t\rangle\right] - \sum_{t=1}^T \langle \bg_t,\bu\rangle
=\E\left[\sum_{t=1}^T \langle \tilde{\bg}_t,\bx_t-\bu\rangle\right] \nonumber \\
&\leq \frac{\ln d}{\eta} + \frac{\eta}{2}\sum_{t=1}^T \E[\|\tilde{\bg}_t\|_\infty^2] \nonumber\\
&\leq \frac{\ln d}{\eta} + \frac{\eta}{2}\sum_{t=1}^T \sum_{i=1}^d \E\left[\frac{g_{t,i}^2}{x_{t,i}}\right], \quad \forall \bu \in \Delta^{d-1},\label{eq:exp3_no_expl}
\end{align}
where in the last inequality we upper bounded the max over the squared coordinates with their sum.
We are now in trouble, because the terms in the sum scale as $\max_{i=1,\dots,d} \frac{1}{x_{t,i}}$. So, we need a way to control the smallest probability over the arms.

One way to do it is to take a convex combination of $\bx_t$ and a uniform probability.
That is, we can predict with $\tilde{\bx}_t=(1-\alpha) \bx_t + \alpha [1/d, \dots, 1/d]^\top$, where $\alpha$ will be chosen in the following. So, $\alpha$ can be seen as the minimum amount of exploration we require from the algorithm. Its value will be chosen by the regret analysis to optimally trade off exploration vs exploitation.
The resulting algorithm is in Algorithm~\ref{alg:eg_bandit}.

\begin{algorithm}[t]
\caption{Exponential Weights with Explicit Exploration for Multi-Armed Bandit}
\label{alg:eg_bandit}
\begin{algorithmic}[1]
{
    \REQUIRE{$\eta,\alpha>0$}
    \STATE{Set $\bx_1 =[1/d, \dots, 1/d]$}
    \FOR{$t=1$ {\bfseries to} $T$}
    \STATE{Set $\tilde{\bx}_t=(1-\alpha)\bx_t + \alpha [1/d, \dots, 1/d]^\top$}
    \STATE{Draw $A_t$ according to $\Pr\{A_t=i\} = \tilde{x}_{t,i}$}
    \STATE{Select arm $A_t$}
    \STATE{Observe \emph{only} the loss of the selected arm $g_{t,A_t} \in [-L_\infty, L_\infty]$ and pay it}
    \STATE{Construct the estimate $\tilde{g}_{t,i}=\begin{cases} \frac{g_{t,i}}{\tilde{x}_{t,i}}, & i=A_t\\ 0, & \text{otherwise}\end{cases}$ for  $i=1,\dots,d$}
    \STATE{$x_{t+1,i} \propto x_{t,i} \exp(-\eta \tilde{g}_{t,i}), \ i=1,\dots,d$}
    \ENDFOR
}
\end{algorithmic}
\end{algorithm}

The same probability distribution is used in the estimator:
\begin{equation}
\label{eq:tilde_bg_t}
\tilde{g}_{t,i}=\begin{cases} \frac{g_{t,i}}{\tilde{x}_{t,i}}, & i=A_t\\ 0, & \text{ otherwise}\end{cases} \text{ for } i=1,\dots,d~.
\end{equation}
So, we now have that $\frac{1}{\tilde{x}_{t,i}}\leq \frac{d}{\alpha}$.
The variance term in the regret is now controllable:
\[
\|\tilde{\bg}_t\|^2_\infty
= \frac{g^2_{t,A_t}}{\tilde{x}^2_{t,A_t}}
\leq \frac{d^2}{\alpha^2} g^2_{t,A_t}
\leq \frac{d^2}{\alpha^2} L^2_\infty~.
\]

However, we pay a price for the bias introduced:
\[
\sum_{t=1}^T \langle \tilde{\bg}_t, \tilde{\bx}_t-\bu\rangle =
(1-\alpha) \sum_{t=1}^T \langle \tilde{\bg}_t, \bx_t - \bu\rangle + \frac{\alpha}{d} \sum_{t=1}^T \sum_{i=1}^d \tilde{g}_{t,i} - \alpha\sum_{t=1}^T \langle \tilde{\bg}_t,\bu\rangle, \quad \forall \bu \in \Delta^{d-1}~.
\]
From $\E[\sum_{i=1}^d \tilde{g}_{t,i}]=\sum_{i=1}^d g_{t,i}\leq d L_\infty$ and $\E[-\langle\tilde{\bg}_{t},\bu\rangle] = -\langle \bg_t,\bu\rangle\leq L_\infty$, we have
\begin{align*}
\E\left[\sum_{t=1}^T \langle \tilde{\bg}_t, \tilde{\bx}_t-\bu\rangle \right]
&\leq (1-\alpha)\E\left[\sum_{t=1}^T \langle \tilde{\bg}_t, \bx_t-\bu\rangle \right] + 2\alpha L_\infty T, \quad \forall \bu \in \Delta^{d-1}~.
\end{align*}

Putting together the last inequality and \eqref{eq:exp3_no_expl}, we have
\begin{align*}
\E\left[\sum_{t=1}^T g_{t,A_t}\right] - \sum_{t=1}^T \langle \bg_t,\bu\rangle
&\leq \frac{(1-\alpha)\ln d}{\eta} + \frac{(1-\alpha)\eta d^2 L_\infty^2 T}{2 \alpha^2} + 2\alpha L_\infty T\\
&\leq \frac{\ln d}{\eta} + \frac{\eta d^2 L_\infty^2 T}{2 \alpha^2} + 2\alpha L_\infty T, \quad \forall \bu \in \Delta^{d-1}~.
\end{align*}
Setting $\alpha \propto \frac{\sqrt{d} \ln^\frac{1}{4} d}{T^\frac{1}{4}}$ and $\eta \propto \sqrt{\frac{\alpha^2\ln d}{d^2 L^2_\infty T}}$, we obtain a regret of $\mathcal{O}(L_\infty \sqrt{d} T^{3/4} \ln^{1/4} d )$ as $T\to\infty$.

We did obtain a sublinear regret, but it is way worse than the $\mathcal{O}(\sqrt{T \ln d})$ of the full-information case. However, while it is expected that the bandit case must be more difficult than the full information one, it turns out that this is not the optimal strategy.

\index{Exp3 algorithm|(textbf}
\subsection{Exponential-weight algorithm for Exploration and Exploitation: Exp3}
\label{sec:exp3}

\begin{algorithm}[h]
\caption{Exponential-weight Algorithm for Exploration and Exploitation (Exp3)}
\label{alg:exp3}
\begin{algorithmic}[1]
{
    \REQUIRE{$\eta>0$}
    \STATE{$\bx_1=[1/d, \dots, 1/d]$}
    \FOR{$t=1$ {\bfseries to} $T$}
    \STATE{Draw $A_t$ according to $\Pr\{A_t=i\}=x_{t,i}$}
    \STATE{Select arm $A_t$}
    \STATE{Observe \emph{only} the loss of the selected arm $g_{t,A_t}\geq 0$ and pay it}
    \STATE{Construct the estimate $\tilde{g}_{t,i}=\begin{cases} \frac{g_{t,i}}{x_{t,i}}, & i=A_t\\ 0, & \text{otherwise}\end{cases}$ for $i=1,\dots,d$}
    \STATE{$x_{t+1,i} \propto x_{t,i} \exp(-\eta \tilde{g}_{t,i}), \ i=1,\dots,d$}
    \ENDFOR
}
\end{algorithmic}
\end{algorithm}

It turns out that the algorithm above actually works, even without mixing the output of the algorithm with the uniform distribution! We were just too loose in our regret guarantee. So, we will analyze Algorithm~\ref{alg:exp3}, \textbf{\ac{Exp3}}, that is nothing else than \ac{OMD} with entropic regularizer and stochastic estimates of the losses.
Note that now we will assume that $g_{t,i} \geq 0$ for all $t$ and $i$.

This time we will use the local norm\index{norm!local} regret bound for \ac{EG} that we proved in \eqref{eq:eg_localnorm}.
Let's consider the \ac{OMD} with the entropic regularizer, learning rate $\eta$, and set $\tilde{\bg}_{t}$ equal to the stochastic estimate of $\bg_t$, as in Algorithm~\ref{alg:exp3}. Applying \eqref{eq:eg_localnorm} and taking the expectation, for all $\bu \in \Delta^{d-1}$, we have
\[
\E\left[\sum_{t=1}^T g_{t,A_t}\right] - \sum_{t=1}^T \langle \bg_t,\bu\rangle
= \E\left[\sum_{t=1}^T \langle \tilde{\bg}_{t},\bx_t-\bu\rangle\right]
\leq \frac{\ln d}{\eta} + \frac{\eta}{2}\E\left[\sum_{t=1}^T \sum_{i=1}^d x_{t,i} \tilde{g}_{t,i}^2\right]~.
\]
Now, focusing on the terms $\E[x_{t,i} \tilde{g}_{t,i}^2]$, we have
\begin{align}
\E\left[\sum_{i=1}^d x_{t,i} \tilde{g}_{t,i}^2\right]
= \E\left[\E\left[\sum_{i=1}^d x_{t,i} \tilde{g}_{t,i}^2\middle|A_1, \dots, A_{t-1}\right]\right]
= \E\left[\sum_{i=1}^d g_{t,i}^2\right]
\leq d L_\infty^2. \label{eq:expectation_exp3}
\end{align}
Finally, we can choose $\eta = \sqrt{\frac{2\ln d}{d L^2_\infty T}}$ to minimize the regret.
Putting all together, we have the following theorem.
\begin{theorem}
\label{thm:exp3}
Assume $0 \leq g_{t,i}\leq L_\infty$ for $t=1,\dots,T$ and $i=1, \dots, d$.
Then, Algorithm~\ref{alg:exp3} with $\eta=\sqrt{\frac{2\ln d}{d L^2_\infty T}}$ guarantees
\[
\E\left[\sum_{t=1}^T g_{t,A_t}\right] - \sum_{t=1}^T \langle \bg_t,\bu\rangle
\leq L_\infty \sqrt{2 d T \ln d}, \quad \forall \bu \in \Delta^{d-1}~.
\]
\end{theorem}


\begin{figure}
\centering
\begin{tikzpicture}
\begin{axis}[
  width=7cm,
  xmin=-.2, xmax={sqrt(2)+.2},
  ymin=-.2, ymax={sqrt(6)/2+.2},
  xtick=\empty, ytick=\empty,
  axis lines=none,
  colormap={bw}{gray(0cm)=(0.9); gray(1cm)=(0.1)},
]
\addplot[contour prepared={labels=false}]
table{code_for_figs/simplex_contours_x0a_xyz.dat};
\draw[black] (axis cs:0,0) -- (axis cs:{sqrt(2)},0) -- (axis cs:{sqrt(2)/2},{sqrt(6)/2}) -- cycle;
\addplot[only marks, mark=x, mark size=3.5pt, thick, black] coordinates {(0.9545941546018391, 0.5511351921262153)};
\node[anchor=north east] at (axis cs:0,0) {$\be_1$};
\node[anchor=north west] at (axis cs:{sqrt(2)},0) {$\be_2$};
\node[anchor=south]      at (axis cs:{sqrt(2)/2},{sqrt(6)/2}) {$\be_3$};
\end{axis}
\end{tikzpicture}
\begin{tikzpicture}
\begin{axis}[
  width=7cm,
  xmin=-.2, xmax={sqrt(2)+.2},
  ymin=-.2, ymax={sqrt(6)/2+.2},
  xtick=\empty, ytick=\empty,
  axis lines=none,
  colormap={bw}{gray(0cm)=(0.9); gray(1cm)=(0.1)},
]
\addplot[contour prepared={labels=false}]
table{code_for_figs/simplex_contours_x0b_xyz.dat};
\draw[black] (axis cs:0,0) -- (axis cs:{sqrt(2)},0) -- (axis cs:{sqrt(2)/2},{sqrt(6)/2}) -- cycle;
\addplot[only marks, mark=x, mark size=3.5pt, thick, black] coordinates {(0.7071067811865475,0.4082482904638631)};
\node[anchor=north east] at (axis cs:0,0) {$\be_1$};
\node[anchor=north west] at (axis cs:{sqrt(2)},0) {$\be_2$};
\node[anchor=south]      at (axis cs:{sqrt(2)/2},{sqrt(6)/2}) {$\be_3$};
\end{axis}
\end{tikzpicture}
\caption{Contour plots of the \ac{KL} divergence in 3-dimensions when $\bx_t=[1/3,1/3,1/3]$ (left) and when $\bx_t=[0.1,0.45,0.45]$ (right).}
\label{fig:simplex}
\commentAlt{Figure~\ref{fig:simplex}. Two triangular simplex plots with contour lines of KL divergence. The left plot is centered at the uniform point [1/3,1/3,1/3]; the right plot is centered near [0.1,0.45,0.45].}
\end{figure}

So, with a tighter analysis, we showed that, even without an explicit exploration term, \ac{OMD} with an entropic regularizer solves the multi-armed bandit problem paying only a factor $\sqrt{d}$ more than the full information case.
The intuition of why this algorithm works is in Figure~\ref{fig:simplex}: coordinates with small probability have large curvature of the KL divergence, which translates into small inverse-Hessian weights in the local dual norm.

However, one can show that this is still not the optimal regret!
In the next section, we will change the regularizer to remove the $\sqrt{\ln d}$ term in the regret.
\index{Exp3 algorithm|)textbf}

\index{Tsallis-INF algorithm!OMD version|(textbf}
\subsection{Optimal Regret Using \ac{OMD} with Tsallis Entropy}

\begin{algorithm}[h]
\caption{Tsallis-INF\index{Implicitly Normalized Forecaster|textbf} (\ac{OMD} version)}
\label{alg:tsallis}
\begin{algorithmic}[1]
{
    \REQUIRE{$\eta>0, q \in (0,1)$}
    \STATE{$\bx_1=[1/d, \dots, 1/d]$}
    \FOR{$t=1$ {\bfseries to} $T$}
    \STATE{Draw $A_t$ according to $\Pr\{A_t=i\}=x_{t,i}$}
    \STATE{Select arm $A_t$}
    \STATE{Observe \emph{only} the loss of the selected arm $g_{t,A_t}\geq 0$ and pay it}
    \STATE{Construct the estimate $\tilde{g}_{t,i}=\begin{cases} \frac{g_{t,i}}{x_{t,i}}, & i=A_t\\ 0, & \text{otherwise}\end{cases}$ for $i=1,\dots,d$}
    \STATE{$\bx_{t+1} = \argmin_{\bx \in \Delta^{d-1}} \ \eta \langle \tilde{\bg}_t, \bx\rangle - \frac{1}{1-q} \sum_{i=1}^d x_i^q + \frac{q}{1-q}\sum_{i=1}^d x_{t,i}^{q-1} x_i $}
    \ENDFOR
}
\end{algorithmic}
\end{algorithm}

%

In this section, we present the \textbf{\ac{INF}}, which corresponds to \ac{OMD} with Tsallis entropy for the multi-armed bandit.

Define $\psi_q:\R^d_{\geq 0} \to \R$ as $\psi_q(\bx)=\frac{1}{1-q}\left(1-\sum_{i=1}^d x_i^q\right)$.
This is the negative \textbf{Tsallis entropy}\index{entropy!Tsallis|textbf} of the vector $\bx$. This is a strict generalization of the Shannon entropy, because when $q$ goes to 1, $\psi_q(\bx)$ converges to the negative Shannon entropy\index{entropy!negative Shannon} for any $\bx \in \Delta^{d-1}$.
Note that $\argmin_{\bx \in \Delta^{d-1}} \ \psi_q(\bx) = [\frac{1}{d}, \dots, \frac{1}{d}]$ and $\min_{\bx \in \Delta^{d-1}} \psi_q(\bx) = \frac{1 - d^{1-q}}{1-q}$.

We will instantiate \ac{OMD} with this regularizer for the multi-armed problem, obtaining the \textbf{Tsallis-INF} algorithm in Algorithm~\ref{alg:tsallis}.
We will not use any interpretation of this regularizer from the information theory point of view. As we will see in the following, the only reason to choose it is its Hessian.
In fact, the Hessian of this regularizer is still diagonal, and it is equal to $(\nabla^2 \psi_q (\bx))_{ii}= \frac{q}{x_i^{2-q}}$.
Once again, condition~\eqref{eq:cond_omd2} holds, so we can use again the second part of Lemma~\ref{lemma:omd_local_norms}.
So, for any $\bu \in \Delta^{d-1}$, we obtain
\begin{align*}
\sum_{t=1}^T \langle \bg_t, \bx_t - \bu\rangle
&\leq \frac{B_{\psi_q}(\bu; \bx_1)}{\eta} +  \frac{\eta}{2} \sum_{t=1}^T \bg_t^\top (\nabla^2 \psi_q(\bz'_t))^{-1} \bg_t\\
&= \frac{d^{1-q} - \sum_{i=1}^d u_i^q}{\eta(1-q)} +  \frac{\eta}{2 q} \sum_{t=1}^T \sum_{i=1}^d g_{t,i}^2 (z'_{t,i})^{2-q},
\end{align*}
where $\bz'_t=\alpha_t \bx_t+(1-\alpha_t) \tilde{\bx}_{t+1}$, $\alpha_t\in[0,1]$, and $\tilde{\bx}_{t+1} = \argmin_{\bx \in \R^d_{\geq 0}} \ \langle \bg_t, \bx\rangle + \frac{1}{\eta}B_{\psi_q}(\bx;\bx_t)$.

%

As we did for Exp3, now we need an upper bound to the $z'_{t,i}$.
From the definition of $\tilde{\bx}_{t+1}$ and $\psi_q$, we have
\[
-\frac{q}{1-q} \tilde{x}^{q-1}_{t+1,i}
= -\frac{q}{1-q} x^{q-1}_{t,i} - \eta g_{t,i},
\]
that is $\tilde{x}_{t+1,i}
= x_{t,i} \left(1+\frac{1-q}{q}\eta g_{t,i} x^{1-q}_{t,i}\right)^{-\frac{1}{1-q}}$.
Assuming $g_{t,i}\geq0$, we have $\tilde{x}_{t+1,i}\leq x_{t,i}$.

Hence, putting everything together, we have
\[
\sum_{t=1}^T \langle \bg_t, \bx_t - \bu\rangle
\leq \frac{d^{1-q} - \sum_{i=1}^d u_i^q}{\eta(1-q)} +  \frac{\eta}{2q} \sum_{t=1}^T \sum_{i=1}^d g_{t,i}^2 x_{t,i}^{2-q}, \quad \forall \bu \in \Delta^{d-1}~.
\]

We can now specialize the above reasoning, considering $q=1/2$ in the Tsallis entropy\index{entropy!Tsallis}, to obtain the following theorem.
\begin{theorem}
Assume $0\leq g_{t,i} \leq L_\infty$. Set $q=1/2$. Then, for any $\eta>0$, Algorithm~\ref{alg:tsallis} satisfies
\[
\E\left[\sum_{t=1}^T g_{t,A_t}\right] - \sum_{t=1}^T \langle \bg_t, \bu\rangle
\leq \frac{2\sqrt{d}}{\eta} + \eta \sqrt{d} L_\infty^2 T, \quad \forall \bu \in \Delta^{d-1}~.
\]
\end{theorem}
\begin{proof}
We only need to calculate the terms $\E[\sum_{i=1}^d \tilde{g}^2_{t,i} x_{t,i}^{3/2}]$.
Proceeding as in \eqref{eq:expectation_exp3}, we obtain
\begin{align*}
\E\left[\sum_{i=1}^d \tilde{g}^2_{t,i} x_{t,i}^{3/2}\right]
&= \E\left[\E\left[\sum_{i=1}^d x_{t,i}^{3/2} \tilde{g}_{t,i}^2\middle|A_1, \dots, A_{t-1}\right]\right]
= \E\left[\sum_{i=1}^d x_{t,i}^{3/2} \frac{g_{t,i}^2}{x_{t,i}^2} x_{t,i}\right]\\
&= \E\left[\sum_{i=1}^d g_{t,i}^2 \sqrt{x_{t,i}}\right]
\leq \E\left[\sqrt{\sum_{i=1}^d g_{t,i}^2} \sqrt{\sum_{i=1}^d g_{t,i}^2 x_{t,i}}\right]
\leq L_\infty^2 \sqrt{d}~. \qedhere
\end{align*}
\end{proof}
Choosing $\eta \propto \frac{1}{L_\infty \sqrt{T}}$, we finally obtain an expected regret of $\mathcal{O}(L_\infty\sqrt{d T})$ as $T\to\infty$, that can be proved to be the optimal one.

There is one last thing: how do we compute the predictions of this algorithm?
In each step, we have to solve a constrained optimization problem.
So, we can write the corresponding Lagrangian:
\[
L(\bx,\beta)
= \sum_{i=1}^d \left(\eta\tilde{g}_{t,i} + \frac{q}{1-q}x_{t,i}^{q-1}\right)x_i - \frac{1}{1-q} \sum_{i=1}^d x_i^q + \beta\left(\sum_{i=1}^d x_i-1\right)~.
\]
From the \acl{KKT} conditions\index{Karush--Kuhn--Tucker conditions}, we have
\[
x_{t+1,i}  = \left[\frac{1-q}{q}\left(\beta+ \frac{q}{1-q}x_{t,i}^{q-1}+\eta \tilde{g}_{t,i}\right)\right]^\frac{1}{q-1},
\]
and we also know that $\sum_{i=1}^d x_{t+1,i}=1$.
So, we have a one-dimensional problem in $\beta$ that must be solved in each round.
\index{Tsallis-INF algorithm!OMD version|)textbf}
\index{multi-armed bandit!adversarial|)}
\index{multi-armed bandit!stochastic|(}

\section{Stochastic Bandits}

We will now consider the \emph{stochastic bandit} setting. Here, each arm is associated with an unknown probability distribution. At each time step, the algorithm selects one arm $A_t$. Then, for each $i=1,\dots, d$, losses $g_{t,i}$ are drawn independently over time from that arm's fixed distribution, and the algorithm only receives the loss $g_{t,A_t}$.
We focus on minimizing the \textbf{pseudo-regret}\index{regret!pseudo-|(textbf}, that is, the regret with respect to the optimal action in expectation, rather than the optimal action on the sequence of realized losses:
\[
\PRegret_T
:=\E\left[\sum_{t=1}^T g_{t,A_t}\right] - \min_{i=1,\dots,d} \ \E\left[\sum_{t=1}^T g_{t,i}\right]
=\E\left[\sum_{t=1}^T g_{t,A_t}\right] - \min_{i=1,\dots,d} \ \mu_i T,
\]
where we denoted by $\mu_i$ the expectation of the distribution associated with the arm $i$.
\index{regret!pseudo-|)textbf}

\begin{remark}
The usual notation in the stochastic bandit literature is to consider rewards instead of losses. Instead, to keep our notation coherent with the \ac{OCO} literature, we will consider losses. The two things are completely equivalent up to a multiplication by $-1$.
\end{remark}

Before presenting our first algorithm for stochastic bandits, we will introduce some basic notions on concentration inequalities that will be useful in our definitions and proofs.

\subsection{Concentration Inequalities Bits}

Suppose that $X_1, X_2, \dots , X_n$ is a sequence of independent and identically distributed random variables with mean $\mu = \E[X_1]$ and variance $\sigma^2 = \Var[X_1]$. Having observed $X_1, X_2, \dots , X_n$, we would like to estimate the
common mean $\mu$. The most natural estimator is the \emph{empirical mean}:
\[
\hat{\mu}=\frac{1}{n}\sum_{i=1}^n X_i~.
\]
Linearity of expectation shows that $\E[\hat{\mu}] = \mu$, which means that $\hat{\mu}$ is an \emph{unbiased estimator}\index{unbiased estimator} of $\mu$. Yet, $\hat{\mu}$ is a random variable itself. So, can we quantify how far $\hat{\mu}$ will be from $\mu$?

We could use Chebyshev's inequality\index{inequality!Chebyshev's} to upper bound the probability that $\hat{\mu}$ is far from $\mu$:
\[
\Pr\{|\hat{\mu}-\mu|\geq\epsilon\}
\leq \frac{\Var[\hat{\mu}]}{\epsilon^2}~.
\]
Using the fact that $\Var[\hat{\mu}] = \frac{\sigma^2}{n}$, we have that
\[
\Pr\{|\hat{\mu}-\mu|\geq\epsilon\}
\leq \frac{\sigma^2}{n \epsilon^2}~.
\]
So, we can expect the probability of having a ``bad'' estimate to go to zero as one over the number of samples in our empirical mean.
Is this the best we can get? To understand what we can hope for, let's take a look at the central limit theorem.

We know that, defining $S_n=\sum_{t=1}^n (X_t-\mu)$, $\frac{S_n}{\sqrt{n \sigma^2}}\to N(0,1)$ in distribution, as $n$ goes to infinity.
This means that
\[
\Pr\{\hat{\mu}-\mu \geq \epsilon\}
= \Pr\{S_n \geq n\epsilon\}
= \Pr\left\{\frac{S_n}{\sqrt{n \sigma^2}} \geq \sqrt{\frac{n}{\sigma^2}}\epsilon\right\}
\approx \int_{\epsilon\sqrt{\frac{n}{\sigma^2}}}^\infty \! \frac{1}{\sqrt{2\pi}} \exp\left(-\frac{x^2}{2}\right) \, \mathrm{d}x,
\]
where the approximation comes from the central limit theorem.
The integral cannot be calculated with a closed form, but we can easily upper-bound it. Indeed, for $a>0$, we have
\begin{align*}
\int_{a}^\infty \! \exp\left(-\frac{x^2}{2}\right) \, \mathrm{d}x
&= \int_{a}^\infty \! \frac{x}{x}\exp\left(-\frac{x^2}{2}\right) \, \mathrm{d}x
\leq \frac{1}{a}\int_{a}^\infty \! x\exp\left(-\frac{x^2}{2}\right) \, \mathrm{d}x\\
&= \frac{1}{a} \exp\left(-\frac{a^2}{2}\right)~.
\end{align*}
Hence, we have
\begin{equation}
\label{eq:conc_gaussian}
\Pr\{\hat{\mu}-\mu \geq \epsilon\}
\lessapprox \sqrt{\frac{\sigma^2}{2 \pi \epsilon^2 n}} \exp\left(-\frac{n \epsilon^2}{2 \sigma^2}\right)~.
\end{equation}
This bound has a better dependency on $n$ than what we got with Chebyshev's inequality, and we would like to obtain a finite-time bound with a similar asymptotic rate.
To do that, we will focus our attention on \emph{subgaussian} random variables.
\begin{definition}
We say that a random variable is \textbf{$\sigma$-subgaussian}\index{random variable!subgaussian|textbf} if for all $\lambda \in \R$ we have that $\E[\exp(\lambda X)]\leq \exp(\lambda^2 \sigma^2/2)$.
\end{definition}

\index{random variable!subgaussian|(}
\begin{example}
The following random variables are subgaussian:
\begin{itemize}
\item If $X$ is Gaussian with mean zero and variance $\sigma^2$, then $X$ is $\sigma$-subgaussian.
\item If $X$ has mean zero and $X \in [a, b]$ almost surely, then $X$ is $(b - a)/2$-subgaussian.
\end{itemize}
\end{example}

We have the following properties for subgaussian random variables, left as exercises (Problem~\ref{exercise:subgaussian}).
\begin{lemma}
\label{lemma:subgauss_properties}
Assume that the random variables $X_1$ and $X_2$ are independent, and $\sigma_1$-subgaussian and $\sigma_2$-subgaussian, respectively.
Then,
\begin{enumerate}[(a)]
\item $\E[X_1] = 0$ and $\Var[X_1] \leq \sigma_1^2$.
\item $c X_1$ is $|c|\sigma_1$-subgaussian.
\item $X_1+X_2$ is $\sqrt{\sigma^2_1+\sigma_2^2}$-subgaussian.
\end{enumerate}
\end{lemma}

Subgaussian random variables behave like zero-mean Gaussian random variables, in the sense that their tail probabilities are upper bounded by the ones of a Gaussian of variance $\sigma^2$.
To prove it, let's first state Markov's inequality.
\begin{theorem}[Markov's inequality]
\index{inequality!Markov's|textbf}
For a non-negative random variable $X$ and $\epsilon>0$, we have that $\Pr\{X\geq \epsilon\} \leq \frac{\E[X]}{\epsilon}$.
\end{theorem}
With Markov's inequality, we can now formalize the above statement on subgaussian random variables.
\begin{theorem}
\label{thm:conc_subgaussian}
Let $\epsilon> 0$. If a random variable is $\sigma$-subgaussian, then $\Pr\{X\geq \epsilon\} \leq \exp\left(-\frac{\epsilon^2}{2\sigma^2}\right)$.
\end{theorem}
\begin{proof}
For any $\lambda>0$, we have
\begin{align*}
\Pr\{X\geq \epsilon\}
= \Pr\{\exp(\lambda X)\geq \exp(\lambda\epsilon)\}
\leq \frac{\E[\exp(\lambda X)]}{\exp(\lambda \epsilon)}
\leq \exp(\lambda^2 \sigma^2/2-\lambda \epsilon)~.
\end{align*}
Minimizing the r.h.s. of the inequality with respect to $\lambda$, we have the stated result.
\end{proof}

An easy consequence of this theorem is that the empirical average of subgaussian random variables concentrates around its expectation, \emph{with the same asymptotic rate as in \eqref{eq:conc_gaussian}}.
\begin{corollary}
\label{cor:conc_mean_subgaussian}
Assume that $X_i - \mu$ are independent, $\sigma$-subgaussian random variables.
Let $\hat{\mu} := \frac1n \sum_{i=1}^n X_i$.
Then, for any $\epsilon \geq 0$, we have
\begin{align*}
\Pr\left\{\hat{\mu} \geq \mu + \epsilon\right\} \leq \exp\left(- \frac{n \epsilon^2}{2\sigma^2}\right)
&&\text{ and }&&
\Pr\left\{\hat{\mu} \leq \mu - \epsilon\right\} \leq \exp\left(- \frac{n \epsilon^2}{2\sigma^2}\right)~.
\end{align*}
\end{corollary}
Equating the upper bounds on the r.h.s. of the inequalities in the Corollary to $\delta$, we have the equivalent statement that, with probability at least $1-2\delta$, we have
\[
\mu \in \left[\hat{\mu}-\sqrt{\frac{2 \sigma^2\ln\frac{1}{\delta}}{n}}, \hat{\mu}+\sqrt{\frac{2 \sigma^2\ln\frac{1}{\delta}}{n}}\right]~.
\]
\index{random variable!subgaussian|)}

\subsection{Explore-Then-Commit Algorithm}
\label{sec:etc}

\index{Explore-then-Commit algorithm|(textbf}

\begin{algorithm}[t]
\caption{Explore-Then-Commit Algorithm}
\label{alg:etc}
\begin{algorithmic}[1]
{
    \REQUIRE{$T,m \in \Nat, 1\leq m\leq \frac{T}{d}$}
    \STATE{$S_{0,i}=0, \hat{\mu}_{0,i}=0, \ i=1, \dots, d$}
    \FOR{$t=1$ {\bfseries to} $T$}
    \STATE{Choose $A_t=\begin{cases} ((t-1) \mod d)+1, & t\leq d m\\ \argmin_i \ \hat{\mu}_{d m,i}, & t> dm\end{cases}$}
    \STATE{Observe $g_{t,A_t}$ and pay it}
    \STATE{$S_{t,i}=S_{t-1,i}+\indevent{\{A_t=i\}}$}
    \STATE{$\hat{\mu}_{t,i}= \frac{1}{S_{t,i}}\sum_{j=1}^t g_{j,A_j}\indevent{\{A_j=i\}}$ for $i$ such that $S_{t,i}>0$, and $\hat{\mu}_{t,i}=0$ otherwise}
    \ENDFOR
}
\end{algorithmic}
\end{algorithm}

We are now ready to present the most natural algorithm for the stochastic bandit setting, called \textbf{\ac{ETC}} algorithm. That is, we first identify the best arm over $m d$ exploration rounds, and then we commit to it. This algorithm is summarized in Algorithm~\ref{alg:etc}.

In the following, we will denote by $S_{t,i}=\sum_{j=1}^t \indevent{\{A_j=i\}}$, that is the number of times that the arm $i$ was pulled in the first $t$ rounds.

%

Define by $\mu^\star$ the expected loss of the arm with the smallest expectation, that is
\[
\mu^\star := \min_{i=1,\dots,d} \ \mu_i~.
\]
Critical quantities in our analysis will be the \emph{gaps}, $\Delta_i:=\mu_i-\mu^\star$ for $i=1, \dots,d$, that measure the expected difference in losses between the arms and the optimal one. In particular, we can decompose the regret as a sum over the arms of the expected number of times we pull an arm multiplied by its gap.
\index{pseudo-regret decomposition|(textbf}
\begin{lemma}
\label{lemma:stoch_band_decomp}
For any online policy of selection of the arms, the pseudo-regret is equal to
\[
\PRegret_T = \sum_{i=1}^d \E[S_{T,i}] \Delta_i~.
\]
\end{lemma}
\begin{proof}
Observe that
\[
\sum_{t=1}^T g_{t,A_t}
= \sum_{t=1}^T \sum_{i=1}^d g_{t,i} \indevent{\{A_t=i\}}~.
\]
Hence,
\begin{align*}
\PRegret_T
&= \E\left[\sum_{t=1}^T g_{t,A_t}\right] - T \mu^\star
= \E\left[\sum_{t=1}^T (g_{t,A_t}- \mu^\star)\right] \\
&= \sum_{i=1}^d \sum_{t=1}^T \E[\indevent{\{A_t=i\}}(g_{t,i}-\mu^\star)]
= \sum_{i=1}^d \sum_{t=1}^T \E[\E[\indevent{\{A_t=i\}}(g_{t,i}- \mu^\star)|A_t]]\\
&= \sum_{i=1}^d \sum_{t=1}^T \E[\indevent{\{A_t=i\}} \E[g_{t,i} - \mu^\star |A_t]]
= \sum_{i=1}^d \sum_{t=1}^T \E[\indevent{\{A_t=i\}}] (\mu_i- \mu^\star)~. \qedhere
\end{align*}
\end{proof}
The above lemma quantifies the intuition that in order to have a small regret, we have to select the suboptimal arms less often than the best one.
\index{pseudo-regret decomposition|)textbf}

We are now ready to prove the regret guarantee of the \ac{ETC} algorithm.
\begin{theorem}
Assume that the losses of the arms minus their expectations are $1$-sub\-gauss\-ian\index{random variable!subgaussian} and $1\leq m \leq T/d$. Then, \ac{ETC} guarantees a pseudo-regret of
\[
\PRegret_T
\leq m \sum_{i=1}^d \Delta_i + (T-md)\sum_{i=1}^d \Delta_i \exp\left(-\frac{m \Delta^2_i}{4}\right)~.
\]
\end{theorem}
\begin{proof}
Without loss of generality, let's assume that the optimal arm is the first one.

So, for $i\neq 1$, we have
\begin{align*}
\E[S_{T,i}]
&= \sum_{t=1}^T \E[\indevent{\{A_t=i\}}]
= m + (T-m d) \Pr\left\{\hat{\mu}_{md,i} \leq \min_{j\neq i} \ \hat{\mu}_{md,j}\right\} \\
&\leq m + (T-m d) \Pr\left\{\hat{\mu}_{md,i} \leq \hat{\mu}_{md,1}\right\}\\
&=m + (T-m d) \Pr\left\{\hat{\mu}_{md,1}-\mu_1-(\hat{\mu}_{md,i} -\mu_i)\geq \Delta_i\right\}~.
\end{align*}
From Lemma~\ref{lemma:subgauss_properties}, we have that $\hat{\mu}_{md,i} -\mu_i -(\hat{\mu}_{md,1}-\mu_1)$ is $\sqrt{2/m}$-subgaussian. So, from Theorem~\ref{thm:conc_subgaussian}, we have
\[
\Pr\left\{\hat{\mu}_{md,1}-\mu_1-(\hat{\mu}_{md,i} -\mu_i) \geq \Delta_i\right\}
\leq \exp\left(-\frac{m \Delta^2_i}{4}\right)~.
\]
Using Lemma~\ref{lemma:stoch_band_decomp} finishes the proof.
\end{proof}

The bound shows the trade-off between exploration and exploitation: if $m$ is too big, we pay too much during the exploration phase (first term in the bound). On the other hand, if $m$ is small, the probability of selecting a suboptimal arm increases (second term in the bound).
Knowing all the gaps $\Delta_i$, it is possible to choose $m$ that minimizes the bound.

For example, in the case that $d=2$, the regret is upper bounded by
\[
m \Delta + (T-2m) \Delta \exp\left(-m \frac{\Delta^2}{4}\right)
\leq m \Delta + T\Delta \exp\left(-m \frac{\Delta^2}{4}\right),
\]
that is minimized by $m = \frac{4}{\Delta^2} \ln \frac{T\Delta^2}{4}$.
Remembering that $m$ must be a natural number, we can choose
\[
m=\max\left(\left\lceil \frac{4}{\Delta^2} \ln \frac{T\Delta^2}{4} \right\rceil,1\right)~.
\]
When $\frac{T\Delta^2}{4}\leq 1$, we select $m=1$. So, we have $\Delta + (T-2) \Delta \leq T \Delta \leq \frac{4}{\Delta}$. Hence, the regret is upper-bounded by
\[
\min\left(\Delta T,\frac{4}{\Delta} \left(1+ \max\left(\ln \frac{T\Delta^2}{4},0\right)\right)  + \Delta\right)
= \mathcal{O}\left(\frac{\ln T}{\Delta}\right) \ \text{ as $T$ goes to infinity.}
\]

The main drawback of this algorithm is that its optimal tuning depends on the gaps $\Delta_i$. Assuming knowledge of the gaps essentially makes the stochastic bandit problem trivial. However, its tuned regret bound gives us a baseline with which to compare other bandit algorithms. In particular, in the next section, we will present an algorithm that achieves the same asymptotic regret without any knowledge of the gaps.

\index{Explore-then-Commit algorithm|)textbf}

\subsection{Upper Confidence Bound Algorithm}

\index{Upper Confidence Bound algorithm|(textbf}

\begin{algorithm}[t]
\caption{Upper Confidence Bound Algorithm}
\label{alg:ucb}
\begin{algorithmic}[1]
{
    \REQUIRE{$\alpha>2, T \in \Nat$}
    \STATE{$S_{0,i}=0, \hat{\mu}_{0,i}=0, \ i=1, \dots, d$}
    \FOR{$t=1$ {\bfseries to} $T$}
    \STATE{Choose $A_t = \argmin_{i=1, \dots, d} \ \begin{cases} \hat{\mu}_{t-1,i}-\sqrt{\frac{2\alpha \ln t}{S_{t-1,i}}}, & \text{if }S_{t-1,i}\neq0\\ -\infty, & \text{otherwise}\end{cases}$}
    \STATE{Observe $g_{t,A_t}$ and pay it}
    \STATE{$S_{t,i}=S_{t-1,i}+\indevent{\{A_t=i\}}$}
    \STATE{$\hat{\mu}_{t,i}= \frac{1}{S_{t,i}}\sum_{j=1}^t g_{j,A_j}\indevent{\{A_j=i\}}$ for $i$ such that $S_{t,i}>0$, and $\hat{\mu}_{t,i}=0$ otherwise}
    \ENDFOR
}
\end{algorithmic}
\end{algorithm}

The \ac{ETC} algorithm has the disadvantage of requiring the knowledge of the gaps to tune the length of the exploration phase. Moreover, it solves the exploration vs. exploitation trade-off in a clunky way. It would be better to have an algorithm that smoothly mixes exploration and exploitation \emph{in a data-dependent way}.
So, we now describe an almost optimal and adaptive strategy called \textbf{\ac{UCB}} algorithm. It employs the principle of \emph{optimism in the face of uncertainty}\index{optimism in the face of uncertainty}, to select in each round the arm that has the \emph{potential to be the best one}.

\ac{UCB} works by keeping an estimate of the expected loss of each arm and also a confidence interval at a certain probability. Roughly speaking, we have that with probability at least $1-\delta$
\[
\mu_i \in \left[\hat{\mu}_i - \sqrt{\frac{2\ln \frac{1}{\delta}}{S_{t-1,i}}}, +\infty\right),
\]
where the ``roughly'' comes from the fact that $S_{t-1,i}$ is a random variable itself.
Then, \ac{UCB} will query the arm with the smallest lower bound, that is, the one that could potentially have the smallest expected loss. The regret bound below shows that this strategy automatically balances exploration (given by the use of the optimistic estimates through the confidence intervals) and exploitation (when the confidence intervals are small enough so that we query the arm with the smallest empirical mean).

\begin{remark}
The name ``Upper Confidence Bound'' comes from the fact that traditionally stochastic bandits are defined over rewards, rather than losses. So, in our case, we actually use the lower confidence bound in the algorithm. However, to avoid confusion with the literature, we still call it Upper Confidence Bound algorithm.
\end{remark}

The proof works by proving that, once we have queried a suboptimal arm enough times, we will pull it again only if its confidence interval or the one around the best arm does not contain the true means. In turn, the probability that the confidence intervals are wrong is small because we use concentration inequalities to construct them. Moreover, a delicate point is to deal with the fact that the number of times we pull an arm is a random variable, making the application of the concentration inequality not immediate. We get around this by considering the worst number of pulls and taking a union bound on it.

The algorithm is summarized in Algorithm~\ref{alg:ucb}, and we can prove the following regret bound.
\begin{theorem}
\label{thm:ucb}
Assume that the losses of the arms minus their expectations are $1$-sub\-gauss\-ian\index{random variable!subgaussian} and let $\alpha>2$. Then, Algorithm~\ref{alg:ucb} guarantees a pseudo-regret of
\[
\PRegret_T
\leq \frac{\alpha}{\alpha-2}\sum_{i=1}^d \Delta_i + \sum_{i: \Delta_i>0} \frac{8\alpha \ln T}{\Delta_i} ~.
\]
\end{theorem}
\begin{proof}
We analyze one arm at a time. Also, without loss of generality, assume that the optimal arm is the first one.
For each arm $i>1$, we want to prove that $\E[S_{T,i}]\leq \frac{8\alpha \ln T}{\Delta^2_i} + \frac{\alpha}{\alpha-2}$.

The proof is based on the fact that once we have sampled a suboptimal arm enough times, the probability of selecting it again becomes small.

Let $t^\star$ be the largest time index such that $S_{t^\star-1,i} \leq \frac{8\alpha \ln T}{\Delta_i^2}$. If $t^\star=T$, then the statement above is true. Hence, we can safely assume $t^\star < T$. Now, for $t> t^\star$, we have
\begin{equation}
\label{eq:thm_ucb_eq1}
S_{t-1,i} > \frac{8\alpha \ln T}{\Delta_i^2}~.
\end{equation}

Consider $t> t^\star$ and such that $A_t = i$, then we claim that at least one of the two following equations must be true:
\begin{align}
\hat{\mu}_{t-1,1} - \sqrt{\frac{2\alpha \ln t}{S_{t-1,1}}} &\geq \mu_1, \label{eq:ucb1}\\
\hat{\mu}_{t-1,i} + \sqrt{\frac{2\alpha \ln t}{S_{t-1,i}}} &< \mu_i~. \label{eq:ucb2}
\end{align}
If the first one is true, the confidence interval around our estimate of the expectation of the optimal arm does not contain $\mu_1$.
On the other hand, if the second one is true, the confidence interval around our estimate of the expectation $\mu_i$ does not contain $\mu_i$. So, we claim that if $t> t^\star$ and we selected the suboptimal arm $i$, then at least one of these two bad events happened.

Let's prove the claim by contradiction: \emph{if both the inequalities above are false}, $t> t^\star$, and $A_t = i$, we have
\begin{align*}
\hat{\mu}_{t-1,1} - \sqrt{\frac{2\alpha \ln t}{S_{t-1,1}}}
&\stackrel{\eqref{eq:ucb1} \text{ false}}{<} \mu_1
= \mu_i - \Delta_i
\stackrel{\eqref{eq:thm_ucb_eq1}}{<} \mu_i - 2\sqrt{\frac{2\alpha \ln T}{S_{t-1,i}}}\\
&\leq \mu_i - 2\sqrt{\frac{2\alpha \ln t}{S_{t-1,i}}}
\stackrel{\eqref{eq:ucb2} \text{ false}}{\leq} \hat{\mu}_{t-1,i} - \sqrt{\frac{2 \alpha \ln t}{S_{t-1,i}}},
\end{align*}
that, by the selection strategy of the algorithm, would imply $A_t \neq i$.

Note that $S_{t^\star,i} \leq \frac{8\alpha \ln T}{\Delta_i^2}+1$. Hence, we have
\begin{align*}
S_{T,i}
&= S_{t^\star,i} + \sum_{t=t^\star+1}^T \indevent{\{A_t=i\}}
\leq \frac{8\alpha \ln T}{\Delta_i^2} + 1 + \sum_{t=t^\star+1}^T \indevent{\{\eqref{eq:ucb1} \text{ or } \eqref{eq:ucb2} \text{ true}\}} \\
&\leq \frac{8\alpha \ln T}{\Delta_i^2} + 1 + \sum_{t=1}^\infty \indevent{\{\eqref{eq:ucb1} \text{ or } \eqref{eq:ucb2} \text{ true}\}} \\
&\leq  \frac{8\alpha \ln T}{\Delta_i^2} + 1 + \sum_{t=1}^\infty \left(\indevent{\{\eqref{eq:ucb1} \text{ true}\}} + \indevent{\{\eqref{eq:ucb2} \text{ true}\}}\right)~.
\end{align*}
Taking expectations on both sides, we have
\[
\E[S_{T,i}]
\leq \frac{8\alpha \ln T}{\Delta_i^2} + 1 + \sum_{t=1}^\infty \left(\Pr\{\eqref{eq:ucb1} \text{ true}\} + \Pr\{\eqref{eq:ucb2} \text{ true}\}\right)~.
\]

Now, we upper bound the probabilities in the sum.
For each arm $i$, let $Y_{i,s}$ denote the $s$-th observed loss from arm $i$.
Given that the losses on the arms are i.i.d. and using the union bound\index{union bound}, we have
\begin{align*}
  \Pr \left\{\hat{\mu}_{t-1,1} - \sqrt{\frac{2\alpha \ln t}{S_{t-1,1}}} \geq \mu_1 \right\}
  &\leq \Pr \left\{\max_{s=1, \dots, t-1} \ \frac{1}{s}\sum_{r=1}^s Y_{1,r} - \sqrt{\frac{2\alpha \ln t}{s}} \geq \mu_1\right\} \\
  &= \Pr\left\{\bigcup_{s=1}^{t-1} \left\{\frac{1}{s}\sum_{r=1}^s Y_{1,r} - \sqrt{\frac{2\alpha \ln t}{s}} \geq \mu_1\right\}\right\}
  ~.
\end{align*}
Hence, using Corollary~\ref{cor:conc_mean_subgaussian} and the union bound\index{union bound}, we have
\[
\Pr\{\eqref{eq:ucb1} \text{ true} \}
\leq \sum_{s=1}^{t-1} \Pr\left\{\frac{1}{s}\sum_{r=1}^s Y_{1,r} - \sqrt{\frac{2\alpha \ln t}{s}} \geq \mu_1\right\}
\leq \sum_{s=1}^{t-1} t^{-\alpha}
= (t-1) t^{-\alpha}~.
\]
Given that the same bound holds for $\Pr\{\eqref{eq:ucb2} \text{ true}\}$, we have
\begin{align*}
\E\left[ S_{T,i}\right]
&\leq \frac{8\alpha \ln T}{\Delta_i^2} + 1 + \sum_{t=1}^{\infty} 2 (t-1) t^{-\alpha}
\leq \frac{8\alpha \ln T}{\Delta_i^2} + 1 + \sum_{t=2}^{\infty} 2 t^{1-\alpha} \\
&\leq \frac{8\alpha \ln T}{\Delta_i^2} + 1 + 2\int_{1}^{\infty} \! x^{1-\alpha} \, \mathrm{d}x
=\frac{8\alpha \ln T}{\Delta_i^2} + \frac{\alpha}{\alpha-2}~.
\end{align*}
Using Lemma~\ref{lemma:stoch_band_decomp},
we have the stated bound.
\end{proof}

The bound above can become meaningless if the gaps are too small. So, here we prove another bound that does not depend on the inverse of the gaps.

\begin{theorem}
\label{thm:ucb2}
Assume that the losses of the arms minus their expectations are $1$-sub\-gauss\-ian\index{random variable!subgaussian} and let $\alpha>2$. Then, Algorithm~\ref{alg:ucb} guarantees a pseudo-regret of
\[
\PRegret_T
\leq 4\sqrt{2\alpha d T \ln T} + \frac{\alpha}{\alpha-2} \sum_{i=1}^d \Delta_i~.
\]
\end{theorem}
\begin{proof}
Let $\beta > 0$ be some value to be tuned subsequently, and recall from the proof of Theorem~\ref{thm:ucb} that for each suboptimal arm $i$, we can bound
\[
\E[S_{T,i}] \leq \frac{\alpha}{\alpha-2} + \frac{8\alpha \ln T}{\Delta_i^2}~.
\]
Hence, using the pseudo-regret decomposition in Lemma~\ref{lemma:stoch_band_decomp}, we have
\begin{align*}
\PRegret_T
&= \sum_{i:\Delta_i < \beta} \Delta_i \E[S_{T,i}] + \sum_{i:\Delta_i \geq \beta} \Delta_i \E[S_{T,i}]
\leq T\beta + \sum_{i:\Delta_i \geq \beta} \Delta_i \E[S_{T,i}] \\
&\leq T\beta + \sum_{i:\Delta_i \geq \beta} \left(\frac{\Delta_i\alpha}{\alpha-2} + \frac{8\alpha \ln T}{\Delta_i}\right)
\leq T\beta + \frac{\alpha}{\alpha-2} \sum_{i=1}^d \Delta_i + \frac{8\alpha d \ln T}{\beta}~.
\end{align*}
Choosing $\beta=\sqrt{\frac{8\alpha d \ln T}{T}}$, we have the stated bound.
\end{proof}

\begin{remark}
Note that in the \ac{UCB} algorithm, we have to know the subgaussian parameter of the arms. While this can be easily upper-bounded for stochastic arms with bounded support, it is unclear how to do it without any prior knowledge of the distribution of the arms.
\end{remark}

It is possible to prove that the \ac{UCB} algorithm is asymptotically optimal, in the sense of the following theorem.
\begin{theorem}[{\citealp[Theorem 2.2]{BubeckCB12}}]
Consider a strategy that satisfies $\E[S_{T,i}] = o(T^a)$ as $T\to\infty$ for any set of Bernoulli loss distributions, any arm $i$ with $\Delta_i>0$ and any $a>0$. Then, for any set of Bernoulli loss distributions, we have
\[
\liminf\limits_{T\to +\infty} \ \frac{\PRegret_T}{\ln T} \geq \sum_{i: \Delta_i>0} \frac{1}{2\Delta_i}~.
\]
\end{theorem}

\index{Upper Confidence Bound algorithm|)textbf}
\index{multi-armed bandit!stochastic|)}

\index{multi-armed bandit!best-of-both worlds|(}
\index{Tsallis-INF algorithm!FTRL version|(textbf}
\section{Best-of-Both Worlds Guarantees with Tsallis-INF}

We considered the two different settings for bandits: adversarial and stochastic. The optimal algorithms for these two settings were different. However, it is natural to ask if it is possible to have a \emph{single} algorithm that achieves the best guarantee in both settings.
It turns out that the answer is positive with the Tsallis-INF, but this time we will need the \ac{FTRL} version, see Algorithm~\ref{alg:tsallis_ftrl}.

\begin{algorithm}[h]
	\caption{Tsallis-INF (\ac{FTRL} version)}
	\label{alg:tsallis_ftrl}
	\begin{algorithmic}[1]
		{
		\REQUIRE{$L_\infty>0$}
		\FOR{$t=1$ {\bfseries to} $T$}
		\STATE{$\bx_{t} = \argmin_{\bx \in \Delta^{d-1}} \ \langle \sum_{n=1}^{t-1} \tilde{\bg}_n, \bx\rangle - 4 L_\infty \sqrt{t} \sum_{i=1}^d \sqrt{x_{i}}  $}
		\STATE{Draw $A_t$ according to $\Pr\{A_t=i\}=x_{t,i}$}
		\STATE{Select arm $A_t$}
		\STATE{Observe \emph{only} the loss of the selected arm $g_{t,A_t}$ and pay it}
		\STATE{Construct the estimate $\tilde{g}_{t,i}=\begin{cases} \frac{g_{t,i}}{x_{t,i}}, & i=A_t\\ 0, & \text{otherwise}\end{cases}$ for $i=1,\dots,d$}
		\ENDFOR
		}
	\end{algorithmic}
\end{algorithm}

\begin{theorem}
	\label{thm:bobw}
	Assume that $0\leq g_{t,i}\leq L_\infty$, for all $t=1, \dots, T$, $i=1, \dots, d$, where $d\geq 1$. Then, Algorithm~\ref{alg:tsallis_ftrl} satisfies
	\[
		\E[\Regret_T]
		\leq 32 L_\infty \sqrt{(d-1) T}~.
	\]
	Moreover, if in addition the $g_{t,i}$ are i.i.d. from a distribution $\rho_i$ with mean $\mu_i$ for $i=1,\dots, d$, and  $\argmin_i \mu_i$ is unique, then we also have
	\[
		\PRegret_T
		\leq 256 L_\infty^2 \sum_{i: \Delta_i\neq 0} \frac{1+\ln T}{\Delta_i}~.
	\]
\end{theorem}
\begin{proof}
	Algorithm~\ref{alg:tsallis_ftrl} is an instantiation of \ac{FTRL} with linear losses $\ell_t(\bx)=\langle \tilde{\bg}_t, \bx\rangle$ and regularizer $\psi_t(\bx)=\lambda_t \psi(\bx)$, where $\psi(\bx)= -\sum_{i=1}^d \sqrt{x_i}$ and $\lambda_t = 4 L_\infty \sqrt{t}$.
	Note that $\lambda_{T+1}$ has no influence on the regret, so we can set it equal to $\lambda_T$ \emph{ex post facto}.
	Moreover, given that $\psi$ has a partial derivative that goes $-\infty$ if a coordinate goes to 0, we know from the first-order optimality condition that $\bx_t$ cannot have any coordinate equal to 0 for any $t$.
	Hence, from Lemma~\ref{lemma:ftrl_local_norms} we have
	\begin{align*}
		\sum_{t=1}^T \langle \tilde{\bg}_t,\bx_t - \bu\rangle
		&\leq \psi_{T}(\bu) - \min_{\bx \in \mathcal{V}} \ \psi_{1}(\bx)
		+ \frac12 \sum_{t=1}^T \|\tilde{\bg}_t\|^2_{(\nabla^2 \psi_t(\bz_t))^{-1}}\\
		&\quad + \sum_{t=1}^{T-1}(\psi_t(\bx_{t+1}) - \psi_{t+1}(\bx_{t+1})),
	\end{align*}
	where $\bz_t$ is on the line segment between $\bx_t$ and $\tilde{\bx}_{t+1}:=\argmin_{\bx \in \R^d_{\geq 0}} \ B_{\psi_t}(\bx; \bx_t) +\langle \tilde{\bg}_t, \bx\rangle$.
	Note that $\bx_1=\argmin_{\bx \in \Delta^{d-1}} \ \psi_1(\bx)=[1/d, \dots, 1/d]$ and $\psi(\bx_1)=-\sqrt{d}$. Also, for $\be_j$ being the $j$-th canonical basis, we have $\psi(\be_j)=-1$ for any $j=1, \dots, d$.

	We now use the fact that the update of the algorithm remains the same if $b_t$ is subtracted from all coordinates of the estimated loss vector in each round $t$. Hence, we can state the bound for a choice of $b_t$ that helps reduce the first sum on the r.h.s. of the above inequality.
	The Hessian of the $\psi$ is the diagonal matrix with entries $(\nabla^2 \psi(\bx))_{ii}=\frac{1}{4 x_i^{1.5}}$. Hence, we have
	\begin{align*}
		\sum_{t=1}^T \langle \tilde{\bg}_t, \bx_t-\bu\rangle
		&\leq \psi_{T}(\bu)-\psi_1(\bx_1)+ \sum_{t=1}^T \frac{2}{\lambda_t}\sum_{i=1}^d (\tilde{g}_{t,i}-b_t)^2 z_{t,i}^{1.5} \\
		&\quad + \sum_{t=1}^{T-1} (\psi_t(\bx_{t+1})-\psi_{t+1}(\bx_{t+1}) ), \quad \forall \bu \in \Delta^{d-1}~.
	\end{align*}

	If we constrain $0\leq b_t\leq L_\infty$ and $\frac{L_\infty}{\lambda_t}\leq 1/4$, we have that $\tilde{x}_{t+1,i}$ satisfies
	\[
		\tilde{x}_{t+1,i}
		=\frac{x_{t,i}}{(1+2\frac{1}{\lambda_t}(\tilde{g}_{t,i}-b_t)x_{t,i}^{1/2})^2}
		\leq 4x_{t,i}~.
	\]
	Then, we select $b_t=\indevent{\{A_t=j\}} g_{t,j}$,
	where $j$ is an arbitrary arm, and we take the expectation to have
	\begin{align*}
		\E&\left[\sum_{i=1}^d (\tilde{g}_{t,i}-\indevent{\{A_t=j\}}g_{t,j})^2 z_{t,i}^{1.5}\right]\\
		&\leq 8\E\left[\sum_{i=1}^d (\tilde{g}_{t,i}-\indevent{\{A_t=j\}}g_{t,j})^2 x_{t,i}^{1.5}\right]\\
		 & = 8 \E\left[\sum_{i=1}^d \left(\frac{g_{t,i}}{x_{t,i}}\indevent{\{A_t=i\}}-\indevent{\{A_t=j\}} g_{t,j}\right)^2 x_{t,i}^{1.5} \right]                                                                                             \\
		 & = 8\E\left[g_{t,j}^2 \left(\frac{1}{x_{t,j}}-1\right)^2 x_{t,j}^{1.5} \indevent{\{A_t=j\}}\right]\\
		 &\quad + 8\E\left[\sum_{i\neq j} \left(\frac{g^2_{t,i}}{x^2_{t,i}}\indevent{\{A_t=i\}}+\indevent{\{A_t=j\}} g^2_{t,j}\right) x_{t,i}^{1.5}\right] \\
		 & = 8\E\left[g_{t,j}^2 \left(\frac{1}{x_{t,j}}-1\right)^2 x_{t,j}^{1.5} x_{t,j} + \sum_{i\neq j} g_{t,i}^2 x_{t,i}^{0.5}+ x_{t,j} \sum_{i\neq j} g^2_{t,j} x_{t,i}^{1.5} \right]                                                         \\
		 & \leq 24 L_\infty^2 \E\left[ \sum_{i \neq j} \sqrt{x_{t,i}} \right],
	\end{align*}
	where in the inequality we used $\left(\frac{1}{x_{t,j}}-1\right)^2 x_{t,j}^{1.5} x_{t,j} = x_{t,j}^{0.5} (1-x_{t,j})^2\leq 1-x_{t,j} \leq \sum_{i\neq j} \sqrt{x_{t,i}}$ and $x_{t,j} \sum_{i\neq j} g^2_{t,j} x_{t,i}^{1.5}\leq L_\infty^2 \sum_{i\neq j} x_{t,i}^{0.5}$.

	Putting everything together, setting $\bu=\be_k$, for all $k=1, \dots, d$, we have
	\begin{align*}
		&\E\left[\sum_{t=1}^T \langle \bg_t, \bx_t-\be_k\rangle\right]\\
		 & \quad\leq -\lambda_{T} -\lambda_1 \psi(\bx_1) + \sum_{t=1}^T \frac{48 L_\infty^2}{\lambda_t}\sum_{i\neq j} \E[\sqrt{x_{t,i}}]
		+ \sum_{t=1}^{T-1} \left(\lambda_{t+1} -\lambda_t\right)(-\E[\psi( \bx_{t+1})])                                            \\
		 & \quad=\lambda_1 (-\psi(\bx_1)-1) +  \sum_{t=1}^T \frac{48 L_\infty^2}{\lambda_t}\sum_{i\neq j} \E[\sqrt{x_{t,i}}]
		+ \sum_{t=1}^{T-1} \left(\lambda_{t+1} -\lambda_t\right)(-\E[\psi( \bx_{t+1})]-1)                                          \\
		 & \quad\leq \lambda_1 \sum_{i\neq j} \sqrt{x_{1,i}} +  \sum_{t=1}^T \frac{48 L_\infty^2}{\lambda_t}\sum_{i\neq j} \E[\sqrt{x_{t,i}}] + \sum_{t=1}^{T-1} \left(\lambda_{t+1} -\lambda_t\right)\sum_{i\neq j} \E[\sqrt{x_{t+1,i}}]\\
		& \quad = \sum_{t=1}^T \frac{48 L_\infty^2}{\lambda_t}\sum_{i\neq j} \E[\sqrt{x_{t,i}}]
		+ \sum_{t=0}^{T-1} \left(\lambda_{t+1} -\lambda_t\right)\sum_{i\neq j} \E[\sqrt{x_{t+1,i}}],
	\end{align*}
	where in the first equality we have used the fact that $\sum_{t=1}^{T-1} (\lambda_{t+1}-\lambda_t)=\lambda_{T}-\lambda_1$, in the second inequality we used that $-1+\sum_{i=1}^d \sqrt{x_{t+1,i}}\leq \sum_{i\neq j} \sqrt{x_{t+1,i}}$ for any $j=1, \dots, d$, and in the last equality we set $\lambda_0=0$.
	Using the fact that $\sqrt{t}-\sqrt{t-1}\leq \frac{1}{\sqrt{t}}$ for $t=1, 2, \dots$, we have
	\begin{align}
		\E\left[\sum_{t=1}^T \langle\bg_t, \bx_t-\be_k\rangle\right]
		&\leq L_\infty \sum_{t=1}^{T} \left(4\sqrt{t}-4\sqrt{t-1}+\frac{12}{\sqrt{t}}\right) \sum_{i\neq j} \E[\sqrt{x_{t,i}}] \nonumber \\
		&\leq L_\infty \sum_{t=1}^{T} \frac{16}{\sqrt{t}} \sum_{i\neq j} \E[\sqrt{x_{t,i}}]~. \label{eq:proof_bobw_1}
	\end{align}
	Using the fact that $\sum_{i\neq j} \sqrt{x_{t,i}} \leq \sqrt{d-1}$ by Cauchy--Schwarz inequality\index{inequality!Cauchy--Schwarz} and $\sum_{t=1}^T \frac{1}{\sqrt{t}}\leq 2 \sqrt{T}$, we get the first stated bound.

	We now move to bound the pseudo-regret. Let $j$ be the index of the best arm.
	From Lemma~\ref{lemma:stoch_band_decomp} and the fact that the expected time we pull arm $i$ is $\sum_{t=1}^T \E[x_{t,i}]$, we have that $\E[\Regret_T(\be_j)]=\PRegret_T=\sum_{t=1}^T \sum_{i\neq j} \E[x_{t,i}] \Delta_i$.
	So, setting $k=j$, from \eqref{eq:proof_bobw_1} we have
	\begin{align*}
		\PRegret_T
		&= \E[\Regret_T(\be_j)]
		= 2\E[\Regret_T(\be_j)] -\sum_{t=1}^T \sum_{i\neq j} \E[x_{t,i}] \Delta_i\\
		&\leq  \E\left[\sum_{t=1}^T \sum_{i\neq j} \left( 32 L_\infty\frac{\sqrt{x_{t,i}}}{\sqrt{t}} - x_{t,i} \Delta_i\right) \right]
        \leq  L_\infty^2 \sum_{t=1}^T \sum_{i\neq j} \frac{256}{t \Delta_i}\\
		&\leq  256 L_\infty^2  \sum_{i\neq j} \frac{1+\ln T}{ \Delta_i},
	\end{align*}
	where in the first inequality we used $a\sqrt{x}-b x\leq \frac{a^2}{4b}$ for any $a,b>0$.
\end{proof}
\index{Tsallis-INF algorithm!FTRL version|)textbf}
\index{multi-armed bandit!best-of-both worlds|)}

\section{History Bits}

The nonstochastic bandit problem was first considered by \citet{Banos68} (see also \citet{Megiddo80}). The algorithm in Algorithm~\ref{alg:eg_bandit} is from \citet[Theorem 6.9]{Cesa-BianchiL06}.
The \ac{Exp3}\index{Exp3 algorithm} algorithm was proposed in \citet{AuerCBRS02} and it used a small exploration rate to achieve the same regret we proved. The observation that the exploration in \ac{Exp3} can be removed completely is from \citet{Stoltz05}.

\index{Tsallis-INF algorithm!OMD version|(}
The \ac{INF}\index{Implicitly Normalized Forecaster} algorithm was proposed by \citet{AudibertB09} and recast as an \ac{OMD} procedure in \citet{AudibertBL11}. The connection with the Tsallis entropy\index{entropy!Tsallis} was done in \citet{AbernethyLT15}.
\index{Tsallis-INF algorithm!OMD version|)}
\index{Tsallis-INF algorithm!FTRL version|(}
The specific proof presented here is new, and it builds on the proof by \citet{AbernethyLT15}, based on \ac{FTRL} (which they call Gradient-Based Prediction Algorithm).
\index{Tsallis-INF algorithm!FTRL version|)}

Stochastic multi-armed bandits were first proposed by \citet{Robbins52} and \citet{LaiR85}.
The \ac{ETC} algorithm goes back to \citet{Robbins52}, although Robbins proposed what is now called epoch-greedy~\citep{LangfordZ08}. For more history on \ac{ETC}, take a look at Chapter 6 in \citet{LattimoreS18}.
The proofs in Section~\ref{sec:etc} are from \citet{LattimoreS18} as well.

The use of confidence bounds and the idea of optimism first appeared in the work by \citet{LaiR85}. The first version
of \ac{UCB}\index{Upper Confidence Bound algorithm} is by \citet{Lai87}. The version of \ac{UCB} I presented is by \citet{AuerCBF02} under the name UCB1. Note that, rather than considering 1-subgaussian environments, \citet{AuerCBF02} considers bandits where the rewards are confined to the $[0, 1]$ interval.
The proof of Theorem~\ref{thm:ucb} is a minor variation of that of Theorem 2.1 in \citet{BubeckCB12}, which also popularized the subgaussian setup. Theorem~\ref{thm:ucb2} is from \citet{BubeckCB12}.

\index{Tsallis-INF algorithm!FTRL version|(}
Theorem~\ref{thm:bobw} is from \citet{ZimmertS19,ZimmertS21}, but here we use the proof from \citet{LeeO25} that does not use Fenchel conjugates and sacrifices constants for simpler calculations.
\index{Tsallis-INF algorithm!FTRL version|)}

\section{Exercises}



\begin{exer}
Design and analyse an \ac{FTRL} version of \ac{Exp3} with non-decreasing regularizers.
\end{exer}


\begin{exer}
\label{exercise:subgaussian}
Prove Lemma~\ref{lemma:subgauss_properties}.
\end{exer}

\begin{exer}
Prove a similar pseudo-regret bound to the one in Theorem~\ref{thm:ucb2} for an optimally tuned Explore-Then-Commit algorithm.
\end{exer}
\index{multi-armed bandit|)}

\acresetall

\chapter{Universal Portfolio Algorithms}
\label{ch:universal_portfolio}

We now describe another online learning problem: sequential investment in a market with $d\geq 2$ stocks. As in the rest of the book, we will assume that the stock market is adversarial and we will design no-regret online investment strategies.

\acresetall

\section{Setting of Portfolio Selection}
\index{portfolio selection problem|(}
In this setting, each day we allocate our wealth across the stocks, and at the end of the day we sell all our shares.
The behavior of the market is specified by arbitrarily chosen non-negative \textbf{market vectors} $\bw_1, \dots, \bw_T$\index{market vector|textbf}, each of them in $\R_{\geq 0}^d\setminus\{\boldsymbol{0}\}$. The coordinates of the market vectors\index{market vector} represent the ratio between the closing and opening prices for the stocks. An investment strategy is specified by a vector $\bx_t \in \Delta^{d-1}$, and its elements specify the fraction of the current wealth invested in each stock at time $t$. We will also assume that the wealth is infinitely divisible.
Hence, assuming an initial wealth of \$1, our wealth at the end of round $T$ is\index{wealth!in portfolio selection}
\[
\Wealth_T
:= \sum_{i=1}^d \Wealth_{T-1} w_{T,i} x_{T,i}
= \Wealth_{T-1} \langle \bw_T, \bx_T\rangle
= \prod_{t=1}^T \langle \bw_t, \bx_t\rangle~.
\]

To define a regret, we have to decide what the class of comparators is. Here, we compare with the \textbf{best constantly rebalanced portfolio}\index{constantly rebalanced portfolio|(textbf}. A constantly rebalanced portfolio follows the same strategy we use, but its allocation among the stocks is the same at each time step.
Denoting by $\bu$ the vector in the simplex of the constantly rebalanced portfolio, and setting the initial wealth to \$1, its wealth is
\[
\Wealth_T(\bu):=\prod_{t=1}^T \langle\bw_t,\bu\rangle~.
\]
\index{constantly rebalanced portfolio|)textbf}

Given the multiplicative nature of this game, the absolute difference between wealth values does not make much sense. Hence, we consider the ratio of the wealth of the best constantly rebalanced portfolio\index{constantly rebalanced portfolio} and that of the algorithm, or equivalently, the difference of the logarithms:
\begin{equation}
\label{eq:portfolio}
\Regret_T(\bu)
= \ln \Wealth_T(\bu) - \ln \Wealth_T
= \sum_{t=1}^T \ln \langle\bw_t, \bu\rangle - \sum_{t=1}^T \ln \langle \bw_t, \bx_t\rangle~.
\end{equation}
Rewritten in this way, it should be clear that this game is just an online convex optimization game where $\mathcal{V}=\Delta^{d-1}$ and the convex losses are $\ell_t: \R^d \to (-\infty, +\infty]$ defined as $\ell_t(\bx)= - \ln \langle \bw_t, \bx\rangle$ if $\langle \bw_t, \bx\rangle>0$ and $\bx \in \Delta^{d-1}$, and $\ell_t(\bx)=+\infty$ otherwise.

We will say that a portfolio algorithm is \textbf{universal}\index{portfolio algorithm!definition of universal|textbf} if the regret against any constantly rebalanced portfolio\index{constantly rebalanced portfolio} and with any sequence of market vectors\index{market vector} is sublinear in time, that is, if the algorithm is no-regret.

\index{constantly rebalanced portfolio|(}
\noindent\textbf{Why Constantly Rebalanced Portfolios?}
This class of strategies has advantages that are not immediately obvious.
First of all, despite the word ``constant'' this strategy is a very active one and very different from the buy-and-hold\index{buy-and-hold strategy} one that consists of buying a set of stocks in the first round and selling them at the end of the game.
Let's consider an example that clearly shows the advantage of the constantly rebalanced portfolio.
Consider only 2 stocks with a sequence of market vectors equal to $(1, \frac12), (1, 2), (1, \frac12), (1, 2), \dots$.\index{market vector}
After any even number of rounds, the value of both stocks returns to the starting price, so a buy-and-hold strategy yields zero growth.
Instead, the best constantly rebalanced portfolio here is $\bu = (\frac12, \frac12)$ and it gives an exponentially increasing wealth of $(\frac{9}{8})^{t/2}\approx 1.06^t$ after $t$ days. This means that we would more than double our wealth every 12 days!

Another motivation to compare with constantly rebalanced portfolios is that it can be shown that in the case where the market vectors\index{market vector} are i.i.d. from some fixed (unknown) distribution, then the constantly rebalanced portfolio is asymptotically optimal.
\index{constantly rebalanced portfolio|)}
\index{portfolio selection problem|)}

\section{Portfolio Selection with the Exponentiated Gradient Algorithm}
\label{sec:universal_portfolio_eg}

\index{Exponentiated Gradient algorithm!OMD version|(}
Given that the problem in \eqref{eq:portfolio} is an online convex optimization game, we could think of using any of the algorithms we saw so far. In particular, given that the feasible set is the probability simplex\index{probability simplex} and that the losses are convex, we can try to use \ac{EG}.
However, this problem turns out to be particularly challenging. In fact, the loss functions are not Lipschitz, so their gradients are unbounded.

In fact, using the upper bound on the regret of \ac{EG}, we have
\[
\sum_{t=1}^T \ln \langle\bw_t, \bu\rangle - \sum_{t=1}^T \ln \langle \bw_t, \bx_t\rangle
\leq \frac{\ln d}{\eta} + \frac{\eta}{2}\sum_{t=1}^T \|\bg_t\|_\infty^2
= \frac{\ln d}{\eta} + \frac{\eta}{2}\sum_{t=1}^T \frac{\|\bw_t\|_\infty^2}{(\langle \bx_t, \bw_t\rangle)^2},
\]
where $\bg_t =\nabla \ell_t(\bx_t)$.
Unfortunately, the terms in the sum on the r.h.s. of the inequality are potentially unbounded. One way to avoid this issue is to assume that the market gains are in the range $[c,C]$, where $0<c \leq C<+\infty$. In this way, we have
$\frac{\|\bw_t\|_\infty^2}{(\langle \bx_t, \bw_t\rangle)^2} \leq \frac{C^2}{c^2}$. Setting the learning rate to $\eta=\frac{\sqrt{2 \ln d}}{\frac{C}{c}\sqrt{T}}$, we obtain a regret upper bound of
\begin{equation}
\label{eq:universal_portfolio_eg_regret}
\frac{C}{c}\sqrt{2 T\ln d }~.
\end{equation}
While the average regret vanishes, we had to use the additional assumption on the market gains, so the resulting algorithm is not universal, according to our definition above. In the next section, we will see that we can remove the constraint on the market gains and also achieve a much better dependence in $T$.

\section{Universal Portfolio Selection with $F$-Weighted Portfolio Algorithms}
\index{Universal Portfolio algorithm|(textbf}
\begin{algorithm}[t]
\caption{$F$-Weighted Portfolio Selection}
\label{alg:universal_portfolio}
\begin{algorithmic}[1]
{
    \REQUIRE{$F$ a probability measure on $\Delta^{d-1}$}
    \STATE{$\Wealth_0=1$}
    \FOR{$t=1$ {\bfseries to} $T$}
    \STATE{Set $\bx_t = \frac{\int_{\Delta^{d-1}}  \! \bx \Wealth_{t-1}(\bx) \, F (\mathrm{d} \bx)}{\int_{\Delta^{d-1}}  \! \Wealth_{t-1}(\bx) \, F (\mathrm{d} \bx)}$}
    \STATE{Receive $\bw_t \in \R^d_{\geq 0}\setminus\{\boldsymbol{0}\}$}
    \STATE{$\Wealth_t = \Wealth_{t-1} \langle \bw_t, \bx_t\rangle$}
    \ENDFOR
}
\end{algorithmic}
\end{algorithm}

\index{Exponentiated Gradient algorithm!OMD version|)}

To solve the problem of the unbounded gradients, we will consider a different strategy. In particular, we will use the \textbf{$F$-weighted portfolio algorithms}\index{portfolio algorithm!$F$-weighted|textbf}, that predict at each step with
\[
\bx_t = \frac{\int_{\Delta^{d-1}}  \! \bx \Wealth_{t-1}(\bx) \, F (\mathrm{d} \bx)}{\int_{\Delta^{d-1}}  \! \Wealth_{t-1}(\bx) \, F (\mathrm{d} \bx)},
\]
where $F$ is a probability measure on $\Delta^{d-1}$.
To understand this formula, consider a generic constantly rebalanced portfolio\index{constantly rebalanced portfolio} allocation $\bx \in \Delta^{d-1}$ whose wealth after $t-1$ rounds is $\Wealth_{t-1}(\bx)$. Now, consider the probability distribution over all constantly rebalanced portfolios\index{constantly rebalanced portfolio} obtained by weighting the prior measure $F$ proportionally to $\Wealth_{t-1}(\bx)$. We have that $\bx_t$ defined above is nothing else than the average portfolio according to this distribution.
Hence, we give more importance in the average to successful portfolios that are also weighted highly by the distribution $F$. The resulting strategy is summarized in Algorithm~\ref{alg:universal_portfolio}.

The above strategy also gives us an identity for the wealth of the algorithm:
\[
\Wealth_t=\int_{\Delta^{d-1}}  \! \Wealth_{t}(\bx) \, F (\mathrm{d} \bx),
\]
that is, \emph{the wealth of the $F$-weighted portfolio is the average wealth of constantly rebalanced portfolios\index{constantly rebalanced portfolio} according to the distribution $F$.}
Let's prove this expression by induction. At $t=0$ the claim is true, so let's assume it is true at $t-1$ and let's prove it at $t$. We have
\begin{align}
\Wealth_t
&=\Wealth_{t-1} \langle \bw_t, \bx_t \rangle \nonumber \\
&= \int_{\Delta^{d-1}}  \! \Wealth_{t-1}(\bx) \, F (\mathrm{d} \bx) \frac{\int_{\Delta^{d-1}}  \! \langle\bw_t, \bx\rangle \Wealth_{t-1}(\bx) \, F (\mathrm{d} \bx)}{\int_{\Delta^{d-1}}  \! \Wealth_{t-1}(\bx) \, F (\mathrm{d} \bx)}\nonumber \\
&= \int_{\Delta^{d-1}}  \! \Wealth_{t}(\bx) \, F (\mathrm{d} \bx)~. \label{eq:wealth_universal_portfolio}
\end{align}

We will now state regret upper bounds for the Universal Portfolio algorithm using two possible settings of $F$: the Dirichlet$(1/2,\dots, 1/2)$ distribution\index{random variable!Dirichlet}, whose density is
\[
f(\bx)=\frac{\Gamma(d/2)}{\pi^{d/2}}\prod_{i=1}^d x_i^{-1/2}, \quad \bx\in\Delta^{d-1},
\]
(where $\Gamma$ is the gamma function\index{gamma function}, see Appendix~\ref{sec:gamma}), and the uniform distribution.
\begin{theorem}
\label{thm:universal_portfolio}
Consider the $F$-weighted portfolio strategy and an arbitrary sequence of market gains $\bw_t \in \R^d_{\geq0}\setminus\{\boldsymbol{0}\}$ for $t=1, \dots, T$.
\begin{itemize}
\item If $F$ is equal to the Dirichlet$(1/2,\dots, 1/2)$ distribution, then we have
\begin{align*}
\Regret_T
&:= \max_{\bu \in \Delta^{d-1}} \ \ln \Wealth_T(\bu) - \ln \Wealth_T
\leq \ln \frac{\sqrt{\pi} \Gamma\left(T+\frac{d}{2}\right)}{\Gamma\left(\frac{d}{2}\right)\Gamma\left(T+\frac12\right)}\\
&\leq \frac{d-1}{2} \ln\left(\frac{T}{(d-1)/2}+1\right) + \frac{d}{2}-1 + \frac{1}{2}\ln \frac{\pi}{2}~.
\end{align*}
\item If $F$ is equal to the uniform distribution, then we have
\begin{align*}
\Regret_T
&:= \max_{\bu \in \Delta^{d-1}} \ \ln \Wealth_T(\bu) - \ln \Wealth_T
\leq \ln \binom{T+d-1}{d-1}\\
&\leq (d-1) \left(\ln \left(\frac{T}{d-1}+1\right)+1\right)~.
\end{align*}
\end{itemize}
\end{theorem}

Observe that both prediction strategies ensure that we never produce a vector where the loss is infinite, and allow us to show a vanishing average regret. Moreover, it can be shown that this first regret upper bound is optimal up to constant additive terms. Also, the proof shows that regret upper bounds are essentially tight in the case that in each round the market gains are vectors whose coordinates are all zero except for one coordinate equal to $1$.

To prove this theorem, we will need some tools from information theory.

\subsection{Information Theory Bits}

\index{method of types|(}
We will need some basic results from information theory and, in particular, upper bounds used in the \emph{method of types}.
\begin{definition}
Consider a sequence of $T$ symbols $\bs=(x_1, \dots, x_T)$, where each $x_t \in \{1, \dots, d\}$.
The \textbf{type of the sequence}\index{type of a sequence|textbf} $\bs$ is the vector $\bn$ of the fractions of times each symbol appears, i.e., $\bn=[N(1;\bs)/T, \dots, N(d;\bs)/T]$, where $N(j; \bs):=\sum_{t=1}^T \indevent{\{x_t=j\}}$.
\end{definition}
\index{method of types|)}

\begin{lemma}
\label{lemma:types}
Let $T\geq 1$, $\bx \in \R_{\geq0}^d$ such that $\sum_{i=1}^d x_i=T$, and $\bu \in \Delta^{d-1}$. Then, we have
\[
\prod_{i=1}^d u_i^{x_i}
\leq \prod_{i=1}^d \left(\frac{x_i}{T}\right)^{x_i},
\]
where $0^0:=1$.
\end{lemma}
\begin{proof}
If there exists $i$ such that $u_i^{x_i}=0$, the inequality is verified. So, in the following we assume that $u_i^{x_i}\neq0$ for $i=1, \dots,d$. So, if for a given $i$ we have $u_i=0$, we must have $x_i=0$ as well.
Hence, we have
\begin{align*}
\prod_{i=1}^d u_i^{x_i}
&= \prod_{i:u_i\neq 0} u_i^{x_i}
= \prod_{i:u_i\neq 0} \exp(x_i \ln u_i)
= \prod_{i:u_i\neq 0} \exp\left[T \left(\frac{x_i}{T} \ln \frac{x_i}{T} + \frac{x_i}{T} \ln \frac{u_i T}{x_i}\right)\right]\\
&= \exp\left[T (-\entropy(\bx/T) - \KL(\bx/T ; \bu))\right]
\leq \exp(- T \, \entropy(\bx/T) )
= \prod_{i=1}^d \left(\frac{x_i}{T}\right)^{x_i},
\end{align*}
where $0\ln0:=0$, $0\ln(0/0):=0$, $\entropy(\bx) := -\sum_{i=1}^d x_i \ln x_i$ is the Shannon entropy\index{entropy!Shannon}, and $\KL(\bx ; \by):=\sum_{i=1}^d x_i \ln \frac{x_i}{y_i}$ is the \ac{KL} divergence, and the inequality holds because the \ac{KL} divergence is non-negative\index{Kullback--Leibler divergence}.
\end{proof}


\index{method of types|(}
\begin{theorem}
\label{thm:size_types}
Consider an alphabet of $d$ symbols and consider the set of all sequences of $T$ symbols.
Let the set of all possible types of these sequences be $\mathcal{Q}\subset \Delta^{d-1}$.
Denote by $\Type_T(\bn)$ the set of all the sequences of length $T$ symbols with a type $\bn$, that is $\Type_T(\bn):=\{\bx \in \{1, \dots, d\}^T: \bx \text{ has type } \bn\}$. Then, for any $\bn \in \mathcal{Q}$, we have $|\Type_T(\bn)|\leq \prod_{i=1}^d n_i^{-n_i T}$, where $0^0:=1$.
\end{theorem}
\begin{proof}
We prove the inequality with a probabilistic argument.
Consider all the possible sequences $\by$ of length $T$ generated i.i.d. from a distribution over $d$ symbols equal to $\bn=[n_1, \dots, n_d]^\top \in \Delta^{d-1}$.
The sum of the probabilities of all these sequences is 1. Among all the sequences, there are the ones with type $\bn$. Hence, for any $\bn\in \mathcal{Q}$, we have
\[
1
\geq \sum_{\by:\by \text{ has type } \bn} \Pr\{\by\}
= \sum_{\by:\by \text{ has type } \bn} \prod_{i=1}^d n_i^{n_i T}
= |\Type_T(\bn)| \prod_{i=1}^d n_i^{n_i T},
\]
where $0^0:=1$, the first equality is due to the i.i.d. assumption and the second one holds because all the sequences of type $\bn$ have the same probability.
\end{proof}
\index{method of types|)}

\subsection{Proof of Theorem~\ref{thm:universal_portfolio}}

First of all, we need a technical lemma.
\begin{lemma}
\label{lemma:max_ratio}
Let $a_1, \dots, a_T\geq 0$ and $b_1, \dots, b_T \geq 0$. Then, we have
\[
\frac{\sum_{t=1}^T a_t}{\sum_{t=1}^T b_t}
\leq \max_{t=1,\dots,T} \ \frac{a_t}{b_t},
\]
where we define $a/0:=+\infty$ for any $a>0$ and $0/0:=0$.
\end{lemma}
\begin{proof}
Let $j^\star = \argmax_{t=1,\dots,T} \frac{a_t}{b_t}$.
If $a_{j^\star} = 0$, then the theorem is true because both sides of the inequality are 0. In the same way, if $b_{j^\star} = 0$, then the inequality is true because the r.h.s. is $+\infty$. So, we can consider $a_{j^\star} > 0$ and $b_{j^\star} > 0$.
Then, we have
\[
\frac{\sum_{t=1}^T a_t}{\sum_{t=1}^T b_t}
= \frac{a_{j^\star}\left(1+\sum_{t\neq j^\star} \frac{a_t}{a_{j^\star}}\right)}{b_{j^\star}\left(1+\sum_{t\neq j^\star} \frac{b_t}{b_{j^\star}}\right)}
\leq \frac{a_{j^\star}}{b_{j^\star}},
\]
because $\frac{a_{j^\star}}{b_{j^\star}}\geq \frac{a_t}{b_t}$ for all $t$, which implies $\frac{a_t}{a_{j^\star}} \leq \frac{b_t}{b_{j^\star}}$.
\end{proof}

Then, the following lemma is the key one.
\begin{lemma}
\label{lemma:master_lemma_universal_portfolio}
For any constantly rebalanced portfolio $\bu \in \Delta^{d-1}$ and any sequence of market gains $\bw_t\in \R^d_{\geq0}\setminus\{\boldsymbol{0}\}$ for $t=1,\dots,T$, an $F$-weighted portfolio guarantees
\[
\frac{\Wealth_T(\bu)}{\Wealth_T}
\leq \max_{\bj \in \{1, \dots, d\}^T} \frac{\prod_{t=1}^T u_{j_t}}{\int_{\Delta^{d-1}} \! \prod_{t=1}^T x_{j_t} F (\mathrm{d} \bx)}~.
\]
\end{lemma}
\begin{proof}
Observe that
\[
\Wealth_T(\bu)
= \prod_{t=1}^T \langle \bw_t, \bu \rangle
= \prod_{t=1}^T \sum_{i=1}^d w_{t,i} u_i
= \sum_{\bj \in \{1, \dots, d\}^T} \prod_{t=1}^T u_{j_t} w_{t,j_t},
\]
where in the last equality we have expressed the products as a sum over all the possible combinations of the terms in the sum.
In the same way, using \eqref{eq:wealth_universal_portfolio}, we have
\begin{align*}
\Wealth_T
&= \int_{\Delta^{d-1}}  \! \Wealth_{T}(\bx) \, F (\mathrm{d} \bx)
= \int_{\Delta^{d-1}}  \! \prod_{t=1}^T \langle \bw_t, \bx \rangle \, F (\mathrm{d} \bx)\\
&= \sum_{\bj \in \{1, \dots, d\}^T} \int_{\Delta^{d-1}}  \! \prod_{t=1}^T x_{j_t} w_{t,j_t} \, F (\mathrm{d} \bx)~.
\end{align*}
We now use Lemma~\ref{lemma:max_ratio}, to obtain
\begin{align*}
\frac{\Wealth_T(\bu)}{\Wealth_T}
&= \frac{\sum_{\bj \in \{1, \dots, d\}^T} \prod_{t=1}^T u_{j_t} w_{t,j_t}}{\sum_{\bj \in \{1, \dots, d\}^T} \int_{\Delta^{d-1}}  \! \prod_{t=1}^T x_{j_t} w_{t,j_t} \, F (\mathrm{d} \bx)}\\
&= \frac{\sum_{\bj \in \{1, \dots, d\}^T, \prod_{t=1}^T w_{t,j_t}>0} \prod_{t=1}^T u_{j_t} w_{t,j_t}}{\sum_{\bj \in \{1, \dots, d\}^T, \prod_{t=1}^T w_{t,j_t}>0} \int_{\Delta^{d-1}}  \! \prod_{t=1}^T x_{j_t} w_{t,j_t} \, F (\mathrm{d} \bx)}\\
&\leq \max_{\bj \in \{1, \dots, d\}^T, \prod_{t=1}^T w_{t,j_t}>0} \frac{\prod_{t=1}^T u_{j_t} w_{t,j_t}}{\int_{\Delta^{d-1}}  \! \prod_{t=1}^T x_{j_t} w_{t,j_t} \, F (\mathrm{d} \bx)}\\
&= \max_{\bj \in \{1, \dots, d\}^T, \prod_{t=1}^T w_{t,j_t}>0} \frac{\prod_{t=1}^T u_{j_t}}{\int_{\Delta^{d-1}} \! \prod_{t=1}^T x_{j_t} \, F (\mathrm{d} \bx)}\\
&\leq \max_{\bj \in \{1, \dots, d\}^T} \frac{\prod_{t=1}^T u_{j_t}}{\int_{\Delta^{d-1}} \! \prod_{t=1}^T x_{j_t} \, F (\mathrm{d} \bx)}~. \qedhere
\end{align*}
\end{proof}

This Lemma has a very easy and powerful interpretation: for a given $\bj \in \{1, \dots, d\}^T$, we have that $\prod_{t=1}^T u_{j_t}$ is nothing else than the wealth of the constantly rebalanced portfolio\index{constantly rebalanced portfolio} $\bu$ when the market gains for $t=1, \dots, T$ are
\[
\bw_t=[0, \dots, 0, \underbrace{1}_\text{position $j_t$}, 0, \dots, 0] \in \R^d~.
\]
Similarly, for the wealth of the algorithm. This means that the lemma shows that the worst case regret ratio is achieved for market gains where only one coordinate is $1$ and all the others are zero, so-called \emph{Kelly market vectors}\index{market vector!Kelly} or ``horse race'' vectors. So, this lemma allows us to simplify the analysis by completely removing the market gains from the upper bound. As such, the upper bound is worst-case with respect to the sequence of market gains. Yet, given that with our choices of $F$ distributions the final regret is logarithmic, it also shows that the market gains can only have a limited influence on the regret of the algorithm.

We can now prove the regret guarantees for the universal portfolio algorithm with two different distributions $F$.
\begin{proof}[Proof of Theorem~\ref{thm:universal_portfolio}]
Our starting point is Lemma~\ref{lemma:master_lemma_universal_portfolio}.
Using Lemma~\ref{lemma:types}, for any $\bj \in \{1, \dots, d\}^T$  and $\bu \in \Delta^{d-1}$, we have
\[
\prod_{t=1}^T u_{j_t}
=\prod_{i=1}^d u_{i}^{N(i;\bj)}
\leq \prod_{i=1}^d \left(\frac{N(i;\bj)}{T}\right)^{N(i;\bj)},
\]
where $N(i;\bj)$ is $\sum_{t=1}^T \indevent{\{j_t = i\}}$.

Now, let's first consider the Dirichlet distribution.
In this case, it can be shown that $\int_{\Delta^{d-1}} \! \prod_{t=1}^T x_{j_t} F (\mathrm{d} \bx)$ has a closed-form expression (see Problem~\ref{exercise:universal_portfolio_integral1}):
\begin{equation}
\label{eq:proof_up_to_be_proven_1}
\int_{\Delta^{d-1}} \! \prod_{t=1}^T x_{j_t} F (\mathrm{d} \bx)
= \frac{\pi^{-d/2} \Gamma\left(\frac{d}{2}\right)}{\Gamma\left(T+\frac{d}{2}\right)}\prod_{i=1}^d \Gamma\left(N(i;\bj)+\frac12\right)~.
\end{equation}
We now use Lemma~\ref{lemma:max_universal_portfolio} to maximize over the scaled simplex, to have
\begin{align*}
\max_{\bj \in \{1, \dots, d\}^T} \frac{\prod_{t=1}^T u_{j_t}}{\int_{\Delta^{d-1}} \prod_{t=1}^T x_{j_t} \, F (\mathrm{d} \bx)}
&\leq \max_{a_1, \dots, a_d\geq 0: \sum_{i=1}^d a_i=T} \frac{\pi^{d/2}\Gamma\left(T+\frac{d}{2}\right)}{\Gamma\left(\frac{d}{2}\right) T^T} \prod_{i=1}^d \frac{a_i^{a_i}}{\Gamma\left(a_i+\frac12\right)}\\
&= \frac{\sqrt{\pi}\Gamma\left(T+\frac{d}{2}\right)}{\Gamma\left(\frac{d}{2}\right)\Gamma\left(T+\frac12\right)}~.
\end{align*}
The second inequality in the stated bound is due to Lemma~\ref{lemma:bound_regret_gamma}.

Analogously, when $F$ is the uniform distribution, it can be shown (see Problem~\ref{exercise:universal_portfolio_integral2}) that
\begin{equation}
\label{eq:proof_up_to_be_proven_2}
\int_{\Delta^{d-1}} \prod_{t=1}^T x_{j_t} \, F (\mathrm{d} \bx)
= \frac{1}{\binom{T+d-1}{d-1} |\Type_T((N(1;\bj)/T, \dots, N(d;\bj)/T))|},
\end{equation}
where $|\Type_T(a_1, \dots, a_d)|$ is the number of sequences of type $(a_1, \dots, a_d)$. From Theorem~\ref{thm:size_types}, we have that $|\Type_T(N(1;\bj)/T, \dots, N(d;\bj)/T)|\leq \prod_{i=1}^d \left(\frac{N(i;\bj)}{T}\right)^{-N(i;\bj)}$.
Hence, when $F$ is the uniform distribution, we have
\[
\max_{\bj \in \{1, \dots, d\}^T} \frac{\prod_{t=1}^T u_{j_t}}{\int_{\Delta^{d-1}} \prod_{t=1}^T x_{j_t} \, F (\mathrm{d} \bx)}
\leq \binom{T+d-1}{d-1}~.
\]
The second inequality in the stated bound is due to the elementary inequality $\binom{n}{k}\leq \left(\frac{e n}{k}\right)^k$.
\end{proof}

\section{Portfolio Selection with the Online Newton Step Algorithm}

While the $F$-weighted portfolio algorithm can have optimal regret, its computational complexity is very high, both in $T$ and $d$. In fact, we have to calculate an integral over the simplex in each step. While it is possible to have a closed-form update with complexity $\mathcal{O}(t)$~\citep{CoverO96}, it is also natural to look for alternatives with a lower computational complexity.

One might be tempted to use any other \ac{OCO} algorithm to solve the portfolio selection problem. However, as we have seen in Section~\ref{sec:universal_portfolio_eg}, the gradients in general can be unbounded. Hence, we again consider an easier setting: we assume that all market gains are lower bounded by $c>0$ and upper bounded by $C<\infty$. We stress that Theorem~\ref{thm:universal_portfolio} does not require this assumption, but we use it to derive a simpler algorithm, still with a logarithmic regret.

Now, we might be tempted to think that the losses are strongly convex, because Theorem~\ref{thm:universal_portfolio} shows a logarithmic regret. This is clearly false: the Hessian of each function has rank 1, and unfortunately, in the worst case, the sum of $T$ losses can also have a singular Hessian. Even assuming the Hessian of the sum of the losses is not singular, we would need its minimal eigenvalue to grow over time. Again, this does not happen in a worst-case scenario. Hence, we need another way.

Observe that the losses in portfolio selection are 1-exp-concave\index{function!exp-concave}. In fact, $\exp(-\ell_t(\bx))=\langle \bw_t, \bx\rangle$ is concave.
Hence, if the gradients and the domain are bounded, we can apply the \ac{ONS} algorithm\index{Online Newton Step algorithm|(}.
In particular, following Example~\ref{ex:exp_concave}, we need to find $\beta\leq \frac{1}{2}$ such that $|\beta \langle \nabla \ell_t(\bx),\bx-\bu\rangle|\leq \frac{1}{2}$ for any $\bx,\bu \in \Delta^{d-1}$. For the gradient, we have $\nabla \ell_t(\bx) = -\frac{\bw_t}{\langle \bw_t, \bx\rangle}$, hence $\|\nabla \ell_t(\bx)\|_\infty\leq \frac{C}{c}$.
Moreover, $\|\bx-\bu\|_1\leq 2$ for any $\bx,\bu \in \Delta^{d-1}$. Hence, we can set $\beta = \frac{c}{4 C}$.
The resulting algorithm is in Algorithm~\ref{alg:ons_portfolio}.

\begin{algorithm}[t]
\caption{Online Newton Step for Portfolio Selection}
\label{alg:ons_portfolio}
\begin{algorithmic}[1]
{
    \REQUIRE{$\lambda>0$, $C<\infty$, $c>0$}
    \STATE{Set $\beta=\frac{c}{4C}$}
    \FOR{$t=1$ {\bfseries to} $T$}
    \STATE{Set $\bx_t = \argmin_{\bx \in \Delta^{d-1}} \ \sum_{i=1}^{t-1} \langle \bg_i, \bx\rangle + \frac{\lambda}{2}\|\bx\|_2^2 + \frac{\beta}{2} \sum_{i=1}^{t-1} (\langle\bg_i,\bx-\bx_i\rangle)^2$}
    \STATE{Receive $\bw_t \in [c,C]^d$}
    \STATE{Pay the loss $\ell_t(\bx_t)$, where $\ell_t(\bx)=-\ln \langle \bw_t, \bx\rangle$}
    \STATE{Set $\bg_t = -\frac{\bw_t}{\langle\bw_t, \bx_t\rangle}$}
    \ENDFOR
}
\end{algorithmic}
\end{algorithm}

As shown in Problem~\ref{exercise:ons_two_steps}, the update can also be divided into two steps: first minimizing over the entire space and then a Bregman projection.

As we have seen in Section~\ref{sec:ons}, the regret upper bound of the \ac{ONS} is
\begin{align*}
\Regret_T (\bu)
&= \ln \Wealth_T(\bu) - \ln \Wealth_T
= \sum_{t=1}^T \ln \langle\bw_t, \bu\rangle - \sum_{t=1}^T \ln \langle \bw_t, \bx_t\rangle\\
&\leq \frac{\lambda}{2}\|\bu\|^2_2 + \frac{d}{2\beta} \ln \left(1+ \frac{\beta T L^2}{d \lambda}\right), \quad \forall \bu \in \Delta^{d-1},
\end{align*}
where $\|\nabla \ell_t(\bx_t)\|_2\leq L$. In our setting, this becomes
\begin{align*}
\Regret_T(\bu)
&= \ln \Wealth_T(\bu) - \ln \Wealth_T
= \sum_{t=1}^T \ln \langle\bw_t, \bu\rangle - \sum_{t=1}^T \ln \langle \bw_t, \bx_t\rangle\\
&\leq \frac{\lambda}{2} + \frac{2 d C}{c} \ln \left(1+ \frac{C T}{4 c\lambda}\right), \quad \forall \bu \in \Delta^{d-1},
\end{align*}
where we used the fact that $\|\nabla \ell_t(\bx_t)\|_2\leq \frac{\sqrt{d} C}{c}$.
\index{Online Newton Step algorithm|)}

\section{History Bits}

The distributional approach to betting and gambling was pioneered by \citet{Kelly56}. This approach assumes that the market gains are i.i.d. from a (known) distribution.
The $F$-weighted portfolios were proposed by \citet{Cover91}.
\citet{CoverO96} proposed the first portfolio algorithm with a minimax regret, without assumptions on the market gains, improving over the result in \citet{Cover91}. The proof presented here follows the one in \citet{CoverO96}, while Lemma~\ref{lemma:types} and Theorem~\ref{thm:size_types} are adapted from the proofs of \citet[Theorem 11.1.2 and Theorem 11.1.3]{CoverT06}, respectively.
\citet{VovkW98} proved that $F$-weighted portfolio algorithms are instantiations of the Aggregating Algorithm~\citep{Vovk90}, which we will see in the next chapter.

The use of \ac{EG} for portfolio selection was proposed by \citet{HelmboldSSW98}.
Stochastic approximations to the $F$-weighted portfolio update were proposed by \citet{BlumK97,BlumK99,KalaiV02}. \citet{BlumK97,BlumK99} also present a simpler proof for the uniform distribution case.
The use of the \ac{ONS} algorithm for portfolio selection was proposed by \citet{HazanKKA06,HazanAK07}\index{Online Newton Step algorithm}.
There is also a line of research that aims at designing algorithms on the efficiency-regret Pareto frontier~\citep[see, e.g.,][and references therein]{TsaiCL23}.

\section{Exercises}

\begin{exer}
Assume that the market gains $w_{t,i}$ are bounded in $[c,C]$.
Find a variant of \ac{EG} that does not need to know $c$ and $C$, and it achieves up to constants the same guarantee in \eqref{eq:universal_portfolio_eg_regret}.
\end{exer}

\begin{exer}
\label{exercise:universal_portfolio_integral1}
Prove \eqref{eq:proof_up_to_be_proven_1}. Hint: use the multivariate beta integral.
\end{exer}

\begin{exer}
\label{exercise:universal_portfolio_integral2}
Prove \eqref{eq:proof_up_to_be_proven_2}. Hint: use the identity $|\Type_T((N(1;\bj)/T,\dots,N(d;\bj)/T))|=\frac{T!}{\prod_{i=1}^d N(i;\bj)!}$.
\end{exer}

\begin{exer}
Let $F$ be the Dirichlet$(\alpha,\dots,\alpha)$ distribution with $\alpha>0$. Assume that the market vectors are Kelly market vectors, that is, $\bw_t=\be_{j_t}$ for some $j_t\in\{1,\dots,d\}$. Show that the $F$-weighted portfolio algorithm predicts
\[
x_{t,i}
= \frac{N_{t-1}(i)+\alpha}{t-1+d\alpha}, \quad i=1,\dots,d,
\]
where $N_{t-1}(i):=\sum_{s=1}^{t-1}\indevent{\{j_s=i\}}$.
Then, prove that its wealth after $T$ rounds is
\[
\Wealth_T
=
\frac{\Gamma(d\alpha)}{\Gamma(T+d\alpha)}
\prod_{i=1}^d \frac{\Gamma(N_T(i)+\alpha)}{\Gamma(\alpha)}~.
\]
\end{exer}

\index{Universal Portfolio algorithm|)textbf}

\acresetall

\chapter{Weighted Average Algorithm and Aggregating Algorithm}
\label{ch:waa}

In this chapter, we will introduce the \ac{WAA} and the \ac{AA}. We will view both of them as instantiations of the \ac{FTRL} algorithm on the space of distributions.

\acresetall

\section{Follow-the-Regularized-Leader with Distributions}

Our analysis of the \ac{WAA} and \ac{AA} will be based on a generalization of the \ac{FTRL} algorithm with the entropic regularizer, to work with distributions with infinite support, either countable or continuous.

Let $\mathcal X\subseteq\R^d$ be a measurable set, and let $m$ be a reference measure on $\mathcal{X}$, used whenever we write densities with respect to a background measure.
In the finite case, it is natural to take the counting measure $m(i)=1$ for $i=1, \dots, k$; in the continuous case, the natural choice of $m$ is the Lebesgue measure; if the relevant set is not full-dimensional, one instead uses the Lebesgue measure on its affine hull.

As usual, we will denote by $\mathcal{V} \subseteq \mathcal{X}$ the feasible set of the problem.
Also, we will make a subtle but important distinction between the support set, $\mathcal{X}$, and a smaller feasible set $\mathcal{V}$. This allows the flexibility to use priors defined on larger sets (like a Gaussian on $\R^d$), while still ensuring that the predictions $\bx_t$ adhere to the problem's constraints $\bx_t \in \mathcal{V}$.
So, for a set $\mathcal{V}$, define \textbf{$\Delta(\mathcal{V})$ as the set of all probability distributions with support on $\mathcal{V}$}, and \textbf{$\Delta_\mu(\mathcal{V})$ as the set of all probability distributions supported on $\mathcal{X}$ whose expectation exists and is in $\mathcal{V}$}.

Throughout this chapter, losses $\ell_t:\R^d\to(-\infty,+\infty]$ are measurable and may take the value $+\infty$ on $\mathcal X$. Allowing $+\infty$ lets us treat barriers or log-loss-type examples on closed domains without changing the feasible set.

For some set of distributions $\mathcal{P}$, in each round we output $P_t \in \mathcal{P}$, then we receive the loss functions $f_t(P):=\E_{\bx \sim P}[\ell_t(\bx)]$.
All expectations are understood in the extended-real sense\index{function!extended-real-valued}, and the algorithms are assumed to select distributions for which the displayed quantities are well defined.

To simplify the notation, we introduce the \textbf{duality pairing}\index{duality pairing|textbf}
\[
\langle \ell, P \rangle
:= \int \!\ell(\bx) P(\mathrm{d} \bx)~.
\]
When the integrals are finite, the duality pairing is bilinear but not symmetric, so it should not be confused with the inner product. Yet, it reduces to the inner product in some cases, for example, when $\mathcal{X}$ is a discrete set.
With this notation, we can write $f_t(P)=\langle \ell_t, P\rangle$.
We also have that the functional derivative of $f_t$ is $\nabla_P f_t=\ell_t$.

The notation above illustrates the similarity between this formulation and the one we saw for \ac{FTRL}. In particular, one would expect the same kind of results to hold here.

Let's use \ac{FTRL} with regularizers $\psi_t:\mathcal{P}\to (-\infty,+\infty]$. Hence, at round $t$ we produce a probability distribution $P_t$ (assuming it exists) defined as
\begin{equation}
\label{eq:ftrl_distr_update}
P_t
:= \argmin_{P \in \mathcal{P}} \ \left(F_t(P):=\psi_t(P) + \left\langle\sum_{i=1}^{t-1} \ell_i,P\right\rangle\right)~.
\end{equation}

It is easy to see that the \ac{FTRL} equality in Lemma~\ref{lemma:ftrl_equality} works even in this setting. Indeed, we have that
\begin{align*}
\left\langle\sum_{t=1}^T - \ell_t,Q\right\rangle
&= \psi_{T+1}(Q) - \min_{M \in \mathcal{P}} \ \psi_{1}(M) + \sum_{t=1}^T [F_t(P_t) - F_{t+1}(P_{t+1}) ] \\
&\quad + F_{T+1}(P_{T+1}) - F_{T+1}(Q), \quad \forall Q \in \mathcal{P},
\end{align*}
where we have canceled the algorithm's loss term on both sides. Define $\bx_t:=\E_{\bx \sim P_t}[\bx]$ and add $\sum_{t=1}^T \ell_t( \bx_t)$ to both sides, to have
\begin{align}
\sum_{t=1}^T \ell_t( \bx_t) - \left\langle\sum_{t=1}^T \ell_t,Q\right\rangle
&= \psi_{T+1}(Q) - \min_{M \in \mathcal{P}} \ \psi_{1}(M) + F_{T+1}(P_{T+1}) - F_{T+1}(Q) \nonumber\\
&\quad +\sum_{t=1}^T [F_t(P_t) - F_{t+1}(P_{t+1}) + \ell_t( \bx_t)], \quad \forall Q \in \mathcal{P}. \label{eq:ftrlwaa_lemma}
\end{align}
The above equality will serve as the starting point for analyzing \ac{WAA} and \ac{AA}.

\section{The Weighted Average Algorithm}
\label{sec:waa}
\index{Weighted Average Algorithm|(textbf}

\begin{algorithm}[h]
\caption{Weighted Average Algorithm (WAA)}
\label{alg:waa}
\begin{algorithmic}[1]
{
    \REQUIRE{Non-empty closed set $\mathcal{V} \subseteq \mathcal{X} \subseteq \R^d$, $\lambda_1, \dots, \lambda_T>0$, $P_1 \in \Delta(\mathcal{X})$}
    \FOR{$t=1$ {\bfseries to} $T$}
    \STATE{Set $P_{t} \in \argmin_{P \in \Delta_\mu(\mathcal{V})} \  \lambda_t \KL(P;P_1) + \sum_{i=1}^{t-1} \E_{\bx \sim P} [\ell_i(\bx)]$}
    \STATE{Output $\bx_t=\E_{\bx \sim P_t}[\bx]$}
    \STATE{Receive $\ell_t:\R^d\to (-\infty,+\infty]$ and pay $\ell_t(\bx_t)$}
    \ENDFOR
}
\end{algorithmic}
\end{algorithm}

In this section, we introduce the \textbf{\acl{WAA}}, stated in Algorithm~\ref{alg:waa}. Essentially, \ac{WAA} predicts with a weighted average of predictions---hence the name of the algorithm---where the weights are proportional to the negative exponential of the cumulative losses of each predictor. We can see \ac{WAA} as a generalization of the \ac{EG} algorithm with infinite distributions.

We will use \ac{FTRL} over distributions from the previous section.
Let $P_1$ be an arbitrary probability distribution on $\mathcal{X}$, and set $\psi_t(P)=\lambda_t \KL(P;P_1)=\lambda_t \E_{\bx \sim P}[\ln(\frac{\mathrm{d} P}{\mathrm{d} P_1}(\bx))]$, where $\KL$ is the \ac{KL} divergence, with the convention $\KL(P;P_1)=+\infty$ if $P\not\ll P_1$. Also, we set $\mathcal{P}=\Delta_\mu(\mathcal{V})$.
With this choice, the prediction rule in \eqref{eq:ftrl_distr_update} becomes
\begin{align}
\tilde{P}_t(\mathrm{d} \bx) &= \frac{\exp(-\frac{1}{\lambda_t} \sum_{i=1}^{t-1}\ell_i(\bx)) P_1(\mathrm{d} \bx)}{\int \!\exp(-\frac{1}{\lambda_t} \sum_{i=1}^{t-1}\ell_i(\bx)) P_1(\mathrm{d} \bx)},\label{eq:ftrl_kl_update_1}\\
P_t &= \argmin_{P \in \Delta_\mu(\mathcal{V})} \ \KL(P;\tilde{P}_t)~. \label{eq:ftrl_kl_update_2}
\end{align}
For this algorithm, we can show the following guarantee.
\begin{lemma}
\label{lemma:master_lemma_waa}
Set $\psi_t(P)=\lambda_t \KL(P;P_1)$, where $P_1 \in \Delta(\mathcal{X})$, $\KL(P;P_1)=+\infty$ if $P\not\ll P_1$, and $0<\lambda_{t}\leq \lambda_{t+1}$ for $t=1,\dots, T-1$ and $\lambda_{T+1}=\lambda_T$. Define $P_t$ as in \eqref{eq:ftrl_distr_update}, where we set $\mathcal{P}=\Delta_\mu(\mathcal{V})$, and $\bx_t:=\E_{\bx \sim P_t}[\bx]$.
Then, for all $Q \in \Delta_\mu(\mathcal{V})$, we have that
\begin{align*}
\sum_{t=1}^T \ell_t( \bx_t)
\leq \E_{\bx \sim Q}\left[\sum_{t=1}^T \ell_t(\bx)\right] + \lambda_T \KL(Q;P_1)
+\sum_{t=1}^T \lambda_t\ln \E_{\bx \sim P_{t}} \left[e^{-\frac{\ell_t(\bx)-\ell_t(\bx_t)}{\lambda_t} }\right]~.
\end{align*}
\end{lemma}

Before proving this lemma, observe that the following holds (proof left as an exercise):
\begin{equation}
\label{eq:aa_alternate_term}
\inf_{P \in\Delta(\mathcal X)} \ \langle f, P\rangle + \lambda \KL(P;Q)
= - \lambda \ln \E_{\bx \sim Q}\left[e^{- \frac{1}{\lambda} f(\bx)}\right],
\end{equation}
for every $\lambda>0$, every $Q\in\Delta(\mathcal{X})$, and every measurable $f:\R^d \to (-\infty,+\infty]$ such that the expectation on the right-hand side is well defined, with the convention $\KL(P;Q)=+\infty$ if $P\not\ll Q$. This is called the Donsker--Varadhan variational formula\index{Donsker--Varadhan variational formula|textbf}.

\begin{proof}[Proof of Lemma~\ref{lemma:master_lemma_waa}]
Given that $\psi_t(P)=\lambda_t\KL(P;P_1)$, one can verify that the Bregman divergence $B_{\psi_t}(P;Q)$ is equal to $\lambda_t\KL(P;Q)$.

Assuming $\frac{1}{\lambda_t}$ non-increasing, reasoning as in the proof of Lemma~\ref{lemma:ftrl_local_norms}, we have
\begin{align*}
F_t(P_t) - F_{t+1}(P_{t+1})
&\leq - \langle\ell_t, P_{t+1}\rangle-B_{F_t}(P_{t+1};P_t)
= - \langle \ell_t,P_{t+1}\rangle-B_{\psi_t}(P_{t+1};P_t) \\
&\leq \sup_{P \in \Delta(\mathcal X)} - \langle\ell_t,P\rangle-B_{\psi_t}(P;P_t)
= \lambda_t\ln \E_{\bx \sim P_{t}} \left[e^{-\frac{\ell_t(\bx)}{\lambda_t}}\right],
\end{align*}
where in the first equality we used the fact that the terms with the losses are linear with respect to the distribution, and in the last equality we used \eqref{eq:aa_alternate_term}.
Observing that the term $F_{T+1}(P_{T+1}) - F_{T+1}(Q)$ in \eqref{eq:ftrlwaa_lemma} is non-positive for all $Q \in \Delta_{\mu}(\mathcal{V})$, we obtain the stated bound.
\end{proof}

\begin{remark}
\label{remark:prior_waa}
If $\mathcal{A}\subseteq \mathcal{X}$ is measurable and $P_1(\mathcal{A})=0$, then $P_t(\mathcal{A})=0$ for all $t$. Hence, if $P_1 \in \Delta(\mathcal{V})$, then $P_t \in \Delta(\mathcal{V})$ for all $t$, and the constraint $P \in \Delta_\mu(\mathcal{V})$ is automatically satisfied.
\end{remark}



We now focus on the case that the \ac{WAA} is used on exp-concave losses.
\begin{theorem}
\label{thm:waa}
In Algorithm~\ref{alg:waa}, assume that $\mathcal{X}$ is convex, and each $\ell_t:\R^d\to (-\infty,+\infty]$ is $\alpha$-exp-concave\index{function!exp-concave} on $\mathcal{X}$.
Also, assume $\frac{1}{\alpha}\leq \lambda_{t}\leq \lambda_{t+1}$ for all $t=1,\dots, T-1$, and $\lambda_{T+1}=\lambda_T$. Then, we have that
\begin{align*}
\sum_{t=1}^T \ell_t( \bx_t)
&\leq \E_{\bx \sim Q}\left[\sum_{t=1}^T \ell_t(\bx)\right] + \lambda_T \KL(Q;P_1), \quad \forall Q \in \Delta_\mu(\mathcal{V})~.
\end{align*}
Moreover, if in addition $P_1\in \Delta(\mathcal{V})$ and $\mathcal{V}$ is convex and bounded, then we also have
\[
\sum_{t=1}^T \ell_t( \bx_t)
\leq -\lambda_T \ln \E_{\bx \sim P_1}\left[ e^{-\frac{1}{\lambda_T} \sum_{t=1}^T \ell_t(\bx)}\right],
\]
where $e^{-\infty}:=0$.
\end{theorem}
\begin{proof}
Given that $\ell_t$ is $\alpha$-exp-concave, using Jensen's inequality\index{inequality!Jensen's} (Theorem~\ref{thm:jensen}), we have
\begin{align*}
\lambda_t\ln \E_{\bx \sim P_t} \left[e^{-\frac{\ell_t(\bx)}{\lambda_t} }\right] + \ell_t\left(\bx_t\right)
=\lambda_t\ln \E_{\bx \sim P_t} \left[e^{-\frac{\ell_t(\bx)}{\lambda_t} }\right] + \ell_t\left(\E_{\bx \sim P_t}[\bx]\right)
\leq0,
\end{align*}
for all $\frac{1}{\lambda_t} \leq \alpha$. Using this inequality in Lemma~\ref{lemma:master_lemma_waa}, we have the first result.

For the second bound, since $P_1\in\Delta(\mathcal{V})$ and $\mathcal{V}$ is convex and bounded, every distribution $Q$ with $\KL(Q;P_1)<\infty$ is supported on $\mathcal{V}$, and therefore belongs to $\Delta_\mu(\mathcal{V})$.
Hence, by the first part of the theorem, for every such $Q$,
\[
\sum_{t=1}^T \ell_t(\bx_t)
\leq \E_{\bx\sim Q}\left[\sum_{t=1}^T \ell_t(\bx)\right] + \lambda_T \KL(Q;P_1)~.
\]
Taking the infimum over all $Q\in\Delta(\mathcal{X})$, we obtain
\[
\sum_{t=1}^T \ell_t(\bx_t)
\leq \inf_{Q\in\Delta(\mathcal{X})}
\left\{\E_{\bx\sim Q}\left[\sum_{t=1}^T \ell_t(\bx)\right] + \lambda_T \KL(Q;P_1)\right\}~.
\]
Applying \eqref{eq:aa_alternate_term} with $f=\sum_{t=1}^T \ell_t$, $\lambda=\lambda_T$, and $Q=P_1$, gives us the stated bound.
\end{proof}

So, the \ac{WAA} algorithm with $\lambda_t$ constant over time has the surprising property of having a \emph{constant} regret on exp-concave\index{function!exp-concave} losses with respect to a stochastic competitor.

Unfortunately, the \ac{WAA} update does not have a closed form in general. In fact, it requires the numerical evaluation of the expectation. However, if the distribution is supported on a finite discrete set, we can always calculate the prediction.

\begin{example}
As a practical example, consider $k$ experts $\bu_1,\dots,\bu_k\in\R^d$ and set $\mathcal{X}=\mathcal{V}=\conv\{\bu_1,\dots,\bu_k\}$.
Let the prior $P_1$ be supported on the experts, with $P_1(\{\bu_i\})=P_{1,i}>0$.
Then, \ac{WAA} maintains weights $P_{t,i}$ on the experts and predicts
$\bx_t=\sum_{i=1}^k P_{t,i}\bu_i \in \mathcal{V}$, and guarantees
\[
\sum_{t=1}^T \ell_t(\bx_t)
\leq \min_{i=1,\dots,k} \left(\sum_{t=1}^T \ell_t(\bu_i) + \lambda_T \ln \frac{1}{P_{1,i}} \right)~.
\]
\end{example}

We have proved an upper bound on the regret of \ac{WAA} that uses a randomized comparator, that is, $\E_{\bx \sim Q}\left[\sum_{t=1}^T \ell_t(\bx)\right]$. However, sometimes one would like to prove an upper bound that depends on a deterministic one. Hence, now we show how to link the performance of the stochastic comparator to the performance of a deterministic one.
\begin{theorem}
\label{thm:waa_bounded}
Assume that $\mathcal{V}\subseteq \R^d$ is a convex, closed, bounded set, full-dimensional in $\R^d$, and let $P_1$ be the uniform distribution on it.
Assume that the losses $\ell_t$ are $\alpha$-exp-concave\index{function!exp-concave} on $\mathcal{V}$, and set $\frac{1}{\lambda_t}=\alpha$. Then, Algorithm~\ref{alg:waa} satisfies
\begin{equation}
\label{eq:waa_uniform}
\sum_{t=1}^T (\ell_t(\bx_t) - \ell_t(\bu))
\leq \frac{d}{\alpha} \left[1+\ln\left(\frac{T}{d}+1\right)\right], \quad \forall \bu \in \mathcal{V}~.
\end{equation}
\end{theorem}
\begin{proof}
Fix $\bu \in \mathcal{V}$, and define $\mathcal{V}'=\{\bx \in \R^d: \bx=\frac{T/d}{T/d+1}\bu + \frac{1}{T/d+1}\bw, \bw \in \mathcal{V}\} \subseteq \mathcal{V}$.
Choose $Q$ to be the restriction of $P_1$ to $\mathcal{V}'$, renormalized by $1/\beta$, where $\beta=\int_{\mathcal{V}'} P_1(\mathrm{d} \bx)$.

By the exp-concavity of $\ell_t$, for any $\bx \in \mathcal{V}'$ there exists $\bw \in \mathcal{V}$ such that
\begin{align*}
e^{-\alpha \ell_t(\bx)}
\geq \frac{T/d}{T/d+1}e^{-\alpha \ell_t(\bu)} +\frac{1}{T/d+1}e^{-\alpha \ell_t(\bw)}
\geq \frac{T/d}{T/d+1}e^{-\alpha \ell_t(\bu)}~.
\end{align*}
Hence, we have for any $\bx \in \mathcal{V}'$
\[
e^{-\alpha \sum_{t=1}^T \ell_t(\bx)}
\geq \left(\frac{T/d}{T/d+1}\right)^T e^{-\alpha \sum_{t=1}^T \ell_t(\bu)}
\geq \frac{e^{-\alpha \sum_{t=1}^T \ell_t(\bu)}}{e^d} ,
\]
which implies
\[
\sum_{t=1}^T \E_{\bx \sim Q} [\ell_t(\bx)]
\leq \frac{d}{\alpha}+ \sum_{t=1}^T \ell_t(\bu)~.
\]
Moreover, the \ac{KL} term becomes
\[
\KL(Q;P_1)
= \int_{\mathcal{V}} \ln\frac{Q(\mathrm{d} \bx)}{P_1(\mathrm{d} \bx)} Q(\mathrm{d} \bx)
= -\ln \beta
= -\ln \int_{\mathcal{V}'} P_1(\mathrm{d} \bx)~.
\]
Given that $P_1$ is uniform on $\mathcal{V}$, we have that
\[
\int_{\mathcal{V}'} P_1(\mathrm{d} \bx)
= \frac{Vol(\mathcal{V}')}{Vol(\mathcal{V})}
= \frac{1}{(T/d+1)^d},
\]
because $\mathcal{V}'$ is a scaling and translation of $\mathcal{V}$.

Using Theorem~\ref{thm:waa} and putting everything together, we have the stated bound.
\end{proof}

In Section~\ref{sec:coin}, we will show an example of how to obtain a regret bound against a deterministic competitor without using a uniform prior.

\subsection{ONS and OGD as instantiations of WAA}

In this section, we show that \ac{WAA} is more general than one might think. In fact, we will see that it can be equivalent to \ac{OGD} and to the \ac{ONS}.

Here, we set $\mathcal{X}=\R^d$ and $P_1$ as the Gaussian distribution $\mathcal{N}(\bx_1, \bSigma_1)$, where $\bx_1 \in \mathcal{V}$ and $\bSigma_1\succ 0$. We use the losses $\tilde{\ell}_t(\bx):=\langle\bg_t, \bx\rangle+\frac12 \|\bx-\bx_t\|^2_{\bM_t}$, where $\bM_t\succeq0$ for an arbitrary sequence of $\bg_t$ in $\R^d$. Assume $\mathcal{V}$ is closed and convex. Finally, we set $\lambda_t=\lambda>0$ for all $t$.
Hence, we have
\[
P_t
= \argmin_{P \in \Delta_\mu(\mathcal{V})} \ \KL(P;\hat{P}_t)
= \mathcal{N}(\bx_t, \bSigma_t),
\]
where
\begin{align*}
\bx_t
&= \argmin_{\bx \in \mathcal{V}} \ (\bx-\hat{\bx}_t)^\top \bSigma^{-1}_{t}(\bx-\hat{\bx}_t),\\
\hat{P}_t(\mathrm{d} \bx)
&\propto e^{-\frac12\|\bx-\bx_1\|_{\bSigma^{-1}_1}^2-\frac{1}{\lambda} \sum_{i=1}^{t-1}\left(\langle \bg_i, \bx\rangle +\frac12 \|\bx-\bx_i\|_{\bM_i}^2\right)}
= \mathcal{N}\!\left(\hat{\bx}_t, \bSigma_t\right),\\
\hat{\bx}_t
&=\bSigma_{t}\left(\bSigma_1^{-1} \bx_1+\frac{1}{\lambda}\sum_{i=1}^{t-1} (\bM_i \bx_i-\bg_i)\right),\\
\bSigma_{t}
&=\left(\bSigma_1^{-1}+\frac{1}{\lambda} \sum_{i=1}^{t-1} \bM_i\right)^{-1}~.
\end{align*}
From the above, we clearly have $\bx_t=\E_{\bx \sim P_t}[\bx]$.

However, $\bx_t$ can also be written as
\[
\bx_t
= \argmin_{\bx \in \mathcal{V}} \ \frac{\lambda}{2}\|\bx-\bx_1\|^2_{\bSigma_1^{-1}}+\sum_{i=1}^{t-1} \left(\langle \bg_i, \bx\rangle+ \frac12 \|\bx-\bx_i\|_{\bM_i}^2\right),
\]
that is, the \ac{FTRL} solution with regularizer $\psi(\bx)=\frac{\lambda}{2}\|\bx-\bx_1\|^2_{\bSigma_1^{-1}}$ on the surrogate losses $\tilde{\ell}_t$.
This implies that $\bx_t$ is equivalent to the output of the algorithm for the weaker notion of strong convexity we have analyzed in Section~\ref{sec:ons}.
So, using the notation $\bS_{t}=\lambda\bSigma_1^{-1}+ \sum_{i=1}^{t} \bM_i$, we can directly use the bound we proved for it:
\[
\sum_{t=1}^T \tilde{\ell}_t(\bx_t) - \sum_{t=1}^T \tilde{\ell}_t(\bu)
\leq \frac{\lambda}{2} \|\bu-\bx_1\|_{\bSigma^{-1}_1}^2 -\min_{\bx \in \mathcal{V}} \ \frac{\lambda}{2} \|\bx-\bx_1\|^2_{\bSigma^{-1}_1} + \frac{1}{2} \sum_{t=1}^T \|\bg_t\|^2_{\bS^{-1}_t}~.
\]

Consider now that we have losses $\ell_t:\R^d \to (-\infty, +\infty]$, set $\bg_t \in \partial \ell_t(\bx_t)$, and choose $\bM_t$ such that
\[
\sum_{t=1}^T \ell_t(\bx_t) - \sum_{t=1}^T \ell_t(\bu)
\leq \sum_{t=1}^T \tilde{\ell}_t(\bx_t) - \sum_{t=1}^T \tilde{\ell}_t(\bu)~.
\]
So, we use the \ac{WAA} algorithm with the surrogate losses $\tilde{\ell}_t$, to upper bound the regret on the losses $\ell_t$.
Hence, we have the following cases:
\begin{itemize}
\item If $\ell_t$ are exp-concave\index{function!exp-concave} on $\mathcal{V}$, we can set $\bM_t$ proportional to $\bg_t \bg_t^\top$ and we recover the regret upper bound of the \ac{ONS} algorithm\index{Online Newton Step algorithm} in Section~\ref{sec:ons}.
\item For convex losses, we have $\bM_t=\boldsymbol{0}$. Hence, for all $\bu \in \mathcal{V}$, we have
\begin{align*}
\sum_{t=1}^T \ell_t(\bx_t) - \sum_{t=1}^T \ell_t(\bu)
&\leq \sum_{t=1}^T \tilde{\ell}_t(\bx_t) - \sum_{t=1}^T \tilde{\ell}_t(\bu)\\
&\leq \frac{\lambda}{2} \|\bu-\bx_1\|_{\bSigma_1^{-1}}^2 -\min_{\bx \in \mathcal{V}} \ \frac{\lambda}{2} \|\bx-\bx_1\|^2_{\bSigma_1^{-1}} + \frac{1}{2\lambda} \sum_{t=1}^T \|\bg_t\|^2_{\bSigma_1}~.
\end{align*}
If in addition $\bSigma_1=\bI_d$, for all $\bu \in \mathcal{V}$, we have
\begin{align*}
\sum_{t=1}^T \ell_t(\bx_t) - \sum_{t=1}^T \ell_t(\bu)
&\leq \sum_{t=1}^T \tilde{\ell}_t(\bx_t) - \sum_{t=1}^T \tilde{\ell}_t(\bu)\\
&\leq \frac{\lambda}{2} \|\bx_1-\bu\|_2^2 -\min_{\bx \in \mathcal{V}} \ \frac{\lambda}{2} \|\bx-\bx_1\|^2_2 + \frac{1}{2\lambda} \sum_{t=1}^T \|\bg_t\|^2_2~.
\end{align*}
\item For $\beta$-strongly convex\index{function!strongly convex} losses with respect to $\|\cdot\|_2$ on $\mathcal{V}$, we have $\bM_t=\beta \bI_d$. Hence, for $\lambda=1$ and any sequence of covariance matrices with $\bSigma_{1,n}^{-1}\succ 0$ and $\bSigma_{1,n}^{-1}\to \boldsymbol{0}$, for all $\bu \in \mathcal{V}$, the bound above yields in the limit
\[
\sum_{t=1}^T \ell_t(\bx_t) - \sum_{t=1}^T \ell_t(\bu)
\leq \sum_{t=1}^T \tilde{\ell}_t(\bx_t) - \sum_{t=1}^T \tilde{\ell}_t(\bu)
\leq \frac{1}{2} \sum_{t=1}^T \frac{\|\bg_t\|^2_2}{\beta t}~.
\]
\end{itemize}

In the next theorem, we also show that the bound in Lemma~\ref{lemma:master_lemma_waa} is powerful enough to capture all the cases above.
\begin{theorem}
Assume that $\mathcal{V}$ is closed and convex.
Let $\mathcal{X}=\R^d$ and run \ac{WAA} on the losses $\tilde{\ell}_t(\bx):=\langle\bg_t, \bx\rangle+\frac12 \|\bx-\bx_t\|^2_{\bM_t}$, where $\bg_t \in \R^d$ and $\bM_t \succeq 0$ are arbitrary for all $t$. Set $P_1=\mathcal{N}(\bx_1, \bSigma_1)$, where $\bx_1 \in \mathcal{V}$ and $\bSigma_1 \succ 0$. Set $\lambda_t=\lambda>0$ for all $t$. Then, for all $\bu \in \mathcal{V}$, we have
\[
\sum_{t=1}^T \left(\tilde{\ell}_t(\bx_t)-\tilde{\ell}_t(\bu)\right)
\leq \frac{\lambda}{2}  \|\bx_1-\bu\|^2_{\bSigma^{-1}_1}  + \frac{1}{2 \lambda}\sum_{t=1}^T \|\bg_t\|^2_{\bSigma_{t+1}},
\]
where $\bSigma_{t+1}=\left(\bSigma_1^{-1}+\frac{1}{\lambda} \sum_{i=1}^{t} \bM_i\right)^{-1}$ for all $t$.
\end{theorem}
\begin{proof}
Select $Q=\mathcal{N}(\bu, \bC)$ for any $\bu \in \mathcal{V}$ and we will specify $\bC$ in the following.

Observe that the quadratic nature of the losses allows us to easily go from a stochastic to a deterministic competitor:
\[
\E_{\bx \sim Q}[\tilde{\ell}_t(\bx)]
= \langle \bg_t, \bu\rangle + \frac{\Tr[\bC \bM_t]}{2}+ \frac{\|\bu-\bx_t\|^2_{\bM_t}}{2}
= \tilde{\ell}_t(\bu)+\frac{\Tr[\bC \bM_t]}{2} ~.
\]
Hence, we have
\[
\sum_{t=1}^T \left(\tilde{\ell}_t(\bx_t)-\tilde{\ell}_t(\bu)\right)
= \sum_{t=1}^T \left(\tilde{\ell}_t(\bx_t)-\E_{\bx\sim Q}[\tilde{\ell}_t(\bx)]\right)+\frac12 \Tr\left(\bC \sum_{t=1}^T \bM_t\right)~.
\]
Now, we focus on the terms in the bound in Lemma~\ref{lemma:master_lemma_waa}. For the terms in the sum, we have
\begin{align}
&\lambda \ln\E_{\bx\sim P_t}\left[e^{-\frac{1}{\lambda} \left(\tilde{\ell}_t(\bx)-\tilde{\ell}_t(\bx_t)\right)}\right] \nonumber \\
&\quad= \lambda \ln\E_{\bx\sim P_t}\left[e^{-\frac{1}{\lambda} \left(\langle \bg_t, \bx-\bx_t\rangle + \frac12 \|\bx-\bx_t\|^2_{\bM_t}\right)}\right]
= \frac{\bg_t^\top \bSigma_{t+1}\bg_t}{2 \lambda} - \frac{\lambda}{2}\ln \frac{|\bSigma_{t}|}{|\bSigma_{t+1}|}, \label{eq:from_waa_to_ogd_eq1}
\end{align}
where we leave the verification of the last equality as an exercise (see Problem~\ref{exercise:gaussian_equality}).
For the \ac{KL} divergence term, we have
\[
\KL(Q;P_1)
= \frac12 \left(\ln\frac{|\bSigma_1|}{|\bC|}+\Tr(\bC \bSigma^{-1}_1)-d+ (\bx_1-\bu)^\top \bSigma^{-1}_1(\bx_1-\bu) \right)~.
\]
Putting everything together, we obtain
\begin{align*}
&\frac{\lambda}{2} \left(\ln\frac{|\bSigma_1|}{|\bC|}+\Tr(\bC \bSigma^{-1}_1)-d+ (\bx_1-\bu)^\top \bSigma^{-1}_1(\bx_1-\bu) \right) + \frac{1}{2 \lambda}\sum_{t=1}^T \bg_t^\top \bSigma_{t+1}\bg_t \\
&\quad - \frac{\lambda}{2}\ln \frac{|\bSigma_{1}|}{|\bSigma_{T+1}|} + \frac{1}{2} \Tr\left[\bC \sum_{t=1}^T \bM_t\right]~.
\end{align*}
We now select $\bC=\bSigma_{T+1}=(\bSigma_1^{-1}+\frac{1}{\lambda} \sum_{t=1}^T \bM_t)^{-1}$, which minimizes the bound, to have the stated bound.
\end{proof}

\index{Weighted Average Algorithm|)textbf}

\section{The Aggregating Algorithm and Mixable Losses}

\index{Aggregating Algorithm|(textbf}
Here, we show how to obtain a guarantee similar to the one of Theorem~\ref{thm:waa} for a larger class of loss functions.
As above, $P_1$ is a probability distribution on $\mathcal{X}$; whenever densities with respect to a background measure are written, they are understood with respect to the reference measure $m$.

In the case of $\alpha$-exp-concave\index{function!exp-concave} losses, we used the fact that
\[
\lambda_t\ln \E_{\bx \sim P_t} \left[e^{-\frac{1}{\lambda_t} \ell_t(\bx)}\right] + \ell_t\left(\E_{\bx \sim P_t}[\bx]\right)
\leq0,
\]
for all $\frac{1}{\lambda_t} \leq \alpha$.
Now, we want to extend the class of functions where we can use a similar inequality. So, we introduce a more general class of losses, the \emph{$\alpha$-mixable} ones.
\index{function!mixable|(textbf}
\begin{definition}
Let $\mathcal{V}\subseteq \mathcal{X}$ and $f:\mathcal{X}\to \R$. We say that $f$ is \textbf{$\alpha$-mixable} on $\mathcal{V}$ if there exists a mapping $s:\Delta(\mathcal{X}) \to \mathcal{V}$ called \textbf{substitution function}\index{substitution function|textbf} such that
\[
\frac{1}{\alpha}\ln \E_{\bx \sim P} \left[e^{-\alpha f(\bx)}\right] + f\left(s(P)\right) \leq0, \quad \forall P \in \Delta(\mathcal{X})~.
\]
\end{definition}

\begin{proposition}
\label{prop:mixable_less}
Let $\alpha >0$ and $\mathcal{X}\subseteq \R^d$. If $f:\mathcal{X} \to \R$ is $\alpha$-mixable, then $f$ is $\beta$-mixable for any $0<\beta<\alpha$.
\end{proposition}
\begin{proof}
To prove that $f$ is $\beta$-mixable, it is enough to show that $h(\alpha)\geq h(\beta)$ for $h(y)=\left(\E_{\bx \sim P} \left[e^{-y f(\bx)}\right]\right)^\frac{1}{y}$.
Observe that
\begin{align*}
h(\alpha)
&=\left(\E_{\bx \sim P} \left[e^{-\alpha f(\bx)}\right]\right)^\frac{1}{\alpha}
= \left(\E_{\bx \sim P} \left[\left(e^{-f(\bx)}\right)^\alpha\right]\right)^\frac{1}{\alpha}\\
&=\left(\E_{\bx \sim P} \left[\left(e^{-f(\bx)}\right)^{\beta\frac{\alpha}{\beta}} \right]\right)^\frac{1}{\alpha}
\geq \left(\E_{\bx \sim P} \left[\left(e^{-f(\bx)}\right)^\beta\right]\right)^\frac{1}{\beta}
=h(\beta),
\end{align*}
where we used Jensen's inequality\index{inequality!Jensen's} (Theorem~\ref{thm:jensen}) on the convex function $x \mapsto x^\frac{\alpha}{\beta}$.
\end{proof}

There is an additional caveat: the substitution function here depends on $f$, and in online learning, we only know the loss function after producing our prediction. So, in our setting, we need to find a generic substitution function that holds for a \emph{class} of loss functions.

Hence, we consider $\ell_t(\bx)=\ell(\bx,\by_t)$ where $\by_t \in \mathcal{Y}$, and we require
\[
    \frac{1}{\alpha}\ln \E_{\bx \sim P} \left[e^{-\alpha \ell(\bx,\by)}\right] + \ell\left(s(P),\by\right)
    \leq 0, \quad \forall P \in \Delta(\mathcal{X}), \forall \by \in \mathcal{Y}~.
\]
\index{function!mixable|)textbf}
An example of such losses is given in the following proposition.
\begin{proposition}
\label{prop:logistic_loss_mixable}
Define the softmax $\sigma: \R^K \to \Delta_{>0}^{K-1}:=\{\bx \in \Delta^{K-1}: x_i>0, i=1, \dots, K\}$ as $\sigma(\bz)_k=\frac{\exp(z_k)}{\sum_{i=1}^K \exp(z_i)}$.\index{softmax}
The multiclass logistic loss\index{logistic loss!multiclass}, also referred to as softmax cross-entropy loss, $\ell:\R^K \times \{1, \dots, K\} \to \R$, defined as $\ell(\bx, y) = -\ln(\sigma(\bx)_y)$, is 1-mixable\index{function!mixable}.
\end{proposition}
\begin{proof}
The proof is by construction: define the mapping $s$ as $s(P)=\ln(\E_{\bx\sim P}[\sigma(\bx)])$ for any distribution $P \in \Delta(\R^K)$, where the logarithm is entry-wise.
So, for any $y \in \{1, \dots, K\}$, we have
\begin{align*}
\E_{\bx\sim P}\left[e^{-\ell(\bx, y)}\right]
= \E_{\bx\sim P}[\sigma(\bx)_y ]
= \sigma\left(\ln \E_{\bx\sim P}[\sigma(\bx) ]\right)_y
= \sigma(s( P ))_y
= e^{-\ell(s( P ), y)},
\end{align*}
where the third equality uses $\sigma(\ln(\bp)) = \bp$ for all $\bp \in \Delta_{>0}^{K-1}$. Thus, $\ell$ is 1-mixable.
\end{proof}

It is worth stressing that not all convex losses are mixable\index{function!mixable}, as proved next.
\begin{proposition}
Let $\mathcal{V}=\mathcal{X}=\R$.
The loss $\ell:\R\times\{-1,1\}\to\mathbb R$, $\ell(x,y)=|x-y|$, is not mixable.
\end{proposition}
\begin{proof}
If the loss were mixable, there would exist $\alpha>0$ and a substitution function $s$ such that
\[
    \frac{1}{\alpha}\ln \E_{x \sim P} \left[e^{-\alpha |x-y|}\right] + |s(P)-y|
    \leq 0, \quad \forall P \in \Delta(\mathcal{X}), y \in \{-1,1\}~.
\]

Consider $P(-1)=P(1)=1/2$.
We need to find a $\mu \in \R$ and $\alpha > 0$ such that
\begin{align*}
\E_{x \sim P} \left[e^{-\alpha |x-1|}\right]\leq e^{-\alpha |\mu-1|} \quad \text{ and } \quad
\E_{x \sim P} \left[e^{-\alpha |x+1|}\right]\leq e^{-\alpha |\mu+1|}~.
\end{align*}
Let's calculate the expectations:
\begin{align*}
\E_{x \sim P} \left[e^{-\alpha |x-1|}\right]=\frac12+\frac12e^{-2\alpha } \quad \text{ and } \quad
\E_{x \sim P} \left[e^{-\alpha |x+1|}\right]=\frac12+\frac12 e^{-2\alpha }~.
\end{align*}
Hence, we need to satisfy
\begin{align*}
-\frac{\ln\left(\frac12+ \frac12 e^{-2\alpha }\right)}{\alpha} \geq  |\mu-1|
\quad\text{ and }\quad
-\frac{\ln\left(\frac12+ \frac12 e^{-2\alpha }\right)}{\alpha} \geq  |\mu+1|~.
\end{align*}
Summing these two inequalities, we have
\[
-2\frac{\ln\left( \frac12+ \frac12 e^{-2\alpha }\right)}{\alpha}
\geq  |\mu-1| + |\mu+1|
\geq 2~.
\]
Solving for $\alpha$, we have $2 \geq e^{-\alpha}+e^{\alpha}$ which has the only solution $\alpha=0$. Hence, the function is not mixable.
\end{proof}

It is clear that the proof of \ac{WAA} holds for $\bx_t=s(P_t)$ and $\lambda_t \geq \frac{1}{\alpha}$, instead of $\bx_t = \E_{\bx \sim P_t} [\bx]$. Moreover, every $\alpha$-exp-concave\index{function!exp-concave} function is $\alpha$-mixable when $\mathcal{V}\equiv \mathcal{X}$ is convex and bounded, because the substitution function is $s(P)=\E_{\bx \sim P}[\bx]$. Note that the boundedness ensures that the expectation exists. This implies that all the following regret guarantees also hold for exp-concave functions.

Equipped with the definition of mixability, we can introduce the \ac{AA} in Algorithm~\ref{alg:aa}. Its regret guarantee is the following one, and its proof is identical to that of \ac{WAA} using $\bx_t=s(P_t)$.
\begin{theorem}
\label{thm:aa}
Let $P_1\in\Delta(\mathcal X)$. Assume $\ell:\mathcal{X}\times \mathcal{Y} \to \R$ is $\alpha$-mixable\index{function!mixable} in the first argument with substitution function $s:\Delta(\mathcal{X}) \to \mathcal{V}$, and $\frac{1}{\alpha}\leq \lambda_{t}\leq \lambda_{t+1}$ for $t=1, \dots, T-1$, and $\lambda_{T+1}=\lambda_T$. Then, Algorithm~\ref{alg:aa} satisfies
\[
\sum_{t=1}^T \ell( \bx_t, \by_t)
\leq \E_{\bx \sim Q}\left[\sum_{t=1}^T \ell(\bx,\by_t)\right] + \lambda_T \KL(Q;P_1), \quad \forall Q \in \Delta(\mathcal{X})~.
\]
Moreover, we also have
\[
\sum_{t=1}^T \ell( \bx_t,\by_t)
\leq -\lambda_T \ln \E_{\bx \sim P_1}\left[ e^{-\frac{1}{\lambda_T} \sum_{t=1}^T \ell(\bx,\by_t)}\right]~.
\]
\end{theorem}

\begin{algorithm}[t]
\caption{Aggregating Algorithm (AA)}
\label{alg:aa}
\begin{algorithmic}[1]
{
    \REQUIRE{Non-empty closed sets $\mathcal{V} \subseteq \mathcal{X}\subseteq \R^d$, $\lambda_1, \dots, \lambda_T>0$, substitution function $s:\Delta(\mathcal{X})\to \mathcal{V}$}
    \STATE{Set $P_1$ to be a distribution supported on $\mathcal{X}$}
    \FOR{$t=1$ {\bfseries to} $T$}
    \STATE{Set $P_{t} \in \argmin_{P \in \Delta(\mathcal{X})} \  \lambda_t\KL(P;P_1)+ \sum_{i=1}^{t-1} \E_{\bx \sim P} [\ell(\bx,\by_i)]$}
    \STATE{Output $\bx_t = s(P_t)$}
    \STATE{Get $\by_t$ and pay the loss $\ell(\bx_t,\by_t)$}
    \ENDFOR
}
\end{algorithmic}
\end{algorithm}

As we did for \ac{WAA}, under additional assumptions, we can also give a regret bound with respect to a deterministic competitor.
\begin{theorem}
\label{thm:mixable_lipschitz}
Let $\mathcal{V}\equiv\mathcal{X}$ be a non-empty closed convex set in $\R^d$, full-dimensional in $\R^d$, and $D:=\max_{\bx, \by \in \mathcal{V}} \ \|\bx-\by\|<\infty$, where $\|\cdot\|$ is an arbitrary norm.
Let $\ell:\mathcal{V}\times \mathcal{Y}\to \R$ and assume that $\ell(\cdot, \by_t)$ is $L_t$-Lipschitz with respect to $\|\cdot\|$ and $\alpha$-mixable\index{function!mixable}.
Then, the Aggregating Algorithm in Algorithm~\ref{alg:aa} with a sequence of $\lambda_t$ such that $\frac{1}{\alpha}\leq \lambda_{t}\leq \lambda_{t+1}$ for $t=1,\dots, T-1$, $\lambda_{T+1}=\lambda_T$, and $P_1$ as the uniform distribution on $\mathcal{V}$ satisfies
\begin{align*}
\sum_{t=1}^T (\ell(\bx_t,\by_t) - \ell(\bu, \by_t))
&\leq d \lambda_T \ln\left(\max\left(1, \frac{\frac{1}{\lambda_T} D \sum_{t=1}^T L_t}{d}\right)\right) + d \lambda_T\\
&\leq 2d \lambda_T \ln\left(\frac{\frac{1}{\lambda_T} D \sum_{t=1}^T L_t}{d}+e\right), \quad \forall \bu \in \mathcal{V}~.
\end{align*}
\end{theorem}
\begin{proof}
We start from the second bound in Theorem~\ref{thm:aa}.
Fix $\theta \in [0,1)$, $\bu \in \mathcal{V}$, and define $\mathcal{V}'=\{\bx \in \R^d: \bx=\theta \bu + (1-\theta)\bw, \bw \in \mathcal{V}\} \subseteq \mathcal{V}$. Choose $P'_1$ to be the restriction of $P_1$ to $\mathcal{V}'$, renormalized to be a probability distribution. That is, $P'_1(\mathcal{A}):=\frac{P_1(\mathcal{A}\cap\mathcal{V}')}{\beta}$ for any $\mathcal{A}\subseteq \mathcal{V}$ measurable, where $\beta= P_1(\mathcal{V}')=(1-\theta)^d$. Since $P_1$ is the uniform distribution, $P'_1$ is uniform over $\mathcal{V}'$.
So, we have
\begin{align*}
\E_{\bx \sim P_1} \left[e^{-\frac{1}{\lambda_T} \sum_{t=1}^T \ell(\bx, \by_t)}\right]
&\geq (1-\theta)^d \E_{\bx \sim P'_1} \left[e^{-\frac{1}{\lambda_T} \sum_{t=1}^T \ell(\bx,\by_t)}\right] \\
&= (1-\theta)^d \E_{\bx \sim P_1} \left[e^{-\frac{1}{\lambda_T} \sum_{t=1}^T\ell(\theta \bu + (1-\theta)\bx,\by_t)}\right] \\
&= (1-\theta)^d \E_{\bx \sim P_1} \left[e^{-\frac{1}{\lambda_T} \sum_{t=1}^T\ell((1-\theta)(\bx-\bu)+\bu,\by_t )}\right] \\
&\geq (1-\theta)^d \E_{\bx \sim P_1} \left[e^{-\frac{1}{\lambda_T} \sum_{t=1}^T\left(\ell(\bu,\by_t ) + L_t (1-\theta)\|\bx-\bu\|\right)}\right] \\
&\geq (1-\theta)^d \E_{\bx \sim P_1} \left[e^{-\frac{1}{\lambda_T} \sum_{t=1}^T \left(\ell(\bu,\by_t ) +  L_t (1-\theta) D\right)}\right] \\
&= (1-\theta)^d e^{-\frac{1}{\lambda_T} \sum_{t=1}^T \ell(\bu,\by_t )} e^{-\frac{1}{\lambda_T} \sum_{t=1}^T L_t (1-\theta)D}~.
\end{align*}

Putting everything together, we have
\[
\sum_{t=1}^T (\ell(\bx_t,\by_t) - \ell(\bu,\by_t))
\leq d \lambda_T \ln \frac{1}{1-\theta} +  (1-\theta)D \sum_{t=1}^T L_t ~.
\]
Now, we set $\theta$ such that $1-\theta=\min\left(\frac{d \lambda_T}{ D \sum_{t=1}^T L_t},1\right)$. So, we have
$(1-\theta) D \sum_{t=1}^T L_t\leq \lambda_T d$ and $\ln \frac{1}{1-\theta} = \ln\left(\max\left(1, \frac{ \frac{1}{\lambda_T} D \sum_{t=1}^T L_t}{d}\right)\right)$, that gives the stated bound.
\end{proof}

\section{Example of AA: Online Multiclass Logistic Regression}
\label{sec:aa_logistic}

In this section, we show an application of \ac{AA} to obtain logarithmic regret for online multiclass logistic regression.
In this problem, in each round we receive a covariate $\bz_t \in \mathcal{Z} \subseteq \R^d$ and we produce a discrete probability distribution over the $K$ classes as $\hat{\by}_t \in \Delta^{K-1}$. Then, we receive the true class $y_t \in \{1, \dots, K\}$ and we pay the loss $- \ln \hat{\by}_{t,y_t}$.

We could use a linear strategy: $\hat{\by}_t=\sigma(\bX^\top_t \bz_t)$ where $\bX_t \in \mathcal{V}\subseteq \R^{d \times K}$ is the linear classifier at time $t$.
Assuming that the covariates $\bz_t$ have bounded norm $Z$ and that the columns of the matrices in $\mathcal{V}$ have bounded norm $D$, one can show that this problem has exp-concave\index{function!exp-concave} losses, so we could use the \ac{ONS}\index{Online Newton Step algorithm} or the \ac{WAA} algorithm. Unfortunately, the exp-concavity\index{function!exp-concave} is of the order of $\exp(- Z D)$.

Here, we show that we can prove a logarithmic bound that depends only logarithmically on $D$.
The price that we pay for this improved rate is that the algorithm we will use is \emph{improper}, that is, our predictor will not be linear in the covariates $\bz_t$, but we will still measure its regret with respect to a linear competitor.

We will proceed a little bit differently than in the standard \ac{AA} algorithm because we will use two loss functions, $\ell_t:\mathcal{V}\times \mathcal{Y} \to \R$ and $\hat{\ell}: \R^K \times \mathcal{Y}\to \R$ that satisfy a \emph{generalized form of $\alpha$-mixability}:
\[
    \frac{1}{\alpha}\ln \E_{\bx \sim P} \left[e^{-\alpha \ell_t(\bx, \by)}\right] + \hat{\ell}\left(s_t(P), \by\right)
    \leq 0, \quad \forall P \in \Delta(\mathcal{V}), \forall \by \in \mathcal{Y}~.
\]
It is easy to see that Theorem~\ref{thm:aa} extends to this case:
\[
\sum_{t=1}^T \hat{\ell}( \bx_t,\by_t)
\leq -\lambda_T \ln \E_{\bx \sim P_1}\left[ e^{-\frac{1}{\lambda_T} \sum_{t=1}^T \ell_t(\bx,\by_t)}\right]~.
\]
Similarly, we have that if $\ell_t$ is $L_t$-Lipschitz in its first argument, we have the corresponding version of Theorem~\ref{thm:mixable_lipschitz}:
\begin{equation}
\label{eq:mixable_lipschitz_extended}
\sum_{t=1}^T (\hat{\ell}(\bx_t,\by_t) - \ell_t(\bu, \by_t))
\leq 2 \dim(\mathcal{V}) \lambda_T \ln\left(\frac{\frac{1}{\lambda_T} D \sum_{t=1}^T L_t}{\dim(\mathcal{V})}+e\right) \quad \forall \bu \in \mathcal{V}~.
\end{equation}
In particular, from Proposition~\ref{prop:logistic_loss_mixable}, we will define $\ell_t(\bX, y) = -\ln (\sigma(\bX^\top \bz_t)_{y})$, $s_t(P) = \ln \E_{\bX\sim P}[\sigma(\bX^\top \bz_t)]$, and $\hat{\ell}(\bx, y)=-\ln (\sigma(\bx)_{y})$.

\begin{algorithm}[t]
\caption{Aggregating Algorithm for Improper Online Multiclass Logistic Regression}
\label{alg:aa_logistic}
\begin{algorithmic}[1]
{
    \REQUIRE{Non-empty bounded closed convex set $\mathcal{V} \subseteq \R^{d\times K}$, full-dimensional in $\R^{d\times K}$}
    \STATE{Set $P_1$ to the uniform distribution over $\mathcal{V}$}
    \FOR{$t=1$ {\bfseries to} $T$}
    \STATE{Set $P_t(\mathrm{d}\bX) = \frac{\exp\left(\sum_{i=1}^{t-1}\ln(\sigma(\bX^\top\bz_i)_{y_i})\right) P_1(\mathrm{d}\bX)}
  {\int_{\mathcal V}\exp\left(\sum_{i=1}^{t-1}\ln(\sigma(\bU^\top\bz_i)_{y_i})\right) P_1(\mathrm{d}\bU)} \in \Delta(\mathcal V)$}
    \STATE{Obtain $\bz_t \in \mathcal{Z}$}
    \STATE{Output $\bx_t=\ln \E_{\bX \sim P_t}[\sigma(\bX^\top \bz_t)] \in \R^{K}$}
    \STATE{Obtain $y_t$ and pay $ -\ln \sigma(\bx_{t})_{y_t}$}
    \ENDFOR
}
\end{algorithmic}
\end{algorithm}

\begin{theorem}
\label{thm:aa_logistic}
Let $\mathcal{B} \subset \R^d$ be closed, convex, and $\max_{\bx, \by \in \mathcal{B}} \ \|\bx-\by\|_2\leq D$.
Assume that $\|\bz_t\|_2\leq Z$ for all $t$ and $\mathcal{V}=\{\bX=[\bx_1, \dots, \bx_K] \in \R^{d\times K}: \bx_i \in \mathcal{B}\}$, full-dimensional in $\R^{d\times K}$.
Then, Algorithm~\ref{alg:aa_logistic} satisfies
\[
-\sum_{t=1}^T \ln (\sigma(\bx_t)_{y_t}) + \sum_{t=1}^T \ln (\sigma(\bU^\top \bz_t)_{y_t})
\leq 2d\, K \ln\left(\frac{2D\, Z\, T}{d\, K}+e\right), \quad \forall \bU \in \mathcal{V}~.
\]
\end{theorem}
\begin{proof}
Write $\bX=[\bx_1,\dots,\bx_K]$ and, similarly, for any matrix $\bG\in\R^{d\times K}$ write $\bG=[\bg_1,\dots,\bg_K]$, where $\bx_j,\bg_j\in\R^d$ denote the $j$-th columns.
We have
\[
\nabla_{\bx_j} \ell_t(\bX,y)
= (\sigma(\bX^\top \bz_t)_j-\indevent{\{y=j\}})\bz_t, \quad \forall j \in \{1,\dots,K\}~.
\]
Define $\|\bX\|_{2,\infty}:=\max_{j\in\{1,\dots,K\}}\|\bx_j\|_2$. By Lemma~\ref{lemma:dual_group_norm}, its dual norm is
$\|\bG\|_{2,1}:=\sum_{j=1}^K \|\bg_j\|_2$.
Hence,
\[
\|\nabla_{\bX} \ell_t(\bX,y_t)\|_{2,1}
= \|\sigma(\bX^\top \bz_t)-\be_{y_t}\|_1 \|\bz_t\|_2
\leq 2Z~.
\]
Therefore, $\ell_t$ is $2Z$-Lipschitz with respect to $\|\cdot\|_{2,\infty}$.
Moreover, if $\bX,\bU\in\mathcal{V}$, then each pair of corresponding columns satisfies
$\|\bx_j-\bu_j\|_2\leq D$, and thus $\|\bX-\bU\|_{2,\infty}\leq D$.
So, the diameter of $\mathcal{V}$ with respect to $\|\cdot\|_{2,\infty}$ is at most $D$.
Using \eqref{eq:mixable_lipschitz_extended} with $\dim(\mathcal{V})=dK$, $L_t=2Z$, and $\lambda_t=1$, we obtain the stated result.
\end{proof}
\index{Aggregating Algorithm|)textbf}

\section{History Bits}

The \ac{WAA}\index{Weighted Average Algorithm} is from \citet{KivinenW99}, as a simplification of the \ac{AA}\index{Aggregating Algorithm} of \citet{Vovk90} (see also \citet[Appendix A]{Vovk98} for an easier description of the \ac{AA}).
This algorithm is known by many names: \citet{Cesa-BianchiL06} calls it ``exponentially weighted mixture forecaster'', \citet{HazanKKA06,HazanAK07} rediscovered it (see below) and named it ``exponentially weighted online optimization algorithm'', Wouter Koolen calls it simply ``exponential weights algorithm'' in his blog post in 2016 (see below). I prefer to use the name its designers gave it, also because its acronym nicely matches that of the Aggregating Algorithm.

The observation that the \ac{AA}\index{Aggregating Algorithm} also works for infinite sets of experts was made by \citet{Freund96,Freund03}.
Theorem~\ref{thm:waa_bounded} is from \citet{HazanKKA06,HazanAK07}, where they seem to rediscover the \ac{WAA}\index{Weighted Average Algorithm} with uniform prior, but I improved it by adding a $1/d$ term in the logarithm. The proof of \citet{HazanKKA06,HazanAK07} is a generalization of that of \citet{BlumK97,BlumK99} for universal portfolio.

The observation that the \ac{WAA}\index{Weighted Average Algorithm} recovers the \ac{ONS}\index{Online Newton Step algorithm} and the gradient descent algorithm and its bounds is from a blog post\footnote{\url{https://blog.wouterkoolen.info/EW4Quadratic/post.html}} by Wouter Koolen, where it is done for \ac{OGD}, while I derived it for \ac{FTRL}. \citet{vanderHoeven16} has extended this equivalence to \ac{OMD}.
The subtlety of defining the prior on $\mathcal{X}$ instead of on $\mathcal{V}$ is by me, and it allows us to state a single theorem that covers all the cases.

The concept of mixability\index{function!mixable} is introduced in \citet{Vovk01}. Proposition~\ref{prop:mixable_less} is in the proof of \citet[Lemma 9]{Vovk98}.
The mixability of the logistic loss\index{logistic loss}, Proposition~\ref{prop:logistic_loss_mixable}, and the content of Section~\ref{sec:aa_logistic} are from \citet{FosterKLMS18}.
Theorem~\ref{thm:mixable_lipschitz} is a generalization of a similar one for the logistic loss from \citet{FosterKLMS18}.

\section{Exercises}

\begin{exer}
\label{exercise:dv_formula}
Prove the Donsker--Varadhan\index{Donsker--Varadhan variational formula} variational formula in \eqref{eq:aa_alternate_term}. Also identify the distribution attaining the infimum when the normalizing constant is finite and positive.
\end{exer}

\begin{exer}
\label{exercise:gaussian_equality}
Prove the equality in \eqref{eq:from_waa_to_ogd_eq1}.
\end{exer}

\begin{exer}
\label{exercise:square_loss_mixable}
Let $\mathcal{X}=\mathcal{V}=[0,1]$ and $\mathcal{Y}=[0,1]$. Show that the square loss $\ell(x,y)=(x-y)^2$ is mixable for some positive constant $\alpha$. Find an explicit value of $\alpha$ and construct a corresponding substitution function.
\end{exer}

\acresetall

\chapter{Black-Box Reductions}
\index{black-box reduction|(textbf}

In this chapter, we show how to reduce one online learning problem into another, in a \emph{black-box} way. That is, we will use an online convex optimization algorithm to solve a problem different from what it was meant to solve, without looking at its internal workings in any way. The only thing we will change is the input we pass to the algorithm. Why do it? Because you might have optimization software that you cannot modify, or just because designing and analyzing an online learning algorithm in one case might be easier than in another.

\acresetall

\section{Solving Constrained Online Convex Optimization}
\index{black-box reduction!for constrained online convex optimization|(textbf}

We might have an \ac{OCO} algorithm designed for constrained online linear optimization over a feasible set $\mathcal{W} \subseteq \R^d$, and we might wonder how we can use it for online convex optimization on a feasible set $\mathcal{V}\subset \mathcal{W}$. How do we deal with the smaller feasible set?

Let's see a prototypical example of this approach. We have a \ac{OCO} algorithm that outputs in each round $\bz_t \in \mathcal{W}=\R^d_{\geq 0}$ and we want to use it for \ac{OCO} in $\mathcal{V}=\Delta^{d-1}$. So, first of all, we have to transform $\bz_t$ into a vector in the probability simplex\index{probability simplex}. One way to do it could be to normalize it by its L$_1$ norm. So, we predict in each round by $\bx_t=\frac{\bz_t}{\|\bz_t\|_1}$ if $\|\bz_t\|_1\neq 0$ and with any other vector in the simplex if $\|\bz_t\|_1=0$. Given that we changed the prediction, the original regret upper bound no longer holds. So, one solution is to pass to the algorithm \emph{surrogate losses}, so that the regret with the surrogate losses upper bounds the regret on the new problem with the modified predictions. In formulas, we have
\begin{align*}
\Regret_T(\bu)
&= \sum_{t=1}^T(\ell_t(\bx_t)-\ell_t(\bu))
\leq \sum_{t=1}^T \langle \bg_t, \bx_t-\bu\rangle \nonumber \\
&= \sum_{t=1}^T (\langle \bg_t, \bz_t-\bu\rangle + \langle \bg_t, \bx_t-\bz_t\rangle),
\end{align*}
where $\bg_t \in \partial \ell_t(\bx_t)$.
Now, we want to find $\hat{\bg}_t$ such that
\[
\langle \bg_t, \bx_t-\bz_t\rangle
\leq \langle \hat{\bg}_t, \bz_t-\bu\rangle, \quad \forall \bu \in \mathcal{V}~.
\]
If we could do it, we would have
\begin{align*}
\Regret_T(\bu)
&\leq \sum_{t=1}^T \langle \bg_t+ \hat{\bg}_t, \bz_t-\bu\rangle,
\end{align*}
that is the regret of the \ac{OCO} algorithm on the surrogate linear losses $\bx \mapsto \langle \bg_t +\hat{\bg}_t,\bx\rangle$.

Let's see how to find $\hat{\bg}_t$. Using our definition of $\bx_t$, we have that if $\|\bz_t\|_1\neq 0$ then
\begin{align*}
\langle \bg_t, \bx_t-\bz_t\rangle
&= \left\langle \bg_t, \frac{\bz_t}{\|\bz_t\|_1}-\bz_t\right\rangle
= \left\langle \bg_t, \frac{\bz_t}{\|\bz_t\|_1}\right\rangle \left(1-\|\bz_t\|_1\right)\\
&= -\langle \bg_t, \bx_t\rangle \langle \ones_d, \bz_t - \bu \rangle,
\end{align*}
where $\ones_d=[1, \dots ,1]^\top\in \R^d$.
Instead, if $\|\bz_t\|_1=0$ we have
\[
\langle \bg_t, \bx_t-\bz_t\rangle
= \langle \bg_t, \bx_t\rangle
= -\langle \bg_t, \bx_t\rangle \langle \ones_d, \bz_t - \bu \rangle~.
\]
Hence, we can set $\hat{\bg}_t=- \langle \bg_t, \bx_t\rangle \ones_d$, and the surrogate linear losses we will pass to the \ac{OCO} algorithm are
\begin{equation}
\label{eq:constrained_reduction_old}
\tilde{\ell}_t(\bx)=\langle \bg_t - \ones_d \langle \bg_t, \bx_t\rangle,\bx\rangle~.
\end{equation}
In this way, we have
\[
\Regret_T(\bu)
= \sum_{t=1}^T(\ell_t(\bx_t)-\ell_t(\bu))
\leq \sum_{t=1}^T(\tilde{\ell}_t(\bz_t)-\tilde{\ell}_t(\bu))~.
\]
It is worth noting that the norm of the gradients of the surrogate losses is controlled: $\|\bg_t-\ones_d \langle \bg_t, \bx_t\rangle\|_\infty\leq 2\|\bg_t\|_\infty$.

We solved our problem, but this approach does not seem to scale to arbitrarily complex feasible sets. Hence, we need a more general way to solve it. In the following, we will see two complementary ways to do it.

\subsection{Dealing with Constraints using (Non-Euclidean) Projections}
\label{sec:constrained_reduction}

Now, we extend the procedure in the previous section to general convex sets and general norms.
First of all, we need some definitions and basic results.

\begin{definition}
Let $\mathcal{V}\subset \R^d$ be a non-empty, convex, and closed set. The \textbf{distance to a set}\index{distance to a set|textbf} with respect to the norm $\|\cdot\|$ is the function $\dist_{\mathcal{V},\|\cdot\|}:\R^d \to \R_{\geq 0}$ defined as $\dist_{\mathcal{V},\|\cdot\|}(\bx)=\min_{\by \in \mathcal{V}} \|\bx-\by\|$. We also define the \textbf{generalized projection}\index{projection!generalized|textbf} $\Pi_{\mathcal{V}, \|\cdot\|}$ as the set-valued map $\Pi_{\mathcal{V}, \|\cdot\|}(\bx) = \argmin_{\by \in \mathcal{V}} \ \|\bx-\by\|$.
\end{definition}

We have the following properties for the distance to a set\index{distance to a set} function.
\begin{lemma}
Let $\mathcal{V}$ be non-empty, convex, and closed. Then, we have
\begin{itemize}
\item $\dist_{\mathcal{V}, \|\cdot\|}$ is convex.
\item $\dist_{\mathcal{V}, \|\cdot\|}$ is 1-Lipschitz with respect to $\|\cdot\|$.
\item For all $\bx \in \mathcal{V}$, we have $\boldsymbol{0} \in \partial \dist_{\mathcal{V},\|\cdot\|}(\bx)$.
\item For all $\bx \notin \mathcal{V}$, let $\by\in\Pi_{\mathcal V,\|\cdot\|}(\bx)$. Then, we have
\[
\partial \dist_{\mathcal V,\|\cdot\|}(\bx)
= \left\{\bg:\ \|\bg\|_{\star}=1,\ \langle \bg,\bx-\by\rangle=\|\bx-\by\|,\ \langle \bg,\bu-\by\rangle\leq 0\ \ \forall \bu\in\mathcal{V} \right\}~.
\]
\end{itemize}
\end{lemma}
\begin{proof}
For the first property, for any $\bx' \in \Pi_{\mathcal{V},\|\cdot\|}(\bx)$ and any $\bz' \in \Pi_{\mathcal{V},\|\cdot\|}(\bz)$, we have
\begin{align*}
\dist_{\mathcal{V},\|\cdot\|}(\lambda \bx + (1-\lambda) \bz)
&= \min_{\by \in \mathcal{V}} \| \lambda \bx + (1-\lambda) \bz - \by\|\\
&\leq \| \lambda \bx + (1-\lambda) \bz - (\lambda \bx' + (1-\lambda) \bz')\| \\
&= \|\lambda (\bx-\bx') + (1-\lambda)(\bz-\bz')\|\\
&\leq \lambda \dist_{\mathcal{V},\|\cdot\|}(\bx) + (1-\lambda) \dist_{\mathcal{V},\|\cdot\|}(\bz),
\end{align*}
where in the first inequality we used the convexity of $\mathcal{V}$.

For the second property, let $\bx,\bz\in\R^d$ and $\bz'\in \Pi_{\mathcal{V},\|\cdot\|}(\bz)$. Then,
\begin{align*}
\dist_{\mathcal{V},\|\cdot\|}(\bx)
&= \min_{\by \in \mathcal{V}} \|\bx-\by\|
\leq \|\bx-\bz'\|
\leq \|\bx-\bz\|+\|\bz-\bz'\|\\
&= \|\bx-\bz\|+\dist_{\mathcal{V},\|\cdot\|}(\bz)~.
\end{align*}
Hence, $\left|\dist_{\mathcal{V},\|\cdot\|}(\bx)-\dist_{\mathcal{V},\|\cdot\|}(\bz)\right| \leq \|\bx-\bz\|$, where the reverse inequality is obtained by exchanging the roles of $\bx$ and $\bz$.

The third property is immediate using the optimality condition in Theorem~\ref{thm:first_order_subdiff}.


We now prove the fourth property. First, observe that $\|\by-\bx\|=\dist_{\mathcal V,\|\cdot\|}(\bx)$. We now show that every subgradient $\bg\in\partial \dist_{\mathcal V,\|\cdot\|}(\bx)$ belongs to the stated set. We have
\[
0=\dist_{\mathcal V,\|\cdot\|}(\by)
\ge \dist_{\mathcal V,\|\cdot\|}(\bx)+\langle \bg,\by-\bx\rangle
= \|\bx-\by\|-\langle \bg,\bx-\by\rangle,
\]
hence $\|\bx-\by\|\le \langle \bg,\bx-\by\rangle$.
Since $\dist_{\mathcal V,\|\cdot\|}$ is $1$-Lipschitz, by Theorem~\ref{thm:subgradient_lipschitz_dual} every subgradient satisfies $\|\bg\|_{\star}\le 1$. Therefore, by the definition of dual norms, we have
\[
\|\bx-\by\|
\le \langle \bg,\bx-\by\rangle
\le \|\bg\|_{\star}\,\|\bx-\by\|
\le \|\bx-\by\|~.
\]
As $\bx\notin\mathcal V$, we have $\|\bx-\by\|>0$, and thus equality holds throughout:
\[
\|\bg\|_{\star}=1,
\qquad
\langle \bg,\bx-\by\rangle
=\|\bx-\by\|~.
\]
It remains to prove the normal-cone condition. Let $\bu\in\mathcal V$. Since
$\bg\in\partial \dist_{\mathcal V,\|\cdot\|}(\bx)$, the subgradient inequality gives
\[
0=\dist_{\mathcal V,\|\cdot\|}(\bu)
\ge \dist_{\mathcal V,\|\cdot\|}(\bx)+\langle \bg,\bu-\bx\rangle~.
\]
Using $\dist_{\mathcal V,\|\cdot\|}(\bx)=\|\bx-\by\|$ and
$\langle \bg,\bx-\by\rangle=\|\bx-\by\|$, we obtain
\[
0
\ge \|\bx-\by\|+\langle \bg,\bu-\bx\rangle
= \langle \bg,\bx-\by\rangle+\langle \bg,\bu-\bx\rangle
= \langle \bg,\bu-\by\rangle~.
\]
Hence, $\langle \bg,\bu-\by\rangle\le 0$, for all $\bu\in\mathcal V$.

Conversely, let $\bg$ satisfy
\[
\|\bg\|_{\star}=1,
\qquad \langle \bg,\bx-\by\rangle=\|\bx-\by\|,
\qquad \langle \bg,\bu-\by\rangle\le 0\quad \forall \bu\in\mathcal{V}~.
\]
Fix any $\bz\in\R^d$ and any $\bu\in\mathcal{V}$. Since $\|\bg\|_{\star}=1$,
\[
\|\bz-\bu\|\ge \langle \bg,\bz-\bu\rangle
= \langle \bg,\bz-\bx\rangle +\langle \bg,\bx-\by\rangle +\langle \bg,\by-\bu\rangle~.
\]
Using $\langle \bg,\bx-\by\rangle=\|\bx-\by\|=\dist_{\mathcal V,\|\cdot\|}(\bx)$ and $\langle \bg,\by-\bu\rangle=-\langle \bg,\bu-\by\rangle\ge 0$, we obtain
\[
\|\bz-\bu\|
\ge
\langle \bg,\bz-\bx\rangle+\dist_{\mathcal V,\|\cdot\|}(\bx)~.
\]
Since this holds for every $\bu\in\mathcal{V}$, minimizing over $\bu\in\mathcal V$ gives
\[
\dist_{\mathcal V,\|\cdot\|}(\bz)
\ge \dist_{\mathcal V,\|\cdot\|}(\bx)+\langle \bg,\bz-\bx\rangle~.
\]
Hence $\bg\in\partial \dist_{\mathcal V,\|\cdot\|}(\bx)$.
\end{proof}

\begin{example}
If we consider the L$_2$ norm, for all $\bx \notin \mathcal{V}$, we have
\[
\partial \dist_{\mathcal{V},\|\cdot\|_2}(\bx) =
\left\{\frac{\bx-\Pi_{\mathcal{V}}(\bx)}{\|\bx-\Pi_{\mathcal{V}}(\bx)\|_2} \right\},
\]
where $\Pi_{\mathcal{V}}$ is the standard Euclidean projection.
\end{example}

\begin{example}
\label{example:projection_l1}
Note that the projection with respect to a norm different from L$_2$ does not have to be unique.
For example, we consider the L$_1$ norm and $\mathcal{V}=\Delta^{d-1}$ and we have that
\[
\Pi_{\mathcal{V},\|\cdot\|_1}(\bx)
=\argmin_{\by \in \Delta^{d-1}} \ \|\bx-\by\|_1~.
\]
From this, we can see that the generalized projection\index{projection!generalized} is not unique. In fact, if $x_i>1$ for all $i=1, \dots, d$, then $\argmin_{\by \in \Delta^{d-1}} \ \|\bx-\by\|_1=\Delta^{d-1}$. Indeed, we have
\[
\argmin_{\by \in \Delta^{d-1}} \ \|\bx-\by\|_1
= \argmin_{\by \in \Delta^{d-1}} \ \sum_{i=1}^d |x_i-y_i|
= \argmin_{\by \in \Delta^{d-1}} \ \sum_{i=1}^d (x_i-y_i)
= \argmin_{\by \in \Delta^{d-1}} \ \sum_{i=1}^d x_i - 1,
\]
that is constant with respect to $\by$.
Similarly, if $x_i\leq 0$ for all $i=1, \dots, d$, then we have $\argmin_{\by \in \Delta^{d-1}} \ \|\bx-\by\|_1=\Delta^{d-1}$ because
\[
\argmin_{\by \in \Delta^{d-1}} \ \|\bx-\by\|_1
= \argmin_{\by \in \Delta^{d-1}} \ \sum_{i=1}^d |x_i-y_i|
= \argmin_{\by \in \Delta^{d-1}} \ \sum_{i=1}^d (y_i-x_i)
= \argmin_{\by \in \Delta^{d-1}} \ 1-\sum_{i=1}^d x_i~.
\]

Denoting by $\phi(\by)= \|\bx-\by\|_1$, from the optimality condition, we have that $\bz=\Pi_{\Delta^{d-1},\|\cdot\|_1}(\bx)$ iff $\boldsymbol{0} \in \mathcal{N}_{\Delta^{d-1}}(\bz) + \partial \phi(\bz)$, where $\mathcal{N}_{\Delta^{d-1}}(\bz)$ is the normal cone\index{normal cone} to $\Delta^{d-1}$ at $\bz$ (see Example~\ref{example:normal_cone}), that is
\[
\mathcal{N}_{\Delta^{d-1}}(\bz)=\{\alpha \ones_d + \sum_{i:z_i=0} \beta_i \be_i: \alpha \in \R, \beta_i\leq 0\},
\]
where $\ones_d$ is a vector in $\R^d$ composed of all ones.
This implies that if $\mathcal{S}_1(\bx)=\{\by \in \Delta^{d-1}: y_i\geq x_i\}$ and $\mathcal{S}_2(\bx)=\{\by\in \Delta^{d-1}: y_i\leq x_i\}$ are non-empty, then they contain solutions of the generalized projection. In particular, if $\|\bx\|_1<1$ then $\mathcal{S}_1(\bx)\neq \emptyset$ and if $\|\bx\|_1>1$ and $x_i\geq 0$ then $\mathcal{S}_2(\bx)\neq \emptyset$. This implies that if $\bx \in \R^d_{\geq 0}$ then, for example, $\frac{\bx}{\|\bx\|_1} \in \Pi_{\Delta^{d-1},\|\cdot\|_1}(\bx)$ when $\|\bx\|_1>0$ and $\Pi_{\Delta^{d-1}, \|\cdot\|_1}(\bx)=\Delta^{d-1}$ if $\|\bx\|_1=0$. Moreover, we can also calculate a subgradient $\bq$ of $\dist_{\Delta^{d-1}, \|\cdot\|_1}(\bx)$ for $\bx \notin \Delta^{d-1}$ as $-\ones_d$ if $\Pi_{\Delta^{d-1},\|\cdot\|_1}(\bx) \in \mathcal{S}_1(\bx)$ and $\ones_d$ if $\Pi_{\Delta^{d-1},\|\cdot\|_1}(\bx) \in \mathcal{S}_2(\bx)$.
\end{example}

\begin{algorithm}[h]
\caption{Constrained \ac{OCO} through (Non-Euclidean) Projections}
\label{alg:constrained_reduction}
\begin{algorithmic}[1]
{
    \REQUIRE{Non-empty closed convex set $\mathcal{V} \subset \R^d$, \ac{OCO} algorithm $\mathscr{A}$ with feasible set $\mathcal{W}\supset \mathcal{V}$, generic projection operator $\Pi_{\mathcal{V}}$, a way to construct surrogate losses $\tilde{\ell}_t$}
    \FOR{$t=1$ {\bfseries to} $T$}
    \STATE{Get $\bz_t \in \mathcal{W}$ from $\mathscr{A}$}
    \STATE{Output $\bx_t \in \Pi_{\mathcal{V}}(\bz_t)$}
    \STATE{Pay the loss $\ell_t(\bx_t)$, where $\ell_t$ is subdifferentiable on $\mathcal{V}$}
    \STATE{Set $\bg_t \in \partial \ell_t(\bx_t)$}
    \STATE{Send the surrogate loss $\tilde{\ell}_t:\mathcal{W}\to \R$ to $\mathscr{A}$}
    \ENDFOR
}
\end{algorithmic}
\end{algorithm}

We can now state the main theorem.
\begin{theorem}
\label{thm:constrained_reduction}
Denote by $\Regret^\mathscr{A}_T(\bu)$ the regret of an \ac{OCO} algorithm $\mathscr{A}$ with feasible set $\mathcal{W} \subseteq \R^d$. Let $\mathcal{V} \subset \mathcal{W}$ be non-empty, convex, and closed, and $\bq_t \in \partial \dist_{\mathcal{V},\|\cdot\|} (\bz_t)$.
Set in Algorithm~\ref{alg:constrained_reduction} $\Pi_\mathcal{V}=\Pi_{\mathcal{V},\|\cdot\|}$\index{projection!generalized} and one of the following choices of surrogate losses $\tilde{\ell}_t:\mathcal{W}\to \R$:
\begin{itemize}
\item $\tilde{\ell}_t(\bx)=\langle \bg_t, \bx\rangle +\|\bg_t\|_\star \dist_{\mathcal{V},\|\cdot\|}(\bx)$
\item $\tilde{\ell}_t(\bx)=\langle \bg_t + \|\bg_t\|_\star \bq_t, \bx\rangle$;
\item $\tilde{\ell}_t(\bx)=\langle \bg_t, \bx\rangle + \begin{cases} 0, &\text{ if } \langle\bg_t, \bx_t-\bz_t\rangle \leq 0\\
\langle\bg_t, \frac{\bx_t-\bz_t}{\|\bx_t-\bz_t\|}\rangle \dist_{\mathcal{V},\|\cdot\|}(\bx), &\text{ if } \langle\bg_t, \bx_t-\bz_t\rangle > 0 \end{cases}$
\item $\tilde{\ell}_t(\bx)=\langle \bg_t, \bx\rangle + \begin{cases} 0, &\text{ if } \langle\bg_t, \bx_t-\bz_t\rangle \leq 0\\
\langle\bg_t, \frac{\bx_t-\bz_t}{\|\bx_t-\bz_t\|}\rangle \langle \bq_t, \bx\rangle, &\text{ if } \langle\bg_t, \bx_t-\bz_t\rangle > 0 \end{cases}$
\end{itemize}
Then, Algorithm~\ref{alg:constrained_reduction} guarantees
\[
\Regret_T(\bu)
= \sum_{t=1}^T (\ell_t(\bx_t)-\ell_t(\bu))
\leq \Regret^{\mathscr{A}}_T(\bu)
=\sum_{t=1}^T (\tilde{\ell}_t(\bz_t)-\tilde{\ell}_t(\bu))
, \quad \forall \bu \in \mathcal{V}~.
\]
Moreover, for any $\bx\in\mathcal W$ and any $\tilde{\bg}_t \in \partial \tilde{\ell}_t(\bx)$, we have $\|\tilde{\bg}_t\|_\star \leq 2\|\bg_t\|_\star$.
\end{theorem}
\begin{proof}
Note that the second and fourth surrogates are just linearizations of the first and third surrogate losses at $\bz_t$, respectively. Hence, it is enough to prove the regret guarantee for the first and third surrogate losses.

For all of them, we start by observing that
\begin{align}
\Regret_T(\bu)
&= \sum_{t=1}^T(\ell_t(\bx_t)-\ell_t(\bu))
\leq \sum_{t=1}^T \langle \bg_t, \bx_t-\bu\rangle \nonumber \\
&= \sum_{t=1}^T (\langle \bg_t, \bz_t-\bu\rangle + \langle \bg_t, \bx_t-\bz_t\rangle) \label{eq:constrained_reduction_eq1} ~.
\end{align}

For the first surrogate loss, we have
\begin{align*}
\sum_{t=1}^T (\langle \bg_t, \bz_t-\bu\rangle + \langle \bg_t, \bx_t-\bz_t\rangle)
&\leq \sum_{t=1}^T (\langle \bg_t, \bz_t-\bu\rangle + \|\bg_t\|_\star \|\bz_t-\bx_t\|)\\
&= \sum_{t=1}^T (\langle \bg_t, \bz_t-\bu\rangle + \|\bg_t\|_\star \min_{\by \in \mathcal{V}} \ \|\bz_t-\by\|)\\
&=\sum_{t=1}^T (\tilde{\ell}_t(\bz_t) -\tilde{\ell}_t(\bu))~.
\end{align*}

For the third surrogate loss, first of all, observe that this loss is convex because $\dist_{\mathcal{V},\|\cdot\|}$ is multiplied by a non-negative number.
In the case that $\langle\bg_t, \bx_t-\bz_t\rangle\leq 0$, the second term in \eqref{eq:constrained_reduction_eq1} is negative, so we can drop it.
For the case that $\langle\bg_t, \bx_t-\bz_t\rangle> 0$, we have
\begin{align*}
\langle \bg_t, \bz_t-\bu\rangle + \langle \bg_t, \bx_t-\bz_t\rangle
&= \langle \bg_t, \bz_t-\bu\rangle + \left\langle \bg_t, \frac{\bx_t-\bz_t}{\|\bx_t-\bz_t\|}\right\rangle\|\bx_t-\bz_t\|\\
&= \langle \bg_t, \bz_t-\bu\rangle + \left\langle \bg_t, \frac{\bx_t-\bz_t}{\|\bx_t-\bz_t\|}\right\rangle \dist_{\mathcal{V},\|\cdot\|}(\bz_t)\\
&= \tilde{\ell}_t(\bz_t)-\tilde{\ell}_t(\bu)~.
\end{align*}

Finally, the same norm bound holds for all the surrogate subgradients. For the first and third losses, this follows from the fact that every subgradient of $\dist_{\mathcal V,\|\cdot\|}$ has dual norm at most $1$, while the scalar multiplying the distance term is at most $\|\bg_t\|_\star$. For the second and fourth losses, it follows from $\|\bq_t\|_\star\le 1$.
\end{proof}

\begin{remark}
In the case that $\|\cdot\|=\|\cdot\|_2$, we can sharpen the bound on the norm of $\tilde{\bg}_t$ for the third and fourth loss. In fact, we have $\bq_t= \frac{\bz_t-\bx_t}{\|\bz_t-\bx_t\|_2}$ for $\bz_t \neq \bx_t$ and $\boldsymbol{0}$ otherwise. So, we have
\[
\|\tilde{\bg}_t\|_2^2
= \|\bg_t\|_2^2 - \left(\left\langle \bg_t, \frac{\bx_t-\bz_t}{\|\bx_t-\bz_t\|_2}\right\rangle\right)^2
\leq \|\bg_t\|_2^2~.
\]
\end{remark}

\begin{remark}
The first surrogate loss has a very natural interpretation: it is a Lipschitz penalty function that we add to the linearized original losses, where the Lipschitz constant is the same as that of the original loss. Hence, we essentially penalize the algorithm for predicting a point outside of the feasible set.
\end{remark}

\begin{remark}
\label{remark:constrained_reduction_exp_concave}
In the case that the loss functions $\ell_t$ satisfy a weaker form of strong convexity, we can show that the black-box reduction will preserve it. In particular, consider the case that we use an online learning algorithm with feasible set $\mathcal{W}$ with bounded diameter $D$ with respect to $\|\cdot\|_2$, the losses $\ell_t$ are $L$-Lipschitz with respect to $\|\cdot\|_2$ on $\mathcal{V}$ and they also satisfy
\[
\ell_t(\bx)-\ell_t(\bu)
\leq \langle \bg, \bx-\bu\rangle -\frac{\beta}{2} \langle \bg, \bx-\bu\rangle^2, \quad \forall \bu,\bx \in \mathcal{V}, \ \forall \bg \in \partial \ell_t(\bx),
\]
where $\beta \leq \frac{1}{D L}$.
For example, we saw that this is true if the functions $\ell_t$ are exp-concave\index{function!exp-concave} on $\mathcal{V}$.
Now, use projections with respect to $\|\cdot\|_2$ onto $\mathcal{V}$, and the fourth surrogate loss in Theorem~\ref{thm:constrained_reduction}, so we have
$\langle \bg_t, \bx_t -\bu\rangle \leq \langle \tilde{\bg}_t, \bz_t-\bu\rangle$, for all $\bu \in \mathcal{V}$, and $\|\tilde{\bg}_t\|_2\leq \|\bg_t\|_2 \leq L$.
Hence, we have $\langle \tilde{\bg}_t, \bz_t-\bu\rangle \leq \|\tilde{\bg}_t\|_2 \|\bz_t-\bu\|_2 \leq L D\leq \frac{1}{\beta}$ for all $\bu \in \mathcal{V}$, which implies
\[
\ell_t(\bx_t)-\ell_t(\bu)
\leq \langle \bg_t, \bx_t-\bu\rangle -\frac{\beta}{2} \langle \bg_t, \bx_t-\bu\rangle^2
\leq \langle \tilde{\bg}_t, \bz_t-\bu\rangle -\frac{\beta}{2} \langle \tilde{\bg}_t, \bz_t-\bu\rangle^2, \quad \forall \bu \in \mathcal{V},
\]
where the second inequality is due to the fact that $x \mapsto x- \frac{\beta}{2}x^2$ is increasing for $x \leq\frac{1}{\beta}$.
\end{remark}

\begin{example}
\label{example:constrained_1d}
Let's start with a simple example: we have an online algorithm that works in $\R$, and we want to constrain it to $\R_{\geq 0}$. In this case, the choice of the norms does not really matter, so we can choose L$_2$. The projection is simply $x_t=\max(z_t,0)$. The subgradient $q_t$ of the distance function can be chosen as
\[
q_t =
\begin{cases}
-1, & z_t <0\\
0,  & z_t\geq 0
\end{cases}
\]
So, using the fourth surrogate loss function in Theorem~\ref{thm:constrained_reduction}, the surrogate linear losses are $\tilde{\ell}_t(x)=\tilde{g}_t x$ where
\[
\tilde{g}_t
=\begin{cases}
\min(g_t,0), & z_t<0\\
g_t, & z_t \geq 0
\end{cases}
\]
\end{example}

\begin{example}[Regret Matching+]
\index{Regret Matching+ algorithm|(textbf}
Consider projected \ac{OGD} with learning rate $\eta=\frac{\alpha}{\sqrt{T}}$, over the feasible set $\mathcal{W}=\R^d_{\geq 0}$ and with initial prediction $\bz_1=\boldsymbol{0}$. Its regret upper bound is
\[
\sum_{t=1}^T \langle \bg_t, \bz_t - \bu \rangle
\leq \frac12 \left(\frac{\|\bu\|^2_2}{\alpha}+\alpha \max_t \|\bg_t\|^2_2\right) \sqrt{T}, \quad \forall \bu \in \R^d_{\geq 0}~.
\]
From Example~\ref{example:projection_l1}, we can project onto $\mathcal{V}=\Delta^{d-1}$ with respect to $\|\cdot\|_1$, using
\[
\bx_t=
\begin{cases}
\frac{\bz_t}{\|\bz_t\|_1}, &\bz_t\neq \boldsymbol{0}\\
[1/d, \dots, 1/d]^\top, & \bz_t= \boldsymbol{0}
\end{cases}~.
\]
Then, using any of the above surrogate losses in Algorithm~\ref{alg:constrained_reduction}, we have
\begin{align*}
\sum_{t=1}^T \langle \bg_t, \bx_t - \bu \rangle
&\leq \frac12 \left(\frac{\|\bu\|^2_2}{\alpha}+\alpha \max_t \|\tilde{\bg}_t\|^2_2\right) \sqrt{T}, \quad \forall \bu \in \Delta^{d-1},
\end{align*}
and $\|\tilde{\bg}_t\|_2 \leq \sqrt{d} \|\tilde{\bg}_t\|_\infty\leq 2 \sqrt{d} \|\bg_t\|_\infty$. Choosing $\alpha$ proportional to $1/\sqrt{d}$ gives an upper bound proportional to $\sqrt{d T}$, which is the expected one when using an algorithm with a Euclidean geometry.

However, we can do even better! Choose the specific surrogate losses in \eqref{eq:constrained_reduction_old} with the projection $\frac{\bz_t}{\|\bz_t\|_1}$, then it should be easy to realize that $\bx_t$ is independent of $\eta>0$ (left as an exercise for the reader). This means that we can set $\eta=1$, but the regret guarantee of the resulting algorithm will be the one corresponding to the best $\eta>0$ in hindsight:
\begin{align*}
\sum_{t=1}^T \langle \bg_t, \bx_t - \bu \rangle
&\leq \min_{\eta>0} \ \frac{\|\bu\|^2_2}{2\eta} + \frac{\eta}{2}\sum_{t=1}^T\|\bg_t -\ones_d\langle \bg_t, \bx_t\rangle\|^2_2\\
&= \|\bu\|_2 \sqrt{\sum_{t=1}^T\|\bg_t -\ones_d\langle \bg_t, \bx_t\rangle\|^2_2}\\
&\leq 2\sqrt{ d T} \max_t \|\bg_t\|_\infty, \quad \forall \bu \in \Delta^{d-1}~.
\end{align*}
The resulting algorithm is known as Regret Matching+ and, despite the suboptimal dependence on $d$, this is one of the best-performing algorithms in practice for the setting of \ac{LEA}. If instead of projected \ac{OGD} we use \ac{FTRL} with a fixed squared L$_2$ regularizer and feasible set $\mathcal{V}=\R^d_{\geq 0}$, we obtain the original Regret Matching algorithm, see Problem~\ref{exercise:rm}.
\index{Regret Matching+ algorithm|)textbf}
\end{example}

\index{projection!Bregman|(}
\subsection{Dealing with Constraints using Bregman Projections}

Now, we show an alternative approach based on Bregman projections.
\begin{theorem}
\label{thm:constrained_reduction_bregman}
With the notation in Algorithm~\ref{alg:constrained_reduction}, assume $\mathcal{V}\subset\mathcal{W} \subset \interior\mathcal{X}$, where $\mathcal{V}$ and $\mathcal{W}$ are nonempty, closed, and convex.
Let $\psi:\mathcal{X} \to \R$ be $\mu$-strongly convex with respect to $\|\cdot\|$ and $\beta$-smooth with respect to $\|\cdot\|$ on $\interior \mathcal{X}$.
Also, set $\Pi_\mathcal{V}$ to the Bregman projection $\Pi_{\mathcal{V}, B_{\psi}}$, and the surrogate losses
\[
\tilde{\ell}_t(\bx)
= \langle \bg_t, \bx\rangle + \begin{cases}0, & \text{ if } \langle \bg_t, \bx_t-\bz_t\rangle\leq 0\\
\frac{\langle \bg_t, \bx_t-\bz_t\rangle}{\mu \|\bx_t-\bz_t\|^2}\langle \nabla \psi(\bz_t)-\nabla \psi(\bx_t), \bx \rangle, & \text{ if } \langle \bg_t, \bx_t-\bz_t\rangle> 0
\end{cases}~.
\]
Then, we have
\[
\Regret_T(\bu)
= \sum_{t=1}^T (\ell_t(\bx_t)-\ell_t(\bu))
\leq \Regret^{\mathscr{A}}_T(\bu), \quad \forall \bu \in \mathcal{V}~.
\]
Moreover, the surrogate losses are linear, with gradients $\tilde{\bg}_t$ satisfying $\|\tilde{\bg}_t\|_\star \leq (1+\frac{\beta}{\mu})\|\bg_t\|_\star$.
\end{theorem}
\begin{proof}
Once again, we have
\begin{align*}
\Regret_T(\bu)
&= \sum_{t=1}^T(\ell_t(\bx_t)-\ell_t(\bu))
\leq \sum_{t=1}^T \langle \bg_t, \bx_t-\bu\rangle \nonumber \\
&= \sum_{t=1}^T (\langle \bg_t, \bz_t-\bu\rangle + \langle \bg_t, \bx_t-\bz_t\rangle) ~.
\end{align*}
If $\langle \bg_t, \bx_t-\bz_t\rangle\leq 0$, we can just ignore this term.
On the other hand, if $\langle \bg_t, \bx_t-\bz_t\rangle>0$, we have
\begin{align*}
\langle \bg_t, \bx_t -\bz_t\rangle
&\leq \frac{\langle \bg_t, \bx_t-\bz_t\rangle}{\mu \|\bx_t-\bz_t\|^2} \mu \|\bx_t-\bz_t\|^2\\
&\leq \frac{\langle \bg_t, \bx_t-\bz_t\rangle}{\mu \|\bx_t-\bz_t\|^2} \langle \nabla \psi(\bz_t)-\nabla \psi(\bx_t), \bz_t -\bx_t\rangle,
\end{align*}
where in the second inequality we used the strong convexity of $\psi$.
Hence, we have
\begin{align*}
\langle \bg_t, \bx_t-\bu\rangle
&\leq \langle \bg_t , \bz_t-\bu\rangle + \frac{\langle \bg_t, \bx_t-\bz_t\rangle}{\mu \|\bx_t-\bz_t\|^2} \langle \nabla \psi(\bz_t)-\nabla \psi(\bx_t), \bz_t -\bx_t\rangle\\
&\leq \left\langle \bg_t + \frac{\langle \bg_t, \bx_t-\bz_t\rangle}{\mu \|\bx_t-\bz_t\|^2} (\nabla \psi(\bz_t)-\nabla \psi(\bx_t)), \bz_t-\bu\right\rangle,
\end{align*}
where in the second inequality we used the first-order optimality condition
\[
\langle \nabla \psi(\bz_t)-\nabla \psi(\bx_t), \bu - \bx_t\rangle \leq 0, \quad \forall \bu \in \mathcal{V}~.
\]

For the bound on $\|\tilde{\bg}_t\|_\star$, observe that when $\langle \bg_t, \bx_t-\bz_t\rangle>0$ we have
\begin{align*}
\|\tilde{\bg}_t\|_\star
&= \left\|\bg_t + \frac{\langle \bg_t, \bx_t-\bz_t\rangle}{\mu \|\bx_t-\bz_t\|^2}(\nabla \psi(\bz_t)-\nabla \psi(\bx_t))\right\|_\star\\
&\leq \|\bg_t\|_\star + \frac{\langle \bg_t, \bx_t-\bz_t\rangle}{\mu \|\bx_t-\bz_t\|^2}  \|\nabla \psi(\bz_t)-\nabla \psi(\bx_t)\|_\star \\
&\leq \|\bg_t\|_\star + \frac{\langle \bg_t, \bx_t-\bz_t\rangle}{\mu \|\bx_t-\bz_t\|^2}  \beta \|\bz_t-\bx_t\|
\leq \left(1+\frac{\beta}{\mu}\right)\|\bg_t\|_\star~.
\end{align*}
The same bound clearly holds even in the case that $\langle \bg_t, \bx_t-\bz_t\rangle\leq0$.
\end{proof}

\begin{example}
\label{example:bregman_proj_kl}
Let's calculate $\Pi_{\Delta^{d-1}, B_{\psi}}(\bx)$ where $\bx \in \R^d_{> 0}$, $\psi(\bx)=\sum_{i=1}^d x_i \ln x_i$, and we define $0 \ln 0 =0$.
Hence, we have to solve the following optimization problem
\[
\argmin_{\by \in \Delta^{d-1}} \ \sum_{i=1}^d y_i \ln \frac{y_i}{x_i}~.
\]
We can rewrite it as an unconstrained optimization over $d-1$ variables, as
\[
\argmin_{\by \in \R^{d-1}_{\geq0}} \ \sum_{i=1}^{d-1} y_i \ln \frac{y_i}{x_i} + \left(1-\sum_{i=1}^{d-1} y_i\right) \ln \frac{1-\sum_{i=1}^{d-1} y_i}{x_d}~.
\]
Since $x_i>0$ for all $i$, the optimal solution has $y_1,\dots,y_{d-1}$ and $1-\sum_{i=1}^{d-1} y_i$ strictly positive. Hence, the inequality constraints are inactive, and we can use the first-order conditions:
\[
\ln \frac{y_i}{x_i}+1 - \ln \frac{1-\sum_{j=1}^{d-1} y_j}{x_d}-1 = 0, \quad i=1, \dots, d-1,
\]
that gives
\begin{equation}
\label{eq:example:bregman_proj_kl_eq1}
y_i = x_i \frac{1-\sum_{j=1}^{d-1} y_j}{x_d},  \quad i=1, \dots, d-1~.
\end{equation}
Summing over $i=1, \dots, d-1$, we have $\sum_{i=1}^{d-1} y_i = \frac{1-\sum_{i=1}^{d-1} y_i}{x_d} \sum_{i=1}^{d-1} x_i$. Solving it, we have $\frac{1-\sum_{i=1}^{d-1} y_i}{x_d}=\frac{1}{\sum_{i=1}^{d} x_i}$ and using it in \eqref{eq:example:bregman_proj_kl_eq1} gives $y_i = \frac{x_i}{\sum_{j=1}^d x_j}$, for $i=1, \dots, d$.
\end{example}
\index{projection!Bregman|)}

\index{black-box reduction!for constrained online convex optimization|)textbf}

\section{Sleeping Experts}
\label{sec:sleeping_experts}
\index{black-box reduction!for sleeping experts|(textbf}

\index{sleeping experts|(textbf}
Consider now the setting of \ac{LEA} where only a subset of the experts is active in each round. In particular, we have that $a_{t,i}=1$ if expert $i$ is active at time $t$, and $a_{t,i}=0$ if the expert is inactive, that is, \emph{sleeping}.
This setting is useful in the case that some experts might become unavailable in some rounds, but also to model the case that the set of experts is growing over time.

The online algorithm receives at the beginning of each round the information about which experts will be active, and it outputs a distribution over the active experts. So, we constrain the \ac{LEA} algorithm to produce a probability distribution only over the active experts. We also assume that there is always at least one active expert.

Here, we will use a different notion of regret, in particular, our regret with respect to any expert $i$ is defined as\index{regret!sleeping experts|textbf}
\[
\Regret^\text{sleeping}_T(\be_i)
:=\sum_{t=1}^T a_{t,i} (\langle \bg_t, \bx_t\rangle - g_{t,i})~.
\]
In words, we measure the regret against expert $i$ only on rounds where the expert was active.
This notion can be generalized to $\bu\in\R^d_{\geq0}$ as
\[
\Regret^\text{sleeping}_T(\bu)
:= \sum_{t=1}^T \sum_{i=1}^d u_i a_{t,i} (\langle \bg_t, \bx_t\rangle- g_{t,i}), \quad \forall \bu \in \R^d_{\geq 0}~.
\]
If $\sum_i u_i a_{t,i}\leq 1$ for all $t$, then $\bu$ can be interpreted as a fixed sleeping comparator that assigns total mass at most one to the active experts on each round.
In the case that all the experts are active on all rounds and $\bu \in \Delta^{d-1}$, this notion recovers the usual one. However, in the general case, this is a different notion than the one we used in \ac{OLO}.

The reduction we consider transforms an \ac{OLO} learner over $\mathcal{V}=\R^d_{\geq 0}$ to a sleeping expert algorithm.
Hence, we will transform a vector $\bz_t \in \R^d_{\geq 0}$ into a vector in the simplex defined by the active experts, and we also need appropriate surrogate losses. The reduction will be similar to the one we used in the previous section.
For the first part, we construct a probability distribution as
\[
x_{t,i} = \frac{a_{t,i} z_{t,i}}{\sum_{j=1}^d a_{t,j} z_{t,j}}, \quad \forall i=1, \dots, d,
\]
when $\sum_{j=1}^d a_{t,j} z_{t,j}\neq 0$ and an arbitrary one otherwise.
For the second part, from the original linear losses $\ell_t(\bx)=\langle \bg_t, \bx\rangle$ we construct the modified losses $\tilde{\ell}_t(\bx)=\langle\tilde{\bg}_t,\bx\rangle$, where $\tilde{g}_{t,i}=a_{t,i} (g_{t,i}-\langle \bg_t,\bx_t\rangle)$. The overall algorithm is in Algorithm~\ref{alg:sleeping_experts}.

\begin{algorithm}[t]
\caption{Sleeping Experts Reduction}
\label{alg:sleeping_experts}
\begin{algorithmic}[1]
{
    \REQUIRE{Online linear optimization algorithm $\mathscr{A}$ with feasible set $\mathcal{V}=\R^d_{\geq 0}$}
    \FOR{$t=1$ {\bfseries to} $T$}
    \STATE{Get $\bz_t \in \R^d_{\geq 0}$ from $\mathscr{A}$}
    \STATE{Receive the state of the experts $a_{t,i} \in \{0,1\}, i=1, \dots, d$}
    \IF{$\sum_{j=1}^d a_{t,j} z_{t,j}\neq 0$}
    \STATE{Output $x_{t,i}=\frac{a_{t,i} z_{t,i}}{\sum_{j=1}^d a_{t,j} z_{t,j}}, \ i=1, \dots, d$}
    \ELSE
    \STATE{Output an arbitrary distribution $\bx_{t}$ over the active experts}
    \ENDIF
    \STATE{Receive $\bg_t \in \R^d$}
    \STATE{Send the surrogate loss $\tilde{\ell}_t(\bx)=\langle \tilde{\bg}_t, \bx\rangle$ to $\mathscr{A}$, where $\tilde{g}_{t,i}=a_{t,i} (g_{t,i} - \langle \bg_t, \bx_t\rangle), \ i=1, \dots, d$}
    \ENDFOR
}
\end{algorithmic}
\end{algorithm}

With the above definitions, we have
\[
\langle \tilde{\bg}_{t},\bz_t\rangle
= \sum_{i=1}^d a_{t,i}(g_{t,i}-\langle \bg_t,\bx_t\rangle ) z_{t,i}
= \sum_{i=1}^d a_{t,i} z_{t,i} g_{t,i} -\langle \bg_t,\bx_t\rangle \sum_{i=1}^d a_{t,i} z_{t,i}~.
\]
If $\sum_{i=1}^d a_{t,i} z_{t,i}=0$, then, since $a_{t,i}\in\{0,1\}$ and $z_{t,i}\ge 0$, we must have $a_{t,i}z_{t,i}=0$, for all $i=1,\dots,d$.
Therefore, $\langle \tilde{\bg}_{t},\bz_t\rangle = 0$.

On the other hand, if $\sum_{i=1}^d a_{t,i} z_{t,i}\neq 0$, then by the definition of $\bx_t$ we have
\[
\sum_{i=1}^d a_{t,i} z_{t,i} g_{t,i}
= \left(\sum_{j=1}^d a_{t,j} z_{t,j}\right)\langle \bg_t,\bx_t\rangle~.
\]
Substituting this into the previous expression gives $\langle \tilde{\bg}_{t},\bz_t\rangle = 0$.

Hence, in all cases, $\langle \tilde{\bg}_{t},\bz_t\rangle = 0$.
In turn, this implies that
\[
a_{t,i}(\langle \bg_t,\bx_t\rangle - g_{t,i})
= -\tilde{g}_{t,i}
= \langle \tilde{\bg}_{t},\bz_t\rangle - \tilde{g}_{t,i}, \quad i=1,\dots,d~.
\]

In words, we can construct surrogate losses to transform the sleeping expert problem into an \ac{OLO} problem over $\R^d_{\geq 0}$, obtaining that
\[
\Regret^\text{sleeping}_T(\bu)
= \sum_{t=1}^T \sum_{i=1}^d u_i a_{t,i} (\langle \bg_t, \bx_t\rangle- g_{t,i})
= \sum_{t=1}^T \langle \tilde{\bg}_{t},\bz_t-\bu\rangle, \quad \forall \bu \in \R_{\geq 0}^d~.
\]

Formally, we have the following theorem.
\begin{theorem}
\label{thm:sleeping_experts_reduction}
Let $a_{t,i}\in\{0,1\}$ denote whether expert $i$ is active at time $t$, and assume that at least one expert is active on every round.
Run Algorithm~\ref{alg:sleeping_experts} with an \ac{OLO} algorithm $\mathscr{A}$ over $\R^d_{\geq 0}$.
If $\mathscr{A}$ guarantees
\[
\sum_{t=1}^T \langle \tilde{\bg}_t,\bz_t-\bu\rangle
\leq R_T^\mathscr{A}(\bu),
\quad \forall \bu\in\R^d_{\geq0},
\]
then Algorithm~\ref{alg:sleeping_experts} satisfies
\[
\Regret^\text{sleeping}_T(\bu)
= \sum_{t=1}^T \sum_{i=1}^d u_i a_{t,i} (\langle \bg_t, \bx_t\rangle- g_{t,i})
\leq R_T^\mathscr{A}(\bu), \quad \forall \bu\in\R^d_{\geq0}~.
\]
Moreover, if the active losses have range at most $R$ on every round, that is,
$\max_{i:a_{t,i}=1} g_{t,i} - \min_{i:a_{t,i}=1} g_{t,i} \leq R$, then $\|\tilde{\bg}_t\|_\infty\leq R$ for all $t=1,\dots,T$.
\end{theorem}

\begin{remark}
\label{remark:sleeping_experts}
Clearly, if we have an algorithm for \ac{LEA}, we can use it in the reduction because $\Delta^{d-1}\subset \R^d_{\geq 0}$.
\end{remark}

\begin{remark}
The above definitions and reduction generalize to the setting where $a_{t,i} \in [0,1]$, see Problem~\ref{exercise:confidence_sleeping_experts}. Here, $a_{t,i}$ denotes the confidence that expert $i$ has in its prediction at time $t$, where $0$ means that the expert abstains from producing a prediction.
\end{remark}

\index{sleeping experts|)textbf}
\index{black-box reduction!for sleeping experts|)textbf}

\section{Reduction to Magnitude and Direction}
\label{sec:mag_dir}

\index{black-box reduction!to magnitude and direction|(textbf}

Here, we will see how to use a one-dimensional algorithm for unconstrained \ac{OCO} to solve an \ac{OCO} problem in any number of dimensions and using any (reasonable) norm.

This reduction requires an unconstrained \ac{OCO} algorithm for the one-dimensional case and an algorithm for learning in $d$-dimensional (or infinite-dimensional) balls.
Given these two learners, we decompose the problem of learning a vector $\bx_t$ into the problem of learning a \emph{direction} and a \emph{magnitude}. The regret of this procedure turns out to be just the sum of the regret of the two learners.
We can formalize this idea in the following theorem.

\begin{algorithm}[t]
   \caption{Learning Magnitude and Direction Separately}
   \label{alg:onedimred}
\begin{algorithmic}[1]
   \REQUIRE One-dimensional online learning algorithm $\mathscr{A}_{\text{1d}}$, Online learning algorithm $\mathscr{A}_{B}$ with feasible set equal to the unit ball $\mathcal{B}\subset \R^d$ with respect to $\|\cdot\|$
   \FOR{$t=1$ {\bfseries to} $T$}
   \STATE Get point $z_t\in \R$ from $\mathscr{A}_{\text{1d}}$, and get point $\tilde{\bx}_t \in \mathcal{B}$ from $\mathscr{A}_{B}$
   \STATE Play $\bx_t = z_t \tilde{\bx}_t\in \R^d$
   \STATE{Pay the loss $\ell_t(\bx_t)$, where $\ell_t$ is subdifferentiable on $\R^d$}
   \STATE Set $\bg_t \in \partial \ell_t(\bx_t)$ and $s_t = \langle \bg_t, \tilde{\bx}_t\rangle$
   \STATE Send $\ell^{\mathscr{A}_{\text{1d}}}_t(x)=s_t x$ as the $t$-th linear loss to $\mathscr{A}_{\text{1d}}$
   \STATE Send $\ell_t^{\mathscr{A}_{B}}(\bx)=\langle \bg_t, \bx\rangle$ as the $t$-th linear loss to $\mathscr{A}_{B}$
   \ENDFOR
\end{algorithmic}
\end{algorithm}

\begin{theorem}
\label{thm:reduction_direction_magnitude}
Denote by $\Regret^{\mathscr{A}_{B}}_T(\bu)$ the linear regret of algorithm $\mathscr{A}_{B}$ for any $\bu$ in the unit ball with respect to a norm $\|\cdot\|$, and $\Regret^{\mathscr{A}_{\text{1d}}}_T(u)$ the linear regret of algorithm $\mathscr{A}_{\text{1d}}$ for any competitor $u \in\R$. Then, for any $\bu \in \R^d\setminus\{\boldsymbol{0}\}$, Algorithm~\ref{alg:onedimred} guarantees
\begin{align*}
\Regret_T(\bu)
&= \sum_{t=1}^T (\ell_t(\bx_t) - \ell_t(\bu))
\leq \sum_{t=1}^T \langle \bg_t, \bx_t -\bu\rangle\\
&= \Regret^{\mathscr{A}_{\text{1d}}}_T(\|\bu\|) + \|\bu\| \Regret^{\mathscr{A}_{B}}_T\left(\frac{\bu}{\|\bu\|}\right),
\end{align*}
and $\Regret_T(\boldsymbol{0})\leq \Regret^{\mathscr{A}_{\text{1d}}}_T(0)$.
Further, we have $|s_t|\le \|\bg_t\|_\star$ for all $t$.
\end{theorem}
\begin{proof}
First, observe that $|s_t|\le \|\bg_t\|_\star \|\tilde{\bx}_t\| \le \|\bg_t\|_\star$ since $\|\tilde{\bx}_t\|\le 1$ for all $t$. Now, assuming $\bu \neq \boldsymbol{0}$ compute:
\begin{align*}
\Regret_T(\bu)
&=\sum_{t=1}^T \ell_t(\bx_t) - \sum_{t=1}^T \ell_t(\bu)
\leq \sum_{t=1}^T \langle \bg_t, \bx_t - \bu\rangle
= \sum_{t=1}^T \langle \bg_t, z_t \tilde{\bx}_t\rangle - \langle \bg_t, \bu\rangle\\
&= \underbrace{\sum_{t=1}^T \left(\langle \bg_t, \tilde{\bx}_t \rangle z_t - \langle \bg_t, \tilde{\bx}_t\rangle \|\bu\|\right)}_{\text{linear regret of }\mathscr{A}_{\text{1d}}\text{ at } \|\bu\| \in \R} + \sum_{t=1}^T \left(\langle \bg_t, \tilde{\bx}_t \rangle \|\bu\|  - \langle \bg_t, \bu \rangle \right)\\
&= \Regret^{\mathscr{A}_{\text{1d}}}_T(\|\bu\|) + \sum_{t=1}^T \left(\langle \bg_t, \tilde{\bx}_t\rangle \|\bu\|  - \langle \bg_t, \bu\rangle \right)\\
&= \Regret^{\mathscr{A}_{\text{1d}}}_T(\|\bu\|) +\|\bu\| \underbrace{\sum_{t=1}^T \left(\langle \bg_t, \tilde{\bx}_t\rangle - \left\langle \bg_t, \frac{\bu}{\|\bu\|}\right\rangle \right)}_{\text{linear regret of }\mathscr{A}_{B}\text{ at } \frac{\bu}{\|\bu\|} \in \R^d}\\
&= \Regret^{\mathscr{A}_{\text{1d}}}_T(\|\bu\|) + \|\bu\| \Regret^{\mathscr{A}_{B}}_T\left(\frac{\bu}{\|\bu\|}\right)~.
\end{align*}
The case $\bu = \boldsymbol{0}$ follows similarly.
\end{proof}

\begin{remark}
Note that the direction vector is not constrained to have norm equal to 1, which would be a non-convex constraint, yet this does not seem to affect the proof.
\end{remark}

We will use this theorem in Section~\ref{sec:parameterfree-direction-magnitude}.
\index{black-box reduction!to magnitude and direction|)textbf}

\section{History Bits}
\index{black-box reduction!for constrained online convex optimization|(}
As far as I know, the idea of using a black-box reduction for constrained optimization does not appear in the offline and stochastic optimization literature.
The reduction from learning from $\R^d_{\geq 0}$ to $\Delta^{d-1}$ is a classic approach hidden in most of the \ac{LEA} algorithms. \citet[Section 2.1]{Cesa-BianchiL06} describes a similar version but through the formalism of potential functions\index{potential function}, which does not result in a black-box reduction.
\citet{CutkoskyO18} proposed the first reduction from constrained \ac{OCO} to unconstrained \ac{OCO} described in Section~\ref{sec:constrained_reduction}. However, their characterization of the subgradient of the distance-to-a-set function is wrong, and I fixed it here. \citet{Cutkosky20} proposed the third surrogate loss. The second and fourth surrogates are new. \citet{FarinaKS19b} independently of the above work proved Theorem~\ref{thm:constrained_reduction_bregman}.
Remark~\ref{remark:constrained_reduction_exp_concave} is taken from the method in \citet{MhammediG23}.
\index{black-box reduction!for constrained online convex optimization|)}

\index{Regret Matching algorithm|(}
\index{Regret Matching+ algorithm|(}
Regret Matching (RM) was proposed by \citet{HartMC00} and Regret Matching+ (RM+) by \citet{TammelinBJB15}. RM was developed to find correlated equilibria in two-player games and is commonly used to minimize regret over the simplex. RM+ is a modification of RM designed to accelerate convergence and used to effectively solve the game of Heads-up Limit Texas Hold'em poker \citep{BowlingBBJT15}.
While the exact reason why RM+ has better empirical convergence than any other similar algorithm is unclear, I suggested on X.com on February 14th, 2020, that it might be due to its scale-freeness\index{algorithm!scale-free}.\footnote{\url{https://x.com/bremen79/status/1228318344195035142}}
The observation that RM can be expressed as \ac{FTRL} is due to Nicol\`{o} Cesa-Bianchi in an unpublished manuscript titled ``The Joys of Regret Matching'' dated October 14, 2015 (see Problem~\ref{exercise:rm}). That result was shared with a number of people (including me) and then later appeared in some papers too, unfortunately, without proper credit. The observation that RM+ can be obtained by projected \ac{OGD} is by \citet{FarinaKS21} and independently by \citet{FlaspohlerOCMOOM21}.
RM and RM+ were extended to $p$-norms by \citet{Cesa-BianchiL01,Cesa-BianchiL03}\footnote{Note that their proof contains a small error that is fixed in the errata of \citet{Cesa-BianchiL06}.}
and \cite{DOrazio20}, respectively. Tuning $p$ with knowledge of $d$, one gets the optimal dependence in $d$ but still without a learning rate or a regularizer weight to tune. Later, \citet{FlaspohlerOCMOOM21} rediscovered both of these results.
\index{Regret Matching algorithm|)}
\index{Regret Matching+ algorithm|)}

\index{sleeping experts|(}
The setting of sleeping experts has been proposed by \citet{Blum97} and \citet{FreundSSW97}.
\index{black-box reduction!for sleeping experts|(}
There is also a different notion of regret for the sleeping-experts setting based on the best ordering of the experts~\citep{KleinbergNMS08, KleinbergNMS10}; however, it does not seem well suited to combining different algorithms as experts, as we will see in Section~\ref{sec:strongly_adaptive_regret}.
The reduction in Section~\ref{sec:sleeping_experts} is an extension of the one from \citet{GaillardSV14} that was designed to reduce the sleeping experts problem to the \ac{LEA} problem rather than \ac{OLO} in $\R^d_{\geq 0}$. I also added the minor improvement of removing the constant term from the surrogate losses. Such reduction is implicitly used by \citet{LuoS15} and by \citet{JunOWW17,JunOWW17b}.
\index{sleeping experts|)}
\index{black-box reduction!for sleeping experts|)}

\index{black-box reduction!to magnitude and direction|(}
Theorem~\ref{thm:reduction_direction_magnitude} is from \citet{CutkoskyO18}. Note that the original theorem is more general because it works even in Banach spaces.
\index{black-box reduction!to magnitude and direction|)}

\section{Exercises}

\begin{exer}
Find a solution to the projection from $\R^d$ to $\Delta^{d-1}$ with respect to $\|\cdot\|_1$.
\end{exer}

\begin{exer}
\label{exercise:rm}
\index{Regret Matching algorithm|(}
The Regret Matching algorithm predicts with
\[
x_{t,i}
=\frac{\max\left(\sum_{n=1}^{t-1} (g_{n,i}-\langle \bg_n,\bx_n\rangle),0\right)}{\sum_{j=1}^d \max\left(\sum_{n=1}^{t-1} (g_{n,j}-\langle \bg_n,\bx_n\rangle),0\right)}, \quad i=1, \dots, d,
\]
if $\sum_{j=1}^d \max\left(\sum_{n=1}^{t-1} (g_{n,j}-\langle \bg_n,\bx_n\rangle),0\right)\neq 0$ and with an arbitrary distribution otherwise.
Prove that it corresponds to \ac{FTRL} with a squared L$_2$ regularizer and $\mathcal{V}=\R^d_{\geq 0}$ with projections onto $\Delta^{d-1}$ with respect to $\|\cdot\|_1$. Then, prove a regret upper bound proportional to $\sqrt{d T}$. Hint: see \citet{FarinaKS21} and \citet{FlaspohlerOCMOOM21}.
\index{Regret Matching algorithm|)}
\end{exer}

\begin{exer}
\label{exercise:confidence_sleeping_experts}
Extend Theorem~\ref{thm:sleeping_experts_reduction} to the case where $a_{t,i}\in[0,1]$ is the confidence of expert $i$ at time $t$.
Find a suitable prediction strategy $\bx_t$ and surrogate losses $\tilde{\bg}_t$ such that, for every $\bu\in\R^d_{\geq0}$, we have
\[
\sum_{t=1}^T \sum_{i=1}^d u_i a_{t,i}\bigl(\langle \bg_t,\bx_t\rangle-g_{t,i}\bigr)
= \sum_{t=1}^T \langle \tilde{\bg}_t,\bz_t-\bu\rangle~.
\]
\end{exer}

\index{black-box reduction|)}

\acresetall

\chapter{Parameter-free Online Linear Optimization}
\label{ch:parameterfree}

In the previous sections, we have shown that \ac{OMD} and \ac{FTRL} achieve a regret of $\mathcal{O}(\sqrt{T})$ for convex Lipschitz losses. We have also shown that for bounded domains, these bounds are optimal up to constant multiplicative factors.
However, in the unbounded case, the bounds we get are suboptimal with respect to the dependence on the competitor.

For example, \ac{OSD}, with $\mathcal{V}=\R^d$ and $\bx_1=\boldsymbol{0}$ over $L$-Lipschitz losses and learning rate $\eta=\frac{\alpha}{L\sqrt{T}}$, gives the regret guarantee
\[
\Regret_T(\bu)
= \sum_{t=1}^T (\ell_t(\bx_t) - \ell_t(\bu))
\leq \frac{\|\bu\|_2^2}{2\eta} + \frac{\eta L^2 T}{2}
= \frac{L \sqrt{T}}{2} \left(\frac{\|\bu\|_2^2}{\alpha}+\alpha\right), \ \forall \bu \in \R^d.
\]
So, to get the best possible guarantee, we should know $\|\bu\|_2$ and set $\alpha=\|\bu\|_2$. As we said, this strategy does not work for a few reasons: (i) $\bu$ is not a fixed vector, rather the regret is a function of $\bu \in \mathcal{V}$; (ii) if we guessed any value of $\|\bu\|_2$ the adversary could easily change the losses to make that value completely wrong; (iii) the resulting regret would violate the lower bound in Chapter~\ref{ch:lower}.

Let's show another example. Use \ac{OSD} with the online-to-batch conversion to minimize a function that is $L$-Lipschitz, whose minimizer is $\bx^\star$. Setting $\bx_1=\boldsymbol{0}$, the convergence rate will be $\frac{L(\|\bx^\star\|_2^2/\alpha+\alpha)}{\sqrt{T}}$ using a learning rate of $\eta=\frac{\alpha}{L \sqrt{T}}$. Consider the case that $\|\bx^\star\|_2=100$, specifying $\alpha=1$ will result in a convergence rate about 50 times slower than specifying the optimal choice in hindsight, $\alpha=100$.
Note that this is a real effect, not an artifact of the proof. Indeed, it is intuitive that the optimal learning rate should be proportional to the distance between the initial point and the optimal solution.

In the following, we will see that it is possible to reduce any \ac{OCO} game to betting on a non-stochastic coin. This will allow us to very easily design \ac{OCO} algorithms that will enjoy (almost) \emph{the optimal regret uniformly over all competitors}. We call these kinds of algorithms \emph{parameter-free}.

\acresetall

\section{The Coin-Betting Game}
\label{sec:coin}

\index{coin-betting game|(textbf}
We now consider again the coin-betting game we described in Section~\ref{sec:lower_unconstrained_olo}. Later, we will show that the solution to this problem allows us to design optimal \ac{OCO} algorithms.

Consider the following repeated game:
\begin{itemize}
\item Set the initial wealth to $\epsilon$: $\Wealth_0=\epsilon$.
\item In each round $t=1,\dots,T$
\begin{itemize}
\item You bet $|x_t|$ money on the side of the coin equal to $\sign(x_t)$. You cannot bet more money than what you currently have, hence $|x_t|\leq \Wealth_{t-1}$.
\item The adversary reveals the outcome of the coin $c_t \in \{-1, 1\}$.
\item You gain money $x_t c_t $, that is $\Wealth_t=\Wealth_{t-1}+c_t x_t = \epsilon + \sum_{i=1}^{t} c_i x_i$.\index{wealth!in coin-betting game}
\end{itemize}
\end{itemize}

Given that we cannot borrow money, we can codify the bets $x_t$ as $\beta_t \Wealth_{t-1}$, where $\beta_t \in [-1,1]$ is the signed betting fraction\index{signed betting fraction}. So, $|\beta_t|$ is the fraction of money to bet and $\sign(\beta_t)$ the side of the coin on which we bet.

The aim of the game is to make as much money as possible. However, given the game's adversarial nature, we cannot expect to always win money.
Instead, we try to gain as much money as the strategy that bets a fixed signed fraction of money $\beta^\star \in [-1, 1]$ for the entire game.

Note that
\begin{align*}
\Wealth_t
&= \Wealth_{t-1} + c_t x_t
=\Wealth_{t-1} + \beta_t \Wealth_{t-1} c_t
= \Wealth_{t-1} (1 + \beta_t c_t)\\
&= \epsilon\prod_{i=1}^{t} (1+ \beta_i c_i)~.
\end{align*}
So, given the multiplicative nature of the wealth, it is also useful to take the logarithm of the ratio of the wealth of the algorithm and the wealth of the optimal betting fraction. Hence, we want to minimize the following regret
\begin{align*}
\ln \max_{\beta \in [-1,1]} \epsilon \prod_{t=1}^T (1+\beta c_t) - \ln \Wealth_T
&= \max_{\beta \in [-1,1]} \sum_{t=1}^T \ln(1+\beta c_t) - \sum_{t=1}^T \ln(1+\beta_t c_t)~.
\end{align*}
This is nothing else than the regret of an \ac{OCO} game where the losses are $\ell_t(x)=-\ln(1+x c_t)$ and $\mathcal{V}=[-1,1]$, but here the losses are allowed to take the value $+\infty$.
We can also extend a bit the formulation allowing ``continuous coins''\index{continuous coin}, where $c_t \in [-1, 1]$ rather than in $\{-1, 1\}$.

\begin{remark}
The constraint to bet a fraction between $-1$ and $1$ is not strictly necessary. We could allow the algorithm to bet more money than it currently has, lending it some money in each round. However, the restriction makes the analysis easier because it allows the transformation above into an \ac{OCO} problem, using the non-negativity of $1+\beta_t c_t$.
\end{remark}

Taking special care of the non-Lipschitzness of the functions, we could just use \ac{OMD} or \ac{FTRL}, but it turns out that there exists a better strategy specifically for this problem: the \textbf{\ac{KT} bettor}.\index{Krichevsky--Trofimov!bettor|(textbf}
The \ac{KT} strategy simply says that on each time step $t$ you bet $\beta_t=\frac{\sum_{i=1}^{t-1} c_i}{t}$.
So, the algorithm is the following one.
\begin{algorithm}[h]
\caption{Krichevsky--Trofimov (KT) Bettor}
\label{alg:kt}
\begin{algorithmic}[1]
{
    \REQUIRE{Initial money $\Wealth_0=\epsilon>0$}
    \FOR{$t=1$ {\bfseries to} $T$}
    \STATE{Calculate the betting fraction $\beta_t=\frac{\sum_{i=1}^{t-1} c_i}{t}$}
    \STATE{Bet $x_t=\beta_t \Wealth_{t-1}$, that is $|x_t|$ money on the side $\sign(x_t)=\sign(\beta_t)$}
    \STATE{Receive the coin outcome $c_t$ in $\{-1,1\}$ or in $[-1,1]$}
    \STATE{Win/lose $x_t c_t$, that is $\Wealth_t=\Wealth_{t-1} + c_t x_t$}
    \ENDFOR
}
\end{algorithmic}
\end{algorithm}

\begin{theorem}
\label{thm:kt}
Let $c_t \in \{-1, 1\}$ for $t=1, \dots, T$ arbitrarily chosen. Then, the \ac{KT} bettor in Algorithm~\ref{alg:kt} guarantees
\[
\ln \Wealth_T \geq \ln \max_{\beta \in [-1,1]} \epsilon \prod_{t=1}^T (1+\beta c_t) - \frac{1}{2} \ln T - 1~.
\]
\end{theorem}
\begin{proof}
There are many ways to prove a regret bound for the \ac{KT} algorithm. One instructive way is to show that the \ac{KT} algorithm is an instantiation of the \ac{WAA} algorithm we saw in Chapter~\ref{ch:waa}. First of all, we set $\epsilon=1$ without loss of generality, because its value does not affect the regret. Then, consider the \ac{WAA} algorithm where we set the losses to $\ell_t(x)=-\ln(1+c_t x)$ where $c_t \in \{-1, 1\}$ with the feasible set $\mathcal{V}=[-1,1]$, where we define $-\ln 0:=+\infty$. It is immediate to verify that these losses are 1-exp-concave\index{function!exp-concave} on $\mathcal{V}$.
Finally, we set $P_1(\mathrm{d} x)=\frac{1}{\pi \sqrt{(1-x)(1+x)}}$ over $\mathcal{V}$.

Given that $c_t \in \{-1, 1\}$, we define $a_t=\sum_{i=1}^t \indevent{\{c_i=1\}}$ and $b_t=\sum_{i=1}^t \indevent{\{c_i=-1\}}$, to have the following closed-form expressions:
\begin{align*}
\E_{x \sim P_1}\left[\prod_{i=1}^t (1+c_i x)\right]
&= \frac{2^{a_t+b_t} \Gamma(a_t+1/2) \Gamma(b_t+1/2)}{\pi \Gamma(a_t+b_t+1)},\\
\E_{x \sim P_1}\left[x \prod_{i=1}^t (1+c_i x)\right]
&= \frac{2^{a_t+b_t}\, (a_t-b_t) \, \Gamma(a_t+1/2)\, \Gamma(b_t+1/2)}{\pi\, (a_t+b_t+1)\,\Gamma(a_t+b_t+1)},\\
\max_v \ \prod_{t=1}^T (1+c_t v)
&=\max_v \ (1+v)^{a_T} (1-v)^{b_T}\\
&= \left(1+\frac{a_T-b_T}{T}\right)^{a_T} \left(1-\frac{a_T-b_T}{T}\right)^{b_T}
= 2^T a_T^{a_T} b_T^{b_T} T^{-T},
\end{align*}
where $0^0:=1$.
With the above formulas, we can calculate that
\[
P_{t+1}(\mathrm{d}x)=\frac{\pi\, \Gamma(a_t+b_t+1) \prod_{i=1}^t (1+c_i x) P_1(\mathrm{d} x)}{2^{a_t+b_t}\, \Gamma(a_t+1/2)\, \Gamma(b_t+1/2)},
\]
and
\[
x_{t+1}
=\E_{x \sim P_{t+1}}[x]
=\frac{a_t-b_t}{a_t+b_t+1}
=\frac{\sum_{i=1}^t c_i}{t+1},
\]
that corresponds to the one in Algorithm~\ref{alg:kt}.

Using the second result in Theorem~\ref{thm:waa}, for all $u \in [-1,1]$, we have
\begin{align*}
\sum_{t=1}^T (\ell_t(x_t)-\ell_t(u))
&\leq \ln \frac{\max_{v \in [-1,1]} \ \prod_{t=1}^T (1+c_t v)}{\E_{x \sim P_1} [\prod_{t=1}^T (1+c_t x)]}\\
&\leq \ln \frac{2^T a_T^{a_T} b_T^{b_T} T^{-T} \, \pi \, \Gamma(a_T+b_T+1)}{2^{a_T+b_T}\,\Gamma(a_T+1/2) \, \Gamma(b_T+1/2) }\\
&= \ln \frac{a_T^{a_T} b_T^{b_T} T^{-T} \pi \, \Gamma(T+1)}{\Gamma(a_T+1/2) \, \Gamma(b_T+1/2) }~.
\end{align*}
Using Lemma~\ref{lemma:max_universal_portfolio}, we have that this expression is maximized for $a_T=T$ and $b_T=0$ (or equivalently $a_T=0$ and $b_T=T$), so we obtain
\begin{align*}
\sum_{t=1}^T (\ell_t(x_t)-\ell_t(u))
\leq \ln \frac{\pi \, \Gamma(T+1)}{\Gamma(T+1/2) \, \Gamma(1/2) }
= \ln \frac{\sqrt{\pi} \, \Gamma(T+1)}{\Gamma(T+1/2) }
\leq \frac12\ln T + 1,
\end{align*}
where in the last inequality we used Lemma~\ref{lemma:bound_regret_gamma}.
\end{proof}

Note that if the outcomes of the coin are skewed towards one side, the a posteriori optimal betting fraction will gain an exponential amount of money, as proved in the next lemma.
\begin{lemma}
\label{lemma:lower_bound_wealth}
Let $c_t\in \{-1,1\}, t=1,\dots,T$ arbitrarily chosen and denote by $C_T=\sum_{t=1}^T c_t$. Then, we have
\[
\max_{\beta \in [-1,1]} \ \prod_{t=1}^T (1+\beta c_t)
= \left(1+\frac{C_T}{T}\right)^\frac{T+C_T}{2} \left(1-\frac{C_T}{T}\right)^\frac{T-C_T}{2}
\geq \exp\left(\frac{C_T^2}{2T}\right)~.
\]
\end{lemma}
\begin{proof}
Through a derivative argument, the equality is derived by finding that
\[
\argmax_{\beta \in [-1,1]} \ \sum_{t=1}^T \ln(1+\beta c_t)
=C_T/T~.
\]

The inequality is true because
\begin{align*}
&\left(1+C_T/T\right)^\frac{T+C_T}{2} \left(1-C_T/T\right)^\frac{T-C_T}{2}\\
&\quad=\exp\left[T \cdot \KL\left(\left[\frac12 + \frac{C_T}{2T}, \frac12 - \frac{C_T}{2T}\right]^\top;\left[\frac12 , \frac12\right]^\top\right)\right],
\end{align*}
where $\KL$ is the \ac{KL} divergence\index{Kullback--Leibler divergence}.
Then, we have
\[
\KL\left(\left[\frac12 + \frac{C_T}{2T}, \frac12 - \frac{C_T}{2T}\right]^\top;\left[\frac12 , \frac12\right]^\top\right)
\geq \frac{C_T^2}{2T^2},
\]
where we used~\eqref{eq:bregman_strongly_convex} coupled with the fact that the \ac{KL} divergence is 1-strongly convex\index{function!strongly convex} with respect to the L$_1$ norm (Lemma~\ref{lemma:entropy_strongly_convex}) for $C_T \in (-T,T)$, and the same inequality holds for $C_T \in \{-T,T\}$ and $T\geq 1$.
%
\end{proof}
Hence, \ac{KT} can guarantee an exponential amount of money, paying only a $\sqrt{T}$ penalty.
It is possible to prove that the guarantee above for the \ac{KT} algorithm is optimal to constant additive factors. Moreover, observe that the \ac{KT} strategy does not require any parameter to be set: no learning rates, nor a regularizer.

Also, we can extend the guarantee of the \ac{KT} algorithm to the case in which the coins are ``continuous'', that is $c_t \in [-1,1]$\index{continuous coin|textbf}, and we have the following theorem.
\begin{theorem}
\label{thm:kt2}
Let $c_t \in [-1, 1]$ for $t=1, \dots, T$. Then, the \ac{KT} bettor in Algorithm~\ref{alg:kt} guarantees
\[
\Wealth_T\geq
\frac{\epsilon 2^T\Gamma\!\left(\frac{T+1+\sum_{t=1}^T c_t}{2}\right)\Gamma\!\left(\frac{T+1-\sum_{t=1}^T c_t}{2}\right)}{\pi \Gamma(T+1)}
\geq \epsilon \exp\!\left( \frac{(\sum_{t=1}^T c_t)^2}{2T} - \frac{\ln T}{2} -1\right)\!,
\]
where $\Gamma$ is the gamma function\index{gamma function} (see Appendix~\ref{sec:gamma}).
\end{theorem}
\begin{proof}
Let $t\geq 0$, and let $F_t:[-t, t]\to \R$ be defined by $F_t(x)=\epsilon\frac{2^t\Gamma\left(\frac{t+1+x}{2}\right)\Gamma\left(\frac{t+1-x}{2}\right)}{\pi \Gamma(t+1)}$.


First, we want to prove that
\begin{equation}
\label{eq:proof_kt2_eq1}
F_t(x+c)
\leq\left(1+c\frac{x}{t}\right)F_{t-1}(x), \quad \forall t\geq 1, \forall c\in [-1,1], \forall x \in [-(t-1), t-1]~.
\end{equation}
This is equivalent to showing that $\phi_t(c,x)\leq \ln F_{t-1}(x)$ for all $c \in [-1,1]$ and $x \in [-(t-1), t-1]$, where $\phi_t(c,x) := \ln F_t(x+c) - \ln \left(1+c\frac{x}{t}\right)$.

One can verify that $F_t(x+c)=\left(1+c\frac{x}{t}\right)F_{t-1}(x)$ for $c\in \{-1,1\}$ (see Problem~\ref{exercise:kt_equality}). Hence $\phi_t(1,x)=\phi_t(-1,x)= \ln F_{t-1}(x)$.
Moreover, from the log-convexity of the Gamma function (Proposition~\ref{prop:gamma}(\ref{prop:gamma_log_cvx})) and the convexity of $c \mapsto - \ln \left(1+c\frac{x}{t}\right)$, we have that $\phi_t$ is convex in its first argument. Hence, from the definition of convexity, we have
\begin{align*}
\phi_t(c,x)
&= \phi_t\left(\frac{1+c}{2}\cdot1 + \frac{1-c}{2}(-1),x\right)
\leq \frac{1+c}{2} \phi_t(1,x) + \frac{1-c}{2} \phi_t(-1,x)\\
&= \ln F_{t-1}(x)~.
\end{align*}

Now, we show that the wealth of the algorithm at the beginning of round $t$ is lower bounded by $F_{t-1}(\sum_{i=1}^{t-1} c_i)$. We prove it by induction. The base case is
\[
F_0(0)
=\epsilon\frac{2^0\Gamma\left(\frac{1}{2}\right)\Gamma\left(\frac{1}{2}\right)}{\pi \Gamma(1)}
=\epsilon~.
\]
Then, we assume $\Wealth_{t-1}\geq F_{t-1}(\sum_{i=1}^{t-1} c_i)$ and prove that $\Wealth_{t}\geq F_{t}(\sum_{i=1}^{t} c_i)$.
\begin{align*}
\Wealth_t
&= \Wealth_{t-1} \left(1+ c_t \frac{\sum_{i=1}^{t-1} c_i}{t}\right)
\geq F_{t-1}\left(\sum_{i=1}^{t-1} c_i\right) \left(1+ c_t \frac{\sum_{i=1}^{t-1} c_i}{t}\right)\\
&\geq F_t\left(\sum_{i=1}^{t} c_i\right),
\end{align*}
where in the last inequality we used \eqref{eq:proof_kt2_eq1}.

Finally, we show the second inequality.
From Lemma~\ref{lemma:max_universal_portfolio} and denoting by $\KL$ the \ac{KL} divergence\index{Kullback--Leibler divergence}, we have
\begin{align*}
\frac{2^t\Gamma\left(\frac{t+1+x}{2}\right)\Gamma\left(\frac{t+1-x}{2}\right)}{\pi \Gamma(t+1)}
&\geq \frac{2^t\Gamma\left(t+\frac{1}{2}\right)\Gamma\left(\frac{1}{2}\right) \left(\frac{t+x}{2}\right)^\frac{t+x}{2}\left(\frac{t-x}{2}\right)^\frac{t-x}{2}}{\pi \Gamma(t+1) t^{t}}\\
&= \frac{\Gamma\left(t+\frac{1}{2}\right) \left(1+\frac{x}{t}\right)^\frac{t+x}{2}\left(1-\frac{x}{t}\right)^\frac{t-x}{2}}{\sqrt{\pi} \Gamma(t+1)}\\
&= \frac{\Gamma\left(t+\frac{1}{2}\right)\exp\left[ t \cdot \KL\left(\left[\frac12 + \frac{x}{2t},\frac12 - \frac{x}{2t}\right]^\top;\left[\frac12, \frac12\right]^\top \right)\right] }{\sqrt{\pi} \Gamma(t+1)}\\
&\geq \frac{\Gamma\left(t+\frac{1}{2}\right)\exp\left( \frac{x^2}{2t}\right) }{\sqrt{\pi} \Gamma(t+1)}
\geq \exp\left( \frac{x^2}{2t} - \frac12 \ln t-1\right),
\end{align*}
where in the second to last inequality we used~\eqref{eq:bregman_strongly_convex} coupled with the fact that the \ac{KL} divergence is 1-strongly convex\index{function!strongly convex} with respect to the L$_1$ norm (Lemma~\ref{lemma:entropy_strongly_convex}) for $x \in (-t,t)$, and the same inequality holds for $x \in \{-t,t\}$ and $t\geq 1$. Finally, Lemma~\ref{lemma:bound_regret_gamma} gives in the last inequality.
\end{proof}
\index{coin-betting game|)textbf}

So, we have introduced the coin-betting game, extended it to continuous coins, and presented a simple and optimal strategy to solve it.
In the next section, we show \emph{how to use the \ac{KT} bettor as a parameter-free one-dimensional \ac{OCO} algorithm!}
\index{Krichevsky--Trofimov!bettor|)textbf}

\index{parameter-free!online convex optimization|(}
\section{Parameter-free One-Dimensional Online Convex Optimization}

So, Theorem~\ref{thm:kt} tells us that we can win almost as much money as a strategy betting the optimal fixed fraction of money at each step. We only pay a logarithmic price in the log wealth, which corresponds to a $\frac{1}{\sqrt{T}}$ term in the actual wealth.

Now, let's see why this problem is interesting in \ac{OCO}. It turns out that \emph{solving the coin-betting game with continuous coins\index{continuous coin} is equivalent to solving a one-dimensional unconstrained online linear optimization problem}. That is, a coin-betting algorithm is equivalent to an online learning algorithm that produces a sequence of $x_t\in \R$ that minimizes the one-dimensional regret with linear losses:
\[
\Regret_T(u):=\sum_{t=1}^T g_t (x_t - u),
\]
where the $g_t$ are adversarial and bounded. Without loss of generality, we will assume $g_t \in [-L,L]$.
Also, remembering that \ac{OCO} games can be reduced to \ac{OLO} games, such a reduction would effectively reduce \ac{OCO} to coin betting!
Moreover, through online-to-batch conversion, any stochastic 1-d problem could be reduced to a coin-betting game!

The key theorem that allows the conversion between \ac{OLO} and coin betting is the following one.
\begin{theorem}[Regret-Reward Duality]
\label{thm:reward-regret}
\index{regret-reward duality}
Let $\phi:\R^d \to (-\infty,+\infty]$ be a proper closed convex function and let $\phi^\star:\R^d \to (-\infty,+\infty]$ be its Fenchel conjugate.
Consider two sequences $\bx_1, \bx_2, \dots, \bx_T \in \R^d$ and $\bg_1,\dots,\bg_T \in \R^d$. Then,
\[
-\sum_{t=1}^T \langle \bg_t, \bx_t \rangle \ge \phi\left( -\sum_{t=1}^T \bg_t \right)
\Leftrightarrow
\sum_{t=1}^T \langle \bg_t, \bx_t - \bu\rangle \le \phi^\star(\bu), \quad \forall \bu \in \R^d~.
\]
\end{theorem}
\begin{proof}
Let's prove the left-to-right implication. For any $\bu \in \R^d$, we have
\begin{align*}
\sum_{t=1}^T \langle \bg_t, \bx_t - \bu\rangle
\le -\sum_{t=1}^T \langle \bg_t, \bu \rangle - \phi\left(-\sum_{t=1}^T \bg_t\right)
\leq \sup_{\btheta \in \R^d} \ \langle \btheta, \bu \rangle - \phi\left( \btheta \right)
= \phi^\star\left(\bu\right)~.
\end{align*}

For the other implication, we have
\begin{align*}
-\sum_{t=1}^T \langle \bg_t, \bx_t \rangle
&= \sup_{\bu \in \R^d} \ -\sum_{t=1}^T \langle \bg_t, \bu \rangle - \sum_{t=1}^T \langle \bg_t, \bx_t - \bu \rangle \\
&\ge \sup_{\bu \in \R^d} \ -\sum_{t=1}^T \langle \bg_t, \bu \rangle - \phi^\star\left(\bu\right)
= \phi\left(-\sum_{t=1}^T \bg_t\right)~. \qedhere
\end{align*}
\end{proof}
To make sense of the above theorem, assume that we are considering a 1-d problem and $g_t \in [-L,L]$. Then, guaranteeing a lower bound to $-\sum_{t=1}^T g_t x_t$ can be done through a betting strategy that bets $x_t$ money on the continuous coins $c_t=-g_t/L \in [-1,1]$. So, the theorem implies that \emph{proving a lower bound for the wealth in a coin-betting game implies a regret upper bound for the corresponding one-dimensional \ac{OLO} game}.
However, proving a reward lower bound is easier because it does not depend on the competitor $\bu$. Indeed, not knowing the norm of the competitor is exactly the reason why tuning the learning rates in \ac{OMD} is hard!

This consideration immediately gives us the conversion between 1-d \ac{OLO} and coin betting: \emph{the outcome of the coin is the negative of the subgradient of the losses on the current prediction.}
Indeed, setting $c_t=-g_t/L$, we have that a coin-betting algorithm that bets $x_t$ would give us
\[
\Wealth_{T}
=\epsilon + \sum_{t=1}^T x_t c_t
=\epsilon - \sum_{t=1}^T x_t \frac{g_t}{L} ~.
\]
So, a lower bound on the wealth corresponds to a lower bound that can be used in Theorem~\ref{thm:reward-regret}.
To obtain a regret guarantee, we only need to calculate the Fenchel conjugate of the reward function, assuming it can be expressed as a function of $\sum_{t=1}^T c_t$.

The last step is to reduce 1-d \ac{OCO} to 1-d \ac{OLO}. But this is an easy step that we have done many times. Indeed, we have
\[
\Regret_T(u)
=\sum_{t=1}^T \ell_t(x_t) -\sum_{t=1}^T \ell_t(u)
\leq \sum_{t=1}^T x_t g_t - \sum_{t=1}^T g_t u, \quad \forall u \in \R, \ \forall g_t\in \partial \ell_t(x_t)~.
\]

So, to summarize, the Fenchel conjugate of the wealth lower bound for the coin-betting game becomes the regret guarantee for the \ac{OCO} game.
In the next section, we specialize all these considerations to the \ac{KT} algorithm.

\subsection{KT as a One-Dimensional Online Convex Optimization Algorithm}

\index{Krichevsky--Trofimov!online convex optimization algorithm|(textbf}
Here, we want to use the considerations in the above section to use \ac{KT} as a parameter-free 1-d \ac{OCO} algorithm.
First, let's see what such an algorithm looks like.
\ac{KT} bets $x_t=\beta_t \Wealth_{t-1}$, starting with $\epsilon$ money.
Now, set $c_t=-g_t/L$ where $g_t \in \partial \ell_t(x_t)$ and assume the losses $\ell_t$ are $L$-Lipschitz. So, we get
\[
x_t = -\frac{\sum_{i=1}^{t-1} g_i}{t L} \left(\epsilon - \sum_{i=1}^{t-1} \frac{g_i}{L} x_i\right)~.
\]
Algorithm~\ref{alg:kt_oco} shows the \textbf{Krichevsky--Trofimov algorithm for \ac{OCO}}.

\begin{algorithm}[t]
\caption{Krichevsky--Trofimov (KT) OCO Algorithm}
\label{alg:kt_oco}
\begin{algorithmic}[1]
{
    \REQUIRE{$\epsilon>0$, $L>0$}
    \FOR{$t=1$ {\bfseries to} $T$}
    \STATE{Predict $x_t=-\frac{\sum_{i=1}^{t-1} g_i}{t L}\left(\epsilon- \sum_{i=1}^{t-1} \frac{g_i}{L} x_i\right) \in \R$}
    \STATE{Pay the loss $\ell_t(x_t)$, where $\ell_t$ is subdifferentiable on $\R$}
    \STATE{Set $g_t \in \partial \ell_t(x_t)$ where $|g_t|\leq L$}
    \ENDFOR
}
\end{algorithmic}
\end{algorithm}

Let's now see what kind of regret we get. From Theorem~\ref{thm:kt2}, we have that the \ac{KT} bettor guarantees the following lower bound on the wealth when used with $c_t=-g_t/L$:
\[
\epsilon-\sum_{t=1}^T x_t \frac{g_t}{L} \geq \frac{\epsilon}{e\sqrt{T}}\exp\left(\frac{(\sum_{t=1}^T g_t)^2}{2T L^2}\right)~.
\]
So, we found the function $\phi$; we just need $\phi^\star$ or an upper bound to it, which can be found with the following lemma.
\begin{lemma}
\label{lemma:dual_exp_square}
Define $f(x)= \beta \exp\frac{x^2}{2 \alpha}$, for $\alpha,\beta>0$, $x \in \R$. Then
\begin{align*}
f^\star(y)
&=|y| \sqrt{\alpha W\left(\frac{\alpha y^2}{\beta^2}\right)} - \beta \exp\left(\frac{1}{2}W\left(\frac{\alpha y^2}{\beta^2}\right)\right)
\leq |y| \sqrt{\alpha W\left(\frac{\alpha y^2}{\beta^2}\right)} - \beta\\
&\leq |y| \sqrt{\alpha \ln\left(1+\frac{\alpha y^2}{\beta^2}\right)} - \beta
\leq |y| \sqrt{2\alpha \ln\left(1+\frac{\sqrt{\alpha} |y|}{\beta}\right)} - \beta,
\end{align*}
where $W$ is the Lambert function\index{Lambert function}, i.e., $W:\R_{\geq 0}\to \R_{\geq 0}$ satisfies $x=W(x) \exp(W(x))$.
\end{lemma}
\begin{proof}
From the definition of the Fenchel conjugate, we have
\begin{align*}
f^\star(y)
= \max_{x \in \R} \  x\, y - f(x)
= \max_{x} \  x\, y - \beta \exp\frac{x^2}{2 \alpha},
= x^\star\,y -\beta \exp\frac{(x^\star)^2}{2 \alpha},
\end{align*}
where $x^\star= \argmax_{x} x\, y - f(x)$. We now use the fact that $x^\star$ satisfies $y = f'(x^\star)$, to have $x^\star=\sign(y)\sqrt{\alpha W(\frac{\alpha y^2}{\beta^2})}$, where $W(\cdot)$ is the Lambert function\index{Lambert function}.
Using Lemma~\ref{thm:lambert_upper_lower} in the Appendix, we obtain the second-to-last bound. The last one is obtained by using $\sqrt{a^2+1}\leq |a|+1$ for all $a\in \R$.
\end{proof}

So, the regret guarantee of \ac{KT} used as a one-dimensional \ac{OCO} algorithm is
\[
\Regret_T(u)
=\sum_{t=1}^T (\ell_t(x_t) - \ell_t(u))
\leq |u| L \sqrt{2 T \ln \left(\frac{e |u| T}{\epsilon}+1\right)} + \epsilon L, \quad \forall u \in \R,
\]
where the only assumption was that the first derivatives (or subderivatives) of $\ell_t$ are bounded in absolute value by $L$.
The bound is almost optimal because the optimal one has $\sqrt{T}$ factor inside the logarithm, instead of $T$, see the lower bound in Theorem~\ref{thm:olo_coin_betting_lower_bound}.
Also, it is important to note that any setting of $\epsilon$ in $[1,\sqrt{T}]$ would not change the asymptotic rate.

To appreciate this guarantee, compare it to the one of \ac{OMD} with learning rate $\eta=\frac{\alpha}{L\sqrt{T}}$:
\[
\Regret_T(u)
=\sum_{t=1}^T \ell_t(x_t) -\sum_{t=1}^T \ell_t(u)
\leq \frac{L}{2}\left(\frac{u^2}{\alpha}+\alpha\right)\sqrt{T}, \quad \forall u \in \R~.
\]
Hence, the coin-betting approach allows us to get almost the optimal bound, without having to solve the impossible task to choose the correct learning rate for each competitor, that is, in a \emph{parameter-free} way. This motivates the following definition.

\index{parameter-free|(textbf}
\begin{definition}
\label{def:parameter-free}
We call \textbf{parameter-free} any online algorithm that satisfies the optimal regret bound with respect to $T$ and to all the comparators in the feasible set $\mathcal{V}$, up to polylogarithmic multiplicative factors.
\end{definition}
\index{parameter-free|)textbf}
We note explicitly that in this definition, the algorithm is allowed to know the Lipschitz constant of the losses. Hence, Algorithm~\ref{alg:kt_oco} is parameter-free.

\begin{figure}[t]
  \centering
    \begin{tikzpicture}
    \begin{axis}[axis line style = thick,
                width=7cm,
                domain = 0:15,
                samples = 200,
                axis x line = middle,
                axis y line = left,
                every axis x label/.style={at={(current axis.right of origin)},anchor=west},
                every axis y label/.style={at={(current axis.north west)},above=0mm},
                xlabel = {$x_t$},
                ylabel = {$\ell_t(x)$},
                ticks = none,
                grid = major
                ]
                \addplot[thick] {abs(x-10)} [yshift=3pt] node[pos=.95,left] {};
                \draw[-{Stealth[scale=0.8,angle'=30]},dashed, gray, thick](axis cs:0,10) to [out=0,in=100] (axis cs:0.5,9.5);
                \draw[-{Stealth[scale=0.8,angle'=30]},dashed, gray,thick](axis cs:0.5,9.5) to [out=0,in=100] (axis cs:1,9);
                \draw[-{Stealth[scale=0.8,angle'=30]},dashed, gray,thick](axis cs:1,9) to [out=0,in=80] (axis cs:1.875,8.125);
                \draw[-{Stealth[scale=0.8,angle'=30]},dashed, gray,thick](axis cs:1.875,8.125) to [out=0,in=80] (axis cs:3.5,6.5);
                \draw[-{Stealth[scale=0.8,angle'=30]},dashed, gray,thick](axis cs:3.5,6.5) to [out=355,in=85] (axis cs:6.5625,3.4375);
                \draw[-{Stealth[scale=0.8,angle'=30]},dashed, gray,thick](axis cs:6.5625,3.4375) to [out=10,in=130] (axis cs:12.3750,2.3750);
                \draw[-{Stealth[scale=0.8,angle'=30]},dashed, gray,thick](axis cs:12.3750,2.3750) to [out=90,in=0] (axis cs:1.2891,8.7109);
            \end{axis}
    \end{tikzpicture}
    \begin{tikzpicture}
    \begin{axis}[axis line style = thick,
                width=7cm,
                domain = 0:15,
                samples = 200,
                axis x line = middle,
                axis y line = left,
                every axis x label/.style={at={(current axis.right of origin)},anchor=west},
                every axis y label/.style={at={(current axis.north west)},above=0mm},
                xlabel = {$x_t$},
                ylabel = {$\ell_t(x)$},
                ticks = none,
                grid = major
                ]
                \addplot[thick] {abs(x-10)} [yshift=3pt] node[pos=.95,left] {};
                \draw[-{Stealth[scale=0.8,angle'=30]}, thick, gray](axis cs:0,10) to [out=0,in=90] (axis cs:0.5,9.5);
                \draw[-{Stealth[scale=0.8,angle'=30]},  thick, gray](axis cs:0.5,9.5) to [out=0,in=90] (axis cs:1,9);
                \draw[-{Stealth[scale=0.8,angle'=30]},  thick, gray](axis cs:1,9) to [out=0,in=90] (axis cs:1.5,8.5);
                \draw[-{Stealth[scale=0.8,angle'=30]},  thick, gray](axis cs:1.5,8.5) to [out=0,in=90] (axis cs:2,8);
                \draw[-{Stealth[scale=0.8,angle'=30]},  thick, gray](axis cs:2,8) to [out=0,in=90] (axis cs:2.5,7.5);
                \draw[-{Stealth[scale=0.8,angle'=30]},  thick, gray](axis cs:2.5,7.5) to [out=0,in=90] (axis cs:3,7);
                \draw[-{Stealth[scale=0.8,angle'=30]},  thick, gray](axis cs:3,7) to [out=0,in=90] (axis cs:3.5,6.5);

                \draw[-{Stealth[scale=0.8,angle'=30]}, densely dotted, thick, gray](axis cs:0,10) to [out=5,in=75] (axis cs:3,7);
                \draw[-{Stealth[scale=0.8,angle'=30]}, densely dotted, thick, gray](axis cs:3,7) to [out=5,in=75] (axis cs:6,4);
                \draw[-{Stealth[scale=0.8,angle'=30]}, densely dotted, thick, gray](axis cs:6,4) to [out=5,in=75] (axis cs:9,1);
                \draw[-{Stealth[scale=0.8,angle'=30]}, densely dotted, thick, gray](axis cs:9,1) to [out=40,in=120] (axis cs:12,2);
                \draw[-{Stealth[scale=0.8,angle'=30]}, densely dotted, thick, gray](axis cs:12,2) to [out=270,in=350] (axis cs:9,1);

                \draw[-{Stealth[scale=0.8,angle'=30]}, dashed, thick, gray](axis cs:0,10) to [out=0,in=80] (axis cs:2,8);
                \draw[-{Stealth[scale=0.8,angle'=30]}, dashed, thick, gray](axis cs:2,8) to [out=0,in=80] (axis cs:3.4142,6.58);
                \draw[-{Stealth[scale=0.8,angle'=30]}, dashed, thick, gray](axis cs:3.4142,6.58) to [out=0,in=80] (axis cs:4.5689,5.4311);
                \draw[-{Stealth[scale=0.8,angle'=30]}, dashed, thick, gray](axis cs:4.5689,5.4311) to [out=0,in=90] (axis cs:5.5689,4.4311);
                \draw[-{Stealth[scale=0.8,angle'=30]}, dashed, thick, gray](axis cs:5.5689,4.4311) to [out=0,in=90] (axis cs:6.4633,3.5367);
                \draw[-{Stealth[scale=0.8,angle'=30]}, dashed, thick, gray](axis cs:6.4633,3.5367) to [out=0,in=90] (axis cs:7.2798,2.7202);
                \draw[-{Stealth[scale=0.8,angle'=30]}, dashed, thick, gray](axis cs:7.2798,2.7202) to [out=0,in=90] (axis cs:8.0358,1.9642);
            \end{axis}
    \end{tikzpicture}
  \caption{Behaviour of \ac{KT} (left) and \ac{OSD} with various learning rates and the same number of steps (right) when $\ell_t(x)=|x-10|$ for all $t$.}
  \label{fig:comparison}
  \commentAlt{Figure~\ref{fig:comparison}. Two trajectory plots on the loss \ell_t(x)=|x-10|. The left panel shows KT iterates moving toward the minimizer and then jumping back. The right panel compares OSD trajectories with different learning rates, showing slow movement for small learning rates and oscillations for a larger one.}
\end{figure}

It is also interesting to look at what the algorithm does on an easy problem, where $\ell_t(x)=|x-10|$. In Figure~\ref{fig:comparison}, we show the different predictions that the \ac{KT} and \ac{OSD} algorithms make. The convergence rate of \ac{OSD} critically depends on the learning rate: too large does not give convergence, and too small slows down the convergence. On the other hand, \ac{KT} goes \emph{exponentially fast} towards the minimum, and then it automatically backtracks. This exponential growth effectively works like a line search procedure.
Later in the iterations, \ac{KT} will oscillate around the minimum, \emph{automatically shrinking its steps, without any parameter to tune.} Of course, this is a simplified example. In a truly \ac{OCO} game, the losses are different at each time step, and the intuition behind the algorithm becomes more difficult. Yet, the regret guarantee of \ac{KT} assures us that this strategy is principled.


\begin{remark}
Although we derived the \ac{KT} algorithm for \ac{OCO} via its connection to the coin-betting game, it remains an \ac{FTRL} algorithm, and with a more involved proof, one can prove the same regret guarantee. Problem~\ref{exercise:parameter-free_ftrl} shows a simplification of this idea, by designing a parameter-free algorithm directly from the \ac{FTRL} equality in Lemma~\ref{lemma:ftrl_equality}.
\end{remark}

Next, we reduce \ac{OCO} in $\R^d$ and \ac{LEA} to coin betting.

\section{Coordinate-wise Parameter-free Online Convex Optimization}
\label{sec:coordinate_kt}

\begin{algorithm}[h]
\caption{Coordinate-Wise KT for OCO}
\label{alg:kt_oco_coordinate}
\begin{algorithmic}[1]
{
    \REQUIRE{$\epsilon>0$, $L_\infty>0$}
    \FOR{$t=1$ {\bfseries to} $T$}
    \STATE{Output $\bx_t\in \R^d$ such that $x_{t,i} = -\frac{\sum_{j=1}^{t-1} g_{j,i}}{L_\infty t} \left(\epsilon - \sum_{j=1}^{t-1} \frac{g_{j,i}}{L_\infty} x_{j,i}\right)$ for $i=1, \dots, d$}
    \STATE{Pay the loss $\ell_t(\bx_t)$, where $\ell_t$ is subdifferentiable in $\R^d$}
    \STATE{Set $\bg_t \in \partial \ell_t(\bx_t)$ where $\|\bg_t\|_\infty\leq L_\infty$}
    \ENDFOR
}
\end{algorithmic}
\end{algorithm}

With AdaGrad (Section~\ref{sec:adagrad}), we have already seen that it is possible to decompose an \ac{OCO} problem over the coordinates and use a different one-dimensional \ac{OLO} algorithm on each coordinate.
In particular, we saw that
\begin{align*}
\Regret_T(\bu)
&= \sum_{t=1}^T \ell_t(\bx_t) - \sum_{t=1}^T \ell_t(\bu)
\leq \sum_{t=1}^T \langle \bg_t, \bx_t\rangle - \sum_{t=1}^T \langle \bg_t, \bu\rangle \\
&= \sum_{t=1}^T \sum_{i=1}^d g_{t,i}(x_{t,i}-u_i)
= \sum_{i=1}^d \sum_{t=1}^T g_{t,i}(x_{t,i}-u_i),
\end{align*}
where the $\sum_{t=1}^T g_{t,i}(x_{t,i}-u_i)$ is exactly the regret with respect to the linear losses constructed by the coordinate $i$ of the subgradient.

Hence, if we have a one-dimensional \ac{OLO} algorithm, we can use $d$ copies of it, each one fed with the coordinate $i$ of the subgradient.
In particular, we might think of using a \textbf{coordinate-wise \ac{KT}} algorithm. The pseudo-code of this procedure is in Algorithm~\ref{alg:kt_oco_coordinate}.

The regret bound we get is immediate: we just have to sum the regret over the coordinates.
\begin{theorem}
With the notation in Algorithm~\ref{alg:kt_oco_coordinate}, assume that $\|\bg_t\|_\infty\leq L_\infty$.
Then, $\forall \bu \in \R^d$, the following regret bounds hold
\begin{align*}
\sum_{t=1}^T (\ell_t(\bx_t)- \ell_t(\bu))
&\leq L_\infty \sum_{i=1}^d |u_i| \sqrt{2 T \ln \left(1+ e |u_i|  T/\epsilon\right)} + d \epsilon L_\infty\\
&\leq \|\bu\|_1 L_\infty \sqrt{2 T \ln \left(1+ e \|\bu\|_\infty T/\epsilon\right)} + d \epsilon L_\infty~.
\end{align*}
\end{theorem}

This theorem suggests that in high-dimensional settings $\epsilon$ should be proportional to $\frac{1}{d}$.
\index{Krichevsky--Trofimov!online convex optimization algorithm|)textbf}

\section{Parameter-free Algorithms in Any Norm}
\label{sec:parameterfree-direction-magnitude}

The above reduction works only in a finite-dimensional space. Moreover, it gives a dependence on the competitor with respect to the L$_1$ norm that might be undesirable. So, here we use the black-box magnitude/direction reduction from Section~\ref{sec:mag_dir}.

Assume the loss functions are $L$-Lipschitz with respect to the $\|\cdot\|_2$.
We can instantiate the above theorem using the \ac{KT} betting algorithm for the one-dimensional learner and \ac{OMD} for the direction learner.
For example, let $\mathscr{A}_{B}$ be \ac{OSD} with $\mathcal{V}=\{\bx \in \R^d : \|\bx\|_2\leq 1\}$, $\bx_1=\boldsymbol{0}$, and learning rate $\eta_t=\frac{\sqrt{2}}{2 L \sqrt{t}}$. Let $\mathscr{A}_{\text{1d}}$ be the \ac{KT} algorithm for one-dimensional \ac{OCO}. Then, using the construction in Algorithm~\ref{alg:onedimred}, we have
\begin{equation}
\label{eq:L_2_coin}
\Regret_T(\bu)
\leq L \|\bu\|_2\left( \sqrt{\ln(e \|\bu\|_2 T/\epsilon+1)} + 1 \right) \sqrt{2T} + L \epsilon, \quad \forall \bu \in \R^d~.
\end{equation}
Using an online-to-batch conversion, this algorithm is a stochastic gradient descent procedure without learning rates to tune.

Once again, to better appreciate this kind of guarantee, let's take a look at the one of \ac{FTRL} (remember that \ac{OSD} can be used in unbounded domains only with constant learning rates). With the regularizer $\psi_t(\bx)=\frac{L\sqrt{t}}{2\alpha}\|\bx\|^2_2$ and $L$-Lipschitz losses we get a regret of
\[
\Regret_T(\bu)
\leq L\sqrt{T}\left(\frac{\|\bu\|^2_2}{2\alpha}+\alpha\right)~.
\]
So, to get the right dependence on $\|\bu\|_2$, we need to tune $\alpha$, but we saw that this is impossible. On the other hand, the regret in \eqref{eq:L_2_coin} suffers from a logarithmic factor, which is the price to pay not to have to tune parameters.

In the same way, we can even have a parameter-free regret bound for L$_p$ norms. Assume that the loss functions are $L_q$-Lipschitz with respect to $\|\cdot\|_p$, where $1<p\leq 2$ and $q$ is such that $1/p+1/q=1$. Let $\mathscr{A}_{B}$ be \ac{OMD} with $\mathcal{V}=\{\bx \in \R^d : \|\bx\|_p\leq 1\}$ and learning rate $\eta_t=\frac{\sqrt{p-1}}{L_q\sqrt{2 t}}$ (see Section~\ref{sec:omd_pnorm}). Let $\mathscr{A}_{\text{1d}}$ be the \ac{KT} algorithm for one-dimensional \ac{OCO}.
Then, using the construction in Algorithm~\ref{alg:onedimred}, we have
\[
\Regret_T(\bu)
\leq L_q\left(\|\bu\|_p \sqrt{\ln(e \|\bu\|_p T/\epsilon+1)} + \frac{\|\bu\|_p}{\sqrt{p-1}}\right) \sqrt{2T} +L_q \epsilon, \quad \forall \bu \in \R^d~.
\]

Note that the regret against $\bu=\boldsymbol{0}$ of the parameter-free construction is \emph{constant} with respect to $T$. It is important to understand that there is nothing special in the origin in the unconstrained setting: we could translate the prediction by any offset and get a guarantee that treats the offset as the point with constant regret. This is shown in the next proposition.
\begin{proposition}
Let $\mathscr{A}$ be an \ac{OLO} algorithm that predicts $\bx_t$ and has linear regret $\Regret^{OLO}_T(\bu)$ for any $\bu \in \R^d$. We have that the regret of the predictions $\hat{\bx}_t=\bx_t+\bx_0$ for \ac{OCO} is
\[
\sum_{t=1}^T \ell_t(\hat{\bx}_t) - \sum_{t=1}^T \ell_t(\bu)
\leq \sum_{t=1}^T \langle \bg_t, \bx_t + \bx_0 - \bu\rangle
= \Regret^{OLO}_T(\bu-\bx_0), \quad \forall \bu \in \R^d,
\]
where each $\bg_t \in \partial \ell_t(\hat{\bx}_t)$.
\end{proposition}
\index{parameter-free!online convex optimization|)}

\section{Combining Parameter-free Algorithms Almost for Free}

We now show a useful application of the property of parameter-free \ac{OCO} algorithms of having constant regret against $\bu=\boldsymbol{0}$.

\begin{theorem}
Let $\mathscr{A}_1$ and $\mathscr{A}_2$ be two \ac{OLO} algorithms that produce the predictions $\bx_{t,1}$ and $\bx_{t,2}$, and have regrets equal to $\Regret^{\mathscr{A}_1}_T(\bu)$ and $\Regret^{\mathscr{A}_2}_T(\bu)$, respectively.
Then, predicting with $\bx_t=\bx_{t,1}+\bx_{t,2}$, we have for any $\bu \in \R^d$
\[
\sum_{t=1}^T \langle \bg_t, \bx_t\rangle - \sum_{t=1}^T \langle \bg_t, \bu\rangle
= \min_{\bu=\bu_1+\bu_2} \ \Regret^{\mathscr{A}_1}_T(\bu_1) + \Regret^{\mathscr{A}_2}_T(\bu_2)~.
\]
Moreover, if both algorithms have a constant regret of $\epsilon$ against $\bu=\boldsymbol{0}$, we have for any $\bu \in \R^d$
\[
\sum_{t=1}^T \langle \bg_t, \bx_t\rangle - \sum_{t=1}^T \langle \bg_t, \bu\rangle
\leq \epsilon + \min\left(\Regret^{\mathscr{A}_1}_T(\bu), \Regret^{\mathscr{A}_2}_T(\bu)\right)~.
\]
\end{theorem}
\begin{proof}
Set $\bu_1+\bu_2=\bu$. Then,
\[
\sum_{t=1}^T \langle \bg_t, \bx_t\rangle - \sum_{t=1}^T \langle \bg_t, \bu\rangle
=\sum_{t=1}^T \langle \bg_t, \bx_{t,1}\rangle - \sum_{t=1}^T \langle \bg_t, \bu_1\rangle
+\sum_{t=1}^T \langle \bg_t, \bx_{t,2}\rangle - \sum_{t=1}^T \langle \bg_t, \bu_2\rangle~.
\]
Taking the minimum between the bounds with $\bu_1=\boldsymbol{0}$ and $\bu_2=\boldsymbol{0}$ concludes the proof.
\end{proof}

In words, the above theorem allows us to combine online learning algorithms. If the algorithms we combine have constant regret against the null competitor, then we always get the best of the two guarantees.

\begin{example}
We can combine two parameter-free \ac{OCO} algorithms, one that gives a bound that depends on the L$_2$ norm of the competitor and subgradients, and another one specialized to the L$_1$/L$_\infty$ norm of competitor/subgradients. The above theorem assures us that we will also get the best guarantee between the two, paying only an additive constant factor in the regret.
\end{example}

Of course, upper-bounding the \ac{OCO} regret with the linear regret, the above theorem also upper bounds the \ac{OCO} regret.

\index{parameter-free!learning with experts|(}
\section{Reduction to Learning with Experts}
\label{sec:parameterfree_lea}

In this section, we consider again the \ac{LEA} setting from Section~\ref{sec:lea}.
First, remember that the regret we got from \ac{EG} in Section~\ref{sec:omd_eg} (and similarly for the \ac{FTRL} version in Section~\ref{sec:ftrl_eg}) is
\[
\Regret_T(\bu)
\leq \frac{\KL(\bu;\bpi)}{\eta} + \frac{\eta T L_\infty^2}{2}, \quad \forall \bu \in \Delta^{d-1},
\]
where $\bpi$ is the prior distribution on the experts and $\KL(\cdot; \cdot)$ is the \ac{KL} divergence\index{Kullback--Leibler divergence}.
As we reasoned in the \ac{OCO} case, to set the learning rate, we should know the value of $\KL(\bu;\bpi)$.
If we could set $\eta$ to $\sqrt{\frac{2 \KL(\bu;\bpi)}{L_\infty^2 T}}$, we would obtain a regret of $L_\infty \sqrt{2 T\, \KL(\bu; \bpi)}$.
However, given the adversarial nature of the game, this is impossible.
So, as we did in the \ac{OCO} case, we will see that even this problem can be reduced to betting on a continuous coin, obtaining optimal regret guarantees with a parameter-free algorithm.

\begin{remark}
\index{regret!$\epsilon$-quantile|(textbf}
There exists a different notion of regret for \ac{LEA}. Order the cumulative losses of all actions from lowest to highest and define the \textbf{$\epsilon$-quantile regret} to be the difference between the cumulative loss of the learner and the $\lceil \epsilon d\rceil$-th element in the sorted list. In formulas
\[
\Regret_T(\epsilon) = \sum_{t=1}^T \langle \bg_t, \bp_t \rangle - \sum_{t=1}^T g_{t,i_\epsilon},
\]
where $\epsilon \in [1/d, 1]$ and $i_\epsilon$ is the $\lceil \epsilon d\rceil$-th best-performing action.

This definition makes sense when the number of actions is very large, and there are many of them that are close to the optimal one. Note that the task of guaranteeing a small regret becomes easier with increasing $\epsilon$. So, we would like to design an algorithm whose regret depends on $\ln \frac{1}{\epsilon}$ and it does not depend on $d$ in any way.

We now show that a regret that depends on the \ac{KL} divergence\index{Kullback--Leibler divergence} between a generic competitor $\bu$ and $\bpi$ implies such regret guarantee. Define $\bu_\epsilon$ as the vector that has the coordinates corresponding to the best $\lceil \epsilon d\rceil$ experts equal to $\frac{1}{\lceil \epsilon d\rceil}$ and 0 in the other coordinates. Also, assume to have an algorithm that guarantees an upper bound $\Regret_T(\bu)$ equal to $F_T(\KL(\bu;\bpi))$ for a sequence of losses $\bg_t \in \R^d$, where $F_T$ is a non-decreasing function. Then, setting $\bpi=[1/d, \dots, 1/d]$, we have
\begin{align*}
\Regret_T(\epsilon)
&= \sum_{t=1}^T \langle \bg_t, \bp_t \rangle - \sum_{t=1}^T g_{t,i_\epsilon}
\leq \sum_{t=1}^T \langle \bg_t, \bp_t - \bu_\epsilon\rangle
\leq F_T(\KL(\bu_\epsilon; \bpi))\\
&= F_T \left( \ln \frac{d}{\lceil \epsilon d\rceil} \right)
\leq F_T \left( \ln \frac{1}{\epsilon} \right),
\end{align*}
where in the first inequality we used the fact that the average of a set of numbers is smaller than the biggest number in the set.
\index{regret!$\epsilon$-quantile|)textbf}
\end{remark}

First, let's introduce some notation.
Let $d \ge 2$ be the number of experts and $\Delta^{d-1}$ be the probability simplex\index{probability simplex}. Let $\bpi = [\pi_1, \pi_2, \dots, \pi_d] \in \Delta^{d-1}$ be any \emph{prior} distribution. Let $\mathscr{A}$ be a coin-betting algorithm. We will instantiate $d$ copies of $\mathscr{A}$, each of them with initial wealth $\pi_i$ for $i=1, \dots , d$.

Consider any round $t$. Let $x_{t,i} \in \R$ be the bet of the $i$-th copy of the coin-betting algorithm $\mathscr{A}$. The \ac{LEA} algorithm computes $\widehat \bp_t = [\widehat p_{t,1}, \widehat p_{t,2}, \dots, \widehat p_{t,d}] \in \R_{\geq 0}^d$ as
\begin{equation}
\label{eq:phat}
\widehat p_{t,i} = \max(x_{t,i},0)~.
\end{equation}
Then, the \ac{LEA} advice algorithm predicts $\bp_t = [p_{t,1}, p_{t,2}, \dots, p_{t,d}] \in \Delta^{d-1}$ as
\begin{equation}
\label{eq:preds_experts}
\bp_t = \begin{cases}
\tfrac{\widehat \bp_t}{\|\widehat \bp_t\|_1}, & \text{ if } \widehat \bp_t\neq \boldsymbol{0}, \\
\bpi, & \text{ otherwise}~.
\end{cases}
\end{equation}
After the prediction, the algorithm receives the vector of losses
$\bg_t = [g_{t,1}, g_{t,2}, \dots, g_{t,d}] \in [-L_\infty,L_\infty]^d$. From these losses, we construct the outcome of the continuous coin $c_{t,i} \in [-1,1]$ for the $i$-th copy of the coin-betting algorithm $\mathscr{A}$, defined as
\begin{equation}
\label{eq:gradients_experts_reduction}
c_{t,i} = \frac{1}{2L_\infty}
\begin{cases}
\langle \bg_t, \bp_t \rangle - g_{t,i}  & \text{if } x_{t,i} > 0, \\
\max\left(\langle \bg_t, \bp_t \rangle - g_{t,i} , 0\right) & \text{if } x_{t,i} \le 0~.
\end{cases}
\end{equation}

The construction above defines a \ac{LEA} algorithm that predicts $\bp_t$, based on the algorithm $\mathscr{A}$. We can prove the following regret bound for it.
\begin{theorem}[Regret Bound for Experts]
\label{theorem:regret-bound-experts}
Let $\mathscr{A}$ be a coin-betting algorithm that guarantees a wealth after $t$ rounds with initial money equal to $a>0$ of at least $a \exp(f_t(\sum_{j=1}^t c'_j))$ for any sequence of continuous coin outcomes $c'_1, \dots, c'_t \in [-1,1]$.
Assume each \ac{LEA} loss to satisfy $|\bg_{t,i}| \leq L_\infty$ for $i=1, \dots, d$. Then, the regret of the \ac{LEA} algorithm with prior $\bpi \in \Delta^{d-1}$ that predicts at each round with $\bp_t$ in \eqref{eq:preds_experts} satisfies
\[
\Regret_T(\bu)
= \sum_{t=1}^T \langle \bg_t, \bp_t -\bu\rangle
\le 2 L_\infty h\left( \KL(\bu;\bpi) \right), \quad \forall T \ge 1, \forall \bu \in \Delta^{d-1},
\]
for any $h:\R\to\R$ concave and non-decreasing such that $x\leq h(f_T(x))$.
\end{theorem}
\begin{proof}
We first prove that $\langle \bc_{t}, \bx_{t}\rangle \le 0$. Indeed,
\begin{align*}
\langle \bc_{t}, \bx_{t}\rangle
&=\sum_{i=1}^d c_{t,i} x_{t,i}\\
&= \sum_{i : x_{t,i} > 0} \frac{\max(x_{t,i}, 0) (\langle \bg_t, \bp_t \rangle-g_{t,i})}{2 L_\infty}  +  \sum_{i :  x_{t,i} \le 0} \frac{x_{t,i} \max(\langle \bg_t, \bp_t \rangle - g_{t,i}, 0)}{2 L_\infty} \\
& = \frac{1}{2 L_\infty}\|\widehat \bp_t\|_1 \sum_{i=1}^d p_{t,i} (\langle \bg_t, \bp_t \rangle - g_{t,i} )  +  \frac{1}{2 L_\infty}\sum_{i  :  x_{t,i} \le 0} x_{t,i} \max(\langle \bg_t, \bp_t\rangle - g_{t,i}, 0) \\
& = 0  + \frac{1}{2 L_\infty} \sum_{i \, : \, x_{t,i} \le 0} x_{t,i} \max(\langle \bg_t, \bp_t\rangle - g_{t,i}, 0)
\ \le 0 ~.
\end{align*}
The first equality follows from the definition of $c_{t,i}$. To see the second equality, consider two cases: if $x_{t,i} \le 0$ for all $i$, then $\|\widehat p_t\|_1 = 0$ and therefore both $\|\widehat \bp_t\|_1 \sum_{i=1}^d p_{t,i} (g_{t,i} - \langle \bg_t, \bp_t \rangle)$ and $\sum_{i \, : \, x_{t,i} > 0} \max(x_{t,i},0) (g_{t,i} - \langle \bg_t, \bp_t \rangle)$ are trivially zero. If $\|\widehat \bp_t\|_1 > 0$ then $ \max(x_{t,i},0) = \widehat p_{t,i} = \|\widehat \bp_t\|_1 p_{t,i}$ for all $i$.

From the assumption on the coin-betting algorithm $\mathscr{A}$, for the arm $i$, we have that
\begin{equation}
\label{equation:experts-one-dimensional-assumption}
\Wealth_{T,i} = \pi_i + \sum_{t=1}^T c_{t,i} x_{t,i}
\geq \pi_i\exp\left(f_T\left(\sum_{t=1}^T c_{t,i}\right)\right)~.
\end{equation}
So, inequality $\langle \bc_t, \bx_t \rangle \le 0$ and \eqref{equation:experts-one-dimensional-assumption} imply
\begin{equation}
\label{equation:bounded-potential}
\sum_{i=1}^d  \pi_i \exp\left(f_T \left(\sum_{t=1}^T c_{t,i} \right)\right)
\le 1 + \sum_{i=1}^d \sum_{t=1}^T  c_{t,i} x_{t,i}
\le 1~ .
\end{equation}
Now, for any competitor $\bu \in \Delta^{d-1}$, we derive
\begin{align*}
\Regret_T(\bu)
&= \sum_{t=1}^T \langle \bg_t, \bp_t - \bu  \rangle \\
&= \sum_{t=1}^T \sum_{i=1}^d u_i \left(\langle \bg_t, \bp_t \rangle - g_{t,i} \right) \\
& \le 2 L_\infty \sum_{t=1}^T \sum_{i=1}^d u_i c_{t,i} \quad \text{(definition of } c_{t,i} \text{)} \\
& \leq 2 L_\infty\sum_{i=1}^d u_i h\left(f_T\left( \sum_{t=1}^T c_{t,i}\right) \right)  \quad \text{(definition of the } h(x) \text{)} \\
& \le 2 L_\infty h\left(\sum_{i=1}^d u_i f_T\left( \sum_{t=1}^T c_{t,i}\right) \right) \quad \text{(Jensen's inequality\index{inequality!Jensen's})} \\
& \le 2 L_\infty h\left(\KL(\bu;\bpi)+\ln \sum_{i=1}^d \pi_i e^{f_T\left(\sum_{t=1}^T c_{t,i}\right)} \right) \quad \text{(Fenchel--Young\index{inequality!Fenchel--Young's})} \\
& \le 2 L_\infty h\left(\KL(\bu;\bpi)\right) \quad \text{(by \eqref{equation:bounded-potential})}~. \qedhere
\end{align*}
\end{proof}

Now, we could think of using the \ac{KT} bettor with this theorem, obtaining the Algorithm~\ref{alg:kt-experts}.

\begin{algorithm}[t]
\begin{algorithmic}[1]
\caption{Learning with Expert Advice based on \ac{KT} Bettors}
\label{alg:kt-experts}
{
\REQUIRE{Number of experts $d$, prior distribution $\bpi \in \Delta^{d-1}$, $L_\infty>0$}
\FOR{$t=1,2,\dots,T$}
\STATE{Set $x_{t,i} = \frac{\sum_{j=1}^{t-1} c_{j,i}}{t} \left(\pi_i + \sum_{j=1}^{t-1} c_{j,i} x_{j,i} \right)$ for $i=1,\dots,d$}
\STATE{Set $\widehat{p}_{t,i} = \max(x_{t,i},0)$ for $i=1,\dots,d$}
\STATE{Predict with $\bp_t =
\begin{cases}
\widehat{\bp}_t/\|\widehat{\bp}_t\|_1 & \text{if $\|\widehat \bp_t\|_1 > 0$} \\
\bpi & \text{if $\|\widehat \bp_t\|_1 = 0$}
\end{cases}$}
\STATE{Receive loss vector $\bg_t \in [-L_\infty,L_\infty]^d$}
\STATE{Set $c_{t,i} = \frac{1}{2 L_\infty}\begin{cases}
\langle \bg_t, \bp_t \rangle - g_{t,i},  & \text{if $x_{t,i} > 0$} \\
\max(\langle \bg_t, \bp_t \rangle - g_{t,i} ,0), & \text{if $x_{t,i} \le 0$}
\end{cases}$ for $i =1,\dots, d$}
\ENDFOR
}
\end{algorithmic}
\end{algorithm}
Unfortunately, this algorithm obtains a sub-optimal regret guarantee. In fact, remembering the lower bound on the wealth of \ac{KT} and setting $h(x)=\sqrt{2T(x+\frac12 \ln T + 1 )}$, we have
\[
\Regret_T(\bu)
\leq L_\infty \sqrt{8 T\left(\KL(\bu;\bpi)+\frac12\ln T+1\right)}~.
\]
We might think that the $\ln T$ is the price we have to pay to adapt to the unknown competitor $\bu$. However, it turns out it can be removed. In the next section, we see how to change the \ac{KT} strategy to obtain the optimal guarantee.

\section{Optimal Regret Bounds by Losing at Most a Constant Fraction of Money}
\label{sec:shifted_kt}

In the reduction before, if we use the \ac{KT} betting strategy, we would have a $\ln T$ term under the square root.
It turns out that we can avoid that term if we know the number of rounds beforehand. Then, in case $T$ is unknown, we can just use a doubling trick\index{doubling trick} (see Problem~\ref{exercise:doubling_trick_omd}), paying only a constant multiplicative factor in the regret.

More in detail, the logarithmic term in the regret comes from the $\frac{1}{\sqrt{T}}$ term in the lower bound on the wealth.
Note that in the case in which the number of heads in the sequence is equal to the number of tails, so that $\sum_{t=1}^T c_t=0$, the guaranteed wealth becomes proportional to $\frac{1}{\sqrt{T}}$. So, for $T$ that goes to infinity, the bettor will lose all of their money.

Instead, we need a more conservative strategy that, in the worst case, loses only a constant fraction of the initial wealth. In this case, the betting strategy has to pace its betting, possibly with the knowledge of the duration of the game. At the same time, it must still gain an exponential amount of money when the coin outcomes are biased towards one side.

We will prove that this is possible, designing a new betting strategy, through a simple reduction of the coin-betting problem to an \ac{OLO} one.
First of all, observe that $\ln(1+x)\geq x-x^2$ for $|x|\leq 1/2$. This means that the log wealth of a coin-betting strategy that in each round bets a signed fraction $\beta_t \in [-1/2, 1/2]$ on the continuous coins $c_t \in [-1,1]$ can be lower bounded as
\begin{equation}
\sum_{t=1}^T \ln(1+ c_t \beta_t)
\geq \sum_{t=1}^T c_t \beta_t - \sum_{t=1}^T c_t^2 \beta^2_t
\geq \sum_{t=1}^T c_t \beta_t - \sum_{t=1}^T \beta^2_t~.
\label{eq:shifted_kt_eq1}
\end{equation}
This suggests a very simple strategy: run \ac{FTRL} on the surrogate losses $\ell_t(\beta)= -c_t \beta$, with regularizer $\psi_t(\beta)=(t+a) \beta^2$ for $t\leq T$, $\psi_{T+1}=\psi_T$, and $\mathcal{V}=[-1/2,1/2]$, where $a> 0$ will be chosen in the following.
From Corollary~\ref{cor:ftrl_lip}, we have
\[
\sum_{t=1}^T (-c_t) (\beta_t - u)
\leq (T+a) u^2 + \sum_{t=1}^T \frac{c_t^2}{4(t+a)} - \sum_{t=1}^{T-1} \beta_{t+1}^2, \quad \forall u \in [-1/2, 1/2],
\]
and $\beta_t=\frac{\sum_{i=1}^{t-1} c_i}{2(t+a)}\in [-1/2,1/2]$.
Reordering, using \eqref{eq:shifted_kt_eq1} and the fact that $\beta_1=0$, we have
\begin{align*}
-\sum_{t=1}^T \ln (1+ c_t \beta_t)
&\leq \sum_{t=1}^T (-c_t \beta_t + \beta_t^2)
\leq (T+a) u^2 -\sum_{t=1}^T c_t u + \sum_{t=1}^T \frac{c_t^2}{4(t+a)}\\
&\leq (T+a) u^2 -\sum_{t=1}^T c_t u + \frac{T}{4a}~.
\end{align*}
Now, choose $u=\frac{\sum_{t=1}^T c_t}{2(T+a)}$, to have
\[
\sum_{t=1}^T \ln (1+ c_t \beta_t)
\geq \frac{1}{4 (T+a)}\left(\sum_{t=1}^T c_t\right)^2 - \frac{T}{4a}~.
\]
Exponentiating, we have
\[
\prod_{t=1}^T (1+c_t \beta_t)
\geq \exp\left[\frac{1}{4 (T+a)}\left(\sum_{t=1}^T c_t\right)^2 - \frac{T}{4a}\right]~.
\]

We can now use this betting strategy in the expert reduction in Theorem~\ref{theorem:regret-bound-experts}, setting $h(x)=\sqrt{4(T+a) (x+\frac{T}{4a})}$. Hence, choosing for example $a=\frac{T}{4}$, we have
\begin{equation}
\label{eq:kl_experts}
\Regret_T(\bu)
\leq L_\infty \sqrt{20 T (\KL(\bu;\bpi)+1)}, \quad \forall \bu \in \Delta^{d-1}~.
\end{equation}

\begin{remark}
The update we obtain is very similar to the \ac{KT} ones. The factor $2$ comes from the approximation of the logarithm and the fact that \ac{FTRL} requires bounded gradients, while we saw in Theorem~\ref{thm:kt} that \ac{KT} is based on an algorithmic framework that allows unbounded gradients and uses a prior over the possible betting fractions. So, the additive factor $a$ can also be obtained in \ac{KT} by putting more mass around the 0 betting fraction in its prior, but the proof is more complex.
\end{remark}

Note that this betting strategy can also be used in the \ac{OCO} reduction. Given that we removed the logarithmic term in the exponent, in the one-dimensional case, we obtain the optimal regret of
\begin{equation}
\label{eq:oco_betting_optimal}
\Regret_T(u)
\leq |u| L \sqrt{5 T \ln \left(\frac{|u| \sqrt{19 T}}{\epsilon}+1\right)}+L \epsilon, \quad \forall u \in \R,
\end{equation}
where we gained in the $\sqrt{T}$ term inside the logarithm, instead of the $T$ term of the \ac{KT} algorithm. This implies that now we can set $\epsilon$ to $\sqrt{T}$ and obtain an asymptotic rate of $\mathcal{O}(\sqrt{T})$ rather than $\mathcal{O}(\sqrt{T \ln T})$ for $T\to \infty$.

\index{parameter-free!learning with experts|)}

\index{parameter-free!sleeping experts|(}
\section{Example: A Parameter-free Algorithm for Sleeping Experts}
\label{sec:from_kt_to_sleeping_experts}

Consider the setting of learning with sleeping experts from Section~\ref{sec:sleeping_experts}. We will show how to easily design a parameter-free algorithm for this setting.


Consider using \ac{KT}\index{Krichevsky--Trofimov!online convex optimization algorithm} to solve \ac{OLO} in $\R$, but we slightly modify \ac{KT} to ignore the rounds where the gradients are zero:
\[
x_t
=\frac{-\sum_{j=1}^{t-1} g_j}{L(1+ \sum_{j=1}^{t-1} \indevent{\{g_j\neq0\}})} \left(\epsilon - \sum_{j=1}^{t-1} \frac{g_j}{L} x_j\right)~.
\]
It is easy to show---see Problem~\ref{exercise:kt_non_zero}---that this algorithm will guarantee the following regret upper bound:
\[
\sum_{t=1}^T g_t (x_t-u)
\leq |u| L \sqrt{2 \sum_{t=1}^T \indevent{\{g_t\neq0\}} \ln \left(1+\frac{e |u|\sum_{t=1}^T \indevent{\{g_t\neq0\}}}{\epsilon}\right)}+ L \epsilon, \ \forall u \in \R~.
\]

We now use \emph{three} black-box reductions in a row: first, we transform this \ac{OLO} algorithm for $\R$ to an \ac{OLO} algorithm for $\R_{\geq0}$. Second, we use $d$ of such algorithms to produce predictions in $\R^d_{\geq 0}$. Finally, we use the sleeping expert reductions from Section~\ref{sec:sleeping_experts}.

\begin{algorithm}[t]
\begin{algorithmic}[1]
\caption{Sleeping Expert Algorithm based on \ac{KT} Bettors}
\label{alg:kt-sleepingexperts}
{
\REQUIRE{Number of experts $d$, $\epsilon_i$ for $i=1, \dots,d$, $R>0$}
\FOR{$t=1,2,\dots,T$}
\STATE{Set $z_{t,i}=\frac{-\sum_{j=1}^{t-1} \tilde{g}_{j,i}}{R(1+\sum_{j=1}^{t-1} \indevent{\{\tilde{g}_{j,i}\neq0\}})} \left(\epsilon_i - \sum_{j=1}^{t-1} \frac{\tilde{g}_{j,i}}{R} z_{j,i}\right), \ i=1, \dots, d$}
\STATE{Receive the state of the experts $a_{t,i} \in \{0,1\}, i=1, \dots, d$}
\IF{$\sum_{j=1}^d a_{t,j} \max(z_{t,j},0)\neq0$}
\STATE{Predict with $x_{t,i} = \frac{a_{t,i} \max(z_{t,i},0)}{\sum_{j=1}^d a_{t,j} \max(z_{t,j},0)}$ for $i=1,\dots,d$}
\ELSE
\STATE{Predict with $\bx_t$ equal to the uniform distribution over the active experts}
\ENDIF
\STATE{Receive loss vector $\bg_t \in \R^d$ such that $\max_{i:a_{t,i}=1} g_{t,i}-\min_{i:a_{t,i}=1} g_{t,i}\leq R$}
\STATE{Set $\hat{g}_{t,i} =a_{t,i}(g_{t,i}-\langle \bg_t, \bx_t\rangle), \ i =1, \dots, d$}
\STATE{$\tilde{g}_{t,i} = \begin{cases} \min(\hat{g}_{t,i},0), & z_{t,i}<0\\ \hat{g}_{t,i}, & z_{t,i} \geq 0 \end{cases}, \ i=1, \dots, d$}
\ENDFOR
}
\end{algorithmic}
\end{algorithm}

Using a black-box reduction to constrain it to $\R_{\geq 0}$, as in Example~\ref{example:constrained_1d}, we have
\begin{align*}
\tilde{g}_t &= \begin{cases}
\min(g_t,0), & z_t<0\\
g_t, & z_t \geq 0
\end{cases}\\
z_t &=\frac{-\sum_{j=1}^{t-1} \tilde{g}_j}{L_\infty(1+\sum_{j=1}^{t-1} \indevent{\{\tilde{g}_j\neq0\}})} \left(\epsilon - \sum_{j=1}^{t-1} \frac{\tilde{g}_j}{L_\infty} z_j\right)\\
x_t &=\max(z_t,0)
\end{align*}
and obtain a regret of $u L_\infty \sqrt{2 T' \ln (e T' u /\epsilon+1)}+\epsilon L_\infty$, where $T'$ is the number of times that $\tilde{g}_{t}\neq 0$. Next, we use this algorithm to produce predictions over $\R^d_{\geq 0}$, using each one-dimensional algorithm on each coordinate, as seen in Section~\ref{sec:coordinate_kt}. So, we obtain a regret of
\[
\Regret_T(\bu)
\leq L_\infty \sum_{i=1}^d u_i \sqrt{2 T'_i \ln (e  u_i T'_i/\epsilon_i +1)} + L_\infty \sum_{i=1}^d \epsilon_i, \quad \forall \bu \in \R_{\geq 0}^d~.
\]
Finally, using the sleeping expert reduction from $\R^d_{\geq 0}$ from Section~\ref{sec:sleeping_experts}, we get Algorithm~\ref{alg:kt-sleepingexperts}.

Given that $a_{t,i}=0$ implies $\tilde{g}_{t,i}=0$, the upper bound on the sleeping expert regret of Algorithm~\ref{alg:kt-sleepingexperts} against any expert $i$ is $R \sqrt{2\sum_{t=1}^T a_{t,i} \ln (1+e \sum_{t=1}^T a_{t,i}/\epsilon_i) } + R \sum_{j=1}^d \epsilon_j$, that depends on the number of rounds that the expert $i$ was active.
\index{parameter-free!sleeping experts|)}

\begin{remark}
\label{remark:kt-sleepingexperts-countable}
Algorithm~\ref{alg:kt-sleepingexperts} extends verbatim to a countable set of
experts, provided that only finitely many experts are active on each round and
the parameters satisfy $\sum_{i=1}^{\infty}\epsilon_i<\infty$. In this case, the
same proof gives, for $\bu \in \R^{\Nat}_{\geq 0}$,
\[
\Regret_T^{\mathrm{sleeping}}(\bu)
\leq R \sum_{i=1}^\infty u_i \sqrt{2\sum_{t=1}^T a_{t,i} \ln\left(1 + u_i e\sum_{t=1}^T a_{t,i}/\epsilon_i\right)} + R \sum_{i=1}^{\infty}\epsilon_i~.
\]
\end{remark}

\section{From Portfolio Selection to Continuous-Coin Betting}
\label{sec:port_to_cb}

Consider the online coin-betting problem again. We said that it corresponds to $\mathcal{V}=[-1,1]$ and $\ell_t(x)=-\ln(1+c_t x)$, where the coin outcome $c_t \in \{-1,1\}$. We also considered an extension of this problem, the \emph{continuous} coin-betting problem, where $c_t \in [-1,1]$. Now, we show how to reduce the continuous coin-betting problem to the portfolio selection problem we saw in Chapter~\ref{ch:universal_portfolio}.

We consider portfolio selection with 2 stocks, and we set the market gains to $w_{t,1}=1+c_t$ and $w_{t,2}=1-c_t$. Note that $w_{t,1},w_{t,2}\geq 0$, so they are legal market gains. Let $[x_t,1-x_t]$ be the play of a 2-stock portfolio algorithm, where $0\leq x_t\leq1$.
Then, using $\beta_t = 2 x_t -1$ as the signed betting fraction of a continuous-coin-betting algorithm on $c_t$ ensures that the gain in the coin-betting problem coincides with the gain in the portfolio selection problem.

The proof is immediate. We have
\begin{align*}
w_{t,1} x_t + w_{t,2} (1-x_t)
&= x_t + x_t c_t + (1-c_t)(1-x_t)\\
&= x_t + x_t c_t + 1 - x_t - c_t + c_t x_t
= 1+ c_t \beta_t,
\end{align*}
where $\beta_t$ is the signed betting fraction equal to $2 x_t -1$.
Given that $x_t \in [0,1]$, the range of the betting fractions is in $[-1,1]$ as we wanted.

This implies that we can use the above reduction to transform any portfolio algorithm with 2 stocks into an algorithm for continuous-coin betting. In turn, given what we said in this chapter, this means that portfolio algorithms with 2 stocks can be used for \ac{OCO}. This does not make \ac{KT} and the optimal algorithm in Section~\ref{sec:shifted_kt} superfluous because the update rule of the universal portfolio algorithm, even with 2 stocks, is still linear in the iteration number.

\section{History Bits}
\index{parameter-free|(}
The keyword ``parameter-free'' was introduced in \citet{ChaudhuriYH09} for a similar strategy for the \ac{LEA} problem. It is now used as an umbrella term for all online algorithms that guarantee the optimal regret uniformly over the competitor class. Another less-used name to denote the same property is ``comparator-adaptive''~\citep[e.g.,][]{vanderHoevenCH20}, but in my opinion, it is a misnomer because it implicitly assumes the existence of an optimal comparator, which is false in the unbounded \ac{OLO} setting.
Note that, given the allure of the name ``parameter-free'', the term has been adopted in many other domains, with many different meanings. However, when used in the online learning literature, it has \emph{only} the meaning in Definition~\ref{def:parameter-free}. This means that in this literature ``parameter-free'' algorithms might still require the knowledge of other characteristics of the problem, for example, the Lipschitz constant of the losses or the number of rounds. This is fine: it is a technical term, and we can assign to it any definition we like, as for ``smooth'' or ``universal''.
\index{parameter-free|)}

\index{parameter-free!online convex optimization|(}
The first algorithm for 1-d parameter-free \ac{OCO} is from \citet{StreeterM12}, but the bound was suboptimal. The algorithm was then extended to Hilbert spaces in \citet{Orabona13}, still with a suboptimal bound. The optimal bound in Hilbert space, matching the one in~\eqref{eq:oco_betting_optimal}, was obtained in \citet{McMahanO14}.
This variant allows us to remove the factor $T$ in the logarithm by setting $\epsilon=\mathcal{O}(\sqrt{T})$, at the price of having a $\mathcal{O}(\sqrt{T})$ regret versus $\bu=\boldsymbol{0}$. The optimal constant for this kind of guarantee is obtained in \citet{ZhangCP22} through a PDE-based analysis, proving also a matching lower bound.

The idea of using coin betting to do parameter-free \ac{OCO} was introduced in \citet{OrabonaP16}. A different reduction from a betting algorithm to the specific case of linear regression with square loss was proposed by \citet{Vovk06}.
The \ac{KT} algorithm is from \citet{KrichevskyT81} and its extension to the ``continuous coin''\index{continuous coin} is from \citet{OrabonaP16}. The regret bound of the \ac{KT} algorithm is optimal up to an additive constant because its regret upper bound matches the one of the minimax strategy for known $T$, the so-called \emph{normalized maximum likelihood strategy}~\citep{Shtarkov87}\index{normalized maximum likelihood strategy}, up to an additive constant.
The regret/reward duality relationship was proved for the first time in \citet{McMahanO14}.
Theorem~\ref{thm:kt2} is new, and it is an improved version of a guarantee with a worse constant in \citet{OrabonaP16}.
There are also more refined betting algorithms that allow for obtaining parameter-free \ac{OCO} algorithms that depend on the sum of the squared norms of the subgradients, rather than time~\citep{CutkoskyO18}, also in a scale-free way~\citep{MhammediK20}\index{algorithm!scale-free}. Recently, parameter-free algorithms have also been extended to some non-convex functions in the stochastic setting~\citep{OrabonaP21}.

The approach of using a coordinate-wise version of the coin-betting algorithm was proposed in the first paper on parameter-free \ac{OLO} in \citet{StreeterM12}. Recently, the same approach with a special coin-betting algorithm was also used for optimization of deep neural networks~\citep{OrabonaT17}.
The idea of combining two parameter-free \ac{OLO} algorithms to obtain the best of the two guarantees is from \citet{Cutkosky19b}.

\citet{OrabonaP16} proposed a different way to transform a coin-betting algorithm into an \ac{OCO} algorithm that works in $\R^d$ or even in Hilbert spaces. However, that approach seems to work only for the L$_2$ norm, and it is not a black-box reduction. That said, the reduction in \citet{OrabonaP16} seems to have a better empirical performance compared to the one in Theorem~\ref{thm:reduction_direction_magnitude}.
\index{parameter-free!online convex optimization|)}

\index{regret!$\epsilon$-quantile|(}
\index{parameter-free!learning with experts|(}
The first parameter-free algorithm for experts is from \citet{ChaudhuriYH09}, named NormalHedge\index{NormalHedge algorithm}, where they introduced the concept of $\epsilon$-quantile regret and obtained a bound of $\mathcal{O}(\sqrt{(T+\ln^2 d)(1+\ln\frac{1}{\epsilon})})$ as $T\to\infty$, but without a closed-form update. Then, \citet{ChernovV10} removed the spurious dependence on $d$, again with an update without a closed form. \citet{OrabonaP16} showed that this guarantee can be efficiently obtained through the novel reduction to coin betting in Theorem~\ref{theorem:regret-bound-experts}. This kind of regret guarantee can be improved to depend on the sum of the squared losses rather than on time, but with an additional $\ln \ln T$ factor, as in the Squint algorithm~\citep{KoolenVE15}\index{Squint algorithm}. It is worth noting that the Squint algorithm can be interpreted exactly as a coin-betting algorithm plus a variant of the reduction in Theorem~\ref{theorem:regret-bound-experts}. The first lower bound for the $\epsilon$-quantile regret is proved in \citet{NegreaBCOR21}.
\index{regret!$\epsilon$-quantile|)}
\index{parameter-free!learning with experts|)}

The betting strategy in Section~\ref{sec:shifted_kt} is new, and it is a simplification of the methods in \citet{ChenCO22} and in \citet{OrabonaP21}.
The arXiv version of this book~\citep{Orabona19} contained yet another derivation, obtained from a distillation of the ideas of the ``shifted-KT'' potentials in \citet{OrabonaP16}, yet more complex than the presented one. See also \citet{Orabona26} for an explanation of the connection between shifted-KT potentials and gambling priors centered around the origin.

The variant of \ac{KT} in Section~\ref{sec:from_kt_to_sleeping_experts} that updates only when $g_t\neq 0$ appeared for the first time in \citet{ChenKLO22}.

The reduction in Section~\ref{sec:port_to_cb} is folklore, and it appears in \citet{OrabonaJ21}.

Recently, coin-betting algorithms have found numerous applications, such as the design of parameter-free online multi-task learning algorithms~\citep{DeneviPS20}, parameter-free particle-based algorithms~\citep{SharrockN23,SharrockDN23,SharrockMN23}, and the derivation of time-uniform concentration inequalities (see Chapter~\ref{ch:online_to_x}).

\section{Exercises}


\begin{exer}
\label{exercise:kt_equality}
Define $F_t(x):=\frac{2^t\Gamma\left(\frac{t+1+x}{2}\right)\Gamma\left(\frac{t+1-x}{2}\right)}{\pi \Gamma(t+1)}$ for $|x|\leq t$. Prove that $F_t(x+c) = F_{t-1}(x)(1+c x/t)$ for all $x \in[ -(t-1),t-1]$ and $c \in \{-1,1\}$.
\end{exer}

\begin{exer}
Prove that $\ell_t(x)=-\ln(1+ z_t\, x)$ with $\mathcal{V}=\{x \in \R: |x|\leq 1/2\}$ and $|z_t|\leq 1, \ t=1,\dots,T$ are exp-concave\index{function!exp-concave}. Then, using the \ac{ONS} algorithm, give an algorithm and a regret bound for a game with these losses. Finally, show a wealth guarantee of the corresponding coin-betting strategy. Hint: see \citet{CutkoskyO18}.
\end{exer}

\begin{exer}
Using the same proof technique in the section, find a betting strategy whose wealth depends on $\sum_{t=1}^T |c_t|$ rather than on $T$.
\end{exer}

\begin{exer}
\label{exercise:kt_non_zero}
Consider the variant of \ac{KT} for \ac{OLO} that ignores the rounds where the $g_t=0$, by predicting with $x_t=\frac{-\sum_{j=1}^{t-1} g_j}{1+\sum_{j=1}^{t-1} \indevent{\{g_j\neq0\}}} \left(\epsilon - \sum_{j=1}^{t-1} g_j x_j\right)$, where each $|g_t|\leq 1$. Prove that it guarantees a regret of $\mathcal{O}(|u|\sqrt{\sum_{t=1}^T \indevent{\{g_t\neq0\}} \ln \frac{1+|u|\sum_{t=1}^T \indevent{\{g_t\neq0\}}}{\epsilon}}+\epsilon)$ for any $u \in \R$ and the linear losses $\ell_t(x)= g_t x$.
\end{exer}

\begin{exer}
\label{exercise:parameter-free_ftrl}
We want to design a one-dimensional \ac{FTRL} parameter-free algorithm, just using Lemma~\ref{lemma:ftrl_equality}. Let $\epsilon>0$, set the number of rounds $T\geq 4$, $\mathcal{V}=\R_{\geq 0}$, and let $\psi_t(x)=\sqrt{T} x \ln \frac{x\sqrt{T}}{\epsilon}+\frac{t-1}{\sqrt{T}}x - x \sqrt{T}$. Show that the regret of this algorithm on linear losses $\ell_t(x)=g_t x$, where $|g_t|\leq 1$ is upper bounded by $u \sqrt{T} \ln \frac{u \sqrt{T}}{\epsilon}+ \epsilon$ for all $u \in \R_{\geq 0}$. Hint: You might need the elementary inequality $1+x\geq \exp(x-x^2)$ for all $|x|\leq 1/2$.
\end{exer}

\acresetall

\chapter{Dynamic, Strongly Adaptive, and Tracking Regret}
\label{ch:dynamic}

We saw that the definition of regret makes sense as a direct generalization of both the stochastic setting and the offline optimization. However, in some cases, we want to guarantee something stronger, that is, we would like the algorithm to have small regret with respect to multiple comparators. There are multiple ways to formalize this idea, and we will see three of these frameworks: \emph{Dynamic}, \emph{Strongly Adaptive}, and \emph{Tracking} regret.

\acresetall

\section{Dynamic Regret}
\index{regret!dynamic|(textbf}
Our first extension is to use multiple comparators, using the concept of \emph{dynamic regret}.
\begin{definition}
The \textbf{dynamic regret} is defined as
\[
\DRegret(\bu_1, \dots, \bu_T)
:=\sum_{t=1}^T \ell_t(\bx_t) - \sum_{t=1}^T \ell_t(\bu_t),
\]
where $\bu_1, \dots, \bu_T \in \mathcal{V}$ and $\mathcal{V} \subseteq \R^d$ is the feasible set.
\end{definition}
For additional clarity, we will also refer to the standard notion of regret as the \emph{static regret}.

The use of a sequence of competitors means that now our class of competitors is strictly larger than in the static case, so the cumulative loss of the competitors could be smaller. In turn, this means that the average loss of an algorithm with sublinear dynamic regret can converge to a potentially smaller average loss with respect to the static case. In this view, guarantees on the dynamic regret are strictly stronger than guarantees on the static regret.

Is it always possible to obtain sublinear dynamic regret? We already know that in the case that $\bu_1= \dots= \bu_T$ it is possible, because we recover the standard regret case. On the other hand, it should be intuitive that the problem becomes more and more difficult the more the sequence of $\bu_t$ changes over time. There are various ways to quantify this complexity, and a common one is the \textbf{path length}\index{path length|textbf} $P^{\phi}:\mathcal{V}^T \to \R_{\geq 0}$ defined as
\[
P^{\phi}(\bu_1, \dots, \bu_T)
:=\sum_{t=2}^T \phi(\bu_t-\bu_{t-1}),
\]
where $\phi:\R^d\to \R_{\geq 0}$ is a function that measures the shift from $\bu_{t-1}$ to $\bu_t$. In particular, we will instantiate $\phi$ to be a norm.
In the following, we will see how to design an online algorithm whose dynamic regret depends on the path length in the case that the feasible set is bounded.

\subsection{Dynamic Regret of Online Mirror Descent}

It turns out that some online learning algorithms already satisfy a dynamic regret bound without any additional change.
For \ac{OMD}, we can state the following theorem.
\begin{theorem}
\label{thm:dynamic}
Let $B_\psi$ be the Bregman divergence\index{Bregman divergence} with respect to $\psi: \mathcal{X} \to \R$ and assume $\psi$ to be closed and $\lambda$-strongly convex with respect to $\|\cdot\|$ in $\mathcal{V}\cap \interior \mathcal{X}$, where $\mathcal{V} \subseteq \mathcal{X}$ is a non-empty closed convex set. Assume that $\bx_t \in \interior \mathcal{X}$ for $t=1, \dots, T+1$. Then, $\forall \bu_1, \dots, \bu_T \in \mathcal{V}$, \ac{OMD} with constant learning rate $\eta$ in Algorithm~\ref{alg:omd} satisfies
\begin{align*}
\DRegret(\bu_1, \dots, \bu_T)
&\leq \frac{B_\psi(\bu_{T};\bx_1) + Q P^{\|\cdot\|}(\bu_1, \dots, \bu_T)}{\eta}+ \frac{\eta}{2\lambda} \sum_{t=1}^T \|\bg_t\|^2_\star\\
&\quad- \frac{B_\psi(\bu_{T}; \bx_{T+1})}{\eta},
\end{align*}
where $Q= \max_{t=1,\dots, T} \ \|\nabla \psi(\bx_t)-\nabla \psi(\bx_1)\|_\star$.
\end{theorem}
\begin{proof}
From the one-step inequality for \ac{OMD} with the competitor $\bu_t$ in Lemma~\ref{lemma:omd_one_step}, we have
\begin{align*}
\eta(\ell_t(\bx_t)-\ell_t(\bu_t))
&\leq B_\psi(\bu_t;\bx_t)-B_\psi(\bu_t;\bx_{t+1}) + \frac{\eta^2}{2 \lambda} \|\bg_t\|^2_\star\\
&=B_\psi(\bu_t;\bx_t)-B_\psi(\bu_{t+1};\bx_{t+1}) + B_\psi(\bu_{t+1};\bx_{t+1}) - B_\psi(\bu_t;\bx_{t+1}) \\
&\quad + \frac{\eta^2}{2 \lambda} \|\bg_t\|^2_\star~.
\end{align*}

Dividing by $\eta$ and summing over time, we have
\begin{align*}
\sum_{t=1}^T (\ell_t(\bx_t)-\ell_t(\bu_t))
&\leq \sum_{t=1}^T \frac{B_\psi(\bu_t;\bx_t)-B_\psi(\bu_{t+1};\bx_{t+1})}{\eta}\\
&\quad + \sum_{t=1}^T \frac{B_\psi(\bu_{t+1};\bx_{t+1}) - B_\psi(\bu_t;\bx_{t+1})}{\eta} + \sum_{t=1}^T \frac{\eta}{2 \lambda} \|\bg_t\|^2_\star \\
&= \frac{B_\psi(\bu_1;\bx_1)-B_\psi(\bu_{T+1};\bx_{T+1})}{\eta} \\
&\quad + \sum_{t=1}^T \frac{B_\psi(\bu_{t+1};\bx_{t+1}) - B_\psi(\bu_t;\bx_{t+1})}{\eta} + \frac{\eta}{2\lambda} \sum_{t=1}^T \|\bg_t\|^2_\star~.
\end{align*}
Now, observe that
\begin{align*}
B_\psi&(\bu_{t+1};\bx_{t+1}) - B_\psi(\bu_t;\bx_{t+1})\\
&=\psi(\bu_{t+1})-\psi(\bx_{t+1})-\langle \nabla \psi(\bx_{t+1}), \bu_{t+1}-\bx_{t+1}\rangle -\psi(\bu_t) +\psi(\bx_{t+1})\\
&\quad +\langle \nabla \psi(\bx_{t+1}), \bu_t-\bx_{t+1}\rangle\\
&=\psi(\bu_{t+1})-\psi(\bu_t)+\langle \nabla \psi(\bx_1), \bu_{t}-\bu_{t+1}\rangle + \langle \nabla \psi(\bx_{t+1})-\nabla\psi(\bx_1), \bu_{t}-\bu_{t+1}\rangle.
\end{align*}
Hence, we have
\begin{align*}
\sum_{t=1}^T &(B_\psi(\bu_{t+1};\bx_{t+1}) - B_\psi(\bu_t;\bx_{t+1}))\\
&=\psi(\bu_{T+1}) - \psi(\bu_1) + \langle \nabla \psi(\bx_1), \bu_1-\bu_{T+1}\rangle\\
&\quad+\sum_{t=1}^T \langle \nabla \psi(\bx_{t+1})-\nabla\psi(\bx_1), \bu_{t}-\bu_{t+1}\rangle\\
&=B_\psi(\bu_{T+1}; \bx_1)-B_\psi(\bu_1;\bx_1)+\sum_{t=1}^T \langle \nabla \psi(\bx_{t+1})-\nabla\psi(\bx_1), \bu_{t}-\bu_{t+1}\rangle~.
\end{align*}
Putting everything together, we have
\begin{align*}
\DRegret(\bu_1, \dots, \bu_T)
&\leq \frac{B_\psi(\bu_{T+1};\bx_1) - B_\psi(\bu_{T+1}; \bx_{T+1})}{\eta}+ \frac{\eta}{2\lambda} \sum_{t=1}^T \|\bg_t\|^2_\star\\
&\quad +\frac{1}{\eta}\sum_{t=1}^T \langle \nabla \psi(\bx_{t+1})-\nabla\psi(\bx_1), \bu_{t}-\bu_{t+1}\rangle~.
\end{align*}
Observe that the choice of $\bu_{T+1}$ does not affect the l.h.s. of the inequality, so we can set it equal to $\bu_T$.
Then, assuming that $\|\nabla \psi(\bx_t)-\nabla \psi(\bx_1)\|_\star \leq Q$ and using the definition of dual norm, we have the stated bound.
\end{proof}

Assuming $\|\bg_t\|_\star \leq L$ and selecting the usual learning rate $\eta=\frac{\alpha \sqrt{2\lambda}}{L\sqrt{T}}$, we have
\[
\DRegret(\bu_1, \dots, \bu_T)
\leq \frac{L}{\sqrt{2\lambda}} \left(\frac{B_\psi(\bu_{T};\bx_1) + Q P^{\|\cdot\|}(\bu_1, \dots, \bu_T)}{\alpha} + \alpha \right) \sqrt{T}~.
\]
In other words, we suffer an additional regret of
\[
L \frac{Q P^{\|\cdot\|}(\bu_1, \dots, \bu_T)}{\alpha \sqrt{2\lambda}} \sqrt{T}
\]
compared to the static case, for any $\bu_1, \dots, \bu_T \in \mathcal{V}$.

\begin{example}
Consider the case that the feasible set has diameter $D$ with respect to the L$_2$ norm, i.e., $D=\max_{\bx,\by \in \mathcal{V}} \|\bx-\by\|_2<\infty$. Set $\psi(\bx)=\frac12 \|\bx\|_2^2$ and assume that the subgradients satisfy $\|\bg_t\|_\star\leq L$ for $t=1, \dots, T$. In this case, we have that $Q\leq D$ and setting $\eta=\frac{\alpha}{L \sqrt{T}}$ we have
\[
\DRegret(\bu_1, \dots, \bu_T)
\leq \frac12 \left(\frac{D^2 + 2D \, P^{\|\cdot\|_2}(\bu_1, \dots, \bu_T)}{\alpha} + \alpha\right) L \sqrt{T}~.
\]
This means that we get sublinear regret with respect to any competitor sequence whose path length is $o(\sqrt{T})$ as $T$ goes to infinity.
\end{example}

Could we obtain a better regret guarantee? Setting $\eta=\frac{\sqrt{2\lambda}\sqrt{B_\psi(\bu_{T};\bx_1) + Q P^{\|\cdot\|}(\bu_1, \dots, \bu_T)}}{L\sqrt{T}}$, we could obtain the dynamic regret
\[
\DRegret(\bu_1, \dots, \bu_T)
\leq \frac{\sqrt{2}}{\sqrt{\lambda}} L \sqrt{B_\psi(\bu_{T};\bx_1) + Q P^{\|\cdot\|}(\bu_1, \dots, \bu_T)} \sqrt{T}~.
\]
This would give us a sublinear regret with respect to any competitor sequence whose path length is $o(T)$ as $T$ goes to infinity, which is a larger set than the one in the example.
However, assuming knowledge of the exact value of $P^{\|\cdot\|}$ is unreasonable because it violates the assumption of the adversarial nature of the problem. This mirrors the problem of tuning the learning rate with knowledge of $\|\bx_1-\bu\|$ in \ac{OSD}.
Fortunately, there is a solution, as we will see in the next section.

\index{ADER algorithm|(textbf}
\subsection{ADER: Optimal Dependence on the Path Length for Online Subgradient Descent}

In the previous section, we saw that the improved bound requires the knowledge of the path length of the competitor sequence.
Here, we show how to construct an online learning algorithm that achieves the same guarantees up to polylogarithmic terms. For simplicity, we will consider the Euclidean case, but it is easy to extend it with Bregman divergences\index{Bregman divergence}. Moreover, we assume the losses to be $L$-Lipschitz with respect to the L$_2$ norm.

We will use an approach similar to the one in Section~\ref{sec:easy_metagrad}: running a number of projected \ac{OSD} algorithms with different learning rates in parallel and using \ac{EG} to learn online the best combination of their iterates. In this way, we will see that the cumulative loss of the resulting algorithm is close to the cumulative loss of the best learning rate, which in turn will give us the right dependence on the path length.

Consider a feasible set $\mathcal{V}$ with bounded L$_2$ diameter $D$ and assume the losses $\ell_t$ to be $L$-Lipschitz on an open set containing $\mathcal{V}$. Using projected \ac{OSD} with learning rate $\eta$, for all $\bu_1, \dots, \bu_T\in \mathcal{V}$, we obtain the following regret upper bound
\[
\DRegret(\bu_1, \dots, \bu_T)
\leq \frac{D^2 + 2D \, P^{\|\cdot\|_2}(\bu_1, \dots, \bu_T)}{2\eta} + \frac{\eta L^2 T}{2}~.
\]
Using the fact that $P^{\|\cdot\|_2}(\bu_1, \dots, \bu_T)\leq D T$, the choice of $\eta^\star=\frac{\sqrt{D^2 + 2D \, P^{\|\cdot\|_2}(\bu_1, \dots, \bu_T)}}{L \sqrt{T}}$ satisfies
\[
\frac{D}{L\sqrt{T}}
\leq \eta^\star
\leq \frac{D\sqrt{1+2T}}{L\sqrt{T}}~.
\]

So, consider a grid of $N=1+\lceil\ln_2 \sqrt{4+8T}\rceil$ learning rates $\eta^{(i)}=\frac{D 2^{i-1}}{L\sqrt{T}}, i=1, \dots, N$, so that $\eta^{(1)}=\frac{D}{L\sqrt{T}}$ and $\eta^{(N)}\geq 2\frac{D\sqrt{1+2T}}{L\sqrt{T}}$. This implies that there exists $i^\star$ such that
\[
\eta^{(i^\star)}
\leq \eta^\star
\leq 2\eta^{(i^\star)}~.
\]

To combine the \ac{OSD} algorithms with different learning rates, we use the \ac{EG} algorithm where the \ac{OSD} algorithms are our experts.
Construct the loss vector for \ac{EG} as $\bz_t = [\ell_t(\bx^{(1)}_{t}), \dots, \ell_t(\bx^{(N)}_{t})]^\top$.
The \ac{EG} algorithm is invariant to adding the same constant to all coordinates of the loss vector (see Problem~\ref{exercise:eg_invariant_constant}), hence we can consider it as running on the losses $\ell_t(\bx^{(i)}_t)-\ell_t(\bx_1)$ which makes them bounded in absolute value by $D L$.
Hence, using a uniform prior and learning rate $\beta=\frac{\sqrt{2\ln N}}{L D \sqrt{T}}$, we obtain the regret
\[
\sum_{t=1}^T \langle \bp_t, \bz_t\rangle - \sum_{t=1}^T z_{t,i^\star}
\leq \sqrt{2} L\,D \sqrt{T \ln N}
\leq \sqrt{2} L\,D \sqrt{T \ln \left(2+\ln_2 \sqrt{4+8T}\right)} ~.
\]
Now, observe that by Jensen's inequality\index{inequality!Jensen's} (Theorem~\ref{thm:jensen}) and the convexity of $\ell_t$, we have
\[
\langle \bp_t, \bz_t\rangle
= \sum_{i=1}^N p_{t,i} \ell_t(\bx^{(i)}_t)
\geq \ell_t\left(\sum_{i=1}^N p_{t,i} \bx^{(i)}_t\right)
= \ell_t(\bx_t)~.
\]
This motivates the choice of using a convex combination of the predictions of the \ac{OSD} algorithms. Moreover, we have $\sum_{t=1}^T z_{t,i^\star}=\sum_{t=1}^T \ell_t(\bx^{(i^\star)}_t)$.

From the regret of \ac{OSD} and the choice of $\eta^{(i^\star)}$, for all $\bu_1, \dots, \bu_T\in \mathcal{V}$, we have
\[
\sum_{t=1}^T \ell_t(\bx^{(i^\star)}_t)
\leq \sum_{t=1}^T \ell_t(\bu_t) + \frac{3}{2} L \sqrt{T (D^2 + 2D \, P^{\|\cdot\|_2}(\bu_1, \dots, \bu_T))}~.
\]

Putting everything together, for all $\bu_1, \dots, \bu_T\in \mathcal{V}$, we have
\begin{align*}
&\DRegret(\bu_1,\dots, \bu_T)\\
&\quad = \sum_{t=1}^T (\ell_t(\bx_t)-\ell_t(\bu_t))\\
&\quad = \sum_{t=1}^T \left(\ell_t(\bx_t)-\ell_t\left(\bx^{(i^\star)}_t\right)\right) + \sum_{t=1}^T \left(\ell_t\left(\bx^{(i^\star)}_t\right)-\ell_t(\bu_t)\right)\\
&\quad \leq  \sqrt{2} L D \sqrt{T \ln \left(2+\ln_2 \sqrt{4+8T}\right)}+\frac32 L \sqrt{T (D^2 + 2D \, P^{\|\cdot\|_2}(\bu_1, \dots, \bu_T))}~.
\end{align*}

We are still not completely done: this construction queries $N = \mathcal{O}(\ln T)$ subgradients in each step, so now we show how to reduce it to only one subgradient per iteration.
This is easily achieved: it is enough to run the construction on the linear surrogate losses $\tilde{\ell}_t(\bx)=\langle \bg_t, \bx\rangle$, where $\bg_t\in \partial \ell_t(\bx_t)$. The advantage is that $\partial \tilde{\ell}_t(\bx) = \{\bg_t\}$ for all $\bx$, hence all the \ac{OSD} algorithms receive the same subgradient! Once again, we have that $|\tilde{\ell}_t(\bx_t^{(i)})-\tilde{\ell}_t(\bx_1)|\leq L \, D$ for all $i=1, \dots, N$. Moreover, we have
\[
\ell_t(\bx_t)-\ell_t(\bu_t)
\leq \langle \bg_t, \bx_t-\bu_t\rangle
= \tilde{\ell}_t(\bx_t)-\tilde{\ell}_t(\bu_t), \quad \forall \bu_t \in \mathcal{V}~.
\]
Hence, a dynamic regret guarantee on $\tilde{\ell}_t$ immediately translates to a dynamic regret guarantee on $\ell_t$.
The resulting algorithm is called \textbf{\ac{ADER}} and it is in Algorithm~\ref{alg:ader}.

\begin{algorithm}[t]
\caption{Adaptive learning for Dynamic EnviRonment (ADER)}
\label{alg:ader}
\begin{algorithmic}[1]
{
    \REQUIRE{Non-empty, closed, and convex feasible set $\mathcal{V}\subset \R^d$, number of iterations $T$, diameter of the feasible set $D<\infty$, Lipschitz constant $L$, initial point $\bx_1 \in \mathcal{V}$}
    \STATE{Set $N=1+\lceil\ln_2(\sqrt{4+8T})\rceil$}
    \STATE{Set $\eta^{(i)}=\frac{D 2^{i-1}}{L \sqrt{T}}, i = 1, \dots, N$}
    \STATE{Set $\beta = \frac{\sqrt{2\ln N}}{L D\sqrt{T}}$}
    \STATE{$\bp_1=[1/N, \dots, 1/N] \in \R^{N}$}
    \STATE{$\bx^{(i)}_1=\bx_1, i=1, \dots, N$}
    \FOR{$t=1$ {\bfseries to} $T$}
    \STATE{Output $\bx_t=\sum_{i=1}^{N} p_{t,i} \bx^{(i)}_t$ and pay $\ell_t(\bx_t)$}
    \STATE{Receive $\bg_t \in \partial \ell_t(\bx_t)$}
    \STATE{Update each \ac{OSD}: $\bx_{t+1}^{(i)} = \Pi_{\mathcal{V}}(\bx^{(i)}_t -\eta^{(i)} \bg_t), \ i=1, \dots, N$}
    \STATE{Update \ac{EG}: $p_{t+1,i} \propto p_{t,i}\exp\left(- \beta \langle \bg_t, \bx^{(i)}_t\rangle\right), \ i=1, \dots, N$}
    \ENDFOR
}
\end{algorithmic}
\end{algorithm}

Formally, we have the following theorem.
\begin{theorem}
Let $\mathcal{V}\subset \R^d$ be a non-empty closed convex set with bounded diameter with respect to the L$_2$ norm equal to $D$. Assume $\ell_t:\R^d \to (-\infty, +\infty]$ to be convex functions that are subdifferentiable on $\mathcal{V}$ and $L$-Lipschitz with respect to the L$_2$ norm on an open set containing $\mathcal{V}$. Then, for all $\bu_1, \dots, \bu_T \in \mathcal{V}$, Algorithm~\ref{alg:ader} satisfies
\begin{align*}
\DRegret(\bu_1,\dots, \bu_T)
&\leq \frac32 L \sqrt{T(D^2 + 2D \, P^{\|\cdot\|_2}(\bu_1, \dots, \bu_T))} \\
&\quad+ \sqrt{2} L D \sqrt{T \ln \left(2+\ln_2 \sqrt{4+8T}\right)}~.
\end{align*}
\end{theorem}

Note that while we query only one subgradient per round, the computational complexity of \ac{ADER} per round is still $\mathcal{O}(\ln T)$ when $T$ goes to infinity. This is due to the fact that in each round we need to update $\mathcal{O}(\ln T)$ different \ac{OSD} algorithms.

\index{ADER algorithm|)textbf}

\subsection{Black-box Reduction for Dynamic Regret in Unbounded Feasible Sets}

\index{black-box reduction!for dynamic regret in unbounded feasible sets|(textbf}
In the previous section, we showed how to obtain the optimal dynamic regret guarantee in bounded feasible sets. Here, we show a black-box reduction to obtain dynamic regret in unconstrained feasible sets, by using the black-box reduction to magnitude and direction of Section~\ref{sec:mag_dir}.

\begin{theorem}
\label{thm:dynamic_reduction_direction_magnitude}
Denote by $\DRegret^{\mathscr{A}_{B}}_T(\bu_1, \dots, \bu_T)$ the linear dynamic regret of algorithm $\mathscr{A}_{B}$ for any $\bu_1, \dots, \bu_T$ in the unit ball with respect to a norm $\|\cdot\|$, and $\Regret^{\mathscr{A}_{\text{1d}}}_T(u)$ the linear regret of algorithm $\mathscr{A}_{\text{1d}}$ for any competitor $u \in\R$. Then, for any sequence of convex losses $\ell_t:\R^d \to \R$ subdifferentiable in $\R^d$ for $t=1, \dots,T$ and any $\bu_1,\dots, \bu_T \in \R^d$, Algorithm~\ref{alg:onedimred} guarantees regret
\begin{align*}
\DRegret_T(\bu_1,\dots, \bu_T)
&= \sum_{t=1}^T \ell_t(\bx_t) - \sum_{t=1}^T \ell_t(\bu_t)
\leq \sum_{t=1}^T \langle \bg_t, \bx_t -\bu_t\rangle\\
&= \Regret^{\mathscr{A}_{\text{1d}}}_T(U) + U \cdot \DRegret^{\mathscr{A}_{B}}_T\left(\frac{\bu_1}{U}, \dots, \frac{\bu_T}{U}\right),
\end{align*}
where $U:=\max_{t=1,\dots, T} \|\bu_t\|$ and $U\neq 0$.
If $U=0$, then
\[
\DRegret_T(\boldsymbol{0},\dots, \boldsymbol{0})
=\Regret_T(\boldsymbol{0})
\leq \Regret^{\mathscr{A}_{\text{1d}}}_T(0)~.
\]
Further, the subgradients $s_t$ sent to $\mathscr{A}_{\text{1d}}$ satisfy $|s_t|\le \|\bg_t\|_\star$.
\end{theorem}
\begin{proof}
First, observe that $|s_t|\le \|\bg_t\|_\star \|\tilde{\bx}_t\| \le \|\bg_t\|_\star$ since $\|\tilde{\bx}_t\|\le 1$ for all $t$. Now, assuming $U \neq 0$ compute:
\begin{align*}
\DRegret_T(\bu_1, \dots, \bu_T)
&=\sum_{t=1}^T \ell_t(\bx_t) - \sum_{t=1}^T \ell_t(\bu_t)\\
&\leq \sum_{t=1}^T \langle \bg_t, \bx_t - \bu_t\rangle
= \sum_{t=1}^T \langle \bg_t, z_t \tilde{\bx}_t\rangle - \langle \bg_t, \bu_t\rangle\\
&= \underbrace{\sum_{t=1}^T \left(\langle \bg_t, \tilde{\bx}_t \rangle z_t - \langle \bg_t, \tilde{\bx}_t\rangle U\right)}_{\text{linear regret of }\mathscr{A}_{\text{1d}}\text{ at } U \in \R} + \sum_{t=1}^T \left(\langle \bg_t, \tilde{\bx}_t \rangle U  - \langle \bg_t, \bu_t \rangle \right)\\
&= \Regret^{\mathscr{A}_{\text{1d}}}_T(U) + \sum_{t=1}^T \left(\langle \bg_t, \tilde{\bx}_t\rangle U  - \langle \bg_t, \bu_t\rangle \right)\\
&= \Regret^{\mathscr{A}_{\text{1d}}}_T(U) +U \underbrace{\sum_{t=1}^T \left(\langle \bg_t, \tilde{\bx}_t\rangle - \left\langle \bg_t, \frac{\bu_t}{U}\right\rangle \right)}_{\text{dynamic regret of }\mathscr{A}_{B}}\\
&= \Regret^{\mathscr{A}_{\text{1d}}}_T(U) + U \cdot \DRegret^{\mathscr{A}_{B}}_T\left(\frac{\bu_1}{U},\dots,\frac{\bu_T}{U}\right)~.
\end{align*}
The case $U=0$ is similarly proven.
\end{proof}

\begin{example}
We can use the \ac{ADER} algorithm\index{ADER algorithm} in $\mathcal{V}=\{\bx \in \R^d: \|\bx\|_2\leq 1\}$ as our direction learner and the parameter-free algorithm for \ac{OCO} in Section~\ref{sec:shifted_kt} as our magnitude learner.
We assume that each $\ell_t$ is $L$-Lipschitz with respect to the L$_2$ norm on $\R^d$.
In this way, for all $\bu_1, \dots, \bu_T \in \R^d$, we obtain
\begin{align*}
\DRegret_T(\bu_1,\dots,\bu_T)
&\le U L\sqrt{5 T\ln\!\left(\frac{U \sqrt{19 T}}{\epsilon}+1\right)} +\epsilon L \\
&\quad + 2 U L\sqrt{2T\ln\!\left(2+\ln_2\!\sqrt{4+8T}\right)}\\
&\quad + 3 U L \sqrt{T\left(1+P^{\|\cdot\|_2}\!\left(\frac{\bu_1}{U},\dots,\frac{\bu_T}{U}\right)\right)},
\end{align*}
where $U=\max_{t=1, \dots, T} \ \|\bu_t\|_2$. The dependence on the path length is the same as for \ac{ADER}, hence optimal, while the dependence on $U$ is optimal.
\end{example}
\index{regret!dynamic|)textbf}
\index{black-box reduction!for dynamic regret in unbounded feasible sets|)textbf}

\section{Strongly Adaptive Regret}
\label{sec:strongly_adaptive_regret}
\index{regret!strongly adaptive|(textbf}
We now introduce the concept of \emph{strongly adaptive regret}.
\begin{definition}
For any $\bu \in \mathcal{V}$, we define the \textbf{strongly adaptive regret} as
\[
\SARegret_{T,\tau}(\bu)
:=\max_{[s,s+\tau-1] \subseteq \{1,2,\dots, T\}} \ \sum_{t=s}^{s+\tau-1} \ell_t(\bx_t) - \sum_{t=s}^{s+\tau-1} \ell_t(\bu)~.
\]
We say that the algorithm is \textbf{strongly adaptive} if, for any fixed $\bu \in \mathcal{V}$, the additional price we pay with respect to the optimal bound on any specific interval of length $\tau$ is at most polylogarithmic in $T$.
\end{definition}



This definition captures the fact that we want the performance of the algorithm to be good on any interval of length $\tau$.

\index{regret!adaptive|(textbf}
\begin{remark}
One might be tempted to remove the dependence on $\tau$ and consider instead $\max_\tau \ \SARegret_{T,\tau}(\bu)$.
This notion is known as \textbf{adaptive regret}\index{adaptive regret|textbf}. However, it is clear that the worst-case value of $\max_\tau \ \SARegret_{T,\tau}(\bu)$ is $\Omega(\sqrt{T})$ for bounded domains and Lipschitz losses (from the lower bound in Chapter~\ref{ch:lower}), which is meaningless for intervals of size $\mathcal{O}(\sqrt{T})$. Hence, the adaptive regret does not allow us to reason about the performance of the algorithm on small intervals.
\end{remark}
\index{regret!adaptive|)textbf}

In the next section, we show how to obtain strongly adaptive algorithms, using once again a combination of online learning algorithms.

\index{CBCE algorithm|(textbf}
\subsection{CBCE: A Meta Algorithm for Strongly Adaptive Regret}


We will introduce now the \textbf{Coin Betting for Changing Environment (CBCE)} algorithm, see Algorithm~\ref{alg:cbce}. We use $T$ different \ac{OLO} algorithms, each one starting at a different time step, and we combine them with a sleeping-experts algorithm over a countably infinite number of experts. In particular, $\bx_t^{(s)}$ is the output at time $t$ of the \ac{OLO} algorithm started at time $s$, while $\bx_t$ is the convex combination of $\bx^{(1)}_t, \dots, \bx^{(t)}_t$ using the probability distribution produced by the sleeping expert algorithm.
Using the sleeping expert algorithm in Section~\ref{sec:from_kt_to_sleeping_experts} immediately yields the following theorem.

\begin{algorithm}[h]
\begin{algorithmic}[1]
    \REQUIRE{Non-empty, closed, and convex feasible set $\mathcal{V}\subset \R^d$, OLO black-box algorithm $\mathscr{B}$, sleeping expert algorithm $\mathscr{M}$}
    \FOR{$t=1$ {\bfseries to} $T$}
    \STATE{Pass the active set to $\mathscr{M}$ and get $\bp_t$ from $\mathscr{M}$ such that $p_{t,i}=0$ for $i>t$}
    \STATE{Get $\bx^{(i)}_t$ from the copy $i$ of $\mathscr{B}$, for $i=1, \dots, t$}
    \STATE{Output $\bx_t=\sum_{i=1}^{t} p_{t,i} \bx^{(i)}_t$ and pay $\ell_t(\bx_t)$}
    \STATE{Receive $\bg_t \in \partial \ell_t(\bx_t)$}
    \STATE{Pass $\tilde{\ell}_t(\bx)=\langle\bg_t, \bx\rangle$ to the copy $i$ of $\mathscr{B}$, for $i=1,\dots,t$}
    \STATE{Let $\bv_t=[\langle \bg_t, \bx^{(1)}_t\rangle, \dots \langle \bg_t, \bx^{(t)}_t\rangle, 0, 0, \dots] \in \R^\infty$}
    \STATE{Pass the loss $\ell^\text{sleeping}_t(\bx)=\langle \bv_t, \bx\rangle$ to the sleeping expert algorithm}
    \ENDFOR
\end{algorithmic}
\caption{Coin Betting for Changing Environment (CBCE)}
\label{alg:cbce}
\end{algorithm}

\begin{theorem}
Let $\mathcal{V}\subset \R^d$ be a non-empty bounded closed convex set and define $D=\max_{\bx,\by \in \mathcal{V}} \ \|\bx-\by\|$. Assume the losses $\ell_t:\R^d \to (-\infty, +\infty]$ to be subdifferentiable and $L$-Lipschitz with respect to $\|\cdot\|$ on an open set containing $\mathcal{V}$.
Let $\mathscr{M}$ be the variant of Algorithm~\ref{alg:kt-sleepingexperts} with a countable number of experts in Remark~\ref{remark:kt-sleepingexperts-countable}, where $\epsilon_i=\frac{1}{i(i+1)}$ for $i=1,2,\dots$, and $\mathscr{B}$ an \ac{OLO} algorithm whose regret we denote by $\Regret^\text{OLO}_t(\bu)$ for all $t \in \Nat$. Let $I=[s,s+\tau-1]\subset [1,2,\dots, T]$.
Then, for all $\bu \in \mathcal{V}$, Algorithm~\ref{alg:cbce} satisfies
\begin{align*}
\Regret_I(\bu)
&:= \sum_{t\in I} (\ell_t(\bx_t) - \ell_t(\bu))\\
&\leq \Regret^\text{OLO}_{\tau}(\bu) + D L \left(\sqrt{2 \tau \ln (e \tau s (s+1) +1)}+1\right)~.
\end{align*}
\end{theorem}
\begin{proof}
Like in the case of \ac{ADER}, a simple implementation of this algorithm would require querying $t$ values of the function $\ell_t$ at time $t$. However, we run the algorithm on the linearized losses $\langle \bg_t, \cdot\rangle$, where $\bg_t \in \partial \ell_t(\bx_t)$, so that the subgradient is only asked once.
Hence, we upper bound the regret using the linearized losses:
\[
\sum_{t=s}^{s+\tau-1} (\ell_t(\bx_t) - \ell_t(\bu))
\leq \sum_{t=s}^{s+\tau-1} \langle \bg_t,\bx_t- \bu\rangle
= \sum_{t=s}^{s+\tau-1} (\tilde{\ell}_t(\bx_t) - \tilde{\ell}_t(\bu))~.
\]
Next, we decompose the regret in the contributions of $\mathscr{B}$ and $\mathscr{M}$ as
\[
\sum_{t=s}^{s+\tau-1} (\tilde{\ell}_t(\bx_t) - \tilde{\ell}_t(\bu))
= \sum_{t=s}^{s+\tau-1} (\tilde{\ell}_t(\bx_t) - \tilde{\ell}_t(\bx^{(s)}_t)) + \sum_{t=s}^{s+\tau-1}(\tilde{\ell}_t(\bx^{(s)}_t) - \tilde{\ell}_t(\bu) )~.
\]
The first sum can be written as
\[
\sum_{t=s}^{s+\tau-1} (\tilde{\ell}_t(\bx_t) - \tilde{\ell}_t(\bx^{(s)}_t))
= \Regret^\text{sleeping}_{s+\tau-1}(\be_s),
\]
because the $s$-th copy of the algorithm $\mathscr{B}$ has been active only on rounds $s$ to $s+\tau-1$.
The losses passed to the sleeping-experts algorithm have an active range at most $DL$, because for any two active experts $m,n$,
\[
|\langle \bg_t,\bx_t^{(m)}-\bx_t^{(n)}\rangle|
\leq \|\bg_t\|_\star \|\bx_t^{(m)}-\bx_t^{(n)}\|
\leq D L ~.
\]
Hence, with the choice of $\epsilon_i=\frac{1}{i(i+1)}$, Algorithm~\ref{alg:kt-sleepingexperts}, through Remark~\ref{remark:kt-sleepingexperts-countable}, achieves a sleeping regret against the $s$-th expert of
$D L \sqrt{2 \tau \ln (e \tau s (s+1) +1)}+DL$, because the expert $s$ has been active for $\tau$ rounds. For the second sum, we simply have $\Regret^\text{OLO}_{\tau}(\bu)$ because the $s$-th copy of algorithm $\mathscr{B}$ was started on round $s$. 
\end{proof}

Hence, assuming that $R_\tau(\bu)=\mathcal{O}(\sqrt{\tau})$, the CBCE algorithm is strongly adaptive because the regret on any interval of length $\tau$ depends on $\sqrt{\tau}$ and, upper bounding $s$ with $T$, only polylogarithmically on $T$.

However, while we ask only one subgradient per round, the computational complexity of CBCE in round $t$ is proportional to $t$, because we have to update $t$ \ac{OLO} algorithms. So, in the next section, we consider an efficient variant of the CBCE algorithm, whose update is proportional to $\ln t$, with the same strongly adaptive regret guarantee.

\subsection{Efficient Version of CBCE}

The previous algorithm has the disadvantage of requiring the start of a new copy of algorithm $\mathscr{B}$ in each iteration. Instead, here we will see how to reduce the number of copies to the logarithm of the number of rounds, using \textbf{geometric coverings}\index{geometric covering|(textbf}.

For $k=0,1,2,\dots$, define the intervals
\[
\mathcal{I}_k
:=\{[i 2^k,(i+1)2^k-1] : i \in \Nat\}~.
\]
That is, each $\mathcal{I}_k$ is a partition of $\Nat \backslash \{1, 2, 3,\dots, 2^k-1\}$ into consecutive intervals of length $2^k$.
Also, define
\[
\mathcal{I}:=\cup_{k=0, 1, 2, \dots} \mathcal{I}_k
\qquad \text{and} \qquad
ACTIVE(t)
:= \{I \in \mathcal{I} : t \in I\},
\]
the set of the ``active'' intervals at time $t$, that is, the intervals that contain $t$.
By the definition of $\mathcal{I}_k$, for every $t \leq 2^k-1$ we have that no interval in $\mathcal{I}_k$ contains $t$, while for every $t \geq 2^k$ we have that a single interval in $\mathcal{I}_k$ contains $t$. Therefore,
$|ACTIVE(t)| = \lfloor\log_2 t\rfloor + 1$.
We will run a copy of an \ac{OCO} algorithm for each active interval. This means that we have a number of algorithms at each time step $t$ that is logarithmic in $t$.
Given an interval $J$ of $\Nat$, we will also use the definition of $\mathcal{I}|_J$
\[
\mathcal{I}|_J:=\{ I \in \mathcal{I}: I \subseteq J \},
\]
to denote the set of the intervals in $\mathcal{I}$ that are contained in $J$.

\begin{figure}[t]
{
\centering
\begin{Verbatim}[commandchars=\\\{\},codes={\catcode`$=3\catcode`_=8}]
\;  1  2  3  4  5  6  7  8  9  10 11 12 13 14 15 16 17 18 ...
$\mathcal{I}_0$[ ][ ][ ][ ][ ][ ][ ][ ][ ][ ][ ][ ][ ][ ][ ][ ][ ][ ] ...
$\mathcal{I}_1$   [    ][    ][    ][    ][    ][    ][    ][    ][   ...
$\mathcal{I}_2$         [          ][          ][          ][         ...
$\mathcal{I}_3$                     [                      ][         ...
\end{Verbatim}
}
\caption{Geometric covering intervals. Each interval is denoted by {\tt [ ]}.}
\label{fig:gc}
\commentAlt{Figure~\ref{fig:gc}. Diagram of geometric covering intervals over time indices. Row I_0 contains unit intervals, row I_1 contains intervals of length two, row I_2 contains longer intervals, and row I_3 contains still longer intervals.}
\end{figure}

The next lemma allows us to decompose the regret over any interval using the regret over intervals in $\mathcal{I}$.
\begin{lemma}
\label{lem:geo-intv}
Let $J = [q, s] \subseteq \Nat$ be an arbitrary interval. Then, the interval $J$ can be partitioned into two finite sequences of disjoint and consecutive intervals,
\[
(J_{-k}, \dots , J_0) \subseteq \mathcal{I}|_J \quad
\text{ and }
\quad (J_1, J_2,\dots, J_p) \subseteq \mathcal{I}|_J,
\]
such that i) $|J_{-i}|/|J_{-i+1}| \leq 1/2, \forall i \geq 1$; ii) $|J_{i}|/|J_{i-1}| \leq 1/2, \forall i \geq 2$.
\end{lemma}
\begin{proof}
The intuition of the proof is the following one: first, we choose the interval $J_0$ to be the biggest leftmost one in $\mathcal{I}|_J$. Then, we show that the others are selected so that they decrease in size to the left. On the right, $J_1$ can have the same size as $J_0$, and the others will decrease in size.

Denote the size of $J_{i}$ by $b_{i}$ for $i=-k, \dots, 0, \dots, p$ and let $b_0 = \max\{|J'| : J' \in \mathcal{I}|_J\}$ be the maximal size of any interval $J' \in \mathcal{I}$ that is contained in $J$. Among all of these intervals, let $J_0$ be the leftmost interval, i.e., we define
\begin{align*}
q_0 &:= \argmin\{q': [q',q'+b_0-1]\in \mathcal{I}|_J\},\\
s_0 &:= q_0+b_0-1,\\
J_0 &:=[q_0,s_0]~.
\end{align*}
Starting from $q_0 - 1$, we now define a sequence of intervals (in reverse order), denoted
$(J_{-1}, \dots , J_{-k})$ to cover the interval $[q, q_0-1]$, as follows:
%
%
\begin{align*}
[q_{-1},s_{-1}]
&:=J_{-1}
:=\argmax_{J'=[q',s'] \in \mathcal{I}|_{[q,q_0-1]}: s'=q_0-1} |J'|\\
&\vdots\\
[q_{-i},s_{-i}]
&:=J_{-i}
:=\argmax_{J'=[q',s'] \in \mathcal{I}|_{[q,q_{-i+1}-1]}: s'=q_{-i+1}-1} |J'|\\
&\vdots
\end{align*}
Clearly, this sequence is finite and the left endpoint of the leftmost interval, $J_{-k}$, is $q$. Denote the size of $J_{-i}$ by $b_{-i}$. We next prove that for every $i \geq 1$, $b_{-i}/b_{-i+1} \leq 1/2 $.
Given that the interval lengths are powers of $2$, it suffices to show that $b_{-i} < b_{-i+1}$ for every $i \geq 1$. We use induction. The base case follows from the fact that $b_0 > q_0-q$, otherwise the interval $[q_0-b_0,q_0-1]$ would belong to $\mathcal{I}|_J$, have the same size as $J_0$, and lie strictly to the left of $J_0$, contradicting the maximality of $b_0$ and the fact that $J_0$ is the leftmost interval of maximal size.
We next assume that the claim holds for every $i \in \{1, \dots , n - 1\}$ and prove for $n$. Assume by contradiction that
$b_{-n} \geq b_{-n+1}$.
Consider the interval $\hat{J}_{-n+1}$ which is obtained by extending $J_{-n+1}$ to its left by an additional size of $b_{-n+1}$, that is, $\hat{J}_{-n+1}:=[q_{-n+1}-b_{-n+1},q_{-n+1}-1]\cup J_{-n+1}$. Hence, $\hat{J}_{-n+1}$ is an interval of size $2b_{-n+1}$ and it is contained in $J$.
According to the induction hypothesis, $|\hat{J}_{-n+1}| = 2b_{-n+1} = 2^j b_{-n+2}$ for some $j\leq 0$. Given that $\hat{J}_{-n+1}$ is adjacent to the interval $J_{-n+2}$ and its size is $2^j|J_{-n+2}|$ for some $j\leq 0$, from the definition of $\mathcal{I}$, we have that $\hat{J}_{-n+1}\in \mathcal{I}|_{J}$, contradicting the maximality of $J_{-n+1}$.

Similarly, we define a sequence of disjoint and consecutive intervals, denoted $(J_1, \dots , J_p)$ that covers $[s_0 + 1, s]$.
%
\begin{align*}
[q_{1},s_{1}]
&:=J_{1}
:=\argmax_{J'=[q',s'] \in \mathcal{I}|_{[s_0+1,s]}: q'=s_0+1} |J'|\\
&\vdots\\
[q_{i},s_{i}]
&:=J_{i}
:=\argmax_{J'=[q',s'] \in \mathcal{I}|_{[s_{i-1}+1,s]}: q'=s_{i-1}+1} |J'|\\
&\vdots
\end{align*}
Clearly, this sequence is finite, and the right endpoint of the rightmost interval, $J_p$, is $s$. Denote the size of $J_i$ by $b_i$. We next prove that for every $i \geq 2$, $b_i/b_{i-1} \leq 1/2$ which is equivalent to $b_i< b_{i-1}$.

We prove it again by induction.
We assume that $b_i/b_{i-1} \in \{2^j: j\leq 0\}$  for every $i \in \{1, \dots, n-1\}$ and prove that $b_n < b_{n-1}$ for $n\geq 2$.
For the base case, from the definition of $b_0$, we have that $b_1 \leq b_0$.
Now, assume by contradiction that $b_n \geq b_{n-1}$. Then, we can consider the interval $\hat{J}_{n-1}$ which is obtained by extending $J_{n-1}$ to its right by an additional size of $b_{n-1}$. It follows that $\hat{J}_{n-1}$ is an interval of size $2b_{n-1}$
and it is contained in $J$.
According to the induction hypothesis, $|\hat{J}_{n-1}| = 2b_{n-1} = 2^j  b_{n-2}$ for some $j\leq 1$. We need to consider the following two cases:
\begin{itemize}
\item Case $j \leq 0$ (i.e., $b_{n-1}/b_{n-2} \leq 1/2$). Then, $\hat{J}_{n-1}$ is consecutive to $J_{n-2} \in \mathcal{I}$ and its size is $2^j |J_{n-2}|$, hence $\hat{J}_{n-1} \in \mathcal{I}|_J$, contradicting the maximality of $J_{n-1}$.
\item Case $j = 1$ (i.e., $b_{n-1} = b_{n-2}$). Then, $|\hat{J}_{n-1}| = 2b_{n-2}$ and this contradicts the maximality of $J_{n-2}$. \qedhere
\end{itemize}
\end{proof}

As an example of the above Lemma, consider the interval $J=[1,30]$ that is partitioned into the intervals $J_{-3}=[1,1]$, $J_{-2}=[2,3]$, $J_{-1}=[4,7]$, $J_{0}=[8,15]$, $J_1=[16,23]$, $J_2=[24,27]$, $J_3=[28,29]$, $J_4=[30,30]$.
\index{geometric covering|)textbf}

We will use an expert algorithm over a countable number of experts, hence we also have a prior over the infinite-dimensional simplex $\Delta^\infty:=\{ \bx \in \R^\infty: x_i\geq 0, \sum_{i=1}^\infty x_i=1\}$.

\begin{algorithm}[t]
\begin{algorithmic}[1]
    \REQUIRE{Non-empty, closed, and convex feasible set $\mathcal{V}\subset \R^d$, OLO black-box algorithm $\mathscr{B}$, sleeping expert algorithm $\mathscr{M}$}
    \FOR{$t=1$ {\bfseries to} $T$}
    \STATE{Pass the active set $ACTIVE(t)$ to $\mathscr{M}$ and get $\bp_t$ from $\mathscr{M}$ such that $p_{t,J}\neq0$ only for $J \in ACTIVE(t)$}
    \STATE{Get $\bx^{(J)}_t$ from the copy of $\mathscr{B}$ for $J \in ACTIVE(t)$}
    \STATE{Output $\bx_t=\sum_{J \in ACTIVE(t)} p_{t,J} \bx^{(J)}_t$ and pay $\ell_t(\bx_t)$}
    \STATE{Receive $\bg_t \in \partial \ell_t(\bx_t)$}
    \STATE{Pass $\tilde{\ell}_t(\bx)=\langle\bg_t, \bx\rangle$ to the copy of $\mathscr{B}$ for $J\in ACTIVE(t)$}
    \STATE{Let $v_{t,J}=\begin{cases}\langle \bg_t, \bx^{(J)}_t\rangle, & J \in ACTIVE(t)\\ 0, & \text{otherwise}\end{cases}$}
    \STATE{Pass the loss $\ell^\text{sleeping}_t(\bx)=\langle \bv_t, \bx\rangle$ to the sleeping expert algorithm $\mathscr{M}$}
    \ENDFOR
\end{algorithmic}
\caption{Efficient Coin Betting for Changing Environment}
\label{alg:efficient_cbce}
\end{algorithm}

\begin{theorem}
Let $\mathcal{V}\subset \R^d$ be a non-empty bounded closed convex set and define $D=\max_{\bx,\by \in \mathcal{V}} \ \|\bx-\by\|$. Assume the losses $\ell_t:\R^d \to (-\infty, +\infty]$ to be subdifferentiable and $L$-Lipschitz with respect to $\|\cdot\|$ on an open set containing $\mathcal{V}$.
Let $\mathscr{M}$ be the variant of Algorithm~\ref{alg:kt-sleepingexperts} with a countable number of experts in Remark~\ref{remark:kt-sleepingexperts-countable}, where the copy associated to the interval $J=[i_1,i_2] \in \mathcal{I}$ has the parameter $\epsilon_{J}$ defined recursively as
\begin{align*}
\epsilon_{J}
=
\begin{cases}
\frac{1}{i_1 (i_1+1) (\lfloor\log_2 i_1\rfloor + 1)},  & J \in \mathcal{I}_0\\
\sum_{I \in \mathcal{I}_{0}|_J} \epsilon_{I}, & J \in \mathcal{I}_k, k\geq 1
\end{cases}~.
\end{align*}
Let $\mathscr{B}$ be an \ac{OLO} algorithm such that, for every $t\in\Nat$, every $\bu \in \mathcal{V}$, and every sequence of linear losses $\tilde{\ell}_s(\bx)=\langle \bg_s,\bx\rangle$ with $\|\bg_s\|_\star\le L$, its regret satisfies $\Regret_t(\bu)\leq c D L\, t^\alpha$, where $\alpha \in (0,1)$ and $c>0$. Let $I=[s,s+\tau-1]\subset [1,2,\dots,T]$.
Then, for all $\bu \in \mathcal{V}$, Algorithm~\ref{alg:efficient_cbce} satisfies
\begin{align*}
\Regret_I(\bu)
:= \sum_{t\in I} (\ell_t(\bx_t) -  \ell_t(\bu))
&\leq 2\frac{2^\alpha-(4 \tau^2)^{-\alpha}}{2^\alpha-1} c D L \tau^\alpha + D L \\
&\quad + 10 D L \sqrt{\tau \ln (e(s+\tau)^4 + 1)}~.
\end{align*}
\end{theorem}
\begin{proof}
First of all, we upper bound the regret over the interval $I$ with the one over linearized losses:
\[
\sum_{t\in I} (\ell_t(\bx_t) - \ell_t(\bu))
\leq \sum_{t\in I} \langle \bg_t, \bx_t-\bu \rangle
= \sum_{t\in I} (\tilde{\ell}_t (\bx_t)- \tilde{\ell}_t (\bu)),
\]
where $\tilde{\ell}_t(\bx)=\langle \bg_t, \bx\rangle$.

Now, suppose we denote by $\bx^{(J)}_t$ the decision from the black-box run $\mathscr{B}_J$ associated with the interval $J$ at time $t$ and by $\bx_t$ the combined decision of the meta algorithm at time $t$. Since the complete algorithm is a combination of a meta algorithm $\mathscr{M}$ and several black-box algorithms $\mathscr{B}$, its regret depends on both $\mathscr{M}$ and $\mathscr{B}$. We now decompose the two sources of regret additively through the geometric covering\index{geometric covering}.

Let $\mathcal{J}=\bigcup_{i=-a}^{b} J_i$ be the partition of $I$ obtained from Lemma~\ref{lem:geo-intv}. Then, the regret on $I$ can be decomposed as follows:
\begin{align}
\Regret_I(\bu)
&= \sum_{t\in I} \left(\ell_t(\bx_t) - \ell_t(\bu) \right)
\leq \sum_{t\in I} \left(\tilde{\ell}_t(\bx_t) - \tilde{\ell}_t(\bu) \right) \notag \\
&= \sum_{i=-a}^{b} \sum_{t\in J_i} \left(\tilde{\ell}_t(\bx_t) - \tilde{\ell}_t(\bx^{(J_i)}_t)
   + \tilde{\ell}_t(\bx^{(J_i)}_t) - \tilde{\ell}_t(\bu) \right) \notag\\
&= \sum_{i=-a}^{b} \underbrace{ \sum_{t\in J_i} \left(\tilde{\ell}_t(\bx_t) - \tilde{\ell}_t(\bx^{(J_i)}_t)\right) }_{\text{\small  meta regret on $J_i$}}
      + \sum_{i=-a}^{b} \underbrace{ \sum_{t\in J_i} \left(\tilde{\ell}_t(\bx^{(J_i)}_t) - \tilde{\ell}_t(\bu) \right) }_{\text{\small black-box regret on $J_i$}}~. \label{regret-decomposition}
\end{align}
The black-box regret on $J_i \in \mathcal{J}$ is exactly the regret of the black-box algorithm on $|J_i|$ rounds, since the black-box run $\mathscr{B}_{J_i}$ was started at the beginning of the interval $J_i$.
In particular, we have
\begin{align*}
\sum_{i=-a}^b \sum_{t\in J_i} (\tilde{\ell}_t(\bx^{(J_i)}_t)-\tilde{\ell}_t(\bu))
&= \sum_{i=-a}^b \sum_{t\in J_i} \langle\bg_t,\bx^{(J_i)}_t-\bu\rangle
\leq c D L \sum_{i=-a}^b |J_i|^\alpha\\
&\leq 2 c D L\sum_{k=0}^{a+b+1} (2^{-k} |I|)^\alpha
= 2 c |I|^\alpha D L\frac{2^\alpha-2^{-\alpha (a+b+1)}}{2^\alpha-1}\\
&\leq 2 c |I|^\alpha D L\frac{2^\alpha-(4 |I|^2)^{-\alpha}}{2^\alpha-1},
\end{align*}
where the second inequality is due to Lemma~\ref{lem:geo-intv}, and in the last one we used that the number of geometric cover intervals is less than $2(\lfloor \log_2 |I|\rfloor+1)$ (again from Lemma~\ref{lem:geo-intv}) and $2^{-2(\lfloor \log_2 |I|\rfloor+1)}\geq \frac{1}{4|I|^2}$.

Now, we bound the meta-regret term in \eqref{regret-decomposition}. We apply the sleeping-experts guarantee once with the fixed comparator that puts weight $1$ on the experts associated with $J_{-a},\dots,J_b$ and weight $0$ on all the others. Since these intervals are disjoint, at each time $t\in I$ exactly one of these selected experts is active.
Moreover, the losses passed to the sleeping-experts algorithm have an active range at most $DL$, because for any two active experts $J,J'$,
\[
|\langle \bg_t,\bx_t^{(J)}-\bx_t^{(J')}\rangle|
\leq \|\bg_t\|_\star \|\bx_t^{(J)}-\bx_t^{(J')}\|
\leq D L ~.
\]
Hence,
\[
\sum_{i=-a}^{b} \sum_{t\in J_i} (\tilde{\ell}_t(\bx_t)-\tilde{\ell}_t(\bx^{(J_i)}_t))
\leq DL\sum_{i=-a}^{b} \sqrt{2 |J_i| \ln\left(1+\frac{e |J_i|}{\epsilon_{J_i}}\right)}
+ DL \sum_{J\in\mathcal{I}}\epsilon_J~.
\]


Also, from Lemma~\ref{lem:geo-intv}, $|J_i|\leq 2^i \tau$ for $i\leq 0$.
This implies that
\[
|J_i|/\epsilon_{J_i}
\leq |J_i| (s+\tau-1)(s+\tau)(\log_2(s+\tau-1)+1)
\leq (s+\tau)^4, \forall i\leq 0~.
\]
Hence, we have
\begin{align*}
\sum_{i=-a}^{0} \sqrt{2 |J_i| \ln\left(1+\frac{e |J_i|}{\epsilon_{J_i}}\right)}
&\leq \sum_{i=-a}^{0} \sqrt{2^{i+1} \tau \ln\left(1+e(s+\tau)^4\right)}\\
&\leq \sqrt{\tau \ln (e(s+\tau)^4 + 1)} \sum_{k=0}^\infty  \sqrt{2^{-k+1}}\\
&\leq 5 \sqrt{ \tau \ln (e(s+\tau)^4 + 1)}
\end{align*}
%
%
Analogously, we have
\[
\sum_{i=1}^b \sqrt{2 |J_i| \ln\left(1+\frac{e |J_i|}{\epsilon_{J_i}}\right)}
\leq 5 \sqrt{ \tau \ln (e(s+\tau)^4 + 1)}~.
\]
Finally, observe that $\sum_{J \in \mathcal{I}} \epsilon_J = 1$.
Putting everything together, we have the stated bound.
\end{proof}

\begin{remark}
Observe that the limit for $\alpha \to 0$ of the first term of the upper bound of the regret is $2 c D L \log_2(8\tau^2)$. However, the second term is always of the order of $\sqrt{\tau}$.
\end{remark}

From the above theorem, it is immediate to get a guarantee on the strongly adaptive regret by observing that
\[
  \SARegret_{T,\tau}(\bu) = \max_{I\subseteq\{1,\dots,T\}:|I|=\tau} \ \Regret_I (\bu)~.
\]

\index{regret!strongly adaptive|)textbf}
\index{CBCE algorithm|)textbf}

\section{Tracking Regret}

\index{regret!tracking|(textbf}
We now see yet another way to measure the ability of an online algorithm to adapt to changes in the behaviour of the adversarial losses.
\begin{definition}
For $m \in \{0, \dots, T-1\}$, define the \textbf{tracking regret} as
\[
\TRegret^m_T(\bu_1, \dots, \bu_T)
:= \DRegret_T(\bu_1, \dots, \bu_T),
\]
where the sequence $\bu_1, \dots, \bu_T\in \mathcal{V}$ changes at most $m$ times. More formally, there exists an integer $K\leq m+1$, a partition of $\{1, 2, \dots, T\}$ into $K$ consecutive intervals $(\mathcal{I}_i)_{i=1}^K$, and $\tilde{\bu}_1, \dots,\tilde{\bu}_K \in \mathcal{V}$ such that $\bu_t=\tilde{\bu}_i$ for all $t\in \mathcal{I}_i$.
\end{definition}

The tracking regret can be upper-bounded both in terms of dynamic regret and strongly adaptive regret.
Indeed, if we assume the diameter $D$ of $\mathcal{V}$ to be bounded with respect to some norm, i.e., $D=\max_{\bu,\bv \in \mathcal{V}} \ \|\bu-\bv\|<\infty$, then for a sequence $\bu_1, \dots, \bu_T \in \mathcal{V}$ that changes at most $m$ times we have
\[
P^{\|\cdot\|}(\bu_1, \dots, \bu_T)
\leq m D~.
\]
Hence, any algorithm with a dynamic regret guarantee that depends on $P^{\|\cdot\|}$ can also be used to obtain a tracking regret guarantee.

For example, we can consider the following algorithm.
\begin{example}
\label{example:shifting_regret}
Consider the setting of \ac{LEA}. In this setting, the tracking regret makes more sense because we are typically interested in competitors that are corners of the simplex.

To tackle this problem, we will use \ac{OMD} with the entropic regularizer over the truncated simplex $\mathcal{V}=\{\bx \in \R^d: x_i\geq \alpha/d, \sum_{i=1}^d x_i =1\}$, where $0<\alpha\leq 1$ is a parameter that will be decided in the following. Hence, we have $x_{t,i}\geq \frac{\alpha}{d}$ and $(\nabla \psi(\bx))_i = 1+\ln x_i$. Hence, in the dynamic case, we have
\[
Q
= \max_t \|\nabla \psi(\bx_t)-\nabla \psi(\bx_1)\|_\infty
= \max_t \|\ln \bx_t-\ln \bx_1\|_\infty
\leq \ln \frac{d}{\alpha},
\]
where the logarithm of vectors is taken element-wise.

Given that we want to consider $\bu_t$ in the simplex rather than in the truncated simplex, we have
\begin{align*}
\langle \bg_t, \bx_t -\bu_t\rangle
&= \left\langle \bg_t, \bx_t - \left((1-\alpha)\bu_t+\frac{\ones_d \alpha}{d}\right)\right\rangle + \left\langle \bg_t,\alpha \left(\frac{\ones_d}{d}-\bu_t\right)\right\rangle \\
&\leq \left\langle \bg_t, \bx_t - \left((1-\alpha)\bu_t+\frac{\ones_d \alpha}{d}\right)\right\rangle + 2\alpha  \|\bg_t\|_\infty~.
\end{align*}
Denoting $\tilde{\bu}_t=(1-\alpha)\bu_t + \frac{\ones_d \alpha}{d} \in \mathcal{V}$, the bound will depend on
\[
\sum_{t=2}^T \|\tilde{\bu}_t-\tilde{\bu}_{t-1}\|_1
=(1-\alpha) \sum_{t=2}^T \|\bu_t-\bu_{t-1}\|_1~.
\]
Hence, using the uniform prior $\bx_1=\ones_d/d$, we have that the tracking regret is upper bounded by
\[
\frac{\ln d + (1-\alpha)P^{\|\cdot\|_1}(\bu_1, \dots, \bu_T) \ln \frac{d}{\alpha}}{\eta}+ \frac{\eta}{2}\sum_{t=1}^T \|\bg_t\|^2_\infty + 2\alpha \sum_{t=1}^T \|\bg_t\|_\infty~.
\]
The sequence $\bu_1,\dots, \bu_T$ changes at most $m$ times, so $P^{\|\cdot\|_1}(\bu_1, \dots, \bu_T)\leq 2 m$. Hence, assuming $\|\bg_t\|_\infty\leq G_\infty$ for $t=1, \dots, T$, we have
\[
\TRegret^m_T(\bu_1, \dots, \bu_T)
\leq \frac{\ln d + 2m \ln \frac{d}{\alpha}}{\eta} + \frac{\eta T G_\infty^2}{2}+2 \alpha T G_\infty~.
\]
Assuming $m\geq1$, choosing $\eta=\sqrt{\frac{2(\ln d + 2m \ln \frac{d}{\alpha})}{ G_\infty^2 T}}$ and $\alpha=\sqrt{\frac{m}{T}}$, we have
\begin{align*}
\TRegret^m_T(\bu_1, \dots, \bu_T)
&\leq G_\infty \sqrt{2 T} \left(\sqrt{\ln d + 2m\ln \frac{d \sqrt{T}}{\sqrt{m}}} + \sqrt{2m} \right)\\
&= G_\infty \sqrt{2 T m} \left(\sqrt{\left(\frac{1}{m}+2\right)\ln d + \ln \frac{T}{m} } + \sqrt{2} \right)~.
\end{align*}
\end{example}

We can also use the strongly adaptive regret guarantee to control the tracking regret, as detailed in the next theorem.
\begin{theorem}
\label{thm:sa_regret_to_tracking}
Consider an algorithm that guarantees $\SARegret_{T,\tau}(\bu) \leq c \tau^\alpha$ for $\alpha \in (0,1)$, for all $\bu \in \mathcal{V}$ and all $\tau \in \{1, \dots, T\}$. Then,
\[
\TRegret^m_T(\bu_1,\dots,\bu_T)
\leq c (m+1)^{1-\alpha} T^{\alpha}, \quad \forall m \in \{0, \dots, T-1\}~.
\]
\end{theorem}
\begin{proof}

\begin{align*}
\TRegret^m_T(\bu_1,\dots,\bu_T)
&= \sum_{i=1}^K \sum_{t \in \mathcal{I}_i} (\ell_t(\bx_t) -  \ell_t(\tilde{\bu}_i))
\leq \sum_{i=1}^K c |\mathcal{I}_i|^\alpha\\
&\leq c \left(\sum_{i=1}^K 1^\frac{1}{1-\alpha}\right)^{1-\alpha} \left(\sum_{i=1}^K |\mathcal{I}_i|\right)^\alpha
= c K^{1-\alpha} T^\alpha\\
&\leq c (m+1)^{1-\alpha} T^\alpha,\end{align*}
where in the first inequality we used the assumption on the strongly adaptive regret and in the second one we used H\"{o}lder's inequality.
\end{proof}
\index{regret!tracking|)textbf}

\section{History Bits}
\citet{LittlestoneW94} introduced the tracking regret\index{regret!tracking} in the \ac{LEA} game, but for upper bounding the number of mistakes, rather than the regret. The extension to arbitrary losses is in \citet{HerbsterW98}.
\citet{Zinkevich03} is often erroneously thought to have introduced the dynamic regret\index{regret!dynamic}, while it first appeared in \citet{HerbsterW01}, with upper bounds that depend on the drift $\sum_{t=2}^T \|\bu_t-\bu_{t-1}\|_p$.

Theorem~\ref{thm:dynamic} is a generalization of \citet[Theorem 11.4]{Cesa-BianchiL06} to arbitrary distance generating functions\index{distance generating function} and considering Lipschitz losses. Also, following \citet[Lemma 4]{JacobsenC22}, I have fixed the offset in the proof so that we do not have to assume that $\boldsymbol{0} \in \mathcal{V}$.


\citet{ZhangLZ18} designed the \ac{ADER}\index{ADER algorithm} algorithm, and they also proved that it is optimal in bounded domains. The idea of combining online algorithms with different learning rates comes directly from the MetaGrad\index{MetaGrad algorithm} algorithm~\citep{vanErvenK16,VanErvenKV21} that also showed how to query a single gradient per round. In turn, MetaGrad builds on prior work on using a grid of learning rates in \ac{EG}~\citep{KoolenvEG14}. By now, this is a well-known method that allows us to solve essentially all problems of tuning learning rates in bounded domains, at least theoretically.
It is worth also stressing that the general idea of combining the outputs of different online learning algorithms through another online learning algorithm is instead much older, and it goes back at least to \citet{BlumM05,BlumM07}.

I described a slightly simpler version of \ac{ADER}\index{ADER algorithm|(} with a flat prior, see Problem~\ref{exercise:ader} for the original bound. I also removed the unnecessary assumption that $\boldsymbol{0} \in \mathcal{V}$.
Recently, \citet{ZhaoZZZ20,ZhaoZZZ24} improved the guarantees of \ac{ADER} using optimistic algorithms, obtaining smaller bounds in easy environments.
\index{ADER algorithm|)}

Before ADER, there was also a line of work on the path length evaluated on $\bu^\star_t \in \argmin_{\bx \in \mathcal{V}} \ \ell_t(\bx)$ for $t=1, \dots, T$~\citep{JadbabaieRSS15}. This is an easier setting because one can evaluate the set of minimizers of $\ell_t$ after receiving the loss, so one can use a simple doubling trick on the learning rate\index{doubling trick}. Also, the resulting bound does not strictly generalize the static regret analysis. Hence, it is difficult to say if the resulting algorithm is better or worse than one with a static regret.

Theorem~\ref{thm:dynamic_reduction_direction_magnitude} is from \citet{JacobsenC22}\index{black-box reduction!for dynamic regret in unbounded feasible sets}.

The adaptive regret\index{regret!adaptive} is defined in \citet{HazanS07, HazanS09}. The strongly-adaptive regret\index{regret!strongly adaptive} was defined by \citet{AdamskiyKCV12,AdamskiyKCV16} (still calling it ``adaptive regret'') and then reinvented in \citet{DanielyGSS15} (that coined the name). Note that the strongly-adaptive regret is usually defined for bounded domains, taking the maximum with respect to $\bu \in \mathcal{V}$, whereas I defined it more generally for any competitor in the feasible set.
Lemma~\ref{lem:geo-intv} is from \citet{DanielyGSS15} as well.
The efficient version of the CBCE algorithm was proposed by \citet{JunOWW17,JunOWW17b} using a stronger version of the sleeping expert algorithm, still based on coin betting. Here, I simplified the sleeping expert algorithm for didactical reasons. CBCE improves over the guarantee in \citet{DanielyGSS15} by improving the logarithmic term. The idea of using linearized losses in CBCE to avoid querying more than one subgradient and the better behavior of the first term in the bound when $\alpha \to 0$ are new. The non-efficient version of CBCE is also new, but straightforward.

The algorithm in Example~\ref{example:shifting_regret} is introduced in \citet{HerbsterW01} and generalized to arbitrary distance generating functions\index{distance generating function} in \citet{Cesa-BianchiGLS12}.
Theorem~\ref{thm:sa_regret_to_tracking} is from \citet{DanielyGSS15}.

\section{Exercises}

\begin{exer}
\label{exercise:ader}
Change the prior in the \ac{EG} algorithm to obtain a dependence of $\ln \ln (1+P^{\|\cdot\|_2}(\bu_1, \dots, \bu_T))$ instead of $\ln \ln T$. In this way, if the path length is zero, that is, we are in the static case, the regret is $\mathcal{O}(\sqrt{T})$ instead of $\mathcal{O}(\sqrt{T \ln \ln T})$, when $T\to \infty$.
\end{exer}

\acresetall

\chapter{Saddle-Point Optimization and Online Algorithms}
\label{ch:saddle-point}

In this chapter, we talk about solving saddle-point problems with \ac{OCO} algorithms, and the connection with game theory.

\acresetall

\section{Saddle-Point Problems}

We want to solve the following saddle-point problem
\begin{equation}
\label{eq:saddle_point}
\inf_{\bx \in \mathcal{X}} \sup_{\by\in \mathcal{Y}} \ f(\bx,\by)~.
\end{equation}

Let's say from the beginning that we need $\inf$ and $\sup$ rather than $\min$ and $\max$ because the minimum or maximum might not exist. Every time we know for sure the $\inf$/$\sup$ are attained, we can substitute them with $\min$/$\max$.

While it is clear what it means to minimize a function, it might not be immediate to see the meaning of ``solving'' the saddle-point problem in \eqref{eq:saddle_point}.
It turns out that the proper notion we are looking for is the one of \emph{saddle point}.
\begin{definition}
Let $\mathcal{X} \subseteq \R^n$, $\mathcal{Y} \subseteq \R^m$, and $f: \mathcal{X} \times \mathcal{Y} \to \R$.
A point $(\bx^\star, \by^\star) \in \mathcal{X} \times \mathcal{Y}$ is a \textbf{saddle point}\index{saddle point|textbf} of $f$ in $\mathcal{X} \times\mathcal{Y}$ if
\[
f(\bx^\star, \by)
\leq f(\bx^\star, \by^\star)
\leq f(\bx, \by^\star), \quad \forall \bx \in \mathcal{X}, \by \in \mathcal{Y}~.
\]
\end{definition}

We will now state conditions under which there \emph{exists} a saddle point that solves \eqref{eq:saddle_point}.
First, we need an easy lemma.
\begin{lemma}
\label{lemma:min_max_less_max_min}
Let $f$ be a function from a non-empty product set $\mathcal{X} \times \mathcal{Y}$ to $\R$. Then,
\[
\inf_{\bx \in \mathcal{X}} \sup_{\by \in \mathcal{Y}} \ f(\bx,\by)
\geq \sup_{\by \in \mathcal{Y}} \inf_{\bx \in \mathcal{X}} \ f(\bx,\by)~.
\]
\end{lemma}
\begin{proof}
For any $\bx' \in \mathcal{X}$ and $\by \in \mathcal{Y}$ we have that $f(\bx',\by) \geq \inf_{\bx \in \mathcal{X}} \ f(\bx,\by)$. This implies that
$\sup_{\by \in \mathcal{Y}} \ f(\bx', \by) \geq \sup_{\by \in \mathcal{Y}} \inf_{\bx \in \mathcal{X}} \ f(\bx,\by)$ for all $\bx' \in \mathcal{X}$, which gives the stated inequality.
\end{proof}

We can now state the following theorem.
\begin{theorem}
\label{thm:saddle_equiv_minmax}
Let $f$ be any function from a non-empty product set $\mathcal{X} \times \mathcal{Y}$ to $\R$. A point $(\bx^\star, \by^\star)$ is a saddle point of $f$ if and only if the supremum in
\begin{equation}
\label{eq:saddle_equiv_minmax_1}
\sup_{\by \in \mathcal{Y}} \inf_{\bx \in \mathcal{X}} \ f(\bx, \by)
\end{equation}
is attained at $\by^\star$, the infimum in
\begin{equation}
\label{eq:saddle_equiv_minmax_2}
\inf_{\bx \in \mathcal{X}} \sup_{\by \in \mathcal{Y}} \ f(\bx, \by)
\end{equation}
is attained at $\bx^\star$, and these two expressions are equal.
\end{theorem}
\begin{proof}
If $(\bx^\star,\by^\star)$ is a saddle point, then we have
\begin{align*}
f(\bx^\star,\by^\star) &= \inf_{\bx \in \mathcal{X}} \ f(\bx, \by^\star) \leq \sup_{\by \in \mathcal{Y}} \inf_{\bx \in \mathcal{X}} \ f(\bx, \by)\\
f(\bx^\star,\by^\star) &= \sup_{\by \in \mathcal{Y}} \ f(\bx^\star, \by) \geq \inf_{\bx \in \mathcal{X}} \sup_{\by \in \mathcal{Y}} \ f(\bx, \by)~.
\end{align*}
From Lemma~\ref{lemma:min_max_less_max_min}, we have that these quantities must be equal, so that the three conditions in the theorem are satisfied.

For the other direction, if the conditions are satisfied, then we have
\[
f(\bx^\star, \by^\star)
\leq \sup_{\by \in \mathcal{Y}} \ f(\bx^\star, \by)
= \inf_{\bx \in \mathcal{X}} \ f(\bx, \by^\star)
\leq f(\bx^\star, \by^\star)~.
\]
Hence, $(\bx^\star,\by^\star)$ is a saddle point.
\end{proof}

\begin{remark}
The above theorem implies that the set of saddle points, when nonempty, is the Cartesian product $\mathcal{X}^\star \times \mathcal{Y}^\star$, where $\mathcal{X}^\star$ and $\mathcal{Y}^\star$ are the sets of optimal solutions of the optimization problems \eqref{eq:saddle_equiv_minmax_2} and \eqref{eq:saddle_equiv_minmax_1}, respectively. In other words, $\bx^\star$ and $\by^\star$ can be independently chosen from the sets $\mathcal{X}^\star$ and $\mathcal{Y}^\star$, respectively, to form a saddle point.
\end{remark}

\begin{remark}
The above theorem tells a surprising thing: if a saddle point exists, then there might be multiple ones, and all of them must have the same minimax value. This might seem surprising, but it is due to the fact that the definition of saddle point is a global and not a local property. Moreover, if the $\inf\sup$ and $\sup\inf$ problem have different values, no saddle point exists.
\end{remark}



\begin{figure}
\centering
\begin{tikzpicture}
\begin{axis}[
  width=7cm,
  view={-35.7}{39},
  xlabel={$x$},
  ylabel={$y$},
  zlabel={$f(x,y)$},
  xtick={-0.5, 0, 0.5, 1},
  ytick={-1,-0.5, 0, 0.5, 1},
  ztick={0,1,2,3,4},
  tick style={draw=none},
  grid,
  xmin=-1, xmax=1,
  ymin=-1, ymax=1,
  zmin=0, zmax=4,
  colormap={bw}{gray(0cm)=(0.9); gray(1cm)=(0.1)},
]
\addplot3[
  surf,
  mesh/rows=41,
] table[
  x=x, y=y, z=z,
  col sep=comma
] {code_for_figs/no_saddle.csv};
\end{axis}
\end{tikzpicture}
\caption{The function $f(x,y)=(x-y)^2$ does not have a saddle point.}
\label{fig:no_saddle_point}
\commentAlt{Figure~\ref{fig:no_saddle_point}. Three-dimensional surface plot of f(x,y)=(x-y)^2 over a square domain, illustrating that the surface has no saddle point.}
\end{figure}

Let's show a couple of examples that show that the conditions above are indeed necessary.
\begin{example}
Let $f(x,y)=(x-y)^2$, $\mathcal{X} =[-1,1]$, and $\mathcal{Y} =[-1,1]$. Then, we have
\[
\inf_{x \in \mathcal{X}} \sup_{y \in \mathcal{Y}} \ (x-y)^2
= \inf_{x \in \mathcal{X}} \ (1+|x|)^2
= 1,
\]
while
\[
\sup_{y \in \mathcal{Y}} \inf_{x \in \mathcal{X}} \ (x-y)^2
= \sup_{y \in \mathcal{Y}} 0
= 0~.
\]
Indeed, from Figure~\ref{fig:no_saddle_point} we can see that there is no saddle point.
\end{example}

\begin{example}
\label{example:no_saddle}
Let $f(x,y)=x y $, $\mathcal{X} =(0,1]$, and $\mathcal{Y} =(0,1]$. Then, we have
\[
\inf_{x \in \mathcal{X}} \sup_{y \in \mathcal{Y}} \ x y
= \inf_{x \in \mathcal{X}} \ x
= 0
\]
and
\[
\sup_{y \in \mathcal{Y}} \inf_{x \in \mathcal{X}} \ xy
= 0~.
\]
Here, even if inf sup is equal to sup inf, the saddle point does not exist because the inf in the first expression is not attained at a point of $\mathcal{X}$.
\end{example}

Theorem~\ref{thm:saddle_equiv_minmax} also tells us that to find a saddle point of a function $f$, we need to find the minimizer in $\bx$ of $\sup_{\by \in \mathcal{Y}} \ f(\bx,\by)$ and the maximizer in $\by$ of $\inf_{\bx \in \mathcal{X}} \ f(\bx,\by)$.
Let's now use this knowledge to design a proper measure of progress towards the saddle point.

We might be tempted to use $f(\bx', \by')-f(\bx^\star,\by^\star)$ as a measure of suboptimality of $(\bx',\by')$ with respect to the saddle point $(\bx^\star,\by^\star)$. Unfortunately, this quantity can be negative or equal to zero for an infinite number of points $(\bx',\by')$ that are not saddle points. We might then think to use some notion of distance to the saddle point, like $\|\bx'-\bx^\star\|^2_2+\|\by'-\by^\star\|^2_2$, but this quantity in general can go to zero at an arbitrarily slow rate. To see why, consider the case that $f(\bx,\by)=h(\bx)$, so that the saddle-point problem reduces to minimizing a convex function. So, assuming only convexity, the rate of convergence to a minimizer of $f$ can be arbitrarily slow because the function can be arbitrarily close to a flat surface, but not flat, around the minimizer.
Hence, we need something different.

Observe that Theorem~\ref{thm:saddle_equiv_minmax} says one of the problems we should solve is
\[
\inf_{\bx \in \mathcal{X}} \ h(\bx),
\]
where $h(\bx) = \sup_{\by\in \mathcal{Y}} \ f(\bx,\by)$. In this view, the problem looks like a standard offline convex optimization problem, where the objective function has a particular structure. Moreover, in this view, we only focus on the variables $\bx$.
The standard measure of convergence in this case, for a point $\bx'$, the \emph{suboptimality gap}\index{suboptimality gap}, can be written as
\[
h(\bx') - \inf_{\bx \in \mathcal{X}} h(\bx)
= \sup_{\by\in \mathcal{Y}} \ f(\bx',\by) - \inf_{\bx \in \mathcal{X}} \sup_{\by\in \mathcal{Y}} \ f(\bx,\by)~.
\]
We also have to find the maximizer with respect to $\by$ of the function $\inf_{\bx \in \mathcal{X}} \ f(\bx,\by)$, hence we have
\[
\sup_{\by \in \mathcal{Y}} \ g(\by),
\]
where $g(\by) = \inf_{\bx\in \mathcal{X}} \ f(\bx,\by)$. This also implies another measure of convergence in which we focus only on the variable $\by$:
\[
\sup_{\by \in \mathcal{Y}} \ g(\by)  - g(\by')
= \sup_{\by\in \mathcal{Y}} \inf_{\bx \in \mathcal{X}} \ f(\bx,\by) - \inf_{\bx\in \mathcal{X}} \ f(\bx,\by') ~.
\]
Finally, in case we are interested in studying the quality of a joint solution $(\bx',\by')$, a natural measure is the sum of the two measures above:
\begin{align*}
&\sup_{\by\in \mathcal{Y}} \ f(\bx',\by) - \inf_{\bx \in \mathcal{X}} \sup_{\by\in \mathcal{Y}} \ f(\bx,\by) + \sup_{\by\in \mathcal{Y}} \inf_{\bx \in \mathcal{X}} \ f(\bx,\by) - \inf_{\bx\in \mathcal{X}} \ f(\bx,\by')\\
&\quad= \sup_{\by\in \mathcal{Y}} \ f(\bx',\by) - \inf_{\bx \in \mathcal{X}} \ f(\bx, \by'),
\end{align*}
where we assumed the existence of a saddle point to say that $\inf_{\bx \in \mathcal{X}} \sup_{\by\in \mathcal{Y}} \ f(\bx,\by) = \sup_{\by\in \mathcal{Y}} \inf_{\bx \in \mathcal{X}} \ f(\bx,\by)$ from Theorem~\ref{thm:saddle_equiv_minmax}.
This measure is called \emph{duality gap}.
\begin{definition}
For a function $f:\mathcal{X} \times \mathcal{Y} \to \R$, define the \textbf{duality gap}\index{duality gap|textbf} on $(\bx', \by') \in \mathcal{X}\times \mathcal{Y}$ as
\[
\sup_{\by\in \mathcal{Y}} \ f(\bx',\by) - \inf_{\bx \in \mathcal{X}} \ f(\bx, \by')~.
\]
\end{definition}
The duality gap\index{duality gap} is always non-negative even when the saddle point does not exist, since $\sup_{\by\in \mathcal{Y}} \ f(\bx',\by)\geq f(\bx',\by')\geq \inf_{\bx \in \mathcal{X}} \ f(\bx, \by')$, for all $\bx' \in \mathcal{X}$ and $\by' \in \mathcal{Y}$.

Let's add even more intuition of the duality gap\index{duality gap} definition, using the notion of \emph{$\epsilon$-saddle point}.
\begin{definition}
Let $\epsilon\geq0$, $\mathcal{X} \subseteq \R^n$, $\mathcal{Y} \subseteq \R^m$, and $f: \mathcal{X} \times \mathcal{Y} \to \R$.
A point $(\bx^\star, \by^\star) \in \mathcal{X} \times \mathcal{Y}$ is an \textbf{$\epsilon$-saddle point}\index{saddle point!$\epsilon$-|textbf} of $f$ in $\mathcal{X} \times\mathcal{Y}$ if
\[
f(\bx^\star,\by) -\epsilon
\leq f(\bx^\star,\by^\star)
\leq f(\bx,\by^\star) + \epsilon, \quad \forall \bx \in \mathcal{X}, \by \in \mathcal{Y}~.
\]
\end{definition}
This definition is useful because we cannot expect to numerically calculate a saddle point with infinite precision, but we can find something that satisfies the saddle point definition up to an $\epsilon$.
Obviously, any saddle point is also an $\epsilon$-saddle point, yet we can have an $\epsilon$-saddle point even when the saddle point does not exist. Indeed, in Example~\ref{example:no_saddle} the point $(\delta,\delta)$, where $\delta \in (0,1]$, is a $\max\{\delta(1-\delta),\delta^2\}$-saddle point, even if the saddle-point does not exist.

Now, the notion of $\epsilon$-saddle point is equivalent up to a multiplicative constant to the $\epsilon$-duality gap, as detailed in the next lemma. The proof is left as an exercise (see Problem~\ref{exercise:duality_gap}).
\begin{lemma}
\label{lemma:equivalence_saddle_point}
If $(\bx^\star,\by^\star)$ is an $\epsilon$-saddle point then its duality gap\index{duality gap} is upper bounded by $2\epsilon$.
On the other hand, a duality gap at most $\epsilon$ implies that the point is a $\epsilon$-saddle point.
\end{lemma}

The above reasoning shows that finding the saddle point of the function $f$ is equivalent to solving both a maximization and a minimization problem. However, as we said above, the saddle point might not exist. So, let's now move to easily checkable sufficient conditions for the existence of a saddle point. For this, we can state the following theorem, which is a weaker version of Sion's minimax theorem~\citep{Sion58}.
\index{minimax theorem|(textbf}
\begin{theorem}
\label{thm:minimax}
Let $\mathcal{X}$, $\mathcal{Y}$ be compact convex subsets of $\R^n$ and $\R^m$ respectively. Let $f : \mathcal{X} \times \mathcal{Y} \to \R$ be a continuous function, convex in its first argument, and concave in its second, and with bounded subgradients with respect to both variables. Then, we have that
\[
\min_{\bx \in \mathcal{X}} \max_{\by\in \mathcal{Y}} \ f(\bx,\by)
=  \max_{\by\in \mathcal{Y}} \min_{\bx \in \mathcal{X}} \ f(\bx,\by)~.
\]
\end{theorem}
\index{minimax theorem|)textbf}
This theorem gives us sufficient conditions to have the min-max problem equal to the max-min one. So, for example, thanks to the Weierstrass theorem (Theorem~\ref{thm:weierstrass})\index{Weierstrass theorem for extended-real-valued functions}, the assumptions in Theorem~\ref{thm:minimax}, in light of Theorem~\ref{thm:saddle_equiv_minmax}, are sufficient conditions for the existence of a saddle point.

We defer the proof of this theorem for a bit, and we now turn to solving the saddle-point problem in \eqref{eq:saddle_point}.

\section{Solving Saddle-Point Problems with Online Convex Optimization}

\begin{algorithm}[h]
\caption{Solving Saddle-Point Problems with \acl{OCO}}
\label{alg:minmax_to_oco}
\begin{algorithmic}[1]
{
    \REQUIRE{$\bx_1 \in \mathcal{X}$, $\by_1 \in \mathcal{Y}$}
    \FOR{$t=1, \dots, T$}
    \STATE{$\mathcal{X}$-Learner and $\mathcal{Y}$-Learner simultaneously decide their outputs $\bx_t \in \mathcal{X}$ and $\by_t \in \mathcal{Y}$}
    \STATE{$\mathcal{X}$-Learner receives $\ell_t(\bx)=f(\bx,\by_t)$}
    \STATE{$\mathcal{Y}$-Learner receives $h_t(\by)=-f(\bx_t,\by)$}
    \ENDFOR
    \RETURN{$\bar{\bx}_T=\frac{1}{T}\sum_{t=1}^T \bx_t$, $\bar{\by}_T=\frac{1}{T}\sum_{t=1}^T \by_t$}
}
\end{algorithmic}
\end{algorithm}

Let's show how to use \ac{OCO} algorithms to solve saddle-point problems. We will state a procedure that is a direct generalization of the online-to-batch conversion we saw in Chapter~\ref{ch:o2b}.

Suppose we use an \ac{OCO} algorithm fed with losses $\ell_t(\bx)= f(\bx,\by_t)$ that produces the iterates $\bx_t$ and another \ac{OCO} algorithm fed with losses $h_t(\by)=-f(\bx_t,\by)$ that produces the iterates $\by_t$. Then, we can state the following theorem.
\begin{theorem}
\label{thm:minmax_to_oco}
Let $\mathcal{X}$ and $\mathcal{Y}$ be non-empty closed convex sets.
Let $f:\mathcal{X}\times \mathcal{Y}\to \R$. Then, with the notation in Algorithm~\ref{alg:minmax_to_oco}, for any $\bx \in \mathcal{X}$, we have
\[
\frac{1}{T}\sum_{t=1}^T f(\bx_t, \by_t) - \frac{1}{T}\sum_{t=1}^T f(\bx,\by_t)
= \frac{\Regret^{\mathcal{X}}_T(\bx)}{T},
\]
where $\Regret^{\mathcal{X}}_T(\bx)=\sum_{t=1}^T \ell_t(\bx_t) - \sum_{t=1}^T \ell_t(\bx)$.

Moreover, for any $\by \in \mathcal{Y}$, we have
\[
\frac{1}{T}\sum_{t=1}^T f(\bx_t,\by) - \frac{1}{T}\sum_{t=1}^T f(\bx_t, \by_t)
= \frac{\Regret^{\mathcal{Y}}_T(\by)}{T},
\]
where $\Regret^{\mathcal{Y}}_T(\by)=\sum_{t=1}^T h_t(\by_t) - \sum_{t=1}^{T} h_t(\by)$.

Also, if $f$ is convex in the first argument, concave in the second, and $\arg\max_{\by\in \mathcal{Y}}f(\bar{\bx}_T,\by)$ and $\arg\min_{\bx\in \mathcal{X}}f(\bx,\bar{\by}_T)$ are non-empty, then we have
\[
\max_{\by\in \mathcal{Y}} f(\bar{\bx}_T, \by) - \min_{\bx\in \mathcal{X}} f(\bx,\bar{\by}_T)
\leq \frac{\Regret^{\mathcal{X}}_T(\bx'_T)+\Regret^{\mathcal{Y}}_T(\by'_T)}{T},
\]
for any $\bx_T'\in\arg\min_{\bx\in \mathcal{X}} f(\bx,\bar{\by}_T)$ and $\by_T'\in\arg\max_{\by\in \mathcal{Y}} f(\bar{\bx}_T,\by)$.
\end{theorem}
\begin{proof}
The first two equalities are obtained by simply observing that $\ell_t(\bx)=f(\bx,\by_t)$ and $h_t(\by)=-f(\bx_t,\by)$.

For the stated inequality, using Jensen's inequality\index{inequality!Jensen's} (Theorem~\ref{thm:jensen}), we obtain
\begin{align*}
f(\bar{\bx}_T,\by)- f(\bx,\bar{\by}_T)
&\leq \frac{1}{T}\sum_{t=1}^{T}f(\bx_t,\by)- \frac{1}{T} \sum_{t=1}^{T}f(\bx,\by_t)~.
\end{align*}
Summing the first two equalities, using the above inequality, and taking $\bx=\bx'_T$ and $\by=\by'_T$, we get the stated inequality.
\end{proof}

From this theorem, we can immediately prove the following corollary.
\begin{corollary}
Let $\mathcal{X}$ and $\mathcal{Y}$ be non-empty closed convex sets.
Let $f:\mathcal{X} \times\mathcal{Y}\to \R$ be continuous, convex in the first argument, and concave in the second one. Assume that $\mathcal{X}$ and $\mathcal{Y}$ are compact.
Consider Algorithm~\ref{alg:minmax_to_oco} and assume that the two online algorithms guarantee that their maximum regret over competitors in their feasible set is sublinear in $T$. Then, we have
\[
\lim_{T\to\infty}\max_{\by\in \mathcal{Y}} f(\bar{\bx}_T, \by) - \min_{\bx\in \mathcal{X}} f(\bx,\bar{\by}_T)=0~.
\]
\end{corollary}

\begin{figure}[t]
\centering
\begin{tikzpicture}
\begin{axis}[
    name=left,
    width=6cm,
    xlabel={$x$}, ylabel={$y$},
    ylabel style={at={(axis description cs:-0.08,0.5)},anchor=south},
    title={Iterates and Running Averages},
    grid=major,
    xmin=-10.2, xmax=10.2,
    ymin=-10.2, ymax=10.2,
    axis equal,
    legend style={font=\tiny, at={(0,0.2)},anchor=north west},
    legend image post style={
        scale=0.5,
    },
    legend cell align={left},
]
\addplot[gray, thick] table[x=x,y=y] {code_for_figs/iterdata.dat};
\addlegendentry{Iterates $(x_t, y_t)$}
\addplot[thick] table[x=xbar,y=ybar] {code_for_figs/avgdata.dat};
\addlegendentry{Running average $(\bar{x}_t,\bar{y}_t)$}
\addplot[black, mark=asterisk, thick, mark size=3pt, only marks] coordinates {(5,5)};
\end{axis}
\end{tikzpicture}
\hspace{1cm}
\begin{tikzpicture}
\begin{axis}[
    at={(left.east)}, anchor=west,
    width=6cm,
    xlabel={$t$}, ylabel={Duality Gap},
    ylabel style={at={(axis description cs:-0.1,0.5)},anchor=south},
    title={Duality gap for $(\bar{x}_t, \bar{y}_t)$},
    grid=major,
    xmode=log, ymode=log,
    xmin=1, xmax=1000,
    legend style={font=\tiny},
    legend image post style={
        scale=0.5,
    },
    legend cell align={left},
]
\addplot[thick] table[x=t,y=vav] {code_for_figs/avgdata.dat};
\addplot[gray, thick] table[x=t,y=v] {code_for_figs/iterdata.dat};
\end{axis}
\end{tikzpicture}
\caption{Left: behaviour of the iterates and running averages of the iterates for Example~\ref{example:saddle_point_ogd}. The saddle point is marked with an asterisk. Right: the duality gap\index{duality gap} on a log-log plot for the same problem evaluated on the iterates and running averages of the iterates.}
\label{fig:saddle_point_ogd}
\commentAlt{Figure~\ref{fig:saddle_point_ogd}. Two-panel plot for a saddle-point OGD example. The left panel shows raw iterates and their running averages in the (x,y) plane, with the saddle point marked by an asterisk. The right panel shows the duality gap on log-log axes for iterates and running averages.}
\end{figure}

\begin{example}
\label{example:saddle_point_ogd}
Consider the saddle-point problem
\[
\min_{|x|\leq 10} \max_{|y|\leq 10} \ (x-5) (y-5)~.
\]
The saddle point of this problem is $(x,y)=(5,5)$. We can find it using, for example, projected \ac{OGD} with stepsizes $\eta_t=\frac{1}{\sqrt{t}}$. So, setting $\bx_1=\by_1=0$, we have the iterations
\begin{align*}
x_{t+1} &= \max\left(\min\left(x_t - \frac{1}{\sqrt{t}} (y_t-5),10\right),-10\right)\\
y_{t+1} &= \max\left(\min\left(y_t + \frac{1}{\sqrt{t}} (x_t-5),10\right),-10\right)~.
\end{align*}
According to Theorem~\ref{thm:minmax_to_oco}, the duality gap\index{duality gap} in $(\frac1T \sum_{t=1}^T x_t, \frac1T \sum_{t=1}^T y_t)$ converges to 0. Hence, given that the saddle point is unique, $(\frac1T \sum_{t=1}^T x_t, \frac1T \sum_{t=1}^T y_t)$ converges to the saddle point of this problem.
However, the iterates $(x_t,y_t)$ will never converge, see Figure~\ref{fig:saddle_point_ogd}.
\end{example}

Surprisingly, we can even prove the simpler version of the minimax theorem in Theorem~\ref{thm:minimax} from the above result!
\begin{proof}[Proof of Theorem~\ref{thm:minimax}]
From Lemma~\ref{lemma:min_max_less_max_min}, we have one inequality. Hence, we now have to prove the other inequality.

We will use a constructive proof.
From the assumptions of the theorem, we know that we can use, for example, two projected \ac{OSD} algorithms on the losses $\ell_t$ and $h_t$.
Let's use Algorithm~\ref{alg:minmax_to_oco} and Theorem~\ref{thm:minmax_to_oco}. For the first player, for any $\bx \in \mathcal{X}$ we have
\[
\frac{1}{T}\sum_{t=1}^T f(\bx_t, \by_t)
= \frac{1}{T}\sum_{t=1}^T f(\bx,\by_t) + \frac{\Regret^{\mathcal{X}}_T(\bx)}{T}
\leq f\left(\bx,\frac{1}{T}\sum_{t=1}^T \by_t\right) + \frac{\Regret^{\mathcal{X}}_T(\bx)}{T}~.
\]
Observe that
\[
\min_{\bx \in \mathcal{X}} f\left(\bx,\frac{1}{T}\sum_{t=1}^T \by_t\right)
\leq \max_{\by \in \mathcal{Y}} \min_{\bx \in \mathcal{X}} f(\bx,\by)~.
\]
Given that projected \ac{OSD} has $o(T)$ regret for each $\bx \in \mathcal{X}$, we have
\[
\frac{1}{T}\sum_{t=1}^T f(\bx_t, \by_t)
\leq \max_{\by \in \mathcal{Y}} \min_{\bx \in \mathcal{X}} f(\bx,\by)+o(1)~.
\]
In the same way, we have
\[
-\frac{1}{T}\sum_{t=1}^T f(\bx_t, \by_t)
\leq -\min_{\bx \in \mathcal{X}} \max_{\by \in \mathcal{Y}} f(\bx,\by)+o(1)~.
\]
Summing the two inequalities, taking $T\to \infty$, and using the sublinear regret assumption, we have
\[
\min_{\bx \in \mathcal{X}} \max_{\by \in \mathcal{Y}} f(\bx,\by)
\leq \max_{\by \in \mathcal{Y}} \min_{\bx \in \mathcal{X}} f(\bx,\by)~.
\]
Together with Lemma~\ref{lemma:min_max_less_max_min}, this implies the stated equality.
\end{proof}

\subsection{Variations with Best Response and Alternation}

In some cases, it is easy to compute the max with respect to $\by \in \mathcal{Y}$ of $f(\bx_t,\by)$ for a given $\bx_t$. For example, this is trivial for bilinear games over the probability simplex\index{probability simplex} that we will see in Section~\ref{sec:game_theory}. In these cases, we can remove the second learner and just use its \textbf{best response}\index{best response|(textbf} in each round, that is, the vector that minimizes the current loss on the feasible set. Note that in this way we are making one of the two players ``stronger'' through the knowledge of their loss in the next round. However, this is perfectly fine: the proof in Theorem~\ref{thm:minmax_to_oco} is still perfectly valid.

\begin{algorithm}[t]
\caption{Saddle-Point Optimization with OCO and $\mathcal{Y}$-Best Response}
\label{alg:minmax_to_oco2}
\begin{algorithmic}[1]
{
    \REQUIRE{$\bx_1 \in \mathcal{X}$}
    \FOR{$t=1, \dots, T$}
    \STATE{Set $\by_{t} \in \argmax_{\by \in \mathcal{Y}} \ f(\bx_{t},\by)$}
    \STATE{$\mathcal{X}$-Learner receives $\ell_t(\bx) = f(\bx,\by_t)$ and produces $\bx_{t+1} \in \mathcal{X}$}
    \ENDFOR
    \RETURN{$\bar{\bx}_T=\frac{1}{T}\sum_{t=1}^T \bx_t$, $\bar{\by}_T=\frac{1}{T}\sum_{t=1}^T \by_t$}
}
\end{algorithmic}
\end{algorithm}

In this case, the $\mathcal{Y}$-player has an easy life: it knows the loss before making the prediction, hence it can just output the minimizer of the loss in $\mathcal{Y}$. Hence, we also have that the regret of the $\mathcal{Y}$-player will be non-positive, and it will not show up in Theorem~\ref{thm:minmax_to_oco}. Putting everything together, we can state the following corollary.
\begin{corollary}
Let $\mathcal{X}$ and $\mathcal{Y}$ be non-empty closed convex sets.
Let $f:\mathcal{X} \times\mathcal{Y}\to \R$ be convex in the first argument, and concave in the second.  With the notation in Algorithm~\ref{alg:minmax_to_oco2}, assume that $\mathcal{X}$ is compact, the argmax of the $\mathcal{Y}$-player is never empty, and $\arg\min_{\bx\in \mathcal{X}}f(\bx,\bar{\by}_T)$ is non-empty. Then, we have
\[
\sup_{\by\in \mathcal{Y}} f(\bar{\bx}_T, \by) - \min_{\bx\in \mathcal{X}} f(\bx,\bar{\by}_T)
\leq \frac{\Regret^{\mathcal{X}}_T(\bx'_T)}{T},
\]
for any $\bx_T'\in\arg\min_{\bx\in \mathcal{X}}f(\bx,\bar{\by}_T)$ and where $\Regret^{\mathcal{X}}_T(\bx)=\sum_{t=1}^T \ell_t(\bx_t) - \sum_{t=1}^{T} \ell_t(\bx)$.
\end{corollary}
This alternative seems interesting from a theoretical point of view because it allows us to avoid the complexity of learning in the $\mathcal{Y}$ space, for example, removing the dependence on its dimension.

\begin{remark}
\label{remark:best-response}
The best-response algorithm requires stronger access to the function $f$ because we need to solve an optimization problem exactly. Hence, for example, we cannot just have access to subgradients of $f$ to use it.
\end{remark}

Of course, an analogous result can be stated using best response for the $\mathcal{X}$-player and an \ac{OCO} algorithm for the $\mathcal{Y}$-player, as in Algorithm~\ref{alg:minmax_to_oco3}.

\begin{algorithm}[h]
\caption{Saddle-Point Optimization with OCO and $\mathcal{X}$-Best Response}
\label{alg:minmax_to_oco3}
\begin{algorithmic}[1]
{
    \REQUIRE{$\by_1 \in \mathcal{Y}$}
    \FOR{$t=1, \dots, T$}
    \STATE{Set $\bx_{t} \in \argmin_{\bx \in \mathcal{X}} \ f(\bx,\by_t)$}
    \STATE{$\mathcal{Y}$-Learner receives $h_t(\by) = - f(\bx_t,\by)$ and produces $\by_{t+1} \in \mathcal{Y}$}
    \ENDFOR
    \RETURN{$\bar{\bx}_T=\frac{1}{T}\sum_{t=1}^T \bx_t$, $\bar{\by}_T=\frac{1}{T}\sum_{t=1}^T \by_t$}
}
\end{algorithmic}
\end{algorithm}
\index{best response|)textbf}

There is a third variant, very common in empirical implementations, especially of Counterfactual Regret Minimization\index{Counterfactual Regret Minimization algorithm} (CFR)~\citep{ZinkevichJBP07}. It is called \textbf{alternation}\index{alternation|(textbf} and it breaks the simultaneous reveal of the actions of the two players. Instead, we use the updated prediction of the first player to construct the loss of the second player.
Empirically, this variant seems to greatly speed up the convergence of the duality gap\index{duality gap}.

\begin{algorithm}[h]
\caption{Saddle-Point Optimization with OCO and Alternation}
\label{alg:minmax_to_oco_alternation}
\begin{algorithmic}[1]
{
    \REQUIRE{$\by_1 \in \mathcal{Y}$}
    \FOR{$t=1, \dots, T$}
    \STATE{$\mathcal{X}$-Learner receives $\ell_t(\bx)=f(\bx,\by_t)$ and produces $\bx_{t+1} \in \mathcal{X}$}
    \STATE{$\mathcal{Y}$-Learner receives $h_t(\by)=-f(\bx_{t+1},\by)$ and produces $\by_{t+1} \in \mathcal{Y}$}
    \ENDFOR
    \RETURN{$\bar{\bx}_T=\frac{1}{T}\sum_{t=1}^T \bx_{t+1}$, $\bar{\by}_T=\frac{1}{T}\sum_{t=1}^T \by_t$}
}
\end{algorithmic}
\end{algorithm}

For this version, Theorem~\ref{thm:minmax_to_oco} does not hold anymore because the terms $f(\bx_t,\by_t)$ and $f(\bx_{t+1},\by_t)$ are now different. However, we can prove a similar guarantee.
\begin{theorem}
\label{thm:minmax_to_oco_alternation}
Let $\mathcal{X}$ and $\mathcal{Y}$ be non-empty closed convex sets.
Let $f:\mathcal{X} \times\mathcal{Y}\to \R$ be convex in the first argument and concave in the second. With the notation in Algorithm~\ref{alg:minmax_to_oco_alternation}, assume that $\arg\min_{\bx\in \mathcal{X}}f(\bx,\bar{\by}_T)$ and $\arg\max_{\by\in \mathcal{Y}}f(\bar{\bx}_T,\by)$ are non-empty.
Then, for any $\bx_T'\in\arg\min_{\bx\in \mathcal{X}}f(\bx,\bar{\by}_T)$ and any $\by_T'\in\arg\max_{\by\in \mathcal{Y}}f(\bar{\bx}_T,\by)$, we have
\begin{align*}
&\max_{\by\in \mathcal{Y}} f(\bar{\bx}_T, \by) - \min_{\bx\in \mathcal{X}} f(\bx,\bar{\by}_T)\\
&\quad \leq \frac{\Regret^{\mathcal{X}}_T(\bx'_T)+\Regret^{\mathcal{Y}}_T(\by'_T)+\sum_{t=1}^{T} (f(\bx_{t+1},\by_t)-f(\bx_{t},\by_t))}{T},
\end{align*}
where $\Regret^{\mathcal{Y}}_T(\by)=\sum_{t=1}^T (h_t(\by_t) - h_t(\by))$, $\Regret^{\mathcal{X}}_T(\bx)=\sum_{t=1}^T (\ell_t(\bx_t) - \ell_t(\bx))$, where $\bx_1$ is the prediction of the $\mathcal{X}$-learner before seeing $\ell_1$.
\end{theorem}
\begin{proof}
Note that $\ell_t(\bx)=f(\bx,\by_t)$, and $h_t(\by)=-f(\bx_{t+1},\by)$. By Jensen's inequality\index{inequality!Jensen's} (Theorem~\ref{thm:jensen}), we have
\begin{align*}
T&(f(\bar{\bx}_T, \by) - f(\bx,\bar{\by}_T))\\
&\leq \sum_{t=1}^{T}f(\bx_{t+1},\by)- \sum_{t=1}^{T}f(\bx,\by_t)\\
&=\sum_{t=1}^{T}f(\bx_{t+1},\by)-\sum_{t=1}^{T}f(\bx_{t+1},\by_t)
+\sum_{t=1}^{T}f(\bx_{t},\by_t)-\sum_{t=1}^{T}f(\bx,\by_t)\\
&\quad +\sum_{t=1}^{T}f(\bx_{t+1},\by_t)-\sum_{t=1}^{T}f(\bx_{t},\by_t)\\
&= \sum_{t=1}^{T} (h_t(\by_t)-h_t(\by))+\sum_{t=1}^{T}(\ell_t(\bx_t)-\ell_t(\bx)) + \sum_{t=1}^{T} (f(\bx_{t+1},\by_t)-f(\bx_{t},\by_t))~.
\end{align*}
Taking $\bx=\bx'_T\in\arg\min_{\bx\in \mathcal{X}} f(\bx,\bar{\by}_T)$ and $\by=\by'_T\in\arg\max_{\by\in \mathcal{Y}} f(\bar{\bx}_T,\by)$, we get the stated result.
\end{proof}

\begin{remark}
In some cases, it is possible to show that the additional term in Theorem~\ref{thm:minmax_to_oco_alternation} is negative, showing a (marginal) improvement to the convergence rate. For example, using \ac{OMD} for the $\mathcal{X}$-player and assuming that the losses it receives are linear, we have that $f(\bx_t,\by_t) = \ell_t(\bx_t) \geq \ell_t(\bx_{t+1})=f(\bx_{t+1},\by_t)$.
\end{remark}
\index{alternation|)textbf}

Next, we will see how to connect saddle-point problems with game theory.

\section{Game-Theory Interpretation of Saddle-Point Problems}
\label{sec:game_theory}
\index{game!two-player zero-sum|(}
An instantiation of a saddle-point problem also has an interpretation in game theory as a \textbf{two-player zero-sum game}. Note that game theory is a vast field, and two-person zero-sum games are only a very small subset of the problems in this domain, and what I describe here is an even smaller subset of this subset of problems.

Game theory studies what happens when self-interested agents interact. By self-interested, we mean that each agent has an ideal state of things they want to reach, which can include a description of what should happen to other agents as well, and works towards this goal.
In two-person games, the players act simultaneously, and then they receive their losses. In particular, the $\mathcal{X}$-player chooses the play $\bx$ and the $\mathcal{Y}$-player chooses the play $\by$, the $\mathcal{X}$-player suffers the loss $f(\bx,\by)$ and the $\mathcal{Y}$-player the loss $-f(\bx,\by)$. It is important to understand that this is only one round, that is, there is only one play for each player.

\begin{remark}
The standard game-theoretic terminology uses payoffs instead of losses, but we will keep using losses for coherence with the \ac{OCO} notation used in this book.
\end{remark}

We consider the so-called \textbf{two-person normal-form games}\index{normal-form game}, that is, when the first player has $n$ possible actions and the second player $m$. A player can use a \textbf{pure strategy}\index{pure strategy|textbf}, that is, a single fixed action, or randomize over a set of actions according to some probability distribution, a so-called \textbf{mixed strategy}\index{mixed strategy|textbf}. In this case, we consider $\mathcal{X} =\Delta^{n-1}$ and $\mathcal{Y} =\Delta^{m-1}$ and they are known as the \textbf{mixed-strategy spaces}\index{mixed-strategy space|textbf} for the two players.
In this setting, for a pair of pure strategies $(\be_i,\be_j)$, the first player receives the loss $f(\be_i,\be_j)$ and the second player $-f(\be_i,\be_j)$, where $\be_i$ is the vector with all zeros but a `1' in position $i$.
The goal of each player is to minimize the received loss. Given the discrete nature of this game, the function $f(\bx,\by)$ is the bilinear function $\bx^\top \bM \by$, where $\bM$ is a matrix with $n$ rows and $m$ columns. Hence, for a pair of mixed strategies $(\bx,\by)$, the \emph{expected} loss of the first player is $\bx^\top \bM \by$ and the one of the second player is $-\bx^\top \bM \by$.

A fundamental concept in game theory is the notion of \textbf{Nash equilibrium}\index{Nash equilibrium|(textbf}. We have a Nash equilibrium if all players are playing their best strategy against the other players' strategies. That is, none of the players has an incentive to change their strategy if the other player does not change it.
For the zero-sum two-person game, this can be formalized by saying that $(\bx^\star,\by^\star)$ is a Nash equilibrium if
\[
(\bx^\star)^\top \bM \by
\leq (\bx^\star)^\top \bM \by^\star
\leq \bx^\top \bM \by^\star, \quad \forall \bx \in \Delta^{n-1}, \by \in \Delta^{m-1}~.
\]
This is \emph{exactly} the definition of saddle point for the function $f(\bx,\by)=\bx^\top \bM \by$ that we gave in the previous section. Given that $f(\bx,\by)=\bx^\top \bM \by$ is continuous, convex in the first argument and concave in the second one, the sets $\mathcal{X} =\Delta^{n-1}$ and $\mathcal{Y} =\Delta^{m-1}$ are convex and compact, we can deduce from Theorem~\ref{thm:minimax} and Theorem~\ref{thm:saddle_equiv_minmax} that a saddle point always exists. Hence, there is always at least one (possibly mixed) Nash equilibrium in two-person zero-sum games. The common value of the minimax and maxmin problem is called \textbf{value of the game}\index{value of the game|textbf} and we will denote it by $v^\star$.

For a zero-sum two-person game, the Nash equilibrium has an immediate interpretation: from the definition above, if the first player uses the strategy $\bx^\star$ then their loss is at most the value of the game $v^\star$, regardless of the strategy of the second player. Analogously, if the second player uses the strategy $\by^\star$ then their loss is at most $-v^\star$, regardless of the strategy of the first player. Both players achieve the value of the game if they both play the Nash strategy. Moreover, even if one of the players announced their strategy in advance to the other player, they would not increase their loss in expectation.

\begin{example}[Cake cutting]
Suppose there is a game between two people: the first player cuts the cake in two, and the second one chooses a piece; the first player receives the piece that was not chosen.
We can formalize it with the following matrix:
\[
\begin{tabular}{c|cc}
 & larger piece & smaller piece\\
\hline
cut evenly & 0 & 0 \\
cut unevenly & 10 & -10
\end{tabular}
\]
When the first player plays action $i$ and the second player action $j$, the first player receives the loss $M_{i,j}$ and the second player receives $-M_{i,j}$.
The losses represent how much less in percentage compared to half of the cake that the first player is receiving. The second player receives the negative of the same number. It should be easy to convince oneself that the strategy pair $(\text{cut evenly}, \text{larger piece})$ is an equilibrium with value of the game of 0.
However, perhaps counterintuitively, this is not the only Nash equilibrium strategy. In fact, for $0.5\leq a\leq 1$, one can verify numerically that any pair $\bx^\star=[1,0]^\top$, $\by^\star=[a,1-a]$ is optimal, where the first player uses a pure strategy and the second player a mixed one.
\end{example}

\begin{example}[Rock-Paper-Scissors]
Let's consider the game of Rock-Paper-Scissors. We describe it with the following matrix:
\[
\begin{tabular}{c|ccc}
 & Rock & Paper & Scissors\\
\hline
Rock & 0 & 1 & -1\\
Paper &-1 & 0 & 1\\
Scissors & 1 & -1 & 0
\end{tabular}
\]
It should be immediate to realize that there are no pure Nash equilibria for this game. However, there is a mixed Nash equilibrium when each player randomizes the action with a uniform probability distribution over the three actions, and the value of the game equals 0.
\end{example}



\begin{figure}[t]
\centering
\begin{tikzpicture}
\begin{axis}[
  width=7cm,
  view={155}{10},
  xlabel={Player 1 $\Pr(\text{head})$},
  ylabel={Player 2 $\Pr(\text{head})$},
  zlabel={Player 1 expected loss},
  xtick={0, 0.5, 1},
  ytick={0, 0.5, 1},
  ztick={-1, -0.5, 0, 0.5, 1},
  grid,
  xmin=0, xmax=1,
  ymin=0, ymax=1,
  zmin=-1, zmax=1,
  colormap={bw}{gray(0cm)=(0.9); gray(1cm)=(0.1)},
  tick style={draw=none}
]
\addplot3[
  surf,
  mesh/rows=21,
] table[
  x=x, y=y, z=z,
  col sep=comma
] {code_for_figs/penny_matching_saddle.csv};
\end{axis}
\end{tikzpicture}
\caption{The saddle point in the Matching Pennies game.}
\label{fig:matching_pennies}
\commentAlt{Figure~\ref{fig:matching_pennies}. Three-dimensional surface for the Matching Pennies game. The horizontal axes are the two players' probabilities of    playing heads, and the vertical axis is Player 1's expected loss, with the saddle point at the mixed equilibrium.}
\end{figure}

\begin{example}[Matching Pennies]
\label{ex:matching_pennies}
In this game, both players show a face of a penny. If the two faces are the same, the first player wins both, otherwise the second player wins both. The associated matrix $\bM$ is
\[
\begin{tabular}{c|cc}
 & head & tail\\
\hline
head & -1 & 1 \\
tail & 1 & -1
\end{tabular}
\]
It is easy to see that the Nash equilibrium is when both players randomize the face to show with equal probability.

In this simple case, we can visualize the saddle point associated with this problem in Figure~\ref{fig:matching_pennies}.
\end{example}

Unless the game is very small, we find Nash equilibria using numerical procedures that typically give us only approximate solutions. Hence, as for $\epsilon$-saddle points, for $\epsilon\geq0$ we also define an \textbf{$\epsilon$-Nash equilibrium}\index{Nash equilibrium!$\epsilon$-} for a zero-sum two-person game when $\bx^\star$ and $\by^\star$ satisfy
\[
(\bx^\star)^\top \bM \by -\epsilon
\leq (\bx^\star)^\top \bM \by^\star
\leq \bx^\top \bM \by^\star + \epsilon, \quad \forall \bx \in \Delta^{n-1}, \by \in \Delta^{m-1}~.
\]
Obviously, any Nash equilibrium is also an $\epsilon$-Nash equilibrium.

From what we said in the previous section, it should be immediate to see how to numerically calculate the Nash equilibrium of a two-person zero-sum game.
In fact, we know that we can use online convex optimization algorithms to find $\epsilon$-saddle points, so we can do the same for an $\epsilon$-Nash equilibrium of two-person zero-sum games. Assuming that the average regret of both players is $\epsilon_T$, Theorem~\ref{thm:minmax_to_oco} says that $\left(\frac{1}{T}\sum_{t=1}^T \bx_t, \frac{1}{T}\sum_{t=1}^T \by_t\right)$ is a $2\epsilon_T$-Nash equilibrium.
\index{Nash equilibrium|)textbf}
\index{game!two-player zero-sum|)}

\subsection{Reducing Extensive to Normal-Form Zero-Sum Games}
\label{sec:extensive_to_normal}

The connection between saddle-point optimization and two-person zero-sum games extends beyond the static case. Many strategic interactions are \emph{sequential}: players move in turns, possibly with knowledge of previous actions. These games are naturally represented as \textbf{extensive-form games}\index{game!extensive form|(}, whose solution concepts can be related to the saddle-point structure of their normal forms.


A \textbf{perfect-information extensive-form zero-sum game} is described by a game tree in which nodes correspond to decision points, edges correspond to actions, and leaves specify the losses. Each non-terminal node is assigned to one of the two players, and all players observe all previous moves. The game is \emph{zero-sum} if the sum of the players' losses at every terminal node is zero.

For example, consider the simple two-level perfect-information game in Figure~\ref{fig:extensive_tree}. Player~1 moves first by choosing between $A$ and $B$. Then, Player~2 observes the choice and selects $L$ or $R$ in one case and $P$ and $Q$ in another case. Each terminal leaf specifies the loss for Player~1 (Player~2's loss is its negation).

\begin{figure}[h]
\centering
\begin{tikzpicture}[
    level distance=1.5cm,
    level 1/.style={sibling distance=3.5cm},
    level 2/.style={sibling distance=1.5cm},
    every node/.style={draw=none, fill=none},
    edge from parent/.style={draw,-latex,thick}
]
\node (root) {}
    child { node (p2a) {}
        child { node[rectangle, draw, fill=gray!10, minimum size=6mm] {$\ell_{A,L}$} edge from parent node[left] {$L$} }
        child { node[rectangle, draw, fill=gray!10, minimum size=6mm] {$\ell_{A,R}$} edge from parent node[right] {$R$} }
        edge from parent node[left] {$A$} }
    child { node (p2b) {}
        child { node[rectangle, draw, fill=gray!10, minimum size=6mm] {$\ell_{B,P}$} edge from parent node[left] {$P$} }
        child { node[rectangle, draw, fill=gray!10, minimum size=6mm] {$\ell_{B,Q}$} edge from parent node[right] {$Q$} }
        edge from parent node[right] {$B$} };

\node at ($(root)+(0,0.01)$) {\small Player 1};
\node at ($(p2a)+(0,0.01)$) {\small Player 2};
\node at ($(p2b)+(0,0.01)$) {\small Player 2};
\end{tikzpicture}
\caption{Example of a perfect-information extensive-form zero-sum game.}
\label{fig:extensive_tree}
\commentAlt{Figure~\ref{fig:extensive_tree}. Game tree for a perfect-information zero-sum game. Player 1 chooses A or B; after observing this, Player 2 chooses between two actions, leading to terminal losses l_{A,L}, l_{A,R}, l_{B,P}, and l_{B,Q}.}
\end{figure}

We can always convert such a game into the normal form by listing all the players' \emph{pure strategies}---that is, complete contingent plans specifying what a player would do at every node where it could move.

In the example above, Player~1 has two possible pure strategies: $\{A,B\}$.
Player~2, however, must decide what to do after each of Player~1's actions, hence has $2 \times 2 = 4$ pure strategies:
\[
\mathcal{Y} = \{(L,P), (L,Q), (R,P), (R,Q)\},
\]
where $(L,P)$ means ``play $L$ after $A$ and $P$ after $B$," etc.

The extensive game can then be rewritten as a normal-form loss matrix\index{normal-form game}
\[
\begin{array}{c|cccc}
& (L,P) & (L,Q) & (R,P) & (R,Q)\\
\hline
A & \ell_{A,L} & \ell_{A,L} & \ell_{A,R} & \ell_{A,R} \\
B & \ell_{B,P} & \ell_{B,Q} & \ell_{B,P} & \ell_{B,Q}
\end{array}
\]
Each entry corresponds to the leaf reached by the combination of both players' strategy profiles. Once this conversion is performed, standard saddle-point and minimax results (like Theorem~\ref{thm:minimax}) apply directly.

It should also be clear from the above example that the conversion from extensive form to the normal form can result in an exponential blowup of the representation of the game.

A foundational result on perfect-information extensive-form games is due to \citet{Zermelo13}.
\begin{theorem}[Zermelo's Theorem]
\label{thm:zermelo}
\index{Zermelo's theorem}
Every finite perfect-information game in extensive form has a pure-strategy Nash equilibrium.
\end{theorem}



When players move without observing all past actions, the game has \textbf{imperfect information}. In this case, decision nodes are grouped into \textbf{information sets}, and a player must take the same action at all nodes within an information set. This also means that we will have the same set of actions in all nodes within an information set.
A classic example is poker, where a player knows their own cards but not the opponent's.

Even imperfect-information extensive-form games can still be expressed as a normal-form game\index{normal-form game}: each player's pure strategies again correspond to complete contingent plans, including choices at all information sets.

\begin{figure}[h]
\centering
\begin{tikzpicture}[
    level distance=1.5cm,
    level 1/.style={sibling distance=3.5cm},
    level 2/.style={sibling distance=1.5cm},
    every node/.style={draw=none, fill=none},
    edge from parent/.style={draw,-latex,thick}
]
\node (root) {}
    child { node (p2a) {}
        child { node[rectangle, draw, fill=gray!10, minimum size=6mm] {$\ell_{L,A}$} edge from parent node[left] {$A$} }
        child { node[rectangle, draw, fill=gray!10, minimum size=6mm] {$\ell_{L,B}$} edge from parent node[right] {$B$} }
        edge from parent node[left] {$L$} }
    child { node (p2b) {}
        child { node[rectangle, draw, fill=gray!10, minimum size=6mm] {$\ell_{R,A}$} edge from parent node[left] {$A$} }
        child { node[rectangle, draw, fill=gray!10, minimum size=6mm] {$\ell_{R,B}$} edge from parent node[right] {$B$} }
        edge from parent node[right] {$R$} };

\node at ($(root)+(0,0.00)$) {\small Player 1};
\node (labA) at ($(p2a)+(0,0.0)$) {\small Player 2};
\node (labB) at ($(p2b)+(0,0.0)$) {\small Player 2};

\draw[dashed, thick]
    ($(labA.east)+(0.05,0)$) -- ($(labB.west)+(-0.05,0)$)
    node[midway, above] {\scriptsize same information set};

\end{tikzpicture}
\caption{An imperfect-information game: Player 2 does not observe Player 1's action ($L$ or $R$).}
\label{fig:imperfect_tree}
\commentAlt{Figure~\ref{fig:imperfect_tree}. Game tree for an imperfect-information game. Player 1 chooses L or R; Player 2 then chooses A or B, but the two Player 2 decision nodes are connected by a dashed information set, indicating that Player 2 does not observe Player 1's action.}
\end{figure}

In Figure~\ref{fig:imperfect_tree}, Player~2's two decision nodes are connected by a dashed line, indicating that they form the same information set. Player~2 must choose the same move ($A$ or $B$) regardless of whether the game reached the left or right subtree.

The corresponding normal-form\index{normal-form game} representation again lists all pure strategies. However, now these plans must respect the information constraints. The resulting loss matrix is
\[
\begin{array}{c|cc}
& A & B\\
\hline
L & \ell_{L,A} & \ell_{L,B}\\
R & \ell_{R,A} & \ell_{R,B}
\end{array}
\]
where Player~2's choice ($A$ or $B$) must be the same across both branches.

\index{game!extensive form|)}

\section{Application: Boosting as a Two-Person Game}

\index{boosting|(textbf}
We will now show that the \emph{boosting} problem can also be seen as the solution of a zero-sum two-person game.

Let $\mathcal{S}=\{\bz_j,y_j\}_{j=1}^m$ be a training set of $m$ feature/label pairs, where $\bz_j \in \R^d$ and $y_j \in \{-1,1\}$.
Let $\mathcal{H} = \{h_i\}_{i=1}^n$ be a set of $n$ functions $h_i:\R^d \to \{-1,1\}$.
The aim of boosting is to find a combination of the functions in $\mathcal{H}$ that has arbitrarily low misclassification error on $\mathcal{S}$. Of course, this is not always possible. However, we will make a (strong) assumption: we assume that it is always possible to find a function $h_i \in \mathcal{H}$ such that its misclassification error over $\mathcal{S}$ weighted by any probability distribution is better than chance by a constant $\gamma>0$.
We now show that this assumption guarantees that the boosting problem is solvable!

First, we construct a matrix of the misclassifications for each function: $\bM \in \R^{n \times m}$ where
\[
M_{i,j}=\begin{cases} 1, & \text{if } h_i(\bz_j)\neq y_j\\ 0, & \text{otherwise.}\end{cases}
\]
Setting $\mathcal{X} =\Delta^{n-1}$ and $\mathcal{Y} =\Delta^{m-1}$, we write the saddle-point problem/two-person zero-sum game
\[
\min_{\bp \in \Delta^{n-1}} \max_{\bq \in \Delta^{m-1}} \ \bp^\top \bM \bq~.
\]
Given the definition of the matrix $\bM$, this is equivalent to
\[
\min_{\bp \in \Delta^{n-1}} \max_{\bq \in \Delta^{m-1}} \ \sum_{i=1}^n \sum_{j=1}^m p_i q_j \indevent{\{h_i(\bz_j)\neq y_j\}}~.
\]

Let's now formalize the assumption on the functions: we assume the existence of a \textbf{weak learning oracle}\index{weak learning oracle|textbf} that, for any $\bq \in \Delta^{m-1}$, returns $i^\star$ such that
\begin{equation}
\label{eq:weak_oracle_assumption}
\be_{i^\star}^\top \bM \bq
= \sum_{j=1}^m q_j \indevent{\{h_{i^\star}(\bz_j)\neq y_j\}}
\leq \frac{1}{2} - \gamma,
\end{equation}
where $\gamma>0$. In words, $i^\star$ is the index of the function in $\mathcal{H}$ that gives a $\bq$-weighted error better than chance.
Moreover, given that
\[
\min_{\bp \in \Delta^{n-1}} \bp^\top \bM \bq
\leq \be_{i^\star}^\top \bM \bq
\leq \frac{1}{2} - \gamma
\]
and
\[
v^\star
= \max_{\bq \in \Delta^{m-1}} \min_{\bp \in \Delta^{n-1}} \bp^\top \bM \bq,
\]
we have that the value of the game satisfies $v^\star \leq \frac12 - \gamma < \frac12$.
Using the inequality on the value of the game and the fact that the Nash equilibrium exists, we obtain that there exists $\bp^\star \in \Delta^{n-1}$ such that
\begin{equation}
\label{eq:boosting_eq1}
\sum_{i=1}^n p^\star_i \indevent{\{h_i(\bz_j) \neq y_j\}}
= (\bp^\star)^\top \bM \be_j
\leq v^\star
\leq \frac12 - \gamma
< \frac12, \quad \forall j=1,\dots,m~.
\end{equation}
In words, this means that every sample $\bz_j$ is misclassified by less than half of the functions $h_i$ when weighted by $\bp^\star$. Hence, \emph{we can correctly classify all the samples using a weighted majority vote rule where the weights over the function $h_i$ are $\bp^\star$.}
This means that we can learn a perfect classifier rule using weak learners, through the solution of a minimax game. So, our job is to find a way to calculate this optimal distribution $\bp^\star$ on the functions.

Given what we have said so far, a natural strategy is to use \ac{OCO} algorithms.
In particular, we can use Algorithm~\ref{alg:minmax_to_oco3}, where in each round the $\mathcal{X}$-player is the weak learning oracle, that knows the play $\bq_t$ by the $\mathcal{Y}$-Learner and produces a best response\index{best response} or an approximate best response that returns a weighted classification error less than or equal to $\frac12 -\gamma$. In particular, from the assumption on the weak-learnability oracle in \eqref{eq:weak_oracle_assumption}, the $\mathcal{X}$-learner can always find a classifier with small enough $\bq_t$-weighted training error. Instead, the $\mathcal{Y}$-player is an \ac{OCO} algorithm.
Specialized to our setting, we have Algorithm~\ref{alg:boosting_to_oco}. In words, the $\mathcal{X}$-player looks for the function that has a small enough weighted misclassification loss, where the weights are constructed by the $\mathcal{Y}$-player.

\begin{algorithm}[t]
\caption{Boosting through Online Convex Optimization}
\label{alg:boosting_to_oco}
\begin{algorithmic}[1]
{
    \STATE{$q_{1,j}=1/m, j=1,\dots,m$}
    \FOR{$t=1, \dots, T$}
    \STATE{$\mathcal{X}$-Learner has two possible strategies:\\
    \hspace{1em} (i) (Approximate best response) Set $i_t$ such that $\be_{i_t}^\top \bM \bq_t = \sum_{j=1}^m q_{t,j} \indevent{\{h_{i_t}(\bz_j)\neq y_j\}}\leq 1/2-\gamma$\\
    \hspace{1em} (ii) (Best response\index{best response}) Set $i_t=\argmin_{i =1, \dots, n} \ \be_{i}^\top \bM \bq_t = \argmin_{i =1, \dots, n} \ \sum_{j=1}^m q_{t,j} \indevent{\{h_{i}(\bz_j)\neq y_j\}}$}
    \STATE{$\mathcal{Y}$-Learner receives $\ell_t(\bq) =-\be_{i_t}^\top \bM \bq = - \sum_{j=1}^m q_{j} \indevent{\{h_{i_t}(\bz_j)\neq y_j\}}$ and produces $\bq_{t+1} \in \Delta^{m-1}$}
    \ENDFOR
    \RETURN{$\bar{h}_T(\bz)=\frac{1}{T}\sum_{t=1}^T h_{i_t}(\bz)$}
}
\end{algorithmic}
\end{algorithm}

Let's show a guarantee on the misclassification error of this algorithm.
From the second equality of Theorem~\ref{thm:minmax_to_oco}, for any $\bq \in \Delta^{m-1}$, we have
\[
\frac{1}{T} \sum_{t=1}^T \be_{i_t}^\top \bM \bq
= \frac{1}{T} \sum_{t=1}^T \be_{i_t}^\top \bM \bq_t + \frac{\Regret^{\mathcal{Y}}_T(\bq)}{T}~.
\]
From both strategies of the $\mathcal{X}$-player, we have $\frac{1}{T} \sum_{t=1}^T \be_{i_t}^\top \bM \bq_t\leq \frac12 - \gamma$.
Moreover, choosing $\bq=\be_j$ we have $\frac{1}{T} \sum_{t=1}^T \be_{i_t}^\top \bM \bq = \frac1T \sum_{t=1}^T \indevent{\{h_{i_t}(\bz_j) \neq y_j\}}$. Putting everything together, we have
\[
\frac1T \sum_{t=1}^T \indevent{\{h_{i_t}(\bz_j) \neq y_j\}}
\leq \frac12 - \gamma + \frac{\Regret^{\mathcal{Y}}_T(\be_j)}{T}, \quad \forall j=1,\dots,m~.
\]
If $\frac{\Regret^{\mathcal{Y}}_T(\be_j)}{T}< \gamma$, less than half of the functions selected by the boosting procedure will make a mistake on $(\bz_j,y_j)$. Using the majority rule of the selected functions as a predictor, we will have 0 mistakes on the training samples.

Observe that in this scheme, we approximate $p^\star_i$ with the frequency with which $i_t$ is equal to $i$.

\begin{remark}
The use of a best response\index{best response} or approximate best response for the $\mathcal{X}$-player avoids the presence of its regret in the guarantee. In turn, this makes the number of rounds to reach training error 0 independent of the number of functions in $\mathcal{H}$.
\end{remark}

Let's now instantiate this framework with a specific \ac{OCO} algorithm.
For example, using \ac{EG} as algorithm for the $\mathcal{Y}$-player, we have that $\Regret^{\mathcal{Y}}_T(\be_j)= \mathcal{O}(\sqrt{T \ln m})$ as $T\to \infty$ for any $j=1,\dots,m$, that implies that after $T = \mathcal{O}(\frac{\ln m}{\gamma^2})$ rounds the error on the training samples of the classifier that uses the sign of $\bar{h}_T$ is exactly 0. This is exactly the same guarantee achieved by AdaBoost~\citep{FreundS95,FreundS97}.

\noindent\textbf{Boosting and Margins.}
What happens if we keep boosting after the training error reaches 0? It turns out we maximize the \textbf{margin}\index{margin!in boosting|textbf}, defined as $y_j \frac{1}{T}\sum_{t=1}^T h_{i_t}(\bz_j)$. In fact, given that $\indevent{\{h_{i_t}(\bz_j)\neq y_j\}} = \frac{1-y_j h_{i_t}(\bz_j)}{2}$, we have for any $j =1, \dots, m$
\[
2 \gamma - \frac{2 \Regret^{\mathcal{Y}}_T(\be_j)}{T}
\leq y_j \frac{1}{T} \sum_{t=1}^T h_{i_t}(\bz_j)
= y_j \bar{h}_T(\bz_j)~.
\]
Hence, when the number of rounds goes to infinity, the minimum margin on the training samples reaches at least $2 \gamma$.
The property of boosting of guaranteeing a minimum margin has been used as an explanation of the fact that in boosting additional rounds after the training error reaches 0 often keep improving the test error on test samples coming i.i.d. from the same distribution that generated the training samples.

The above reduction does not tell us how the training error precisely behaves. However, we can get this information by changing the \ac{LEA} algorithm. Indeed, we have seen in Chapter~\ref{ch:parameterfree} that there are \ac{LEA} algorithms provably better than \ac{EG}. We can use the parameter-free \ac{LEA} algorithm in Section~\ref{sec:shifted_kt} that guarantees a regret $\mathcal{O}(\sqrt{T \cdot \KL(\bq;\bpi)})$ as $T\to\infty$ for a prior $\bpi$, where we set $\bpi$ equal to the uniform prior. These kinds of algorithms allow us to upper bound the fraction of mistakes after any $T$ iterations. Denote by $\mathcal{I}_T$ the set of samples misclassified by $\sign(\bar{h}_T)$ after $T$ iterations of boosting, let $k=|\mathcal{I}_T|$, and, if $k>0$, set $\bq_k$ as the vector whose coordinate $j$ is $1/k$ if $j\in \mathcal{I}_T$ and $0$ otherwise.
Hence, we have
\begin{equation}
\label{eq:better_bound_boosting}
2 \gamma - \frac{2 \Regret^{\mathcal{Y}}_T(\bq_k)}{T}
\leq \frac{1}{k}\sum_{j\in \mathcal{I}_T} y_j \bar{h}_T(\bz_j)
\leq 0~.
\end{equation}
Using the expression of the regret, we have that the fraction of misclassified samples $\frac{k}{m}$ is bounded by $\mathcal{O}(\exp(-\gamma^2 T))$ as $T\to\infty$. That is, the fraction of misclassification error goes to zero exponentially fast in the number of boosting rounds. Again, a similar guarantee was proved for AdaBoost, but here we quickly derived it using a reduction from boosting to \ac{LEA}, passing through zero-sum two-person games.
\index{boosting|)textbf}

\section{Faster Rates Through Optimism}
\label{sec:saddle_point_optimism}

We saw that it is possible to solve convex/concave saddle-point optimization problems using two \ac{OCO} algorithms playing against each other. We obtained a rate of convergence for the duality gap\index{duality gap} of $\mathcal{O}(1/\sqrt{T})$.
Now, we show that if the function is smooth, we can achieve a faster rate using \emph{optimistic} algorithms.

Assume that $f:\mathcal{X} \times\mathcal{Y} \to \R$ is smooth\index{function!smooth} in an open set containing its domain, in the sense that for any $\bx,\bx'\in \mathcal{X}$ and $\by,\by' \in \mathcal{Y}$, we have
\begin{align}
\|\nabla_{\bx} f(\bx,\by)-\nabla_{\bx} f(\bx',\by)\|_{\mathcal{X},\star} &\leq L_{\mathcal{X}\mathcal{X}} \|\bx-\bx'\|_{\mathcal{X}} \label{eq:minmax_smooth_1},\\
\|\nabla_{\bx} f(\bx,\by)-\nabla_{\bx} f(\bx,\by')\|_{\mathcal{X},\star} &\leq L_{\mathcal{X}\mathcal{Y}} \|\by-\by'\|_{\mathcal{Y}} \label{eq:minmax_smooth_2},\\
\|\nabla_{\by} f(\bx,\by)-\nabla_{\by} f(\bx',\by)\|_{\mathcal{Y},\star} &\leq L_{\mathcal{X}\mathcal{Y}} \|\bx-\bx'\|_{\mathcal{X}} \label{eq:minmax_smooth_3},\\
\|\nabla_{\by} f(\bx,\by)-\nabla_{\by} f(\bx,\by')\|_{\mathcal{Y},\star} &\leq L_{\mathcal{Y}\mathcal{Y}} \|\by-\by'\|_{\mathcal{Y}}, \label{eq:minmax_smooth_4}
\end{align}
where $\nabla_{\bx}$ and $\nabla_{\by}$ denote the gradients with respect to the first and second variable respectively, and we have denoted by $\|\cdot\|_{\mathcal{X}}$ and $\|\cdot\|_\mathcal{Y}$ the norms on the ambient spaces of $\mathcal{X}$ and $\mathcal{Y}$ respectively, while the norms with the $\star$ are their duals.
\begin{remark}
At this point, one might be tempted to consider the maximum between the three quantities ``to simplify the math'', but the units are different, and it does not make any sense to take the maximum between quantities with different units.
\end{remark}

Let's use two online algorithms again to solve the saddle-point problem.
However, instead of using two standard no-regret algorithms, we will use two optimistic ones, as the ones we saw in Sections~\ref{sec:optimistic_omd} and~\ref{sec:optimistic_ftrl}.
We will use the same strategy and proof of the algorithm in Section~\ref{sec:oftrl_gradual_variations}, that is, we will use the previous observed gradient as a prediction for the next one.

For example, use two Optimistic \ac{FTRL} algorithms with fixed strongly convex\index{function!strongly convex} regularizers and hint at time $t$ constructed using the previous observed gradient: $\tilde{\ell}_t(\bx)=\langle \bg_{t-1}, \bx\rangle$ where we set $\bg_0=0$, see Algorithm~\ref{alg:minmax_to_optimistic_ftrl}. We now show that these hints allow us to cancel out terms when we consider the sum of the regrets and obtain a faster rate of $\mathcal{O}(1/T)$ rather than just $\mathcal{O}(1/\sqrt{T})$.

\begin{algorithm}[t]
\caption{Solving Saddle-Point Problems with Optimistic \ac{FTRL}}
\label{alg:minmax_to_optimistic_ftrl}
\begin{algorithmic}[1]
{
    \REQUIRE{$\mathcal{X}$ and $\mathcal{Y}$ non-empty closed convex sets, regularizers  $\psi_\mathcal{X}$ and $\psi_\mathcal{Y}$, $\lambda_{\mathcal{X}}> 0$, $\lambda_{\mathcal{Y}}> 0$}
    \STATE{$\bg_{\mathcal{X},0}=\boldsymbol{0}, \bg_{\mathcal{Y},0}=\boldsymbol{0}$}
    \FOR{$t=1, \dots, T$}
    \STATE{$\bx_t = \argmin_{\bx \in \mathcal{X}} \  \psi_\mathcal{X}(\bx) + \langle\bg_{\mathcal{X},t-1},\bx\rangle+\sum_{i=1}^{t-1} \langle\bg_{\mathcal{X},i}, \bx\rangle$}
    \STATE{$\by_t = \argmin_{\by \in \mathcal{Y}} \  \psi_\mathcal{Y}(\by) + \langle\bg_{\mathcal{Y},t-1},\by\rangle+\sum_{i=1}^{t-1} \langle \bg_{\mathcal{Y},i}, \by\rangle$}
    \STATE{Set $\bg_{\mathcal{X},t} = \nabla_{\bx} f(\bx_t,\by_t)$}
    \STATE{Set $\bg_{\mathcal{Y},t} = -\nabla_{\by} f(\bx_t,\by_t)$}
    \ENDFOR
    \RETURN{$\bar{\bx}_T=\frac{1}{T}\sum_{t=1}^T \bx_t$, $\bar{\by}_T=\frac{1}{T}\sum_{t=1}^T \by_t$}
}
\end{algorithmic}
\end{algorithm}

From the regret of Optimistic \ac{FTRL} in Theorem~\ref{thm:ftrl_optimistic}, for the $\mathcal{X}$-player, for any $\bu \in \mathcal{X}$, we have
\begin{align*}
\sum_{t=1}^T \langle \bg_{\mathcal{X},t}, \bx_t-\bu\rangle
&\leq \psi_\mathcal{X}(\bu) - \psi_\mathcal{X}(\bx_1) \\
&\quad + \sum_{t=1}^T \left(\langle \bg_{\mathcal{X},t} - \bg_{\mathcal{X},t-1}, \bx_t - \bx_{t+1}\rangle -\frac{\lambda_{\mathcal{X}}}{2} \|\bx_t-\bx_{t+1}\|^2_\mathcal{X}\right)~.
\end{align*}
From the Fenchel--Young's inequality\index{inequality!Fenchel--Young's}, we have
\begin{align*}
&\langle \bg_{\mathcal{X},t} - \bg_{\mathcal{X},t-1}, \bx_t - \bx_{t+1}\rangle
- \frac{\lambda_{\mathcal{X}}}{2} \|\bx_t-\bx_{t+1}\|^2_\mathcal{X}\\
&\quad\leq \frac{\lambda_{\mathcal{X}}}{4}\|\bx_t - \bx_{t+1} \|^2_\mathcal{X}+ \frac{1}{\lambda_{\mathcal{X}}} \|\bg_{\mathcal{X},t} - \bg_{\mathcal{X},t-1}\|^2_{\mathcal{X},\star} - \frac{\lambda_{\mathcal{X}}}{2} \|\bx_t-\bx_{t+1}\|^2_\mathcal{X}\\
&\quad =\frac{1}{\lambda_{\mathcal{X}}}\|\bg_{\mathcal{X},t} - \bg_{\mathcal{X},t-1}\|^2_{\mathcal{X},\star} - \frac{\lambda_{\mathcal{X}}}{4} \|\bx_t-\bx_{t+1}\|^2_\mathcal{X}~.
\end{align*}
Note that there are multiple choices of the coefficient in the Fenchel--Young's inequality\index{inequality!Fenchel--Young's}, but without additional information, all choices are equally good.

Now, using the smoothness assumption, for $t\geq 2$ we have
\begin{align*}
&\|\bg_{\mathcal{X},t} - \bg_{\mathcal{X},t-1}\|^2_{\mathcal{X},\star}\\
&= \|\nabla_{\bx} f(\bx_t,\by_t) - \nabla_{\bx} f(\bx_{t-1}, \by_{t-1})\|^2_{\mathcal{X},\star} \\
&\leq \left(\|\nabla_{\bx} f(\bx_t,\by_t) - \nabla_{\bx} f(\bx_{t-1}, \by_{t})\|_{\mathcal{X},\star} + \|\nabla_{\bx} f(\bx_{t-1},\by_t) - \nabla_{\bx} f(\bx_{t-1}, \by_{t-1})\|_{\mathcal{X},\star}\right)^2 \\
&\leq 2L_{\mathcal{X}\mathcal{X}}^2 \|\bx_{t-1}-\bx_{t}\|^2_\mathcal{X} + 2L^2_{\mathcal{X}\mathcal{Y}} \|\by_{t-1} - \by_{t}\|^2_\mathcal{Y}~.
\end{align*}
We can proceed in the exact same way for the $\mathcal{Y}$-player too.

Summing the regret of the two algorithms, we have
\begin{align*}
\sum_{t=1}^T (f(\bx_t,\by) -  f(\bx,\by_t))
&\leq \psi_\mathcal{X}(\bx) - \psi_\mathcal{X}(\bx_1) + \psi_\mathcal{Y}(\by) - \psi_\mathcal{Y}(\by_1)\\
&\quad+ \frac{\|\bg_{\mathcal{X},1}\|^2_{\mathcal{X},\star}}{\lambda_{\mathcal{X}}} + \frac{\|\bg_{\mathcal{Y},1}\|^2_{\mathcal{Y},\star}}{\lambda_{\mathcal{Y}}}\\
&\quad + \sum_{t=2}^T \left(\frac{2L_{\mathcal{X}\mathcal{X}}^2}{\lambda_{\mathcal{X}}}+\frac{2L^2_{\mathcal{X}\mathcal{Y}}}{\lambda_{\mathcal{Y}}}-\frac{\lambda_{\mathcal{X}}}{4}\right) \|\bx_t-\bx_{t-1}\|^2_\mathcal{X}\\
&\quad + \sum_{t=2}^T \left(\frac{2L^2_{\mathcal{Y}\mathcal{Y}}}{\lambda_{\mathcal{Y}}}+\frac{2L^2_{\mathcal{X}\mathcal{Y}}}{\lambda_{\mathcal{X}}}-\frac{\lambda_{\mathcal{Y}}}{4}\right) \|\by_t-\by_{t-1}\|^2_\mathcal{Y}~.
\end{align*}
Choosing $\lambda_\mathcal{X} \geq 2 \sqrt{2} (L_{\mathcal{X}\mathcal{X}}+L_{\mathcal{X}\mathcal{Y}} \alpha)$ and $\lambda_{\mathcal{Y}}\geq 2 \sqrt{2}(L_{\mathcal{Y}\mathcal{Y}}+L_{\mathcal{X}\mathcal{Y}}/\alpha)$ for any $\alpha>0$ kills all the terms in the sum.
In fact, we have
\[
\frac{2L_{\mathcal{X}\mathcal{X}}^2}{\lambda_{\mathcal{X}}}+\frac{2L^2_{\mathcal{X}\mathcal{Y}}}{\lambda_{\mathcal{Y}}}
\leq \frac{2L_{\mathcal{X}\mathcal{X}}^2}{2\sqrt{2} L_{\mathcal{X}\mathcal{X}}}+\frac{2L^2_{\mathcal{X}\mathcal{Y}} \alpha}{2 \sqrt{2} L_{\mathcal{X}\mathcal{Y}}}
\leq \frac{\lambda_{\mathcal{X}}}{4},
\]
and similarly for the other term.
One might wonder why we need to introduce $\alpha$ and if it can be just set to 1. However, $\alpha$ has units, and it allows the sum of the smoothness coefficients with different units, so it is better to keep it around to remember it.

Assuming that the regularizers are bounded over $\mathcal{X}$ and $\mathcal{Y}$ and using the usual online-to-batch conversion, we have that the duality gap\index{duality gap} evaluated at the pair $\left(\frac{1}{T}\sum_{t=1}^T \bx_t, \frac{1}{T}\sum_{t=1}^T \by_t\right)$ goes to zero as $\mathcal{O}(1/T)$ when $T\to\infty$.

Overall, we can state the following theorem.
\begin{theorem}
\label{thm:minmax_to_optimistic_ftrl}
Let $\mathcal{X}$ and $\mathcal{Y}$ be non-empty closed convex sets.
With the notation in Algorithm~\ref{alg:minmax_to_optimistic_ftrl}, let $f:\mathcal{X} \times\mathcal{Y}\to \R$ be convex in the first argument and concave in the second, satisfying assumptions~\eqref{eq:minmax_smooth_1}-\eqref{eq:minmax_smooth_4}. For a fixed $\alpha>0$, let $\lambda_\mathcal{X} \geq 2 \sqrt{2} (L_{\mathcal{X}\mathcal{X}}+L_{\mathcal{X}\mathcal{Y}} \alpha)$ and $\lambda_{\mathcal{Y}}\geq 2 \sqrt{2}(L_{\mathcal{Y}\mathcal{Y}}+L_{\mathcal{X}\mathcal{Y}}/\alpha)$. Let $\psi_\mathcal{X}:\mathcal{X}\to\R$ be $\lambda_{\mathcal{X}}$-strongly convex with respect to $\|\cdot\|_\mathcal{X}$ and $\psi_\mathcal{Y}:\mathcal{Y}\to\R$ be $\lambda_{\mathcal{Y}}$-strongly convex with respect to $\|\cdot\|_\mathcal{Y}$.
Assume that $\arg\max_{\by\in \mathcal{Y}}f(\bar{\bx}_T,\by)$ and $\arg\min_{\bx\in \mathcal{X}}f(\bx,\bar{\by}_T)$ are non-empty.
Then, we have
\begin{align*}
&\max_{\by\in \mathcal{Y}} f(\bar{\bx}_T, \by) - \min_{\bx\in \mathcal{X}} f(\bx,\bar{\by}_T)\\
&\quad\leq \frac{\psi_\mathcal{X}(\bx'_T)-\psi_\mathcal{X}(\bx_1)+\psi_\mathcal{Y}(\by'_T)-\psi_\mathcal{Y}(\by_1)+\frac{\|\bg_{\mathcal{X},1}\|^2_{\mathcal{X},\star}}{\lambda_{\mathcal{X}}}+\frac{\|\bg_{\mathcal{Y},1}\|^2_{\mathcal{Y},\star}}{\lambda_{\mathcal{Y}}}}{T},
\end{align*}
for any $\bx_T'\in\arg\min_{\bx\in \mathcal{X}}f(\bx,\bar{\by}_T)$ and $\by_T'\in\arg\max_{\by\in \mathcal{Y}}f(\bar{\bx}_T,\by)$.
\end{theorem}

Looking back at the proof of the algorithm, we have a faster convergence because the regret of one player depends on the ``stability'' of the other player, measured by the terms $\|\bx_t- \bx_{t-1}\|_\mathcal{X}^2$ and $\|\by_t- \by_{t-1}\|_\mathcal{Y}^2$. Hence, we have a sort of stabilization loop in which the stability of one algorithm makes the other more stable, which in turn stabilizes the first one even more. Indeed, we can also show that the regret of the two algorithms is not growing over time. Note that such a result cannot be obtained just by looking at the fact that the sum of the regret does not grow over time.

In fact, setting for example $\lambda_\mathcal{X} \geq 4 \sqrt{2} (L_{\mathcal{X}\mathcal{X}}+L_{\mathcal{X}\mathcal{Y}} \alpha)$ and $\lambda_{\mathcal{Y}}\geq 4 \sqrt{2}(L_{\mathcal{Y}\mathcal{Y}}+L_{\mathcal{X}\mathcal{Y}}/\alpha)$, we have that \[
\frac{2L^2_{\mathcal{Y}\mathcal{Y}}}{\lambda_{\mathcal{Y}}}+\frac{2L^2_{\mathcal{X}\mathcal{Y}}}{\lambda_{\mathcal{X}}}-\frac{\lambda_{\mathcal{Y}}}{4}\leq -\frac{\lambda_{\mathcal{Y}}}{8}
\]
and
\[
\frac{2L_{\mathcal{X}\mathcal{X}}^2}{\lambda_{\mathcal{X}}}+\frac{2L^2_{\mathcal{X}\mathcal{Y}}}{\lambda_{\mathcal{Y}}}-\frac{\lambda_{\mathcal{X}}}{4} \leq -\frac{\lambda_{\mathcal{X}}}{8}~.
\]
Hence, using the fact that the existence of a saddle point $(\bx^\star, \by^\star)$ guarantees that $f(\bx_t,\by^\star) - f(\bx^\star,\by_t)\geq 0$, we have
\begin{align}
&\sum_{t=2}^T\left(\frac{\lambda_{\mathcal{X}}}{8} \|\bx_t-\bx_{t-1}\|^2_\mathcal{X} + \frac{\lambda_{\mathcal{Y}}}{8} \|\by_t-\by_{t-1}\|^2_\mathcal{Y}\right) \nonumber\\
&\quad\leq \psi_\mathcal{X}(\bx^\star)-\psi_\mathcal{X}(\bx_1)+\psi_\mathcal{Y}(\by^\star)-\psi_\mathcal{Y}(\by_1) +\frac{\|\bg_{\mathcal{X},1}\|^2_{\mathcal{X},\star}}{\lambda_{\mathcal{X}}}+\frac{\|\bg_{\mathcal{Y},1}\|^2_{\mathcal{Y},\star}}{\lambda_{\mathcal{Y}}}~. \label{eq:oftrl_minmax_bounded_sum}
\end{align}
Plugging this guarantee back into the regret of each algorithm, we have that their regret is bounded and independent of $T$.
From \eqref{eq:oftrl_minmax_bounded_sum}, we also have that $\|\bx_t-\bx_{t-1}\|^2_\mathcal{X}$ and $\|\by_t-\by_{t-1}\|^2_\mathcal{Y}$ converge to 0. Hence, the algorithms are getting more and more stable over time, even if they use fixed regularizers.

\begin{remark}
Note that the fact that $\lim_{t\to \infty} \|\bx_t-\bx_{t-1}\|_\mathcal{X} =0$ is not enough to infer that $\bx_t$ is converging. In fact, it is enough to consider the case of $x_t = \ln t$ on an unbounded domain.
\end{remark}

\noindent\textbf{Version with Optimistic \ac{OMD}.}
The exact same reasoning holds for Optimistic \ac{OMD}, because the key terms of its regret bound are exactly the same as those of Optimistic \ac{FTRL}.
To better show this fact, we instantiate the Optimistic \ac{OMD}\index{Online Mirror Descent algorithm!optimistic} with stepsizes equal to $\frac{1}{\lambda_{\mathcal{X}}}$ and $\frac{1}{\lambda_{\mathcal{Y}}}$ for $\mathcal{X}$-player and $\mathcal{Y}$-player respectively.
Following the same reasoning as above and the regret bound of Optimistic \ac{OMD}, we obtain the following theorem.

\begin{theorem}
\label{thm:minmax_to_optimistic_omd}
Let $\mathcal{X}$ and $\mathcal{Y}$ be non-empty closed convex sets.
With the notation in Algorithm~\ref{alg:minmax_to_optimistic_omd}, let $f:\mathcal{X} \times\mathcal{Y}\to \R$ be convex in the first argument and concave in the second, satisfying assumptions~\eqref{eq:minmax_smooth_1}-\eqref{eq:minmax_smooth_4}. For a fixed $\alpha>0$, let $\lambda_\mathcal{X} \geq 2 \sqrt{2} (L_{\mathcal{X}\mathcal{X}}+L_{\mathcal{X}\mathcal{Y}} \alpha)$ and $\lambda_{\mathcal{Y}}\geq 2 \sqrt{2}(L_{\mathcal{Y}\mathcal{Y}}+L_{\mathcal{X}\mathcal{Y}}/\alpha)$. Let $\psi_\mathcal{X}:\mathcal{X}\to\R$ be $1$-strongly convex with respect to $\|\cdot\|_\mathcal{X}$ and $\psi_\mathcal{Y}:\mathcal{Y}\to\R$ be $1$-strongly convex with respect to $\|\cdot\|_\mathcal{Y}$. Also, assume that $\psi_\mathcal{X}$ and $\psi_\mathcal{Y}$ are differentiable in the interior of their domains.
Assume that $\arg\max_{\by\in \mathcal{Y}}f(\bar{\bx}_T,\by)$ and $\arg\min_{\bx\in \mathcal{X}}f(\bx,\bar{\by}_T)$ are non-empty.
Then, we have
\[
\max_{\by\in \mathcal{Y}} f(\bar{\bx}_T, \by) - \min_{\bx\in \mathcal{X}} f(\bx,\bar{\by}_T)
\leq \frac{B_{\psi_\mathcal{X}}(\bx'_T;\bx_1)+B_{\psi_\mathcal{Y}}(\by'_T;\by_1)+\frac{\|\bg_{\mathcal{X},1}\|^2_{\mathcal{X},\star}}{\lambda_{\mathcal{X}}}+\frac{\|\bg_{\mathcal{Y},1}\|^2_{\mathcal{Y},\star}}{\lambda_{\mathcal{Y}}}}{T},
\]
for any $\bx_T'\in\arg\min_{\bx\in \mathcal{X}}f(\bx,\bar{\by}_T)$ and $\by_T'\in\arg\max_{\by\in \mathcal{Y}}f(\bar{\bx}_T,\by)$.
\end{theorem}

\begin{algorithm}[t]
\caption{Solving Saddle-Point Problems with Optimistic \ac{OMD}}
\label{alg:minmax_to_optimistic_omd}
\begin{algorithmic}[1]
{
    \REQUIRE{$\lambda_{\mathcal{X}}> 0, \lambda_{\mathcal{Y}}> 0, \bx_1 \in \mathcal{X}, \by_1 \in \mathcal{Y}$}
    \STATE{$\bg_{\mathcal{X},0}=\boldsymbol{0}, \bg_{\mathcal{Y},0}=\boldsymbol{0}$}
    \FOR{$t=1, \dots, T$}
    \STATE{Set $\bg_{\mathcal{X},t} = \nabla_{\bx} f(\bx_t,\by_t)$}
    \STATE{Set $\bg_{\mathcal{Y},t} = -\nabla_{\by} f(\bx_t,\by_t)$}
    \STATE{$\bx_{t+1} = \argmin_{\bx \in \mathcal{X}} \  \langle 2\bg_{\mathcal{X},t}-\bg_{\mathcal{X},t-1},\bx\rangle+\lambda_\mathcal{X} B_{\psi_\mathcal{X}}(\bx;\bx_t)$}
    \STATE{$\by_{t+1} = \argmin_{\by \in \mathcal{Y}} \  \langle 2\bg_{\mathcal{Y},t}-\bg_{\mathcal{Y},t-1},\by\rangle+\lambda_{\mathcal{Y}} B_{\psi_\mathcal{Y}}(\by;\by_t)$}
    \ENDFOR
    \RETURN{$\bar{\bx}_T=\frac{1}{T}\sum_{t=1}^T \bx_t$, $\bar{\by}_T=\frac{1}{T}\sum_{t=1}^T \by_t$}
}
\end{algorithmic}
\end{algorithm}

\begin{example}
Consider the bilinear saddle-point problem
\[
\min_{\bx \in \mathcal{X}} \max_{\by \in \mathcal{Y}} \ \bx^\top \bA \by~.
\]
In this case, we have that $\nabla_{\bx} f(\bx,\by)=\bA \by$, $\nabla_{\by} f(\bx,\by)= \bA^\top \bx$, $L_{\mathcal{X}\mathcal{X}}=0$, $L_{\mathcal{Y}\mathcal{Y}}=0$, and $L_{\mathcal{X}\mathcal{Y}}=\|\bA\|_\text{op}$ where $\|\cdot\|_\text{op}$ is the operator norm of the matrix $\bA$. The specific shape of the operator norm depends on the norms we use on $\mathcal{X}$ and $\mathcal{Y}$. For example, if we choose the Euclidean norm on both $\mathcal{X}$ and $\mathcal{Y}$, the operator norm of $\bA$ is the largest singular value of $\bA$.
On the other hand, if $\mathcal{X} =\Delta^{n-1}$ and $\mathcal{Y}=\Delta^{m-1}$, as in the two-person zero-sum games, and we use the L$_1$ norm on both, then the operator norm of a matrix $\bA$ is the maximum absolute value of the entries of $\bA$.
\end{example}

\subsection{Prescient Online Mirror Descent and Be-The-Regularized-Leader}
\label{sec:prescient_omd}

The results in the previous section are interesting from a game-theoretic point of view, because they show that two players can converge to an equilibrium without any ``communication''. However, if we only care about converging to the saddle point without worrying about communication, we can easily do better.
For example, we can use the fact that it is fine if one of the two players ``cheats'' by looking at the loss at the beginning of each round, making its regret non-positive, while the other uses optimistic updates. However, as discussed in Remark~\ref{remark:best-response}, this will change the oracle, because now one of the algorithms needs to access the functions rather than only their gradients.

For example, we saw the use of best response\index{best response}. However, best response only guarantees non-positive regret, while for the optimistic proof above we need some specific negative terms. It turns out we can achieve them with another algorithm: \textbf{Prescient Online Mirror Descent}\index{Online Mirror Descent algorithm!prescient|(textbf}, that predicts in each round with $\bx_t \in \argmin_{\bx \in \mathcal{V}} \ \ell_t(\bx) + \frac{1}{\eta_t}B_\psi(\bx;\bx_{t-1})$, see Algorithm~\ref{alg:rbr}.

\begin{algorithm}[t]
\caption{Prescient Online Mirror Descent}
\label{alg:rbr}
\begin{algorithmic}[1]
{
    \REQUIRE{Non-empty closed convex $\mathcal{V} \subset \mathcal{X}\subseteq \R^d$, $\psi: \mathcal{X} \to \R$ strictly convex and differentiable on $\interior \mathcal{X}$, $\bx_0 \in \interior \mathcal{X}$, $\eta_1,\dots,\eta_T>0$}
    \FOR{$t=1$ {\bfseries to} $T$}
    \STATE{Receive $\ell_t: \R^d \to (-\infty, +\infty]$, where $\mathcal{V} \subseteq \dom \ell_t$}
    \STATE{$\bx_{t} \in \argmin_{\bx \in \mathcal{V}} \ \ell_t(\bx) + \frac{1}{\eta_t}B_\psi(\bx; \bx_{t-1})$}
    \STATE{Pay the loss $\ell_t(\bx_t)$}
    \ENDFOR
}
\end{algorithmic}
\end{algorithm}

\begin{theorem}
\label{thm:pomd}
Let $\psi: \mathcal{X} \to \R$ differentiable in $\interior \mathcal{X}$, closed, and strictly convex. Let $\mathcal{V} \subseteq \mathcal{X}$ be a non-empty closed convex set. Assume $\bx_t \in \interior \mathcal{X}$, $\ell_t$ subdifferentiable on $\mathcal{V}$, and $\eta_{t+1}\leq \eta_{t}$, for $t=1, \dots, T$. Then, $\forall \bu \in \mathcal{V}$, Algorithm~\ref{alg:rbr} satisfies
\[
\sum_{t=1}^T \ell_t(\bx_{t}) - \sum_{t=1}^T  \ell_t(\bu)
\leq \max_{0\leq t \leq T-1} \frac{B_\psi(\bu;\bx_t)}{\eta_{T}} - \sum_{t=1}^T \frac{1}{\eta_t} B_\psi(\bx_{t}; \bx_{t-1} ) ~.
\]
Moreover, if $\eta_t$ is constant, i.e., $\eta_t=\eta \ \forall t=1,\dots,T$, we have
\[
\sum_{t=1}^T (\ell_t(\bx_t)- \ell_t(\bu))
\leq \frac{B_\psi(\bu;\bx_0)}{\eta} - \frac{1}{\eta}\sum_{t=1}^T B_\psi(\bx_{t}; \bx_{t-1} )~.
\]
\end{theorem}
\begin{proof}
From the first-order optimality condition of Theorem~\ref{thm:constr_opt_condition} on the update, we have that there exists $\bg_t \in \partial \ell_t(\bx_t)$ such that
\[
\langle \eta_t \bg_t + \nabla \psi(\bx_{t}) - \nabla \psi(\bx_{t-1}), \bu - \bx_{t}  \rangle \geq 0, \quad \forall \bu \in \mathcal{V}~.
\]
Hence, we have
\begin{align*}
\eta_t(\ell_t(\bx_{t}) - \ell_t(\bu))
& \leq \langle \eta_t \bg_t, \bx_{t} - \bu \rangle
\leq  \langle \nabla \psi(\bx_{t-1}) - \nabla \psi(\bx_{t}), \bx_{t} - \bu \rangle \\
& = B_\psi (\bu; \bx_{t-1} ) - B_\psi(\bu; \bx_{t}) - B_\psi(\bx_{t}; \bx_{t-1} ),
\end{align*}
where in the last equality we used Lemma~\ref{lemma:bregman_3_points}.
Dividing by $\eta_t$ and summing over $t=1, \dots, T$, we have
\begin{align*}
\sum_{t=1}^T &(\ell_t(\bx_t) - \ell_t(\bu))
\leq \sum_{t=1}^T \left(\frac{B_\psi(\bu;\bx_{t-1})}{\eta_t} - \frac{B_\psi(\bu;\bx_{t})}{\eta_t}\right) - \sum_{t=1}^T \frac{B_\psi(\bx_{t}; \bx_{t-1})}{\eta_t} \\
&= \frac{B_\psi(\bu;\bx_{0})}{\eta_1} - \frac{B_\psi(\bu;\bx_{T})}{\eta_T}  + \sum_{t=1}^{T-1} \left(\frac{1}{\eta_{t+1}}-\frac{1}{\eta_t}\right)B_\psi(\bu;\bx_{t}) - \sum_{t=1}^T \frac{B_\psi(\bx_{t}; \bx_{t-1} )}{\eta_t} \\
&\leq \frac{D^2}{\eta_1} + D^2 \sum_{t=1}^{T-1} \left(\frac{1}{\eta_{t+1}}-\frac{1}{\eta_{t}}\right) - \sum_{t=1}^T \frac{B_\psi(\bx_{t}; \bx_{t-1} )}{\eta_t} \\
&= \frac{D^2}{\eta_1}  + D^2 \left(\frac{1}{\eta_{T}}-\frac{1}{\eta_1}\right) - \sum_{t=1}^T \frac{B_\psi(\bx_{t}; \bx_{t-1} )}{\eta_t} \\
&= \frac{D^2}{\eta_{T}} - \sum_{t=1}^T \frac{B_\psi(\bx_{t}; \bx_{t-1})}{\eta_t},
\end{align*}
where we denoted by $D^2=\max_{0\leq t\leq T-1} B_\psi(\bu;\bx_t)$.

The second statement is left as an exercise.
\end{proof}

The regret of Prescient Online Mirror Descent contains the negative terms we needed from the optimistic algorithms.\index{Online Mirror Descent algorithm!prescient|)textbf}

Analogously, we can obtain a version of \ac{FTRL} that uses the knowledge of the current loss: \textbf{Be-The-Regularized-Leader (BTRL)}\index{Be-the-Regularized-Leader algorithm|(textbf}, that predicts in each time step with $\bx_t \in \argmin_{\bx \in \mathcal{V}} \ \psi_t(\bx)+\sum_{i=1}^t \ell_i(\bx)$, see Algorithm~\ref{alg:btrl}.

\begin{algorithm}[t]
\caption{Be-the-Regularized-Leader Algorithm}
\label{alg:btrl}
\begin{algorithmic}[1]
{
    \REQUIRE{Non-empty closed convex $\mathcal{V} \subset \mathcal{X}\subseteq \R^d$, $\psi_1, \dots, \psi_T$ regularizers}
    \FOR{$t=1$ {\bfseries to} $T$}
    \STATE{Receive $\ell_t: \R^d \to (-\infty, +\infty]$, where $\mathcal{V} \subseteq \dom \ell_t$}
    \STATE{$\bx_{t} \in \argmin_{\bx \in \mathcal{V}} \ \psi_t(\bx)+\sum_{i=1}^t \ell_i(\bx)$}
    \STATE{Pay the loss $\ell_t(\bx_t)$}
    \ENDFOR
}
\end{algorithmic}
\end{algorithm}

If $\psi_t\equiv 0$, then Be-The-Regularized-Leader becomes the \textbf{Be-The-Leader} algorithm\index{Be-the-Leader!algorithm|textbf}.
BTRL can be thought of as Optimistic \ac{FTRL} where $\tilde{\ell}_t=\ell_t$.
Hence, from the regret of Optimistic \ac{FTRL} in Theorem~\ref{thm:ftrl_optimistic}, we immediately have the following theorem.
\begin{theorem}
Let $\mathcal{V} \subset \R^d$ be convex, closed, and non-empty.
Assume for $t=1, \dots, T$ that $\psi_{t} + \sum_{i=1}^{t} \ell_i$ is proper, closed, and $\lambda_t$-strongly convex with respect to $\|\cdot\|$.
Then, for all $\bu \in \mathcal{V}$, Algorithm~\ref{alg:btrl} satisfies
\begin{align*}
&\sum_{t=1}^T \ell_t(\bx_t) - \sum_{t=1}^T \ell_t(\bu)\\
&\quad \leq \psi_{T}(\bu) - \psi_{1}(\bx_1) - \sum_{t=1}^T \frac{\lambda_t}{2} \|\bx_t-\bx_{t+1}\|^2 +\sum_{t=1}^{T-1}(\psi_t(\bx_{t+1}) - \psi_{t+1}(\bx_{t+1})) ~.
\end{align*}
\end{theorem}

\begin{remark}
In the Be-The-Leader algorithm, if all the $\psi_t\equiv0$, then the theorem states that the regret is non-positive.
\end{remark}

Notably, the non-negative gradient terms are missing in the bound of BTRL, but we still have the negative ones associated with the change in $\bx_t$.

Using, for example, BTRL for the $\mathcal{X}$-player and Optimistic \ac{FTRL} for the $\mathcal{Y}$-player, we have
\begin{align*}
\sum_{t=1}^T &f(\bx_t,\by) - \sum_{t=1}^T f(\bx,\by_t)
\leq \psi_\mathcal{X}(\bx) + \psi_\mathcal{Y}(\by) +\frac{\|\bg_{\mathcal{Y},1}\|^2_{\mathcal{Y},\star}}{\lambda_{\mathcal{Y}}}\\
&+ \sum_{t=2}^T \left(\left(\frac{2L^2_{\mathcal{X}\mathcal{Y}}}{\lambda_{\mathcal{Y}}}-\frac{\lambda_{\mathcal{X}}}{4}\right) \|\bx_t-\bx_{t-1}\|^2_\mathcal{X} + \left(\frac{2L^2_{\mathcal{Y}\mathcal{Y}}}{\lambda_{\mathcal{Y}}}-\frac{\lambda_{\mathcal{Y}}}{4}\right) \|\by_t-\by_{t-1}\|^2_\mathcal{Y}\right)~.
\end{align*}
\index{Be-the-Regularized-Leader algorithm|)textbf}

\section{History Bits}
\label{sec:saddle-point_history}

Theorem~\ref{thm:saddle_equiv_minmax} is \citet[Lemma~36.2]{Rockafellar70}.
The proof of Theorem~\ref{thm:minmax_to_oco} is from \citet{LiuO21} and, in turn, is based on the one in \citet{AbernethyW17}: \citet{LiuO21} stressed the dependence of the regret on a competitor that can be useful for refined bounds. Different variants of this theorem are known in the game theory community as ``Folk Theorems'', because such a result was widely known among game theorists in the 1950s, even though no one had published it.

\index{minimax theorem|(}
The celebrated minimax theorem for zero-sum two-person games was first discovered by John von Neumann in the 1920s~\citep{Neumann28,NeumannM44}. The version stated here is a simplification of the generalization due to \citet{Sion58}. The proof here is from \citet{AbernethyW17}. A similar proof is in \citet{Cesa-BianchiL06} based on a discretization of the space that in turn is based on the one in \citet{FreundS96,FreundS99b}.\index{minimax theorem|)}

Algorithms~\ref{alg:minmax_to_oco2} and \ref{alg:minmax_to_oco3} are a generalization of the algorithm for boosting\index{boosting} in \citet{FreundS96,FreundS99b}. Algorithm~\ref{alg:minmax_to_oco2} was also used in \citet{AbernethyW17} to recover variants of the Frank--Wolfe algorithm\index{Frank--Wolfe algorithm}~\citep{FrankW56}.

It is not clear who invented alternation\index{alternation|(}: it was a known trick used in implementations of CFR\index{Counterfactual Regret Minimization algorithm} for the computer poker competition from 2010 or so\footnote{Christian Kroer, 2021, personal communication.}. Note that in CFR, the method of choice is Regret Matching~\citep{HartMC00}. However, \citet{Kroer20} empirically shows that alternation improves even \ac{OGD} for solving bilinear games.
\citet{TammelinBJB15} explicitly include this trick in their implementation of an improved version of CFR called CFR+, claiming that it would still guarantee convergence.
However, \citet{FarinaKS19} pointed out that averaging of the iterates in alternation might not produce a solution to the min-max problem, providing a counterexample. Theorem~\ref{thm:minmax_to_oco_alternation} is from \citet{BurchMS19}.
Recently, \citet{NanGIK25} showed that alternating gradient descent-ascent, without the use of optimism, gives a $\mathcal{O}(1/T)$ convergence for the duality gap\index{duality gap} of two-player zero-sum games when the saddle-point is in the interior of the set, confirming the empirical observation of \citet{Kroer20}.

There is also a complementary view on alternation: \citet{ZhangWLG21} link alternating updates to Gauss-Seidel methods in numerical linear algebra, in contrast to the simultaneous updates of the Jacobi method. Also, they provide a good review of the optimization literature on the advantages of alternation, but this paper and the papers they cite do not seem to be aware of the use of alternation in CFR.
\index{alternation|)}

The reduction from boosting\index{boosting} to \ac{LEA} is from \citet{FreundS96}. It seems\footnote{I could not verify this claim because I could not find the paper anywhere.} that the question if a weak learner can be boosted into a strong learner was originally posed by \citet{KearnsV88} (see also \citet{Kearns88}). It was answered in the positive by \citet{Schapire90}. The AdaBoost algorithm is from \citet{FreundS95,FreundS97}.
The idea of using algorithms to guarantee a regret bound that depends on the \ac{KL} divergence, as in \eqref{eq:better_bound_boosting}, is from \citet{HaipengS14}.

\citet{DaskalakisDK11} proposed the first no-regret algorithm that achieved a rate of $\mathcal{O}(\frac{\ln T}{T})$ for the duality gap\index{duality gap} when used by the two players of a zero-sum game without any communication between the players.
However, the algorithm was rather complex, and they posed the problem of obtaining the same or faster rate with a simpler algorithm. \citet{RakhlinS13} solved this problem, showing that two Optimistic \ac{OMD} algorithms can solve the problem in a simpler way. Theorems~\ref{thm:minmax_to_optimistic_ftrl} and~\ref{thm:minmax_to_optimistic_omd} derive directly from \citet[Corollary 5]{RakhlinS13}. For some reason, \citet[Corollary 5]{RakhlinS13} was missed in recent years, so the $\mathcal{O}(1/T)$ convergence for smooth saddle-point problems using optimistic gradient descent/ascent has been rediscovered several times~\citep[e.g.,][]{HsiehIMM19,MokhtariOP20}.
However, the optimistic gradient descent/ascent for saddle-point problems is much older: it was proposed for the first time by \citet{Popov80}, as a modification of the Arrow-Hurwicz method~\citep{ArrowHU58}. Roughly 30 years later, the optimistic algorithms were rediscovered, first as pure online learning algorithms~\citep{Chiang12,RakhlinS13b} and then used to solve saddle-point problems~\citep{RakhlinS13}.

The observation after Theorem~\ref{thm:minmax_to_optimistic_ftrl} that it is possible to achieve constant regret for each player is from \citet{Luo22}. For general games, it is possible to achieve logarithmic regret for each player by using modified versions of optimistic algorithms~\citep[see, e.g.,][and references therein]{SoleymaniPF25}.

The use of Prescient Online Mirror Descent\index{Online Mirror Descent algorithm!prescient} in saddle-point optimization is from \citet{WangAL21}, but when renaming $\bx_{t}$ to $\bx_{t+1}$ it is also equivalent to implicit online mirror descent~\citep{KivinenW97,KulisB10,CampolongoO20}. Theorem~\ref{thm:pomd} is from the guarantee of implicit online mirror descent in \citet{CampolongoO20}.

There is also a tight connection between optimistic updates using the previous gradients and classic approaches to solving saddle-point optimization. In fact, \citet{GidelBVVLJ18} showed that using two optimistic gradient descent algorithms to solve a saddle-point problem can be seen as a variant of the Extra-gradient updates~\citep{Korpelevich76}, while \citet{MokhtariOP20} show that they can be interpreted as an approximated proximal point algorithm.

For generic saddle-point problems, \citet{Popov80} proved the asymptotic convergence of the iterates when using two optimistic \ac{OGD} algorithms. This old result was unknown to the majority of the community until recently, and it implies some later weaker results \citep[e.g.,][]{DaskalakisLSZ18}. The extension to the Mirror Descent case was done by \citet{Semenov17} to solve the more general problem of variational inequalities, but only for distance generating functions\index{distance generating function} that are differentiable on the entire feasible set.\footnote{Equation 17 in \citet{Semenov17} needs the fact that the distance generating function\index{distance generating function} is continuously differentiable everywhere for a limit operation. This assumption is not stated, making the proof essentially wrong, for example, for the negative entropy function. As far as I know, no one had realized this issue in that paper before.} This means that this proof does not cover the optimistic \ac{EG}.
In turn, this little-known result was also recently rediscovered~\citep[e.g.,][Theorem 4]{LeeKL21}.
\citet[Theorem 7]{HsiehAM21} proved the asymptotic convergence of the last iterate for Optimistic \ac{FTRL} with linear losses and an adaptive regularization weight, assuming either strict convexity/concavity or that the regularizer is differentiable on the entire domain. Note that the proof in \citet[Theorem 7]{HsiehAM21} can be easily adapted to prove the same result for optimistic \ac{OMD} with a fixed and small enough learning rate. \citet{LeiNPW21} prove the asymptotic convergence of optimistic \ac{EG} for saddle-point problems, under the assumption that some of the optimality conditions are satisfied in a strict way.

For the specific case of bilinear games, stronger results can be proven on the last-iterate convergence.
\citet{LiangS19} proved that if the matrix $\bA$ is square, full-rank, and the problem is unconstrained, then the iterates of two Optimistic \ac{OGD} will converge exponentially fast to the saddle point in the origin. \citet{DaskalakisP19} proved the asymptotic convergence of optimistic \ac{OMD}/\ac{FTRL} \ac{EG} with fixed stepsize for bilinear games over probability simplices, assuming a unique saddle point. \citet{WeiLZL21} proved an exponential rate for the same algorithm under the same assumptions. \citet[Theorem 8]{HsiehAM21} removed the assumption of a unique saddle point, proving asymptotic convergence for Optimistic \ac{FTRL} with entropic regularization and an adaptive regularization weight. Once again, this proof can be easily modified to prove the same result for optimistic \ac{OMD} with a fixed and small enough learning rate.

\section{Exercises}

\begin{exer}
\label{exercise:duality_gap}
Prove Lemma~\ref{lemma:equivalence_saddle_point}.
\end{exer}

\begin{exer}
Let $f(\bx,\by)$ be a convex-concave function, Lipschitz with respect to both variables. Use two \ac{FTRL} algorithms to find the saddle-point, using as regularizers $\psi_t(\bx) \propto \sqrt{t} \psi(\bx)$, where $\psi$ is differentiable and $1$-strongly convex with respect to a norm $\|\cdot\|$. Then, using the inequality in Problem~\ref{exercise:ftrl_last_bregman}, show that the trajectory of the iterates of the above strategy is bounded on any convex-concave saddle-point problem with at least one saddle-point, even in unbounded domains and any number of dimensions.
\end{exer}

\acresetall

\chapter{From Online Learning to X}
\label{ch:online_to_x}

In this chapter, we show that online learning is much more general than one might think. We already saw that online learning algorithms can be used for stochastic convex optimization (Chapter~\ref{ch:o2b}) and for saddle-point optimization (Chapter~\ref{ch:saddle-point}). Yet, these were still convex optimization problems. Instead, here we will show that online learning results can be used to prove generalization guarantees, design algorithms for the optimization of non-convex non-smooth functions, and construct confidence intervals.

\acresetall

\section{From Online Learning to Rademacher Complexity}
\label{sec:online_to_rademacher}

Here, we will quantify the statistical difficulty of a hypothesis class\index{hypothesis class} through its \emph{Rademacher complexity}, which measures how well the functions in the class can correlate with pure noise. Rademacher complexity is the key quantity behind many uniform convergence and generalization results.
In this section, \emph{we show a simple reduction that upper bounds the Rademacher complexity of a class using the regret guarantee of an online learning algorithm.}

\noindent\textbf{Rademacher Complexity.}
We consider again Vapnik's general setting of learning\index{Vapnik's general setting of learning} that we saw in Section~\ref{sec:agnostic_pac}, this time for generic function classes. Hence, we seek to minimize
\begin{equation}
\label{eq:vapnik_general_obj}
\argmin_{f \in \mathcal{H}} \ \E_{\bxi\sim\rho} [f(\bxi)]
\end{equation}
where $\mathcal{H}$ is a class of real-valued functions on $\mathcal{Z}$ and $\rho$ is a distribution over $\mathcal{Z}$.
One common way to solve this problem is through an \acl{ERM} process. That is, we gather a sample of size $T$, $\{\bxi_1, \dots, \bxi_T\}$, drawn i.i.d. from $\rho$, then we solve the empirical version of \eqref{eq:vapnik_general_obj}:
\[
f_\mathcal{S}
:=\argmin_{f \in \mathcal{H}} \ \frac{1}{T} \sum_{t=1}^T f(\bxi_t)~.
\]
It is now natural to ask if $f_\mathcal{S}$ will also minimize the objective function in \eqref{eq:vapnik_general_obj}.

One way to study this problem is through the concept of \emph{Rademacher complexity} of a function class.
\begin{definition}
Let $\mathcal{S}=\{\bxi_1,\dots,\bxi_T\}$ be an arbitrary set of $T$ vectors in some domain $\mathcal{Z}$.
Let $\mathcal{H}$ be a class of real-valued functions on $\mathcal{Z}$.
Let $\epsilon_1,\dots,\epsilon_T$ be i.i.d. Rademacher random variables\index{random variable!Rademacher}, that is, $\Pr\{\epsilon_t=1\}=\Pr\{\epsilon_t=-1\}=1/2$.
The \textbf{empirical Rademacher complexity}\index{Rademacher complexity!empirical|textbf} of $\mathcal{H}$ on $\mathcal{S}$ is defined as
\begin{equation}
\label{eq:emp_rad}
\wh{\mathfrak{R}}_{\mathcal{S}}(\mathcal{H})
:= \frac{1}{T}\, \E_{\epsilon_1, \dots, \epsilon_T}\left[ \sup_{f\in\mathcal{H}} \ \sum_{t=1}^T \epsilon_t f(\bxi_t)\right]~.
\end{equation}
\end{definition}

Intuitively, if $\wh{\mathfrak{R}}_{\mathcal{S}}(\mathcal{H})$ is small, then no function in the class can fit random signs too well on the sample. Hence, we can expect the minimizers of the empirical risk to be close to the minimizers of the true risk.

This idea can be made formal and yields generalization guarantees; for example, bounds of the form
\[
\sup_{f\in\mathcal{H}} \ \left| \E[f(\bxi)] - \frac{1}{T}\sum_{t=1}^T f(\bxi_t) \right|
\lesssim \wh{\mathfrak{R}}_{\mathcal{S}}(\mathcal{H}) + \sqrt{\frac{\ln(1/\delta)}{T}}
\]
with probability at least $1-\delta$ under mild boundedness assumptions. Given that the above bound applies uniformly over $\mathcal{H}$, it also applies to the empirical risk minimizer $f_\mathcal{S}$.

\begin{remark}
One can easily obtain a generalization guarantee that depends on $f_\mathcal{S}$, rather than being uniform over a class. It is enough to consider a countable nested family $(\mathcal{H}_k)_{k=1}^\infty$ whose union covers the entire space of possible outcomes of the empirical risk minimization process. Then, apply the uniform bound to each $\mathcal{H}_k$ with a union bound\index{union bound}, and choose the class with the smallest Rademacher complexity containing the empirical risk minimizer.
\end{remark}

We will not prove the generalization bounds here; instead, we focus on the following question:
\emph{how can we upper bound $\wh{\mathfrak{R}}_{\mathcal{S}}(\mathcal{H})$ for interesting classes?}

We consider function classes given by composing Lipschitz functions with linear predictors.
Given a prediction function class $\{\phi_{\bx}:\bx\in\mathcal{V}\}$ and a loss $\ell(\hat{y},y)$,
we obtain the induced function class
\[
\mathcal{H}_{\ell}
:= \left\{\bxi=(\bz,y)\mapsto \ell(\phi_{\bx}(\bz),y)\ :\ \bx\in\mathcal{V}\right\},
\]
where $(\bz,y) \in \mathcal{Z} \times \mathcal{Y}$.
Hence, we would like to control $\wh{\mathfrak{R}}_{\mathcal{S}}(\mathcal{H}_{\ell})$.
A common situation is that $\ell(\cdot,y)$ is Lipschitz, say $L$-Lipschitz for all $y$.
In this case, it is intuitive that a Lipschitz loss cannot increase the ability to correlate with noise by more than a factor $L$.
This is formalized by contraction inequalities, as in the next lemma.
\begin{lemma}[Contraction Lemma]
\index{contraction lemma|(textbf}
\label{lemma:contraction}
Let $h:\R \to \R$ be $L$-Lipschitz. Consider the function class $h \circ \mathcal{H}=\{h \circ f: f \in \mathcal{H}\}$. Then, for any sample $\mathcal{S}$, we have $\wh{\mathfrak{R}}_{\mathcal{S}}(h\circ \mathcal{H}) \leq L \, \wh{\mathfrak{R}}_{\mathcal{S}}(\mathcal{H})$.
\index{contraction lemma|)textbf}
\end{lemma}
Hence, in the following, we will focus on controlling the complexity of linear functions.

\noindent\textbf{Rademacher Complexity of Linear Classes through Regret Guarantees.}
We now state the core reduction. Let $\mathcal{V}\subset\R^d$ be a non-empty, bounded, closed set, and consider the linear class
\[
\mathcal{H}_{\mathrm{lin}}
:= \left\{ \bz \mapsto \langle \bx,\bz\rangle \ :\ \bx\in\mathcal{V} \right\}~.
\]
Fix a sample $\mathcal{S}=\{\bz_1,\dots,\bz_T\}$ and define the empirical Rademacher complexity of $\mathcal{H}_{\mathrm{lin}}$ on $\mathcal{S}$ as in \eqref{eq:emp_rad}. The following theorem shows how to use the regret of an online learning algorithm to upper bound such empirical Rademacher complexity.
\begin{theorem}
\label{thm:rad_from_regret}
Let $\mathcal{Z}\subset\R^d$ such that for any $\bz \in \mathcal{Z}$ we also have $-\bz \in \mathcal{Z}$.
Fix any sequence $\bz_1,\dots,\bz_T\in\mathcal{Z}$ and any non-empty, bounded, closed set $\mathcal{V}\subset\R^d$.
Assume that there exists an online algorithm that, on any sequence of linear losses
$\ell_t(\bx)=\langle \bg_t,\bx\rangle$, where $\bg_t \in \mathcal{Z}$, produces $\bx_t\in\mathcal{V}$ and satisfies the regret guarantee
\begin{equation}
\label{eq:regret_generic}
\sum_{t=1}^T \langle \bg_t,\bx_t\rangle
- \min_{\bu\in\mathcal{V}} \ \sum_{t=1}^T \langle \bg_t,\bu\rangle
\leq R(T)~.
\end{equation}
Then, the empirical Rademacher complexity of the linear class satisfies
\begin{equation}
\label{eq:rad_bound_linear}
\wh{\mathfrak{R}}_{\mathcal{S}}(\mathcal{H}_{\mathrm{lin}})
= \frac{1}{T} \E_{\epsilon_1, \dots, \epsilon_T}\left[ \sup_{\bx\in\mathcal{V}}\  \sum_{t=1}^T \epsilon_t \langle \bx,\bz_t\rangle \right]
\leq \frac{R(T)}{T}~.
\end{equation}
\end{theorem}
\begin{proof}
Fix $\bz_1,\dots,\bz_T$ and draw $\epsilon_1,\dots,\epsilon_T$.
Consider the online linear game with gradients $\bg_t:= -\epsilon_t \bz_t$.
Let $\bx_1,\dots,\bx_T$ be the iterates produced by the online algorithm when fed with $\bg_1,\dots,\bg_T$.
By the regret guarantee \eqref{eq:regret_generic} we have, for every realization of the signs,
\begin{align}
\label{eq:rad_regret_step}
\sum_{t=1}^T \langle -\epsilon_t \bz_t,\bx_t\rangle - \min_{\bu\in\mathcal{V}} \ \sum_{t=1}^T \langle -\epsilon_t \bz_t,\bu\rangle
&\leq R(T)~.
\end{align}
Rearranging, we get
\[
\max_{\bu\in\mathcal{V}} \ \sum_{t=1}^T \epsilon_t \langle \bu,\bz_t\rangle
\leq R(T) + \sum_{t=1}^T \epsilon_t \langle \bx_t,\bz_t\rangle~.
\]
Taking expectation with respect to $\epsilon_1,\dots,\epsilon_T$ and dividing by $T$ yields
\begin{equation}
\label{eq:rad_exp_split}
\wh{\mathfrak{R}}_{\mathcal{S}}(\mathcal{H}_{\mathrm{lin}})
\leq \frac{R(T)}{T} + \frac{1}{T}\E_{\epsilon_1, \dots, \epsilon_T}\left[\sum_{t=1}^T \epsilon_t \langle \bx_t,\bz_t\rangle\right]~.
\end{equation}

It remains to show that the second term is $0$.
The key observation is that $\bx_t$ is measurable with respect to the past signs $\epsilon_1,\dots,\epsilon_{t-1}$ and possibly its internal randomization,
while $\epsilon_t$ is independent of the past and has mean zero.
Hence, conditioning on the past,
\[
\E_{\epsilon_1, \dots, \epsilon_T}\!\left[\epsilon_t \langle \bx_t,\bz_t\rangle \middle| \epsilon_1,\dots,\epsilon_{t-1}, \text{internal randomness}\right]
= \langle \bx_t,\bz_t\rangle\,\E[\epsilon_t]
=0,
\]
and by taking expectation again we obtain $\E[\epsilon_t \langle \bx_t,\bz_t\rangle]=0$.
Summing over $t$ makes the last term in~\eqref{eq:rad_exp_split} equal to 0, so we obtain \eqref{eq:rad_bound_linear}.
\end{proof}

\begin{remark}
The proof of Theorem~\ref{thm:rad_from_regret} has the same flavor as the probabilistic method argument used in lower bounds in Chapter~\ref{ch:lower}: we introduce Rademacher signs, feed them to an online algorithm, and then exploit the fact that the algorithm cannot correlate with the \emph{fresh} randomness at time $t$.
The entire complexity term is paid by the regret upper bound $R(T)$.
\end{remark}

We can instantiate the above theorem using the \ac{OMD} algorithms with $p$-norms, see Section~\ref{sec:omd_pnorm}.
\begin{corollary}
Let $q\in [2, \infty)$ and $p\in (1,2]$ such that $1/p+1/q=1$. For $U_p>0$, let $\mathcal{V}_p:=\{\bx \in \R^d: \|\bx\|_p\leq U_p\}$.
Let $\mathcal{S}:=\{\bz_1, \dots, \bz_T\}$ where $\bz_t\in \R^d$ and $\|\bz_t\|_q\leq G_q$ for $t=1, \dots, T$.
Let $\mathcal{H}_{p} := \left\{ \bz \mapsto \langle \bx,\bz\rangle \ :\ \bx\in\mathcal{V}_p \right\}$. Then,
\[
\frac{1}{T} \E_{\epsilon_1, \dots, \epsilon_T}\left[ \sup_{\bx\in\mathcal{V}_p} \ \sum_{t=1}^T \epsilon_t \langle \bx,\bz_t\rangle \right]
\leq \frac{U_p G_q}{\sqrt{(p-1)T}}~.
\]
\end{corollary}
\begin{proof}
We instantiate Theorem~\ref{thm:rad_from_regret} with \ac{OMD} with $p$-norms, learning rate $\eta=\frac{U_p\sqrt{p-1}}{G_q \sqrt{T}}$ and $\bx_1=\boldsymbol{0}$.
\end{proof}

We can also consider the case of classes of functions with finite cardinality.
\index{Massart's lemma|(textbf}
\begin{corollary}[Massart's Lemma]
Let $\mathcal{S}:=\{\bz_1, \dots, \bz_T\}$ where $\bz_t\in \R^d$ and $\|\bz_t\|_\infty\leq G_\infty$ for $t=1, \dots, T$.
Let $\mathcal{H} := \left\{ \bz \mapsto \langle \bx,\bz\rangle \ :\ \bx\in\Delta^{d-1} \right\}$. Then,
\[
\frac{1}{T} \E_{\epsilon_1,\dots, \epsilon_T}\left[ \max_{i=1,\dots, d} \ \sum_{t=1}^T \epsilon_t z_{t,i}\right]
\leq \frac{G_\infty \sqrt{2 \ln d}}{\sqrt{T}}~.
\]
\end{corollary}
\begin{proof}
We instantiate Theorem~\ref{thm:rad_from_regret} with \ac{OMD} with entropic regularizer, learning rate $\eta=\frac{\sqrt{2 \ln d}}{G_\infty \sqrt{T}}$ and $\bx_1=[1/d, \dots, 1/d]^\top$.
\end{proof}
\index{Massart's lemma|)textbf}

\section{From Online Learning to PAC-Bayes}
\label{sec:online_to_pacbayes}
\index{PAC-Bayes generalization bound|(}

In this section, we show that one can directly obtain generalization bounds for any machine learning algorithm from upper bounds on linear regret.

We will consider the same setting we used in the online-to-batch reduction (Chapter~\ref{ch:o2b}), where one aims at minimizing the risk of a function $f:\mathcal{Z} \to [0,1]$, defined as
\[
\Risk(f)=\E_{\bxi\sim \rho} [f(\bxi)]~.
\]
For example, in a regression setting, $f(\bxi)$ is the composition of a loss with a prediction function, evaluated on the data $\bxi=(\bz,y)$.

Given a training set $\mathcal{S}=\{\bxi_1, \dots, \bxi_T\}$ drawn i.i.d. from $\rho$, here we are interested in the \textbf{generalization gap}\index{generalization gap|(}, defined as the difference between the risk and the training error of a function $f$:
\[
\Gen(f,\mathcal{S})
:= \Risk(f) - \frac{1}{T}\sum_{t=1}^T f(\bxi_t)~.
\]
We can also consider the case where the predictor is randomized, in the sense that we draw $f$ according to a probability distribution $Q$ in $\Delta(\mathcal{H})$, that is the set of probability distributions over a set of functions $\mathcal{H}:=\{f \mid f:\mathcal{Z} \to [0,1]\}$. In this case, we study
\[
\barGen(Q,\mathcal{S})
:= \E_{f \sim Q}[\Gen(f,\mathcal{S})]~.
\]
In particular, we are interested in upper bounding $\barGen(Q_\mathcal{S},\mathcal{S})$ in high probability, where $Q_\mathcal{S}$ is selected from $\Delta(\mathcal{H})$ after looking at $\mathcal{S}$.
\index{generalization gap|)}

We now describe the reduction from this problem to an online learning one.
Consider an online algorithm that at each round $t$ produces a distribution $P_t \in \Delta(\mathcal{H})$ after observing the training samples $\bxi_1, \dots, \bxi_{t-1}$.
Define $g_t(f)=f(\bxi_t) - \E_{\bxi\sim\rho} [f(\bxi)]$ and $\ell_t(P)=\E_{f\sim P}[g_t(f)]$.
So, for any $Q\in \Delta(\mathcal{H})$, we have that
\begin{align*}
\barGen(Q,\mathcal{S})
&= \E_{f \sim Q}[\E_{\bxi\sim\rho}[f(\bxi)]] - \frac{1}{T}\sum_{t=1}^T \E_{f \sim Q}[f(\bxi_t)]
= \frac{1}{T}\sum_{t=1}^T -\ell_t(Q) \\
&= \frac{1}{T}\sum_{t=1}^T (\ell_t(P_t)-\ell_t(Q)) - \frac{1}{T}\sum_{t=1}^T \ell_t(P_t)
= \frac{\Regret_T(Q)}{T} - \frac{1}{T}\sum_{t=1}^T \ell_t(P_t)~.
\end{align*}

We use $\mathscr{F}_t$ to denote the $\sigma$-algebra generated by the data points
$\bxi_1,\dots,\bxi_t$ and by all the random variables generated by the online learning algorithm up to the end of round $t$.
Now, given that $P_t$ is $\mathscr{F}_{t-1}$-measurable and $\bxi_t$ is independent of $\mathscr{F}_{t-1}$, we have
\[
\E[\ell_t(P_t)|\mathscr{F}_{t-1}]
=\E[\E_{f \sim P_t}[f(\bxi_t)-\E_{\bxi\sim\rho}[f(\bxi)]]|\mathscr{F}_{t-1}]
=0~.
\]
Since $f(\bxi)\in[0,1]$, we have
\[
g_t(f)=f(\bxi_t)-\E_{\bxi\sim\rho}[f(\bxi)] \in [-1,1],
\]
and therefore
\[
\ell_t(P_t)=\E_{f\sim P_t}[g_t(f)] \in [-1,1]~.
\]
Hence, $\ell_1(P_1),\ell_1(P_1)+\ell_2(P_2), \dots, \sum_{t=1}^T\ell_t(P_t)$ is a martingale\index{martingale} with bounded increments in $[-1,1]$, so we can apply the Hoeffding--Azuma inequality\index{inequality!Hoeffding--Azuma} in Theorem~\ref{thm:azuma}. This concentration does not depend on $Q$, so it holds simultaneously for all $Q \in \Delta(\mathcal{H})$.

We are done! We can now put everything together to have the following theorem.
\begin{theorem}
Let $\mathcal{H}$ be a measurable space of measurable functions $f:\mathcal{Z}\to [0,1]$, $\delta \in (0,1)$, and $\pi \in \Delta(\mathcal{H})$.
Let $\mathcal{S}=\{\bxi_1,\dots, \bxi_T\}$ be drawn i.i.d. from a distribution $\rho$ over $\mathcal{Z}$.
Consider any online learning algorithm that outputs a distribution over $\mathcal{H}$ and is fed with the linear losses $\ell_t(P)=\E_{f \sim P}[f(\bxi_t) - \E_{\bxi\sim\rho} [f(\bxi)]]$.
Then, with probability at least $1-\delta$, for all $Q\in \Delta(\mathcal{H})$ such that $Q \ll \pi$, even selected with the knowledge of $\mathcal{S}$, we have
\[
\barGen(Q,\mathcal{S})
= \frac{\Regret_T(Q)}{T} - \frac{1}{T} \sum_{t=1}^T \ell_t(P_t)
\leq \frac{\Regret_T(Q)}{T} + \sqrt{\frac{2\ln \frac{1}{\delta}}{T}}~.
\]
\end{theorem}

We can now instantiate this theorem with the regret of any online learning algorithm. For example, we could use the \ac{EG} algorithm, where we think of each $f \in \mathcal{H}$ as an expert. There are only two caveats: first, we would have to extend it to the case where the number of experts is infinite, possibly continuous; second, we would have to select an appropriate learning rate.

The first issue is easy to deal with: roughly speaking, it is enough to substitute any sum over the experts with integrals. Using the fact that $g_t(f) \in [-1,1]$,  the regret guarantee of the continuous version of \ac{EG} is
\[
\frac{\KL(Q;\pi)}{\eta} + \frac{\eta T}{2},
\]
where $\pi$ is the prior distribution over the infinite experts and $Q$ is the competitor distribution. We might not know how to run this algorithm because of the integrals, but we do not need to! We only need to know that such an algorithm and its regret bound exist.

The second problem, instead, is a difficult one: the optimal learning rate depends on $Q$. However, we want generalization guarantees that hold uniformly for any $Q$.  This is \emph{exactly} the problem we saw many times in online learning when the optimal choice of the learning rate depends on the unknown comparator. In standard \ac{EG}, we upper bounded the \ac{KL} divergence\index{Kullback--Leibler divergence} term for the case of a uniform prior with $\ln d$, but here we cannot do it because the number of experts is infinite.
One could construct a grid of learning rates, instantiate the bound for each of them, and use a union bound, but this approach could introduce additional poly-logarithmic terms in the final bound. However, we already know how to solve this problem in online learning, by simply using \emph{parameter-free} algorithms.

\index{parameter-free!learning with continuous experts|(}
So, we can consider a continuous version of the parameter-free algorithm for learning with expert advice in Section~\ref{sec:shifted_kt}, and we can show this regret bound.
\begin{theorem}
\label{thm:continuous_lea_kt}
Let $\mathcal{H}$ be a measurable space of measurable functions $f:\mathcal{Z}\to [0,1]$, and $\pi \in \Delta(\mathcal{H})$.
At each round $t$, the adversary reveals a measurable loss function $g_t:\mathcal{H}\to [-1,1]$.
There exists an online learning algorithm that depends on $\pi$, outputs a distribution $P_t \in \Delta(\mathcal{H})$, and incurs the loss $\ell_t(P)=\E_{f \sim P}[g_t(f)]$.
For any competitor distribution $Q \in \Delta(\mathcal{H})$ such that $Q \ll \pi$, its regret satisfies
\[
\Regret_T(Q)
:=\sum_{t=1}^T \left(\E_{f\sim P_t}[g_t(f)]-\E_{f\sim Q}[g_t(f)]\right)
\leq \sqrt{20 T\left(\KL(Q;\pi)+1\right)}~.
\]
\end{theorem}
The proof is a straightforward adaptation of the finite-expert case, but for completeness, we leave it as an exercise.
\index{parameter-free!learning with continuous experts|)}

Combining the two previous theorems,
we obtain the following so-called \textbf{PAC-Bayes} bound.
\begin{corollary}[PAC-Bayes Bound]
\label{cor:online_to_pacbayes}
Let $\mathcal{H}$ be a measurable space of measurable functions $f:\mathcal{Z}\to [0,1]$, $\delta \in (0,1)$, and $\pi \in \Delta(\mathcal{H})$.
Let $\mathcal{S}=\{\bxi_1,\dots, \bxi_T\}$ be drawn i.i.d. from a distribution $\rho$ over $\mathcal{Z}$.
Then, with probability at least $1-\delta$, for all $Q\in \Delta(\mathcal{H})$ such that $Q \ll \pi$, even selected with the knowledge of $\mathcal{S}$, we have
\[
\barGen(Q,\mathcal{S})
\leq \frac{\sqrt{20(\KL(Q;\pi)+1)}}{\sqrt{T}} + \sqrt{\frac{2\ln \frac{1}{\delta}}{T}}~.
\]
\end{corollary}

The role of $\pi$ is to allow the bound to hold simultaneously for all posterior distributions $Q$.
Indeed, the theorem guarantees that, with probability at least $1-\delta$ over the draw of the sample $\mathcal{S}$, the inequality holds for every $Q \in \Delta(\mathcal{H})$ such that $Q \ll \pi$, even if $Q$ is selected after observing the data.
This uniformity is possible because $\pi$ acts as a continuous analogue of the union bound\index{union bound}. The divergence term $\KL(Q;\pi)$ plays the role of the logarithmic penalty that appears in the discrete union bound. Hence, choosing $\pi$ corresponds to specifying how the bound is distributed over the different functions in $\mathcal{H}$.

\begin{remark}
As in the previous section, the \emph{existence} of an online learning algorithm with a regret upper bound is enough to prove our bound. We never need to actually run the online algorithm.
Moreover, in this case, we could not run the reduction even if we wanted to, since the linear losses $\ell_t$ depend on the unknown distribution $\rho$.
\end{remark}
\index{PAC-Bayes generalization bound|)}

\section{From Online Learning to Non-Convex Non-Smooth Optimization}
\label{sec:olo-to-nonsmooth-nonconvex}

This section describes a simple (and surprisingly sharp) reduction from OLO to the task of finding approximate stationary points of \emph{non-convex, non-smooth} objectives.
The key idea is to feed stochastic gradients to an online algorithm, which then decides the \emph{updates} rather than the \emph{iterates}.

We consider the objective function $F:\R^d\to\R$, which is differentiable but not necessarily smooth, i.e., the gradient map might not be Lipschitz.

\begin{example}
Consider $F:\R \to \R$ defined as $F(x)=\frac{2}{3}|x|^\frac{3}{2}$. Then, $F'(x)=\sign(x) \sqrt{|x|}$, where $\sign(0):=0$. Hence, $F$ is differentiable everywhere, but its derivative is not Lipschitz because $F''(x)$ is unbounded when $x \to 0$.
\end{example}

To connect function values to gradients without convexity or smoothness, we isolate the only calculus identity needed by the reduction. So, we will define \emph{well-behaved} functions as those that satisfy the specific equality we need.
\begin{definition}
\label{def:well-behaved}
\index{function!well-behaved|(textbf}
A differentiable function $F:\R^d\to\R$ is \textbf{well-behaved} if for every $\bx,\by\in\R^d$,
\begin{equation}
\label{eq:well-behaved}
F(\by)-F(\bx)
= \int_0^1 \! \langle \nabla F(\bx+t(\by-\bx)), \by-\bx\rangle \, \mathrm{d} t~.
\end{equation}
\index{function!well-behaved|)textbf}
\end{definition}

\begin{remark}
\label{rem:well-behaved}
Identity \eqref{eq:well-behaved} is immediate for smooth\index{function!smooth} functions (by the fundamental theorem of calculus applied to $t\mapsto F(\bx+t(\by-\bx))$), but it can hold well beyond smooth objectives.
In particular, \citet{CutkoskyMO23} show that if $F$ is locally Lipschitz, then an arbitrarily small randomized smoothing yields a differentiable, well-behaved surrogate with a natural stochastic gradient oracle.
Here, we focus on the differentiable, well-behaved case to keep the presentation simple.
\end{remark}

We now define our notion of optimality. Since $F$ is non-convex and not necessarily smooth, small gradients are not obviously implied by function decrease; we instead certify a local averaged stationarity notion that measures how close we are to a local minimum by considering \emph{local averages} of gradients.
\begin{definition}
\label{def:delta-grad-norm}
\index{barycentric $(\delta,\epsilon)$-stationary point|(textbf}
Let $F$ be differentiable almost everywhere, and let $\delta>0$.
Let $\mathcal{Q}(\bx,\delta)$ be the set of random variables $\bQ$ with finite support in the ball of radius $\delta$ around $\bx$, such that $\nabla F(\by)$ exists for all $\by$ in the support of $\bQ$.
Define
\[
\|\nabla F(\bx)\|_\delta
:= \inf_{\bQ \in \mathcal{Q}(\bx,\delta)} \ \left\{ \left\| \E\left[\nabla F(\bQ)\right]\right\|_2 : \E[\bQ]=\bx \right\}~.
\]
We say that $\bx$ is a \textbf{barycentric $(\delta,\epsilon)$-stationary point} if $\|\nabla F(\bx)\|_\delta\le \epsilon$.
\index{barycentric $(\delta,\epsilon)$-stationary point|)textbf}
\end{definition}

\begin{remark}
This notion is closely related to Goldstein-type $\delta$-subdifferential stationarity, but we use gradients, and it is \emph{barycentric}: the same convex weights used to average gradients must also average the sampled points back to $\bx$.
\end{remark}

\begin{algorithm}[t]
\caption{Online-to-Non-Convex Conversion}
\label{alg:online-to-nonconvex}
\begin{algorithmic}[1]
{
\REQUIRE{Well-behaved $F:\R^d \to \R$, OLO algorithm $\mathscr{A}$ over $\mathcal{V}=\{\bx\in \R^d:\|\bx\|_2\le D\}$, cycle length $K$, horizon $T$ as a multiple of $K$, initial point $\bx_0$}
\STATE{$j = 0$}
\FOR{$t=1$ {\bfseries to} $T$}
    \IF{$t \equiv 1 \!\!\!\pmod K$}
        \STATE{Reset $\mathscr{A}$}
        \STATE{$j = j+1$,\quad $\bar{\bx}_j = \boldsymbol{0}$}
    \ENDIF
    \STATE{Receive update $\bm_t\in\mathcal{V}$ from $\mathscr{A}$}
    \STATE{$\bx_{t} = \bx_{t-1} - \bm_t$}
    \STATE{Sample $s_t\sim \mathrm{Unif}[0,1]$ and set $\bx'_{t} = \bx_{t-1} - s_t \bm_t$}
    \STATE{Query a stochastic gradient $\bg_t$ at $\bx'_t$}
    \STATE{Define loss $\ell_t(\bm) := -\langle \bg_t,\bm\rangle$ and pass it to $\mathscr{A}$}
    \STATE{$\bar{\bx}_j = \bar{\bx}_j + \bx'_{t}/K$}
\ENDFOR
\STATE{Return $\bar{\bx}_J$ where $J$ is uniform on $\{1,\dots,T/K\}$}
}
\end{algorithmic}
\end{algorithm}

\noindent\textbf{The Reduction Algorithm.}
We now present the reduction from non-convex optimization of well-behaved functions to online linear optimization.

Fix a cycle length $K\in\Nat$ and a radius $D>0$.
We will run an OLO algorithm $\mathscr{A}$ on the Euclidean ball $\mathcal{V} = \{\bx\in\R^d:\|\bx\|_2\le D\}$,
where the OLO algorithm outputs \emph{updates} $\bm_t\in\mathcal{V}$.
The iterates are defined by $\bx_{t}=\bx_{t-1}-\bm_t$.
Then, at each step, we draw a random point $\bx'_{t}$ on the segment from $\bx_{t-1}$ to $\bx_{t}$ and query a stochastic gradient there.
The final output $\bar{\bx}_J$ is the average of $K$ points $\bx'_{t}$ inside a random cycle $J$.
The complete procedure is given in Algorithm~\ref{alg:online-to-nonconvex}.

\begin{theorem}
\label{thm:non_convex_conversion}
Let $F$ be well-behaved \index{function!well-behaved}(Definition~\ref{def:well-behaved}).
Let $\mathscr{F}_t$ be the sigma-field generated by all randomness up to round $t$ before querying $\bg_t$, so that $\bx'_t$ is $\mathscr{F}_t$-measurable. With the notation in Algorithm~\ref{alg:online-to-nonconvex}, we assume
\begin{equation}
\label{eq:stoch-oracle}
\E[\bg_t\mid \mathscr{F}_t]=\nabla F(\bx'_t),
\qquad
\E[\|\bg_t\|_2^2\mid \mathscr{F}_t]\le G^2,
\end{equation}
for some $G>0$.
Let $T\in \Nat$ be a multiple of the cycle length $K$.
Run Algorithm~\ref{alg:online-to-nonconvex} for $T$ steps with cycle length $K$, and instantiate $\mathscr{A}$ as projected \ac{OGD} over $\mathcal{V}=\{\bx\in\R^d:\|\bx\|_2\le D\}$ with learning rate $\eta=D/(G\sqrt{K})$ and initial point equal to $\boldsymbol{0}$ on each cycle (i.e., after each reset).
Then,
\[
\E\left[ \frac{1}{T/K}\sum_{j=1}^{T/K} \left\|\frac{1}{K}\sum_{t=(j-1)K+1}^{jK}\nabla F(\bx'_t)\right\|_2 \right]
\le \frac{F(\bx_0)-\inf_{\bx} F(\bx)}{D\,T} + \frac{2G}{\sqrt{K}}~.
\]
In particular, for $\delta>0$, choose $D=\delta/K$ and then set
\[
K \asymp \left(\frac{G \, T \, \delta}{F(\bx_0)-\inf_{\bx} F(\bx)}\right)^{2/3}~,
\]
rounded to an integer.
Then,
\[
\E\big[\|\nabla F(\bar{\bx}_J)\|_\delta\big]
= \mathcal{O}\left(G^\frac{2}{3}\,\left(\frac{F(\bx_0)-\inf_{\bx} F(\bx)}{T\delta}\right)^{1/3}\right)~.
\]
\end{theorem}
\begin{proof}
We first prove the key identity linking change in function value to an \emph{expected} gradient.
By Definition~\ref{def:well-behaved}, we have
\begin{align*}
F(\bx_{t})-F(\bx_{t-1})
&= \int_0^1 \! \langle\nabla F(\bx_{t-1}+a(\bx_{t}-\bx_{t-1})),\bx_{t}-\bx_{t-1}\rangle\,\mathrm{d}a\\
&= \int_0^1 \! \langle\nabla F(\bx_{t-1}-a \bm_{t}),-\bm_{t}\rangle\,\mathrm{d}a\\
&=-\left\langle \E_{s_t}\big[\nabla F(\bx_{t-1}-s_t \bm_{t})\big],\bm_{t}\right\rangle
=-\left\langle \E_{s_t}\big[\nabla F(\bx'_{t})\big],\bm_{t}\right\rangle,
\end{align*}
where the second-to-last equality is due to the fact that $s_t$ is uniform on $[0,1]$.

Taking expectation over all randomness up to time $t$ (including $s_t$) and for any $\bu_j \in \mathcal{V}$, we have
\begin{align}
\E\left[F(\bx_t)-F(\bx_{t-1})\right]
&= -\E\left[\langle\nabla F(\bx'_t),\bm_t\rangle\right] \nonumber \\
&= \E\left[\langle-\bg_t,\bm_t-\bu_j\rangle\right]
+\E\left[\langle\bg_t-\nabla F(\bx'_t),\bm_t\rangle\right]
-\E\left[\langle\bg_t,\bu_j\rangle\right] \nonumber \\
&= \E\left[\langle-\bg_t,\bm_t-\bu_j\rangle\right]
-\E\left[\langle\bg_t,\bu_j\rangle\right] \nonumber\\
&= \E\left[\langle-\bg_t,\bm_t-\bu_j\rangle\right]
+\E\left[\langle-\nabla F(\bx'_t),\bu_j\rangle\right] \nonumber\\
&\quad +\E\left[\langle\nabla F(\bx'_t)-\bg_t,\bu_j\rangle\right],
\label{eq:ftc-step}
\end{align}
where in the second-to-last equality we used that $\bm_t$ is determined before querying $\bg_t$, together with the unbiasedness condition $\E[\bg_t \mid \mathscr{F}_t] = \nabla F(\bx'_t)$.

We now analyze one cycle, and then sum over cycles.
Fix a cycle $j$ and let its time indices be $t=(j-1)K+1, \dots, jK$.
Summing \eqref{eq:ftc-step} over $t=(j-1)K+1, \dots, jK$, we obtain
\begin{align}
\E\left[F(\bx_{jK})-F(\bx_{(j-1)K})\right]
&= \E\left[\sum_{t=(j-1)K+1}^{jK}\langle -\bg_t,\bm_t-\bu_j\rangle\right] \nonumber\\
&\quad - \E\left[\sum_{t=(j-1)K+1}^{jK} \langle\nabla F(\bx'_t),\bu_j\rangle\right] \nonumber \\
&\quad + \E\left[\sum_{t=(j-1)K+1}^{jK} \langle\nabla F(\bx'_t)-\bg_t,\bu_j\rangle\right]~. \label{eq:cycle-ftc}
\end{align}
We now focus on each term on the r.h.s. of this equality.

For the first term on the r.h.s. of~\eqref{eq:cycle-ftc}, for any $\bu_j\in\mathcal{V}$, from the regret guarantee of projected \ac{OGD} with learning rate $\eta=D/(G\sqrt{K})$ and initial point equal to $\boldsymbol{0}$, we have
\begin{align*}
\sum_{t=(j-1)K+1}^{jK}\langle -\bg_t,\bm_t-\bu_j\rangle
&\leq \frac{\|\bu_j\|_2^2}{2\eta} + \frac{\eta}{2} \sum_{t=(j-1)K+1}^{jK} \|\bg_t\|_2^2
\leq \frac{D^2}{2\eta} + \frac{\eta}{2} \sum_{t=(j-1)K+1}^{jK} \|\bg_t\|_2^2~.
\end{align*}
Since this inequality holds for every $\bu_j$, we can use a hindsight choice depending on the gradients in cycle $j$.
Moreover, taking expectation and using the definition of $\eta$, we have
\begin{equation}
\E\left[\sum_{t=(j-1)K+1}^{jK}\langle -\bg_t,\bm_t-\bu_j\rangle\right]
\leq \frac{D^2}{2\eta} + \frac{\eta}{2} K G^2
= D G \sqrt{K}~. \label{eq:proof_non_convex_conversion_eq1}
\end{equation}

Now, we choose $\bu_j$ as
\[
\bu_j
:=
D\,\frac{\sum_{t=(j-1)K+1}^{jK}\nabla F(\bx'_t)}{\left\|\sum_{t=(j-1)K+1}^{jK}\nabla F(\bx'_t)\right\|_2}
\quad \text{(and } \bu_j=\boldsymbol{0} \text{ if the numerator is } \boldsymbol{0}\text{)}~.
\]

Our choice of $\bu_j$ guarantees $\|\bu_j\|_2\le D$, and the second term of the r.h.s. of \eqref{eq:cycle-ftc} becomes
\[
\E\left[\sum_{t=(j-1)K+1}^{jK}\langle\nabla F(\bx'_t),\bu_j\rangle\right]
= D\E\left[\left\|\sum_{t=(j-1)K+1}^{jK}\nabla F(\bx'_t)\right\|_2\right]~.
\]

Now, we upper bound the last term in \eqref{eq:cycle-ftc}. First, observe that
\[
\E[\|\bg_t-\nabla F(\bx'_t)\|_2^2|\mathscr{F}_t]
= \E[\|\bg_t\|_2^2|\mathscr{F}_t]-\|\nabla F(\bx'_t)\|_2^2
\leq G^2,
\]
and
\[
\E[\langle\bg_t-\nabla F(\bx'_t),\bg_n-\nabla F(\bx'_n)\rangle |\mathscr{F}_t]=0, \quad \forall n< t~.
\]
Hence, we have
\begin{align*}
\E\left[\sum_{t=(j-1)K+1}^{jK} \langle\nabla F(\bx'_t)-\bg_t,\bu_j\rangle\right]
&\leq D \E\left[\left\|\sum_{t=(j-1)K+1}^{jK} (\bg_t-\nabla F(\bx'_t))\right\|_2\right] \\
&\leq D\sqrt{\E\left[\left\|\sum_{t=(j-1)K+1}^{jK} (\bg_t-\nabla F(\bx'_t))\right\|^2_2\right]}\\
&= D\sqrt{\E\left[\sum_{t=(j-1)K+1}^{jK} \|\bg_t-\nabla F(\bx'_t)\|^2_2\right]}\\
&\leq G D\sqrt{K},
\end{align*}
where we used the Cauchy--Schwarz inequality\index{inequality!Cauchy--Schwarz} and Jensen's inequality\index{inequality!Jensen's} (Theorem~\ref{thm:jensen}).

Putting everything together, we have
\[
D \, \E\left[\left\|\sum_{t=(j-1)K+1}^{jK}\nabla F(\bx'_t)\right\|_2\right]
\le 2DG\sqrt{K} + \E\left[F(\bx_{(j-1)K})-F(\bx_{jK})\right]~.
\]
Summing over $j=1,\dots,T/K$ telescopes the function values:
\begin{align*}
D\sum_{j=1}^{T/K} \E\left[\left\|\sum_{t=(j-1)K+1}^{jK}\nabla F(\bx'_t)\right\|_2\right]
&\leq 2\frac{T}{K}\,DG\sqrt{K} + F(\bx_0)-\E[F(\bx_T)]\\
&\leq 2\frac{T}{K}\,DG\sqrt{K} + F(\bx_0)-\inf_{\bx} F(\bx)~.
\end{align*}
Dividing by $DT$ gives the stated bound.

Now observe that within each cycle we have $\|\bx_t-\bx_{t-1}\|_2=\|\bm_t\|_2\le D$.
Hence all points $\bx'_t$ lie in the convex hull\index{convex hull} of $\{\bx_{(j-1)K},\bx_{(j-1)K+1},\dots,\bx_{jK}\}$ whose diameter is at most $KD$ (triangle inequality over at most $K$ steps).
Since $\bar{\bx}_j=\frac1K\sum_{t=(j-1)K+1}^{jK}\bx'_t$, every $\bx'_t$ lies at a distance at most the diameter of the set from $\bar{\bx}_j$, i.e., $KD$.
Therefore, if $D=\delta/K$, then for all $t$ in cycle $j$ we have $\bx'_t \in \mathcal{B}(\bar{\bx}_j,\delta)$.
By Definition~\ref{def:delta-grad-norm},
\[
\|\nabla F(\bar{\bx}_j)\|_\delta
\le \left\|\frac{1}{K}\sum_{t=(j-1)K+1}^{jK}\nabla F(\bx'_t)\right\|_2~.
\]
Sampling $J$ uniformly over cycles and taking expectation yields the stated barycentric $(\delta,\epsilon)$-stationarity bound.
Choosing $K$ to balance the two terms gives the stated rate.
\end{proof}

It is interesting to compute the update $\bm_t$ in Algorithm~\ref{alg:online-to-nonconvex} when $\mathscr{A}$ is projected OGD:
\[
\bm_{t+1} = \Pi_{\mathcal{V}}(\bm_{t} + \eta \bg_t)~.
\]
This is reminiscent of the SGD update with momentum and clipping, which are common heuristics for optimizing non-convex objectives in deep learning. Yet, here the update comes naturally from the theory.

We can also prove that this bound is optimal. To see this, we consider smooth functions and the following lemma.
\begin{lemma}
\label{lemma:smooth_non_convex_reduction}
Suppose that $F:\R^d \to \R$ is $H$-smooth\index{function!smooth} and $\bx$ satisfies $\|\nabla F (\bx)\|_\delta \leq \epsilon$. Then, $\|\nabla F(\bx)\|_2 \leq \epsilon + H\delta$.
\end{lemma}
\begin{proof}
Using the assumption that $\|\nabla F(\bx)\|_\delta\leq \epsilon$, for any $p>0$ there exists $Q\in\mathcal{Q}(\bx,\delta)$ such that $\left\|\E_{\by\sim Q}[\nabla F(\by)]\right\|_2 \leq \epsilon + p$.
By definition of $\mathcal{Q}(\bx,\delta)$, we have $\text{supp}(Q)\subseteq \mathcal{B}(\bx,\delta)$. So, by $H$-smoothness, for every $\by\in\text{supp}(Q)$, we have $\|\nabla F(\by)-\nabla F(\bx)\|_2 \leq H\|\by-\bx\|_2 \leq H\delta$.
Hence,
\begin{align*}
\epsilon + p
& \geq \left\|\E_{\by\sim Q}[\nabla F(\by)]\right\|_2
= \left\|\nabla F(\bx) + \E_{\by\sim Q}\big[\nabla F(\by)-\nabla F(\bx)\big]\right\|_2 \\
&\geq \|\nabla F(\bx)\|_2 - \left\|\E_{\by\sim Q}\big[\nabla F(\by)-\nabla F(\bx)\big]\right\|_2 \\
&\geq \|\nabla F(\bx)\|_2 - \E_{\by\sim Q}\left[\|\nabla F(\by)-\nabla F(\bx)\|_2\right]
\geq \|\nabla F(\bx)\|_2 - H\delta~.
\end{align*}
Therefore, $\|\nabla F(\bx)\|_2 \leq \epsilon + H\delta + p$. Since this holds for every $p>0$, we conclude that $\|\nabla F(\bx)\|_2 \leq \epsilon + H\delta$.
\end{proof}

Now, recall that Theorem~\ref{thm:non_convex_conversion} shows that we can find a $(\delta, \epsilon)$ barycentric stationary point in $\mathcal{O}(\epsilon^{-3} \delta^{-1})$ iterations. Thus, Lemma~\ref{lemma:smooth_non_convex_reduction} implies that by setting $\delta = \epsilon/H$, we can find a $2\epsilon$-stationary point of an $H$-smooth\index{function!smooth} objective $F$ in $\mathcal{O}(\epsilon^{-4})$ iterations, which matches the optimal guarantee of standard SGD.
Hence, the $\mathcal{O}(\epsilon^{-3}\delta^{-1})$ rate in Theorem~\ref{thm:non_convex_conversion} is optimal for all $\epsilon$ of the order of $\delta H$.

\section{From Online Learning to Time-Uniform Concentration Inequalities}

In this section, we will use the existence of a universal portfolio algorithm to derive time-uniform concentration inequalities.

For a sequence of random variables $Z_1, Z_2, \dots$, assume that $\E[Z_t| Z_1, \dots, Z_{t-1}]=\mu$ for all $t$. Also, assume that $Z_t$ takes values in $[0,1]$. We want to estimate $\mu$, giving confidence intervals $[l_t, u_t]$ that are valid for any $t$ with probability at least $1-\delta$, that is, $\Pr\{\exists t : \mu \notin [l_t, u_t]\} \leq \delta$. This kind of confidence intervals are usually called \emph{confidence sequences}\index{confidence sequence}.
The classic approach is to estimate $\mu$ with $\hat{\mu}_t = \frac{1}{t} \sum_{i=1}^t Z_i$, then invoke a concentration inequality that holds uniformly over time to obtain the confidence intervals. However, most of the well-known concentration inequalities produce \emph{vacuous} confidence intervals when $t$ is small, that is, $u_t-l_t>1$ for all $t$ smaller than some constant.

Here, we will see how a portfolio algorithm immediately gives rise to non-vacuous time-uniform confidence intervals, even with a single sample!

The idea is the following: we will construct a continuous coin-betting game on a fair coin from the problem of estimating the unknown mean. Then, we will use the following theorem, which says that the probability of making a large amount of money at any moment in time is small.
\begin{theorem}[Ville's inequality]
\index{inequality!Ville's|textbf}
Let $Y_0, Y_1, \dots,$ be a non-negative supermartingale\index{supermartingale} and $\delta \in (0,1]$. Then, we have
\[
\Pr\left\{\max_t \ Y_t \geq \frac{1}{\delta}\right\}\leq \E[Y_0]\delta~.
\]
\end{theorem}

\noindent\textbf{Warm-up: From \ac{KT} to a concentration inequality.}
As a warm-up example, consider the \ac{KT} algorithm that bets $\beta_t \Wealth_{t-1}$ in round $t$ on the outcome of the continuous coin $c_t=Z_t-\mu$. Let $\mathscr{F}_t:=\sigma(Z_1,\dots,Z_t)$. Then, $(\Wealth_t)_{t\ge0}$ is a martingale\index{martingale}, and hence also a supermartingale\index{supermartingale}, with respect to the filtration $(\mathscr{F}_t)_{t\ge0}$.
As in Example~\ref{example:wealth_martingale}, we have
\begin{align*}
\E[\Wealth_t\mid \mathscr{F}_{t-1}]
&= \E[\Wealth_{t-1}(1+\beta_t c_t)\mid \mathscr{F}_{t-1}] \\
&= \Wealth_{t-1}\E[1+\beta_t (Z_t-\mu)\mid \mathscr{F}_{t-1}]
= \Wealth_{t-1}~.
\end{align*}
Moreover, the wealth is non-negative because \ac{KT} guarantees a non-negative wealth on any sequence of coins.

So, starting with \$1, using Ville's inequality\index{inequality!Ville's} and Theorem~\ref{thm:kt2}, we obtain
\begin{align*}
\delta
&\geq \Pr\left\{\max_t \Wealth_t \geq \frac{1}{\delta}\right\}
= \Pr\left\{\max_t \ln \Wealth_t \geq \ln\frac{1}{\delta}\right\}\\
&\geq \Pr\left\{\max_t \frac{1}{2t}\left(\sum_{i=1}^t c_i\right)^2 - \frac{1}{2}\ln (e^2 t)\geq \ln \frac{1}{\delta}\right\}~.
\end{align*}
Equivalently, with probability at least $1-\delta$, we have uniformly over $t$ that
\[
\left|\frac{1}{t}\sum_{i=1}^t Z_i-\mu\right|
= \frac{1}{t}\left|\sum_{i=1}^t c_i\right|
\leq \sqrt{\frac{2}{t}\ln \frac{\sqrt{e^2 t}}{\delta}}~.
\]
Observe that $Y_t=\sum_{i=1}^t (Z_i-\mu)$ is a martingale\index{martingale} with respect to the filtration $\mathscr{F}_t:=\sigma(Z_1,\dots,Z_t)$.
So, let's compare this concentration to Hoeffding--Azuma inequality\index{inequality!Hoeffding--Azuma} in Theorem~\ref{thm:azuma}.
We almost get the same thing, but here we have an additional $\ln\sqrt{t}$ term. However, this concentration is uniform over time, which justifies the additional term in the logarithm. To summarize, the regret/reward guarantee of \ac{KT} implies a concentration inequality.

\noindent\textbf{From Universal Portfolio to a concentration inequality.}
\index{Universal Portfolio algorithm|(}
Let's now improve this reasoning using a universal portfolio algorithm to bet on the same outcomes.
Similarly to Section~\ref{sec:port_to_cb}, we will transform the random variables $Z_t$ to a sequence of two market gains. Set $w_{t,1}=1+\frac{Z_t-\mu}{\mu}$ and $w_{t,2} = 1 - \frac{Z_t-\mu}{1-\mu}$. Given that $Z_t - \mu \in [-\mu, 1-\mu]$, we have that $w_{t,1}$ and $w_{t,2}$ are non-negative. As before, we have that the wealth will be a martingale\index{martingale}.
So, we have
\[
\max_{\bu \in \Delta} \ \sum_{i=1}^t \ln \langle \bw_i,\bu\rangle
= \max_{\beta \in [-\frac{1}{1-\mu},\frac{1}{\mu}]} \ \sum_{i=1}^t \ln\left(1+(Z_i-\mu)\beta\right)~.
\]
Reasoning as above, we have
\[
\Pr\left\{\max_t \max_{\beta \in [-\frac{1}{1-\mu},\frac{1}{\mu}]} \ \sum_{i=1}^t \ln\left(1+(Z_i-\mu)\beta\right) - \Regret_t\geq \ln \frac{1}{\delta}\right\}
\leq \delta,
\]
where $\Regret_t$ is the regret of the portfolio algorithm with 2 stocks.
If we use a portfolio algorithm with logarithmic regret, given that $[-1,1]\subset [-\frac{1}{1-\mu},\frac{1}{\mu}]$, this expression always gives a bigger wealth than the \ac{KT} strategy we have just seen. In turn, a bigger wealth corresponds to a tighter concentration.

Let's be more precise by instantiating the above idea with an $F$-weighted portfolio algorithm from Chapter~\ref{ch:universal_portfolio}, with $F$ equal to the Dirichlet(1/2,1/2) distribution.
\begin{theorem}
\label{thm:portfolio_to_ci}
Let $\delta \in (0,1)$. Assume $Z_1, Z_2, \dots$ are a sequence of random variables such that for each $i$ we have $0 \leq Z_i \leq 1$ and $\E[Z_i| Z_1, \dots, Z_{i-1}]=\mu$ almost surely and $\mu \in (0,1)$.
Let $G_t(\beta,\mu) := \sum_{i=1}^t \ln\left(1+ \beta (Z_i-\mu)\right)$, $R_t := \ln\frac{\sqrt{\pi}\Gamma(t+1)}{\Gamma(t + \frac 1 2)}$, and
\[
\mathcal{S}_t:=\left\{m \in (0,1): \max_{\beta \in [-\frac{1}{1-m},\frac{1}{m}]} G_t(\beta,m) - R_t \leq  \ln\frac{1}{\delta}\right\}~.
\]
Then, with probability at least $1-\delta$ and uniformly over $t$, after observing $t$ random variables we have $\mu \in \cap_{i=1}^t \mathcal{S}_i$.

Moreover, we have that $\mathcal{S}_t$ is an interval $[l_t,u_t] \subseteq [0,1]$ for all $t\geq 1$.
\end{theorem}
\begin{proof}
Following the reasoning above and the fact that the regret of the $F$-weighted portfolio algorithm with 2 stocks and $F=Dirichlet(1/2,1/2)$ is upper bounded by $R_t$ using Theorem~\ref{thm:universal_portfolio}, we get that $\mu \in \mathcal{S}_t$. From the fact that this holds with probability $1-\delta$ uniformly over time, we get that at time $t$ $\mu$ must be in the intersection of the sets $\mathcal{S}_1, \dots, \mathcal{S}_t$.

For the second claim, first denote by $\beta^\star(m):=\argmax_{\beta \in [-\frac{1}{1-m},\frac{1}{m}]}\ G_t(\beta,m)$ and $\hat{G}_t(m)=\max_{\beta \in [-\frac{1}{1-m},\frac{1}{m}]} G_t(\beta, m)$.

Observe that the derivative of $G_t$ with respect to its first argument is
\[
G'_t(\beta,m)= \sum_{i=1}^t \frac{Z_i-m}{1+\beta (Z_i -m)}~.
\]
So, we have that $G'_t(0,\hat{\mu}_t)=0$. Given that $G_t(\beta,m)$ is concave in $\beta$, then $G_t(\beta,\hat{\mu}_t)$ has a maximum with respect to the first argument in $\beta=0$ and the value of the function is 0.

For $m'>\hat{\mu}_t$, $G'_t(0,m')<0$.
Since $G_t(\beta,m')$ is concave in $\beta$, we have $\beta^\star(m')<0$.
In the same way, for $m'<\hat{\mu}_t$ we have $\beta^\star(m')>0$.

Now, let us start with $m'>\hat{\mu}_t$ and prove that $\hat{G}_t(m)$ is nondecreasing; the other side is analogous. Consider $m_1>m_2>\hat{\mu}_t$.
Given that $\beta^\star(m_2)<0$, we have
\[
\hat{G}_t(m_2)
= G_t(\beta^\star(m_2),m_2)
\leq G_t(\beta^\star(m_2),m_1)
\leq G_t(\beta^\star(m_1),m_1)
= \hat{G}_t(m_1),
\]
where the first inequality holds because $G_t(\beta,m)$ is nondecreasing in $m$ when $\beta<0$ and the second inequality holds because the negative part of the interval $[-\frac{1}{1-m_1},\frac{1}{m_1}]$ contains the negative part of the interval $[-\frac{1}{1-m_2},\frac{1}{m_2}]$ and we know the maximum $\beta$ is negative.
Hence, $\hat{G}_t(m)$ is a quasiconvex function of $m$, and hence $\mathcal{S}_t$ is an interval.
\index{Universal Portfolio algorithm|)}
\end{proof}
\begin{remark}
Given that $\mathcal{S}_t$ is an interval, we can find $\mathcal{S}_t=[l_t, u_t]$ efficiently using the bisection algorithm.
\end{remark}

Now, we gather some more intuition on the inequality of Theorem~\ref{thm:portfolio_to_ci}.
It may not seem obvious if the maximum log wealth is a better candidate for constructing a confidence sequence than the standard ones like Bernoulli \ac{KL}-divergence-based bound~\cite[e.g.,][Theorem 10]{GarivierC11}, which works for random variables supported in $[0,1]$:
\begin{align*}
  \Pr\left\{\max_{t} \ t\cdot \KLBern(\hat{\mu}_t; \mu) - \ln f(t) \ge \ln \frac{1}{\delta}\right\} \le \delta,
\end{align*}
where $\hat{\mu}_t=\frac{1}{t} \sum_{i=1}^t Z_i$, $\KLBern(p;q):=p \ln \frac{p}{q}+(1-p)\ln\frac{1-p}{1-q}$\index{Kullback--Leibler divergence!between Bernoulli distributions}, and $f(t)$ grows polynomially in $t$ or slower.
So, in the following proposition we show that the maximum log wealth is never worse than the \ac{KL} divergence\index{Kullback--Leibler divergence}, which supports a viewpoint that the \ac{KL} divergence is a special case of the maximum log wealth and that confidence bounds constructed with the maximum wealth are never worse than those with \ac{KL} divergence, ignoring the minor difference in $\ln f(t)$.
\begin{proposition}
\label{prop:kl}
Let $Z_1,\dots,Z_t \in [0,1]$, $\hat{\mu}_t = \frac1t \sum_{i=1}^t Z_i$, and $\mu \in (0,1)$.
Then,
\[
\max_{\beta \in [-\frac{1}{1-\mu},\frac{1}{\mu}]} \ \sum_{i=1}^t \ln\left(1+ \beta (Z_i-\mu)\right)
\geq t\cdot \KLBern\left(\hat{\mu}_t;\mu\right),
\]
where we achieve the equality if $Z_1,\dots,Z_t \in \{0,1\}$ almost surely.
\end{proposition}
\begin{proof}
Using Jensen's inequality\index{inequality!Jensen's} (Theorem~\ref{thm:jensen}), we have for any $Z\in[0,1]$ that
\begin{align*}
\ln(1 + \beta(Z - \mu))
&= \ln[Z(1 + \beta(1-\mu)) + (1-Z) (1+\beta(0-\mu))]
\\&\ge Z \ln(1 + \beta(1-\mu)) + (1-Z)\ln (1 + \beta(0-\mu))~.
\end{align*}
Note that we achieve equality when $Z = 1$ or $Z = 0$.
Then, we have
\begin{align*}
&\max_{\beta \in [-\frac{1}{1-\mu},\frac{1}{\mu}]} \ \sum_{i=1}^t \ln(1 + \beta (Z_i - \mu))\\
&\quad\ge \max_{\beta \in [-\frac{1}{1-\mu},\frac{1}{\mu}]} \ \sum_i Z_i \ln(1 + \beta(1-\mu)) + (1-Z_i)\ln (1 - \beta \mu)
\\
&\quad= \max_{\beta \in [-\frac{1}{1-\mu},\frac{1}{\mu}]} \ t [ \hat{\mu}_t \ln(1 + \beta(1-\mu)) + (1-\hat{\mu}_t)\ln (1 - \beta\mu) ]~.
\end{align*}
As the r.h.s. is concave in $\beta$, it remains to maximize the r.h.s. over $\beta$.
The solution is $\beta = \frac{\hat{\mu}_t - \mu}{\mu(1-\mu)}$ with which the maximum becomes $t \cdot \KLBern(\hat{\mu}_t, \mu)$.
\end{proof}

Second, we show that the confidence intervals obtained by the numerical inversion of Theorem~\ref{thm:portfolio_to_ci} are \emph{never vacuous}.
\begin{theorem}
\label{thm:never_vacuous}
Under the assumptions of Theorem~\ref{thm:portfolio_to_ci}, we have that $u_1-l_1=1-\frac{\delta}{2}$.
\end{theorem}
\begin{proof}
With only one sample, the upper and lower bounds have a closed form.
Indeed, the argmax of $G_t(\beta,m)$ with respect to $\beta$ over $[-\frac{1}{1-m},\frac{1}{m}]$ is achieved in $\beta=\frac{1}{m}$ if $Z_1-m>0$ and in $\beta=-\frac{1}{1-m}$ for $Z_1-m<0$. This implies that
\begin{align*}
l_1 = \frac{Z_1}{\exp(R_1+ \ln \tfrac{1}{\delta})}
&& \text{and} &&
u_1 = 1-\frac{1-Z_1}{\exp(R_1+ \ln \tfrac{1}{\delta})}~.
\end{align*}
Given that $R_1=\ln 2$, subtracting the lower bound from the upper bound, we get the stated bound for any $Z_1$.
\end{proof}
Given that we can return the intersection of all intervals $[l_t, u_t]$, the above theorem implies the same bound for all $t$.

Finally, besides implying the \ac{KL} bound above, the implicit concentration in Theorem~\ref{thm:portfolio_to_ci} also implies an empirical Bernstein time-uniform concentration. It is worth stressing that the intersection of the numerically evaluated intervals $[l_i, u_i]$ from $i=1, \dots, t$ is strictly smaller than the upper bound in the following theorem, given that Theorem~\ref{thm:never_vacuous} tells us that the widths of numerically calculated intervals are always strictly smaller than 1.
\begin{theorem}
\label{thm:approx_inversion}
Under the assumptions of Theorem~\ref{thm:portfolio_to_ci},
denote by $\hat{\mu}_i = \frac{1}{i}\sum_{j=1}^i Z_j$, $V_i = \sum_{j=1}^i (Z_j - \hat{\mu}_i)^2$, $\hat{R}_i =  \ln\frac{\sqrt{\pi}\Gamma(i+1)}{\delta\Gamma(i + \frac 1 2)}$, and
\[
\epsilon_i = \frac{(4/3) i \hat{R}_i +\sqrt{16/9 i^2 \hat{R}_i^2+8 V_i \hat{R}_i(i^2-2 i \hat{R}_i)}}{2 i^2-4i \hat{R}_i}~.
\]
Then, with probability at least $1-\delta$ uniformly for all $t$ such that $t> 2 \hat{R}_t$, we have
\[
\max_{i=1,\dots,t} \ \hat{\mu}_i - \epsilon_i \leq \mu \leq \min_{i=1,\dots,t} \ \hat{\mu}_i + \epsilon_i~.
\]
\end{theorem}
For sufficiently large $t$ the deviation is roughly $\frac{(4/3) \ln(\sqrt{t}/\delta)}{t}+\frac{\sqrt{2 V_t \ln(\sqrt{t}/\delta)}}{t}$, similarly to the inequalities first proposed in \citet{AudibertMC09,MaurerP09}.
\begin{remark}
One can easily see that, in terms of the scaling with $\ln(1/\delta)$, the factor $\sqrt{2(\frac1tV_t)/t}$ is the optimal one due to the central limit theorem.
For the scaling with $t$, by changing the weight distribution $F$ it is also possible to obtain $\sqrt{\frac{2 (\frac1t V_t)\ln(\ln(V_t))+o(\ln \ln t)}{t}}$ as $t\to \infty$, which matches the law of the iterated logarithm (thus asymptotically optimal), see \citet{OrabonaJ21}.
\end{remark}

We can see the performance of this approach compared to a state-of-the-art closed-form concentration in Figures~\ref{fig:precise1}-\ref{fig:precise4}. We repeat each experiment 10 times and use the first 5 in which all the algorithms do not fail. In fact, all the algorithms/inequalities will fail with probability $\delta$, and averaging over their failing runs would result in confidence sequences with smaller widths. We set $\delta=0.05$ in all the experiments.


\begin{figure}[t]
\centering
\begin{tikzpicture}
\begin{axis}[
    width=6cm,
    xlabel={Number of samples}, ylabel={Confidence sets},
    grid=both,
    xmode=log,
    xmin=1,
    xmax=100000,
    xtick={1,10,100,1000,10000,100000},
    ytick={0, 0.2, 0.4, 0.6, 0.8, 1},
    ymin=0,
    ymax=1,
    legend style={font=\tiny},
    legend image post style={
        scale=0.5,
    },
    legend cell align={left},
]
\addplot[thick] table[col sep=comma, x=t, y=pu] {code_for_figs/Bernoulli_0.1.csv};
\addplot[gray, thick] table[col sep=comma, x=t, y=cu] {code_for_figs/Bernoulli_0.1.csv};
\addlegendentry{Theorem~\ref{thm:portfolio_to_ci}}
\addlegendentry{Howard et al. (2021), equation (4.2)}
\addplot[thick] table[col sep=comma, x=t, y=pl] {code_for_figs/Bernoulli_0.1.csv};
\addplot[gray, thick] table[col sep=comma, x=t, y=cl] {code_for_figs/Bernoulli_0.1.csv};
\end{axis}
\end{tikzpicture}
\hspace{1cm}
\begin{tikzpicture}
\begin{axis}[
    width=6cm,
    xlabel={Number of samples}, ylabel={Confidence sets},
    grid=both,
    xmode=log,
    xmin=1,
    xmax=100000,
    xtick={1,10,100,1000,10000,100000},
    ytick={0, 0.2, 0.4, 0.6, 0.8, 1},
    ymin=0,
    ymax=1,
    legend style={font=\tiny},
    legend image post style={
        scale=0.5,
    },
    legend cell align={left},
]
\addplot[thick] table[col sep=comma, x=t, y=pw] {code_for_figs/Bernoulli_0.1_width.csv};
\addplot[gray, thick] table[col sep=comma, x=t, y=cw] {code_for_figs/Bernoulli_0.1_width.csv};
\addlegendentry{Theorem~\ref{thm:portfolio_to_ci}}
\addlegendentry{Howard et al. (2021), equation (4.2)}
\end{axis}
\end{tikzpicture}
\caption{Confidence intervals for Theorem~\ref{thm:portfolio_to_ci} and \citet{HowardRMS21} for a sequence of random variables drawn i.i.d. from Bernoulli(0.1).}
\label{fig:precise1}
\commentAlt{Figure~\ref{fig:precise1}. Two log-scale plots for Bernoulli(0.1) data comparing Theorem~\ref{thm:portfolio_to_ci} confidence intervals with the Howard et al. (2021) intervals. The left panel shows upper and lower confidence bounds over sample size; the right panel shows confidence-set widths.}
\end{figure}


\begin{figure}[t]
\centering
\begin{tikzpicture}
\begin{axis}[
    width=6cm,
    xlabel={Number of samples}, ylabel={Confidence sets},
    grid=both,
    xmode=log,
    xmin=1,
    xmax=100000,
    xtick={1,10,100,1000,10000,100000},
    ytick={0, 0.2, 0.4, 0.6, 0.8, 1},
    ymin=0,
    ymax=1,
    legend style={font=\tiny},
    legend image post style={
        scale=0.5,
    },
    legend cell align={left},
]
\addplot[thick] table[col sep=comma, x=t, y=pu] {code_for_figs/Bernoulli_0.5.csv};
\addlegendentry{Theorem~\ref{thm:portfolio_to_ci}}
\addlegendentry{Howard et al. (2021), equation (4.2)}
\addplot[gray, thick] table[col sep=comma, x=t, y=cu] {code_for_figs/Bernoulli_0.5.csv};
\addplot[thick] table[col sep=comma, x=t, y=pl] {code_for_figs/Bernoulli_0.5.csv};
\addplot[gray, thick] table[col sep=comma, x=t, y=cl] {code_for_figs/Bernoulli_0.5.csv};
\end{axis}
\end{tikzpicture}
\hspace{1cm}
\begin{tikzpicture}
\begin{axis}[
    width=6cm,
    xlabel={Number of samples}, ylabel={Confidence sets},
    grid=both,
    xmode=log,
    xmin=1,
    xmax=100000,
    xtick={1,10,100,1000,10000,100000},
    ytick={0, 0.2, 0.4, 0.6, 0.8, 1},
    ymin=0,
    ymax=1,
    legend style={font=\tiny},
    legend image post style={
        scale=0.5,
    },
    legend cell align={left},
]
\addplot[thick] table[col sep=comma, x=t, y=pw] {code_for_figs/Bernoulli_0.5_width.csv};
\addplot[gray, thick] table[col sep=comma, x=t, y=cw] {code_for_figs/Bernoulli_0.5_width.csv};
\addlegendentry{Theorem~\ref{thm:portfolio_to_ci}}
\addlegendentry{Howard et al. (2021), equation (4.2)}
\end{axis}
\end{tikzpicture}
\caption{Confidence intervals for Theorem~\ref{thm:portfolio_to_ci} and \citet{HowardRMS21} for a sequence of random variables drawn i.i.d. from Bernoulli(0.5).}
\label{fig:precise2}
\commentAlt{Figure~\ref{fig:precise2}. Two log-scale plots for Bernoulli(0.5) data comparing Theorem~\ref{thm:portfolio_to_ci} confidence intervals with the Howard et al. (2021) intervals. The left panel shows upper and lower confidence bounds over sample size; the right panel shows confidence-set widths.}
\end{figure}

\begin{figure}[t]
\centering
\begin{tikzpicture}
\begin{axis}[
    width=6cm,
    xlabel={Number of samples}, ylabel={Confidence sets},
    grid=both,
    xmode=log,
    xmin=1,
    xmax=100000,
    xtick={1,10,100,1000,10000,100000},
    ytick={0, 0.2, 0.4, 0.6, 0.8, 1},
    ymin=0,
    ymax=1,
    legend style={font=\tiny},
    legend image post style={
        scale=0.5,
    },
    legend cell align={left},
]
\addplot[thick] table[col sep=comma, x=t, y=pu] {code_for_figs/Beta_1,_1.csv};
\addplot[gray, thick] table[col sep=comma, x=t, y=cu] {code_for_figs/Beta_1,_1.csv};
\addlegendentry{Theorem~\ref{thm:portfolio_to_ci}}
\addlegendentry{Howard et al. (2021), equation (4.2)}
\addplot[thick] table[col sep=comma, x=t, y=pl] {code_for_figs/Beta_1,_1.csv};
\addplot[gray, thick] table[col sep=comma, x=t, y=cl] {code_for_figs/Beta_1,_1.csv};
\end{axis}
\end{tikzpicture}
\hspace{1cm}
\begin{tikzpicture}
\begin{axis}[
    width=6cm,
    xlabel={Number of samples}, ylabel={Confidence sets},
    grid=both,
    xmode=log,
    xmin=1,
    xmax=100000,
    xtick={1,10,100,1000,10000,100000},
    ytick={0, 0.2, 0.4, 0.6, 0.8, 1},
    ymin=0,
    ymax=1,
    legend style={font=\tiny},
    legend image post style={
        scale=0.5,
    },
    legend cell align={left},
]
\addplot[thick] table[col sep=comma, x=t, y=pw] {code_for_figs/Beta_1,_1_width.csv};
\addplot[gray, thick] table[col sep=comma, x=t, y=cw] {code_for_figs/Beta_1,_1_width.csv};
\addlegendentry{Theorem~\ref{thm:portfolio_to_ci}}
\addlegendentry{Howard et al. (2021), equation (4.2)}
\end{axis}
\end{tikzpicture}
\caption{Confidence intervals for Theorem~\ref{thm:portfolio_to_ci} and \citet{HowardRMS21} for a sequence of random variables drawn i.i.d. from Beta(1,1).}
\label{fig:precise3}
\commentAlt{Figure~\ref{fig:precise3}. Two log-scale plots for Beta(1,1) data comparing Theorem~\ref{thm:portfolio_to_ci} confidence intervals with the Howard et al. (2021) intervals. The left panel shows upper and lower confidence bounds over sample size; the right panel shows confidence-set widths.}
\end{figure}


\begin{figure}[h]
\centering
\begin{tikzpicture}
\begin{axis}[
    width=6cm,
    xlabel={Number of samples}, ylabel={Confidence sets},
    grid=both,
    xmode=log,
    xmin=1,
    xmax=100000,
    xtick={1,10,100,1000,10000,100000},
    ytick={0, 0.2, 0.4, 0.6, 0.8, 1},
    ymin=0,
    ymax=1,
    legend style={font=\tiny},
    legend image post style={
        scale=0.5,
    },
    legend cell align={left},
]
\addplot[thick] table[col sep=comma, x=t, y=pu] {code_for_figs/Beta_10,_30.csv};
\addplot[gray, thick] table[col sep=comma, x=t, y=cu] {code_for_figs/Beta_10,_30.csv};
\addlegendentry{Theorem~\ref{thm:portfolio_to_ci}}
\addlegendentry{Howard et al. (2021), equation (4.2)}
\addplot[thick] table[col sep=comma, x=t, y=pl] {code_for_figs/Beta_10,_30.csv};
\addplot[gray, thick] table[col sep=comma, x=t, y=cl] {code_for_figs/Beta_10,_30.csv};
\end{axis}
\end{tikzpicture}
\hspace{1cm}
\begin{tikzpicture}
\begin{axis}[
    width=6cm,
    xlabel={Number of samples}, ylabel={Confidence sets},
    grid=both,
    xmode=log,
    xmin=1,
    xmax=100000,
    xtick={1,10,100,1000,10000,100000},
    ytick={0, 0.2, 0.4, 0.6, 0.8, 1},
    ymin=0,
    ymax=1,
    legend style={font=\tiny},
    legend image post style={
        scale=0.5,
    },
    legend cell align={left},
]
\addplot[thick] table[col sep=comma, x=t, y=pw] {code_for_figs/Beta_10,_30_width.csv};
\addplot[gray, thick] table[col sep=comma, x=t, y=cw] {code_for_figs/Beta_10,_30_width.csv};
\addlegendentry{Theorem~\ref{thm:portfolio_to_ci}}
\addlegendentry{Howard et al. (2021), equation (4.2)}
\end{axis}
\end{tikzpicture}
\caption{Confidence intervals for Theorem~\ref{thm:portfolio_to_ci} and \citet{HowardRMS21} for a sequence of random variables drawn i.i.d. from Beta(10,30).}
\label{fig:precise4}
\commentAlt{Figure~\ref{fig:precise4}. Two log-scale plots for Beta(10,30) data comparing Theorem~\ref{thm:portfolio_to_ci} confidence intervals with the Howard et al. (2021) intervals. The left panel shows upper and lower confidence bounds over sample size; the right panel shows confidence-set widths.}
\end{figure}


To prove Theorem~\ref{thm:approx_inversion}, we first need a technical lemma.
\begin{lemma}
  \label{lemma:max_approx_wealth}
  Let $f(x)=a x + b ( \ln(1-|x|)+|x|)$, where $a \in \R$ and $b\geq 0$. Then, $\argmax_{x \in [-1,1]} \ f(x) = \frac{a}{|a|+b}$ and $\max_{x \in [-1,1]} \ f(x) = b\psi(\frac{a}{b}) \geq \frac{a^2}{(4/3)|a|+2b}$, where $\psi(x)=|x|-\ln(|x|+1)$.
\end{lemma}
\begin{proof}
If $b=0$, we have that argmax is $\sign(a)$. If $a=0$, the argmax is 0. Hence, in the following, we can assume $a$ and $b$ are nonzero.

We can rewrite the maximization problem as
\begin{align*}
\argmax_{x} \ f(x)
= b \argmax_{x} \ \frac{a}{b}x + \ln(1-|x|)+|x|~.
\end{align*}
From the optimality condition, we have that
$\frac{a}{b} - \frac{\sign(x^\star)}{1-|x^\star|}+\sign(x^\star) = 0$,
which implies $x^\star=\frac{a}{|a|+b}$.
Substituting this expression in $f$, we obtain the stated expression.
The inequality is obtained by the elementary inequality $\ln(1+x) \le x\cdot\frac{6+x}{6+4x}$ for $x\ge 0$.
\end{proof}

We can now prove Theorem~\ref{thm:approx_inversion}.
\begin{proof}[Proof of Theorem~\ref{thm:approx_inversion}]
For a given $t$, set $\epsilon_t$ equal to $\mu-\hat{\mu}_t$, so that $\epsilon_t+\hat{\mu}_t \in [0,1]$.

Consider the function $f(a) = \frac{\ln(1 + a) - a}{a^2/2}$ for $a>-1$.
Set $|x| \leq 1$ and $|\beta| < 1$, so we have $\beta x \geq -|\beta| > -1$. From the sign of the first derivative, we have that $f(a)$ is increasing. Hence, we have
\[
\ln(1 + \beta x)
= \beta x + \frac12 (\beta x)^2 f(\beta x)
\geq \beta x + \frac12 (\beta x)^2 f(-|\beta|)
= \beta x + x^2(|\beta| + \ln(1 -|\beta|))~.
\]
Define $\phi(\beta):=\ln(1-|\beta|)+|\beta|$, so that, for any $\beta \in (-1,1)$, we have
\begin{align*}
&\sum_{i=1}^t \ln(1+\beta (Z_i-\mu) )\\
&\quad = \sum_{i=1}^t \ln(1+\beta (Z_i-\hat{\mu}_t-\epsilon_t) ) \\
&\quad \geq \beta \sum_{i=1}^t (Z_i-\hat{\mu}_t -\epsilon_t) + \phi(\beta) \left(\sum_{i=1}^t (Z_i-\hat{\mu}_t)^2 + \epsilon_t^2 t - 2 \epsilon_t \sum_{i=1}^t (Z_i-\hat{\mu}_t)\right) \\
&\quad = -\epsilon_t \beta t + \phi(\beta) \left(\sum_{i=1}^t (Z_i-\hat{\mu}_t)^2 +  \epsilon_t^2 t\right) ~.
\end{align*}
Hence, we have
\begin{align*}
&\max_{\beta \in [-1,1]} \  \sum_{i=1}^t \ln(1+\beta (Z_i-\hat{\mu}_t-\epsilon_t)) \\
&\quad= \left(\sum_{i=1}^t(Z_i-\hat{\mu}_t)^2 +  \epsilon_t^2 t\right) \psi\left(\frac{|\epsilon_t| t}{\sum_{i=1}^t(Z_i-\hat{\mu}_t)^2 +  \epsilon_t^2 t }\right),
\end{align*}
where $\psi(x) = |x| - \ln(|x|+1)$ and the equality is due to Lemma~\ref{lemma:max_approx_wealth}. From the inequality in Lemma~\ref{lemma:max_approx_wealth} we also obtain
\[
\max_{\beta \in [-1,1]}\  \sum_{i=1}^t \ln(1+\beta (Z_i-\hat{\mu}_t-\epsilon_t))
\geq \frac{\epsilon_t^2 t^2}{(4/3) |\epsilon_t| t + 2 \sum_{i=1}^t(Z_i-\hat{\mu}_t)^2 +  2 \epsilon_t^2 t}~.
\]

Now, note that for any $\mu \in (0,1)$ the interval $[-1,1]$ is contained in $[-\frac{1}{1-\mu},\frac{1}{\mu}]$.
Hence, from Theorem~\ref{thm:portfolio_to_ci}, uniformly on all $t$ with probability at least $1-\delta$, we have
\begin{align*}
\frac{\epsilon_t^2 t^2}{(4/3) |\epsilon_t| t + 2 \sum_{i=1}^t (Z_i-\hat{\mu}_t)^2 +  2 \epsilon_t^2 t}
\leq R_t + \ln \frac{1}{\delta}
= \hat{R}_t~.
\end{align*}
Assuming $\epsilon_t$ is positive and solving for it, we have the stated upper bound.
By the symmetry of the formula, the expression for negative $\epsilon_t$ has the opposite sign.
\end{proof}

\section{History Bits}

The contraction lemma (Lemma~\ref{lemma:contraction})\index{contraction lemma} appears in many places in the literature~\citep[see, e.g.,][]{LedouxT91,BartlettM01,BartlettM02,MeirZ03}.
\citet{KakadeST08} proved Theorem~\ref{thm:rad_from_regret} using strong convexity, while \citet{KakadeSST09} proved it through the proof of \ac{FTRL}. Here, I simply used a black-box conversion to the linear regret upper bound for any online learning algorithm.
Massart's lemma\index{Massart's lemma} was originally proved by \citet{Massart00}.

\index{PAC-Bayes generalization bound|(}
PAC-Bayes bounds were first proposed by \citet{McAllester98} and, as explained by \citet{vanErven14}, they can be seen as a continuous generalization of the union bound\index{union bound}. There is now a vast literature on this subject, and we refer the reader to the recent review by \citet{Alquier24} for an introduction. Recently, PAC-Bayes bounds gained popularity because they can yield non-vacuous generalization bounds for deep neural networks. Yet, one should not confuse a certificate of generalization with an ``explanation of generalization'' in deep neural networks~\citep[see, e.g.,][]{Picard-WeibelCMG25}.
The reduction from online learning to PAC-Bayes bounds, as well as the idea of using parameter-free algorithms, is from \citet{LugosiN23}. Despite what is claimed in \citet{LugosiN23}, the PAC-Bayes bound in Corollary~\ref{cor:online_to_pacbayes} is not new; it has been obtained by \citet{KakadeST08}. Such a bound shaves off a $\ln T$ term under the square root compared to previous PAC-Bayes guarantees. \citet{LugosiN23} also present many more results based on the same reduction.
Independently, \citet{JangJKO23} proposed another reduction from coin betting to PAC-Bayes. While less general than the one presented here, it allows one to prove tighter results by changing the surrogate loss $\ell_t$ to a non-linear one.
Later, \citet{KuzborskijJWJO24} used a similar approach to show the first PAC-Bayes bound with a better-than-$\KL$ divergence\index{Kullback--Leibler divergence}.
\index{PAC-Bayes generalization bound|)}

\citet{ZhangLJSJ20} proposed to use the Goldstein stationarity condition~\citep{Goldstein77} for non-convex non-smooth objectives. They also proved a suboptimal bound of $\mathcal{O}(\epsilon^{-4}\delta^{-1})$. The barycentric definition, Theorem~\ref{thm:non_convex_conversion}, and Lemma~\ref{lemma:smooth_non_convex_reduction} are from \citet{CutkoskyMO23}. Note that I changed Definition~\ref{def:delta-grad-norm} slightly because Algorithm~\ref{alg:online-to-nonconvex} may return repeated points, which would not be allowed by the original definition in \citet{CutkoskyMO23}.
The barycentric part of the definition is not strictly necessary for the results I presented, yet it is necessary for the additional results in \citet{CutkoskyMO23} on functions whose Hessian is Lipschitz. A nonsmooth generalization of the identity used in the definition of well-behaved functions is given by \citet{BolteP21}, who replace gradients with conservative set-valued fields satisfying an analogous integration identity along absolutely continuous paths.
\citet{ZhangC24} improved the reduction in Algorithm~\ref{alg:online-to-nonconvex} using an exponentially distributed $s_t$, which allows one to query the gradient at $\bx_t$ rather than at the intermediate point $\bx'_t$, and removes the feasible set $\mathcal{V}$.
\citet{JordanKLSZ23} proved that the randomization is essential in this class of problems to achieve dimension-free rates.

The results and proofs from Section~\ref{sec:port_to_cb} are taken from \citet{OrabonaJ21}.
Note that the computational complexity to calculate $\mathcal{S}_t$ in Theorem~\ref{thm:portfolio_to_ci} is $\mathcal{O}(t^2)$. Looser confidence intervals from portfolio algorithms but with complexity $\mathcal{O}(t)$ have been proposed in \citet{OrabonaJ21} and \citet{RyuB22}.
Recently, \citet{VoracekO25} obtained state-of-the-art fixed-time confidence intervals as well, using a betting strategy that takes explicitly into account the temporal horizon and the threshold the algorithm's wealth must pass to test each value of the unknown mean.

Ville's inequality\index{inequality!Ville's} is proved in \citet[Page 84]{Ville39}, where its meaning is exactly associated with betting. In fact, \citet{Ville39} also introduced the concept of martingale\index{martingale} as the wealth of a betting strategy on a fair coin.
The ideas of Ville were later used to design ideal tests for randomness of infinite sequences~\citep{Schnorr71,Levin76,Gacs05}, ``ideal'' because none of these tests is computable.

There are two papers that, at the same time and independently, explicitly link hypothesis testing on a finite sequence of outcomes to betting. One is \citet[Example 3]{Cover74}, where an optimal betting strategy is defined in terms of the null hypothesis.
The other one is \citet{RobbinsS74} that constructs confidence sequences from novel betting schemes and explicitly recognizes the connection between the sequential probability ratio test~\citep{Wald45} and betting. Very surprisingly, \citet[Section 9]{RobbinsS74} seem to have proposed and analyzed the famous strategy of \citet{KrichevskyT81} 7 years before them.
However, while in the information theory literature these ideas flourished and gave birth to results on coding, compression, minimum description length, and gambling, they seem to have disappeared from the statistics community for 30 years. In this view, it is remarkable that \citet{Cover74} was submitted to Annals of Statistics and probably rejected, as it can be inferred from the footnote on its first page.
In fact, in the statistics literature, gambling strategies reappeared only between the 1990s and 2000s, thanks to the book and papers by Shafer and Vovk. In particular, \citet{Vovk93,ShaferV05} aimed to found probabilities on a game-theoretic ground through betting schemes.
However, the foundational approach in \citet{ShaferV05} also means that all the betting strategies they propose do not have closed-form expressions, and they cannot be easily implemented. In addition, they do not use the regret analysis for the study of the proposed betting schemes.
Finally, \citet{ShaferV05} does not contain any explicit concentration inequality, while the first concentration for game-theoretic probability derived by a betting scheme is in \citet{Vovk07}, which derives a game-theoretic Hoeffding's inequality.

As far as I know, the first work to derive an explicit new concentration from the finite-time worst-case guarantee of an online gambling algorithm is in an unpublished work of mine from 2017, shared with a number of researchers in 2017--2018, and later published on arXiv~\citep{Orabona26b}.
This work was based on the seminal work of \citet{RakhlinS17}, which showed an \emph{equivalence} between the regret guarantees of online learning algorithms with linear losses and concentration inequalities. However, the proof technique in \citet{Orabona26b} is different from the one in \citet{RakhlinS17}, and it is specific to online algorithms that guarantee a non-negative exponential wealth for biased inputs. In particular, it allows us to derive time-uniform concentrations, like the law of the iterated logarithm, that are not possible with the method in \citet{RakhlinS17}.

The first paper to consider an implementable strategy for testing through betting is by \citet{Hendriks18}. Directly building on \citet{ShaferV05} and \citet{ShaferSVS11}, \citet{Hendriks18} proposes to construct testing martingales\index{martingale} and confidence sequences for bounded random variables as uniform mixtures of constant betting strategies, effectively a $F$-weighted portfolio with the uniform distribution. \citet{Hendriks18} also showed the good performance of the proposed approach empirically in a simple statistical test. \citet{Waudby-SmithR21} seem to follow the same approach as \citet{Hendriks18} but propose several heuristic betting algorithms to maximize the wealth, as well as a discrete version of the uniform mixture that appeared in \citet{Hendriks18}.
\citet{JunO19}, based on the ideas in \citet{Orabona26b}, show how to easily derive a Law of Iterated Logarithm\index{law of iterated logarithm} for sub-Gaussian random vectors in Banach spaces from the regret analysis of a one-dimensional betting algorithm and the direction/magnitude reduction in Section~\ref{sec:parameterfree-direction-magnitude}.
More recently, \citet{JangJKO23} showed how to derive tighter PAC-Bayes bounds from the regret of portfolio algorithms.

\acresetall

\backmatter
\oneappendix

\chapter{Appendix}

\section{More Convex Analysis}

%
%
%


\begin{theorem}[{Fundamental Theorem of Calculus for Extended-real-valued Convex Functions~\citep[Theorem 2.3.4, page 179]{Hiriart-UrrutyL04}}]
\label{thm:ftc}
\index{fundamental theorem of calculus for extended-real-valued convex functions|textbf}
Let $f : \R^d \to (-\infty, +\infty]$ be proper and convex. For $\bx, \by \in \interior \dom f$ and $t\in [0,1]$, define $\bx_t= t \by + (1-t) \bx$. Then, for any $\bg_t \in \partial f(\bx_t),$ we have
\[
f(\by)-f(\bx)
= \int_{0}^1 \! \langle \bg_t , \by - \bx\rangle \, \mathrm{d}t~.
\]
\end{theorem}
The above theorem hinges on the fact that the derivative of $\phi(t)=f(t \by + (1-t) \bx)$ exists except possibly for a countable set.

%

\begin{definition}
\label{def:hausdorff}
A topological space $\mathcal{X}$ is a \textbf{Hausdorff space}\index{Hausdorff space|textbf} if any two distinct points can be separated by disjoint open sets.
That is, for any $\bx,\by\in\mathcal{X}$ with $\bx\neq \by$, there exist open sets $\mathcal{U},\mathcal{V}\subseteq\mathcal{X}$ such that $\bx\in \mathcal{U}$, $\by\in \mathcal{V}$, and $\mathcal{U}\cap \mathcal{V}=\emptyset$.
\end{definition}

Note that Euclidean spaces are Hausdorff spaces.

\begin{theorem}[Weierstrass theorem for extended-real-valued functions\index{function!extended-real-valued}~{\citep[Theorem 1.29]{BauschkeBC03}}]
\label{thm:weierstrass}
\index{Weierstrass theorem for extended-real-valued functions|textbf}
Let $\mathcal{X}$ be a Hausdorff space, let $f: \mathcal{X} \to [-\infty, +\infty]$ be lower semicontinuous, and let $\mathcal{V}$ be a compact subset of $\mathcal{X}$. Suppose $\mathcal{V}\cap \dom f \neq \emptyset$. Then, $f$ achieves its infimum over $\mathcal{V}$.
\end{theorem}

\section{Inequalities for Transcendental Functions}

\subsection{The Lambert Function}
\label{sec:lambert}
\index{Lambert function|(textbf}

The Lambert function $W:\R_{\geq 0} \to \R_{\geq 0}$ is defined by the equality
\begin{equation}
\label{eq:lambert}
x=W(x) \exp \left(W(x)\right), \quad \text{for } x \ge 0~.
\end{equation}
Hence, from the definition we have that $\exp(W(x)/2)=\sqrt{\frac{x}{W(x)}}$ for $x>0$.

We have the following property of the Lambert function, stated without proof.
\begin{theorem}
\label{thm:limit_lambert}
The Lambert function satisfies $\lim_{x \to +\infty} \frac{W(x)}{\ln x}=1$
\end{theorem}

\begin{theorem}[{\citealp[Theorem~2.3]{HoorfarH08}}]
\label{thm:lambert_upper}
\[
W(x) \leq \ln\frac{x+C}{1+\ln(C)}, \quad \forall x\geq 0, \ \forall C>\frac{1}{e}.
\]
\end{theorem}

The following theorem provides upper and lower bounds on $W(x)$.
\begin{theorem}[{\citealp[Lemma~17]{OrabonaP16}}]
\label{thm:lambert_upper_lower}
The Lambert function $W(x)$ satisfies
\[
0.6321 \ln(x+1) \leq W(x) \leq \ln(x+1), \quad \forall x \ge 0~.
\]
\end{theorem}
\begin{proof}
The inequalities are satisfied for $x=0$, hence in the following we assume $x>0$.
We first prove the lower bound. From \eqref{eq:lambert} we have
\begin{equation}
W(x) = \ln\left(\frac{x}{W(x)}\right)~. \label{eq:lm_lambert_1}
\end{equation}
From this equality, using the elementary inequality $\ln(u) \leq \frac{1}{a e} u^a$ for any $a>0$, we get
\[
W(x) \leq \frac{1}{a\, e}\left(\frac{x}{W(x)}\right)^a,  \quad \forall a>0,
\]
that is
\begin{equation}
\label{eq:lm_lambert_2}
W(x) \leq \left(\frac{1}{a\, e}\right)^\frac{1}{1+a} x^\frac{a}{1+a}, \quad \forall a>0.
\end{equation}
Using \eqref{eq:lm_lambert_2} in \eqref{eq:lm_lambert_1}, we have
\begin{align*}
W(x)
\geq \ln\left(\frac{x}{\left(\frac{1}{a\, e}\right)^\frac{1}{1+a} x^\frac{a}{1+a}}\right)
= \frac{1}{1+a}\ln\left(a \, e\, x\right), \quad \forall a>0~.
\end{align*}
Consider now the function $g(x)=\frac{x}{x+1} - \frac{b}{\ln(1+b) (b+1)}
\ln(x+1)$ defined in $[0,b]$ where $b$ is a positive number that will be decided in the following.
This function has a maximum in $x^\star=(1+\frac{1}{b}) \ln(1+b)-1$, the derivative is positive in $[0,x^\star]$ and negative in $[x^\star,b]$. Hence, the minimum is attained at $x=0$ and $x=b$, where it is equal to $0$. Using the property just proved on $g$, setting $a=\frac{1}{x}$, we have
\begin{align*}
W(x)
\geq \frac{x}{x+1} \geq \frac{b}{\ln(1+b) (b+1)} \ln(x+1), \quad  \forall x\leq b~.
\end{align*}
For $x>b$, setting $a=\frac{x+1}{e x}$, we have
\begin{align*}
W(x)
&\geq \frac{e\,x}{(e+1) x + 1} \ln(x+1) \geq \frac{e\,b}{(e+1) b + 1} \ln(x+1)~.
\end{align*}
Hence, we set $b$ such that
\[
\frac{e\, b}{(e+1)b + 1} = \frac{b}{\ln(1+b) (b+1)}~.
\]
Numerically, $b=1.71825\ldots$, so $W(x) \geq 0.6321 \ln(x+1)$.

For the upper bound, we use Theorem~\ref{thm:lambert_upper} and set $C=1$.
\end{proof}

\begin{theorem}
\label{thm:dual_exp_square}
Let $a,b>0$. Then, the Fenchel conjugate of $f:\R\to \R$ defined as $f(x)=b \exp(x^2/(2a))$ is
\begin{align*}
f^\star(\theta)
&= \sqrt{a} |\theta| \sqrt{W(a \theta^2/b^2)} - b \exp\left(\frac{W(a \theta^2/ b^2)}{2}\right)\\
&= \sqrt{a} |\theta| \left(\sqrt{W(a \theta^2/b^2)} - \frac{1}{\sqrt{W(a \theta^2/ b^2)}}\right)~.
\end{align*}
Moreover,
\[
\sqrt{a} |\theta| \left(\sqrt{0.6321\ln(a \theta^2/b^2+1)}-1\right) - b
\leq f^\star(\theta)
\leq \sqrt{a} |\theta| \sqrt{\ln(a \theta^2/b^2+1)} - b~.
\]
\end{theorem}
\begin{proof}
First, observe that
\[
\max_{x} \ \theta x - b \exp(x^2/(2a))
= \max_{y} \ b\left(\frac{\sqrt{a} \theta}{b} y - \exp(y^2/2)\right)~.
\]
Also, by the definition of the Lambert function, we have
\[
\argmax_{y} \ u y - \exp(y^2/2)
= \sign(u) \sqrt{W(u^2)},
\]
where $\sign(u)=0$ for $u=0$.
Hence, we have
\[
\argmax_{x} \ \theta x - b \exp(x^2/(2 a))
= \sign(\theta) \sqrt{a} \sqrt{W(a \theta^2/b^2)}~.
\]
So, we obtain
\begin{align*}
\max_x \ \theta x - b \exp(x^2/(2a))
&=\max_{y} \ b\left(\frac{\sqrt{a} \theta}{b} y - \exp(y^2/2)\right) \\
&= \sqrt{a} |\theta| \sqrt{W(a \theta^2/b^2)} - b \exp\left(\frac{W(a \theta^2/ b^2)}{2}\right) \\
&= \sqrt{a} |\theta| \sqrt{W(a \theta^2/b^2)} - b \sqrt{\frac{a \theta^2/ b^2}{W(a \theta^2/ b^2)}} \\
&= \sqrt{a} |\theta| \left(\sqrt{W(a \theta^2/b^2)} - \frac{1}{\sqrt{W(a \theta^2/ b^2)}}\right)~.
\end{align*}


The upper bound is obtained using Theorem~\ref{thm:lambert_upper_lower} and upper bounding $- b \exp(\frac{W(\frac{a \theta^2}{b^2})}{2})$ with $-b$.

For the lower bound, we use again Theorem~\ref{thm:lambert_upper_lower} and observe that
\begin{align*}
- b \exp\left(\frac{W(a \theta^2/ b^2)}{2}\right)
&\geq - b \exp\left(\frac{1}{2}\ln\left(a \theta^2/ b^2+1\right)\right)
= -b \sqrt{a \theta^2/ b^2+1}\\
&= -\sqrt{a \theta^2 +b^2}
\geq -\sqrt{a}|\theta| - b~. \qedhere
\end{align*}
\end{proof}
\index{Lambert function|)textbf}

\subsection{The Gamma Function}
\label{sec:gamma}

\index{gamma function|(textbf}

\begin{definition}
For $\Re (x)>0$, the \textbf{gamma function} is defined as
\[
\Gamma(x) := \int_{0}^{\infty} \! t^{x-1}e^{-t} \, \mathrm{ d}t~.
\]
\end{definition}

\begin{definition}
The \textbf{digamma function}\index{digamma function|textbf} is defined as
\[
\psi(x) := \frac{d}{d x} \ln \Gamma(x) = \frac{\Gamma'(x)}{\Gamma(x)}, \quad \forall x>0~.
\]
\end{definition}

\begin{proposition}
\label{prop:gamma}
For the gamma and digamma\index{digamma function} functions, we have the following properties:
\begin{enumerate}[(a)]
\item \label{prop:gamma_log_cvx} The gamma function is log-convex on $\R_{>0}$, i.e., $\ln \Gamma$ is convex.
\item \label{prop:gamma_wendel} Wendel's double inequality~\citep{Wendel48}\index{inequality!Wendel's|textbf}
\[
\left(\frac{x}{x+s}\right)^{1-s}\leq\frac{\Gamma(x+s)}{x^s \Gamma(x)}\leq 1, \quad \forall 0<s<1, \ \forall x>0~.
\]
\item \label{prop:gamma_gamma_bound} $\Gamma(x) \geq \left(\frac{x-1/2}{e}\right)^{x-1/2} \sqrt{2 e}, \quad x\geq1$ \citep[Theorem 1.5]{Batir08}.
\item \label{prop:gamma_incr} The digamma function\index{digamma function} is strictly increasing and strictly concave on $( 0 , \infty )$.
\item \label{prop:gamma_psi_bound} For all $x > 0$, we have $\ln x-{\frac {1}{x}}\leq \psi (x)\leq \ln x-{\frac {1}{2x}}$.
\item \label{prop:gamma_psi_prime_series} For all $x > 0$, we have $\psi'(x)
=\sum_{k=0}^\infty \frac{1}{(x+k)^2}$.
\end{enumerate}
\end{proposition}

\begin{lemma}
\label{lemma:ratio_gamma}
Let $T \ge 1$ and $d \geq 2$ be integers. Then, we have
\[
\frac{\Gamma(T+d/2)}{\Gamma(T+1/2)}
\leq (T+d/2-1/2)^{(d-1)/2}~.
\]
\end{lemma}
\begin{proof}
From the definition of the digamma\index{digamma function} function $\psi$, we have
\begin{align*}
\ln \frac{\Gamma(T+d/2-1/2)}{\Gamma(T+1/2)}
&= \int_{T+1/2}^{T+d/2-1/2} \! \psi(t) \,\mathrm{d}t
\leq \int_{T+1/2}^{T+d/2-1/2} \! \psi\left(T+d/2-1/2\right) \, \mathrm{d}t\\
&= \frac{d-2}{2} \psi\left(T+d/2-1/2\right)
\leq \frac{d-2}{2} \ln\left(T+d/2-1/2\right),
\end{align*}
where we used the fact that the digamma\index{digamma function} is increasing (Proposition~\ref{prop:gamma}(\ref{prop:gamma_incr})) and the upper bound $\psi(x)< \ln(x)$ (Proposition~\ref{prop:gamma}(\ref{prop:gamma_psi_bound})).
Then, from Proposition~\ref{prop:gamma}(\ref{prop:gamma_wendel})\index{inequality!Wendel's}, we use the fact that $\Gamma(T+d/2)\leq \sqrt{T+d/2-1/2} \, \Gamma(T+d/2-1/2)$. Hence, we have
\[
\ln \frac{\Gamma(T+d/2)}{\Gamma(T+1/2)}
\leq \frac{d-1}{2} \ln\left(T+d/2-1/2\right)~. \qedhere
\]
\end{proof}

\begin{lemma}
\label{lemma:bound_regret_gamma}
Let $T \ge 1$ and $d \geq 2$ be integers. Then, we have
\[
\ln \frac{\sqrt{\pi} \Gamma(T+d/2)}{\Gamma(d/2)\Gamma(T+1/2)}
\leq \frac{d-1}{2} \ln\left(\frac{T}{(d-1)/2}+1\right) + \frac{d}{2}-1 + \frac{1}{2}\ln \frac{\pi}{2}~.
\]
Hence, for $d=2$, the upper bound is $\frac{1}{2} \ln\left(\pi T + \frac{\pi}{2}\right)\leq \frac{1}{2} \ln T + 1$.
\end{lemma}
\begin{proof}
We use Proposition~\ref{prop:gamma}(\ref{prop:gamma_gamma_bound}):
\[
\Gamma(d/2)
\geq \left(\frac{d/2-1/2}{e}\right)^\frac{d-1}{2} \sqrt{2 e}~.
\]
Hence, using Lemma~\ref{lemma:ratio_gamma}, we have
\begin{align*}
\ln \frac{\sqrt{\pi} \Gamma\left(T+\frac{d}{2}\right)}{\Gamma\left(\frac{d}{2}\right)\Gamma\left(T+\frac12\right)}
&\leq \frac{1}{2}\ln \pi + \frac{d-1}{2}\ln\left(T+\frac{d-1}{2}\right) - \frac{d-1}{2}\ln\frac{d-1}{2e} - \frac{1}{2}\ln (2 e)\\
&= \frac{d-1}{2}\ln\left(\frac{T}{(d-1)/2}+1\right) + \frac{d}{2}-1 + \frac{1}{2}\ln \frac{\pi}{2}~. \qedhere
\end{align*}
\end{proof}

\begin{lemma}
\label{lemma:max_universal_portfolio}
Let $T\ge 1$ and let $n\geq1$ be an integer. Define $f:\R^n_{\geq 0}\to \R$ as
\[
f(a_1,\dots,a_n)=\prod_{i=1}^n \frac{a_i^{a_i}}{\Gamma\left(a_i+\frac12\right)},
\]
where $0^0=1$.
Then, $f$ is maximized over the set $\mathcal{V}=\{(a_1,\dots,a_n)\in\R^n_{\ge 0}: \sum_{i=1}^n a_i=T\}$ at the extreme points of the simplex, namely when one coordinate is equal to $T$ and all the others are equal to $0$.
\end{lemma}
\begin{proof}
Define
\[
h(x):=x\ln x-\ln\Gamma\left(x+\frac{1}{2}\right), \quad \forall x>0~.
\]
Then, $\ln f(a_1,\dots,a_n)=\sum_{i=1}^n h(a_i)$.
So, it is enough to prove that $h$ is convex on $(0,\infty)$.

We have
\[
h''(x)=\frac{1}{x}-\psi'\left(x+\frac12\right),
\]
where $\psi$ is the digamma function.

Using Proposition~\ref{prop:gamma}(\ref{prop:gamma_psi_prime_series}), we have
\[
\psi'\left(x+\frac12\right)
=\sum_{k=0}^\infty \frac{1}{\left(x+k+\frac12\right)^2}~.
\]
Now, define
\[
g_x(t):=\frac{1}{(x+t)^2}, \quad t\ge 0~.
\]
Since $g_x$ is convex on $[0,\infty)$, using Jensen's inequality\index{inequality!Jensen's} (Theorem~\ref{thm:jensen}), for every $k\ge 0$ we have
\[
g_x\left(k+\frac12\right)
= g_x\left(\int_k^{k+1} \! t \,\mathrm{d}t\right)
\le \int_k^{k+1} \! g_x(t)\,\mathrm{d}t~.
\]
Therefore,
\[
\psi'\left(x+\frac12\right)
=\sum_{k=0}^\infty g_x\left(k+\frac12\right)
\le \sum_{k=0}^\infty \int_k^{k+1} g_x(t)\,dt
=\int_0^\infty \! \frac{\mathrm{d}t}{(x+t)^2}
=\frac{1}{x}~.
\]
Hence
\[
h''(x)=\frac{1}{x}-\psi'\left(x+\frac12\right)\ge 0~.
\]
Therefore, $\ln f(a_1,\dots,a_n)$ is convex on the relative interior of $\mathcal{V}$. We can extend the function $\sum_i h(a_i)$ by continuity to all of $\mathcal{V}$.

A convex function on a simplex attains its maximum at an extreme point. This proves the claim.
\end{proof}

\index{gamma function|)textbf}

\section{Probabilistic Inequalities}

\subsection{Khintchine's inequality}

\begin{theorem}[Khintchine's inequality]
\label{thm:khintchine}
\index{inequality!Khintchine's|textbf}
Let $\{\epsilon_{n}\}_{n=1}^{N}$ be i.i.d. Rademacher random variables\index{random variable!Rademacher}, i.e.,
\[
\Pr\{\epsilon_{n} = 1\}=\Pr\{\epsilon_{n} = -1\}=\frac{1}{2}, \quad \forall n=1,\dots,N~.
\]
Let $0<p<\infty$ and let $x_{1},\dots,x_{N}\in \mathbb{C}$. Then, we have
\[
A_{p}\left(\sum _{n=1}^{N}|x_{n}|^{2}\right)^{1/2}
\leq \left(\E \left|\sum _{n=1}^{N}\epsilon_{n} x_{n}\right|^{p}\right)^{1/p}
\leq B_{p}\left(\sum _{n=1}^{N}|x_{n}|^{2}\right)^{1/2},
\]
where
\begin{align*}
A_{p}&={\begin{cases}2^{1/2-1/p}&0<p\leq p_{0},\\
2^{1/2}(\Gamma ((p+1)/2)/{\sqrt {\pi }})^{1/p}&p_{0}<p<2\\
1&2\leq p<\infty \end{cases}}\\
B_{p}&={\begin{cases}1&0<p\leq 2\\2^{1/2}(\Gamma ((p+1)/2)/{\sqrt {\pi }})^{1/p}&2<p<\infty \end{cases}},
\end{align*}
where $p_{0}\approx 1.847$ and $\Gamma$ is the gamma function.
\end{theorem}
The constants in the above theorem are due to \citet{Haagerup81}.

\subsection{Lower Bound for the Tail of the Binomial Distribution}

First, we prove a new handy result for the sum of binomial coefficients\index{binomial coefficient}.
\begin{lemma}
\label{lemma:sum_binomial}
Let $n\geq 1$ and $n/2 \leq k \leq n$ be integers. Define
\[
\phi(k,n):=\frac{1}{2}\sqrt{4k^2 - 4kn + (n+2)^2}-k-n/2-1~.
\]
Then, $\sum_{i=k}^n \binom{n}{i} \geq \binom{n}{k} \frac{k}{2k + \phi(k,n)}$. Moreover, $\phi(k,n)$ is non-increasing in the first argument.
\end{lemma}
\begin{proof}
First of all, observe that $\phi(k,n)$ is non-increasing in the first argument.
Indeed, the first derivative is
\begin{align*}
\frac{2k -n}{\sqrt{4k^2-4kn+(n+2)^2}}-1
&= \frac{2k -n - \sqrt{4k^2-4kn+(n+2)^2}}{\sqrt{4k^2-4kn+(n+2)^2}} \\
&\leq \frac{2k -n - \sqrt{4k^2-4kn+n^2}}{\sqrt{4k^2-4kn+(n+2)^2}}
= 0~.
\end{align*}

We will prove the statement by backward induction.
The base case is easy: $\sum_{i=n}^n \binom{n}{i} = 1 = \binom{n}{n} \frac{n}{2n+\phi(n,n)}$.
So, now assume that $\sum_{i=k+1}^n \binom{n}{i} \geq \binom{n}{k+1} \frac{k+1}{2(k+1) + \phi(k+1,n)}$.
We have that
\begin{align*}
\sum_{i=k}^n \binom{n}{i}
&= \sum_{i=k+1}^n \binom{n}{i} + \binom{n}{k}
\geq \binom{n}{k+1} \frac{k+1}{2(k+1) + \phi(k+1,n)} + \binom{n}{k}\\
&= \binom{n}{k}\left(\frac{n-k}{k+1}\frac{k+1}{2k+2+\phi(k+1,n)} + 1\right)\\
&= \binom{n}{k}\left(\frac{n-k}{2k+2+\phi(k+1,n)} + 1\right)
\geq \binom{n}{k}\left(\frac{n-k}{2k+2+\phi(k,n)} + 1\right)\\
&= \binom{n}{k} \frac{k}{2k + \phi(k,n)},
\end{align*}
where the last equality is due to the definition of $\phi(k,n)$.
\end{proof}

\begin{remark}
For $n$ even, this lower bound gives us $\sum_{i=n/2+q}^n \binom{n}{i} \geq \binom{n}{n/2+q} \frac{n/2+q}{q+\sqrt{q^2+n+1}-1}$. Hence, it is tight in $q=n/2$, while in $q=0$ it is
\[
\sum_{i=n/2}^n \binom{n}{i}
\geq \binom{n}{n/2} \frac{n/2}{\sqrt{n+1}-1}
\geq \binom{n}{n/2} \frac{n/2}{\sqrt{n}}
= \binom{n}{n/2} \frac{\sqrt{n}}{2},
\]
that is larger by a factor $\frac{\sqrt{n}}{2}$ than the trivial lower bound $\sum_{i=n/2}^n \binom{n}{i} \geq \binom{n}{n/2}$.
\end{remark}

\index{Kullback--Leibler divergence!between Bernoulli distributions|(}
We denote the \ac{KL} divergence between two Bernoulli distributions with parameters $p\in [0,1]$ and $q\in (0,1)$ as
\[
\KLBern(p;q) := p \ln \left( \frac{p}{q} \right) + (1-p) \ln \left( \frac{1-p}{1-q} \right),
\]
where we define $0\ln0:=0$.
\index{Kullback--Leibler divergence!between Bernoulli distributions|)}

We will also denote by $B(n,p)$ the binomial distribution with parameters $n$ and $p$.

We can now state the bound on the tail of the Binomial distribution.\index{random variable!Binomial}
\begin{lemma}[Bound on Binomial Tail]
\label{lemma:bin}
Let $n \ge 1$ be an even integer and $0 \le q \le n/2-1$ integer. Then, $X_n \sim B(n,\frac{1}{2})$ satisfies
\begin{align*}
\Pr \left\{ X_n \geq \frac{n}{2}+q \right\}
&\geq \frac{1}{3} \frac{1}{2q+\sqrt{n}}\sqrt{n\frac{n/2+q}{n/2-q}}\ \exp\left(-n \cdot \KLBern\left(\frac12+\frac{q}{n} ; \frac{1}{2}\right)\right)\\
&\geq \frac{1}{3} \frac{1}{\frac{2q}{\sqrt{n}}+1}\exp\left(-n \cdot \KLBern\left(\frac12+\frac{q}{n} ; \frac{1}{2}\right)\right)~.
\end{align*}
Note that the second lower bound also holds in $q=n/2$.
\end{lemma}
\begin{proof}
By definition, we have
\[
\Pr \left\{ X_n \geq \frac{n}{2}+q\right\}
= 2^{-n} \sum_{i=n/2+q}^n \binom{n}{i}~.
\]

From Lemma~\ref{lemma:sum_binomial}, we have
\[
\sum_{i=n/2+q}^n \binom{n}{i}
\geq \binom{n}{n/2+q} \frac{n/2+q}{n+2q + \phi(n/2+q,n)}
\geq \binom{n}{n/2+q} \frac{n/2+q}{2q+\sqrt{n}}~.
\]

We bound the binomial coefficient\index{binomial coefficient} $\binom{n}{k}$ using Stirling's formula\index{Stirling's formula|textbf} for factorials.
In particular, we use the explicit upper and lower bounds due to~\citet{Robbins55} valid for any $n\ge 1$,
\[
\sqrt{2 \pi n} \left( \frac{n}{e} \right)^n
< n!
< e^{1/12} \sqrt{2 \pi n} \left( \frac{n}{e} \right)^n ~.
\]
Hence, for any $1\le k \le n-1$, we have
\begin{align*}
\binom{n}{k}
& = \frac{n!}{k! (n-k)!} \\
& > \frac{\sqrt{2\pi n} \ n^n e^{-n}}{\sqrt{2\pi k} \ k^k e^{-k} e^{1/12} \cdot \sqrt{2\pi (n-k)} \ (n-k)^{n-k} e^{-(n-k)} e^{1/12}} \\
& = \frac{1}{e^{1/6} \sqrt{2 \pi}} \left(\frac{n}{n-k}\right)^{n-k} \left(\frac{n}{k}\right)^k \sqrt{\frac{n}{k(n-k)}} \\
& = \frac{1}{e^{1/6} \sqrt{2 \pi}} \ 2^n \exp\left(-n \cdot \KLBern\left(\frac{k}{n} ; \frac{1}{2}\right)\right) \sqrt{\frac{n}{k(n-k)}},
\end{align*}
where in the equality we used the definition of $\KLBern(p;q)$.
Hence, we have
\[
\binom{n}{n/2+q}
\geq \frac{2^n}{e^{1/6} \sqrt{2 \pi}} \exp\!\left(-n \ \KLBern\left(\frac{1}{2}+\frac{q}{n} ; \frac{1}{2}\right)\right) \sqrt{\frac{n}{(n/2+q)(n/2-q)}}.
\]
Combining all the inequalities, using the fact that $\sqrt{\frac{n/2+q}{n/2-q}}\geq 1$, and overapproximating, we get the stated bound.

For $q=n/2$, we verify the statement of the lemma by direct substitution. The l.h.s. is $\Pr[X_n \ge n] = 2^{-n}$.
Since $\exp\left(-n \cdot \KLBern(1;\frac{1}{2})\right) = 2^{-n}$, it is easy to see that the r.h.s. is smaller than $2^{-n}$.
\end{proof}

\endappendix

\bibliography{learning}

\clearpage
\phantomsection
\printindex

\end{document}